\documentclass[
12pt, % Main document font size
oneside, % One page layout (no page indentation)
layout=a5paper,paperheight=216mm,paperwidth=154mm,layoutvoffset=3mm,layouthoffset=3mm,
headinclude,footinclude, % Extra spacing for the header and footer
] {book}%

\usepackage{tocloft,lipsum,pgffor, sectsty}

\usepackage{CJKutf8}

\usepackage{eurosym}
\usepackage{booktabs} % for nice tables with \toprule, \midrule
\usepackage{csvsimple,longtable}
\usepackage[
beramono=false, % Use the Bera Mono font for monospaced text (\texttt)
eulermath=false,% Use the Euler font for mathematics
eulerchapternumbers=false,
palatino=false,
pdfspacing, % Makes use of pdftex letter spacing capabilities via the microtype package
dottedtoc % Dotted lines leading to the page numbers in the table of contents
]{classicthesis} % The layout is based on the Classic Thesis style

\usepackage[T1]{fontenc} % Use 8-bit encoding that has 256 glyphs

\usepackage[utf8]{inputenc} % Required for including letters with accents

\usepackage{graphicx} % Required for including images

\usepackage{enumitem} % Required for manipulating the whitespace between and within lists

\usepackage{lipsum} % Used for inserting dummy 'Lorem ipsum' text into the template

\usepackage{amsmath,amssymb,amsthm} % For including math equations, theorems, symbols, etc

\usepackage{varioref} % More descriptive referencing
\usepackage{color}
\usepackage{listings}
\usepackage{setspace}
\usepackage{lscape}
\usepackage[authoryear]{natbib}
\theoremstyle{definition} % Define theorem styles here based on the definition style (used for definitions and examples)
\newtheorem{definition}{Definition}

\theoremstyle{plain} % Define theorem styles here based on the plain style (used for theorems, lemmas, propositions)
\newtheorem{theorem}{Theorem}

\theoremstyle{remark} % Define theorem styles here based on the remark style (used for remarks and notes)

\hypersetup{
colorlinks=true, breaklinks=true, % bookmarks=true,
bookmarksnumbered,
urlcolor=webbrown, linkcolor=RoyalBlue, citecolor=webgreen, % Link colors
pdftitle={}, % PDF title
pdfauthor={\textcopyright}, % PDF Author
pdfsubject={}, % PDF Subject
pdfkeywords={}, % PDF Keywords
pdfcreator={pdfLaTeX}, % PDF Creator
plainpages=false,
pdfproducer={LaTeX with hyperref and ClassicThesis} % PDF producer
}

\definecolor{Code}{rgb}{0,0,0}
\definecolor{Decorators}{rgb}{0.5,0.5,0.5}
\definecolor{Numbers}{rgb}{0.5,0,0}
\definecolor{MatchingBrackets}{rgb}{0.25,0.5,0.5}
\definecolor{Keywords}{rgb}{0,0,1}
\definecolor{self}{rgb}{0,0,0}
\definecolor{Strings}{rgb}{0,0.63,0}
\definecolor{Comments}{rgb}{0,0.63,1}
\definecolor{Backquotes}{rgb}{0,0,0}
\definecolor{Classname}{rgb}{0,0,0}
\definecolor{FunctionName}{rgb}{0,0,0}
\definecolor{Operators}{rgb}{0,0,0}
\definecolor{Background}{rgb}{0.98,0.98,0.98}

\lstdefinestyle{mystyle}{
numbers=left,
numberstyle=\footnotesize,
numbersep=1em,
xleftmargin=1em,
framextopmargin=10em,
framexbottommargin=2em,
showspaces=false,
showtabs=false,
showstringspaces=false,
frame=l,
tabsize=4,
basicstyle=\ttfamily\small\setstretch{1},
backgroundcolor=\color{Background},
language=Python,
commentstyle=\color{Comments}\slshape,
stringstyle=\color{Strings},
morecomment=[s][\color{Strings}]{"""}{"""},
morecomment=[s][\color{Strings}]{'''}{'''},
morekeywords={import,from,class,def,for,while,if,is,in,elif,else,not,and,or,print,break,continue,return,True,False,None,access,as,,del,except,exec,finally,global,import,lambda,pass,print,raise,try,assert},
keywordstyle={\color{Keywords}\bfseries},
morekeywords={[2]@invariant},
keywordstyle={[2]\color{Decorators}\slshape},
emph={self},
emphstyle={\color{self}\slshape},
captionpos=b, 
}
 
\usepackage[section=section]{glossaries}
\makeglossaries

\newacronym{AI}{AI}{
    Artificial intelligence, the field that develops and studies methods that enables machines to learn and take actions in a way that imitates human intelligence
}

\newacronym{API}{API}{Application programming interface, a way for computer programs to communicate with each other through a specified interfance}

\newacronym{AUPR}{AUPR}{Area under the precision-recall curve, an evaluation metric that measures the trade-off between the precision and recall under varying decision thresholds for a binary classification problem}

\newacronym{AUROC}{AUROC}{Area under the receiver-operator characteristic, an evaluation metric that measures the trade-off between the true positive rate and false positive rate under varying decision thresholds for a binary classification problem}

\newacronym{Bert}{Bert}{Bidirectional encoder representations from transformers by \citet{devlin2019bert}, a transformer-based language trained to predict a masked-out token given some context, with predicting the order of two subsequent sentences as an auxiliary task}

\newacronym{BLEU}{BLEU}{Bilingual evaluation understudy, an evaluation metric initially proposed by \citet{papineni2002bleu} to evaluate machine translation methods based on the $n$-gram overlap between the translation and a reference}

\newacronym{COMET}{COMET}{Crosslingual optimized metric for evaluation of translation, a metric proposed by \citet{rei2020comet} that tries to predict the quality of a machine-generated translation using a neural model}

\newacronym{CP}{CP}{
    Conformal prediction, a technique orginally developed by \citet{vovk2005algorithmic} that creates prediction sets (classification) or intervals (regression) that include the correct prediction with a pre-defined probability, given that a test point is from the same distribution as the calibration data used to construct the prediction sets / intervals
}

\newacronym{DL}{DL}{
    Deep learning, the field concerned with the study of artificial neural network of increased depth.
    Can be considered a subfield of machine learning
}

\newacronym{CoT}{CoT}{
    Chain-of-through prompting, a prompting technique originally proposed by \citet{wei2022chain}, where an LLM is instructed to solve a task by performing step-by-step reasoning
}

\newacronym{APRICOT}{APRICOT}{
    Auxiliary prediction of confidence targets.
    Method to create calibrated confidence scores for black-box LLMs through an external secondary model, discussed in Chapter 6
}

\newacronym{DDU}{DDU}{Deep deterministic uncertainty transformer \citealp{mukhoti2021deterministic}, a type of transformer for which a Gaussian discriminant analysis model is fit on its latent features in order to quantify uncertainty}

\newacronym{FAISS}{FAISS}{ (Meta's) Fundamental AI similarity search, a software library proposed by \citet{johnson2019billion} to quickly find the nearest neighbors for a vector in a datastore}

\newacronym{HDBSCAN}{HDBSCAN}{
    Hierarchical density-based spatial clustering of applications with noise \citep{campello2013density}, an improvement on the earlier DBSCAN clustering algorithm \citep{ester1996density}.
    The algorithm is unsupervised, i.e.\@ does not require a specification of the number of clusters a priori, and works by merging points into cluster by distance in a bottom-up fashion
}

\newacronym{LLM}{LLM}{
    Large language model or foundation model; typically a large neural model based on the transformer architecture \citep{vaswani2017attention}, that has been trained on huge swaths of data to model the statistical distribution of words
}

\newacronym{LM}{LM}{
    Language model or language modeling (i.e.\@ the process or task of modeling the statistical distribution of words underlying language).
    More general than LLMs, since language models can also be based on recurrent or $n$-gram models
}

\newacronym{$k$-NN}{$k$-NN}{$k$-nearest neighbors, the idea to use $k$ points most similar to a point of interest for purposes such as classification or clustering}

\newacronym{LSTM}{LSTM}{Long-short term memory network \citep{hochreiter1997long}, a type of recurrent neural architecture used for sequential data}

\newacronym{MAUVE}{MAUVE}{Neural metric to assess the quality of machine-generated, general text by \citet{pillutla-etal:mauve:neurips2021}}

\newacronym{ML}{ML}{
    Machine learning, the subfield of artificial intelligence concerned with the development and study of statistical algorithms that can learn from data and generalize to unseen data
}

\newacronym{MT}{MT}{
    Machine translation, the study or task of automatically translating text or speech with the help of computers
}

\newacronym{NLG}{NLG}{
    Natural language generation, a set of tasks involving the generation of language, including language modeling, machine translation, question-answering, and image captioning
}

\newacronym{NLP}{NLP}{
    Natural language processing, an interdisciplinary subfield of artificial intelligence and linguistics, primarily concerned with providing computers the ability to process data encoded in natural language 
}

\newacronym{OOD}{OOD}{
    Out-of-distribution or out-of-domain, used to refer to test inputs to a machine learning model that are different come from a different distribution than the training data the model was originally fit on
}

\newacronym{PAC}{PAC}{Probably approximately correct learning, a framework for the mathematical analysis of machine learning}

\newacronym{ReLU}{ReLU}{
    Rectifier linear unit, a non-linear activation function defined as $\phi(x) = \max(0, x)$, often used on the activations between neural network layers
}

\newacronym{RLHF}{RLHF}{Reinforcement learning from human feedback \citep{christiano2017deep, stiennon2020learning}, a technique to optimize neural models based on human preference data}

\newacronym{ROUGE}{ROUGE}{An evaluation metric initially proposed by \citep{lin2004rouge} to text summarization methods based on the $n$-gram overlap between the summarization and a reference}

\newacronym{SDE}{SDE}{Stochastic differential equation, a type of equation regarding the derivative of some function in which one or more of the terms is a stochastic process}

\newacronym{SNGP}{SNGP}{Spectrally-normalized Gaussian process transformer \citep{liu2022simple}, a type of transformer architecture whose weights are regularized through spectral normalization and which features a Gaussian process output layer.}

\newacronym{UQ}{UQ}{
    Uncertainty Quantification; methods to assess the reliability or trustworthiness of the predictions of (in this thesis) neural models
}

\glsaddallunused

\usepackage{pdfpages}

\usepackage{imakeidx}
\indexsetup{level=\chapter*}
\makeindex[columns=2, title=Index, intoc]
\index{Gaussian distribution|see{Normal distribution}}
\index{Significance testing|see{Hypothesis testing}}
\index{Bayes' theorem|see{Bayes' rule}}
\index{Quantile function|see{Cumulative distribution function, Inverse}}
\index{Uncertainty!Model|see{Uncertainty!Epistemic}}
\index{Uncertainty!Data|see{Uncertainty!Aleatoric}}
\index{Distributional drift|see{Shift!Distributional}}
\index{Shift!Dataset|see{Shift!Distributional}}
\index{In-context example|see{In-context learning}}
\index{Language generation|see{Natural language generation}}
\index{Bert!SNGP|see{Transformer!SNGP}}
\index{Bert!DDU|see{Transformer!DDU}}
\index{Bert!Variational|see{Transformer!Variational}}
\index{Ensemble|see{Ensembling}}

\PassOptionsToPackage{utf8}{inputenc} % Set the encoding of your files. UTF-8 is the only sensible encoding nowadays. If you can't read äöüßáéçèê∂åëæƒÏ€ then change the encoding setting in your editor, not the line below. If your editor does not support utf8 use another editor!
\usepackage{inputenc}

\usepackage{tcolorbox}

\providecommand{\mLyX}{L\kern-.1667em\lower.25em\hbox{Y}\kern-.125emX\@}
\usepackage{lipsum} % Used for inserting dummy 'Lorem ipsum' text into the template

\PassOptionsToPackage{american}{babel}  % Change this to your language(s)
\usepackage{babel}

\usepackage{csquotes}
 
 \usepackage{amsmath, amssymb}
 \numberwithin{equation}{chapter}

 \usepackage{savesym}
 \usepackage{bbding}
\savesymbol{Square}
\savesymbol{Cross}
\savesymbol{TriangleUp}
\savesymbol{TriangleDown}
\usepackage[geometry]{ifsym}
\usepackage{wasysym}
\usepackage{diagbox}

\PassOptionsToPackage{T1}{fontenc} % T2A for cyrillics
\usepackage[T1]{fontenc}
\DeclareTextSymbolDefault{\ae}{T1}

\usepackage{textcomp} % Fix warning with missing font shapes

\usepackage{xspace} % To get the spacing after macros right

\usepackage{mparhack} % To get marginpar right

\usepackage{fixltx2e} % Fixes some LaTeX stuff 

\PassOptionsToPackage{smaller}{acronym} % Include printonlyused in the first bracket to only show acronyms used in the text
\usepackage{acronym} % Nice macros for handling all acronyms in the thesis

\PassOptionsToPackage{pdftex}{graphicx}
\usepackage{graphicx} 

\usepackage{tabularx} % Better tables
\usepackage{multirow}
\usepackage{arydshln}
\usepackage{float}
\usepackage{chngcntr} 
\counterwithout*{table}{chapter} 
\counterwithout*{footnote}{chapter}

\usepackage{subcaption}
\usepackage{listings} 
\usepackage[noabbrev,capitalize,nameinlink]{cleveref}
\newcounter{researchquestion}
\newenvironment{researchquestion}{\refstepcounter{researchquestion}\begin{center}\rule{\textwidth}{0.8pt}\\[0.2cm]}{\rule{\textwidth}{0.8pt}\end{center}}
\crefname{researchquestion}{\emojimangnifyingglass RQ\hspace{-0.12cm}}{\emojimangnifyingglass RQs\hspace{-0.12cm}}

\usepackage{wrapfig}

\usepackage{epigraph}
\epigraphnoindent
\usepackage{titlesec}
\titlespacing*{\chapter}{0pt}{55pt}{15pt}

\usepackage{tikz}
\usetikzlibrary{shapes,positioning,trees}
\usepackage{tikzpagenodes}
\usepackage{forest}

\definecolor{bg}{rgb}{0.96,0.96,0.96}
\definecolor{gray75}{gray}{0.75}
\newcommand{\hsp}{\hspace{20pt}}
\titleformat{\chapter}[hang]{\Huge\bfseries\raggedright}{\thechapter\hsp\textcolor{gray75}{|}\hsp}{0pt}{\Huge\bfseries}
\titleformat{\section}[hang]{\large\bfseries\raggedright}{\thesection\hsp}{0pt}{\large\bfseries}
\titleformat{\subsection}[hang]{\bfseries\raggedright}{\thesubsection\hsp}{0pt}{\large\bfseries}
\titleformat{\paragraph}[runin]{\bfseries}{\theparagraph}{0pt}{\normalfont\bfseries}

\usepackage{xcolor,pifont}
\newcommand{\gcheck}{{\color{green!55!black}\ding{52}}}
\newcommand{\rcross}{{\color{red!70!black}\ding{55}}}
\definecolor{textbgcolor}{HTML}{F5F5F5}
\definecolor{promptbgcolor}{HTML}{ecf9ec}
\definecolor{slotcolor}{HTML}{b30000}

\usepackage{tcolorbox}
\tcbuselibrary{breakable,xparse,skins}

\renewcommand{\cftchapfont}{\bfseries}
\renewcommand{\cftsecfont}{\bfseries}
\renewcommand{\cftsubsecfont}{\bfseries}

\usepackage{phaistos}
\newcommand*\coin[1]{\tikz[baseline=(char.base)]{
            \node[shape=circle,draw,minimum size=8mm, fill={rgb:orange,1;yellow,4;pink,4}, draw={rgb:orange,1;yellow,2;pink,5}, inner sep=0.4mm, line width=1mm] (char) {#1};}}

\newcommand{\heads}{\coin{\raisebox{-.2\height}{\includegraphics[decodearray={0 0 0 0 0 0}, scale=0.032]{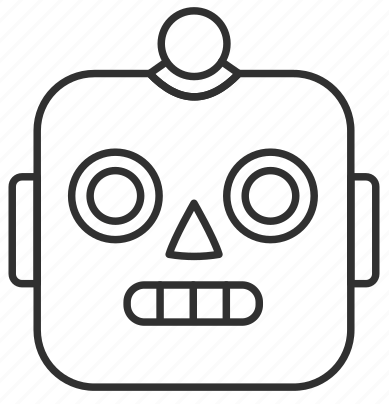}}}}
\newcommand{\tails}{\coin{\large \textbf{1}}}% \euro{1}}}

\usepackage[T1]{fontenc}
\usepackage{bm}
\usepackage{bbm}
\usepackage{amsfonts} 
\usepackage{amsthm}
\DeclareMathOperator{\Expect}{\mathbb{E}}
\DeclareMathOperator{\bx}{\mathbf{x}}

\DeclareMathOperator{\bW}{\mathbf{W}}
\DeclareMathOperator{\bw}{\mathbf{w}}
\DeclareMathOperator{\bV}{\mathbf{V}}
\DeclareMathOperator{\bb}{\mathbf{b}}
\DeclareMathOperator{\ba}{\mathbf{a}}
\DeclareMathOperator{\by}{\mathbf{y}}

\DeclareMathOperator{\bz}{\mathbf{z}}
\DeclareMathOperator{\pprime}{{\prime\prime}}
\DeclareMathOperator{\btheta}{\bm{\theta}}
\DeclareMathOperator{\bTheta}{{\bm{\Theta}}}
\DeclareMathOperator{\balpha}{\bm{\alpha}}
\DeclareMathOperator{\bDelta}{\bm{\Delta}}

\DeclareMathOperator{\bphi}{\bm{\phi}}
\DeclareMathOperator{\bPhi}{\bm{\Phi}}
\DeclareMathOperator{\blambda}{\bm{\lambda}}
\DeclareMathOperator{\bgamma}{\bm{\gamma}}
\DeclareMathOperator{\brho}{\bm{\rho}}
\DeclareMathOperator{\bpi}{\bm{\pi}}
\DeclareMathOperator{\bmu}{\bm{\mu}}
\DeclareMathOperator{\beps}{\bm{\varepsilon}}
\DeclareMathOperator{\epsmin}{\varepsilon_{\text{min}}}

\DeclareMathOperator{\T}{\hspace{-0.08cm}^\text{T}}
\newcommand\dd{\mathrm{d}\hspace{-0.01cm}}
\newcommand\ddd{\mathrm{d}\hspace{0.02cm}}
\newcommand{\bind}[1]{\bm{1}(#1)}

\usepackage{colortbl}  % Rowcolors
\usepackage{makecell}
\newcommand{\rs}[4]{\makecell[cc]{\underset{\scriptstyle \pm #2}{\large #1}\ \big/\underset{\pm \scriptstyle #4}{\large #3}}}
\newcommand{\srs}[2]{\makecell[c]{\underset{\scriptstyle \pm #2\phantom{-}}{{\large #1}\phantom{-}}}}
\newcommand{\wsrs}[2]{\makecell[c]{\underset{\scriptstyle \pm #2}{ \phantom{-}{\large #1}\phantom{-}}}}
\newcommand{\cmbold}[1]{\underline{\mathbf{#1}}}

\newcommand*\circled[1]{\tikz[baseline=(char.base)]{
            \node[shape=circle,draw,inner sep=.6pt] (char) {#1};}}

\usepackage[finalizecache,cachedir=.,draft]{minted}
\usepackage{cancel}

\usepackage{algorithm}
\usepackage{algpseudocode}
\algnewcommand{\LineComment}[1]{\State \(\triangleright\) #1}

\usepackage{dsfont}

\DeclareMathOperator*{\argmax}{\arg\!\max}
\newcommand{\indicator}[1]{\mathds{1}\big({#1}\big)}

\newtheoremstyle{mystyle}%                % Name
  {}%                                     % Space above
  {}%                                     % Space below
  {\itshape}%                                     % Body font
  {}%                                     % Indent amount
  {\bfseries}%                            % Theorem head font
  {.}%                                    % Punctuation after theorem head
  { }%                                    % Space after theorem head, ' ', or \newline
  {}%                                     % Theorem head spec (can be left empty, meaning `normal')

\theoremstyle{mystyle}
\newtheorem{example}{Example}

\newtheorem{lemma}{Lemma}
\newtheorem{proposition}{Proposition}

\newcommand*\mystrut[1]{\vrule width0pt height0pt depth#1\relax}

\usepackage{scalerel,xparse}
\NewDocumentCommand\emojiperson{}{
    \raisebox{-.1\height}{
        \includegraphics[scale=0.06]{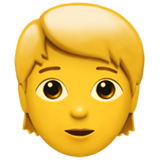}
    }
}

\NewDocumentCommand\emojirobot{}{
    \raisebox{-.1\height}{
        \includegraphics[scale=0.06]{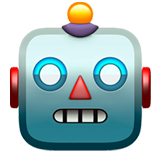}
    }
}

\NewDocumentCommand\emojimangnifyingglass{}{
    \raisebox{-.1\height}{
        \includegraphics[scale=0.06]{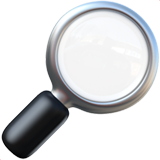}\hspace{-1.6mm}
    }
}

\newcommand{\peach}{\raisebox{-0.03cm}{\includegraphics[width=0.35cm]{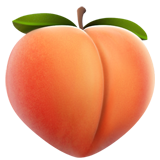}}}

\usepackage[dvipsnames]{xcolor}

\newcommand{\trustor}{\textcolor{Bittersweet}{trustor}\hspace{-1.5mm}\emojiperson\hspace{-1.6mm}}
\newcommand{\trustee}{\textcolor{MidnightBlue}{trustee}\hspace{-1.5mm}\emojirobot\hspace{-1.6mm}}

\PassOptionsToPackage{pdftex,hyperfootnotes=false,pdfpagelabels,pagebackref=true,colorlinks=true}{hyperref}
\usepackage{hyperref} 
\usepackage{backref} %, pagebackref} % , linktocpage} % pagebackref}
 
\renewcommand*{\backref}[1]{}
\renewcommand*{\backrefalt}[4]{%
    \ifcase #1 (Not cited.)%
    \or        (Cited on page~#2)%
    \else      (Cited on pages~#2)%
    \fi}

\usepackage{url}

\usepackage{breakurl}

\usepackage{lipsum}

\newcommand*\justify{%
  \fontdimen2\font=0.4em% interword space
  \fontdimen3\font=0.2em% interword stretch
  \fontdimen4\font=0.1em% interword shrink
  \fontdimen7\font=0.1em% extra space
  \hyphenchar\font=`\-% allowing hyphenation
}

\hypersetup{
colorlinks=true, linktocpage=true, pdfstartpage=3, pdfstartview=FitV,
breaklinks=true, pdfpagemode=UseNone, pageanchor=true, pdfpagemode=UseOutlines,%
plainpages=false, bookmarksnumbered, bookmarksopen=true, bookmarksopenlevel=1,%
hypertexnames=true, pdfhighlight=/O,%nesting=true,%frenchlinks,%
urlcolor=webbrown, linkcolor=RoyalBlue, citecolor=webgreen, %pagecolor=RoyalBlue,%
pdftitle={Uncertainty in Natural Language Processing},
pdfauthor={Dennis Ulmer, IT University of Copenhagen},
pdfsubject={},
pdfkeywords={},
pdfcreator={pdfLaTeX},
pdfproducer={LaTeX with hyperref and classicthesis}
}

\title{\normalfont\spacedallcaps{title}} % The article title

\author{\spacedlowsmallcaps{author}} % The article author(s) - author affiliations need to be specified in the AUTHOR AFFILIATIONS block

\date{} % An optional date to appear under the author(s)

\begin{document}

\includepdf[pages=-, fitpaper=true]{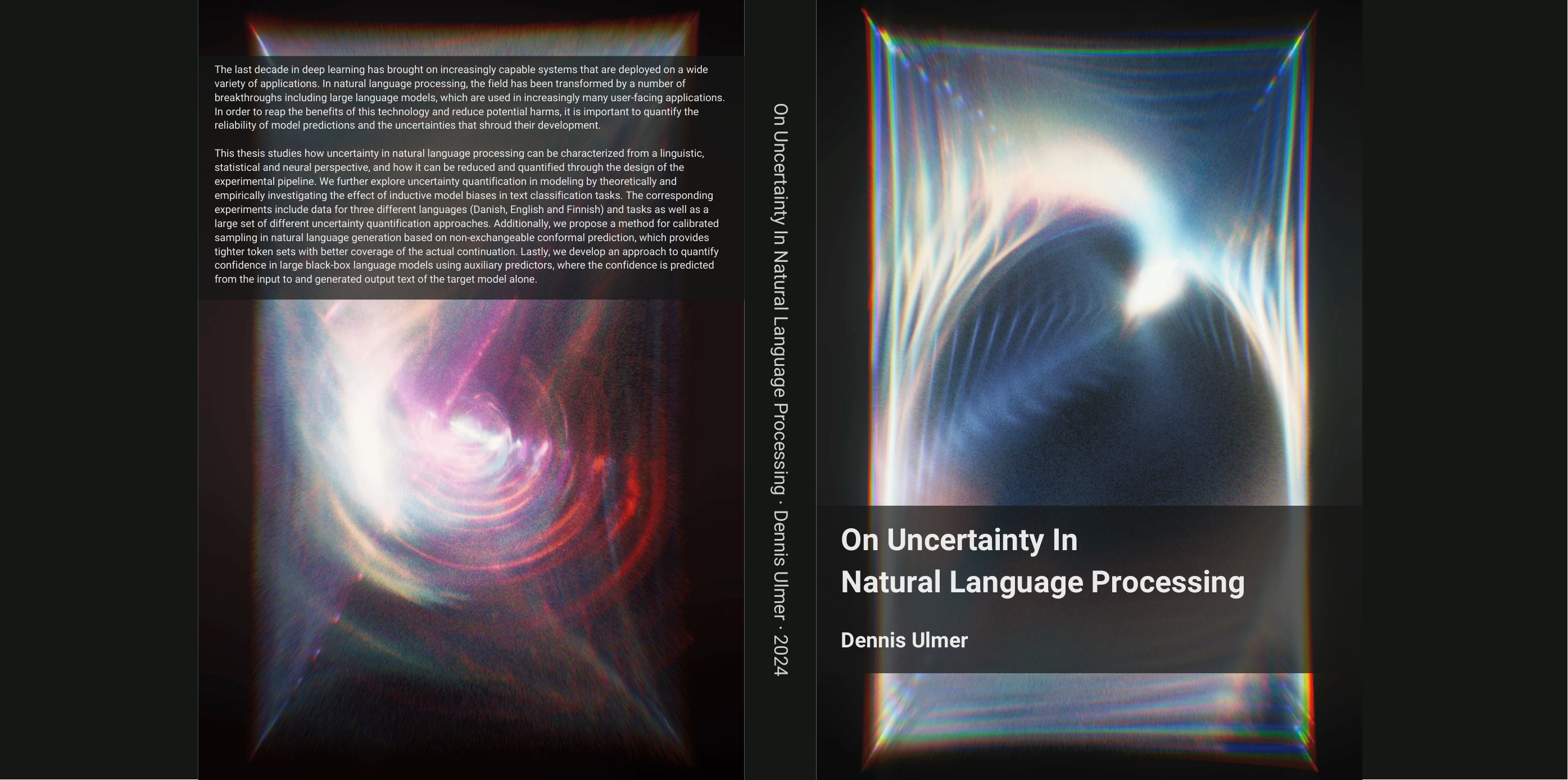}
\newpage

\raggedbottom  % Restrict vertical stretching
\sloppy  % Prevent overflow into margins

%----------------------------------------------------------------------------------------
%	HEADERS
%----------------------------------------------------------------------------------------

%\renewcommand{\chaptermark}[1]{\markright{\spacedlowsmallcaps{#1}}} % The header for all pages (oneside) or for even pages (twoside)
%\renewcommand{\subsectionmark}[1]{\markright{\thesubsection~#1}} % Uncomment when using the twoside option - this modifies the header on odd pages
%\lehead{\mbox{\llap{\small\thepage\kern1em\color{halfgray} \vline}\color{halfgray}\hspace{0.5em}\rightmark\hfil}} % The header style
% \lehead{\mbox{\llap{\small\thepage\kern1em\color{halfgray} \vline}\color{halfgray}\hspace{0.5em}\rightmark\hfil}} % The header style

\pagestyle{scrheadings} % Enable the headers specified in this block

% Contiously number tables

%----------------------------------------------------------------------------------------
%	TITLE PAGE
%----------------------------------------------------------------------------------------

\hypersetup{pageanchor=false}
\begin{titlepage}
\thispagestyle{empty}
\begin{figure}[h!] %  figure placement: here, top, bottom, or page
\centering
\includegraphics[width=4in]{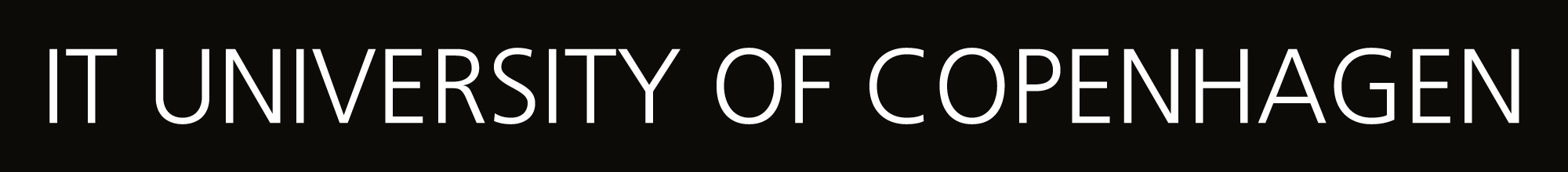} 
\end{figure}

\begin{center}
\vspace{1.25cm}
\begingroup \linespread{1,75} \selectfont 
\textsc{\LARGE On Uncertainty In Natural Language Processing}\\
% \textsc{\Large And this is an optional subtitle}
[3cm]
\endgroup

\textsc{DENNIS ULMER}\\[0.4cm]
%NLPnorth Group\\[0.1cm]
Department of Computer Science\\[0.1cm]
IT University of Copenhagen\\[3cm]

\includegraphics[width=1.25in]{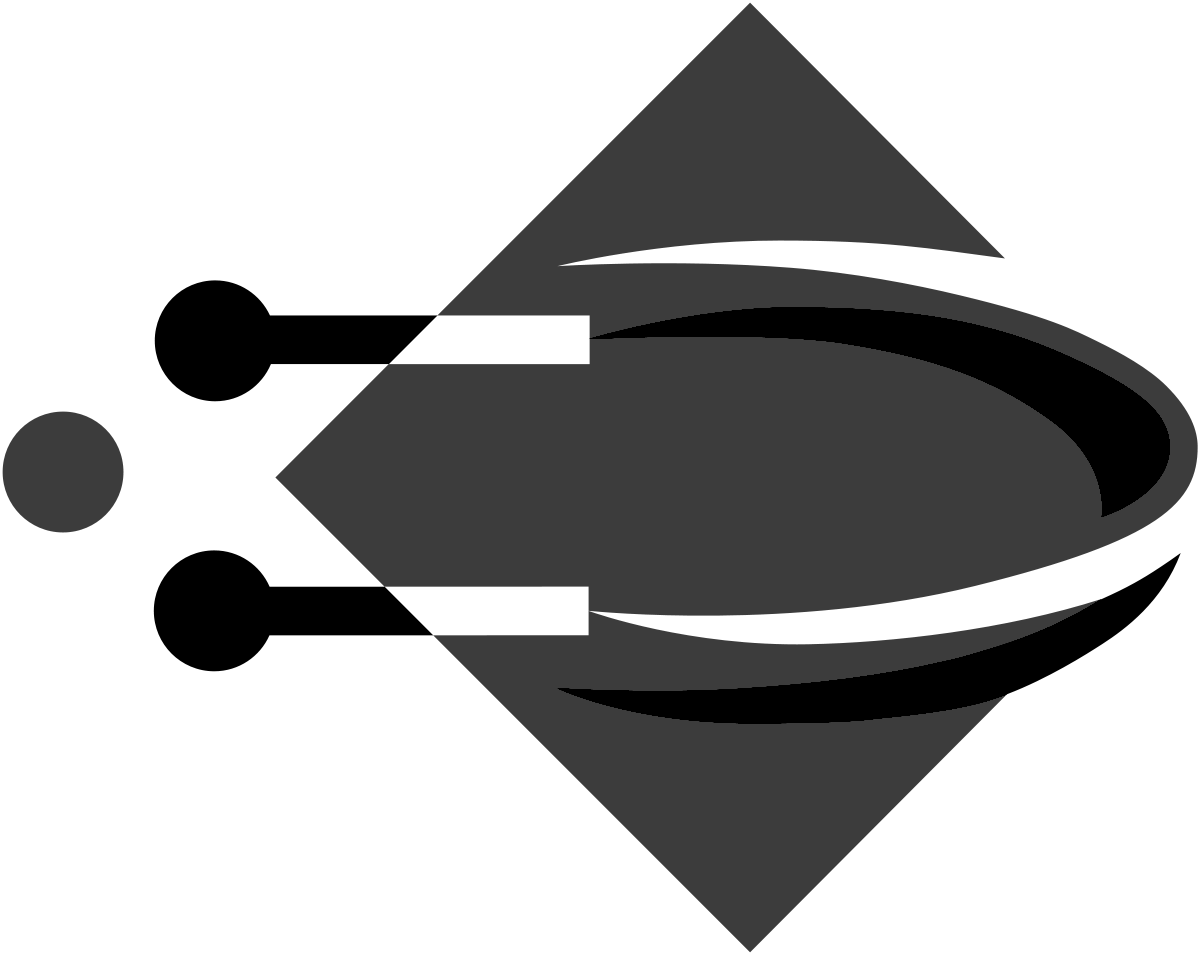}
\end{center}

\vfill
\begin{center}
Dissertation Submitted in Partial Fulfillment of the\\
Requirements for the Degree of\\
Doctor of Philosophy\\[0.5cm]  %\textbf{Master thesis}\\

June 14, 2024%\today
\end{center}
%Data Science\\  
%AUTHOR(S)\\

%\vspace{1cm}

%\begin{flushleft}
%\textbf{Supervisors}\\[0.2cm]
%Prof.\@ Dr.\@  Christian Hardmeier (IT University of Copenhagen)\\[0.1cm]
%Prof.\@  Dr.\@  Jes Frellsen (Technical University of Denmark) 
%\end{flushleft}

\end{titlepage}
\hypersetup{pageanchor=true}

%----------------------------------------------------------------------------------------
%	COMMITTEE
%----------------------------------------------------------------------------------------

\thispagestyle{empty}

\chapter*{Committee}

\begin{figure}[ht!]
    \centering

    \vspace{0.5cm}

    \resizebox{\textwidth}{!}{
    \begin{tabular}{lll}
        \textbf{Advisor} & Dr.\@ Christian Hardmeier & IT Universitetet i København \\[0.4cm]
        \textbf{Co-Advisor} & Dr.\@ Jes Frellsen & Danmarks Tekniske Universitet \\[0.4cm]
        \textbf{Members}  & Dr.\@ Leon Derczynski & IT Universitet i København \\[0.1cm]
        & Prof.\@ Dr.\@ Ole Winther & Danmarks Tekniske Universitet \\[0.1cm]
        & Prof.\@ Dr.\@ Mário A.\@ T.\@ Figueiredo & Instituto Superior Técnico
    \end{tabular}%
    }

    \vfill
\end{figure}

\vfill
\newpage

%----------------------------------------------------------------------------------------
%	ABSTRACT
%----------------------------------------------------------------------------------------

% Abstract

\pagenumbering{roman}
\chapter*{Abstract}
\markboth{Abstract}{Abstract}
\addcontentsline{toc}{chapter}{Abstract}

The last decade in deep learning has brought on increasingly capable systems that are deployed on a wide variety of applications.
In natural language processing, the field has been transformed by a number of breakthroughs including large language models, which are used in increasingly many user-facing applications.
In order to reap the benefits of this technology and reduce potential harms, it is important to quantify the reliability of model predictions and the uncertainties that shroud their development.\\

This thesis studies how uncertainty in natural language processing can be characterized from a linguistic, statistical and neural perspective, and how it can be reduced and quantified through the design of the experimental pipeline.
We further explore uncertainty quantification in modeling by theoretically and empirically investigating the effect of inductive model biases in text classification tasks.
The corresponding experiments include data for three different languages (Danish, English and Finnish) and tasks as well as a large set of different uncertainty quantification approaches.
Additionally, we propose a method for calibrated sampling in natural language generation based on non-exchangeable conformal prediction, which provides tighter token sets with better coverage of the actual continuation.
Lastly, we develop an approach to quantify confidence in large black-box language models using auxiliary predictors, where the confidence is predicted from the input to and generated output text of the target model alone.

% Danish abstract

\chapter*{Resum\'{e}}
\markboth{Resum\'{e}}{Resum\'{e}}
\addcontentsline{toc}{chapter}{Resum\'{e}}

Det sidste årti i deep learning har medført stadig mere dygtige systemer, der anvendes på mange forskellige områder. 
Feltet natural language processing (naturlig sprogbehandling) er blevet transformeret af en række gennembrud, herunder store sprogmodeller, som bruges i stadigt flere anvendelser med menneskelige brugere.
For at udnytte fordelene ved denne teknologi og reducere potentielle skader, er det vigtigt at kvantificere pålideligheden af modelforudsigelser og de usikkerheder, der omkranser deres udvikling.\\

Dette afhandling undersøger, hvordan usikkerhed i natural language processing kan karakteriseres ud fra et sprogligt, statistisk og neuralt perspektiv, og hvordan den kan reduceres og kvantificeres gennem design af den eksperimentelle pipeline.
Vi udforsker yderligere kvantificering af usikkerhed i modellering ved teoretisk og empirisk at undersøge effekten af modellers induktive bias i tekstklassificeringsopgaver.
De tilsvarende eksperimenter omfatter data for tre forskellige sprog (dansk, engelsk og finsk) og opgaver samt et stort sæt forskellige tilgange til kvatificering af usikkerheder.
Derudover foreslår vi en metode til kalibreret sampling i naturlig sproggenerering baseret på non-exchangeable conformal prediction, der giver smallere tokensæt med bedre dækning af den faktiske fortsættelse.
Til sidst udvikler vi en tilgang til at kvantificere tillid i store black-box sprogmodeller ved hjælp af såkaldte hjælpeprædiktorer, hvor tilliden forudsiges ud fra input til og genereret outputtekst fra sprogmodellen alene.

%----------------------------------------------------------------------------------------
%	Preface
%----------------------------------------------------------------------------------------

\chapter*{Acknowledgements}

\markboth{Acknowledgements}{Acknowledgements}
\addcontentsline{toc}{chapter}{Acknowledgements}

%Acknowledgements will be added in the print version of the thesis.

\epigraph{

    ``\emph{There's a ruinous misconception that a PhD must be smart.}\\[0.5cm]

    \emph{This can't be true.}\\[0.5cm]
    
    \emph{A smart person would know better than to get a PhD.}''
}
{---Matt Might}

Surely almost every PhD, across disciplinary boundaries, can attest the uniqueness of this degree.
You seize being a student despite continuing to study. 
However, any pre-defined curricula cease to exist and there is an almost overwhelming number of paths forward.
The idea of working on an open, unsolved problem is as exciting as it is terrifying---it might not known whether a solution even exists, and a million possible solutions are waiting to be explored.
Success depends on whether an idea actually proves to work in practice, and even if one is lucky enough to experience this feeling, publication is not guaranteed. 
Even a great paper might be reject in the reviewing lottery or drowned out in the current deluge of research.
Almost by design, these and other factors create an emotional rollercoaster that is easy to glorify by the time one writes the acknowledgements of one's thesis, but can be a formidable test of resilience in the moment.
I am deeply grateful to all the people that have allowed me to come this far and supported me along this way, and this section is dedicated to them.\\

Firstly, I would like to thank Natalie Schluter for making it possible for me to start my PhD journey.
I highly appreciate you taking a chance on me and giving me full flexibility.
Secondly, I owe a lot of gratitude to my supervisors Christian Hardmeier and Jes Frellsen.
You were incredibly supportive and gave me the freedom to explore any direction I desired, and immensely enjoyed collaborating with you.
I also would like to thank the member of my PhD committee, namely Leon Derczynski from ITU and NVIDIA, Ole Winther from DTU, and M\'{a}rio A.\@ T.\@ Figueiredo from IST in Lisbon.
During the last three and a half years, I also had the pleasure to collaborate with and learn from a number of outstanding people, namely
Giovanni Cinà at Pacmed, André F.\@ T.\@ Martins at the Instituto Superior Técnico in Lisbon, Elman Mansimov at Amazon Web Services, and Seong Joon Oh at Parameter Lab.
A special and repeated thanks goes to Dieuwke Hupkes, who together with my other former master thesis co-supervisor Elia Bruni were instrumental in putting me on this path in the first place.
In addition, I want to thank all the other collaborators across several projects during my PhD that I was lucky enough to work with and learn from.\\

I could not have done this either without the unwavering support from my partner Bea;
you have been with me through the lows and you were there to celebrated the highs, and I will never be able to accurately thank you for your presence and patience.
I also owe a lot to my family for their support over the course of many years of university education without which I would have never made it this far; 
thank you to my parents Michaela and Siegfried, my sister Anna and my grandma Ingrid for your encouragement (Oma, I am sorry that I still haven't gotten a regular, stable job).
A large amount of gratitude is owed to all the friends in and outside of academia for their their support and role in making this an overall unforgettable time.
Thank you António, Atilla, Ben, Christina, Chryssa, Constanza, Daniel, David, Doro, Elisa, Ellie, Hannes, Isra, Jon, Joris, Joscha, Karolina, Kristoffer, Laura, Mareike, Marija, Mark, Martin, Max, Mike, Nuno, Putri, Rita, Rui, Sean, Yova.
This thank you is generally extended to all the members of the NLPnorth at ITU, CoAStal at KU, SARDINE at IST and Jes' group at DTU. 
In addition, a thank you Daniel, Max, and Mike as well as Yiyi and Laura for helping with the organization of Beers with NLPeers.
The list above to is impossible to make comprehensive, so should your name not be in it, please be assured that I value and appreciate your part in this chapter of my life.\\

Lastly, I want to thank Lotti from ITU's high performance computing cluster for her swift help with any technical problems, and the members of ITU's IT department, facility management, HR, finance, front desk and PhD school for all other problems small and large that arose over the course of my PhD.
I am also grateful to Ricardo B.\@ Ponce (\href{https://www.instagram.com/pixel.flux/}{@pixel.flux} on Instagram) for letting me license his artwork for this thesis.
Last but not least, thanks to the members of Café Analog and producers of Club Mate for providing much needed morale and caffeine boosts throughout my PhD.

\newpage

%----------------------------------------------------------------------------------------
%	Declaration of Work 
%----------------------------------------------------------------------------------------

\chapter*{Declaration of Work}
\markboth{Declaration of Work}{Declaration of Work}
\addcontentsline{toc}{chapter}{Declaration of Work}

\vspace{0.5cm}

I, Dennis Ulmer, declare that this thesis---submitted in partial fulfillment of the requirements conferral of a PhD from the IT University of Copenhagen---is solely my own work unless otherwise referenced or attributed.
Neither the thesis nor its content have been submitted (or published) for qualifications at another academic institution.\\

\begin{flushleft}
    \includegraphics[width=0.3\textwidth]{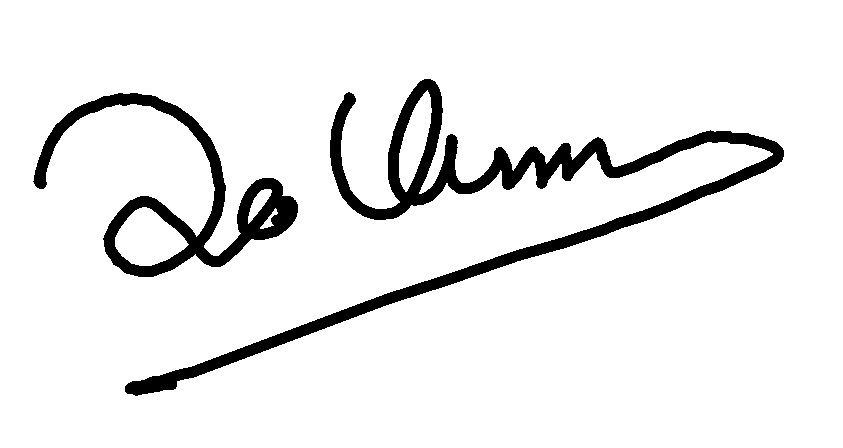}
\end{flushleft}
---Dennis Ulmer

\newpage

%----------------------------------------------------------------------------------------
%	TABLE OF CONTENTS & LISTS OF FIGURES AND TABLES
%----------------------------------------------------------------------------------------
\clearpage
\setcounter{tocdepth}{3} % Set the depth of the table of contents to show sections and subsections only
\counterwithin{figure}{chapter}
\counterwithin{table}{chapter}
%\renewcommand{\familydefault}{\rmdefault}
%\renewcommand{\cftpartfont}{\normalfont}% \part font in ToC
%\renewcommand{\cftchapfont}{\normalfont}    % \chapter font in ToC
%\renewcommand{\cftsecfont}{\normalfont}           % \section font in ToC
%\renewcommand{\cftsubsecfont}{\normalfont}        % \subsection font in ToC
%\renewcommand{\cftsubsubsecfont}{\normalfont}   

% Font changes to ToC content of sectional units
%\renewcommand{\partfont}{\normalfont}
%\renewcommand{\chapterfont}{\normalfont}
%\renewcommand{\sectionfont}{\normalfont}

\renewcommand{\partfont}{\normalfont}
\renewcommand{\cftchapfont}{\fontfamily{lmodern}\selectfont\normalfont}   
\renewcommand{\cftsecfont}{\fontfamily{lmodern}\selectfont\normalfont}  
\renewcommand{\cftsubsecfont}{\normalfont}

\tableofcontents % Print the table of contents

\listoffigures % Print the list of figures (optional, only if you have many figures)

\listoftables % Print the list of tables (optional, only if you have many tables)

%\lstlistoflistings

% Glossary and notations

\chapter*{Notation}\label{ch:notation}
\markboth{Notation}{Notation}
\addcontentsline{toc}{chapter}{Notation}

\begin{tikzpicture}[remember picture,overlay]
    \node[anchor=north,inner sep=0pt] at (current page text area.north) {\includegraphics[width=\linewidth, clip=true, trim = 8cm 75cm 8cm 50cm]{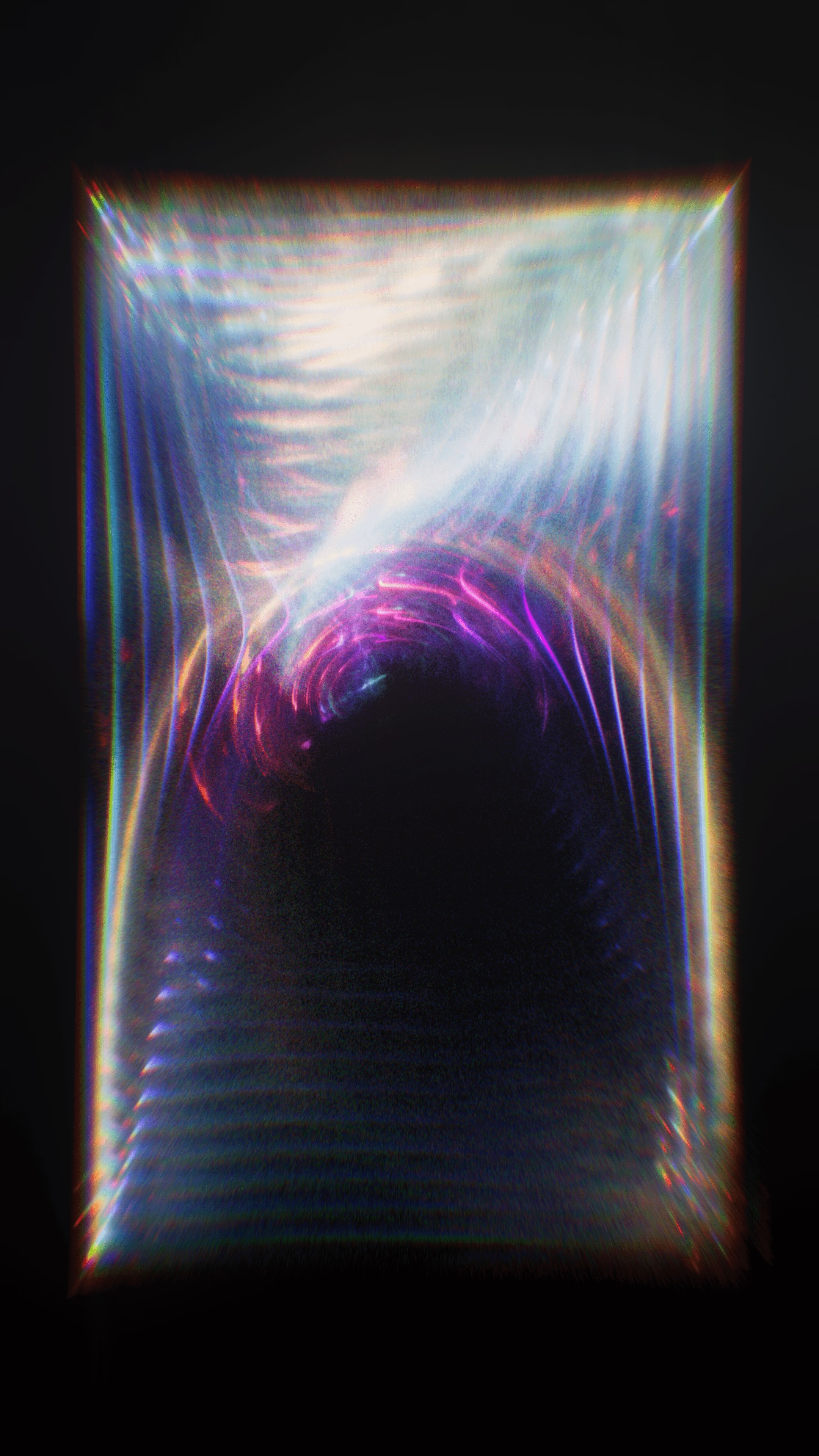}};
\end{tikzpicture}

\epigraph{``\emph{It is time for mathematicians, physicists, and computer scientists to forget their differences and admit that nobody really has a clue about what's going on in high dimensions.}''}{---Cl\'ement Canonne}

In the following, we generally follow the notational guidelines used in the book by \citet{goodfellow2016deep} and by other organizations such as the Transactions on Machine Learning Research (TMLR) journal, with some modifications.
These include the use of the following:

\begin{itemize}
    \item Lowercase latin and greek letters for scalars, e.g.\@ $a, b, c$ and $\alpha, \beta, \gamma$.
    \item Bold lowercase latin and greek letters for vectors, e.g.\@ $\mathbf{a}, \mathbf{b}, \mathbf{c}$ and $\bm{\alpha}, \bm{\beta}, \bm{\gamma}$.
    \item Bold uppercase latin and greek letters for matrices, e.g.\@ $\mathbf{A}, \mathbf{W}$ and $\bm{\Theta}, \bm{\Psi}$.
    \item Uppercase letters such as $\mathbb{A}$ and $\mathbb{D}$ to denote sets. Sometimes, calligraphical letters like $\mathcal{C}$ might be used to denote sets when the notation might conflict with common conventions (e.g.\@ $\mathbb{C}$ usually denoting the set of complex numbers.).
    \item $\{x_i\}_{i=1}^N$ to denote a set of elements $\{x_1, \ldots, x_N\}$. We also use the condensed shorthand $\{x_{ij}\}_{i,j=1}^{M,N}$ to denote a set of elements $\{x_{1,1}, \ldots, x_{M, 1}, \ldots x_{M, N}\}$ indexed along two dimensions.
    \item $[K]$ to denote an set $\{1, 2, \ldots, K\}$, or more formally, for any $K \in \mathbb{N}^{+}$, $[K] = \{n \mid n \in \mathbb{N}^{+} \text{ and } n \le K\}$.

\end{itemize}

We denote an element-wise multiplication for vectors and matrices by $\circ$, and the same symbol may be used in some contexts to denote function compositions, i.e.\@ $(f \circ g)(x) = g(f(x))$.

\section*{Definitions}\label{sec:definitions}

\paragraph{Neural Network.}\index{Neural network} Some concepts occur often enough to warrant their separate definitions.
Since this thesis revolves around neural networks, we denote $\btheta$ as the (flattened) vector of network parameters and $\mathbf{\Theta}$ as the space of all possible weight parameters.
Neural networks usually comprise a number of linear layers, consisting of a weight matrix $\mathbf{W}$ and a bias term $\mathbf{b}$, transforming inputs $\bx$ into hidden encodings $\bz$.
A superscript or index might be added to indicate one of these objects belonging to a specific layer $l \in [L]$ or to a specific time step $t \in [T]$.
Furthermore, we indicate with a index $\btheta$ when a function is parameterized by $\btheta$ (or some other set of parameters).
This is generally done to reduce clutter and make equations more readable, but might be made explicit with conditioning when it is important in a statistical context.
For instance, the probability distribution over classes $k \in [K]$ of a neural classifier will be denoted as $p_{\btheta}(y \mid \bx) \equiv p(y \mid \bx, \btheta)$.
In the same fashion, we denote $f_{\btheta}(\bx)$ as the logits, i.e.\@ the unnormalized output of a neural classifier, and use $f_{\btheta}(\bx)_k$ to refer to the $k$-th logit.

\paragraph{Neural Network Functions.} 
There also exist several functions that play a specific role in the neural network context.
On of these is the sigmoid function\index{Sigmoid function} defined as 

\begin{equation}
    \sigma(x) = \frac{1}{1 + \exp(-x)},
\end{equation}

\noindent as well as its multivariate generalization, the softmax function\index{Softmax function}:

\begin{equation}\label{eq:softmax}
    \text{softmax}(\bx)_k \equiv \bar{\sigma}(\bx)_k = \frac{\exp(x_k)}{\sum_{k=1}^K \exp{x_k}},
\end{equation}

\noindent where we sometimes will use the notation $\bar{\sigma}(\cdot)$ to avoid visual clutter.

\paragraph{Indicator function.} 
The indicator function\index{Indicator function} takes as input some condition, and evaluates as 

\begin{equation}
    \indicator{\text{condition}} = \begin{cases} 1\quad\text{if condition is true} \\ 0\quad\text{else}\end{cases}
\end{equation}

%For instance, we can check whether a singe number $n$ is larger than some value $\tau$ by $\indicator{n > \tau}$
In some cases it is useful to apply the indicator function element-wise to the contents of a vector. 
In that case, we use a bolded version, namely $\bind{\cdot}$, which will be a vector of the same dimensionality.
Take the example of a vector $\bx$ whose elements are compared against a threshold $\tau$. 
Then

\begin{equation}
    \bind{\bx > \tau}_i = \begin{cases} 1\quad\text{if }x_i > \tau \\ 0\quad\text{else}\end{cases}
\end{equation}

\paragraph{Statistics.}
In the context of statistics, we use the Dirac delta function\index{Dirac delta function}, which is defined as $0$ everywhere except for the origin, where it is $+\infty$:

\begin{equation}
    \delta(x) = \begin{cases} +\infty\quad\text{if }x = 0 \\ 0\quad\text{else}\end{cases}
\end{equation}

In addition, its integral over the entire real number line is $1$.
Another set of definitions denotes common statistical concepts as the expectation of a random variable $x$

\begin{equation}
    \mathbb{E}[x] = \sum_x P(x)x \quad \text{ or } \quad \int_x p(x)\ddd x.
\end{equation}

In this case, we also use $P$ to denote probability mass functions and $p$ to denote probability density functions.
From this, we can also define the variance as 

\begin{equation}
    \text{Var}[x] = \mathbb{E}\big[(x - \mathbb{E}[x] )^2 \big]
\end{equation}

\noindent as well as the Shannon entropy\index{Shannon entropy}

\begin{equation}
    \text{H}[x] = -\sum_x P(x) \log P (x)  \quad \text{ or } \quad -\int p(x) \log p(x) \ddd x.
\end{equation}

\paragraph{Special Functions.} Another set of definitions is dedicated to some mathematical functions, including 
the Gamma function $\Gamma(\cdot)$\index{Gamma function}, which is a continuous version of the factorial and defined as 

\begin{equation}
    \Gamma(z) = \int_0^\infty t^{z - 1}\exp(-t) \ddd t. 
\end{equation}

Another important function is the Beta function\index{Beta function}:

\begin{equation}
    \textbf{B}(\alpha_1, \alpha_2) = \frac{\Gamma(\alpha_1)\Gamma(\alpha_2)}{\Gamma(\alpha_1 + \alpha_2)}.
\end{equation}

In some cases, we will consider the Beta function with an arbitrary number of $\alpha$ values.
In that case, we collect them in a vector $\balpha = [\alpha_1, \ldots, \alpha_K]^T$ and write the Beta function as  

\begin{equation}
    \text{B}(\balpha) = \frac{\prod_{k=1}^K\Gamma(\alpha_k)}{\Gamma\Big(\sum_{k=1}^K \alpha_k\Big)}.
\end{equation}

%----------------------------------------------------------------------------------------
% Chapters
%----------------------------------------------------------------------------------------

\clearpage
\pagenumbering{arabic}
% Chapter 1

\chapter{Introduction} % Chapter title
\label{ch:introduction} % For referencing the chapter elsewhere, use \autoref{ch:introduction} 

\begin{tikzpicture}[remember picture,overlay]
    \node[anchor=north,inner sep=0pt] at (current page text area.north) {\includegraphics[width=\linewidth, clip=true, trim = 8cm 75cm 8cm 50cm]{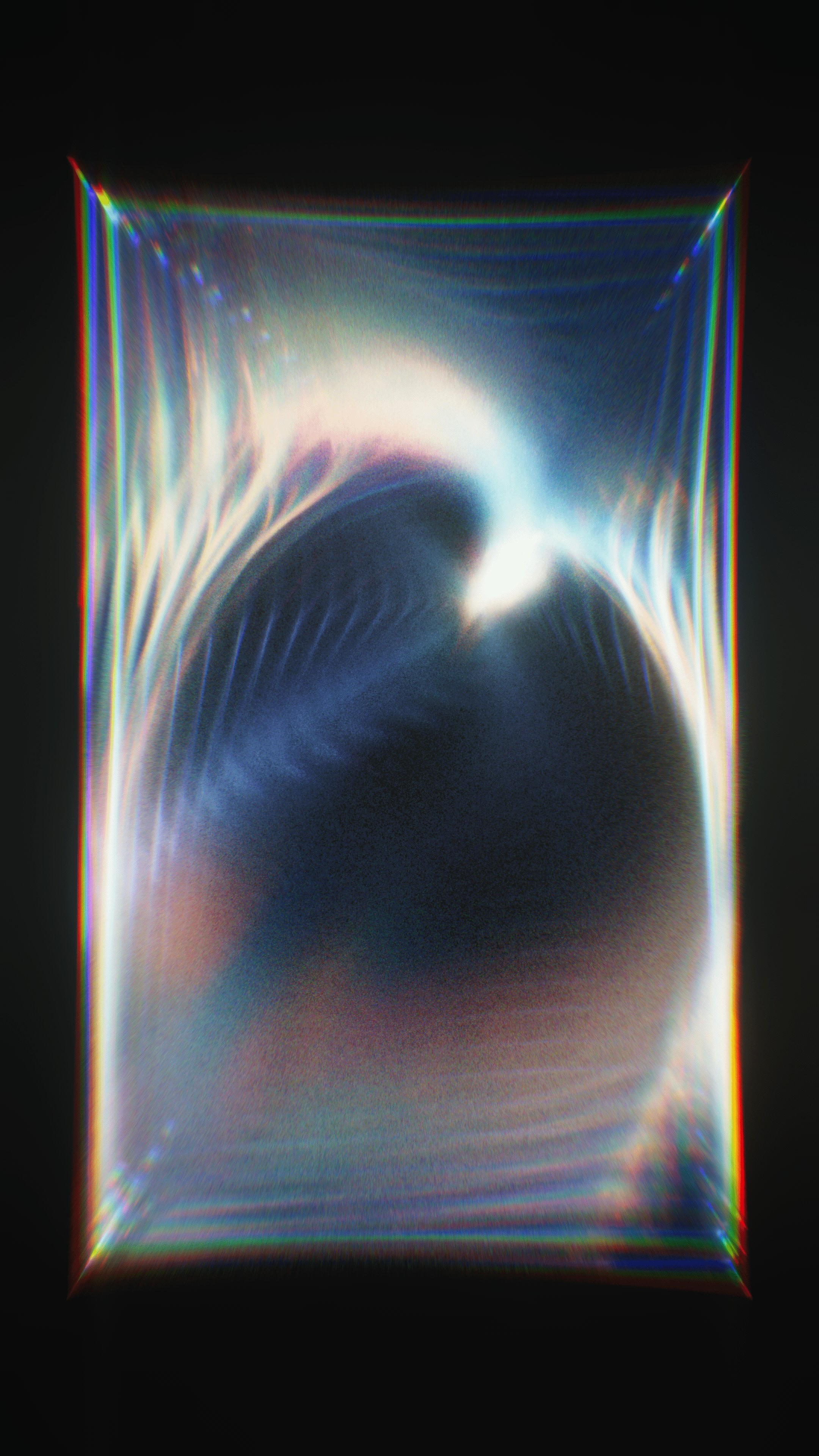}};
\end{tikzpicture}

\epigraph{

    ``\emph{Forudsigelse er meget vanskelig, is{\ae}r om fremtiden.}''\\[0.2cm]
    ``\emph{Prediction is very difficult, especially about the future.}''
}
{---Niels Bohr}

\section{Motivation}

Every person's life is full of decisions. 
Is this restaurant really as good as the reviews suggest? Should I take a job here or take a more interesting job in a city far away? 
These decisions can be hard to evaluate, since not all necessary information is known beforehand:
Restaurant reviews might be fraudulent or biased, and a promising job opportunity might turn out to be different than advertised.
Compare that with the example of making a move in a game of chess: 
Chess is called a game with \emph{perfect} information, so all the positions and possible moves of the pieces on the board are known, and one could in theory make the optimal move at every step (assuming good chess-playing abilities).
However, in real life we often do not have all the information necessary to make a perfect decision.
As such, humans take into account the uncertainty that permeates their decision-making in order to manage risk.\\

In this way, machines are (or should be) no different. 
The decades-old research in machine learning (ML)\index{Machine learning}---and especially the most recent advances in the last decade or so---have produced systems that make decisions from the mundane (``is this a picture of a cat or an airplane?'') to the potentially risky (``what treatment should be recommended to this patient?'').
This trend has been accelerated by the paradigm of \emph{deep learning}\index{Deep learning}, which allows us to build evermore complex systems that could solve increasingly complex tasks.
The complexity of these systems through comes at the cost of losing a detailed understanding of all the ``cogs and gears'' involved due to  the sheer size of models (including millions, billions and sometimes even trillions of such ``gears'').
This fact has spurred numerous lines of research to develop methods to make deep learning systems more robust, fair and safe.\\

One such line of research is concerned with \emph{uncertainty quantification}, i.e.\@ reflecting the degree of trustworthiness of a prediction. 
In systems with automatic decision-making, such scores can for instance be used to withhold a prediction or request human oversight. 
One popular example is autonomous driving: Consider an important traffic sign that cannot be accurately evaluated by the onboard computer, or a traffic situation that is hard to analyze.
In these cases, a human driver might appreciate the opportunity to intervene with the car, e.g.\@ by reducing its speed in the face of uncertainty, instead of the car sticking to a wrong assessment and endangering the driver's or other traffic users' lives.\\

At this point, the reader might be rightfully wonder whether such high-risk scenarios also exist for language applications. 
And indeed, such problems can arise in sometimes more, sometimes less obvious places.
An intuitive application with these considerations is healthcare: 
More and more work has recently gone into building artificial intelligence (AI)\index{Artificial intelligence} systems that provide decision-support for medical staff.
For instance, models could analyze text written by a user to detect signs of mental illness or triage (i.e.,\@ prioritize) patients when resources are limited \citep{cohan2016triaging, rozova2022detection, stewart2022applications}.
In this case, uncertainty can serve as a signal to request an additional human review of a case.
Confident but wrong predictions here can lead to a waste of resources, a loss of trust of the medical professionals in the system, and, in the worst case, leaving urgent cases untreated.
As another example, natural language systems are also used to assist in legal deliberations \citep{chalkidis2019neural, martinez2023survey, chalkidis2023chatgpt}.
While the scenario of a ``robo judge'' is usually ruled out, there still remain risks where models used for legal discovery or research might overlook relevant or produce misleading or incorrect outputs.\\
%In contrast, Machine Translation appears like a rather innocuous application. 
%But in 2017, Facebook's own translation system mistranslated the post of a Palestinian construction worker:
%While he posted a picture of himself standing next to a construction vehicle titled ``Good morning'', the post was translated to ``Attack them'' in Hebrew.
%That tipped Israeli security forces off, and the man was apprehended for multiple hours until he was released later.\footnote{See \url{https://gizmodo.com/palestinian-man-arrested-after-facebook-auto-translates-1819782902} (last accessed on 14.03.2023).}
%One could speculate that this translation could have been due to a lack of context and a lack of high-quality data for Palestinian Arabic.
%As such, the model could have signalled its uncertainty about the correct translation, and potentially warranted oversight from a human translator.\\

Uncertainty quantification\index{Uncertainty quantification} is an active research area for systems that operate for instance on images or tabular data, but it has only recently started to receive attention in the natural language processing (NLP)\index{Natural language processing} community.
This thesis gives an introduction to uncertainty quantification in machine learning and natural language processing for novices, summarizes the current state of progress in the field, and presents some novel and relevant methods for some of the most pressing problems for automated languages processing:
These include for instance determining the most viable methods in text classification\index{Classification!Text} and proposing new approaches to calibrated sampling for natural language generation\index{Natural language generation}, as well as confidence estimation for black-box models.

\section{Applications}\label{sec:applications}

%Even though the previous section has outlined some possible application of uncertainty quantification, there are many more.
%I briefly will describe some of them in here in order to emphasize the relevance of this thesis and the research contained in it.
A lot of research on uncertainty quantification makes only superficial statements or tacit assumptions about its usefulness.
The following, non-exhaustive list of aspects therefore underline potential practical use-cases.

\paragraph{Safety.} 
%The example of the automated translation system from the beginning is a stark example of the concrete safety implications that an automated system can have on an individual.
In general, uncertainty estimates can improve safety whenever a system with automated decision capabilities could potentially have real-world effects. 
Some of these situations are studied in the AI safety literature (see e.g.\@ \citealp{amodei2016concrete}): 
They can include preventing an intelligent agent from exploring unsafe options, or acting in a risky manner as its environment changes from the version it was trained with, which is often
referred to as \emph{distributional shift} \citep{shimodaira2000improving, moreno2012unifying}\index{Shift!Distributional}. 
In these cases, uncertain options can either be outright rejected or decisions can be delegated to a human user.

\paragraph{Trust.} In order to reap the benefits of automation and the ability to extract intricate patterns from large amounts of data, users have to trust the system's output, or otherwise
run the danger of being mislead.
In the worst case, they might grow to ignore or even antagonize an automatic system. Since our systems are inanimate---and often inscrutable---building trust between humans and machines can be a tricky endeavor.
Nevertheless, there exists a notion of trust\index{Trust} that can be built by consistency (i.e.,\@ knowing what to expect from a system) and by using uncertainty to understand the behavior of a model \citep{jacovi2021formalizing}. 
We dedicate \cref{sec:uncertainty-trust} to discuss this connection in more detail. 

\paragraph{Fairness.}\index{Fairness} A long line of works has demonstrated how modern deep learning systems have a tendency to discriminate against subpopulations in the dataset and how to mitigate
these effects (see \citealp{caton2020fairness, mehrabi2021survey} for an overview). 
Additional studies have argued that this is the result of human biases in the machine learning pipeline \citep{waseem2021disembodied} as well as biases 
and underrepresentation of groups in the training dataset \citep{meng2022interpretability}. 
In the latter case, specific uncertainty quantification methods can indicate whenever the correct prediction is uncertain due to a lack of similar training data (see \cref{sec:bayesian-neural-networks,sec:evidential-neural-networks}). 
In other instances, \emph{un}fairness might occur when models favor a prediction corresponding to a majority group in the dataset in the face of an inherently ambiguous input.
Consider the example of machine translation\index{Machine translation} system that is supposed to translate ``\emph{the doctor is here}'' into Spanish. 
In English, we do not have to specify the gender of \emph{doctor}, while this is necessary in Spanish. 
And thus, without any additional context, two translations are equally plausible (``\emph{el doctor est\'{a} aqui}'' versus ``\emph{la doctora est\'{a} aqui}'').
Deep learning systems have an inclination to prefer the version that has appeared more often in the training data, which due to real-world human biases might be \emph{el doctor} \citep{vanmassenhove2019getting}.
By exposing the inherent uncertainty however, we can delegate a series of decisions to the user or other specialized systems and avoid such pitfalls.

\paragraph{Efficiency.}\index{Computational efficiency} Not all inputs a deep learning system faces are equally difficult. 
Imagine a system that has been trained to distinguish images of lions and tigers. 
Upon receiving an picture of a lion similar to its training instances, we would expect a well-trained model to come to a confident (and correct) prediction.
Many of our contemporary deep learning systems have grown to include from millions up to billions and sometimes trillions of learnable parameters, 
and thus incur considerable computational cost for every prediction. 
Therefore, some works have explored whether we can use notions of uncertainty to detect when a model has arrived at a secure prediction in order to skip unnecessary computations (i.e.\@ \citealp{schuster2021consistent, schuster2022confident}).
Conversely, consider that our fictional lion vs.\@ tiger detector is faced with a liger, or an albino tiger displaying differently-colored fur.\footnote{A liger is a tiger / lion hybrid, see \url{https://en.wikipedia.org/wiki/Liger}.}
In light of these difficult examples, we could use uncertainty to trigger additional computations to come to a conclusion (see an example for such a mechanism for machine translation by \citealp{vanderpoel2022mutual}).
There is evidence that the human brain operators in a similar fashion, for instance when the reading time in human subjects increases when confronted with a surprising sentence structure \citep{ferreira1991recovery}.

\paragraph{Interpretability.}\index{Interpretability} Due to the scale of modern architectures, the mechanisms in which they arrive at a prediction can be opaque and hard to deduce for humans.
Here, research also has produced a variety of approaches to tackle this problem (see for instance \citealp{madsen2023posthoc} for a non-exhaustive selection). 
Uncertainty can be used as an additional angle to understand when the model might behave unreasonably confident or uncertain, with some studies 
already conducted for natural language generation \index{Natural language generation}\citep{ott2018analyzing, xu2020understanding, xiao2021hallucination, chen2022explaining}.\\

Despite the variety of useful applications, there are a number of challenges to UQ that are very common or even unique in NLP, and distinguish this line of research from similar works on images or tabular data.

\section{Challenges in Natural Language Processing}\label{sec:context}

\begin{figure*}[htb]
    \centering
    \begin{subfigure}[t]{0.49\textwidth}
        \centering
        \includegraphics[width=\textwidth]{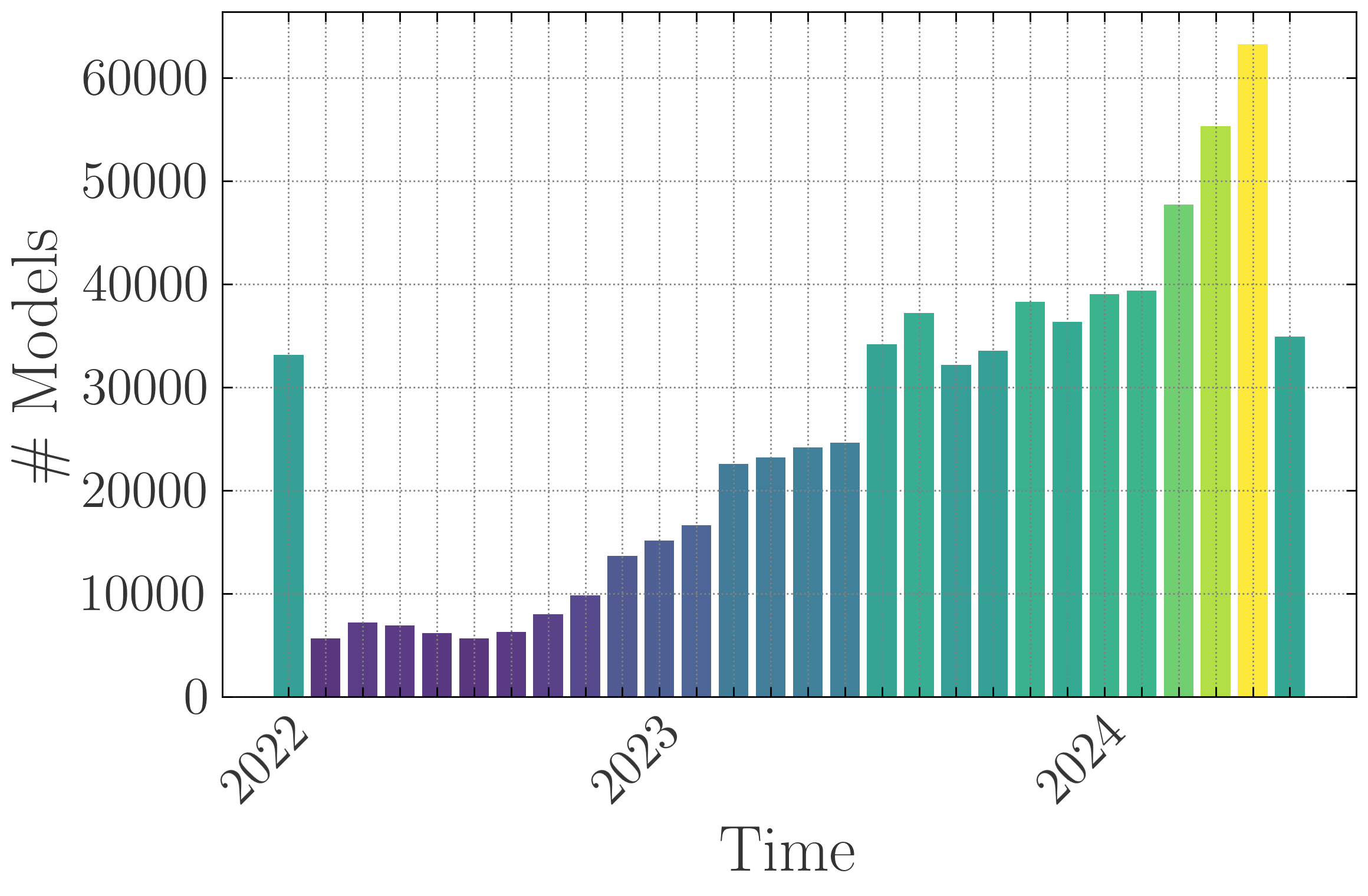}
        \caption{Number of models published on the Huggingface Hub.}
        \label{subfig:huggingface-hub}
    \end{subfigure}
    \hfill
    \begin{subfigure}[t]{0.49\textwidth}
        \centering
        \includegraphics[width=\textwidth]{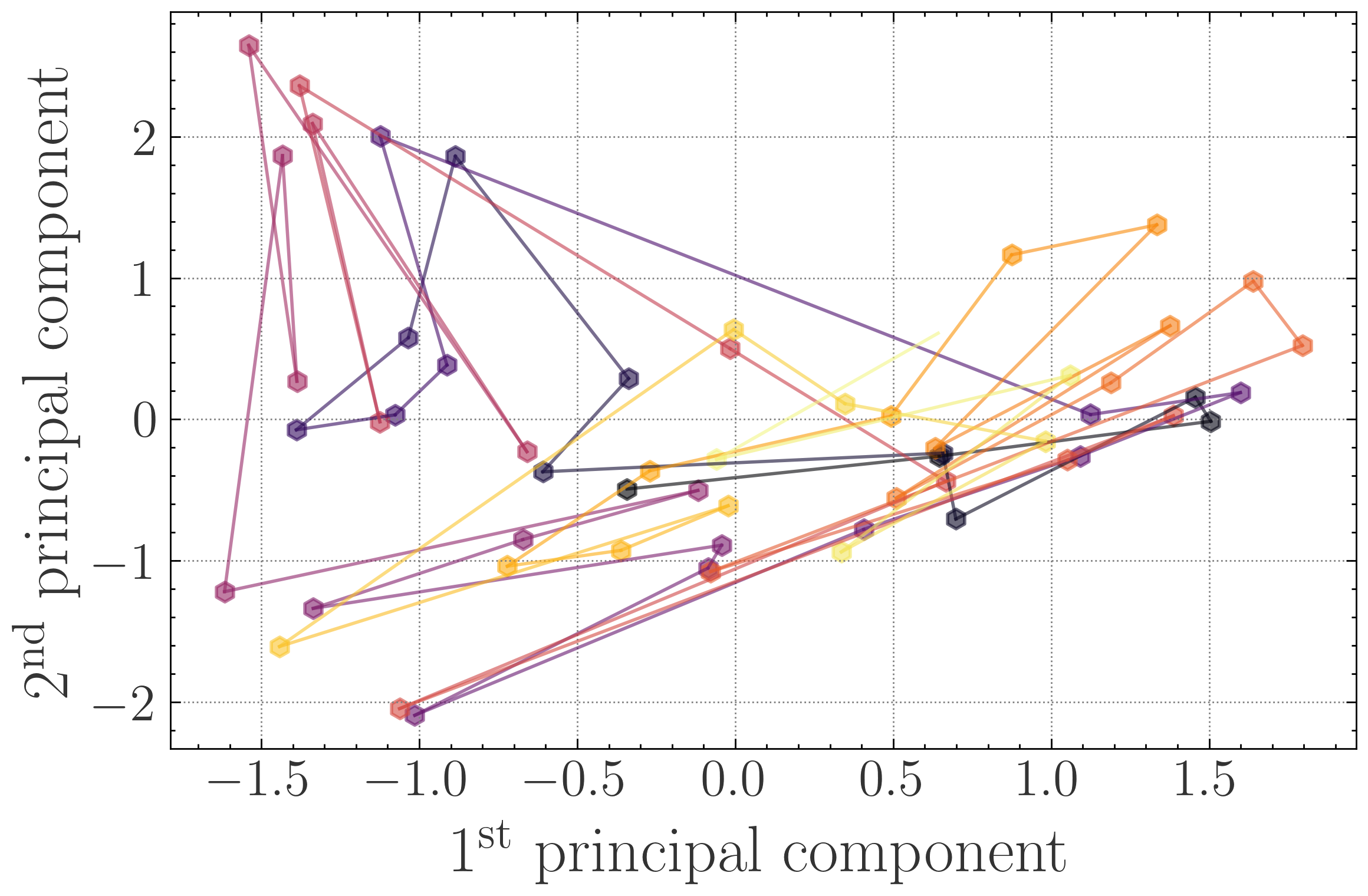}
        \caption{Bert latent representations for a sentence.}
        \label{subfig:bert-latents}
    \end{subfigure}
    \\[0.5cm]
    \begin{subfigure}[t]{0.995\textwidth}
        \centering
        \includegraphics[width=\textwidth]{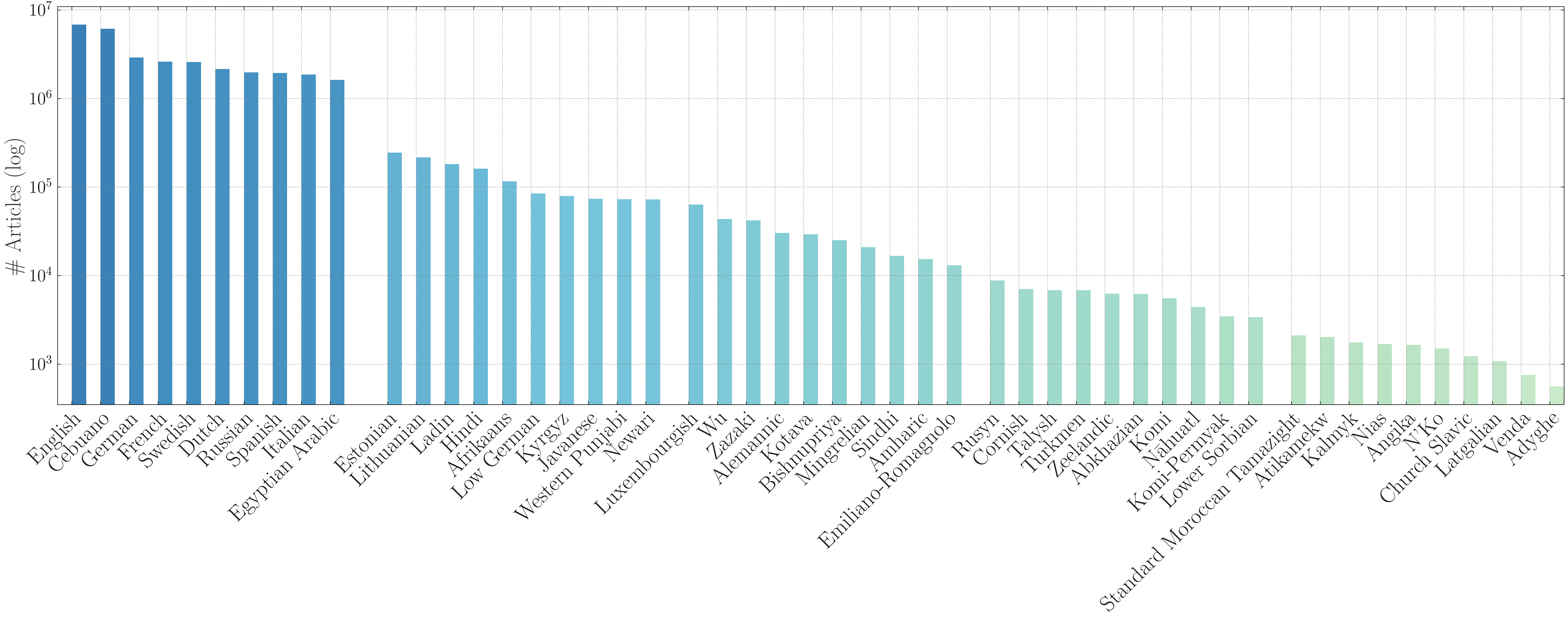}
        \caption{Number of Wikipedia articles by language (log-scale).}
        \label{subfig:wikipedia-articles}
    \end{subfigure}
    \caption[Number of published open-source models, Bert latent trajectory, size of different Wikipedias.]{
        (a) Number of models published per month on the \href{https://huggingface.co/}{HuggingFace Hub} (gathered on 17.06.2024).
        (b) Trajectory of the first two sentences of \citet{turing1950computing} using in the latent space of the uppermost layer of Bert \citep{devlin2019bert}, after projecting them into two-dimensional space using PCA and whitening them.
        Time is indicated by color, reaching from dark (first token) to light (last token).
        (c) Number of articles by Wikipedia, log-scale (gathered on 11.04.2024). Shown are the top ten languages, and then ten randomly chosen languages of the remaining four quantiles of the distribution, each.
        All figures are best viewed in color and digitally.
    }
\end{figure*}

The research in this thesis aims to fill a literature gap:
While uncertainty quantification is a vibrant research field in machine learning, the availability of methods for natural language data is limited,
and very few works develop solutions for this purpose specifically. 
This is disconcerting for the following reasons:

\paragraph{Challenges of Natural Language.}
In contrast to other machine learning problems, processing language is a rather messy affair.
First of all, language is incredibly diverse, displaying vast differences between languages, dialects, demographics, domains or even individual speakers \citep{bender2011achieving, plank2016what, zampieri2020natural, van2022writing}.
It is secondly embedded in a social and cultural context that is often necessary to understand its meaning \citep{hershcovich2022challenges}, and due to its paraphrastic nature, the idea same idea can often be expressed in a multitude of ways \citep{baan2023uncertainty}.  
Thirdly, the sequential nature often breaks the i.i.d.\@ assumption that is a fundamental underlying assumption for many algorithms.
One might assume that language data could just be treated as a time series\index{Time series} and apply corresponding methods for uncertainty quantification (see e.g.\@ \citealp{zhu2017deep, wang2020deeppipe, blasco2024survey}).
Unfortunately though, encodings of language usually behave very erratically, as \cref{subfig:bert-latents} demonstrates.
Modeling techniques for time series however subtly assume a certain behavior of the underlying data, e.g.\@ a limit in the allowed rate of change encountered between two time steps.\footnote{This trait can for instance be formalized through the Lipschitz\index{Lipschitz} constant of the true data-generating function \citep{qu2022data}.}
The sometimes abrupt token-level changes encountered during language processing therefore prevent the application of time series modeling techniques. 

\paragraph{Data Scarcity.}\index{Data scarcity} Large amounts of both unstructured and annotated data exist for English, 
but most of the world's 7000+ languages are not blessed with such resources \citep{ruder2020beyondenglish, joshi2020state}. 
\cref{subfig:wikipedia-articles} shows the number of articles of a variety of Wikipedias on a log-scale.
Due to its openness, Wikipedia remains a popular source of training data in NLP, however high-resource languages like English and German provide exponentially more potential training data compared to languages such as Afrikaans, Amharic or N'Ko.\footnote{
    Cebuano, the second-most spoken language in the Philippines, features the second-largest Wikipedia due to a bot called \href{https://ceb.wikipedia.org/wiki/Gumagamit:Lsjbot}{Lsjbot}, which tries to create Cebuano Wikipedia articles for all living creatures. This makes around $99.6 \%$ of its articles bot-generated \citep{wiki:statistics}.}
This runs contrary to the strength of modern deep learning architectures: 
Weak architectural inductive biases\index{Inductive bias} such as in transformers enable us to learn complex meaning representations, but only when enough data is supplied \citep{tay2022scaling}. 
In the case of low-resource languages\index{Low-resource language} for instances, such data is often not available, and thus we can end up with a model that is \emph{underspecified} \citep{d2020underspecification}:
The fewer training data points are available, the more possible models are able to fit them.
While this might not lead to any problems on inputs similar to the training data, models might behave unpredictably on out-of-distribution data\index{Out-of-distribution data} in ways that might not be immediately detectable by a user.\\

\paragraph{Trust \& Safety.}
In machine learning, much of the research on uncertainty quantification is motivated by concerns regarding the trust in and the safety of automation. 
Despite this, similar research has until recently mostly remained nascent in NLP, despite being equally as relevant.
Furthermore, the rapid developments in NLP with respect to large language models \citep{kalyan2021ammus, sevilla2022compute}\index{Large language model} have accelerated their adoption for a variety of applications by end users, albeit 
without appropriate techniques to ensure their safety.
This trend is illustrated by \cref{subfig:huggingface-hub}, showing the number of models by month uploaded to the \href{https://huggingface.co/}{HuggingFace Hub}, a platform to share open-source models.
After a drop in submissions after its initial release in 2022, numbers have steadily increased to around $60k$ models in mid 2024.
This, alongside a number of available proprietary models that can accessed through web interfaces and APIs, lowers the barrier of access to language models.\\

%For these reasons, this thesis tries to document existing efforts in this direction and contributes to developing novel methods that take these challenges into account.

\section{Objectives}\label{sec:objective}

This thesis analyzes the current state of uncertainty quantification research in deep learning and connects it to the methodological challenges that arise when they are applied to language data.
As such, it aims to familiarize the reader with the most popular strategies for uncertainty quantification, as well as giving an intuition about their limitations.
%Afterwards, I summarize the state-of-the-art for uncertainty quantification in classification settings in NLP in \cref{ch:uncertainty-classification} and provide a new method based on a recent class of models group under the term of \emph{Evidential Deep Learning}.
%One common task in NLP that does not exist for other data modalities is Natural Language Generation, for instance producing a translation for a sentence or a summary for a document. 
%I discuss the currently available methods in \cref{ch:uncertainty-generation} and propose a new method based on the principle of \emph{Conformal Prediction}.
%Lastly, generation is the usual method for \emph{Large Language Models} to produce outputs, which is why I dedicate a separate chapter to this topic in \cref{ch:uncertainty-llms}.\\
In this thesis, we seek to answer the following research questions:

\begin{enumerate}
    \item[\emojimangnifyingglass] \textbf{RQ1}: How can uncertainty in NLP\index{Natural language processing} be characterized?\\[0.2cm]
        \noindent
        Uncertainty can be a somewhat vague concept, and its definition is often passed over in different research works.
        Therefore, this thesis tries to gain a multi-disciplinary perspective on the matter, investigating different perspectives on the concept and how they are related.

    \item[\emojimangnifyingglass] \textbf{RQ2}: How can choices in experimental design help to reduce and quantify uncertainty?\\[0.2cm]
        \noindent
        Another overlooked factor in empirical research in NLP is the role of experimental design\index{Experimental design}. 
        Specifically, this work investigates how more conscious design decisions can not only help to reduce and quantify uncertainty, but also open new ways to model it.

    \item[\emojimangnifyingglass] \textbf{RQ3}: How do inductive model biases influence uncertainty quantification?\\[0.2cm]
        \noindent
        The \emph{inductive biases} of a model usually refers to a set of implicit or explicit assumptions made by a learning algorithm \citep{Hüllermeier2013}.
        For instance, linear regression assumes that the target variable can be recovered as a linear combination of its input variables, and neural networks have an inductive bias\index{Inductive bias} to non-linear higher-order combinations of input features.
        The behavior of uncertainty in neural models is a priori often idealized (e.g.\@ the model always displaying uncertainty on OOD), with this expectation sometimes being unfulfilled in practice.
        The reasons for this however are not so well understood, and thus this thesis sheds some light on the interaction of model biases and uncertainty.

    \item[\emojimangnifyingglass] \textbf{RQ4}: How can we address some of the challenges of uncertainty quantification in NLP?\\[0.2cm]
        \noindent
        \cref{sec:context} has listed some of the unique hurdles that UQ on NLP produces.
        Therefore, this thesis puts forth some new insights along with methodological advances to tackle these challenges. 

\end{enumerate}

\section{Publications}\label{sec:publications}

The following works were produced during the PhD and are discussed in detail in this thesis (ordered chronologically).
In all cases, the author's contributions amount to the main or complete share of the conception, implementation and description of ideas and experiments and the writing of the resulting publications, unless shared authorship is indicated.

\renewcommand{\thefootnote}{\fnsymbol{footnote}}
\begin{enumerate}
    \item \textbf{Ulmer, Dennis}$^*$, and Giovanni Cinà\footnote[1]{Equal Contribution.}. ``Know Your Limits: Uncertainty Estimation with ReLU Classifiers Fails at Reliable OOD Detection.'' In: Uncertainty in Artificial Intelligence. PMLR (2021) (discussed in \cref{sec:uq-classification-pitfalls}).
    \item \textbf{Ulmer, Dennis}, Christian Hardmeier, and Jes Frellsen. ``deep-significance-Easy and Meaningful Statistical Significance Testing in the Age of Neural Networks.'' In: The Workshop on Machine Learning Evaluation Standards at ICLR (2022) (discussed in \cref{sec:hypothesis-testing}).
    \item \textbf{Ulmer, Dennis}, Elisa Bassignana, Max Müller-Eberstein, Daniel Varab, Mike Zhang, Rob van der Goot, Christian Hardmeier, and Barbara Plank. ``Experimental Standards for Deep Learning in Natural Language Processing Research.'' In: Findings of the Association for Computational Linguistics: EMNLP (2022), pp.\@ 2673--2692 (discussed in \cref{sec:reproducibility-replicability}).
    \item \textbf{Ulmer, Dennis}, Jes Frellsen, and Christian Hardmeier. ``Exploring Predictive Uncertainty and Calibration in NLP: A Study on the Impact of Method \& Data Scarcity.'' In: Findings of the Association for Computational Linguistics: EMNLP (2022), pp.\@ 2707–-2735 (discussed in \cref{sec:benchmarking-nlp-uncertainty}).
    \item \textbf{Ulmer, Dennis}, Christian Hardmeier, and Jes Frellsen. ``Prior and Posterior
    Networks: A Survey on Evidential Deep Learning Methods for Uncertainty Estimation.''
    In: Transactions on Machine Learning Research. JMLR (2023) (discussed in \cref{sec:evidential-neural-networks}).
    \item \textbf{Ulmer, Dennis}, Chrysoula Zerva, André F.\@ T.\@ Martins: ``Non-Exchangeable Conformal Language Generation with Nearest Neighbors''. In: Findings of the Association for Computational Linguistics: EACL (2024), pp.\@ 1909--1929 (discussed in \cref{ch:uncertainty-generation}).
    \item \textbf{Ulmer, Dennis}, Martin Gubri, Hwaran Lee, Sangdoo Yun, Seong Joon Oh. ``Calibrating Large Language Models Using Their Generations Only''. In: Proceedings of the 62nd Annual Meeting of the Association for Computational Linguistics (Volume 1: Long Papers) (discussed in \cref{ch:uncertainty-llms}).
\end{enumerate}

The following works were produced during the PhD, but will not be discussed in detail, either since the author was not the main author, or because they were not a good fit for the topic of this thesis:

\begin{enumerate}
    \setcounter{enumi}{7}
    \item Baan, Joris$^*$, Nico Daheim$^*$, Evgenia Ilia$^*$, \textbf{Dennis Ulmer}$^*$, Haau-Sing Li, Raquel Fernández, Barbara Plank, Rico Sennrich, Chrysoula Zerva, Wilker Aziz. ``Uncertainty in Natural Language Generation: From Theory to Applications.'' Under review, 2024.
    \item Hupkes, Dieuwke, Mario Giulianelli, Verna Dankers, Mikel Artetxe, Yanai Elazar, Tiago Pimentel, Christos Christodoulopoulos, Karim Lasri, Naomi Saphra, Arabella Sinclair, \textbf{Dennis Ulmer}, Florian Schottmann, Khuyagbaatar Batsuren, Kaiser Sun, Koustuv Sinha, Leila Khalatbari, Rita Frieske, Ryan Cotterell, Zhijing Jin: ``A Taxonomy and Review of Generalization Research in NLP''. In: Nature Machine Intelligence 5 (10), p.\@ 1161--1174.
    \item Farinhas, Ant\'{o}nio, Chrysoula Zerva, \textbf{Dennis Ulmer}, Andr\'{e} F.\@T.\@ Martins. ``Non-exchangeable Conformal Risk Control''. In: Proceedings of the International Conference on Learning Representations, 2024.
    \item \textbf{Ulmer, Dennis}, Elman Mansimov, Kaixiang Lin, Justin Sun, Xibin Gao, Yi Zhang. ``Bootstrapping LLM-based Task-Oriented Dialogue Agents via Self-Talk''. In: Findings of the Association for Computational Linguistics: ACL 2024.
    \item Gubri, Martin, \textbf{Dennis Ulmer}, Hwaran Lee, Sangdoo Yun, Seong Joon Oh. ``TRAP: Targeted Random Adversarial Prompt Honeypot for Black-Box Identification''. In: Findings of the Association for Computational Linguistics: ACL 2024.
\end{enumerate}

Another published document concerns the proceedings of the UncertaiNLP workshop, in which the author was involved as an editor and workshop co-organizer.
The workshop was co-located with the European meeting of the Association for Computational Linguistics (EACL) in St.\@ Julians, Malta, in 2024:

\begin{itemize}
    \item Vázquez, Raúl, Hande Celikkanat, \textbf{Dennis Ulmer}, Jörg Tiedemann, Swabha Swayamdipta, Wilker Aziz, Barbara Plank, Joris Baan, and Marie-Catherine de Marneffe. Proceedings of the 1st Workshop on Uncertainty-Aware NLP (UncertaiNLP 2024).
\end{itemize}

All of the above publications are accompanied by open-source code, that is listed in detail in \cref{app:open-source}.
However, some of the more important open-source contributions are highlighted here:

\begin{enumerate}
    \item \href{https://github.com/Kaleidophon/nlp-uncertainty-zoo}{\texttt{nlp-uncertainty-zoo}}: A Python package implementing different methods for uncertainty quantification in sequence classificationand sequence labeling in NLP.
    \item \href{https://github.com/Kaleidophon/deep-significance}{\texttt{deep-significance}}: A Python package including many functions to simplify statistical significance testing in deep learning.
    \item \href{https://github.com/Kaleidophon/awesome-experimental-standards-deep-learning}{\texttt{awesome-experimental-standards-deep-learning}}: A pointer to useful resources in order to improve experimental standards as well as reproducibility and replicability in Deep Learning experiments.
\end{enumerate}
\renewcommand{\thefootnote}{\arabic{footnote}} 
\setcounter{footnote}{3}

\section{Structure}\label{sec:structure}

This thesis is structured as to provide a comprehensive overview over the topic of uncertainty from both a statistical and linguistic point of view.
Both perspectives are then woven together in a overview over uncertainty quantification in deep learning and natural language processing.
This part serves as a foundation for later chapters about the uncertainty in the experimental design in NLP, before concretely tackling specific problem scenarios:
Uncertainty in text classification problems, uncertainty in language generation problems and uncertainty in latter problems specifically involving the use of large language models.\\

To be more detailed, \cref{ch:background} introduces the reader to different concepts in uncertainty quantification and related literature.
It begins with a definition of uncertainty from a variety of perspectives, for instance frequentist and Bayesian statistics\index{Statistics!Frequentist}\index{Statistics!Bayesian}, linguistics\index{Linguistics}, and several popular approaches in deep learning.
In this context, we also discuss \citet{ulmer2023prior}, which surveys works related to a novel class of uncertainty quantification methods called \emph{evidential deep learning}.
In the end, this includes a discussion of the relationship of uncertainty quantification with the end-user with both a motivation in trust and communication.\\

While most of the research that makes up this thesis is focused on uncertainty in \emph{modeling} language, uncertainty also occurs in the experimental design and execution of day-to-day research.
Therefore, \cref{ch:methodology} presents an interlude on challenges with the notions of reproducibility \& replicability in deep learning, their connection to uncertainty, and the use of statistical hypothesis testing, all of which inform the methodology of later chapters.
This encompasses the published works of \citet{ulmer-etal-2022-experimental}, giving an account of ongoing discussions about experimental methods in deep learning, as well as \citet{ulmer2022deep}, introducing a package for better statistical hypothesis testing and its application to a case study with large language models.\\

In the subsequent \cref{ch:uncertainty-classification}, we tackle the problem of uncertainty in classification problems.\index{Classification}  
First we demonstrate the pitfalls of uncertainty quanitification for classification using simple ReLU networks, drawing from \citet{ulmer2020know}.
Afterwards, we discuss uncertainty quantification in the context of different classification problems specific to NLP, based on \citet{ulmer-etal-2022-exploring}.
Here we show how well exisiting methods for NLP fare on different languages and tasks, and how much that performance---including the reliability of uncertainty estimates---depends on the amount of available training data.\\
%We further show that an extension of EDL for NLP can mitigate some of these issues.\\

In \cref{ch:uncertainty-generation}, we move to the problem of natural language generation\index{Natural language generation}, and develop a new calibrated sampling method based on conformal prediction\index{Conformal prediction} \citep{papadopoulos2002inductive,vovk2005algorithmic}. 
In language generation, we often restrict the set of possible candidate tokens to generate to a subset of (hopefully) plausible continutations.
Using the theoretical underpinning of conformal prediction, we introduce a novel method that does so with statistical guarantees.
This chapter is based on \citet{ulmer2024non}, and we demonstrate how this novel way to construct prediction sets is theoretically sound and produces flexibel prediction sets that come with guarantees about containing plausible tokens to generate.\\

In \cref{ch:uncertainty-llms}, we discuss a new method for quantifying the confidence of LLMs\index{Large language model} originally published in \citet{ulmer2024calibrating}, that tries to circumvent many of the pratical constraints that come with 
large model sizes.
Compared to other alternatives, this method is furthermore applicable to black-box models that do not allow any internal access to model states or weights, and is comparatively lightweight to train.\\

The thesis continues with a general discussion of the overall results in \cref{ch:discussion}.
There, we answer the overarching research questions stated in \cref{sec:objective} and reflect on current research directions in the field.
Lastly, \cref{ch:conclusion} takes on a bird's-eye view by contextualizing uncertainty quantification in the current zeitgeist and discussing its relationship with contemporary policies.
In addition, the thesis comprises an appendix with theoretical results (\cref{app:theoretical-appendix}) and one with experimental details (\cref{app:empirical-appendix}) that were omitted from the main text.
Details that are necessary for an accurate reproduction of experimental results are bundled in \cref{app:reproducibility-appendix}.

% Chapter 2

\chapter{Background} % Chapter title

\label{ch:background} % For referencing the chapter elsewhere, use \autoref{ch:examples} 

\begin{tikzpicture}[remember picture,overlay]
    \node[anchor=north,inner sep=0pt] at (current page text area.north) {\includegraphics[width=\linewidth, clip=true, trim = 8cm 75cm 8cm 50cm]{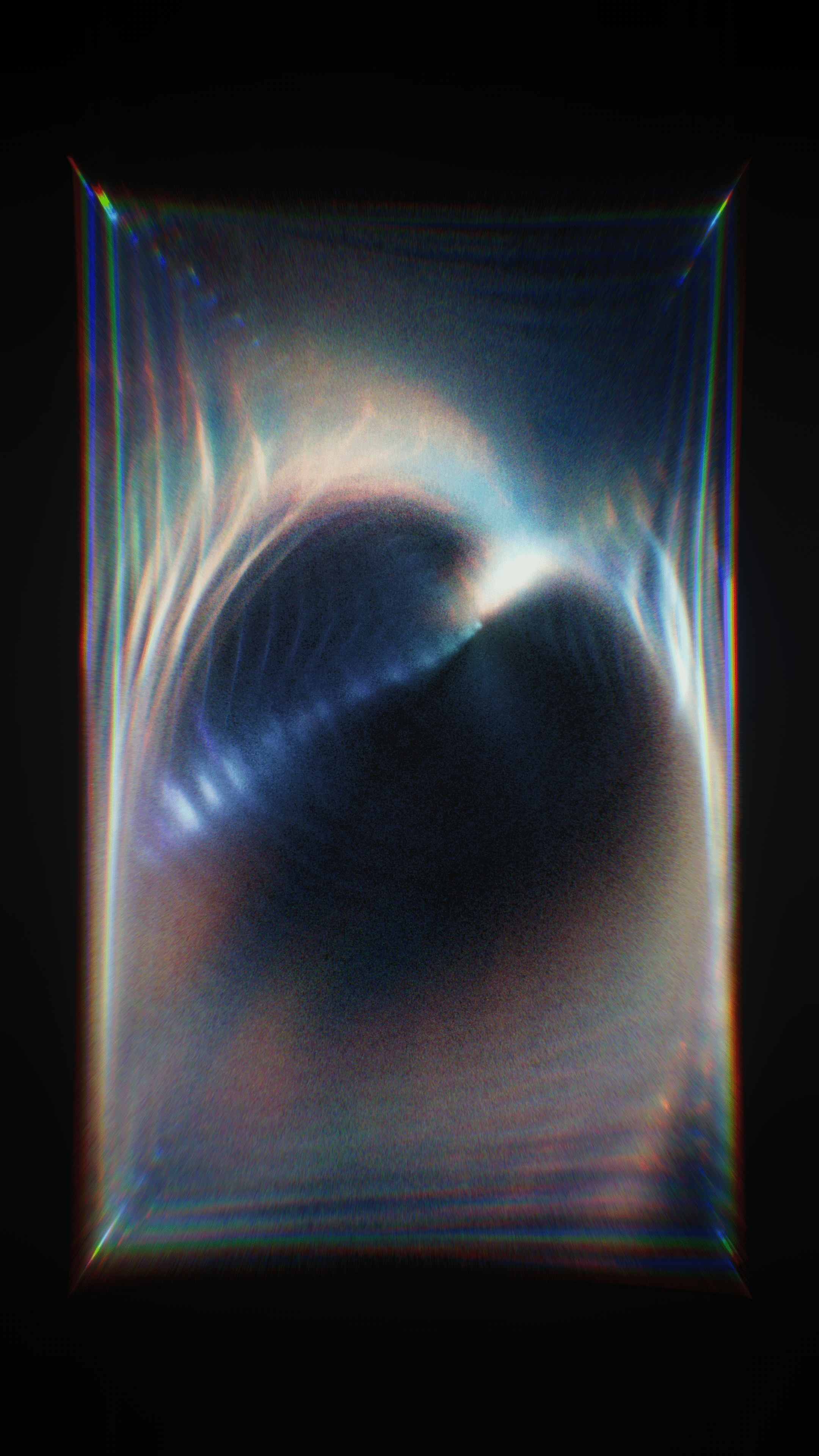}};
\end{tikzpicture}

\epigraph{
    ``\emph{Le doute n'est pas une état bien agréable, mais l'assurance est un état ridicule.}''\\[0.2cm]
    ``\emph{Doubt is not a pleasant condition, but certainty is absurd.}''
}{---Voltaire}

Uncertainty is a common occurrence in everyone's life, and thus most people have an intuitive understanding of the concept.
To define it concretely, however, can be challenging.
Colloquially, we might define uncertainty as a phenomenon or state that is filled with doubts, lack knowledge or that is simply hard to predict.
In research papers, the term uncertainty often only remains vaguely defined, either building on an intuitive definition or presupposing a certain school of thinking.\\

The aim of this chapter is to bring some clarity to the different ways uncertainty is defined, and to give an fairly comprehensive account of its applications.
This entails a journey from its origins in statistics (\cref{sec:frequentist-perspective,sec:bayesian-perspective}) and linguistics (\cref{sec:uncertainty-linguistics,sec:expressing-uncertainty}) to its implementation with neural networks, specifically in deep learning (\cref{sec:uncertainty-deep-learning}) and natural language processing (\cref{sec:uncertainty-nlp}).
In the latter contexts, this comes with a focus on \emph{modeling} uncertainty, and this is indeed also where many of the research papers in the field end. 
Therefore, an additional goal of this chapter is to not take uncertainty modeling as the ultimate goal per se, but to see beyond it and grasp the bigger picture.
As uncertainty quantification is often motivated by increasing trustworthiness and safety, we take a closer look at the relationship between uncertainty and trust in \cref{sec:uncertainty-trust}, as well as how to communicate uncertainty in \cref{sec:communicating-uncertainty}.
Furthermore, the chapter outlines diverse applications of uncertainty in \cref{sec:applications-uncertainty}.

\section{What Is Uncertainty, anyway?}\label{sec:what-is-uncertainty}

%It will become more apparent in the following sections (and chapters) that uncertainty in modeling is a very multi-faceted question.
%Different fields and sub-fields and people propose different solutions, and the adequacy of solutions also depends on the intended application.
%Bringing clarity into this tangle can be quite challenging, and therefore different concepts will be introduced carefully, pointing out their relation to one another. 
%For this reason, we will start with uncertainty in a statistical---but decidedly non-neural---context.
%Since this thesis is concerned with human language data, we will also review the concept of uncertainty from a linguistics perspective.
%We will later see in \cref{sec:uncertainty-deep-learning,sec:uncertainty-nlp} how works for uncertainty quantification for neural networks (in NLP) pick up and sometimes marry these different concepts to create new methdologies.

We start by defining the most central concept in this thesis: Uncertainty\index{Uncertainty}.
Since this thesis is focused on NLP, we aim to define the concept from all the perspectives modern NLP touches on.
This includes building up some basic concepts from frequentist and Bayesian statistics as well from different parts of linguistics.

\subsection{The Frequentist Perspective}\label{sec:frequentist-perspective}

\epigraph{``\emph{Statistical inference is serious business.}''}{---Bradley Efron, Robert J.\@ Tibshirani in \emph{An Introduction to the Bootstrap} \citep{tibshirani1993introduction}}

\emph{Frequentist statistics}\index{Statistics!Frequentist} in an approach to statistics that aims to make inferences and draw conclusions from sampled data, alone. 
The term is based on the fact that probabilities are seen as equivalent to the observed frequencies of events in the data, assuming (potentially infinitely) many repetitions of an experiment \citep{willink2011disentangling}. 
Let us reason about the popular example of a coin flip here to illustrate this notion. 
We are given a coin and would like to estimate the probability of heads, which we define as the parameter of interest to estimate and will denote by $\theta$. 
We do not know whether the coin is fair, so we flip it a number of times and count the heads and tails to estimate this probability. 
We obtain the following five coin flips:\\

\begin{center}
    \tails \hspace{0.2cm} \tails \hspace{0.2cm} \heads \hspace{0.2cm} \heads \hspace{0.2cm} \tails \hspace{0.2cm}
\end{center}

Based on this experiment, we then estimate the probability of heads as $\hat{\theta} = \frac{\# \text{heads}}{\# \text{coin flips}} = \frac{2}{5} = 0.4$.
However, how can we be sure that this reflects the actual probability of heads?
We thus repeat the experiment three more times, and obtain:

\begin{center}
    \tails \hspace{0.2cm} \tails \hspace{0.2cm} \heads \hspace{0.2cm} \heads \hspace{0.2cm} \tails \hspace{0.2cm} $\rightarrow \hat{\theta}_2 = \frac{2}{5} = 0.4$
\end{center}
\begin{center}
    \tails \hspace{0.2cm} \tails \hspace{0.2cm} \tails \hspace{0.2cm} \heads \hspace{0.2cm} \heads \hspace{0.2cm} $\rightarrow \hat{\theta}_3 = \frac{2}{5} = 0.4$
\end{center}
\begin{center}
    \heads \hspace{0.2cm} \tails \hspace{0.2cm} \tails \hspace{0.2cm} \heads \hspace{0.2cm} \heads \hspace{0.2cm} $\rightarrow \hat{\theta}_4 = \frac{3}{5} = 0.6$
\end{center}

As we gather more and more samples and take their average, we will provably converge to the true value of $\theta$ in the limit due to the law of large numbers \citep{dekking2005modern}.
But in light of a limited number of samples like above, how can we quantify the uncertainty of our estimate?

\paragraph{Confidence Intervals.}\index{Confidence interval} One common approach to compute some frequentist uncertainty estimate is the use of \emph{confidence intervals} \citep{neyman1937outline}. 
Confidence intervals try to capture a range of values for the parameter $\theta$ such that, if we were to repeat our experiment, the computed confidence intervals would cover the true value with some probability (e.g.\@ $95 \%$).
We can do this by assuming that any estimates $\hat{\theta}$ for $\theta$ are independently and normally distributed.
The normality assumption holds in this case since the central limit theorem\index{Central limit theorem} applies, stating that if we were to repeat this experiment over and over, the sample mean (which is our estimate $\hat{\theta}$) will be normally distributed.
Possessing the knowledge about the estimate being normally distributed lets us define the confidence intervals. 
The procedure is as follows: 
We know that our samples are normally distributed according to some specific (constant but unknown) mean $\theta$ and standard deviation.
In our example, $\theta$ would correspond to the true probability of heads that we are interested in.
For convenience, we now standardize this distribution.
We can achieve this by simply subtracting the mean $\theta$ from the mean of our estimates $\bar{\theta}$ and dividing by an estimate of the standard deviation denoted as $s$, which we obtain as:

\begin{align}\label{eq:confidence-intervals}
    \bar{\theta} &  = \frac{1}{N}\sum_{i=1}^N \hat{\theta}_i \\ 
    s^2 & = \frac{1}{N - 1}\sum_{i=1}^N (\hat{\theta}_i - \bar{\theta})^2.
\end{align}

%We can now use these two quantities
According to the central limit theorem\index{Central limit theorem}, the estimate of the standard deviation improves in accuracy by a factor of $\frac{1}{\sqrt{N}}$, leading to the standard error $s / \sqrt{N}$\index{Standard error}, and we thus arrive at:

\begin{equation}\label{eq:t-statistic}
    t = \frac{\bar{\theta} - \theta}{s / \sqrt{N}}.
\end{equation}

Now, we would like to know how the statistic in \cref{eq:t-statistic} is distributed in order to identify confidence intervals for $\theta$.
We already know that $\bar{\theta}$ is distributed according to a Normal distribution, and assuming that the sample variance $s^2$ is distributed according to a $\chi^2$ distribution\index{$\chi^2$ distribution} with $N$ degrees of freedom, we obtain a Student's-$t$ distribution\index{Student's-$t$ distribution} with $N-1$ degrees of freedom.
%and thus we obtain a variable $t$ that is distributed according to a Student's t distribution with $N-1$ degrees of freedom.
We can now determine the bounds such that $p(-c \le t \le c) = 0.95$.
Using some intermediate steps, we find that 

\begin{align}\label{eq:confidence-intervals-continued}
    & p(-c \le t \le c) \\
    = & p\Big(-c \le \frac{\bar{\theta} - \theta}{s / \sqrt{N}} \le c\Big)\\
    = & p\Big(-\frac{cs}{\sqrt{N}} \le \bar{\theta} - \theta \le \frac{cs}{\sqrt{N}}\Big)\\
    = & p\Big(-\frac{cs}{\sqrt{N}} - \bar{\theta} \le - \theta \le \frac{cs}{\sqrt{N}} - \bar{\theta}\Big)\\
    = & p\Big( \bar{\theta}-\frac{cs}{\sqrt{N}} \le \theta \le \bar{\theta} + \frac{cs}{\sqrt{N}}\Big).
\end{align}

Therefore, we know that our unknown mean $\theta$ of the distribution will be contained within these bounds. 
Lastly, we need to choose $c$ such that the proposed interval corresponds to $95 \%$ (or some other desired amount) of the total probability density. 
Since the shape of the Student's-$t$ distribution\index{Student's-$t$ distribution} is known, we can choose $c$ to correspond to the $97.5$-th percentile (leaving $2.5 \%$ of the total density to either side).
This number can be easily computed through the distributions's inverse cumulative distribution function\index{Cumulative distribution function!Inverse}.\footnote{
    Using a software library such as \texttt{scipy}, we can for instance compute \texttt{scipy.stats.t.ppf(0.975, df=3).}
}
In our example above, we have $N = 4$, $\bar{\theta} = \frac{1}{4}(\frac{2}{5} + \frac{2}{5} + \frac{2}{5} + \frac{3}{5}) = 0.45$ 
and $s^2 = \frac{1}{3}(0.05^2 + 0.05^2 + 0.05^2 + 0.15^2) \approx 0.0111$. 
With $c \approx 3.182$, this gives us a confidence interval of $[0.31, 0.59]$.
This interval can be interpreted in the following way:
If we were to repeat this experiment, say, $100$ times, the resulting confidence intervals would contain the true value $95$ times.
For the samples shown above a fair coin was used, and thus the confidence interval contains the true value of $0.5$.
It should be noted here that confidence intervals are sometimes available in closed form, for instance for the normal distribution\index{Normal distribution}.\footnote{
    The reasoning goes as follows: By Cochran's theorem\index{Cochran's theorem}, if the distribution is normal, the sample mean and variance are independent \citep{cochran1934distribution}.
    But the reverse is also true, so with an independent sample mean and variance we can assume that the underlying distribution is normal. 
    We can again construct the $t$-statistic in \cref{eq:t-statistic} to obtain confidence intervals e which are 
    $\theta \in [\bar{\theta} - t_{n-1, 1 - \alpha/2}\frac{s}{\sqrt{N}}, \bar{\theta} + t_{n-1, 1 - \alpha/2}\frac{s}{\sqrt{N}}]$, 
    where $t_{N-1, 1 - \alpha/2}$ stands for the $1 - \alpha / 2$-th quantile of a $t$-distribution with $N-1$ degrees of freedom and a $1 - \alpha$ confidence level (\citealp{krishnamoorthy2006handbook}, p.\@ 130).
}
This simplistic example also assumed confidence intervals to be both symmetrical and two-sided (i.e.\@ having a lower and upper bound).
For a more thorough and comprehensive treatment of confidence intervals we refer to other works such as \citet{zech2002frequentist, smithson2003confidence}.

\paragraph{Bootstrapping and Jackknife.} In the previous example, we were able to successfully obtain a confidence interval for the probability of heads. 
However, this required us to repeat the coin flipping experiment multiple times. 
In the case of flipping a coin, this is rather straightforward---nevertheless, we might also interested in quantifying the uncertainty about the estimate in cases where obtaining a new sample is difficult or expensive (e.g.\@ when an experiment require expensive computational hardware).
A tool for these cases is given in the form of \emph{bootstrapping}\index{Bootstrap} \citep{efron1992bootstrap, tibshirani1993introduction}:
Instead of collecting new data, we can instead perform inferences from our existing sample through re-sampling:
We sample randomly from our initial set of coin flips with replacement,\footnote{Sampling with replacement implies that our new samples might contain duplicates.} and obtain a number of new (pseudo-)samples. 
We then use these to estimate the confidence intervals of our estimate in a similar fashion to the confidence intervals in the previous paragraph:
Drawing $N=10$ re-samples, using the same procedure as in \cref{eq:confidence-intervals,eq:confidence-intervals-continued}, we obtain a confidence interval of $[0.26, 0.62]$. % which contains the true value ($0.5$), but is wider than the confidence intervals originating from multiple samples.
Nevertheless, there are known problems with the bootstrap:
When our sample size small (just five coin flips), the sample might not be representative, and bootstrap samples can amplify any bias present in the sample.
Here, having two heads and three tails differs slightly from the actual probability of $0.5$, which is then carried over into the bootstrap samples. 
Another, similar estimator is the \emph{jackknife}\index{Jackknife} \citep{quenouille1949approximate, tukey1958bias}, where we do not resample, but instead create new samples by leaving out one observation at a time.
Therefore, the original set of coin flips would yield the following new pseudo-samples:\\

\begin{center}
    \begin{tabular}{ccc}
        & \tails \hspace{0.2cm} \heads \hspace{0.2cm} \heads \hspace{0.2cm} \tails \hspace{0.2cm} & $\rightarrow \hat{\theta}_{-1} = \frac{2}{4} = 0.5$\\[0.4cm]
        & \tails \hspace{0.2cm} \heads \hspace{0.2cm} \heads \hspace{0.2cm} \tails \hspace{0.2cm} & $\rightarrow \hat{\theta}_{-2} = \frac{2}{4} = 0.5$\\[0.4cm]
        & \tails \hspace{0.2cm} \tails \hspace{0.2cm} \heads \hspace{0.2cm} \tails \hspace{0.2cm} & $\rightarrow \hat{\theta}_{-3} = \frac{1}{4} = 0.25$\\[0.4cm]
        & \tails \hspace{0.2cm} \tails \hspace{0.2cm} \heads \hspace{0.2cm} \tails \hspace{0.2cm} & $\rightarrow \hat{\theta}_{-4} = \frac{1}{4} = 0.25$\\[0.4cm]
        & \tails \hspace{0.2cm} \tails \hspace{0.2cm} \heads \hspace{0.2cm} \heads \hspace{0.2cm} & $\rightarrow \hat{\theta}_{-5} = \frac{2}{4} = 0.5$\\[0.4cm]
    \end{tabular}
\end{center}

We again repeat our procedure in \cref{eq:confidence-intervals,eq:confidence-intervals-continued} to obtain the confidence interval of $[0.25, 0.55]$.
In order to make a prediction about a new coin flip, in all of three cases we would simply declare head with a probability of the estimated $\hat{\theta}$ or hedge our bets using the range of values contained in the confidence interval.

\paragraph{Likelihood Functions.} 
A useful tool to evaluate the fit of a parameter estimate for the data are \emph{likelihood functions}\index{Likelihood}.
The likelihood $p(\mathbb{D}\mid\theta)$ quantifies how well the choice of a value for $\theta$ ``explains'' the observations $\mathbb{D}$, meaning how likely the value is to have generated the data or how consistent the data are with the chosen value.
Accordingly, a high likelihood expresses that a value of $\theta$ is consistent with the observations, while a low likelihood suggest that $\theta$ is unlikely to have generated the data.
For the coin flipping example, we can choose a $\emph{Bernoulli}$ likelihood\index{Bernoulli distribution}:

 \begin{equation}\label{eq:bernoulli-likelihood}
     \text{Bernoulli}(x\mid\theta) = \theta^x(1 - \theta)^{(1 - x)}.
 \end{equation}

Given the probability of heads $\theta$, it assigns a probability of an outcome $x$, i.e.\@ heads or tails.
In line with the intuition that more suitable values of $\theta$ assign higher likelihoods to the data, a quick derivation reveals that the mean $\hat{\theta}$ we used is actually 
the parameter that maximizes the likelihood of our sample:

\begin{align}
        p(\mathbb{D} \mid \theta) & = \prod_{i=1}^N p(x_i \mid \theta) = \prod_{i=1}^N \theta^x(1 - \theta)^{(1 - x)} \label{eq:likelihood-bernoulli} \\
        \log p(\mathbb{D} \mid \theta) & = \sum_{i=1}^N x_i\log(\theta) + (1 - x_i)\log(1 - \theta) \\
        \frac{\partial}{\partial \theta} \log p(\mathbb{D} \mid \theta) & = \sum_{i=1}^N \frac{x_i}{\theta} - \frac{1 - x_i}{1 - \theta} \overset{!}{=} 0 \\
        0 & = \sum_{i=1}^N \frac{x_i(1-\theta)}{\theta(1-\theta)}- \frac{(1 - x_i)\theta}{\theta(1 - \theta)} \\
         & =\sum_{i=1} ^Nx_i - \cancel{x_i\theta} - \theta + \cancel{x_i\theta} \\
        \theta_\text{MLE} & = \frac{1}{N} \sum_{i=1}x_i. \label{eq:bernoulli-mle}
\end{align}

Here, we used the i.i.d.\@ assumption (identically, independently distributed) to argue that since observations $x_i$ are independent, we can  
factor the joint distribution $p(\mathbb{D}\mid \theta) \equiv p(x_1, \ldots, x_N\mid \theta)$ into a product of its individual likelihoods.
Then, we transfer the computation into log-space for convenience, and identify the estimate by taking the derivative w.r.t.\@ $\theta$, setting to zero, and solving for it.
This is referred to as the \emph{maximum likelihood estimate} (MLE)\index{Maximum likelihood estimate}.

\paragraph{Recap.} The above methods have illustrated the viewpoint of frequentist statistics:
Parameter estimates are derived from the actually observed data, and our uncertainty about the estimates can expressed through confidence intervals, in which we expect our true value to fall.
These can be obtained by relating estimates of the parameter for instance to the Student's-$t$ distribution\index{Student's-$t$ distribution} by exploiting the central limit theorem\index{Central limit theorem}. 
Furthermore, other estimates can be obtained by collecting more data or through procedures such as the bootstrap\index{Bootstrap} or the jackknife\index{Jackknife}.
In our case, we knew that the used coin was fair, and that the initial sample simply ended up not representative due to using an uneven number of observations.
In a similar but more realistic scenario, we might not know anything about the properties of the coin (or the phenomenon of interest), but might still suspect, at least without any other information available, that it is fair.
The frequentist framework does not give us any means to incorporate this belief into our reasoning, but the Bayesian view presented in the next section does.

\subsection{The Bayesian Perspective}\label{sec:bayesian-perspective}

\epigraph{``\emph{There is a valid defence of using non-Bayesian methods, namely incompetence.}''}{---John Skilling in \emph{Fundamentals of MaxEnt in Data Analysis} \citep{skilling1990fundamentals}.}

Bayesian statistics\index{Statistics!Bayesian} delineates itself from frequentist statistics by seeing probability itself as more than just the mere relative frequency of an event, and instead as the degree of belief in the occurrence of an event.\footnote{
    Even though there are also subtle nuances to this definition, see for instance \citet{good197146656}.
}
This difference has caused (and is still causing) ideological chasms among statisticians, as illustrated by the quote above.
The name of Bayesian statistics is derived from Thomas Bayes, an English presbytarian minister in the 18th century who first formulated the eponymous \emph{Bayes' theorem}\index{Bayes' theorem}.
It should be noted however that Bayes only formulated his theory in a very specific setting,\footnote{Namely, using a uniform prior. See \cref{eq:bayes-theorem} and onward.} and that a general version of Bayesian statistics was instead pioneered by Pierre-Simon Laplace \citep{mcgrayne2011theory, leonard2014personal}.
The theorem can be formulated as follows: 
Given a set of observations $\mathbb{D}$ and a parameter of interest $\theta$, we can express the probability of the parameter given the observational data as 

\begin{equation}\label{eq:bayes-theorem}
    p(\theta\mid\mathbb{D}) = \frac{p(\mathbb{D}\mid\theta)p(\theta)}{p(\mathbb{D})},
\end{equation}
where the different parts of the equation are commonly referred to as the \emph{posterior} $p(\theta\mid\mathbb{D})$\index{Posterior distribution}, the likelihood $p(\mathbb{D}\mid\theta)$\index{Likelihood}, the prior $p(\theta)$\index{Prior distribution}, and the evidence $p(\mathbb{D})$\index{Evidence}.
We already discussed likelihoods in the previous section.
The prior $p(\theta)$ is a probability distribution over possible values of $\theta$, and thus allows us to express our prior belief by attributing higher probability to values of $\theta$ we deem more likely.
This also implies a philosophical difference with frequentist statistics\index{Statistics!Frequentist}:
While $\theta$ was treated as an unknown constant before, it is now seen as another random variable.
The evidence $p(\mathbb{D})$\index{Evidence} encodes the general probability of the observed data under any value of $\theta$. 
This somewhat hidden interpretation becomes more clear when rewriting the term:

\begin{equation}\label{eq:evidence}
    p(\mathbb{D}) = \int p(\mathbb{D},\theta)\dd\theta = \int p(\mathbb{D}\mid\theta)p(\theta) \dd\theta.
\end{equation}

We can therefore interpret the evidence\index{Evidence} as the likelihood\index{Likelihood} of the data averaged over all possible parameter values of $\theta$, weighed by their prior probabilities.
Lastly, the posterior $p(\theta \mid \mathbb{D})$\index{Posterior distribution} describes a probability distribution over values of $\theta$ given our observations.
We can think of the posterior as starting with our prior belief, using the data to update it and arriving at a final distribution that takes both of these into account.
This has several advantages: 
We can now choose to encode our suspicions about the value of the target parameter into the prior\index{Prior distribution}. 
%When we do not have many observations available, this prior might then help to correct an unrepresentative sample.
But as we will see, obtaining more and more data points results in outweighing the prior belief, completely relying on the observations in the limit.

\begin{figure}
    \begin{subfigure}[t]{0.48\textwidth}
        \centering
        \includegraphics[width=0.985\textwidth]{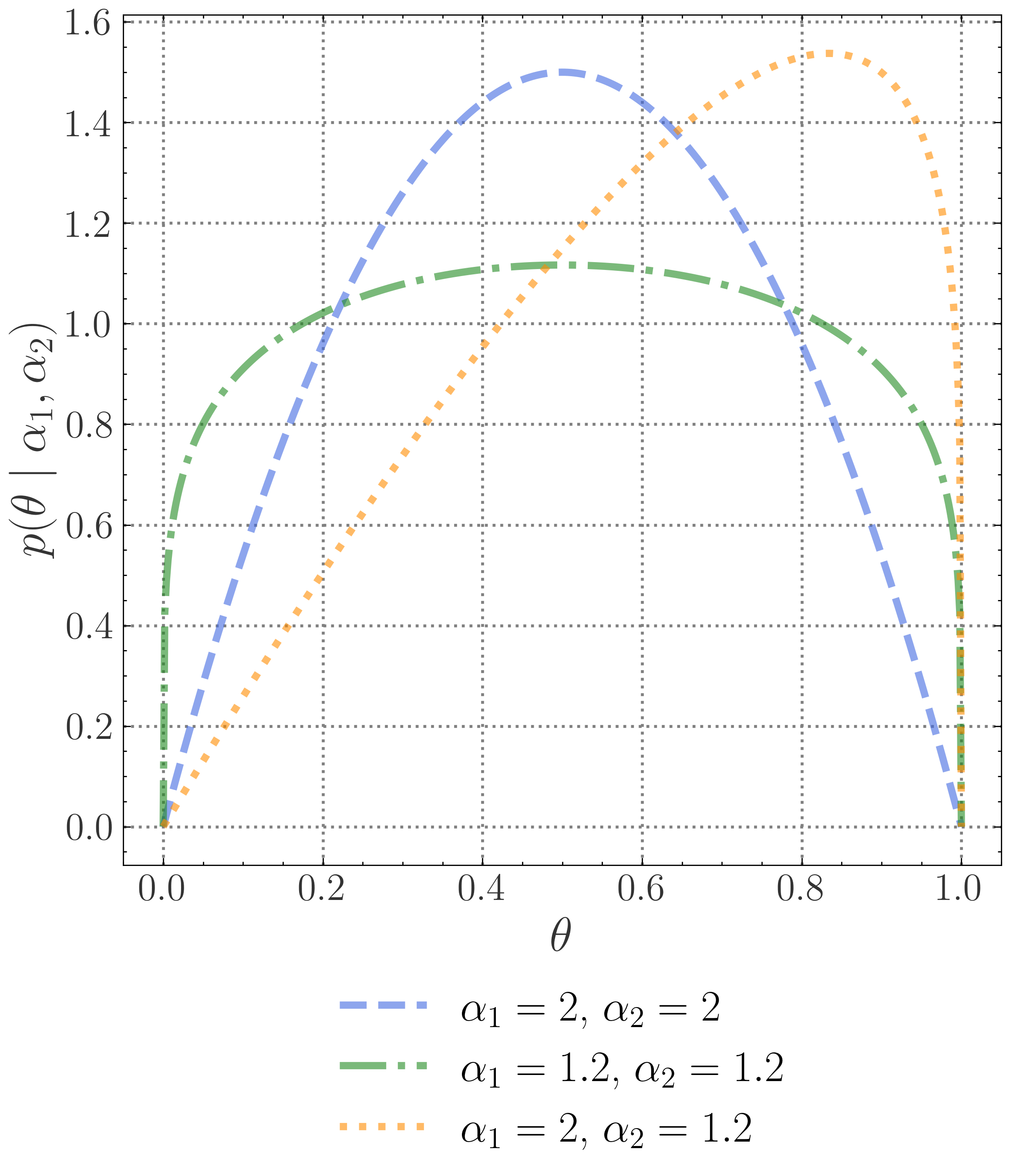}
        \caption{Beta priors.}\label{subfig:beta-priors}
    \end{subfigure}
    \hfill
    \begin{subfigure}[t]{0.48\textwidth}
        \centering
        \includegraphics[width=0.985\textwidth]{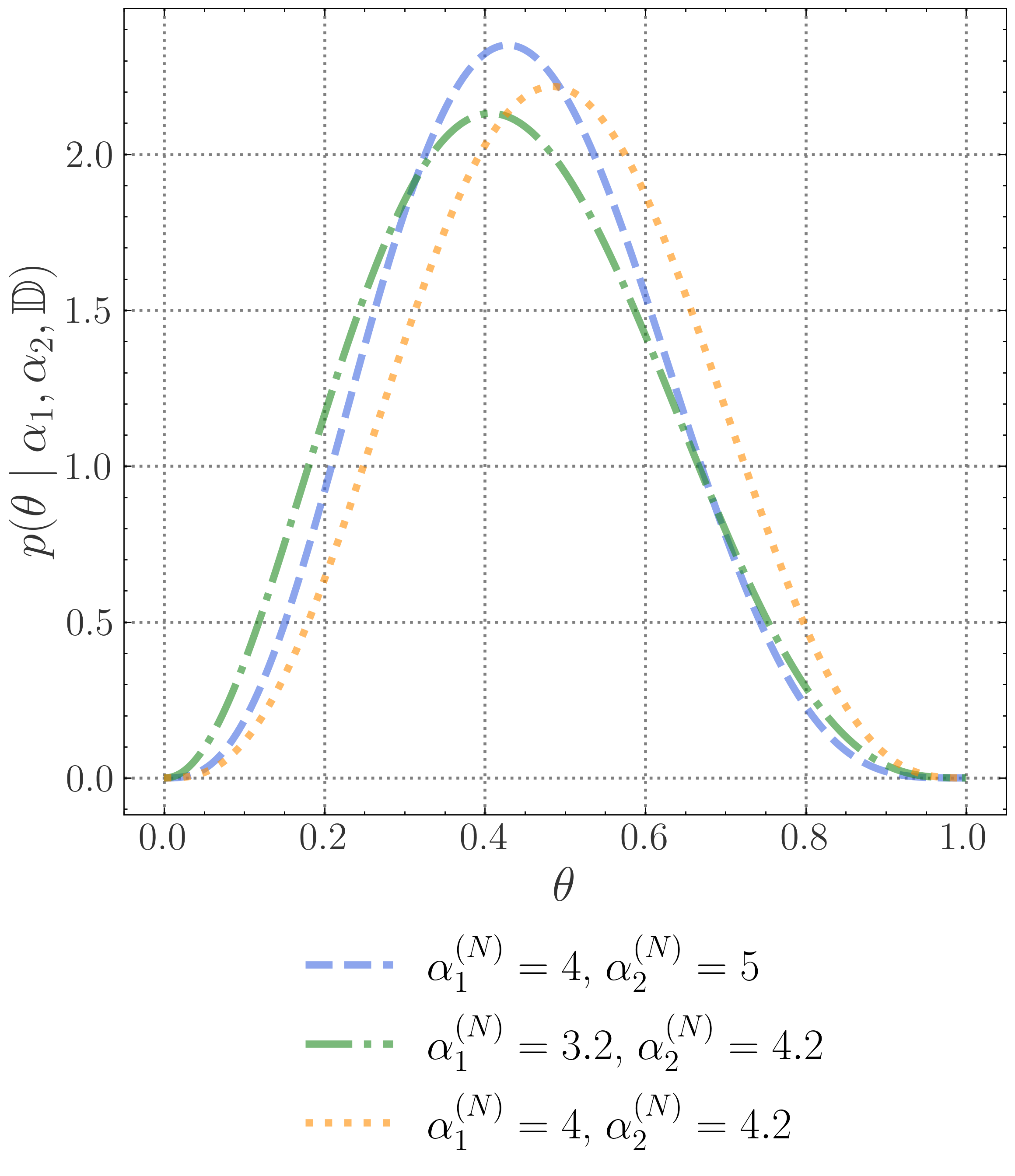}
        \caption{Beta posteriors.}\label{subfig:beta-posteriors}
    \end{subfigure}
    \caption[Different Beta prior and posterior shapes.]{Plots of different choices of (a) Beta prior distributions and their resulting (b) posterior distributions after observing our initial set of coin flips.
    Juxtaposing the two plots illustrates how the choice of prior belief can influence the shape of the resulting posterior distribution.}\label{fig:beta-priors-posteriors}
\end{figure}

\paragraph{Coin Flipping Redux.} We now illustrate these concepts using the coin flipping example from \cref{sec:frequentist-perspective}, showing how uncertainty is modeled from the Bayesian perspective.
In order to do so, we first have to make some design choices, i.e.\@ the choice of likelihood and prior function as well as prior parameters.
We again use the Bernoulli\index{Bernoulli distribution} likelihood from the previous section, and now would like to define a prior over $\theta$.
A good choice for a prior for the Bernoulli distribution is the \emph{Beta} distribution\index{Beta distribution}:

 \begin{align}\label{eq:beta-prior}
    \text{Beta}(\theta; \alpha_1, \alpha_2) & = \frac{1}{\text{B}(\alpha_1, \alpha_2)} \theta^{\alpha_1 - 1}(1 - \theta)^{\alpha_2 - 1} \\
     \text{B}(\alpha_1, \alpha_2) & = \frac{\Gamma(\alpha_1)\Gamma(\alpha_2)}{\Gamma(\alpha_1 + \alpha_2)},
\end{align}

\noindent with $\Gamma(\cdot)$ denoting the Gamma function\index{Gamma function}, a generalization of the factorial to the real numbers. 
The distribution has its support on $[0, 1]$ and possesses two shape parameters $\alpha_1, \alpha_2 \in \mathbb{R}^+$, which we can use to encode our prior belief about $\theta$\index{Prior distribution}.
A few examples for the resulting distribution are shown in \cref{subfig:beta-priors}. 
Choosing the Beta distribution\index{Beta distribution} as a prior comes in handy when analytically deriving the posterior distribution\index{Posterior distribution}, since it is \emph{conjugate} to the likelihood\index{Likelihood}.
Conjugacy\index{Conjugacy} here means that using a Beta prior together with a Bernoulli likelihood as in \cref{eq:likelihood-bernoulli}, the posterior has the form of a Beta distribution\index{Beta distribution}.\footnote{
    Conjugate priors are available for distributions that can be generalized to a particular form which is referred to as \emph{exponential families}\index{Exponential families}, including popular distributions such as the Normal\index{Normal distribution}, Poisson\index{Poisson distribution}, Bernoulli\index{Bernoulli distribution}, and categorical distribution\index{categorical distribution} and more \citep{bishop2006pattern, gelman2021bayesian, efron2022exponential}.
}
Bayes' rule\index{Bayes' theorem} contains the unwieldy evidence term\index{Evidence}, which we established in \cref{eq:evidence} can in some cases be evaluated analytically using an integral over parameters.
However, we can notice that the evidence $p(\mathbb{D})$ does not depend on the parameters directly, and only scales the posterior $p(\theta\mid\mathbb{D})$\index{Posterior distribution} by a constant.
As such, we can declare the posterior to be proportional to the product of the likelihood\ and prior:

\begin{equation}
    p(\theta\mid\mathbb{D}) = \frac{p(\mathbb{D}\mid\theta)p(\theta)}{p(\mathbb{D})} \propto p(\mathbb{D}\mid\theta)p(\theta) = \prod_{i=1}^N p(x_i\mid\theta)p(\theta), 
\end{equation}

\noindent which simplifies solving for the posterior parameters.
We now substitute the expressions in \cref{eq:beta-prior,eq:bernoulli-likelihood} into Bayes' rule in \cref{eq:bayes-theorem} and continue in log-space for convenience:
\begin{align}\label{eq:coin-flip-posterior}
    % & p(\theta\mid\mathbb{D}) \propto p(\mathbb{D}\mid\theta)p(\theta) = \prod_{i=1}^N p(x_i\mid\theta)p(\theta) \nonumber \\
    \log p(\theta\mid\mathbb{D}) \propto & \sum_{i=1}^N \log p(x_i\mid\theta) + \log p(\theta) \\
    = & \sum_{i=1}^N x_i \log \theta + (1 - x_i)\log (1 - \theta) \nonumber \\
    &  + (\alpha_1 - 1)\log \theta + (\alpha_2 - 1)\log (1 - \theta) \\
   = & \Big(\alpha_1 + \sum_{i=1}^N x_i - 1\Big)\log \theta \nonumber \\
   & + \Big(\alpha_2 + N - \sum_{i=1}^N x_i - 1\Big) \log(1 - \theta),
\end{align}

\noindent where we can see that in the end---after dropping the log-Beta function\index{Beta function} as it is just a constant---we obtain the form of a Beta distribution\index{Beta distribution}, but this time with 
the new shape parameters $\alpha_1^{(N)} = \alpha_1 + \sum_{i=1}^N x_i$ and $\alpha_2^{(N)} = \alpha_2 + N - \sum_{i=1}^N x_i$.
Compared to the prior parameter values, they now contain information about the observations that we have made.
We can use these new parameters to visualize our posterior for our initial set of coin flips and our initial choices of priors in \cref{subfig:beta-posteriors}.
Similar to the frequentist confidence intervals\index{Confidence interval} of the previous sections, the uncertainty about the true value of $\theta$ is encoded in the spread of the posterior distribution\index{Posterior distribution}.
As we gather more observations, we expect the posterior to become more and more narrow around one (or few) values of $\theta$.
Similarly to the maximum likelihood estimate\index{Maximum likelihood estimate} in \cref{eq:bernoulli-mle} that helps us determine the parameter value which is most likely to have generated our observations, we can derive a similar quantity in the Bayesian setting.
This is referred to as the posteriori estimate (MAP)\index{Maximum a posterior estimate}, and can be interpreted as the most likely value of $\theta$ given the data and a choice of prior.
We can derive the MAP using the posterior in \cref{eq:coin-flip-posterior} and solving for $\theta$:

\begin{align}\label{eq:coin-flip-map}
    \frac{\partial}{\partial \theta} \log p(\theta\mid\mathbb{D}) & \overset{!}{=} 0 \\
    (1 - \theta)\Big(\alpha_1 + \sum_{i=1}^N x_i - 1\Big) & = \theta \Big(\alpha_2 + N - \sum_{i=1}^N x_i - 1\Big) \\
    \hat{\theta}_\text{MAP} & = \frac{\alpha_1 + \sum_{i=1}^N x_i - 1}{\alpha_1 + \alpha_2 + N - 2}.
\end{align}

\begin{figure}[htb]
    \centering
    \includegraphics[width=0.985\textwidth]{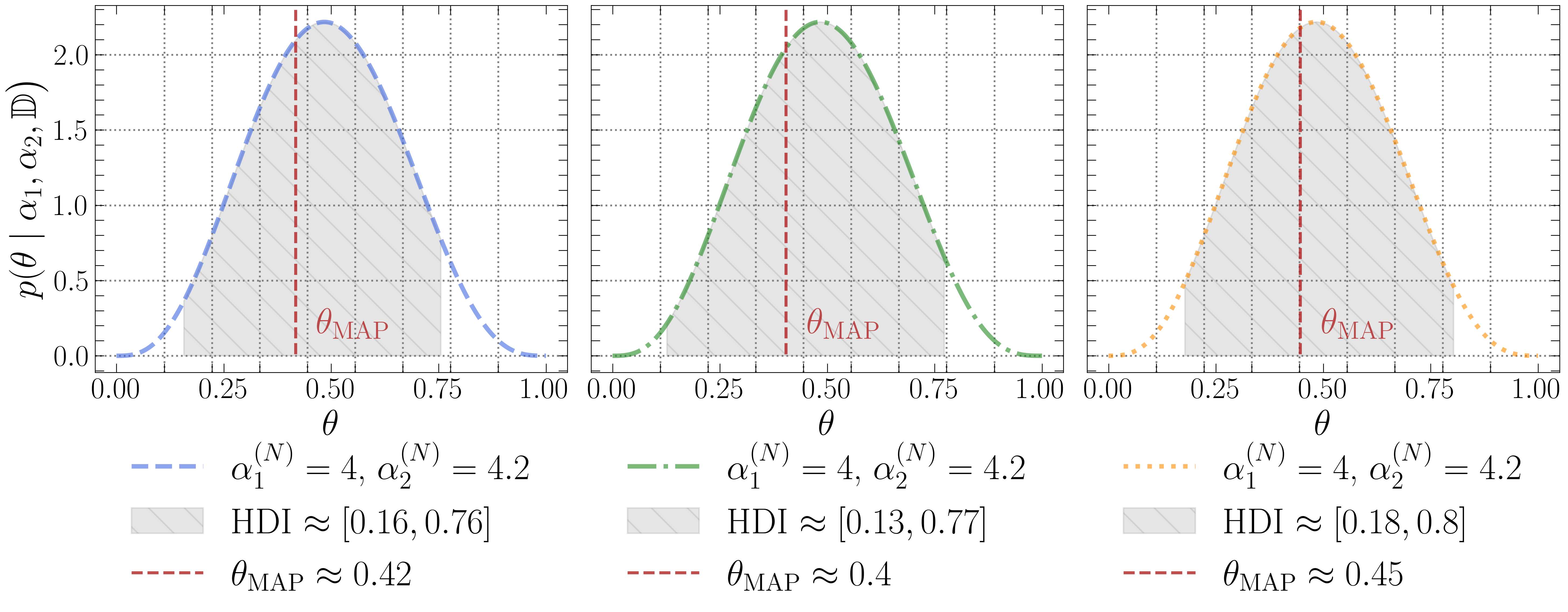}
    \caption[Highest density intervals and maximum a posteriori estimates for different Beta posteriors.]{Highest density intervals (gray regions) and maximum a posteriori estimates (red vertical lines) for different Beta posteriors.}\label{fig:hdis-and-maps}
\end{figure}

\paragraph{Highest Density Intervals.} One way to now quantify the uncertainty about our estimate for $\theta$ is to create the Bayesian counterpart of confidence intervals\index{Confidence interval}: 
The \emph{highest density interval} (HDI;\@ also referred to as the \emph{credible interval})\index{Highest density interval}. 
The HDI describes the ranges of values of the posterior distribution that covers $95 \%$ (or some other number) of the total density.
Thus, our estimate has a posterior probability of $95 \%$ to fall within this interval. 
For the prior and posterior distributions in \cref{fig:beta-priors-posteriors}, we obtain $\hat{\theta}_1 \approx 0.44$, $\text{HDI}_1 \approx [0.16, 0.76]$, $\hat{\theta}_2 \approx 0.43$, $\text{HDI}_2 \approx [0.13, 0.77]$, and $\hat{\theta}_3 \approx 0.49$, $\text{HDI}_3 \approx [0.18, 0.80]$, with the HDIs and MAP estimates shown in \cref{fig:hdis-and-maps}.
Since the second prior places less belief on a value of $\theta = 0.5$, the slightly skewed initial sample of coins $\hat{\theta} = 0.4$ shifts the posterior estimate and HDI slightly towards the left.
In the third case, our prior belief is highly biased towards higher values of $\theta$, which is also reflected in the obtained posterior estimate, the MAP\index{Maximum a posterior estimate} estimate $\theta_\text{MAP}$ and its HDI.
However, confidence intervals from \cref{sec:frequentist-perspective} and the HDIs have very different interpretations, which echo the differences in frequentist and Bayesian thinking:
The confidence intervals\index{Confidence interval} imply that, if we were to repeat our experiment $100$ times, the true value for $\theta$ would be covered by the CIs $95$ out of $100$ times.
In contrast, the HDI\index{Highest density interval} draws a conclusion about the range of values be believe the true parameter value to lie in, based on our prior belief updated using our actual observations.

\paragraph{Predictive Uncertainty.} \index{Uncertainty!Predictive}
So far, we have discussed the uncertainty in our parameter estimate, but Bayesian statistics\index{Statistics!Bayesian} also provides a useful tool to reason about new observations: Predictive distributions\index{Predictive distribution}.
Let us assume we would like to make a prediction about the observation $x^\prime$ stemming from a new coin flip.
We can write this probability as follows:

\begin{equation}\label{eq:predictive-prior}
    p(x^\prime) = \int_\Theta p(x^\prime\mid\theta)p(\theta)\dd\theta.
\end{equation}

This is referred to as the \emph{prior predictive distribution}\index{Predictive distribution!Prior}, which gives us an instrument to reason about the outcome using the specified prior alone, disregarding any observations.
One way to interpret this distribution is as a weighted aggregate of predictions for $x^\prime$ using different values of $\theta$, which are weighed according to our prior belief.
In the case of the Bernoulli and Beta distribution in the coin flip sample, this distribution has an analytical form:

\begin{align}\label{eq:predictive-prior-coin-flip}
    P(x^\prime) & = \int_\Theta P(x^\prime\mid\theta)p(\theta)\dd\theta \\
    & = \int_0^1 \theta^{x^\prime}(1 - \theta)^{(1 - x^\prime)}\frac{1}{\text{B}(\alpha_1, \alpha_2)} \theta^{\alpha_1 - 1}(1 - \theta)^{\alpha_2 - 1}\dd\theta \\
    & = \frac{1}{\text{B}(\alpha_1, \alpha_2)} \int_0^1 \theta^{\alpha_1 + x^\prime - 1}(1 - \theta)^{(\alpha_2 - x^\prime)}\dd\theta\\
    & = \frac{\text{B}(\alpha_1 + x^\prime, \alpha_2 - x^\prime + 1)}{\text{B}(\alpha_1, \alpha_2)},
\end{align}

\noindent where the last step used the fact the Beta function\index{Beta function} can be expressed as $\text{B}(\alpha_1, \alpha_2) = \int_0^1 x^{\alpha - 1}(1 - x)^{\alpha_2 - 1} \dd x$ (see \cref{app:relationship-beta-gamma} for more details).
While it is useful to check whether a chosen prior is suitable for a given task, the prior predictive does not take any observations into account yet, so we would usually consider a predictive distribution given some available data.
This is the purpose of the \emph{posterior predictive distribution}\index{Predictive distribution!Posterior}, defined as 

\begin{equation}\label{eq:predictive-posterior}
    p(x^\prime \mid \mathbb{D}) = \int_\Theta p(x^\prime \mid \theta)p(\theta \mid\mathbb{D})\ddd\theta.
\end{equation}

Again, we arrive at a prediction by ``averaging'' predictions made using different values of $\theta$.
Since not all values of $\theta$ are equally plausible given our data, they are furthermore weighted by their probability under the posterior $p(\theta \mid \mathbb{D})$\index{Posterior distribution}.
%We can also see how the frequentist approach relates to \cref{eq:predictive-posterior}:
%In this case, instead of considering a posterior distribution over value of $\theta$, the frequentist approach relies on a single estimate, which we can express through a Dirac delta function:
In a frequentist analysis, we only consider a single point estimation of $\hat{\theta}$ instead of a distribution.
In terms of \cref{eq:predictive-posterior}, this can be expressed with the help of a Dirac delta function\index{Dirac delta function}:

\begin{equation}\label{eq:predictive-posterior-frequentist}
    p(x^\prime \mid \mathbb{D}) \approx \int_\Theta p(x^\prime \mid \theta)\delta(\theta - \hat{\theta})\dd\theta = p(x^\prime \mid \hat{\theta}),
\end{equation}

\noindent where we recover using only a single estimate $\hat{\theta}$ for our prediction.
Back to the coin flip example, we can apply a similar argument as in \cref{eq:predictive-prior-coin-flip} with the posterior instead of the prior to arrive at 

\begin{equation}\label{eq:predictive-posterior-coin-flip}
    P(x^\prime \mid \mathbb{D}) = \frac{\text{B}(x^\prime + \alpha_1^{(N)}, 1 - x^{\prime} + \alpha_2^{(N)})}{\text{B}(\alpha_1^{(N)}, \alpha_2^{(N)})}.
\end{equation}

Interestingly, we can interpret the two terms in the right-hand side of \cref{eq:predictive-posterior} as two different sources of uncertainty:
The aforementioned uncertainty about the true value of $\theta$ given observed data is encoded in $p(\theta \mid\mathbb{D})$, and the uncertainty about $x^\prime$ given a fixed parameter value in $p(x^\prime \mid \theta)$.
This interpretation gives rise to the distinction of \emph{data} (or \emph{aleatoric})\index{Uncertainty!Aleatoric} uncertainty and \emph{model} (or \emph{epistemic}) uncertainty\index{Uncertainty!Epistemic}.
The former usually refers to \emph{irreducible} uncertainty that is inherent to the phenomenon we would like to model, like inherent ambiguity or unavoidable noise, and refers to $p(x^\prime \mid \theta)$.
The latter describes our uncertainty about the correct model parameters and resides in $p(\theta \mid\mathbb{D})$.\footnote{Here, this categorization is approached from a general standpoint.
We will discuss how these notions of uncertainty materialize in a language context in \cref{sec:uncertainty-nlp} and point out some problems and nuances.}
The more data we gather, the more we assume the posterior to be concentrated on only the most plausible parameter values, and thus the uncertainty is reduced.
In the frequentist approach, tools like confidence intervals\index{Confidence interval} can only tell us about the total uncertainty of our estimate. 
In the Bayesian approach however, these different notions of uncertainty are represented by different distributions.
%Here, w%e arrive at a prediction by ``averaging'' predictions made using different values of $\theta$.
%Since not all values of $\theta$ are equally plausible given our data, they are furthermore weighted by their probability under the posterior $p(\theta|\mathbb{D})$.
%These two different distributions mirror the aspects of aleatoric and epistemic uncertainty:
%The posterior $p(\theta|\mathbb{D})$ quantifies our model uncertainty, since the distribution will peak around the most likely value of $\theta$ in our coin flip model the more data we observe.
%The likelihood $p(x^\prime|\theta)$ represents the data uncertainty: 
%In cases where the value of $\theta$ is extreme (say, close to $0$ and $1$), we have little uncertainty about the outcome of the new coin flip (namely almost surely only tails or heads).
%In cases where the value of $\theta$ is closer to $0.5$ like in our initial example, we will never a good idea about which side the coin will land on---no matter how many observations we collect;
%The uncertainty is irreducible.
These considerations are the basis for Bayesian deep learning methods, which we will discuss more in \cref{sec:bayesian-neural-networks}.

\paragraph{Recap.} We have seen in this section how Bayesian statistics\index{Statistics!Bayesian} takes a very different approach to uncertainty than frequentist statistics\index{Statistics!Frequentist}: 
In frequentist statistics, probabilities are seen as relative frequencies of an event as we repeat an experiment. 
In Bayesian statistics, this interpretation is abandoned in favor of an viewpoint that sees probabilities as the degree of belief in an event, and parameters of interest becoming random variables instead of unobserved constants.
It allows us to specify a prior belief which is updated using observations, and which diminishes in importance as we encounter more and more data.
Furthermore, we can use predictive distributions to reason about unseen outcomes.
In the posterior predictive distribution, we can also distinguish two kinds of uncertainty:
Irreducible data uncertainty and model uncertainty\index{Uncertainty!Aleatoric}\index{Uncertainty!Epistemic}, reducible by obtaining more data.\\

So far we have only discussed view of uncertainty from the perspective of statistics, defining models that explain observations and make new predictions.
Despite their usefulness, it is no obvious how to apply these statistical models to phenomena as complex as human language, which we turn to next.

\subsection{The Linguistic Perspective: Underspecification, Ambiguity \& Vagueness}\label{sec:uncertainty-linguistics}

%In this section I want to summarize some of the linguistic perspectives on uncertainty. 
%This will aid us in distilling the NLP-specific challenges for uncertainty quantification later (see for instance \cref{ch:uncertainty-generation}).
Linguistics\index{Linguistics} can be categorized into multiple sub-disciplines that are concerned with different aspects of human language \citep{akmajian2017linguistics}.
This thesis focuses on written language, which is why we will not discuss any uncertainty in e.g.\@ \emph{phonetics}\index{Phonetics} and \emph{phonology}\index{Phonology} (the studies of the production of sounds and how they are organized in a language).
Instead, we focus on the following three levels: \emph{semantics}\index{Semantics}, \emph{syntax}\index{Syntax} and \emph{pragmatics}\index{Pragmatics}.
%Since each of these topics could fill an entire dissertation, these sections include a brief overview that inform later sections.
In linguistics, uncertainty appears through different phenomena, for instance ambiguity\index{Ambiguity} or polysemy\index{Polysemy} \citep{tuggy1993ambiguity, kennedy2011ambiguity}, underspecification\index{Underspecification} \citep{pustejovsky1998generative, pustejovsky2017semantics} and vagueness\index{Vagueness} \citep{tuggy1993ambiguity, brown2005encyclopedia, kennedy2011ambiguity},
which manifests in different ways in different linguistic levels.
This creates uncertainty by creating multiple different interpretations of a sentence, which are often---but not always---resolved through additional context, either linguistic, situational or from world knowledge.
Describing this interplay between uncertainty and resolve on different linguistic levels is goal of this chapter.

\paragraph{Uncertainty in Semantics.}\index{Uncertainty} The field of semantics\index{Semantics} is concerned with the literal meaning of words and the ways in which these are combined \citep{kearns2017semantics}.
One way in which uncertainty arises in semantics is \emph{polysemy}\index{Polysemy}, a phenomenon where two or more distinct senses are associated with the same word \citep{gries2015polysemy}.
\citeauthor{gries2015polysemy} for instance mentions the examples of ``I emptied the glass'' compared to ``I drank a glass'', where \emph{glass} corresponds in the first case to a container, and to its content in the second.
A more subtle case of polysemy is exemplified by the examples 

\begin{enumerate}[label=(\alph*)]
    \item Jocelyn walked to the school.
    \item The concerned mother talked to the school.
\end{enumerate}

\noindent where ``school'' in the former refers to the physical building, and the latter to the an administrative unit inside the organization that operates within the school building \citep{frisson2009semantic}.
Resolving these cases can be highly non-trivial, leading in NLP\index{Natural language processing} to the field of \emph{word sense disambiguation}\index{Word sense disambiguation} (see e.g.\@ \citealp{schutze1997ambiguity, agirre2007word, navigli2009word}).
Another case is \emph{homonymy}\index{Homonymy}, where two unrelated meanings map onto the same form \citep{devos2003semantic}, as in the case of \emph{bank} as a financial institute, a place for sitting, or the terrain alongside a river bed.
\emph{Vagueness}\index{Vagueness} can be defined in contrast to these notions as whether ``a piece of semantic information is part of the underlying semantic structure of the item, or the result of a contextual specification'' or simply ``the notion that certain features are not expressed in a representation'' \citep{frisson2009semantic, geeraerts1993vagueness}.
In their example, they show how for ``\emph{my neighbor is a civil servant}'', \emph{neighbor} is not ambiguous since it does not require disambiguation in the given context, despite the word being underspecified (i.e.\@, the neighbor's gender is for instance underspecified).
Vagueness\index{Vagueness} and underspecification\index{Underspecification} are ubiquitous in language, since terms like \emph{tall} or \emph{red} are gradual and highly subjective terms \citep{brown2005encyclopedia} or simply because a speaker (or listener) is lacking information \citep{williamson2002vagueness}. 
In addition, the meaning of some words might be underspecified unless or because it is combined with other words \citep{pustejovsky1998generative}.
The principle of \emph{compositionality}\index{Compositionality} states that the meaning of a more complex expression depends---completely or at least in part---on the meaning of its constituents \citep{fodor2001language, szabo2004compositionality, brown2005encyclopedia}.
While composition of simpler to more complex expressions can help to resolve underspecification\index{Underspecification} (``\emph{my \textbf{female} neighbor is a civil servant}''), it can also create new underspecification, for instance through multiple quantifiers or prepositional phrases with multiple attachments (\citealp{pustejovsky2017semantics}, see next paragraph for an example).

%\citet{pustejovsky2017semantics} list three different categories, namely \emph{weak structural underspecification}, when underspecification is caused by the composition, \emph{weak lexical underspecification}, which is created by accidental ambiguity, and \emph{strong lexical underspecification}, where underspecification is resolved through specification.
%Therefore, the glass and school examples are instances of strong lexical underspecification.
%\extodo{Talk about compositionality and its limits, see linguistic handbook p. 1595 - 1600}

  \begin{figure}[htb]
    \centering
    \subfloat[Parsing \emph{duck} as a noun.]{
        \label{subfig:parse-trees-ambiguity1}
        \begin{forest}
            [S    [NP [I]]
              [VP      [V [saw]]
                [NP         [Pron [her]]
                  [N [duck]]
                ]
              ]
            ]
          \end{forest}
    }
    \subfloat[Parsing \emph{duck} as a verb.]{
        \label{subfig:parse-trees-ambiguity2}
        \begin{forest}
            [S    [NP [I]]
              [VP      
                [VP [V [saw]]]
                [NP [Pron [her]]]
                [VP [V [duck]]]
              ]
            ]
          \end{forest}
    } 
    \caption[Parse trees for a sentence with structural ambiguity.]{Two equally valid parse trees for the sentence ``I saw her duck'' using a constituency grammar. }\label{fig:parse-trees-ambiguity}
\end{figure}
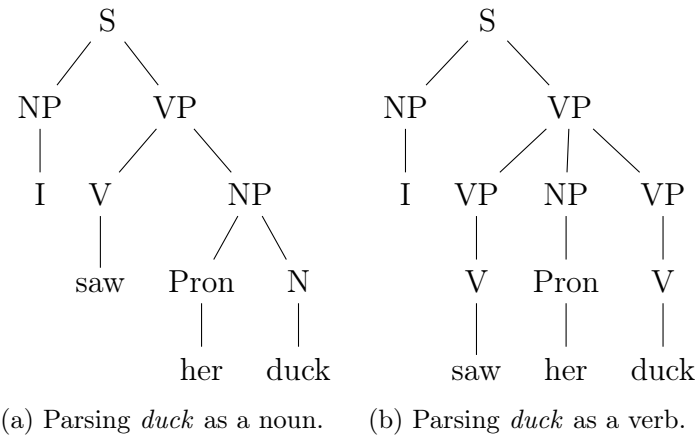

\paragraph{Uncertainty in Syntax.} Syntax\index{Syntax}\index{Uncertainty} describes the machinery that combines the meaning of words and subwords into bigger units, such as phrases and sentences \citep{koeneman2017introducing}. 
In order to model this system, different grammatical formalisms have been proposed (\citealp{varile1997survey};\@ section 3.3), which describe sets of rules that analyze a sentence in terms of a hierarchical structure that describes the relationship between words.
These include \emph{constituency grammars}\index{Constituency grammar}, which will be used for illustrative purposes here.
The core idea of this concept lies in observation that words can behave as either single units, or clump together to comprise units of meanings, called constituents \citep{jurafsky2022speech}.
Constituency grammars\index{Constituency grammar} describe the rules according to which these constituents combine into more and more complex units of meanings.
For instance, the phrase ``the duck'' consists of a \emph{determiner} (Det), or article, ``the'', as well as a noun (N), ``duck''. 
Together, they are denoted as a \emph{noun phrase}, or simply NP.
In the same fashion, we can assign categories like pronoun (Pron), verb (V) and verb phrase (VP), that culminate in a sentence (S). 
An example of an analysis using a constituency grammar is given in \cref{fig:parse-trees-ambiguity}:
Here, the words in the sentence ``I saw her duck'' are combined along these rules.\footnote{
    These rules can also defined more formally in the form of a \emph{context-free grammar}, where rules are applied regardless of a context.
    The exact rules are omitted here for the sake of clarity, but it should be noted that is a simplifying assumption, as natural language is not context-free \citep{savitch2012formal}.
}
However, the word \emph{duck} can be read both as the action of suddenly crouching and a word describing aquatic fowl.
In this former interpretation, ``her'' is read as an object instead of a possessive pronoun.
The corresponding parse tree is given in \cref{subfig:parse-trees-ambiguity1}.
The alternative reading as a possessive pronoun is shown in \cref{subfig:parse-trees-ambiguity2}.
By themselves, the two parse trees might be equally valid grammatical analyses given a constituency grammar.
This implies that this \emph{structural ambiguity}\index{Ambiguity!Structural} is unresolvable without any further context or world knowledge.
Structural ambiguity can arise in a variety of situations depending on the language in question (see for instance \citealp{taha1983types} for examples in English).
\cref{fig:parse-trees-ambiguity} depicts an attachment ambiguity\index{Ambiguity!Attachment}: It is unclear whether \emph{her} and \emph{duck} attach as a combined NP to the VP of \emph{saw}, or whether all three parts are equal constituents of a combined VP.
Other popular examples include the attachment of (specifically) propositional phrases (``I saw the man with the telescope''; \citealp{schutze1995pp, hindle1993structural}) or coordination (``old men and women''; \citealp{frazier2000processing, engelhardt2010processing}).
Uncertainty\index{Uncertainty} can also appear in the processing of language when awaiting additional context.
A famous example of this are \emph{garden path sentences}\index{Garden path sentences}, i.e.\@ sentences that contain surprising syntactical elements that require a re-analysis of the sentence structure thus far \citep{sturt1999structural}.
The most famous example is the sentence ``the horse raced past the barn fell'', for which the corresponding syntactical parse trees are shown in \cref{fig:parse-trees-garden-path}.
Before observing the last word, the sentence in \cref{subfig:parse-trees-garden-path1} exhibits a simple structure of subject (``the horse''), verb (``raced'') and a prepositional phrase (``past the barn'').
After encountering ``fell'', we realize that ``raced'' was indeed not the main verb of the sentence, and instead is used to describe that the horse that fell did so after having raced past the barn.
Structurally, this requires the VP of ``raced'' in \cref{subfig:parse-trees-garden-path2} to be grouped under the subject NP, and a new VP to be created for ``fell''.
Experimental evidence has shown that such ambiguous or challenging constructions can lead to an increase in human reading and processing times (see e.g.\@ \citealp{milne1982predicting,ferreira1991recovery,swets2008underspecification}), suggesting that some form of re-analysis might occur.\footnote{Interestingly, similar effects have been observed in neural models \citep{van2018modeling, irwin2023bert}, although the relationship is weaker in recent transformer models \citep{oh2023does, oh2024frequency}.}

\begin{figure}[htb]
    \centering
    \subfloat[Parse before the last word.]{
        \label{subfig:parse-trees-garden-path1}
        \resizebox{0.45\textwidth}{!}{%
        \begin{forest}
            [S    [NP      [Det       [the]
                ]
                [NN        [horse]
                ]
            ]
            [VP      [V        [raced]
                ]
                [PP        [IN          [past]
                ]
                [NP          [Det            [the]
                    ]
                    [NN            [barn]
                    ]
                ]
                ]
            ]
            ]
        \end{forest}%
        }
    }
    \subfloat[Parse after the last word.]{
        \label{subfig:parse-trees-garden-path2}
        \resizebox{0.45\textwidth}{!}{%
        \begin{forest}
            [S  [NP [NP      [Det        [the]
                    ]
                    [NN        [horse]
                    ]
                    ]
                    [VP      [V        [raced]
                        ]
                        [PP        [IN          [past]
                        ]
                        [NP          [Det            [the]
                            ]
                            [NN            [barn]
                            ]
                        ]
                        ]
                    ]
                ]
                [VP      [V        [fell]
                    ]
                ]
            ]
        \end{forest}
        }
    }
    \caption[Parse trees for a garden path sentence, before and after the last word.]{Parse trees for the garden path sentence ``The horse raced past the barn fell'', before and after adding the last word, prompting a re-analysis of the sentence, where ``raced past the barn'' attached to the NP and ``fell'' becomes the new main verb.}\label{fig:parse-trees-garden-path}
\end{figure}
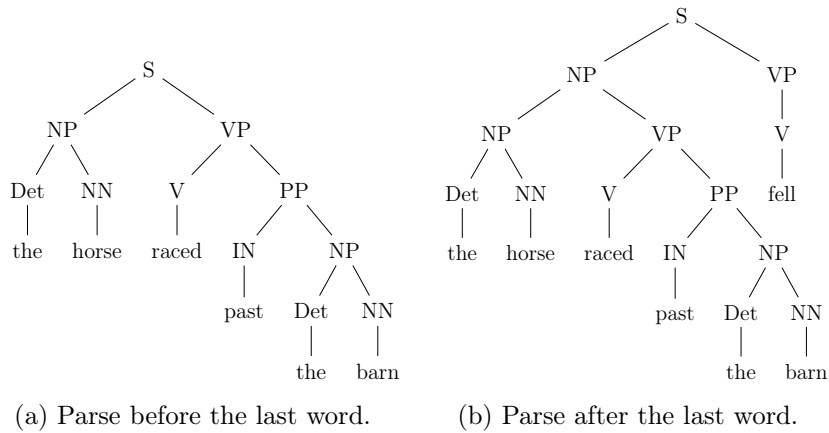

\begin{figure}[htb]
    \centering
    \includegraphics[width=0.6\textwidth]{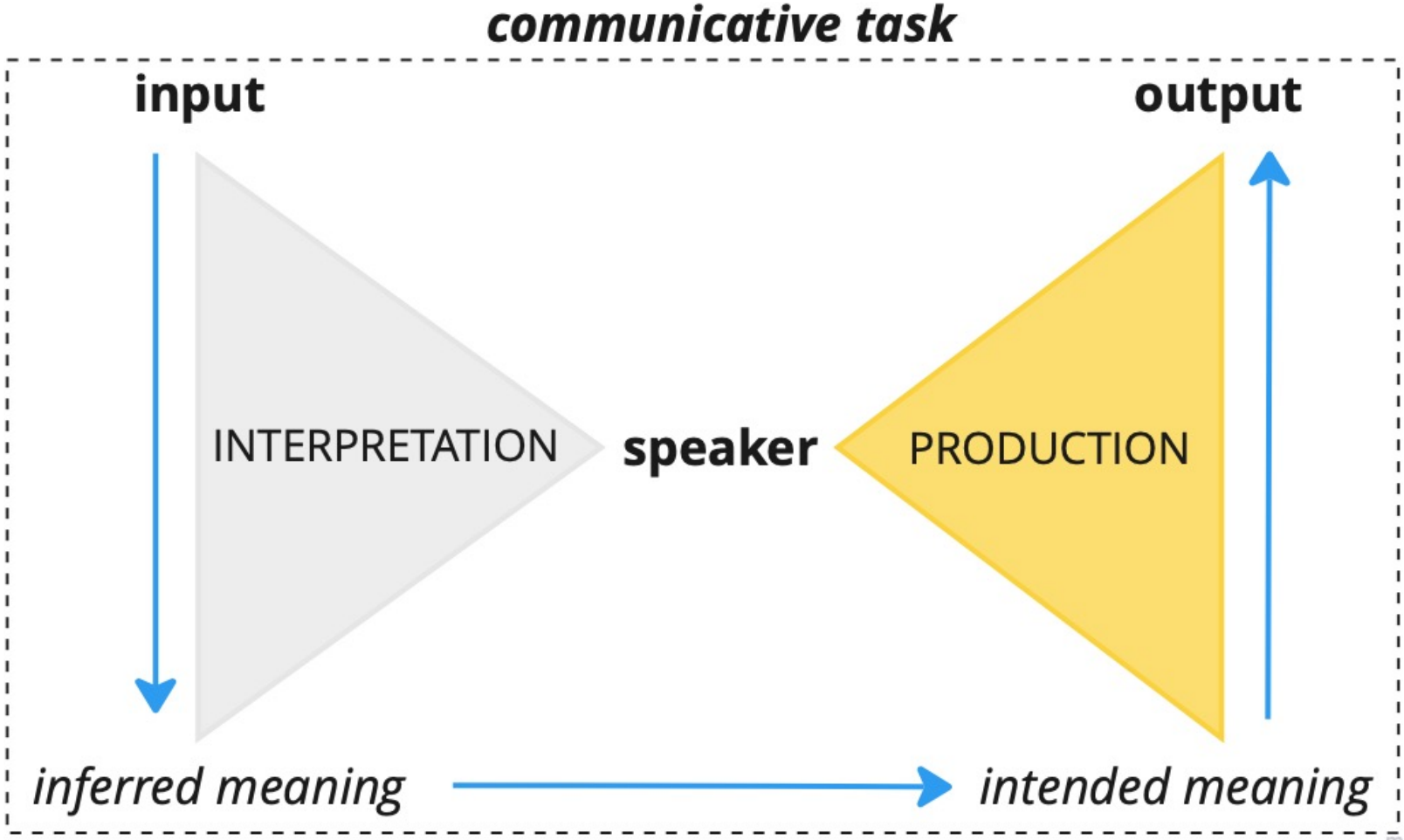}
    \caption[Double triangle of language production.]{Double triangle of language production by \citet{baan2023uncertainty} as an extension of the triangle of reference by \citep{ogden1923meaning}.}
    \label{fig:double-triangle}
\end{figure}

\paragraph{Uncertainty in Pragmatics.}\index{Uncertainty}\index{Pragmatics} Pragmatics can be defined as the study of language in use, especially in social interactions and speech \citep{mey2006pragmatics, huang2014pragmatics}.
Compared to semantics\index{Semantics}, it also studies how word meanings are affected in the context of a specific utterance \citep{kearns2017semantics}.
\citet{baan2023uncertainty} demonstrate its connection to uncertainty, specifically in natural language generation\index{Natural language generation}, through an extension of the ``triangle of reference''\index{Triangle of reference} by \citet{ogden1923meaning}, which is shown in \cref{fig:double-triangle}: 
Given an input to the speaker, there is a potentially wide set of possible inferred meanings;
this can be caused by errors, underspecification\index{Underspecification} (for instance where in some language the gender of a subject is not specified explicitly) or ambiguities\index{Ambiguity} of syntactical or semantic nature, as discussed in the previous paragraphs.
This mapping from utterance to meaning is therefore not one-to-one, but rather one-to-many \citep{grice1957meaning, kennedy2011ambiguity}. 
As the speaker prepares their utterance, they then choose one of a variety of similar or even equivalent meaning to express the intended utterance.
This production process is influenced by the speaker's social and cognitive idiosyncrasies \citep{levelt1993speaking}. 
We will refer to these two sources of uncertainty as \emph{input} and \emph{output variability} or \emph{paraphrasticity}\index{Variability}\index{Paraphrasticity} in the rest of thesis.
This describes an important difference between language and other modalities: Since language is \emph{paraphrastic}, there are (almost) equally valid ways to express the same intended meaning, which however might differ completely in their realizations, i.e.\@ wordings.\footnote{
    Some works argue against the concept that two expressions can be fully equivalent;
    for instance \citet{widoff2022equivalence} points out how two expressions can be equal in some form, but unequal in others (e.g. \@ ``the water in the glass'' and ``the glass is half-full'' convey a similar meaning, but the second one does not specify the content) or how for instance instruments like passive voice can be be used to convey intent (``Hans beats Peter'' vs.\@ ``Peter is beaten by Hans'').
}

\subsection{The Linguistic Perspective: Expressing Uncertainty}\label{sec:expressing-uncertainty}

Besides paraphrastic language\index{Paraphrasticity}, a different type of uncertainty\index{Uncertainty} lies in explicit uncertainty expressions by the speaker. %, similar to the word or phrase-level strategies discussed in the semantics paragraph.
This spans the overall tone of a series of utterances to the usage of diverse linguistic expression (see e.g.\@ \citealp{rubin2006identifying} pp.\@ 21--40; \citealp{lorson2023epistemicity}, \citealp{zhou2023navigating}), for instance \emph{hedges}\index{Hedging} \citep{lakoff1973hedges, fraser1975hedged, prince1982hedging, holmes1982expressing}, i.e.\@ words or phrases to express ambiguity or uncertainty.
Additionally, uncertainty expressions might also be chosen circumstantially, for instance based on whether the other speaker is cooperative or uncooperative \citep{lorson2021strategic}, politeness \citep{sirota2015direct, holtgraves2016politeness} or power differences between speakers \citep{bonnefon2006tactful}.\\
% Other kinds of semantic uncertainty arise through ways in which a speaker expresses their own uncertainty about the world (or possible worlds).

\begin{figure}[htb]
    \centering
    \definecolor{nodecolor}{RGB}{240,240,240}
    \definecolor{edgecolor}{RGB}{150,150,150}
    \begin{tikzpicture}[
        level distance=3em,
        every node/.style = {
            shape=rectangle,
            rounded corners,
            draw,
            align=center,
            inner sep=2mm,
            font=\footnotesize,
            opacity=0.8,
            fill=nodecolor,
            draw=edgecolor,
            %top color=nodecolor,
            %bottom color=white,
            %shading angle=180, 
            anchor=north,
        },
        edge from parent/.style={draw,-latex}
    ]
    \node (Root) [sibling distance=6cm] {\textsc{Semantic Uncertainty}}
        child [xshift=-2cm] { node (Category 1) {\textsc{Epistemic} \\ {\scriptsize Lack of knowledge;}\\{\scriptsize Neither true or false}}}
        child [sibling distance=4cm] { node {\textsc{Hypothetical} \\ {\scriptsize Possibilities that}\\ {\scriptsize can be true or false}}
            child [sibling distance=1cm, xshift=-4cm] { node (Paradoxical) {\textsc{Paradoxical}}
                child { node (Investigative) {\textsc{Investigative}\\{\scriptsize Uncertain until}\\{\scriptsize investigated}} }
                child { node [right=of Investigative, xshift=-0.5cm, yshift=-0.2cm] (Conditional) {\textsc{Conditional} \\ {\scriptsize Uncertainty about}\\{\scriptsize conditions; if / else}\\{\scriptsize expressions}}}
                edge from parent [bend right]
            }
            child [sibling distance=2.5cm, xshift=2cm] { node {\textsc{Non-epistemic}}
                child [xshift=-0.5cm] { node (Doxastic) {\textsc{Doxastic} \\ {\scriptsize Expression}\\ {\scriptsize of beliefs}}}
                child { node {\textsc{Dynamic} \\ {\scriptsize Duties, plans,}\\{\scriptsize desires}} }
            }
        };
    \end{tikzpicture}
    \caption[Taxonomy of semantic uncertainties.]{Taxonomy of different semantic uncertainties adapted from \citet{kolagar2024aligning}, originally based on the work by \citet{szarvas2012cross}.}\label{fig:uncertainty-taxonomy}
\end{figure}

\cref{fig:uncertainty-taxonomy} shows a taxonomy of semantic uncertainties\index{Uncertainty!Semantic} by \citet{kolagar2024aligning}, based on the works of \citet{szarvas2012cross, vincze2014uncertainty}.
It proposes a categorization of expressed uncertainty based on the truth value of an utterance.
The taxonomy divides semantic uncertainty first into \emph{epistemic}\index{Uncertainty!Epistemic},\footnote{Not be confused with the epistemic or model uncertainty in a statistical sense in \cref{sec:bayesian-perspective}.} where the speaker expresses worlds which are neither true or untrue and do not coincide with their actual world.
To make this notion less abstract, consider the following example:
In the sentence ``This is the best dessert I have ever had'', we take the sentence to be a fact, and therefore assign a positive truth value.
Now, we can instead use a modal verb to say ``This may be the best dessert I have ever had''.
While one can imagine a possible world in which this statement is true, we cannot assign it a definitive truth value per se.
The alternative branch in the taxonomy are \emph{hypotheticals}\index{Hypotheticals}, which can also be uncertain, but in contrast to epistemic uncertainty\index{Uncertainty!Epistemic}, also have the possibility of being evaluated as true or false.
One bifurcation, \emph{paradoxical}\index{Uncertainty!Paradoxical}, refers to cases in which the truth value depends on another propositions, for instance if / else expressions (\emph{conditional})\index{Uncertainty!Conditional} or cases in which the truth value can only be evaluated after further examination (\emph{investigative}\index{Uncertainty!Investigative}).
The other fork, \emph{non-epistemic}\index{Uncertainty!Non-epistemic}, describes circumstances in which a speaker expresses beliefs (\emph{doxastic})\index{Uncertainty!Doxastic} and duties, plans or desires (\emph{dynamic})\index{Uncertainty!Dynamic}.\\

\paragraph{Recap.} We have now discussed a variety of sources of uncertainty in linguistics\index{Linguistics}\index{Uncertainty}, located on different levels of language use, including semantics\index{Semantics}, syntax\index{Syntax} and pragmatics\index{Pragmatics}.
These discussions were mostly informed by phenomena in the English language and are thus limited, as other types of ambiguity\index{Ambiguity} exists that were not discussed here (see for instance \citealp{li2024taxonomy}).
We can nevertheless distill certain insights:
On the one hand, uncertainty arises as an inherent property of language, through polysemy\index{Polysemy}, structural ambiguities\index{Ambiguity!Structural} or possible paraphrases\index{Paraphrasticity}.
On the other hand, uncertainty is a tool that be employed by a speaker to express their own uncertainty and to express the state of potential worlds.
In both cases, this creates challenges for any processing system that operates on language and tries to infer its meaning.
% But before we shine a light on currently available deep learning models in \cref{sec:uncertainty-deep-learning}, I will try to answer the overarching question this part of the chapter tries to address. 

\subsection{A Pragmatic Answer}\label{sec:uncertainty-recap}

%\epigraph{``\emph{Probability does not exist}''}{---Bruno de Finetti: \emph{Theory of Probablity - A Critical Introductionary Treatment} (1974).}

The astute reader might have noticed that while the title of \cref{sec:what-is-uncertainty} was ``what is uncertainty, anyway?''\index{Uncertainty}, it might appear that we have thus far been tiptoeing around this question, enumerating and explaining different perspectives to it without giving a satisfying answer.\\

In the end, uncertainty is a multifaceted and perhaps vague concept, whose definition varies based on the phenomenon of interest.
At its very core, it describes a lack of knowledge about the true state of the world among competing alternative states.
The definitions of these world states can differ tremendously on the context, and can include all the possible interpretations of the sentence ``I saw her duck'' to plausible values of a data-generating parameter $\theta$.
For the purpose of this thesis, we reduce its definition to the following aspects:
Firstly, there is the uncertainty that is inherent to language described in \cref{sec:uncertainty-linguistics}, describing how interpretation and production are not a one-to-one processes of meaning;
Secondly, the statistical models we apply to language are themselves faced with multiple possible specifications and can produce different potential predictions.
As these models are at best informative but incomplete abstractions of reality that are fit on finite data, we accept their uncertainty as the price for practicality.
While the last two points refer to uncertainty as phenomena, however and thirdly, uncertainty is also a tool:
It enables us to reason about and express our own knowledge about possible states, and convey our lack thereof.
This notion captures both the statistical sense, like considering different parameter values in the posterior predictive distribution\index{Predictive distribution!Posterior}, as well as making conditional statements or using linguistic modifiers to convey our beliefs in natural language.
As NLP\index{Natural language processing} involves different kinds of uncertainty both in its modeling tools and data modality, this creates an intricate interplay between these uncertainties.\\

%My argument that arises from the reflections in this chapter so far is that, specifically with respect to possible applications, that \textbf{uncertainty is not a concept to be defined, but rather a tool be used}.
%In the opening quote of this section, Bruno de Finetti also poses in his treaty on the theory of probability that there is no objective notion of probability, only a subjective interpretation, a notion further echoed by \citet{lindley2013understanding}.
%We can transfer this idea into our domain using the following thought:
%Natural language processing aims to \emph{model} language data.
%A model is an informative abstraction of reality, but not a complete one.
%We as modelers can thus define the kinds of uncertainty that we are most interested in to inform our modelling choices, which is underscored by the variety of existing definitions of and approaches to uncertainty.\\

So far, we have looked at uncertainty in a fashion that is completely independent from neural networks\index{Neural network}, the core modeling tool of this thesis and main workhorse of contemporary artificial intelligence\index{Artificial intelligence}.
Natural language processing specifically has adopted the use of large neural models operating on language inputs (and sometimes also outputs). 
It thus lies in the intersection of linguistics\index{Linguistics} and statistics\index{Statistics}, and we will review the implications on uncertainty modeling next.

%The next two sections, \cref{sec:uncertainty-deep-learning,sec:uncertainty-nlp}, are aimed to rectify this by discuss how the previous statistical concepts come into play when looking at uncertainty in deep learning in a very general way;
%which is followed up by uncertainty in natural language processing, building on the sections regarding uncertainty in linguistics in \cref{sec:uncertainty-linguistics}.

\section{Uncertainty in Deep Learning}\label{sec:uncertainty-deep-learning}

\epigraph{
    ``\emph{In the 1950s and 60s, scientists built a few working perceptrons, as these artifical brains were called. [\ldots] 
    This perceptron is being trained to recognize the difference between males and females. [\ldots] 
    After training on lots of examples, it [\ldots] is able to successfully distinguish male from female. 
    It has learned. While promising, this approach to machine intelligence has virtually died out.}''
}{---\href{https://www.youtube.com/watch?v=cNxadbrN_aI}{Clip about AI research in the 1950s and 60s}, date unknown.}

\begin{figure}
    \centering
    \vspace{2cm}
    \definecolor{nodecolor}{RGB}{240,240,240}
    \definecolor{edgecolor}{RGB}{150,150,150}
    \resizebox{\textwidth}{!}{
    \noindent\rotatebox{90}{%
    \begin{forest}
        for tree={
          grow'=0,
          rounded corners,
          font=\footnotesize,
          parent anchor=east,
          child anchor=west,
          text width=3.5cm,
          inner xsep=4pt,
          outer xsep=4pt,
          align=left,
          opacity=0.8,
            fill=nodecolor,
            draw=edgecolor,
          edge path={
            \noexpand\path[\forestoption{edge}]
            (!u.parent anchor) -- +(10pt,0pt) |- (.child anchor)\forestoption{edge label};
          },
        }
        [
          Uncertainty Quant.\\in Deep Learning
            [Frequentist
                [Confidence
                    [Calibration]
                ]
                [Prediction Sets
                    [Conformal Prediction]
                ]
            ]
            [Bayesian
                [Ensembling]
                [Sampling
                    [Markov Chain\\Monte Carlo]
                    [Stochastic Langevin\\Dynamics]
                ]
                [Laplace\\Approximations]
                [Variational
                    [Stochastic\\Regularizers]
                    [Bayes-by-backprop]
                ]
                [Deep Kernel Learning]
                [Evidential
                    [Prior Networks]
                    [Posterior Networks]
                ]
            ]
            [Other
                [Heuristic]
                [Direct Prediction
                    [Density-based]
                    [Uncertainty Heads]
                    [Auxiliary Models]
                ]
                [Neural SDEs]
                [Credal Sets]
                [NLP-specific
                    [Black-box methods]
                    [Verbalized Uncertainty]
                ]
            ]
        ]
        \end{forest}%
    }}
\caption[Taxonomy of uncertainty quantification approaches in deep learning.]{
    Hierarchical Taxonomy, showing the different methods discussed in \cref{sec:uncertainty-deep-learning}.
    Note that these shown categories are not necessarily disjoint, as different methods can sometimes be placed into multiple categories at once.
}
\label{fig:taxonomy}
\end{figure}
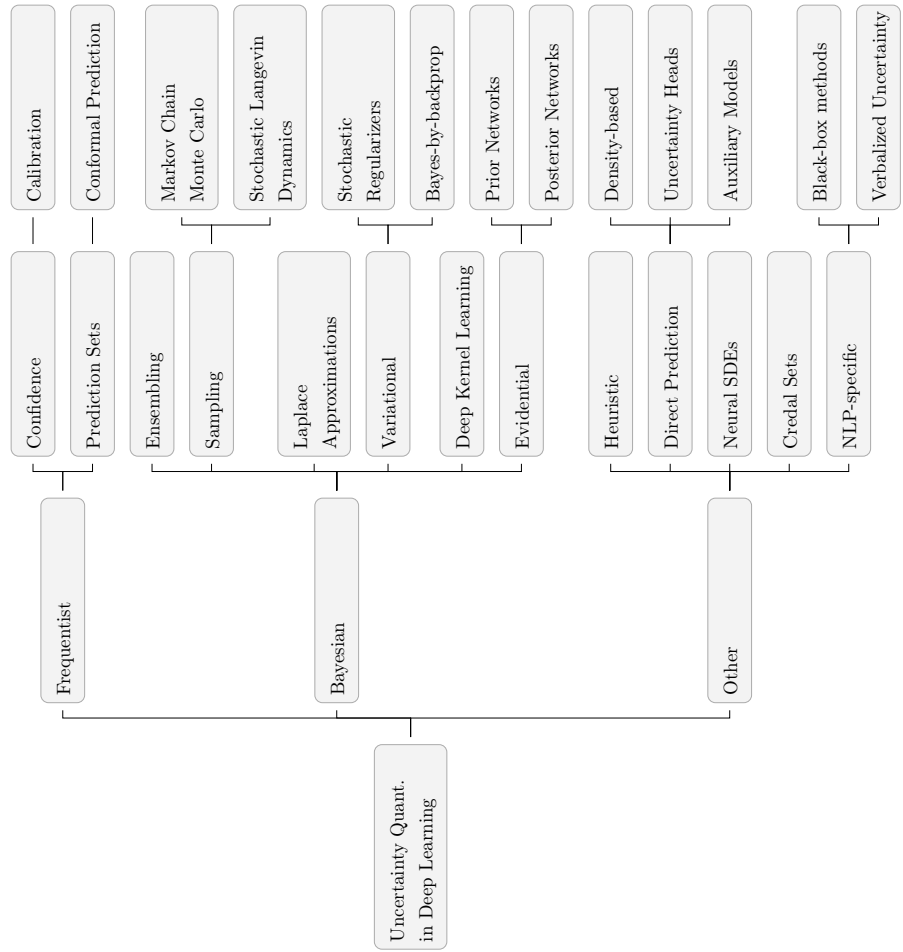

In contrast to the interviewer's quote in the epigraph, the very promising approach of using computational models of neurons did not completely die out, but rather remained dormant for decades.
First known as \emph{cybernetics}\index{Cybernetics} at the time of the first models of artificial neurons \citep{mcculloch1943logical, rosenblatt1958perceptron, rosenblatt1962principles}, it became known as \emph{connectionism}\index{Connectionism} in the 1980--1990s, before assuming its current name \emph{deep learning}\index{Deep learning} in 2006 \citep{goodfellow2016deep}.
Nowadays, deep learning is commonly and vaguely defined as a family of machine learning networks that employ artificial neural networks\index{Neural network} of increasing depth \citep{goodfellow2016deep}.\footnote{
    One might argue that modern NLP might be at least in a large part subsumed by this definition; for this purpose of this thesis we will treat them as overlapping but different disciplines due to their history (see for instance Chapter 1 of \citealp{jurafsky2022speech}) and the peculiarity of language as a data modality, especially compared to images or tabular data.
}
%This section will discuss the ways in which uncertainty in neural networks is quantified, and is structured as follows: 
%We will build up from frequentist neural networks in \cref{sec:frequentist-neural-networks} to Bayesian neural networks in \cref{sec:bayesian-neural-networks} and evidential neural networks in \cref{sec:evidential-neural-networks}.
%The remaining sections concern a set of conceptually simpler approach called direct uncertainty prediction and other approaches that could not be grouped under the previous section but deserve a mention nonetheless (\cref{sec:other-approaches}).
%An overarching taxonomy over uncertainty quantification approaches in deep learning is shown in \cref{fig:taxonomy}.
Uncertainty\index{Uncertainty} in deep learning materializes in a wide variety of approaches, as depicted as a hierarchical taxonomy in \cref{fig:taxonomy}:
As shown in \cref{sec:frequentist-neural-networks}, the frequentist school uses neural networks as powerful estimators of predictive parameters, which can be interpreted similarly to the models in frequentist statistics\index{Statistics!Frequentist} we discussed earlier.
In the same way, Bayesian methods can be applied to neural networks to quantify uncertainty through parameter posterior\index{Posterior distribution} and posterior predictive distributions\index{Predictive distribution!Posterior} (\cref{sec:bayesian-neural-networks}).
Other approaches take inspiration from the Dempster-Shafer theory of evidence\index{Dempster-Shafer theory of evidence} (\cref{sec:evidential-neural-networks}) or draw from entirely different ideas such as framing uncertainty quantification\index{Uncertainty quantification} as a supervised learning task, stochastic differential equations, and more (\cref{sec:other-approaches}).

\subsection{Frequentist Neural Networks}\label{sec:frequentist-neural-networks}

Before we turn to how frequentist methods allow the quantification of uncertainty in neural networks, we first review the similarities in parameter estimation when applied to neural predictors.
In the following, we term the application of frequentist methods to neural network as \emph{frequentist neural networks}\index{Neural network!Frequentist}.\\

As introduced in \cref{sec:frequentist-perspective}, frequentist statistics\index{Statistics!Frequentist} refers to an interpretation of probability as the relative frequency of an event.
In a neural network setting, the estimation of the parameter(s) of interest, in this case the network's parameters $\btheta$, is analogous to the maximum likelihood estimation\index{Maximum likelihood estimate} in \cref{sec:frequentist-perspective}.
The main differences are that firstly, instead of parameterizing a distribution with $\theta$ directly, we parameterize it with the prediction obtained from a neural net with parameters $\btheta$.
Whereas in the coin flipping example, $\theta$ referred to the probability of heads, a neural network in a binary classification setting is equipped with some parameters $\btheta$ now predicts the probability of the positive class $\hat{p}$. 
And secondly, due to the model's non-linear and hierarchical dependencies, the solution to $\btheta$ is not available in closed form anymore.
Instead, we iteratively optimize $\btheta$ through procedures such as gradient descent, where we compute the gradient of some loss function w.r.t.\@ the parameters and take a step in the direction of the (anti-)gradient.
The loss functions vary depending on the intended purpose, but in some cases can be directly related to maximum likelihood estimation\index{Maximum likelihood estimate}.
In analogy to the coin flipping in the previous section, a network trained on a binary classification\index{Classification!Binary} task for instance predicts $\hat{p} = \sigma(f_{\btheta}(\bx))$ (with $\sigma(\cdot)$ denoting the sigmoid function\index{Sigmoid function}) and is then optimized using the binary cross-entropy loss\index{Entropy!Cross-} (here for a single input using a gold label $y \in \{0, 1\}$):

\begin{equation}\label{eq:binary-cross-entropy}
    \mathcal{L}_\text{BCE}(y, \hat{p}) = - y\log \hat{p} + (1 - y)\log (1 - \hat{p}).
\end{equation}

The resemblance to the Bernoulli\index{Likelihood}\index{Bernoulli distribution} log-likelihood in \cref{eq:bernoulli-likelihood} is no coincidence, and we can see that the loss is minimized when the network prediction $\hat{p}$ correspond to the actual probability $p$, e.g.\@ the relative occurrence of the positive class.
Thus we view model predictions, at least for classification problems, through a similar, frequentist lens.\footnote{
    Here, we mainly focus on classification problems, which tend to be more frequent in NLP\index{Natural language processing}.
    Nevertheless, we can also consider a prediction from a trained regressor as frequentist by considering it as the mean of normal distribution with a variance equal to some (inverse and homoskedastic) noise, see for instance \citet{bishop2006pattern}, chapter 3.1.
 }

\paragraph{Confidence \& Calibration.} 
To illustrate frequentist uncertainty estimation further, we now move from a binary classification problem\index{Classification!Binary} to a multi-class classification\index{Classification!Multi-class} problem.
Formally, consider a neural predictor $f_{\btheta}$, a function mapping from an input space $\mathbb{R}^D$ to an output space $\mathbb{R}^K$ and with parameter vector $\btheta$.
Here, $K$ typically refers to the number of classes in a classification problem and the output of the network is referred to as \emph{logits}.
These logits are then normalized, typically by using the softmax function\index{Softmax function} $\bar{\sigma}(\cdot)$, to produce a categorical probability distribution over classes\index{categorical distribution}.
Since each class is now associated with a probability score, we can refer to each of these probabilities as the \emph{confidence} of $f_{\btheta}$\index{Confidence} regarding a certain class, or more formally

\begin{equation}\label{eq:confidence}
    \hat{p}_k = P_{\btheta}(y = k \mid \bx) \equiv \bar{\sigma}(f_{\btheta}(\bx))_k.
\end{equation}

\noindent Also, let

\begin{equation}
    \hat{y} = \argmax_{k \in [K]}\hat{p}_k;\quad \hat{p} = \max_{k \in [K]}\hat{p}_k
\end{equation}

\noindent be the class predicted by the model and its corresponding probability, respectively.
Ideally, a predicted probability of e.g.\@ $45 \%$ for some class would thus indicate that this is the correct prediction, in $45$ out of $100$ times, if we were to repeat the experiment.
We can formulate this requirement as

\begin{equation}\label{eq:calibration}
    p\big(y = \hat{y} \mid \hat{p} \big) = \hat{p}.
\end{equation}
The degree to which this requirement is violated is measured through the expected calibration error\index{Expected calibration error} (ECE;\@ \citealp{naeini2015obtaining}), which is defined as
\begin{equation}\label{eq:calibration-error}
    \text{ECE} = \mathbb{E}\Big[\big|p\big(y = \hat{y} \mid \hat{p} \big) - \hat{p}\big|\Big].
\end{equation}
This expectation can for instance be computed by grouping $N$ test predictions into $M$ equally wide bins according to their confidence $\hat{p}$.
Defining $\mathcal{B}_m$ as the set indices that belong to bin $m$, we can write the ECE as 
\begin{equation}\label{eq:ece}
    \text{ECE} \approx \sum_{m=1}^M \frac{|\mathcal{B}_m|}{N}\Big|\underbrace{\frac{1}{|\mathcal{B}_m|}\sum_{i \in \mathcal{B}_m}\indicator{\hat{y}_i = y_i}}_{\text{Bin accuracy (target)}} - \underbrace{\frac{1}{|\mathcal{B}_m|}\sum_{i \in \mathcal{B}_m}\hat{p}_i}_{\text{Avg. bin confidence}}\Big|, 
\end{equation}
\noindent where $\indicator{\hat{y}_i = y_i}$ is the indicator function showing whether the prediction was correct.
As \citet{guo2017calibration} note, the terms in the difference approximate the left-hand and right-hand side in \cref{eq:calibration} per bin, respectively.
However, the ECE\index{Expected calibration error} has also drawn several points of criticisms:
The number of bins can also distort the results when test points are unequally distributed or an unsuitable number of bins is chosen.
Also, it is not a \emph{proper scoring rule}\index{Proper scoring rule} \citep{savage1971elicitation, gneiting2007strictly}, meaning that it is not necessarily minimized by the true distribution.
For proper scoring rules, the true distribution should constitute the minimum, but for the ECE we can often minimize through a uniform distribution instead.
Therefore, many other alternatives to the ECE\index{Expected calibration error} have been proposed (e.g.\@ \citealp{nixon2019measuring, kumar2019verified, zhang2020mix, gruber2022better, kirchenbauer2022your, roelofs2022mitigating, blasiok2023smooth, chidambaram2024flawed}).\\

Unfortunately, several works have shown that neural network\index{Neural network} models tend to be miscalibrated in general, with a tendency to be overconfident (e.g.\@ \citealp{guo2017calibration, minderer2021revisiting, zhu2023calibration}).
Therefore, a vast library of methods has been proposed to improve the calibration of neural networks\index{Calibration}.
This includes post-processing of predictions, for instance by retraining or adjusting the logits through additional scale and shift parameters \citep{platt1999probabilistic, guo2017calibration, mozafari2019unsupervised, kull2019beyond, wenger2020non, gupta2021calibration, ma2021metacal}.
Others have introduced custom loss functions \citep{mukhoti2020calibrating, karandikar2021soft, bohdal2021meta, ghosh2022adafocal, hebbalaguppe2022stitch, tao2023dual} that are meant to disincentivize overconfidence on a specific class.
This is since performing maximum likelihood estimation\index{Maximum likelihood estimate} of network parameters $\btheta$ with objectives such as \cref{eq:binary-cross-entropy} only seeks to maximize the probability of the true class, but does not incentivize calibration per se.
Other strategies involve tempering with the training data.
Through label smoothing\index{Label smoothing} \citep{szegedy2016rethinking, muller2019does, lukasik2020does, lienen2021label, zhang2021delving, liu2022devil, park2023acls}, where probability mass is dispersed from the ground truth class onto other classes, the network is taught to not assign maximal confidence to the ground truth.
For further regularization, mixup can be used \citep{zhang2018mixup, thulasidasan2019mixup, maronas2021calibration, zhang2022and, noh2023rankmixup, wang2023pitfall}, where the network is trained on interpolations of two inputs.
In this case, both the input representations as well as their gold label distributions are mixed.
Since miscalibration might also stem from a lack of training data, an intuitive way to improve models is data augmentation \citep{hendrycks2019augmix, patel2021manifold}.
It has also been observed that ensembling\index{Ensembling} \citep{lakshminarayanan2017simple, wen2020improving, ashukha2020pitfalls, zhang2020mix, wu2021should, rahaman2021uncertainty, wen2021combining, seligmann2024beyond} and Bayesian modeling approaches \citep{mitros2019validity, maronas2020calibration, izmailov2021what, fortuin2022bayesian} can improve calibration (see \cref{sec:bayesian-neural-networks}).

\paragraph{Prediction Sets \& Conformal Prediction.} 
Instead of simply presenting a single prediction $\hat{y}$, we can also present the most likely outcomes in a \emph{prediction set} instead, similar to a confidence interval.\index{Prediction set}
Let $\alpha \in [0, 1]$ be a hyperparameter controlling the desired width of a prediction set by defining a cutoff for probabilities.
We can then define the prediction set $\mathcal{C}$ for a new point $\bx^\prime$ as 

\begin{align}\label{eq:naive-prediction-set}
    \mathcal{C}(\bx^\prime) & = \Big\{y_{\pi^{-1}(1)}, \ldots, y_{\pi^{-1}(k^\prime)}\Big\}\\
     k^{\prime} & = \sup\Big\{k\ \Big|\ \sum_{j=1}^{k^\prime} \hat{p}_{\pi^{-1}(j)} < 1 - \alpha \Big\} + 1.
\end{align}

The above formulation includes a sorting function $\pi(\cdot)$ that sorts indices $k$ by their corresponding class probabilities $\hat{p}_k$, in a descending order, encompasses the most classes while staying under the probability threshold $1 - \alpha$, and adds one to avoid empty sets.
Therefore, a more intuitive construction is the following: 
We sort all predicted probabilities from highest to lowest, and the select the classes for the prediction set until their sum exceeds a threshold of $1 - \alpha$.
Ideally, we would like prediction sets to fulfil two criteria\index{Prediction set}: 
They should contain the correct answer (\emph{coverage}) and they should be as tight as possible\index{Coverage}.\footnote{
    Since one can always contain the correct answer by having the widest possible prediction sets, evaluating coverage alone is usually not meaningful.
}
$1 - \alpha$ corresponds to the desired probability with which the correct answer should be contained in the prediction set in expectation, similarly how frequentist confidence scores correspond to a probability of correctness under many repetitions of an experiment.
In this way, we can also interpret the width of the set as a proxy for confidence;
The wider the set, the more uncertain the underlying model and the more classes it has to include in order to fulfill a coverage probability of $1 - \alpha$.
Unfortunately, and very similar to confidence scores\index{Confidence}, prediction sets\index{Prediction set} are usually not calibrated by default \citep{kompa2021empirical}.
The analogous solution to the calibration of prediction sets\index{Calibration} is \emph{conformal prediction}\index{Conformal prediction} \citep{vovk2005algorithmic, papadopoulos2002inductive, angelopoulos2021gentle}:
By using a calibration set of data points and following the algorithm shown in \cref{alg:conformal-prediction},\footnote{
    This algorithm displays \emph{split} conformal prediction\index{Conformal prediction!Split}, which can be applied to already trained predictors.
    \emph{Full} conformal prediction however requires the re-training of the predictor on all the leave-one-out subsets of the training set, and is therefore infeasible for many modern settings.
    See for instance \citet{angelopoulos2021gentle}, section 6.

} 
we can determine a probability threshold $\hat{q}$ in the following way:
First, we collect a number of \emph{non-conformity scores}\index{Non-conformity score} $s_i$ on a held-out calibration set that reflect the correctness of the model.
The design of these scores is arbitrary, but should reflect the correctness of a model's prediction for a point, e.g.\@ $s(\bx_i) = 1 - p_{\btheta}(y_i\mid\bx_i)$.
Afterwards we choose $\hat{q}$ as the $\lceil(N+1)(1-\alpha)/N\rceil$-th quantile of the empirical score distribution.
Using $\hat{q}$, our prediction sets now provably contain the correct prediction in expectation with a probability of $1 - \alpha$.
One simple way is to include all classes with a probability higher than $\hat{q}$:

\begin{equation}\label{eq:conformal-prediction-guarantees}
    \mathcal{C}(\bx^\prime) = \big\{y_k\ \big|\ \hat{p}_k \ge \hat{q}\big\},
\end{equation}

\noindent otherwise we can also repeat the construction in \cref{eq:naive-prediction-set}, but replace the $1 - \alpha$ threshold by $\hat{q}$.
Prediction sets in this way then fulfil the following guarantee:

\begin{equation}\label{eq:conformal-guarantee}
    p\big( y^\prime \in \mathcal{C}(\bx^\prime) \big) \ge 1 - \alpha.
\end{equation}

\begin{algorithm}
    \caption{Conformal Prediction}\label{alg:conformal-prediction}
    \begin{algorithmic}
    \Require Calibration data set $\{(\bx_i, y_i)\}_{i=1}^N$, predictor $p_{\btheta}$, non-conformity function $s: \mathbb{R}^D \rightarrow \mathbb{R}$ .
    \State
    \LineComment{1. Retrieve non-conformity scores for calibration points, e.g.\@}
    \State $s_i = s(\bx_i) = 1 - p_{\btheta}(y_i\mid\bx_i)$\\
    
    \LineComment{2. Find quantile $\hat{q}$ using empirical inverse CDF $F_\mathbb{S}^{-1}$}
    \State $\hat{q} \gets\ F_\mathbb{S}^{-1}\big(\lceil(N+1)(1-\alpha)/N\rceil$\big)\\

    \LineComment{3. Create prediction set, e.g.\@}
    \State $\mathcal{C}(\bx^\prime) \gets\ \{y_k\ |\ \hat{p}_k \ge \hat{q}\}$\\
    
    \end{algorithmic}
\end{algorithm}

Conformal prediction\index{Conformal prediction} has enjoyed great interest in recent years, since it is agnostic to the form of the underlying predictor and can therefore easily be applied to neural networks.
Recent work has for instance be dedicated to apply conformal prediction for time series \citep{xu2021conformal, stankeviciute2021conformal, lin2022conformal, zaffran2022adaptive} and other non-i.i.d.\@ settings \citep{gibbs2021adaptive, oliveira2022split, bhatnagar2023improved, barber2023conformal,farinhas2024nonexchangeable}.
It should also be noted that the conformal guarantee in \cref{eq:conformal-prediction-guarantees} can be rewritten in terms of the indicator function:

\begin{equation}\label{eq:conformal-guarantee}
    p\Big(\indicator{y^\prime \in \mathcal{C}(\bx^\prime)}\Big) \ge 1 - \alpha.
\end{equation}

This fact is exploited by \citet{angelopoulos2022conformal} and subsequent works \citep{fisch2022conformal, farinhas2024nonexchangeable, evans2023conformal} to generalize this guarantee to families of functions that go beyond coverage, for instance controlling for false-negative rate \citep{angelopoulos2022conformal, fisch2022conformal, farinhas2024nonexchangeable, evans2023conformal}.

\paragraph{Uncertainty Quantification in Frequentist Networks.} \index{Uncertainty metric}
In the case of prediction sets, their width can be interpreted as a confidence score\index{Confidence}:
When the probability distribution is more uniform, more classes have to be added to the set to reach a specific probability threshold, and thus the set size grows.
Without prediction sets, we turn to the (calibrated) confidence score, which is usually taken to be the maximum probability among all classes \citep{hendrycks2017baseline}.
Alternatively, a popular measure of uncertainty is to compute the Shannon entropy\index{Entropy!Shannon} of the distribution, which is given by

\begin{equation}\label{eq:predictive-entropy}
    \text{H}\big[P_{\btheta}(y \mid \bx) \big] = -\sum_{k=1}^K P_{\btheta}(y=k \mid \bx) \log P_{\btheta}(y=k \mid \bx).
\end{equation}

The entropy is maximal when the distribution is uniform, and conversely its value is minimal when all the probability mass rests on a single outcome.

\subsection{Bayesian Neural Networks}\label{sec:bayesian-neural-networks}

\index{Neural network!Bayesian}After reviewing the frequentist approach to neural networks in the previous section, the question naturally arises whether we can also apply Bayesian thinking in a neural network setting.
This question can be answered affirmatively and has been studied since the 1990s (see e.g.\@ \citealp{tishby1989consistent, mackay1992practical, neal1995bayesian}).\footnote{
    Due to the volume of the corresponding literature, we will restrict ourselves to some core ideas and important works, a brief history of the field can for instance be found in \citet{gal2016uncertainty}, pp.\@ 20--23. 
}
We start by Bayesian parameter estimation for a neural network parameterized by weights $\btheta$. 
By placing a prior $p(\btheta)$ over the weights, we obtain a posterior using Bayes' rule in \cref{eq:bayes-theorem}\index{Bayes' theorem}:

\begin{equation}\label{eq:neural-posterior}
    p(\btheta \mid \mathbb{D}) \propto p(\mathbb{D}\mid \btheta)p(\btheta).
\end{equation}

We can again find the maximum a posteriori estimate\index{Maximum a posterior estimate} like in \cref{sec:bayesian-perspective}, but due to the nature of neural networks, have to resort to an iterative optimization procedure to find the parameters like for the neural maximum likelihood estimate\index{Maximum likelihood estimate} in the previous \cref{sec:frequentist-neural-networks}.
Luckily, we can optimize for $p(\btheta \mid \mathbb{D})$ by simply using a loss function such as in the previous section, and either explicitly or implicitly define a prior $p(\btheta)$\index{Prior distribution}.
Explicitly, this can be performed by for instance sampling the initial values of $\btheta$ from some prior distribution, or implicitly through regularization.\footnote{
    Regularizing the $l_2$-norm of the network parameters for instance corresponds to the use of a isotropic normal prior\index{Normal distribution} (see e.g.\@ \citealp{bishop2006pattern}; Section 3.3.1).
}
While that makes it comparatively easy to find the parameters $\btheta$ that maximize \cref{eq:neural-posterior}, it is much harder to find the analytical form of the posterior $p(\btheta \mid \mathbb{D})$ or to sample from it.
This is because the full form of \cref{eq:neural-posterior} derived via Bayes' rule\index{Bayes' theorem} includes the evidence term $p(\mathbb{D})$\index{Evidence} as normalizing constant, which as shown in \cref{eq:evidence}, involves the marginalization over $\btheta$.
This same infeasible marginalization also appears in the corresponding predictive distribution\index{Predictive distribution!Posterior}:

\begin{equation}\label{eq:neural-posterior-predictive}
    p(\bx^\prime \mid \mathbb{D}) = \int_{\bTheta} p(\bx^\prime \mid \btheta)p(\btheta \mid \mathbb{D})\ddd\!\btheta.
\end{equation}

Why is this marginalization prohibitive?
Compared to the conjugacy that allowed for the elegant solutions in \cref{eq:bernoulli-mle,eq:predictive-prior-coin-flip,eq:predictive-posterior-coin-flip}, neural networks\index{Neural network} typically involve non-linear components in the form of activation functions, which enable their flexibility and modeling power as their depth increases \citep{hornik1989multilayer, barron1994approximation, lu2017expressive}.\footnote{
    As an illustrative counter-example, consider a simple two-layer network without non-linear activation functions in the form of
    \[f(\bx) = \begin{bmatrix} w_5 & w_6\end{bmatrix}\bigg(\begin{bmatrix} w_1 & w_2 \\ w_3 & w_4 \\ \end{bmatrix} \begin{bmatrix} x_1 \\ x_2\end{bmatrix}\bigg),\]
    which we can rewrite as $f(\bx) = \mathbf{a}\T\bx$ with $a_1 = w_1w_5 + w_3w_6$ and $a_2 = w_2w_5 + w_4w_6$.
    Therefore, despite using two linear layers, we effectively obtain a single linear layer in practice, thus providing motivation for non-linear activation functions.
}
Formulating the likelihood $p(\mathbb{D}\mid\btheta)$\index{Likelihood}, this non-linear dependence of parameters makes it impossible to marginalize the parameters out.
Numerical integration is also usually not feasible, since network parameters are real-valued and typically high-dimensional.
However, that does not mean that Bayesian inference with neural networks is impossible, it rather means that we have to employ a number of different strategies.
A common red thread between them is that evaluating \cref{eq:neural-posterior-predictive} does not require us to have access to the distribution itself, only (high-quality) samples.
As such, we can approximate the integral using Monte Carlo sampling:

\begin{equation}\label{eq:neural-posterior-predictive-mc-estimate}
    p(\bx^\prime \mid \mathbb{D}) = \int_{\bTheta} p(\bx^\prime \mid \btheta)p(\btheta \mid \mathbb{D})\ddd\!\btheta \approx \frac{1}{M}\sum_{m=1}^M p(\bx^\prime \mid \btheta^{(m)}), 
\end{equation}

\noindent where we assume access to a set $\{\btheta^{(m)}\}^M_{m=1}$ of $M$ sampled parameter vectors.
This Monte Carlo integration\index{Monte Carlo estimation} approximates \cref{eq:neural-posterior-predictive} with an error $1/\sqrt{M}$, that decreases as a function of the number of samples.
It should be noted however that this approximation will be only asymptotically correct for samples from the true (and not an approximate) posterior, which we can obtain using the now following methods.

\paragraph{Markov Chain Monte Carlo \& Stochastic Gradient Langevin Dynamics.}
In order to obtain representative samples from the posterior, we do not necessarily need the analytical form of the posterior.
This idea is used by techniques such as \emph{Markov chain Monte Carlo} (MCMC)\index{Markov chain Monte Carlo} and \emph{stochastic gradient Langevin dynamics} (SGLD)\index{Langevin dynamics}.
In the case of MCMC, the core insight is that as long as we can evaluate $p(\btheta \mid \mathbb{D})$ up to the pesky evidence term $p(\mathbb{D})$, we can, in relative terms, determine whether 
one sample is more likely under the posterior than another. 
That means that upon formulating a suitable update rule, we can construct a chain of samples that leads from unlikely samples from $p(\btheta \mid \mathbb{D})$ to more likely ones.
A thorough introduction to and overview over this family of methods is out of scope for this section, which is why we instead refer to \citep{robert1999monte} and the corresponding chapters in \citet{bishop2006pattern,gelman2021bayesian}.
The technique has found numerous applications for neural networks, e.g.\@ \citet{andrieu2003introduction, neal1995bayesian, cobb2021scaling, li2023entropy}.
Stochastic Gradient Langevin dynamics (SGLD)\index{Langevin dynamics} follows a similar intuition \citep{welling2011bayesian}, however instead of formulating probabilistic transition rules, the constructed chain of samples follows the gradient of the prior and log-likelihood to seek posterior modes, similar to gradient descent.
Trying to combine the advantages of both methods has even birthed SGLD / MCMC \index{Markov chain Monte Carlo} hybrids \citep{ma2015complete, chen2016stochastic, liu2016stochastic, zhang2019cyclical}. 
In all cases, sampling methods remain challenging due to the high dimensional parameter space of neural networks and the often multi-modal nature of the weight posterior $p(\btheta \mid \mathbb{D})$.

\paragraph{Variational Inference.} 
\index{Variational inference}
Instead of trying to sample from the posterior $p(\btheta \mid \mathbb{D})$, we can instead sample from an easier proposal distribution $q(\btheta\mid\bphi)$ with parameters $\bphi$. 
For this proposal distribution to reasonably represent the original weight posterior, we try to minimize the difference between the two \citep{hinton1993keeping, graves2011practical}.
At first glance, this seems paradoxical---how can we minimize the distance from the posterior if we do not know its form?
However, using the Kullback-Leibler (KL)\index{Kullback-Leibler divergence} divergence, we can rewrite this difference as follows:

\begin{align}\label{eq:elbo}
     & \min_{\bphi} \text{KL}\big[q(\btheta\mid\bphi)\ \big|\big|\ p(\btheta\mid \mathbb{D})\big] \\
    = & \min_{\bphi} - \int_{\bTheta} q(\btheta\mid\bphi) \log \frac{q(\btheta\mid\bphi)}{p(\btheta\mid \mathbb{D})} \dd\!\btheta\\
    = & \min_{\bphi} \int_{\bTheta} q(\btheta\mid\bphi) \log \frac{q(\btheta\mid\bphi)}{p(\mathbb{D}\mid \btheta)p(\btheta)} \dd\!\btheta \label{eq:elbo-fac} \\
    = & \min_{\bphi} \text{KL}\big[q(\btheta\mid\bphi)\ \big|\big|\ p(\btheta)\big] - \mathbb{E}_{q(\btheta\mid\bphi)}\big[p(\mathbb{D}\mid \btheta)\big]. \label{eq:actual-elbo}
\end{align}

To derive this expression, we exploited the fact that in \cref{eq:elbo-fac}, the expectation of the evidence $p(\mathbb{D})$ under $q(\btheta\mid\bphi)$ is a constant that does not influence the result of the optimization problem.
Since this term is missing from the expression in \cref{eq:actual-elbo}, we refer to the result as the \emph{evidence lower bound} or ELBO.
We can now evaluate the KL divergence in closed form when the proposal distribution and prior are chosen in a convenient form (e.g.\@ Gaussian distributions\index{Normal distribution}), and the respective integrals can again be approximated via Monte Carlo approximation\index{Monte Carlo estimation}, and the parameters $\bphi$ be optimized via gradient descent.
The only missing component is that sampling $\btheta$ from $q(\btheta\mid \bphi)$ must be differentiable, which is achieved via the \emph{reparameterization trick}\index{Reparameterization trick} \citep{opper2009variational, kingma2013auto, rezende2014stochastic}.
To show this, let $\bphi = \{\bmu, \brho\}$ be the parameters of a Gaussian proposal distribution. Then we can obtain differentiable samples by 

\begin{align}\label{eq:reparameterization-trick}
    \beps \sim \mathcal{N}(\mathbf{0}, \mathbf{I});\quad \btheta = \bmu + \brho \circ \beps.
\end{align}

After training, networks parameters can be sampled from $q(\btheta\mid \bphi)$ directly to facilitate Bayesian neural networks\index{Neural network!Bayesian} and evaluate the predictive distribution in \cref{eq:neural-posterior-predictive-mc-estimate}\index{Predictive distribution!Posterior}.
Examples for variational methods for Bayesian neural networks are given by \citet{blundell2015weight, hernandez2015probabilistic, louizos2016structured, krueger2017bayesian, pawlowski2017implicit, zhang2018noisy}.

\paragraph{Stochastic Regularizers.} 
Another line of research has been concerned with the interpretation of neural network regularizers as sources for stochastic network parameter samples.
For instance, dropout \citep{srivastava2014dropout}\index{Dropout} regularizes neural network weights by setting a random subset of them to zero.
This is implemented by sampling a mask from a Bernoulli distribution\index{Bernoulli distribution} with dropout probability $p_\text{dropout}$ and multiplying it with the corresponding weight matrix $\mathbf{W} \in \mathbb{R}^{M \times N}$:\footnote{
    While the intuition of dropout lies in severing neural connections randomly, in practice it is often realized as an additional layer that is applied by zeroing out parts of activations.
    For instance, the parallel work of \citet{kingma2015variational} explores a variational objective using dropout that is applied directly to the activations.
}

\begin{equation}\label{eq:dropout}
    \mathbf{W}_\text{dropout} = \mathbf{W} \circ \mathbf{M}; \quad \{\mathbf{M}_{ij}\}_{i,j=1}^{M,N} \sim \text{Bernoulli}(p_\text{dropout}).
\end{equation}

As \citet{gal2016dropout} argues, we can actually interpret a set of parameters $\btheta_\text{dropout}$ with dropout masks applied to it as a sample from a variational posterior;
therefore, by using dropout at inference time (as opposed to just training time in its original form), we obtain a set of samples $\{\btheta^{(m)}_\text{dropout}\}^M_{m=1}$ that can be inserted back into the MC estimate\index{Monte Carlo estimation} of the predictive distribution in \cref{eq:neural-posterior-predictive-mc-estimate}.
This technique is referred to as \emph{Monte Carlo dropout} (or MC dropout)\index{Dropout!Monte Carlo} and has found a number of extensions over the years \citep{gal2016theoretically, li2017dropout, gal2017concrete, nalisnick2019dropout, boluki2020learnable, durasov2021masksembles}.
A similar reasoning can be applied to batch normalization \citep{ioffe2015batch}:
Batch normalization\index{Batch normalization} works by normalizing the input $\bz^{(l)}$ to a layer via an estimate of its mean and variance 

\begin{equation}\label{eq:batch-normalization}
    \bz^{(l)}_\text{BN} = \frac{\bz^{(l)} - \mathbb{E}[\bz^{(l)}]}{\sqrt{\text{Var}[\bz^{(l)}] + \varepsilon}},
\end{equation}

\noindent where $\varepsilon$ is a small value added to avoid numerical issues, and the mean and variance statistics are estimated empirically during training.
Similar to MC dropout\index{Dropout!Monte Carlo}, \citet{teye2018bayesian, mukhoti2020batch} re-interpret this as a source of stochasticity:
By sampling a single batch from the training set at inference time, we can use it to set our batch statistics for expectation and variance.
By using this batch mean and variance for our current inference, \citeauthor{teye2018bayesian}, we sample different hidden representations, than we interpreted as a result of the randomness in the underlying weights.
In both cases, the advantages are obvious: 
These regularization components are already part of many deep learning architectures, and the only overhead added is by running multiple forward passes per test input, which---for smaller models---might only add negligible overhead.
The more subtle downside lies in the fact variational inference\index{Variational inference} techniques, as these techniques are counted as, tend to only explore limited regions of the posterior distribution \citep{wilson2020bayesian}.
As such, obtained samples might simply not be very representative of $p(\btheta \mid \mathbb{D})$ and lead to subpar predictions and uncertainty estimates.

\paragraph{Laplace Approximations.}
The idea of Laplace approximations\index{Laplace approximation} can indeed by traced back to the eponymous Pierre-Simon Laplace \citep{laplace1774} and has been applied to deep learning first by \citet{mackay1992bayesian}.
In order to approximate a complex distribution $p(\btheta \mid \mathbb{D})$, we first obtain a MAP estimate of the network parameters  $\btheta_\text{MAP}$ as described in the beginning of this section.
We then consider a second-order Taylor expansion for the loss function $\mathcal{L}(\mathbb{D}, \btheta)$ at $\btheta_\text{MAP}$:

\begin{align}\label{eq:laplace-approx}
    & \mathcal{L}(\mathbb{D}, \btheta) \approx \nonumber \\
    & \mathcal{L}(\mathbb{D}, \btheta_\text{MAP}) + \frac{1}{2}(\btheta - \btheta_\text{MAP})\T \big(\nabla^2_{\btheta}\mathcal{L}(\mathbb{D}, \btheta)\big|_{\btheta_\text{MAP}} \big)(\btheta - \btheta_\text{MAP}).
\end{align}

By assuming that $\mathcal{L}(\mathcal{D}, \btheta_\text{MAP})$ is negligible for a fully trained network, we can identify\index{Normal distribution}

\begin{equation}\label{eq:laplace-posterior}
    p(\btheta \mid \mathbb{D}) \approx \mathcal{N}\Big(\btheta\ \Big|\ \btheta_\text{MAP}, -\nabla^2_{\btheta}\mathcal{L}\big(\mathbb{D}, \btheta)\big|_{\btheta_\text{MAP}} \big)^{-1}\Big).
\end{equation}

Unfortunately, the computation of the covariance matrix quickly becomes infeasible for larger models due to the quadratic nature of the Hessian\index{Hessian}.
Therefore, different compromises have been proposed \citep{daxberger2021laplace}, including last-layer approximations \citep{kristiadi2020being, snoek2015scalable}, approximation on subsets of weights \citep{daxberger2021bayesian}, factorizing the Hessian \citep{ritter2018scalable, ritter2018online, kristiadi2020being, yu2023riemannian, bergamin2024riemannian}, or variational approximations \citep{ortega2023variational}.
At inference time, network parameters can be drawn from the approximate posterior as with previous methods.

\paragraph{Ensembling.}\index{Ensembling}
A long-existing method to boost predictive performance has been to train multiple predictors on a problem and to ensemble their outputs \citep{bauer1999empirical, dietterich2000ensemble}.
Combining predictions has already been studied since the late 1960s, e.g.\@ in \citet{bates1969combination, clemen1989combining}, with some works on neural network ensembles already in the 1990s \citep{hansen1990neural, levin1990statistical, liu1999ensemble, zhou2002ensembling}.
After the deep learning\index{Deep learning} revival, \citet{lakshminarayanan2017simple} discovered that deep ensembles do not only improve generalization, but also tend to be well-calibrated and produce high-quality estimates of predictive uncertainty.
While \citet{lakshminarayanan2017simple} frame deep ensembles explicitly as non-Bayesian, \citet{fort2019deep, wilson2020bayesian} later argued that ensembling actually \emph{is} a form of Bayesian model averaging.
Since the members of an ensemble are usually trained independently, they are better at converging to different solutions in the parameter space.
Therefore, ensembles are argued to better represent the often multi-modal weight posterior than some of the methods discussed earlier like MC dropout or Laplace approximations\index{Dropout!Monte Carlo}, which rely on local approximations \citep{fort2019deep}.\\

Naturally, the disadvantage of ensembling\index{Ensembling} lies in having to train multiple predictors, which can be costly for modern, large neural neural networks.
A flurry of research works has investigated alternatives to this costly procedure, such as having ensemble members share weights \citep{antoran2020depth, liu2021deep, durasov2021masksembles, laurent2022packed} or ensembling checkpoints of a model collected over the training \citep{izmailov2018averaging, maddox2019simple, izmailov2020subspace, wilson2020bayesian, yashima2022feature}.
As our understanding of neural loss landscapes improves, works such as \citet{garipov2018loss, cha2021swad} have suggested to create ensembles along low-loss basins.
Other ways to curb computational inference costs involve efficient weight factorization techniques \citep{wenzel2020hyperparameter, wen2020batchensemble, dusenberry2020efficient} or distilling properties of an entire ensemble into a single predictor \citep{malinin2020ensemble, kim2024fast}.
Even when ensemble members are trained independently, they can converge to similar solutions, offsetting their advantage.
Several methods to improve the diversity in ensembles have been proposed \citep{jain2020maximizing, d2021repulsive, el2023deep}, including the ensembling of different architectures \citep{zaidi2021neural}.
Notable are also other explicitly Bayesian ways of ensembling \citep{pearce2020uncertainty, deng2022deep} or connections to mixture-of-experts models \citep{allingham2021sparse}. 

\paragraph{Deep Kernel Learning.}
Gaussian processes\index{Gaussian process} (GP;\@ \citealp{kolmogoroff1941interpolation, wiener1949extrapolation, williams2006gaussian}) are (typically) non-parametric models that predict targets and corresponding uncertainties based on the similarity between training and test points.
These similarities are computed through covariance or kernel functions.
In theory, this creates appealing properties for uncertainty quantification, as unusual inputs should be labeled as uncertain because of their dissimilarity with the observed data.
Nevertheless, scaling Gaussian processes to large amounts of data in known to be challenging (see e.g.\@ discussions in \citealp{williams2006gaussian} or in \citealp{bishop2006pattern}, Chapter 6).
Therefore, \emph{deep kernel learning} \citep{wilson2016deep} fits a Gaussian process layer on top of a deep neural feature extractor.
This has created a number of follow-up works using deep kernel learning\index{Deep kernel learning} for UQ \citep{bradshaw2017adversarial, daskalakis2020scalable, liu2021deep2, van2021feature}, however several authors have noted shortcomings with the approach due to the challenging joint optimization of the GP and neural feature extractor \citep{ober2021promises, van2021feature, schwobel2022last}:
This includes overfitting and feature collapse, where OOD data points are mapped to similar regions of the latent space as training points.
On top of deep kernel learning, there are several connections between neural networks and GPs are given through deep Gaussian processes \citep{damianou2013deep, dunlop2018deep, jakkala2021deep} and the theoretical links between neural networks and GPs \citep{neal1995bayesian, williams1998computation, hensman2014nested, dutordoir2021deep}.

\paragraph{Uncertainty Quantification in Bayesian Networks.}\index{Uncertainty quantification}\index{Uncertainty metric}
So far, we have discussed multiple different methods how to obtain samples from the (approximate) weight posterior, but without mentioning how this aids in obtaining new, useful and disentangled uncertainties.
One way to assess epistemic uncertainty\index{Uncertainty!Epistemic} in this framework is to measure disagreement between predictions for the same input.
Since models tend to be underspecified on OOD\index{Out-of-distribution data} inputs, this is where the predictions from different models will disagree the most in case of high model uncertainty.
In classification, this can be done for instance using the \emph{variation ratio} \citep{freeman1965elementary, gal2016uncertainty}\index{Variation ratio}:
Assuming a set of $B$ samples from the weight posterior, let $\hat{y}^{(b)}$ the predicted label using each set of weights and let $y^{*}$ denote the most commonly predicted label among these.
Then, the variation-ratio \index{Variation ratio} is defined as 

\begin{equation}\label{eq:variation-ratio}
    \text{VR} = 1 - \frac{1}{B}\indicator{\hat{y}^{(b)} = y^{*}}.
\end{equation}

Another way to measure the disagreement between predictions is to simply quantify the average variance of predictions per class\index{Class variance}:

\begin{equation}\label{eq:class-variance}
    \bar{\sigma}^2 = \frac{1}{K}\sum_{K=1}^K \mathbb{E}_{p(\btheta \mid \mathbb{D})}\big[P_{\btheta}(y=k \mid \bx)^2\big] - \mathbb{E}_{p(\btheta \mid \mathbb{D})}\big[P_{\btheta}(y=k \mid \bx)\big]^2.
\end{equation}

A more theoretically motivated approach to isolate epistemic uncertainty\index{Uncertainty!Epistemic} is to consider the mutual information\index{Mutual information} between model parameters and a data sample \citep{depeweg2018decomposition, gal2018understanding}:

\begin{equation}\label{eq:mutual-information}
    \underbrace{\text{I}\big[y, {\btheta}\ \big|\ \mathbb{D}, \bx\big]}_{\text{Model uncertainty}} = \underbrace{\text{H}\Big[\mathbb{E}_{p(\btheta \mid \mathbb{D})}\big[P_{\btheta}(y \mid \bx)\big]\Big]}_{\text{Total uncertainty}} - \underbrace{\mathbb{E}_{p(\btheta \mid \mathbb{D})}\Big[\text{H}\big[P_{\btheta}(y \mid \bx)\big]\Big]}_{\text{Data uncertainty}}.
\end{equation}

The term itself can be interpreted as the gain in information about the ideal model parameters and correct label upon receiving an input. 
If we can only gain a little, that implies that parameters are already well-specified and that the epistemic uncertainty is low.
In both cases of \cref{eq:class-variance,eq:mutual-information} the expectation can be approximated through Monte Carlo approximation, i.e.

\begin{equation}
    \mathbb{E}_{p(\btheta \mid \mathbb{D})}\big[P_{\btheta}(y=k \mid \bx)\big] \approx \frac{1}{B}\sum_{b=1}^B P(y=k \mid \bx, \btheta^{(b)}).
\end{equation}

\subsection{Evidential Neural Networks}\label{sec:evidential-neural-networks}

\begin{footnotesize}
    \vspace{-2.25ex}
    \emph{The following work is based on \citet{ulmer2023prior}}.\\
    \vspace{2.25ex}
\end{footnotesize}

In the last section, we explored many different approaches to quantify different kinds of uncertainty by obtaining samples from the weight posterior $p(\bm{\theta} \mid \mathbb{D})$.
However, we saw that this can be a challenging endeavor, since samples might be expensive to obtain or not very representative of the actual posterior distribution.
Alternatively, we can factorize \cref{eq:neural-posterior-predictive} further and use a point estimate for the weights to obtain a tractable form\index{Predictive distribution!Posterior}:

\begin{align}\label{eq:prior-networks-fac}
    P(y \mid \bx, \mathbb{D}) & = \int p(\bx^\prime \mid \btheta)p(\btheta \mid \mathbb{D}){\ddd}\!\btheta \\
    & = \int\hspace{-0.25cm}\int \underbrace{P(y \mid \bm{\pi})}_{\vphantom{\big[}\text{Aleatoric}}\underbrace{p(\bm{\pi} \mid \bx, \bm{\theta})}_{\vphantom{\big[}\text{\ Distributional\ }}\underbrace{p(\bm{\theta} \mid \mathbb{D})}_{\vphantom{\big[}\text{Epistemic}}\!{\ddd}\!\hspace{0.05cm}{\bpi}{\ddd}\!\btheta \\
    & \approx \int P(y \mid \bm{\pi})\underbrace{\vphantom{\big[}p(\bm{\pi} \mid \bx, \hat{\bm{\theta}})}_{ p(\bm{\theta} \mid \mathbb{D}) \approx \delta(\bm{\theta}-\hat{\bm{\theta}})}\!{\ddd}\!\bpi.
\end{align}

%\noindent where we add another distribution representing the \emph{distributional} uncertainty.
In the last step, \citeauthor{malinin2018predictive} replace $p(\bm{\theta} \mid \mathbb{D})$ by a point estimate $\hat{\bm{\theta}}$ using the Dirac delta function\index{Dirac delta function}, i.e.\@ a single trained neural network, to get rid of the intractable integral.\footnote{
    In the context of \cref{eq:prior-networks-fac}, it should be noted that restricting oneself to a point estimate of the network parameters prevents the epistemic uncertainty estimation through the weight posterior $p(\bm{\theta} \mid \mathbb{D})$, as discussed in the previous section. 
    However, there are works like \citet{haussmann2019bayesian, zhao2020uncertainty} that combine both approaches.
}
This factorization contains another type of uncertainty, which \citet{malinin2018predictive} call the \emph{distributional} uncertainty\index{Uncertainty!Distributional}; uncertainty caused by the mismatch of training and test data distributions. 
Although another integral remains, retrieving the uncertainty from this predictive distribution actually has a closed-form analytical solution, as we will see later. 
The advantage of this approach is further that it allows us to distinguish uncertainty about a data point because it is ambiguous\index{Ambiguity}, from uncertainty caused by a point coming from an entirely different data distribution\index{Shift!Distributional}. 
This approach to UQ\index{Uncertainty quantification} it called \emph{evidential deep learning} (EDL)\index{Deep learning!Evidential}, and originates from the work of \citet{sensoy2019evidential}.
They originally base their motivation on the \emph{theory of evidence}\index{Dempster-Shafer theory of evidence} \citep{dempster1968generalization, audun2018subjective}: 
Within the theory, belief mass is assigned to set of possible states, e.g.\@ class labels, and can also express a lack of evidence, i.e.\@ an ``I don't know''. 
We can apply this idea to the predicted output of a neural classifier using the Dirichlet distribution\index{Dirichlet distribution}, allowing us to express a lack of evidence through a uniform Dirichlet.
In this way, the neural network does not parameterize a single (categorical) distribution\index{categorical distribution}, but a \emph{distribution over distributions}, also referred to as a second-order distribution.
This is different from a uniform (first-order) categorical distribution, which does not distinguish an equal probability for all classes from a lack of evidence, or differently phrased:
One cannot distinguish whether the distribution is uniform due to uncertainty, or confidently uniform due to ambiguity.
In the following, we define EDL\index{Deep learning!Evidential} as a family of approaches in which a neural network can fall back onto a uniform prior for unknown inputs. 
While neural networks usually parameterize likelihood functions\index{Likelihood}, approaches in this survey parameterize prior or posterior distributions\index{Prior distribution}\index{Posterior distribution} instead, as we will show next.

\begin{figure}[t]
    \centering

    \begin{minipage}{0.985\textwidth}
        \begin{subfigure}[t]{\linewidth}  
            \centering
            \includegraphics[width=\textwidth]{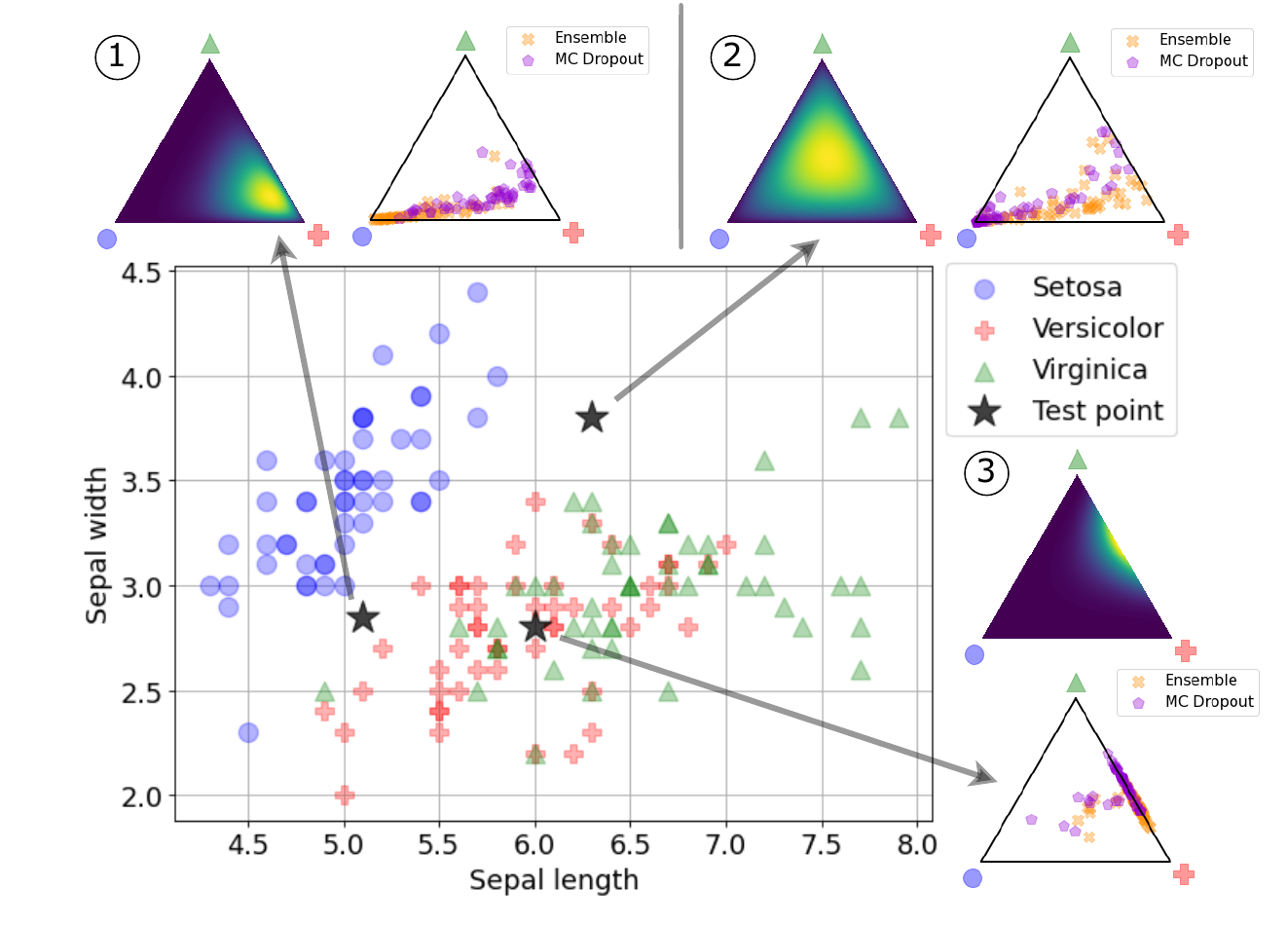}
        \end{subfigure}
    \end{minipage}
    \begin{minipage}{0.985\textwidth}
        \begin{subfigure}[h]{0.32\linewidth}
            \centering
            \includegraphics[width=0.85\textwidth]{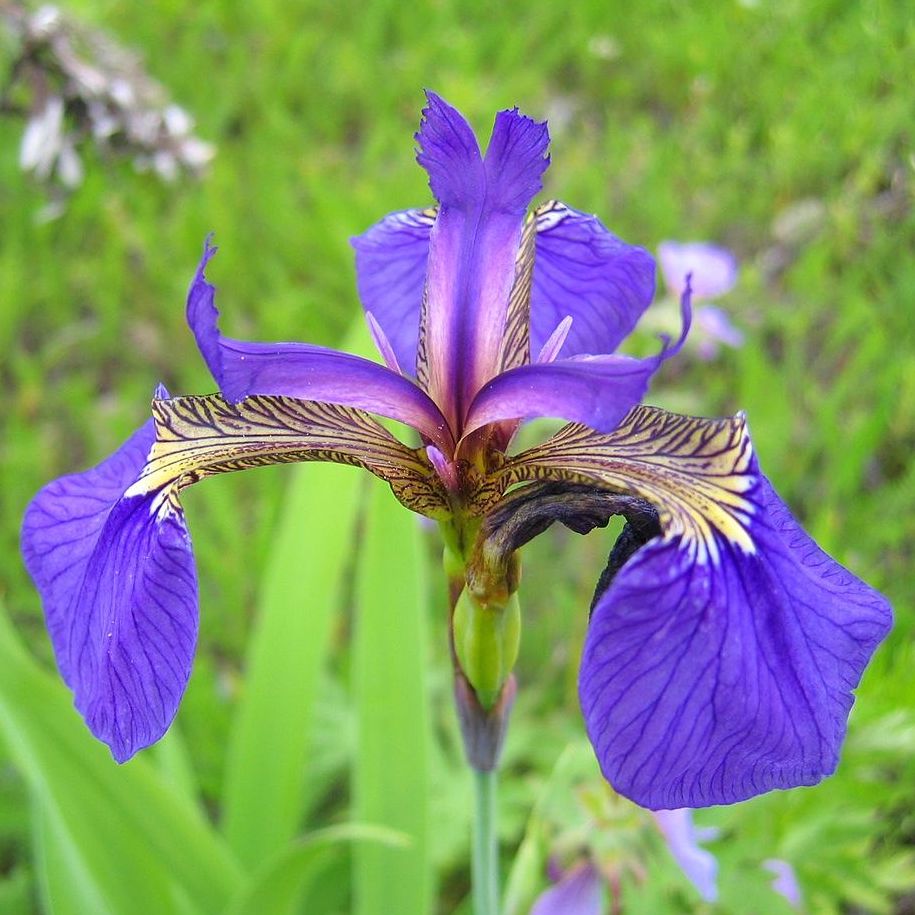}
            \caption{\emph{Iris setosa}}\label{subfig:setosa}
        \end{subfigure}
        \hfill
        \begin{subfigure}[h]{0.32\linewidth} 
            \centering
            \includegraphics[width=0.85\textwidth]{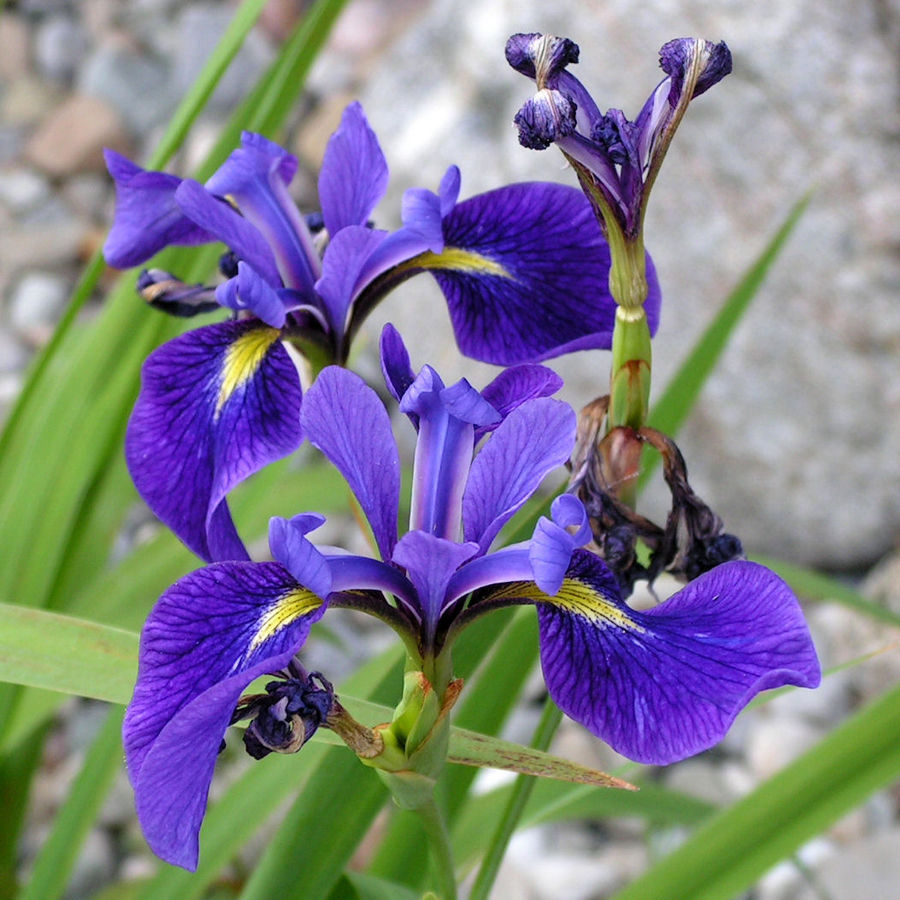}
            \caption{\emph{Iris versicolor}}\label{subfig:versicolor}
        \end{subfigure}
        \hfill
        \begin{subfigure}[h]{0.32\linewidth} 
            \centering
            \includegraphics[width=0.85\textwidth]{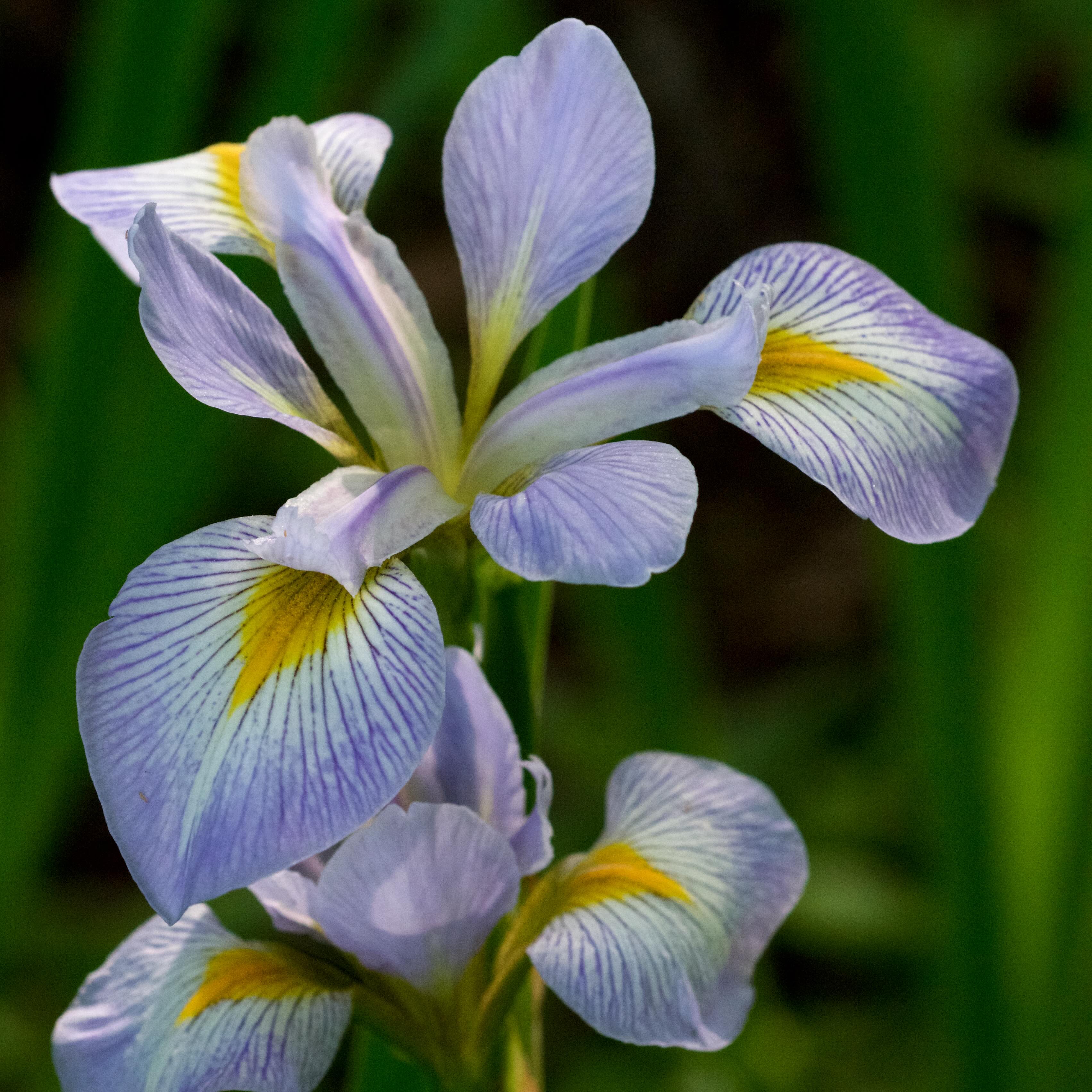}
            \caption{\emph{Iris virginica}}\label{subfig:virginica}
        \end{subfigure}
    \end{minipage}%

    \caption[Comparison of MC dropout, deep ensembles and prior networks on the Iris dataset.]{Illustration of different approaches to uncertainty quantification on the Iris dataset, with examples for the classes given on the left (\cref{subfig:setosa,subfig:versicolor,subfig:virginica}). On the right, the data  is plotted alongside some predictions of a prior network (lighter colors indicate higher density) and an ensemble and MC dropout model on the probability simplex, with $50$ predictions each. Iris images were taken from \citealp{iris_setosa, iris_versicolor, iris_virginica}.}
    \label{fig:iris-example}
\end{figure}

\paragraph{An Illustrating Example: The Iris Dataset.} 
We train a deep neural network ensemble\index{Ensembling} \citep{lakshminarayanan2017simple} with $50$ model instances, a model with MC Dropout\index{Dropout!Monte Carlo} \citep{gal2016dropout} with $50$ predictions and a prior network\index{Prior network} \citep{sensoy2019evidential}, an example of EDL\index{Deep learning!Evidential}, on all available data points, and plot their predictions on three test points on the probability simplex in \cref{fig:iris-example}.\footnote{
    For information about training and model details, see \cref{app:iris-code-details}.
} 
On these simplices, each point signifies a categorical distribution, with the proximity to one of the corners indicating a higher probability for the corresponding class. 
EDL methods for classification do not predict a single output distribution, but an entire \emph{density over output distributions}.
Test point \circled{3} lies in a region of overlap between instances of \emph{Iris versicolor} and \emph{Iris virginica}, thus inducing high aleatoric uncertainty\index{Uncertainty!Aleatoric}. 
In this case, we can see that the prior network\index{Prior network} places all of its density around the vertex between these two classes, similar to most of the predictions of the ensemble\index{Ensembling} and MC dropout\index{Dropout!Monte Carlo} (bottom right). 
However, some of the latter predictions still land in the center of the simplex. The point \circled{1} is located in an area without training examples between instances of \emph{Iris versicolor} and \emph{setosa}, as well as close to a single \emph{virginica} outlier. 
As shown in the top left, ensemble and MC dropout predictions agree that the point belongs to either the \emph{setosa} or \emph{versicolor} class, with a slight preference for the former. 
The prior network concentrates its prediction on \emph{versicolor}, but admits some uncertainty towards the two other choices. 
The last test point \circled{2} is placed in an area of the feature space devoid of any data, roughly equidistant from the three clusters of flowers. 
Similar to the previous example, the ensemble and MC dropout predictions on the top right show a preference for \emph{Iris setosa} and \emph{versicolor}, albeit with higher uncertainty. 
The prior network however shows an almost uniform density, admitting distributional uncertainty about this particular input.
This simple example provides some insights into the potential advantages of EDL: 
First of all, the prior network\index{Prior network} was able to provide reasonable uncertainty estimates in comparison with Bayesian model averaging methods\index{Bayesian model averaging}. 
Secondly, the prior network is able to admit its lack of knowledge for the OOD\index{Out-of-distribution data} data point by predicting an almost uniform prior, something that the other models are not able to. 
%As laid out in \cref{sec:uncertainty-dirichlet}, EDL actually allows the user to disentangle model uncertainty due to a simple lack of data and due to the input being out-of-distribution. 
Lastly, training the prior network only required a single model, which is a noticeable speed-up compared to MC dropout and especially the training of ensembles. 

\paragraph{Parameterization.}
We start from a categorical distribution\index{categorical distribution} over classes, defined as:

\begin{equation}
    \text{Categorical}(y \mid \bm{\pi}) = \prod_{k=1}^K \pi_k^{\indicator{y=k}},
\end{equation}

\noindent in which $K$ denotes the number of categories or classes, and the class probabilities are expressed using a vector $\bm{\pi} \in [0, 1]^K$ with $\sum_k \pi_k = 1$, and $\indicator{\cdot}$ is the indicator function. 
%This distribution appears for instance in classification problems when using neural networks, since most neural networks for classification use a softmax function after their last layer to produce a categorical distribution of classes s.t.\@ $\pi_k \equiv P(y=k \mid x)$. 
In this setting, the Dirichlet distribution\index{Dirichlet distribution} arises as a suitable prior and multivariate generalization of the Beta distribution\index{Beta distribution} (and is thus also called the \emph{multivariate Beta distribution}):

\begin{align}\label{eq:conjugate-dirichlet}
    \text{Dir}(\bm{\pi}; \bm{\alpha}) = \frac{1}{\text{B}(\bm{\alpha})}\prod_{k=1}^K \pi_k^{\alpha_k-1};\quad \text{B}(\bm{\alpha}) = \frac{\prod_{k=1}^K\Gamma(\alpha_k)}{\Gamma(\alpha_0)};\quad \alpha_0 = \sum_{k=1}^K \alpha_k,
\end{align}

\noindent where $\alpha_k \in \mathbb{R}^+$ and the Beta function $\text{B}(\cdot)$ is defined for $K$ shape parameters compared to \cref{eq:beta-prior}.
The distribution is characterized by its \emph{concentration parameters} $\bm{\alpha}$, the sum of which, often denoted as $\alpha_0$, is called the \emph{precision}.\footnote{
    The precision is analogous to the precision of a Gaussian\index{Normal distribution}, where a larger $\alpha_0$ signifies a sharper distribution.
} 
The Dirichlet is a \emph{conjugate prior}\index{Conjugacy} for such a categorical likelihood\index{Likelihood}, meaning that according to Bayes' rule\index{Bayes' theorem}, it produces a Dirichlet posterior with parameters $\bm{\beta}$, given a data set $\mathbb{D} = \{(x_i, y_i)\}_{i=1}^N$ of $N$ observations with corresponding labels:

\begin{align}\label{eq:dirichlet-posterior}
    & p(\bm{\pi} \mid \mathbb{D}, \bm{\alpha}) \propto p\big(\{y_i\}_{i=1}^N \mid \bm{\pi}, \{x_i\}_{i=1}^N\big)p(\bm{\pi} \mid \bm{\alpha}) \nonumber \\
    & = \prod_{i=1}^N\prod_{k=1}^K \pi_k^{\indicator{y_i = k}}\frac{1}{\text{B}(\bm{\alpha})}\prod_{k=1}^K \pi_k^{\alpha_k-1} \\
    & = \prod_{k=1}^K \pi_k^{\big(\sum_{i=1}^N\indicator{y_i = k}\big)}\frac{1}{\text{B}(\bm{\alpha})}\prod_{k=1}^K \pi_k^{\alpha_k-1} \\
    & = \frac{1}{\text{B}(\bm{\alpha})}\prod_{k=1}^K \pi_k^{N_k + \alpha_k-1} \propto \text{Dir}(\bm{\pi}; \bm{\beta}),
\end{align}

\noindent where $\bm{\beta}$ is a vector with $\beta_k = \alpha_k + N_k$, with $N_k$ denoting the number of observations for class $k$. 
Intuitively, this implies that the prior belief encoded by the initial Dirichlet is updated using the actual data, sharpening the distribution for classes for which many instances have been observed. 
The Dirichlet is a \emph{distribution over categorical distributions} on the $K-1$ probability simplex---while a neural classifier is usually realized as a function $f_{\bm{\theta}}: \mathbb{R}^D \rightarrow \mathbb{R}^K$, mapping an input $\bx \in \mathbb{R}^D$ to \emph{logits} for each class. 
Followed by a softmax function\index{Softmax function}, this then defines a categorical distribution over classes with a vector $\bm{\pi}$ with $\pi_k \equiv P_{\btheta}(y=k \mid \bx)$. 
The same underlying architecture can be used without any major modification to instead parameterize a Dirichlet, predicting a distribution over categorical distributions $p(\bm{\pi} \mid \bx, \hat{\bm{\theta}})$ as in \cref{eq:conjugate-dirichlet}.\footnote{
    The only thing to note here is that the every $\alpha_k$ has to be strictly positive, which can for instance be enforced by using an additional softplus, exponential or ReLU function \citep{sensoy2019evidential, malinin2018predictive, sensoy2020uncertainty}.
} 
In order to classify a data point $\bx$, a categorical distribution is created from the predicted concentration parameters of the Dirichlet\index{Dirichlet distribution} as follows (this corresponds to the mean of the Dirichlet, see \cref{app:expectation-dirichlet}):

\begin{equation}\label{eq:dirichlet-parameterization}
    \bm{\alpha} = \exp\big(f_{\bm{\theta}}(\bx)\big);\quad \pi_k = \frac{\alpha_k}{\alpha_0};\quad \hat{y} = \argmax_{k \in [K]}\ \pi_1, \ldots, \pi_K.
\end{equation}

\begin{figure}[tb]
    %\begin{wrapfigure}[18]{r}{0.675\textwidth}
        \centering
        %\vspace{-1.25cm}
        \begin{subfigure}[t]{0.235\linewidth}
            \centering
            \includegraphics[width=0.99\linewidth]{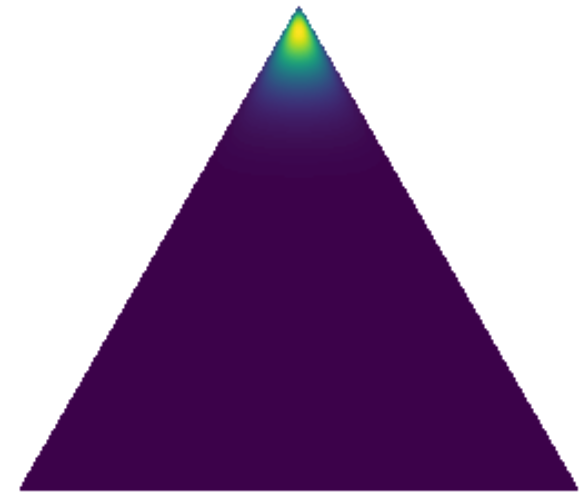}
            \caption{\footnotesize Confident prediction.}
            \label{subfig:simplex-confident}
        \end{subfigure}
        \hfill
        \begin{subfigure}[t]{0.235\linewidth}
            \centering
            \includegraphics[width=0.99\linewidth]{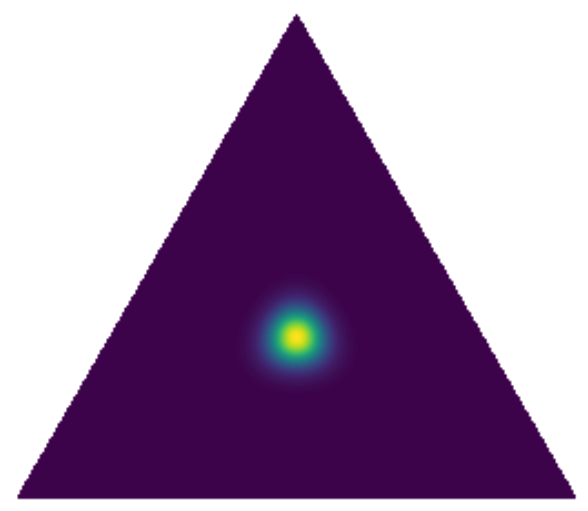}
            \caption{\footnotesize Aleatoric uncertainty.}
            \label{subfig:simplex-aleatoric}
        \end{subfigure}
        \hfill
        \begin{subfigure}[t]{0.235\linewidth}
            \centering
            \includegraphics[width=0.99\linewidth]{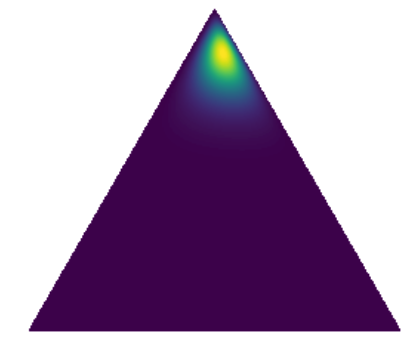}
            \caption{\footnotesize Epistemic uncertainty.}
            \label{subfig:simplex-epistemic}
        \end{subfigure}
        \hfill
        \begin{subfigure}[t]{0.235\linewidth}
            \centering
            \includegraphics[width=0.99\linewidth]{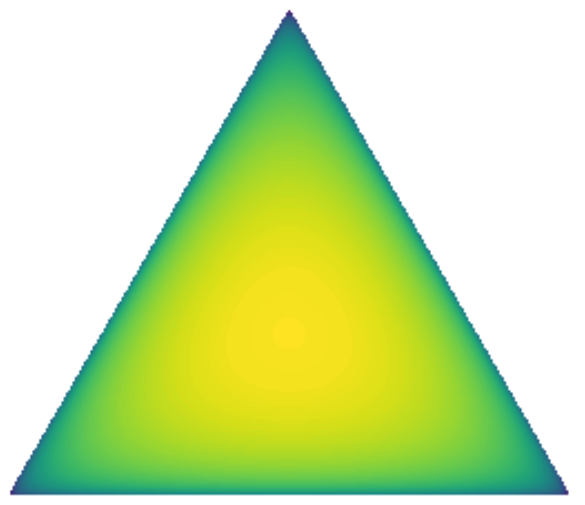}
            \caption{\footnotesize Distributional uncertainty.}
            \label{subfig:simplex-distributional}
        \end{subfigure}
        \caption[Examples of the probability simplex for a $K=3$ classification problem.]{Examples of the probability simplex for a $K=3$ classification problem, where every corner corresponds to a class and every point to a categorical distribution, and brighter colors correspond to higher density. 
        Shown is the (desired) Behavior of Dirichlet in different scenarios by \citet{malinin2018predictive}: 
        (a) For a confident prediction, the density is concentrated in the corner of the simplex corresponding to the assumed class. 
        (b) In the case of aleatoric uncertainty, the density is concentrated in the center, and thus uniform categorical distributions are most likely. 
        (c) In the case of model uncertainty, the density may still be concentrated in a corner, but more spread out, expressing the uncertainty about the right prediction. 
        (d) In the case of an OOD input, a uniform Dirichlet expresses that any categorical distribution is equally likely, since there is no evidence for any known class.  
        }
        \label{fig:simplex}
    \end{figure}

\paragraph{Uncertainty Quantification in EDL.}
\index{Deep learning!Evidential}\index{Uncertainty metric}
Let us now turn our attention to how to estimate the aleatoric\index{Uncertainty!Aleatoric}, epistemic\index{Uncertainty!Epistemic} and distributional uncertainty\index{Uncertainty!Distributional} within the Dirichlet framework. 
In \cref{fig:simplex}, we show different (ideal) shapes of a Dirichlet distribution parameterized by a neural network, corresponding to different cases of uncertainty, where each point on the simplex represents a categorical distribution, with proximity to a corner indicating a high probability for the corresponding class.
However, since we do not want to inspect Dirichlets visually, we instead use closed-form expressions to quantify uncertainty. 
% Athough stated for the prior parameters $\bm{\alpha}$, the following methods can also be applied to the posterior parameters%  $\bm{\beta}$ without loss of generality.
To obtain a measure of data uncertainty, we can evaluate the expected entropy\index{Entropy!Expected} of the data distribution $P(y \mid \bm{\pi})$.
As the entropy captures the ``peakiness'' of the output distribution, a lower entropy indicates that the model is concentrating most probability mass on a single class, while high entropy characterizes a more uniform distribution---the model is undecided about the right prediction. 
For Dirichlet networks, this quantity has a closed-form solution (for the full derivation, refer to \cref{app:expected-entropy}):

\begin{equation}\label{eq:expected-entropy}
    \mathbb{E}_{p(\bm{\pi} \mid \bx, \hat{\bm{\theta}})}\Big[\text{H}\big[P(y \mid \bm{\pi})\big]\Big] = - \sum_{k=1}^K\frac{\alpha_k}{\alpha_0}\bigg(\psi(\alpha_k+1) -  \psi(\alpha_0+1)\bigg),
\end{equation}

\noindent where $\psi$ denotes the digamma function\index{Digamma function}, defined as $\psi(x) = \frac{d}{d x} \log \Gamma(x)$, and $\text{H}$ the Shannon entropy.
As we saw in \cref{eq:prior-networks-fac}, we can avoid the intractable integral over network parameters $\btheta$ by using a point estimate $\hat{\btheta}$.\footnote{
    %With exceptions such as \citet{haussmann2019bayesian,zhao2020uncertainty}. 
    When the distribution over parameters in \cref{eq:prior-networks-fac} is retained, alternate expressions of the aleatoric and epistemic uncertainty are derived by \citet{woo2022analytic}\index{Uncertainty!Aleatoric}\index{Uncertainty!Epistemic}.
}
This means that computing the model uncertainty via the weight posterior $p(\bm{\theta} \mid \mathbb{D})$ like in \cref{sec:bayesian-neural-networks} is not possible.  
Nevertheless, a key property of Dirichlet networks is that epistemic uncertainty\index{Uncertainty!Epistemic} is expressed through the spread of the Dirichlet distribution\index{Dirichlet distribution} (for instance in \cref{fig:simplex} (c) and (d)). 
Therefore, the epistemic uncertainty can be quantified considering the concentration parameters $\bm{\alpha}$ that shape this distribution: 
\citet{charpentier2020posterior} simply consider the maximum $\alpha_k$ as a score akin to the maximum probability score by \citet{hendrycks2017baseline}, while \citet{sensoy2019evidential} compute it by $K / \sum_{k=1}^K (\alpha_k + 1)$ or simply $\alpha_0$ \citep{charpentier2020posterior}. 
In both cases, the underlying intuition is that larger $\alpha_k$ produce a sharper density, and thus indicate increased confidence in a prediction.
Lastly, the distributional uncertainty\index{Uncertainty!Distributional} can be quantified by computing the difference between the total amount of uncertainty and the data uncertainty\index{Uncertainty!Aleatoric} (similar to the reasoning behind \cref{eq:mutual-information}), which can be expressed through the mutual information\index{Mutual information} between the label $y$ and its categorical distribution $\bm{\pi}$:

\begin{equation}\label{eq:dirichlet-mi}
    \text{I}\big[y, \bm{\pi}\ \big|\ \bx, \mathbb{D}\big] = \underbrace{\text{H}\Big[\mathbb{E}_{p(\bm{\pi} \mid \bx, \mathbb{D})} \big[P(y \mid \bm{\pi})\big]\Big]}_{\text{Total Uncertainty}} - \underbrace{\mathbb{E}_{p(\bm{\pi} \mid \bx, \mathbb{D})}\Big[\text{H}\big[P(y \mid \bm{\pi})\big]\Big]}_{\text{Data Uncertainty}}.
\end{equation}

This quantity expresses how much information we would receive about $\bpi$ if we were given the label $y$, conditioned on the new input $\bx$ and the training data $\mathbb{D}$. 
In regions in which the model is well-defined, receiving $y$ should not provide much new information about $\bpi$---and thus the mutual information would be low. 
Yet, such knowledge should be very informative in regions in which few data have been observed, and there this mutual information would indicate higher distributional uncertainty\index{Uncertainty!Distributional}.
Given that $\mathbb{E}[\pi_k] = \frac{\alpha_k}{\alpha_0}$ (\cref{app:expectation-dirichlet}) and assuming the point estimate $p(\bm{\pi} \mid \bx, \mathbb{D}) \approx p(\bm{\pi} \mid \bx, \hat{\bm{\theta}})$ to be sufficient \citep{malinin2018predictive}, we obtain an expression very similar to \cref{eq:expected-entropy}:

\begin{equation}
    \text{I}\big[y, \bm{\pi}\ \big|\ \bx, \mathbb{D}\big] = - \sum_{k=1}^K \frac{\alpha_k}{\alpha_0}\Big(\log \frac{\alpha_k}{\alpha_0} -\psi(\alpha_k+1) + \psi(\alpha_0+1)\Big).
\end{equation}

We mentioned before how \cref{fig:simplex} illustrates idealized behaviors of the Dirichlet distributions\index{Dirichlet distribution}. 
Therefore, any closed-form expressions of different uncertainties can only be effective when the desired shape of the distribution is attained.
Similarly, the naive parameterization in \cref{eq:dirichlet-parameterization} is not to guaranteed to succeed in this goal, and the literature has proposed different methods to attain this goal.
They can broadly be classified into two families:
\emph{Prior networks}\index{Prior network}, which parameterize the Dirichlet prior distribution and employ custom training procedures and regularizers,
and \emph{posterior networks}\index{Posterior network}, which instead parameterize a Dirichlet posterior like in \cref{eq:dirichlet-posterior} instead.\footnote{
    We now give a brief overview over these approaches with a focus on classification problems.
    For a more comprehensive account that also includes regression problems, refer to \citet{ulmer2023prior}.
}

\paragraph{Prior Networks.}
Prior networks\index{Prior network} can be further subcategorized into two sets, namely OOD-free approaches or OOD-dependent approaches\index{Out-of-distribution data}.
In the first case, we regulate the behavior of the Dirichlet distribution on OOD inputs by adding a regularizer that penalizes any density allocated to regions that do not correspond to the gold label.
One such option is to decrease the Kullback-Leibler divergence\index{Kullback-Leibler divergence} from a uniform Dirichlet (see \cref{app:kl-dirichlets}):

\begin{equation}
    \text{KL}\big[p(\bm{\pi} \mid \bm{\alpha})\ \big|\big|\ p(\bm{\pi} \mid \bm{1})\big] = \log \frac{\Gamma(K)}{\text{B}(\bm{\alpha})} + \sum_{k=1}^K (\alpha_k - 1)\big(\psi(\alpha_k) - \psi(\alpha_0)\big).
\end{equation}

Other options are the use of Rényi divergences\index{Rényi divergence} \citep{tsiligkaridis2019information}, regularizers derived from PAC-bounds \citep{haussmann2019bayesian}, or $l_p$-norms \citep{sensoy2019evidential, tsiligkaridis2019information}.
Alternatively, some works also try to transfer the uncertainty from a set of Bayesian predictors into a single prior network\index{Prior network} \citep{malinin2020ensemble, fathullah2022self} using knowledge distillation \citep{hinton2015distilling}.
When OOD data is available, we also explicitly train the prior network to maximize its entropy on such examples \citep{malinin2018predictive, malinin2019reverse, nandy2020towards}, which can for instance be implemented using the closed-form solution in \cref{app:entropy-dirichlet}:

\begin{equation}
    \text{H}\big[p(\bm{\pi} \mid \balpha)\big] = \log \text{B}(\bm{\alpha}) + (\alpha_0 - K)\psi(\alpha_0) - \sum_{k=1}^K (\alpha_k - 1)\psi(\alpha_k).
\end{equation}

Unfortunately though, it should be noted that such data is often not available or in the first place, or cannot guarantee robustness against \emph{other} kinds of unseen OOD data\index{Out-of-distribution data}, of which infinite types exist in a real-valued feature space.\footnote{
    The same applies to the synthetic OOD data in \citet{chen2018variational, shen2020modeling, sensoy2020uncertainty}.
}

\paragraph{Posterior Networks.}\index{Posterior network}
When parameterizing \cref{eq:dirichlet-posterior} instead of the Dirichlet prior\index{Prior distribution}, the neural networks\index{Neural network} now predicts the update $N_k$ instead, and the prior parameters $\balpha$ are typically set to be uniform.
Nevertheless, we still need to gently guide the resulting Dirichlet posterior\index{Posterior distribution} to attain its desired uncertainty behavior.
Similar to prior networks, this can be done with an entropy regularizer \citep{sensoy2019evidential} or additional training objective on OOD examples, including works that create synthetic OOD inputs using additional generative models \citep{sensoy2020uncertainty, hu2021multidimensional}.
More interestingly, \citet{charpentier2020posterior, stadler2021graph, charpentier2021natural} use normalizing flows \citep{rezende2015variational} trained on the model's latent representations to compute the update $N_k$.
By modeling the latent density, this allows us to update the uniform prior by a lot when the latent encoding is familiar, and leave the prior ignorance intact when it is not, and is therefore assigned a low probability by the normalizing flow\index{Normalizing flow}.

\subsection{Other Approaches}\label{sec:other-approaches}

A number of other methods for UQ do not neatly fall into the categories we discussed so far.
This includes for instance some works that see the layer-wise transformations happening inside a neural network as a dynamical system that can be modeled through neural stochastic differential equations (SDEs; \citealp{kong2020sde, wang2021curved, wang2021neural, xu2022infinitely}).
By parameterizing the drift and diffusion terms of a SDE\index{Stochastic differential equation} by neural networks\index{Neural network}, the diffusion network can be used to predict model uncertainty.
\citet{ma2023probabilistic} parameterize a layer-wise mean and covariance instead, but do not embed these in a SDE.
In a completely different approach, \citet{hu2022vibration} obtain a sequence of probabilities for a specific input from different model snapshots during training, and then quantify the uncertainty in the frequency domain after applying a discrete Fourier transform.
\citet{papernot2018deep, jiang2018trust} compare the output of a predictor to that of a simple nearest-neighbor classifier to quantify uncertainty, and \citet{anirudh2021delta} compare latent embeddings to a number of anchor points.

\paragraph{Direct Uncertainty Prediction.}\index{Uncertainty}
So far, we have treated uncertainty as something to be extracted from a model that, in general, is performing a different task, such as classification or regression.
But what if we can just treat UQ\index{Uncertainty quantification} as a supervised learning task, learning to predict an uncertainty score from an input?
For instance, \citet{geifman2019selectivenet} propose to add another prediction head to a model which predicts when the model should abstain from a potentially false output.
The same option is instead parameterized as an additional class in a classification problem by \citet{liu2019deep}.
Alternatively, the the confidence\index{Confidence} of a network can also be obtained from an independent network \citep{corbiere2019addressing, corbiere2021confidence, luo2021learning, fathullah2023needs, liu2024uncertainty}, which is also what \cref{ch:uncertainty-llms} discusses in the context of LLMs.
This model can also take the shape of a Gaussian process\index{Gaussian process}, as demonstrated by \citet{qiu2022detecting}.\\

Instead of setting up this additional model as a classifier, we can also employ a density estimator to derive the uncertainty of a target model, similar to posterior networks\index{Posterior network} in \cref{sec:evidential-neural-networks}.
This again follows the idea that a density estimator would be able to indicate when a given test point lies outside of the known training distribution.
As estimation of density can be achieved through Gaussian discriminant analysis on the latent representations \citep{mukhoti2021deterministic, franchi2022latent}, distances between latent features \citep{huang2021decomposing}, kernel density estimators \citep{kotelevskii2022nonparametric, sun2024flagged} or normalizing flows \citep{lahlou2021deup}.
Some of these methods are benchmarked by \citet{postels2021practicality}, showing some sensitivity to distributional shifts\index{Shift!Distributional} nevertheless.

\begin{figure}[htb]
    \centering
    \begin{subfigure}[t]{0.46\linewidth}
        \centering
        %\vspace{0.02cm}
        \includegraphics[width=0.95\linewidth]{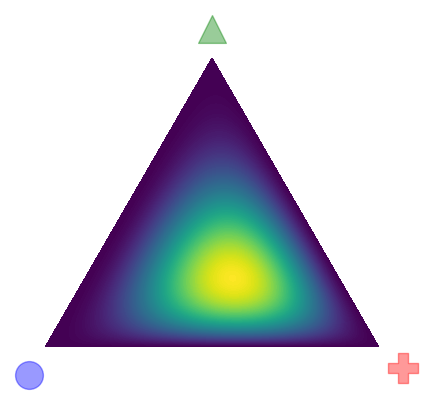}
        \caption{Prior network prediction.}
        \label{subfig:simplex-epistemic}
    \end{subfigure}
    \hfill
    \begin{subfigure}[t]{0.52\linewidth}
        \centering
        \includegraphics[width=0.99\linewidth]{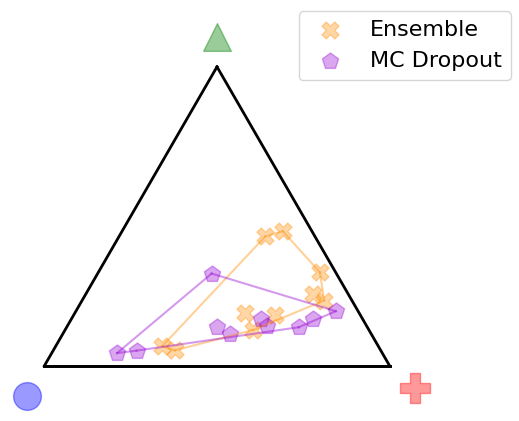}
        \caption{Credal sets from convex hulls.}
        \label{subfig:credal-sets}
    \end{subfigure}
    \caption[Juxtaposition of a prior network and credal sets constructed from the convex hull of ensemble and MC dropout predictors.]{Juxtaposition of a prior network and credal sets constructed from the convex hull of ensemble and MC dropout predictors.}\label{fig:credal-sets}
\end{figure}

\paragraph{Credal Sets.}\index{Credal sets}
Credal sets are based on the theory of imprecise probabilities\index{Imprecise probabilities} \citep{boole1854investigation, keynes1921treatise, walley1991statistical}.
The theory focuses on the idea that while there might a model that precisely describes a probability of interest, it may not be known, for instance due to vague, conflicting or scarce data\index{Data scarcity} \citep{caprio2023credal}.
One option to model this impreciseness is the use \emph{credal sets}, which are sets of credible probability distributions.
Like EDL\index{Deep learning!Evidential} methods in \cref{sec:evidential-neural-networks}, they are defined on the probability simplex, but in contrast are not distributions, but convex sets instead.
More intuitively, we can see a label $y$ as a sample from the conditional distribution $y \sim P(y \mid \mathbf{x})$.
Now, let the probability simplex for a classification problem with $K$ classes be defined as 

\begin{equation}
    \Delta^{K-1} = \{ \blambda = (\lambda_1, \ldots, \lambda_K)\T \mid \lambda_k \ge 0, ||\blambda||_1 = 1 \} \subset \mathbb{R}^K,
\end{equation}

\noindent and thus we can see that every $P(y \mid \mathbf{x}) \in \Delta^{K-1}$.
A credal set $\mathcal{Q}$ is now a convex subset of this simplex, i.e.\@ $\mathcal{Q} \subseteq \Delta^{K-1}$.\index{Credal sets}
As with evidential methods, ignorance about a prediction can be represented through including the whole simplex, so $\mathcal{Q} = \Delta^{K-1}$.
Since the combination with neural models is still a nascent field of research, learning credal sets can be challenging.
Existing ideas include self-supervised learning \citep{lienen2021credal, lienen2023conformal}, or creating a convex hull around predictions produced by Bayesian methods such as ensembles \citep{mortier2022calibration}.
We show an example of this for the second test point from the Iris dataset example from \cref{sec:evidential-neural-networks} in \cref{fig:credal-sets}.
It is also possible to learn credal sets from Dirichlet networks when target \emph{distributions} (instead of labels) are available \citep{javanmardi2024conformalized}, using interval neural networks\index{Neural network!Interval} (which produce intervals over predictions and activations; \citealp{wang2024creinns}).
A more complex approach involves defining credal sets\index{Credal sets} for priors and likelihood functions\index{Likelihood}, from which credal sets of posterior distributions can be learned using variational inference \citep{caprio2023credal}.
In terms of uncertainty quantification\index{Uncertainty quantification}, \citet{mortier2022calibration} develop several metrics to assess the calibration of credal predictors, and \citet{hullermeier2022quantification, sale2023volume} investigate different uncertainty metrics.\index{Uncertainty metric}
It should be mentioned that while a notion of volume of the credal sets appears as an intuitive metric (analogous to prediction set size), this intuition is flawed for multi-class classification\index{Classification!Multi-class} problems \citep{sale2023volume}.
In this regard, \citet{hullermeier2022quantification} offer alternative metrics based class dominance (whether a certain class in more likely than all others for all the distributions in the credal set), which also allows to distinguish aleatoric\index{Uncertainty!Aleatoric} from epistemic uncertainty\index{Uncertainty!Epistemic}.

\section{Uncertainty in Natural Language Processing}\label{sec:uncertainty-nlp}

Many of the approaches of uncertainty\index{Uncertainty} in the previous sections have also been applied to natural language processing\index{Natural language processing}, and we thus only mention some of the relevant works briefly:
Calibration\index{Calibration} for instance has been investigated for classification \citep{desai2020calibration, dan2021effects, xiao2022uncertainty, ulmer-etal-2022-exploring, ahuja2022calibration, park2022calibration, holm2023revisiting, chen2023close, zhu2023calibration, li2024few, ye2024benchmarking, plaut2024softmax}.
It has also been looked into in the context of generation tasks like language modeling\index{Language modeling} \citep{zhu2023calibration}, machine translation\index{Machine translation} \citep{wang2020inference}, and especially question-answering\index{Question-answering} \citep{zhang2021knowing, si2022re, si2022prompting, lin2022teaching, huang2023look, zhang2023study, geng2023survey, detommaso2024multicalibration, ulmer2024calibrating}.
Conformal prediction\index{Conformal prediction} has also been applied to NLP in various ways (see e.g.\@ \citet{campos2024conformal} for a more comprehensive survey):
These applications include natural language generation\index{Natural language generation} \citep{schuster2022confident, ravfogel2023conformal, deutschmann2024conformal, ulmer2024non}, prompt selection \citep{zollo2023prompt}, planning problems with LLMs \citep{ren2023robots}, and behavioral alignment, i.e.\@ the avoidance of toxic or otherwise undesired behaviors \citep{gui2024conformal}.
Furthermore, some works have also sought out applications of evidential deep learning\index{Deep learning!Evidential} in NLP \citep{shen2020modeling, he2023uncertainty}, however with no application to language generation at the time of writing of this thesis.

\paragraph{Token- and Sequence-Level Uncertainty.}\index{Uncertainty} 
Due to the sequentiality of language, uncertainty in NLP can be quantified on different scales.
On the one hand, we might be interested in quantifying uncertainty on a (subword-)token level in order to e.g.\@ identify mistranslations or factual errors.
On the other hand, sequence-level uncertainties are of interest when the whole generation might be unreliable, or when we are trying to assess its usefulness for a downstream task.
Similarly, there are potential applications to quantify uncertainty even on a paragraph-, document-, or dialogue-level.
In order to now quantify the uncertainty on these scales, one might intuitively resort to the approaches for frequentist networks in \cref{sec:frequentist-neural-networks}\index{Neural network!Frequentist},
i.e.\@ take the probability of the most likely token or the likelihood of a generated sequence as confidence\index{Confidence}.
This runs into multiple problems:
Due to the paraphrasticity\index{Paraphrasticity} of language (\cref{sec:uncertainty-linguistics}), a distribution over tokens might simply be uncertain due to the natural variability\index{Variability} of language, not due to the uncertainty of the model.\footnote{Neural models also have been show to be ill-calibrated towards the human word distribution, see \citealp{liu2024infini, ilia2024predict}.}
Since we would like confidence scores to reflect some notion of correctness or reliability, using the likelihood\index{Likelihood} of a generated sequence is also problematic;
for one, token probabilities likely do not reflect confidence by themselves, but there is even a mismatch between the frequency of generated sequences compared to the (true) human distribution \citep{ott2018analyzing, lebrun2022evaluating, ji2023tailoring}, implying that sequence likelihoods are not even representative as the expected relative frequency of a generated sentence.
This rules out their use to for instance reliably identify anomalous outputs.
More importantly, there is no explicit inductive bias\index{Inductive bias} in modern architecture or training procedures that models the variability directly \citep{baan2024interpreting} or would push sequence likelihoods to reflect confidence per se (see for instance the results by \citealp{xue2024comprehensive, becker2024cycles}).
While calibrating these likelihoods \citep{ulmer2024calibrating, xie2024calibrating} or reweighing token probabilities in a sequence \citep{lin2024contextualized} can lead to some success, ECE\index{Expected calibration error} results might also be misleading when comparing models with humans in a language context \citep{ilia2024predict}. % and results are not always consistent.
%Besides, \citet{kalai2023calibrated} argue that even calibrated models will not be able to avoid hallucinating.
Therefore, uncertainty on a sequence-level has instead been investigated by resampling generations (see next paragraph; \citealp{ott2018analyzing, aina2021language}).
On a token-level, several approaches have emerged, for instance computing uncertainty given specific claims \citep{fadeeva2024fact}, predicting the confidence based on the quantiles of the token distribution \citep{gupta2024language}, or training an additional prediction head \citep{kadavath2022language}.
In order to compare a wide variety of different such uncertainty metrics,\index{Uncertainty metric} \citet{huang2024uncertainty} proposed the use of \emph{rank calibration}, i.e.\@ testing whether higher certainty indeed implies higher generation quality.
Some works also exists that quantify uncertainty for long texts, for instance based on the entailment probabilities of segments \citep{zhang2024luq}, and \citet{sicilia2024deal} model the uncertainty inherent in long conversations (but therefore not the uncertainty of the model processing the conversation itself).

\paragraph{Self-consistency, Prompt Ensembling and Output Diversity.} 
While there has been some research over the years into Bayesian methods \citep{xiao2020wat, malinin2021uncertainty, gidiotis2021should, xiong2023can}, these have become less applicable in the era of large language models due to their sheer size.\footnote{
    This comes with the exception of methods like \citealp{yang2023bayesian, onal2024gaussian}.
    Besides, \citet{papamarkou2024position} sketch avenues with which Bayesian methods can still provide advantages in the age of large-scale methods.
}
Therefore, a number of works ensemble\index{Ensembling} predictions for the same input (also referred to as \emph{self-consistency}; \citealp{wang2022self, manakul2023selfcheckgpt, chen2023quantifying, li2024think}), from the same prompt with different pieces of additional information \citep{hou2023decomposing}, or from different prompts altogether \citep{li2023making, hou2023promptboosting, pitis2023boosted, gao2024spuq} instead of predictions from different parameter sets.
The intuition remains similar to Bayesian methods in \cref{sec:bayesian-neural-networks}:
If similar prompts for the same input produce vastly different predictions, the network must be uncertain.
We can therefore interpret prompt\index{Prompting} ensembling techniques as evaluating a predictive distribution over distribution of prompts\index{Predictive distribution!Posterior} $p(\brho)$ and in-context samples $p(\mathcal{C})$:

\begin{equation}\label{eq:prompt-predictive-prior}
    \mathbb{E}_{p(\brho, \mathcal{C})}\big[p(\by \mid \bx, \btheta, \brho, \mathcal{C})\big] = \int\hspace{-0.25cm}\int p(\by \mid \btheta, \bx, \brho, \mathcal{C})p(\brho)p(\mathcal{C})\ddd\!\brho\ddd\hspace{0.05cm}\!\mathcal{C}.
\end{equation}

Any disagreement in responses however can also be influenced by the generation hyperparameters, and thus this method does not admit a clean distinction between aleatoric\index{Uncertainty!Aleatoric} and epistemic uncertainty\index{Uncertainty!Epistemic} like in \cref{eq:mutual-information}.\footnote{In contrast to the claims of \citealp{hou2023decomposing}.}
Furthermore, \citet{ling2024uncertainty} investigate how the choice of in-context samples can also induce additional uncertainty into the LLMs generation.
\citet{kuhn2023semantic} base their idea of \emph{semantic entropy}\index{Entropy!Semantic} on a similar intuition:
Trough the use of a bi-directional entailment classifier, generations are clustered by meaning.\footnote{
    The idea is that if the classifier indicates that a generation implies another and vice versa, they must (ought to) be equivalent.
}
Instead of the Shannon entropy\index{Entropy!Shannon} over classes in \cref{eq:predictive-entropy}, we evaluate entropy over all the sequences $\mathbf{s}$ given some $\mathbb{M}$ out of $M$ clustered meaning classes:

\begin{align}
    \text{SE}(\mathbf{x}) & = -\sum_{m=1}^M p(\mathbb{M}_m \mid \bx) \log p(\mathbb{M}_m \mid \bx) \\
    & = -\sum_{m=1}^M \Big(\sum_{\mathbf{s} \in \mathbb{M}_m} p(\mathbf{s} \mid \bx) \Big) \log \Big(\sum_{\mathbf{s} \in \mathbb{M}_m} p(\mathbf{s} \mid \bx) \Big) \\
    & \approx - \frac{1}{M} \sum_{m=1}^M \log \Big(\sum_{\mathbf{s} \in \mathbb{M}_m} p(\mathbf{s} \mid \bx) \Big),
\end{align}

\noindent where the last step is obtained through Monte Carlo integration\index{Monte Carlo estimation}. 
\citet{aichberger2024many} improve on this estimator by producing more variable generations through targeted token substitutions.
Instead of computing the entropy over hard meaning clusters, \citet{nikitin2024kernel} propose to instead compute the entropy using semantic kernels that measure the similarity in meaning between model responses, replacing hard clusters.

\paragraph{Verbalized Uncertainty.}\label{Uncertainty!Verbalized}
Originating from works like T5 \citep{raffel2020exploring}, natural language has become a general interface for modern NLP models.
This refers both to embedding other, traditionally non-generative tasks such as sequence classification into a sequence-to-sequence task, but also to users increasingly interacting with language models through prompting\index{Prompting}.
\citet{mielke2022reducing} already demonstrated that pre-trained models could be finetuned to express different levels of uncertainty in words\index{Uncertainty}.
This however required finetuning on human-annotated data, while modern approaches simply prompt the LLM to express its uncertainty in words \citep{kadavath2022language, xiong2023can, tian2023just, chen2023reconcile}, often through percentage values (``Confidence: 96 \%'') or confidence expressions (``Confidence: Very high''), which are then mapped back onto numerical values for evaluation purposes.
\citet{tian2023just} for instance find that through the combination of suitable prompts and temperature-scaling, the calibration error of such methods can be noticeably reduced.
However, they also find that the distributions of confidence\index{Confidence} expressions are highly skewed---while it does differ between datasets, the tested GPT models (GPT-3.5 and GPT-4) tend to mostly confident expressions, likely due to the unequal usage of these terms in their training data.
This finding is corroborated by \citet{yona2024large, singh2024large, krause2023confidently}, indicating that LLMs\index{Large language model} always generate decisive answer even for uncertain questions, and that this is challenging to change through prompting alone.
When results are strong, this might coincide with cases in which the dataset is too easy and the skewed confidence expression distribution actually conforms to the results (as for instance for TriviaQA in \citealp{ulmer2024calibrating, xue2024comprehensive}).
\citet{lin2022teaching} also finetune an LLM to verbalize its uncertainty, but do so on automatically generated confidence targets that are obtained by checking the model's performance on some sub-category of a task, like different question types for mathematical reasoning.
A similar approach is taken by \citet{zhang2023r}, finetuning them to admit their uncertainty for incorrect answers.
In the case of \citet{kadavath2022language}, the LLM is simply asked directly whether its answer was true or false.
\citet{band2024linguistic} finetune verbalized uncertainty\index{Uncertainty!Verbalized} from a Bayesian decision-making standpoint, increasing factuality.
\citet{zhou2023navigating} investigate the general use of linguistic confidence expressions in LLMs, and show that accuracy can be influenced through the use of such expressions in the prompt\index{Prompting}.

\paragraph{Uncertainty for Black-box Models.}
The commercialization of LLM-based chatbots such as ChatGPT \citep{chatgpt} also created a trend of black-box models, which are shielded by an API.
As such, any UQ\index{Uncertainty quantification} method has to do without any access to model latent representations, logits or output probabilities.
The question of whether and how uncertainty can be estimated from text generations alone therefore also has become an active area of research.
Such approaches include predicting confidence directly from the generated text using an auxiliary model (\cref{ch:uncertainty-llms}; \citealp{ulmer2024calibrating}), verbalized uncertainty\index{Uncertainty!Verbalized} methods from the previous paragraph, or comparing the similarity of generations given the same input \citep{lin2023generating}.
\citet{su2024api} further show that LLM predictions can be conformalized even without access to the probabilities through repeated sampling and word frequencies analysis alone.
%A number of mentioned methods are implemented in an open-source repository by \citet{fadeeva2023lm}.

\paragraph{Reward Modeling.} 
As part of the contemporary language model pipeline, models are first pre-trained on large amounts of text using a language modeling objective \citep{devlin2019bert, radford2019language}, then finetuned on a number of instructions, and finally undergo a step that aims to align their behavior with general human values \citep{ouyang2022training}.
This last step is often performed using reinforcement learning from human feedback\index{Reinforcement learning from human feedback} (RLHF;\@ \citealp{christiano2017deep, stiennon2020learning}).
This involves the use of a trained reward model, that predicts the quality of a generation based on human preference data.
While it has been found that this step can hurt model calibration\index{Calibration} \citep{zhu2023calibration}, the reward modeling itself has also been characterized as brittle, and thus a number of works have proposed Bayesian approaches to the target model finetuning or the reward model to increase robustness \citep{zhai2023uncertainty, yang2024bayesian, zhang2024improving, zhang2024overcoming}.

\paragraph{Human Label Variation.}\index{Human label variation}
Compared to other input modalities, the variability\index{Variability}, ambiguity\index{Ambiguity} and underspecification\index{Underspecification} of language (\cref{sec:uncertainty-linguistics}) calls the validity of a single ground truth for training into question.
Indeed, there have been several calls to embrace this diversity for classification \citep{basile2021we, plank2022problem, baan2022stop, gruber2024more} and language generation tasks \citep{baan2023uncertainty}.
Importantly, this opens up new avenues for better modeling and representing of the uncertainty in the underlying data \citep{nie2020can, zhou2021distributed, uma2021learning, davani2022dealing, wu2023don}, modeling annotators \citep{deng2023you}, and to learn from fewer instances \citep{gruber2024more}.
Training on single labels or references has for instance been hypothesized to cause the miscalibration of neural models to human language variability \citep{giulianelli2023comes, ilia2024predict}, and to potentially be responsible for the inadequacy of greedy decoding in natural language generation\index{Natural language generation} \citep{eikema2020map, eikema2024effect}.
The variation of labels should therefore be reframed as an opportunity, as it for instance also allows to more easily learn second-order predictors like evidential neural networks\index{Deep learning!Evidential} or credal sets\index{Credal sets} \citep{javanmardi2024conformalized}.

\section{Uncertainty \& Trust}\label{sec:uncertainty-trust}

Even though we have already discussed several applications of uncertainty quantification in \cref{sec:applications} in the first chapter, it is useful to  
zoom in on the aspect of trust\index{Trust}, why it matters, and how quantifying the uncertainty of a ML\index{Machine learning} system can help.
The reason for this is the following:
The main promise of machine learning algorithm lies in its ability to analyze and process large swaths of data, identifying 
potential patterns that remain elusive for even the most astute humans. 
As such, it promises to either replace or support human decision-makers.
However, even if a part of the deliberation for a decision is taken over by a machine, people are the ones that remain affected by it.
This is true for all the examples of decision support including for medical staff, self-driving cars or automated translation systems.
Trust\index{Trust} is the social mechanism that governs this relationship, and is a necessary requirement for it to have a positive effects.
If trust is not present, we run the risk of alienating the people affected, leading to them ignoring the automation and thus foregoing any benefits, or even creating negative consequences.
Indeed, \citet{inie2024motivates} finds in a diverse survey that participants perceive AI\index{Artificial intelligence} systems as less trustworthy when problems they are trying to solve or the models themselves are complex, and when no human expert is in the loop.\\

\citet{jacovi2021formalizing} formalize this dynamic using notions of interpersonal trust from sociology. 
They thereby define two roles: The trustor (i.e.\@ the person trusting someone) and the trustee (i.e.\@ the person being trusted).
In order to make this distinction clearer in our context, we will notate these roles by 
\trustor and \trustee. They employ the following definition of interpersonal trust\index{Trust!Interpersonal}:

\begin{definition}[Interpersonal Trust; \citealp{mayer1995integrative}]\label{def:interpersonal-trust}
    If a \trustor\@ believes that a \trustee\@ will act in their best interest and accepts vulnerability to the \trustee's actions, 
    then the \trustor\@ trusts the \trustee.
\end{definition}

The authors admit that this definition is somewhat simplistic: AI systems are not people, and as such, terms such 
as \emph{reliance} (i.e.,\@ the trust put into an object) might be more applicable \citep{baier1986trust}. 
However, users often show tendencies to anthropomorphize AI systems \citep{miller2019explanation, jacovi2021aligning}.
And thus, we can use a variation of \cref{def:interpersonal-trust} to define human-AI trust.
\citeauthor{jacovi2021formalizing} here use the notion of contract between the \trustor\@ and the \trustee, which in the human-AI case 
has to be explicit instead of implicit. 
Such contracts define certain properties or behaviors that model is expected to uphold. 
This can include things as for instance robustness, fairness w.r.t.\@ certain group in the datasets, or interpretability and finally leads us to the definition of human-AI trust\index{Trust!Human-AI}:

\begin{definition}[Human-AI Trust; \citealp{jacovi2021formalizing}]\label{def:human-ai-trust}
    A \trustee in the form of an AI model is trustworthy if it is capable of maintaining a specific contract with the \trustor. 
\end{definition}

\citeauthor{jacovi2021formalizing} further distinguish two kinds of trust: 
\emph{Intrinsic trust}\index{Trust!Intrinsic}, when the decision process of the \trustee\@ is observable and matches the \trustor's own priors. 
This is possible in a decision tree, but very hard for neural networks, as their size can obscure the decision process.
Therefore, we focus here on \emph{extrinsic trust}\index{Trust!Extrinsic}, which is built by observing symptoms of a trustworthy model.
A symptom of a trustworthiness can for example be its (consistent) performance of the \trustee\@ model, as for instance explored by \citet{yin2019understanding,rechkemmer2022confidence}. 
While the above assumed the \trustor\@ to be human and the \trustee\@ to be an AI system, recent work has also started exploring whether AI systems can exhibit human trust behaviors \citep{xie2024can}.\\

It can be argued that one such tool for building extrinsic trust\index{Trust!Extrinsic} can be uncertainty quantification\index{Uncertainty quantification} methods:
Using them, the \trustee\@ can communicated how much weight should be assigned to its predictions, and when they are better to be ignored. 
Further, explicit contracts like in \cref{def:human-ai-trust} can be formed by providing model cards that for instance report the calibration of a model on specific datasets.
Overall, \citet{liao2022designing} describe that such trust in automation is not inherent, and that additional care has to be put into how to design the trust cues for an end user. 
For this reason, we will discuss ways of communicating uncertainty next.

\section{Communicating Uncertainty}\label{sec:communicating-uncertainty}\

\index{Uncertainty!Communication of}
Understanding the usefulness of a model can be challenging for laypeople and experts alike.
Even when possessing technical domain knowledge, NLP practitioners for instance struggle to select the best encoder model for a task \citep{bassignana2022evidence}.
Even accuracy scores or other performance metrics can be hard to interpret, especially when they may unknowingly degrade under distributional shift\index{Shift!Distributional} in an application.
The previous sections have demonstrated the diversity of ways in which uncertainty is measured, often requiring knowledge about the model, methods or entire schools of thought (as in the frequentist vs.\@ the Bayesian example).
In practice, requiring such knowledge from laypeople is unrealistic; furthermore, the interpretation of such measure is also influenced by human numeracy \citep{zikmund2007validation, galesic2010statistical} and cognitive biases \citep{reyna2008numeracy, daniel2017thinking, spiegelhalter2017risk}. 
Therefore, \citet{bhatt2021uncertainty} advocate that in practice, uncertainty measure should be tailored to and tested with the different stakeholders they are targeted towards.
This includes an arsenal of ways such as communicating numerical values, graphical means or the verbalized uncertainty\index{Uncertainty!Verbalized} from \cref{sec:uncertainty-nlp}.
However, the best way of communicating uncertainty in an NLP\index{Natural language processing} context remains application-dependent and underexplored.
One promising avenue is the verbalized uncertainty in \cref{sec:uncertainty-nlp}, although this approach at its current stage remains quite simplistic:
Usually, uncertainties are communicated as percentage values or values on a discrete scale, instead of making use of the rich variety in human uncertainty expressions (\cref{sec:expressing-uncertainty}).

\paragraph{Effects of Communicating Uncertainty.}
Some works have investigated how communicated uncertainty influences the trust of human users.\index{Uncertainty!Communication of}
For instance, \citet{zhang2020effect} show how displaying confidence\index{Confidence} scores can help to calibrate people's trust\index{Trust} in a model, but that it may not necessarily improve the outcomes of AI-assisted decision making, whereas \citet{kim2024m} find a positive effect on accuracy in a human study with LLMs\index{Large language model}.
Paradoxically, \citet{vodrahalli2022uncalibrated} show how these outcomes can be improved even when the underlying confidence scores are not calibrated.
In another experiment with human participants, \citet{dhuliawala2023diachronic} showcase how misleading uncertainty can produce lose-lose situations.
In their study, they quantify human trust in uncertainty estimates through monetary bets on a model's answers in a question-answering\index{Question-answering} task. 
They find two things in the face of unreliable uncertainty estimates:
Firstly, a smaller overall pay-off for the participants and a loss of trust in the system, both caused due to or signified by more conservative bets.
%These results together do not necessarily suggest that calibration confidence scores are useless or that increased trust does not preduce better overall outcomes, but merely that some of these question are merely not well understood and warrant future research.
In general, it should also be noted that notions like trust are notoriously hard to isolate in human experiments, and that any stated results also presuppose a specific model between model predictions and their influence on human decision-making.

\section{Applications of Uncertainty}\label{sec:applications-uncertainty}

Previous sections have focused on characterizing and quantifying uncertainty that one encounters in machine learning\index{Machine learning} and natural language processing\index{Natural language processing}.
This is not a purely intellectual quest, and we have already touched on some potential use-cases in \cref{sec:applications}.
There is exists a trove of research works on several downstream applications that uncertainty quantification\index{Uncertainty quantification} can be used for, a (non-exhaustive) list of which we present here.

\paragraph{Fairness.}\index{Fairness}
Algorithmic fairness has recently increased in popularity as a field that studies systematic biases and mitigation strategies in AI algorithms \citep{pessach2023algorithmic}. 
In this regard, some works have researched Bayesian treatments of fairness metrics \citep{ji2020can, kuzucu2023uncertainty, barrainkua2024uncertainty}.
Others have argued that uncertainty can be a source of unfairness \citep{singh2021fairness, ali2021accounting, tahir2023fairness, wang2024aleatoric, madiha2024my} and propose its quantification as a way to reduce bias during training \citep{stone2022epistemic}.
The relationship between debiasing techniques and UQ has further been investigated by \citet{kuzmin2023uncertainty}.

\paragraph{Error Detection.}\index{Error detection}
Since uncertainty estimates are usually hard to evaluate due to the lack of ground truth, and thus error detection is both a downstream application as well as an evaluation strategy.
The intuition lies in the fact that predictions with higher uncertainty should assumed to be more likely to be wrong.
Examples for this are for instance the works of \citet{kong2020calibrated, ashukha2020pitfalls, vazhentsev2022uncertainty, thuy2023explainability, vazhentsev2023hybrid}, among many others.
In the context of LLMs, uncertainty quantification has also been applied specifically to hallucination\index{Hallucination} detection \citep{xiao2021hallucination, manakul2023selfcheckgpt, zhang2023enhancing, band2024linguistic, detommaso2024multicalibration}.

\paragraph{Out-of-distribution Detection.} 
Out-of-distribution detection follows a similar logic as error detection.\index{Out-of-distribution data}
As inputs different from the training data of a model should could lead to unexpected predictions since the model is underspecified on them (i.e.\@ different models that fit the training data will create disagreeing predictions on unseen data; \citealp{d2020underspecification}), we want the model to be generally more uncertain about its prediction.
In contrast to error detection however, this applications focuses on model uncertainty, since errors can be caused by high model uncertainty\index{Uncertainty!Epistemic} or inherent difficulty alike.
This assumptions has been shown to be formally incorrect for some simple ReLU networks (\cref{sec:uq-classification-pitfalls}; \citealp{hein2019relu, ulmer2020know}), and the ability of uncertainty to detect OOD inputs has been investigated in a larger number of works (see, among many others, \citealp{devries2018learning, ovadia2019can, liu2022simple, ulmer2020trust, kong2020calibrated, stadler2021graph, arora2021types, ulmer-etal-2022-exploring, uppal2024implicit}).

\paragraph{Conditional Computation.} 
\index{Conditional computation}
Uncertainty can also be used as a signal to switch the intended way of processing for an input, which can be motivated by cognitive reasons (e.g.\@ based on system 1 and system 2 in humans; \citealp{daniel2017thinking}), boosting performance \citep{gerych2024knows} or to improve efficiency\index{Computational efficiency} \citep{schuster2022confident, varshney2022model}.
\citet{gerych2024knows} for instance use confidence scores\index{Confidence} to route inputs to a pool of models to find the best-performing one, and \citet{zheng2019self} use uncertainty to determine the right module from a mixture of experts.\index{Mixture of experts}
In NLG\index{Natural language generation}, \citet{vanderpoel2022mutual} use mutual information\index{Mutual information} to switch the decoding algorithm, and \citet{xiao2021hallucination} adapt beam search based on uncertainty in order to alleviate hallucinations\index{Hallucination}. 
Another usage of uncertainty enables the early exciting from a model, i.e.\@ where not all layers of a deep learning model are used \citep{schuster2022confident, fei2022deecap, bajpai2024ceebert}.
Lastly, uncertainty has also been utilized in \emph{model cascades}\index{Model cascade}, where we try to select one of a pool of increasingly-sized model based on the difficulty of an input \citep{teerapittayanon2016branchynet, varshney2022model, jitkrittum2024does, gupta2024language}.

\paragraph{Active Learning.}
Active learning\index{Active learning} describes a field of machine learning\index{Machine learning} in which an algorithm selects unlabeled instances that are given to a human for labeling, and are subsequently added to the algorithm's training data \citep{settles.tr09}.
The use of uncertainty measures for this purpose has long predated deep neural networks\index{Neural network} (e.g.\@ \citealp{lewis1994sequential, lewis1994heterogeneous, scheffer2001active}),
and has found many applications since their revival \citep{ren2021survey, zhang2022survey}.
When using uncertainty to identify samples of interest, there also exists a colorful bouquet of approaches:
Frequentist methods usually rely on some measure of model confidence\index{Confidence} \citep{wang2014new, matiz2019inductive, ebrahimi2020minimax, zhang2021cartography, wang2023actor}, Bayesian methods quantify metrics such as mutual information\index{Mutual information} \citep{gal2017deep, kirsch2019batchbald, kim2021task, kirsch2022unifying, smith2023prediction} and evidential methods\index{Deep learning!Evidential} utilize distributional uncertainty\index{Uncertainty!Distributional} \citep{zhu2021evidential, park2022active, hemmer2022deal}.

\paragraph{Requesting Human Oversight.}
Active learning is a specific case of human-in-the-loop problems that is focused on resource-efficient data labeling, but can be seen as just one instance of a class of applications in which human oversight or intervention is requested upon uncertainty.
Other examples include for for instance planning problems in reinforcement learning \citep{singi2023decision}, industrial applications  \citep{treiss2021uncertainty}, clarifying uncertain parts in image segmentation for remote sensing \citep{garcia2020uncertainty}, text moderation \citep{andersen2022efficient, andersen2022more}, and co-annotation of data \citep{li2023coannotating}.
In general, these applications promise to alleviate the workload that would be otherwise assigned to human experts, and only request their assistance in the case of difficult inputs.

\section{Summary}\label{sec:background-summary}

This chapter has given a fairly comprehensive account of uncertainty and its relevant concepts, definitions, methods and applications for deep learning\index{Deep learning} and natural language processing\index{Natural language processing}.
It has provided an overview over the different definitions of uncertainty in statistics, i.e.\@ the frequentist and Bayesian viewpoints, and how uncertainty in linguistics plays a layered role as an inherent feature of language on the one side, and a tool for communication of one's world state on the other.
These different notions crystallize in their applications to neural networks:
Statistical uncertainties permeate model training and inference, and linguistic uncertainties influence the processing of natural language inputs.
Not only is the quantification of these uncertainties challenging and methods to do so are multifarious, but the adequate communication of uncertainty is equally difficult.
This last step is pivotal to enable human-AI collaboration, in which trust\index{Trust} relationships are formed between users and their silicate collaborators.
As with human relationships, this trust can be built but also lost, which suggests more research is needed to understand this dynamic better.
%Lastly, applications of uncertainty quantification can also be found outside of these collaborations to support automated systems.

% Chapter 3

\chapter{Addressing Uncertainty in Experimental Design}\label{ch:uncertainty-experimental-design} % Chapter title

\epigraph{``\emph{When you run an experiment, you take notes, think for a while, then publish your results. If you don't publish, nobody will learn from your experience. The whole idea is to save other from repeating what you've done.}''}{---Clifford Stoll in \emph{The Cuckoo's Egg: Tracking a Spy Through the Maze of Computer Espionage}.}

\begin{tikzpicture}[remember picture,overlay]
    \node[anchor=north,inner sep=0pt] at (current page text area.north) {\includegraphics[width=\linewidth, clip=true, trim = 8cm 50cm 8cm 75cm]{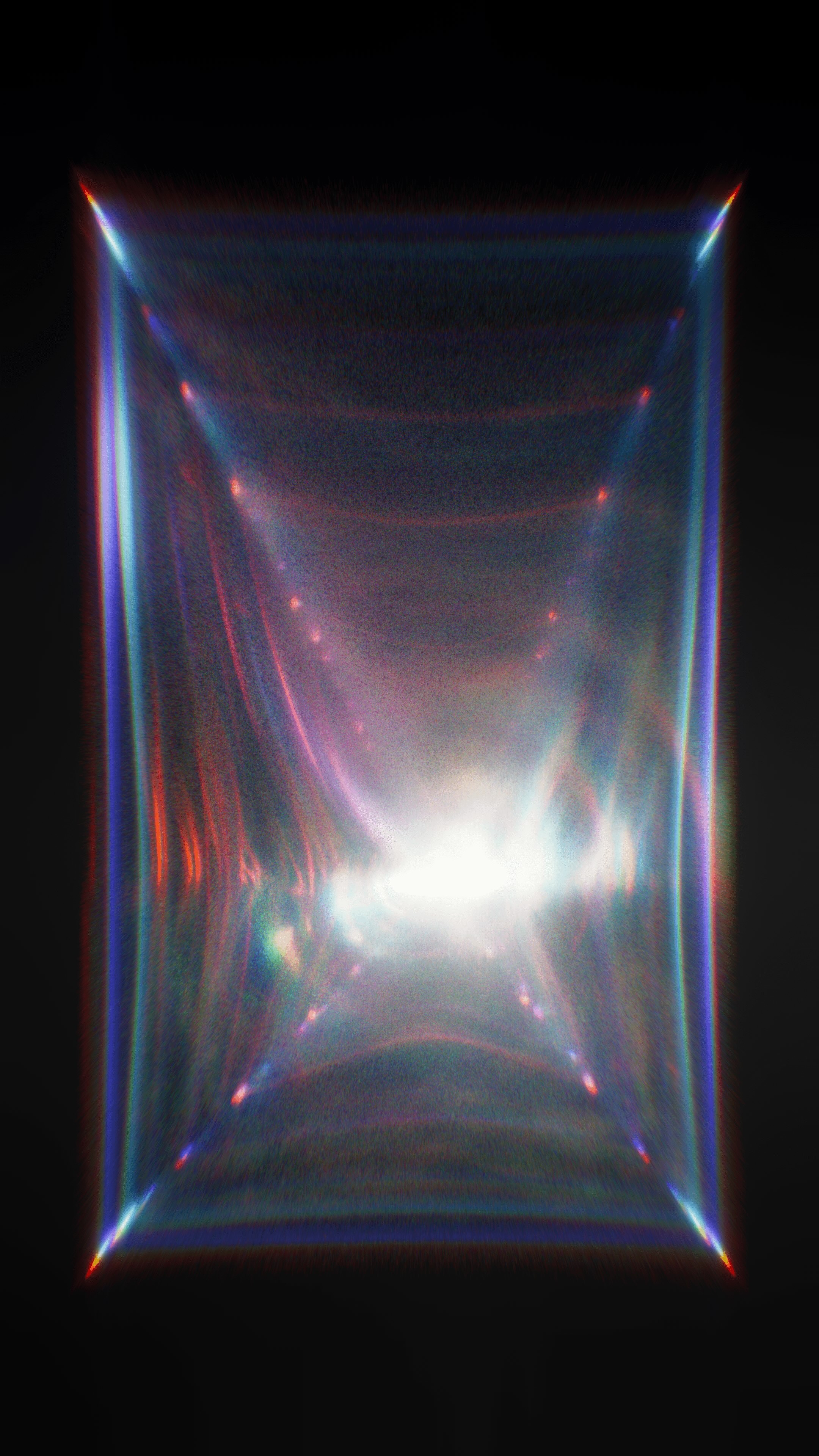}};
\end{tikzpicture}

\label{ch:methodology} % For referencing the chapter elsewhere, use \autoref{ch:mathtest}

\begin{figure}[ht]
    \centering
    \includegraphics[width=0.99\textwidth]{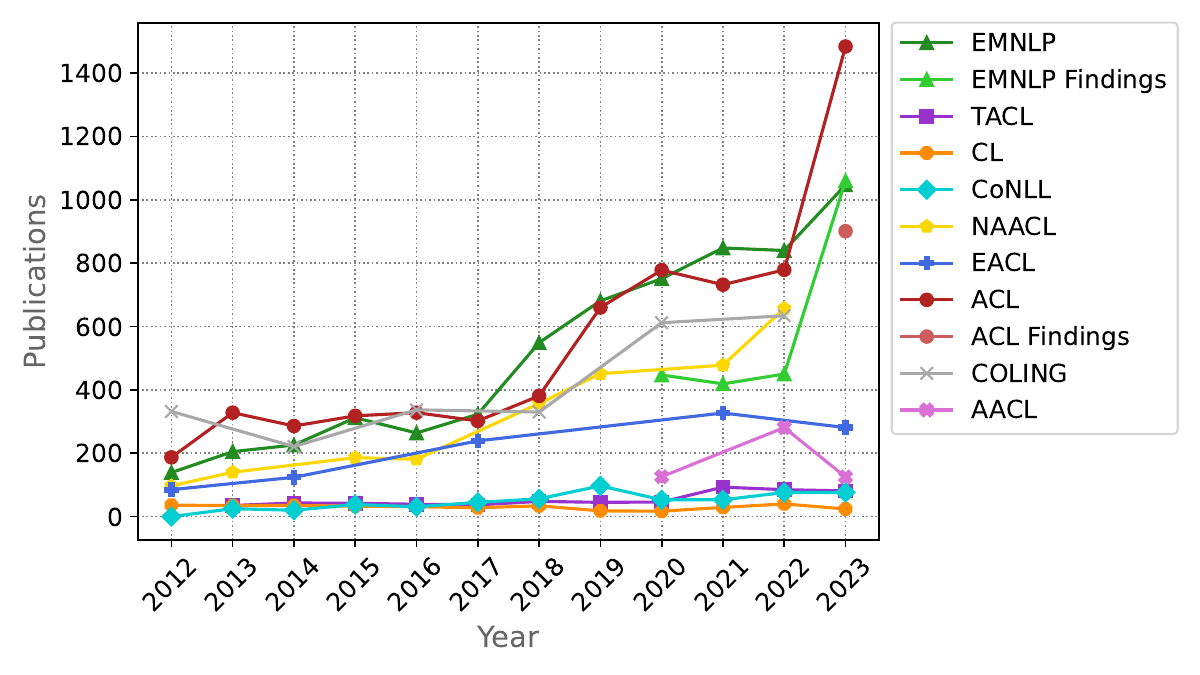}
    \caption[Published papers at NLP venues from 2012 to today.]{
        Published papers at NLP venues. Gaps are due to some venues not producing proceedings in any given year.
        Notably, this plot does not include NLP papers published at venues such as NeurIPS, ICML, or ICLR.
    }\label{fig:nlp-venues}
\end{figure}

Before returning to the uncertainty\index{Uncertainty} in NLP\index{Natural language processing} \emph{models}, we will first engage in another, wider perspective on where uncertainty hides in the NLP pipeline.
DL\index{Deep learning} in general, and NLP in its current form, are largely empirical sciences:
We obtain new knowledge by forming hypotheses, running experiments, and then analyzing results to come to a conclusion about our initial suppositions.
In the the last decade or so, this field has ballooned in size:
In \cref{fig:nlp-venues}, we show the number of published conference papers in NLP venues since 2012, which has more than quadrupled.\\

While such growth is remarkable, it comes at a cost: 
Akin to concerns in other disciplines \citep{john2012measuring, jensen2021there}, several authors have noted major obstacles to reproducibility\index{Reproducibility} \citep{gundersen2018state, belz2021systematic} and a lack of hypothesis testing\index{Hypothesis testing} \citep{marie2021scientific} or published results not carrying over to different experimental setups, for instance in text generation \citep{gehrmann2022repairing} and with respect to new model architectures \citep{narang2021transformer}.
Others have questioned commonly-accepted experimental protocols \citep{gorman2019we, sogaard2021we, bouthillier2021accounting, groot2021we} as well as the (negative) impacts of research on society \citep{hovy2016social, mohamed2020decolonial, bender2021dangers, birhane2021values} and environment \citep{strubell2019energy, schwartz2020green, henderson2020towards}. 
Lastly, the adoption of large language models\index{Large language model} that are also possibly closed-source have exacerbated problems about experimental protocols further \citep{mizrahi2023state, balloccu2024leak}.
These problems have not gone unnoticed---many of the mentioned works have proposed a cornucopia of solutions. 
In a quickly-moving environment however, keeping track and implementing these proposals becomes challenging.\\

This chapter addresses these issue in two ways:
On the one hand, open issues in reproducibility\index{Reproducibility} and replicability\index{Replicability} are woven together into a cohesive set of guidelines for gathering stronger experimental evidence, that can be implemented with reasonable effort and which are discussed in \cref{sec:reproducibility-replicability}.
On the other hand, we zoom into the question of hypothesis testing\index{Hypothesis testing} (\cref{sec:hypothesis-testing}), with a specific focus on the almost stochastic order\index{Stochastic order!Almost} test (ASO;\@ \citealp{del2018optimal, dror2019deep}) in \cref{sec:aso} and its application to question-answering\index{Question-answering} with LLMs in \cref{sec:experimental-comparison-aso}.
The core thesis of this chapter is that increased efforts in reproducibility\index{Reproducibility} and replicability\index{Replicability} are intricately linked to the question of uncertainty\index{Uncertainty} in empirical research:
For example, transparent and diligent data curation enables better modeling of uncertainty (referring to the discussion on language paraphrasticity\index{Paraphrasticity} in \cref{sec:uncertainty-linguistics} and human label variation\index{Human label variation} in \cref{sec:uncertainty-nlp}),
and a more rigorous experimental\index{Experimental design} protocol and statistical hypothesis testing can help to unveil the uncertainty lingering in results, aiding the development of better methods and bringing more clarity to the research landscape.
Therefore, we build these ideas up from the scientific method\index{Scientific method} and show their implementation in the experimental pipeline.

\section{Experimental Standards for NLP}\label{sec:reproducibility-replicability}

\begin{footnotesize}
    \vspace{-2.5ex}
    \emph{The following work is based on \citet{ulmer-etal-2022-experimental}}.\\
    \vspace{2.5ex}
\end{footnotesize}

\begin{figure}[ht]
    \centering
    \includegraphics[width=0.685\textwidth]{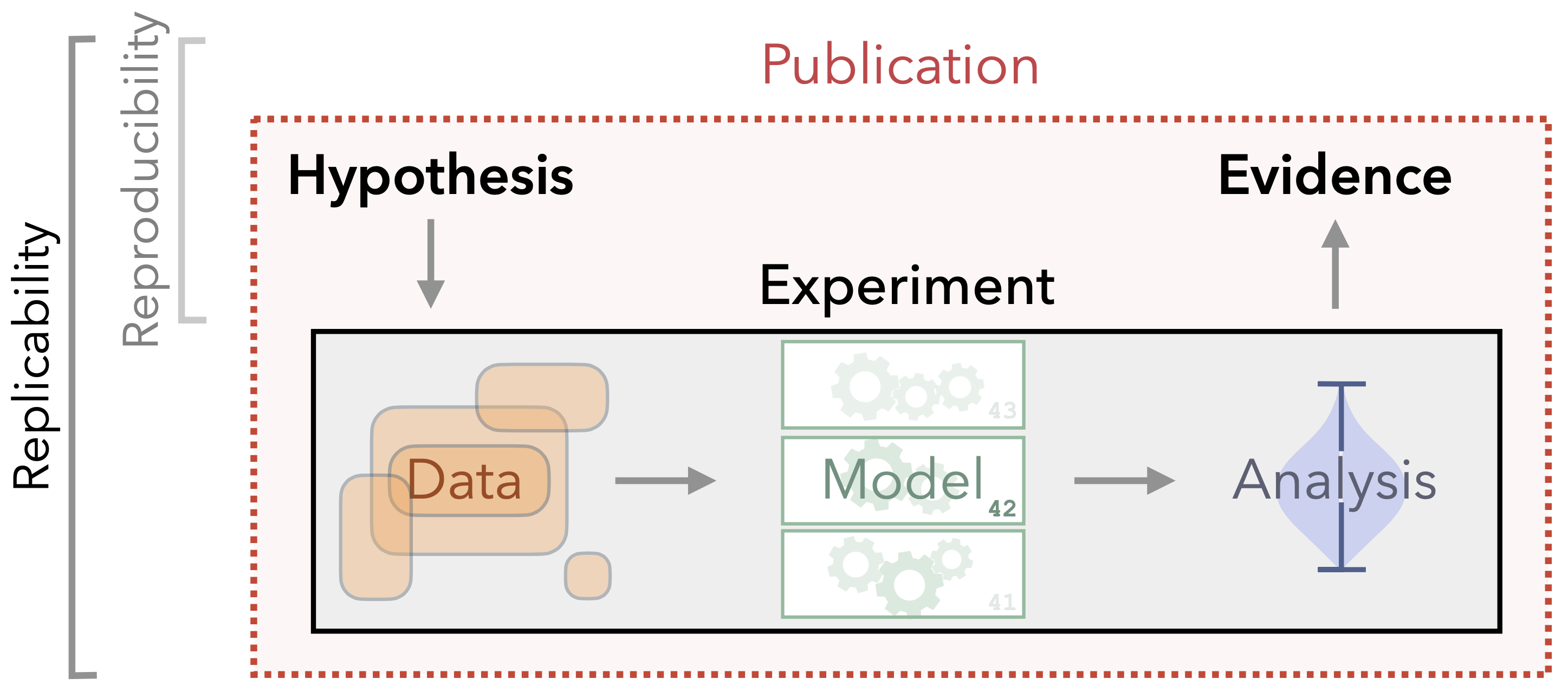}
    \caption[Schematic representation of the scientific method in deep learning.]{
        Schematic representation of the scientific method in Deep Learning.
        After forming hypotheses, we conduct our experiments by modeling some data of interest and analyzing the results to obtain some evidence to support or reject our initial assumptions.
        While reproducibility entails the \emph{reproduction} of evidence based on the hypotheses and a description of the experiments, replicability refers to a step-by-step copy of the pipeline using the original data, model, and analyses.
    }\label{fig:scientific-method-deep-learning}
\end{figure}

%Any proposed guidelines must be built on the scientific principles for generating strong evidence for the general advancement of knowledge.
%To aid that process, we defined by the following terms first:

\paragraph{The Scientific Method.}\index{Scientific method}
Knowledge can be obtained through several ways including theory building, qualitative methods, and empirical research \citep{kuhn1970structure, simon1995artificial}. 
Here, we focus on the latter aspect, in which (exploratory) analyses lead to falsifiable hypotheses that can be tested and iterated upon \citep{popper1934}.\footnote{While such hypothesis-driven science is not always applicable or possible \citep{carroll2019beyond}, it is a strong common denominator that encompasses most empirical ML\index{Machine learning} research.} 
This process requires that \emph{anyone} must be able to back or dispute these  hypotheses in the light of new evidence.\\

In the following, we focus on the evidence-based evaluation of hypotheses and how to ensure the scientific soundness of the experiments which gave rise to the original empirical evidence, with a focus on \emph{replicability}\index{Replicability} and \emph{reproducibility}\index{Reproducibility}. 
In computational literature, one term requires access to the original code and data in order to re-run experiments exactly, while the other requires sufficient information in order to reproduce the original findings even in the absence of code and original data (see also \cref{fig:scientific-method-deep-learning}).\footnote{
    Strikingly, these central terms already lack agreed-upon definitions \citep{peng2011reproducible,fokkens2013offspring,liberman2015replicability,cohen2018three}, however we follow the prevailing definitions in the NLP\index{Natural language processing} community~\citep{drummond2009repl, emnlp_rep} as the underlying ideas are equivalent.
}

\paragraph{Replicability.}\index{Replicability}
Within DL\index{Deep learning}, we take replicability to mean the (near-)exact replication of prior reported evidence. 
In a computational environment, access to the same data, code and tooling should be sufficient to generate prior results. 
However, many factors, such as hardware differences, make exact replication difficult to achieve.
Nonetheless, we regard experiments to be replicable if a practitioner is able to re-run them to produce the same evidence within a small margin of error dependent on the environment, without the need to approximate or guess experimental details.

\paragraph{Reproducibility.}\index{Reproducibility}
In comparison, we take reproducibility to mean the availability of all necessary and sufficient information such that an experiment's findings can be independently reaffirmed when the same research question is asked. 
As discussed later, the availability of all components for replicability is rare---even in a computational setting. 
An experiment then is reproducible if anyone with access to the publication is able to re-identify the original evidence, i.e.\@ exact results differing, but patterns across experiments being equivalent.
This is illustrated by \cref{fig:scientific-method-deep-learning}, where replicability involves access to all data, modeling and analysis steps, whereas reproducibility only involves knowledge of the hypotheses, a description of the experiments, as well as their results.\\

We assume that the practitioner aims to follow these principles in order to find answers to a well-motivated research question by gathering the strongest possible evidence for or against their hypotheses.
The guidelines in the following sections therefore aim to model or reduce uncertainty in each step of the experimental pipeline through enhancing its reproducibility\index{Reproducibility} and / or replicability\index{Replicability}.

\subsection{Data}\label{sec:data}

Frequently, it is claimed that a model solves a particular cognitive task, however in reality it merely scores higher than others on some specific dataset according to some predefined metric \citep{schlangen2020targeting}. 
Of course, the broader goal is to improve systems more generally by using individual datasets as proxies. 
Admitting that our experiments cover only a small slice of the real-world sample space will help more transparently measure progress towards this goal. 
In light of these limitations and as there will always be private or otherwise unavailable datasets which violate replicability, a practitioner must ask themselves: \emph{Which key information about the data must be known in order to reproduce an experiment's findings?} 
In this section we define requirements for putting this question into practice during dataset creation and usage such that anyone can draw the appropriate conclusions from a published experiment.

\paragraph{Choice of Dataset.} 
The choice of dataset arises from the need to answer a specific research question within the limits of the available resources. 
Such answers typically come in the form of comparisons between different experimental setups while using the equivalent data and evaluation metrics.
Using a publicly available, well-documented dataset will likely yield more comparable work, and thus stronger evidence. 
In absence of public data, creating a new dataset according to guidelines which closely follow prior work can also allow for useful comparisons. 
Should the research question be entirely unexplored, creating a new dataset will be necessary.
In any case, the data itself must contain the information necessary to generate evidence for the researcher's hypothesis. 
For example, a model for a classification task will not be learnable unless there are distinguishing characteristics between data points and consistent labels for evaluation. 
Therefore, an exploratory data analysis is recommended for assessing data quality and anticipating problems with the research setup. 
Simple baseline methods such as regression analyses or simply manually verifying random samples of the data may provide indications regarding the suitability and difficulty of the task and associated dataset \citep{caswell2021quality}.
On the flip side, a lower-quality dataset runs the danger of introducing noise and therefore aleatoric uncertainty into the dataset \citep{baan2023uncertainty}.

\paragraph{Metadata.} 
At a higher level, data sheets and statements \citep{gebru2020,bender2018data} aim to standardize metadata for dataset authorship in order to inform future users about assumptions and potential biases during all levels of data collection and annotation---including the research design \citep{hovyshrimai21}. 
Simultaneously, they encourage reflection on whether the authors are adhering to their own guidelines \citep{waseem2021disembodied}. 
Generally, higher-level documentation should aim to capture the dataset's \emph{representativeness} with respect to the global population. 
This is especially crucial for ``high-stakes'' environments in which subpopulations may be disadvantaged due to biases during data collection and annotation \citep{he2019practical, sap2021annotators}. 
Even in lower-stake scenarios, a model trained on only a subset of the global data distribution can have inconsistent behavior when applied to a different target data distribution and display high model uncertainty \citep{d2020underspecification,wilds2020}. 
For instance, domain differences have a noticeable impact on model performance \citep{white2021examining,rameshkashyap2021}. 
Increased data diversity can improve the ability of models to generalize to new domains and languages \citep{benjamin2018}, however diversity is difficult to quantify \citep{gong2019} and full coverage is unachievable. 
This highlights the importance of documenting representativeness in order to ensure reproducibility\label{Reproducibility}---even in absence of the original data.
For replicability using the original data, further considerations include long-term storage and versioning, as to ensure equal comparisons in future work.

\paragraph{Instance Annotation.} 
Achieving high data quality requires that the data must be accurate and relevant for the task to enable effective learning \citep{pustejovsky2012natural,bestpractices} and reliable evaluation \citep{bowman2021will,basile2021we}. 
Since most datasets involve human annotation, a careful annotation design is crucial \citep{pustejovsky2012natural,paun2022statistical}. 
Ambiguity\index{Ambiguity} in natural language poses inherent challenges and disagreement is genuine (see \cref{sec:uncertainty-linguistics,sec:uncertainty-nlp} or \citealp{basile2021we,lucia-keynote,uma2021learning, plank2022problem}). 
As insights into the annotation process are valuable, yet often inaccessible, we recommend to release datasets with individual-coder annotations, as also put forward by \citet{basile2021we,prabhakaran2021releasing,plank2022problem} and to complement data with insights like statistics on inter-annotator coding \citep{paun2022statistical}, e.g., over time \citep{braggaar2021challenges}, or coder uncertainty \citep{elisaspaper}. 
When creating new datasets such information strengthens the reproducibility of future findings, as they transparently communicate the inherent variability instead of obscuring it.
Furthermore, this opens up new avenues to model distributions instead of single gold labels to more accurately reflect uncertainty \citep{javanmardi2024conformalized, gruber2024more} or modeling single annotators \citep{deng2023you}.

\paragraph{Pre-processing.} 
Given a well-constructed or well-chosen dataset, the first step of an experimental setup will be the process by which a model takes in the data. 
This must be well documented or replicated---most easily by publishing the associated code---as perceivably tiny pre-processing choices can lead to huge accuracy discrepancies \citep{fokkens2013offspring} and influences model uncertainty\index{Uncertainty!Epistemic} during inference.\footnote{
    One could for instance imagine a case where data uncertainty is created by not removing certain characters like rare symbols or fragments of code, or increasing model uncertainty through suboptimal tokenization of a language, for instance through another language's or multilingual tokenizer \citep{rust2020good}.
}
 Typically, this involves decisions such as sentence segmentation, tokenization and normalization. 
 In general, the data setup pipeline should ensure that a model ``observes'' the same kind of data across comparisons.
Next, the dataset must be split into representative subsamples which should only be used for their intended purpose, i.e.\@ model training, tuning and evaluation (see \cref{sec:experiments-analysis}). 
In order to support claims about the generality of the results, it is necessary to use a test split without overlap with other splits.
Alternatively, a tuning / test set could consist of data that is completely foreign to the original dataset \citep{ye2021ood}, ideally even multiple sets \citep{bouthillier2021accounting}, which is also essential when trying to quantify any model uncertainty in the face of distributional drifts\index{Shift!Distributional} (see \cref{sec:benchmarking-nlp-uncertainty}). 
It should be noted that even separate, static test splits are prone to unconscious ``overfitting'', if they have been in use for a longer period of time, as people aim to beat a particular benchmark \citep{gorman2019we}.
If a large variety of resources are not available, it is also possible to construct challenging test sets from existing data \citep{ribeiro2020beyond,kiela2021dynabench,sogaard2021we}.
Finally, the metrics by which models are evaluated should be consistent across experiments and thus benefit from standardized evaluation code \citep{dehghani2021benchmark}. 
For some tasks, metrics may be driven by community standards and are well-defined (e.g.\@ classification accuracy). 
In other cases, approximations must stand in for human judgment (e.g.\@ in machine translation\index{Machine translation}). 
In either case---but especially in the latter---dataset authors should inform users about desirable performance characteristics and recommended metrics.

\paragraph{Appropriate Conclusions.} 
The results a model achieves on a given data setup should first and foremost be taken as just that. 
Appropriate, broader conclusions can be drawn using this evidence provided that biases or incompleteness of the data are addressed (e.g., results only being applicable to a subpopulation).
Even with statistical tests for the significance of comparisons\index{Hypothesis testing}, properties such as the size of the dataset and the distributional characteristics of the evaluation metric may influence the statistical power\index{Statistical power} of any evidence gained from experiments \citep{card2020with}. 
In experiments with large models, practitioners might decide to only run the model on a subset of the data.
But again, such a sample might not be powerful enough and not enable fair comparisons with other models \citep{balloccu2024leak}.
It is therefore important to keep in mind that in order to claim the reliability of the obtained evidence, for example, larger performance differences are necessary on less data than what might suffice for a large dataset, or across multiple comparisons (see \cref{sec:experiments-analysis}).
Finally, a practitioner should be aware that a model's ability to achieve high scores on a certain dataset may not be directly attributable to its capability of simulating a cognitive ability, but rather due to spurious correlations in the input \citep{ilyas2019adversarial, schlangen2020targeting, nagarajan2021understanding}. 
By for instance only exposing models to a subset of features that should be inadequate to solve the task, we can sometimes detect when they take unexpected shortcuts \citep{fokkens2013offspring,zhou2015simple}. 
Communicating the limits of the data helps future work in reproducing prior findings more accurately.

\definecolor{betterorange}{HTML}{EBAA65}
\begin{tcolorbox}[
    title=\centering{\textcolor{black}{Best Practices: \textbf{Data}}}, 
    colback=white,
    colframe=betterorange, 
    left=2pt, 
    right=8pt, 
    enlarge top by=5pt,
    boxsep=2pt,
    coltext=black
]
    \footnotesize

  \begin{itemize}[leftmargin=12pt]
      \small
      \itemsep0em
      \item[$\diamond$] Consider dataset \& experimental limitations \citep{schlangen2020targeting};
      \item[$\diamond$] Document task adequacy, representativeness and pre-processing \citep{bender2018data};
      \item[$\diamond$] Split the data such as to avoid spurious correlations;
      \item[$\diamond$] Publish the dataset accessibly \& indicate changes;
\item[$\star$] Perform exploratory data analyses to ensure task adequacy \citep{caswell2021quality};
       \item[$\star$] Publish the dataset with individual-coder annotations;
      \item[$\star$] Consider the dataset's statistical power \citep{card2020with}.
  \end{itemize}
\end{tcolorbox}

\subsection{Codebase \& Models}\label{sec:model-codebase}

%\looseness=-1
The NLP community has historically taken pride in promoting open access to papers, data, code, and documentation, but some have also noted room for improvement \citep{wieling2018reproducibility,belz2021systematic}. 
%One practice has been to open-source all components of the experimental procedure in a repository, consisting of all code, necessary scripts, and detailed documentation. 
The benefit of such a repository is in its ability to enable direct \emph{replication}\index{Replicability}, helping to reduce uncertainty in modeling when building upon others work.
%In particular, a comprehensive code base directly enables replicability.
%In practice, such documentation is often communicated through a \texttt{README} file, in which user-oriented information is described. %\footnote{In~\cref{app:readme}, we propose minimal requirements for a \texttt{README} file and give pointers on files and code structure.} 
In DL\index{Deep learning} however, full datasets can be large and impractical to share. 
Due to their importance however, it is essential to carefully consider how one can share the data with researchers in the future. 
Therefore, repositories for long-term data storage backed by public institutions should be preferred (e.g.\@ LINDAT / CLARIN by \citealp{varadi2008clarin}).
Nevertheless, practitioners often can not distribute data due to privacy, legal, or storage reasons. 
In such cases, practitioners must instead carefully consider how to distribute data and tools to allow future research to produce accurate replications of the original data \citep{zong2020}. 

\paragraph{Hyperparameter Search.} 
Hyperparameter tuning strategies remain an open area of research (e.g.\@ \citealp{bischl2021hyperparameter}), but are central to the replication of contemporary models.
Well-chosen hyperparameters promote stability in model predictions, while ill-chosen parameters induce additional additional uncertainty.\footnote{
    Whether such uncertainty would be aleatoric\index{Uncertainty!Aleatoric} or epistemic\index{Uncertainty!Epistemic} is difficult to decide;
    while more data could compensate for suboptimal hyperparameter values, it is intuitive that a model will be unlikely to converge and reduce its uncertainty for e.g.\@ adversarially chosen values.
    This reinforces the argument by \citet{baan2023uncertainty} that data and model uncertainty should not be seen as a dichotomy, but rather as a spectrum.
} 
The following rules of thumb exist: 
Grid search or Bayesian optimization can be applied if few parameters can be searched exhaustively under the computation budget. 
Otherwise, random search is preferred, as it explores the search space more efficiently \citep{bergstra2012random}. 
Advanced methods like Bayesian optimization \citep{snoek2012practical} and bandit search-based approaches \citep{li2017hyperband} can be used as well if applicable \citep{bischl2021hyperparameter}. 
To avoid unnecessary guesswork, the following information is expected: 
Hyperparameters that were searched per model (including options and ranges), the final hyperparameter settings used, number of trials, and settings of the search procedure if applicable.
As tuning of hyperparameters is typically performed using specific parts of the dataset, it is essential to note that any modeling decisions based on them automatically invalidate their use as \emph{test} data.

\paragraph{Models.}
Contemporary models (e.g.\@  \citealp{vaswani2017attention, devlin2019bert, dosovitskiy2021image, chen2021decision, touvron2023llama, touvron2023llama2, llama3modelcard, jiang2024mistral, groeneveld2024olmo}) have very large computational and memory footprints. 
To avoid retraining models, and more importantly, to allow for replicability, it is recommended to save and share model weights. 
This may face similar challenges as those of datasets (namely, large file sizes), but it remains an impactful consideration. 
In most cases, simply sharing the best or most interesting model could suffice, although sharing multiple models enables more robust significance testing and allows for modeling of uncertainty through ensembling\index{Ensembling} (\cref{sec:bayesian-neural-networks}).
It should be emphasized that distributing model weights should always complement a well-documented repository as libraries and hosting sites might not be supported in the future.

\paragraph{Model Evaluation.}
The exact model and task evaluation procedure can differ significantly (e.g.\@ \citealp{post2018call}).
It is important to either reference the exact evaluation script used (including parameters, citation, and version, if applicable) or include the evaluation script in the codebase. 
Moreover, to ease error or post-hoc analyses, we highly recommend saving model predictions whenever possible and making them available at publication \citep{card2020with, gehrmann2022repairing}
and using standardized and tested implementations (e.g.\@ \citealp{von2022evaluate}).
Using single metrics can also distort results or paint a restrictive picture, which is why using multiple different evaluation metrics is commendable \citep{marie2021scientific}.

\paragraph{Model Cards.} 
Apart from quantitative evaluation and optimal hyperparameters, \citet{mitchell2019model} propose model cards: 
A type of standardized documentation, as a step towards responsible ML\index{Machine learning} and AI\index{Artificial intelligence} technology, accompanying trained ML models that provide benchmarked evaluation in a variety of conditions, across different cultural, demographic, or phenotypic and intersectional groups that are relevant to the intended application domains. 
They can be reported in the paper or project, and can help to collect important information for reproducibility\index{Reproducibility}, such as preprocessing and evaluation results. 
We refer to \citet{mitchell2019model,menon2020pulse} for examples of model cards.

\begin{tcolorbox}[
    title={\centering\textcolor{black}{Best Practices: \textbf{Codebase \& Models}}}, 
    colback=white,
    colframe=red!25!green!70!blue!55, enlarge top by=5pt,
    left=2pt, right=8pt,
    boxsep=2pt,
    coltext=black
]
  \footnotesize
  \begin{itemize}[leftmargin=12pt]
      \small
      \itemsep0em
    %   \item[$\diamond$] Publish a code repository with detailed documentation including licensing to distribute code for replicability;
      \item[$\diamond$] Publish a code repository with documentation and license;
      \item[$\diamond$] Report all details about hyperparameter search and model training;
      \item[$\diamond$] Specify the hyperparameters for replicability;
      \item[$\diamond$] Publish model predictions and evaluation scripts.;
      \item[$\diamond$] Use multiple, complementary evaluation metrics;
      \item[$\star$] Use model cards;
      \item[$\star$] Publish models; 
  \end{itemize}
\end{tcolorbox}

\subsection{Experiments \& Analysis}\label{sec:experiments-analysis}

Experiments and their analyses constitute the core of most scientific works, and empirical evidence is valued especially highly in ML research \citep{birhane2021values}. 
However, there are common issues that practitioners are faced with model training and experimental analyses, for which we discuss counter-strategies here.

\paragraph{Model Training.} 
For model training, it is advisable to set a random seed for replicability\index{Replicability}, and train multiple initializations per model in order to obtain a sufficient sample size for later statistical tests. 
The number of runs should be adapted based on the observed variance: Using for instance bootstrap power analysis, existing model scores are raised by a constant compared to the original sample using a significance test in a bootstrapping\index{Bootstrap} procedure \citep{yuan2003bootstrap,tuffery2011data, henderson2018deep}. 
If the percentage of significant results is low, we should collect more scores.\footnote{The resulting tensions with modern DL hardware requirements are discussed in \cref{sec:discussion}.}
\citet{bouthillier2021accounting} further recommend to vary as many sources of randomness in the training procedure as possible (i.e., data shuffling, data splits etc.) to obtain a closer approximation of the true model performance. 
When training more runs is not feasible such as in the case of LLMs\index{Large language model}, we can for instance obtain additional observations by varying the generation process (see for instance the case study in \cref{sec:case-study}).
Nevertheless, any drawn conclusion are still surrounded by a degree of statistical uncertainty, which can be combated by the use of statistical hypothesis testing\index{Hypothesis testing}.

\paragraph{Significance Testing.} \index{Hypothesis testing}
Using deep neural networks\index{Neural network}, a number of (stochastic) factors such as the random seed \citep{dror2019deep} or even the choice of hardware \citep{yang2018design} or framework \citep{leventi2022deep} can influence performance and need to be taken into account. 
First of all, the size of the dataset should support sufficiently powered statistical analyses (see \cref{sec:data}).
Secondly, an appropriate significance test should be chosen. 
We give a few rules of thumb based on \citet{dror2018hitchhiker}:
When the distribution of scores is known, for instance a normal distribution\index{Normal distribution} for the Student's-$t$ test\index{Student's-$t$ test}, a \emph{parametric} test should be chosen. 
Parametric tests are designed with a specific distribution for the test statistic in mind, and have strong statistical power (i.e.\@ a lower Type II error\index{Type II error}). 
The underlying assumptions can sometimes be hard to verify (see \citealp{dror2018hitchhiker},  Section 3.1), thus when in doubt \emph{non-parametric} tests can be used. 
This category features tests like the bootstrap\index{Bootstrap}, employed in case of a small sample size, or the Wilcoxon signed-rank test\index{Wilcoxon signed-rank test} \citep{wilcoxon1992individual}, when plenty observations are available. 
Depending on the application, the usage of specialized tests might furthermore be desirable \citep{dror2019deep, agarwal2021deep}. 
We also want to draw attention to the fact that comparisons between multiple models and / or datasets, \emph{require} an adjustment of the confidence level, for instance using the Bonferroni correction\index{Bonferroni correction} \citep{bonferroni1936teoria}, which is a safe and conservative choice and easily implemented for most tests \citep{dror2017replicability, ulmer2022deep}. 
\citet{azer2020not} provide a guide on how to adequately word insights when a statistical test was used, and \citet{greenland2016statistical} list common pitfalls and misinterpretations of results. 
Due to spatial constraints, we refer to \cref{sec:hypothesis-testing} for a slightly more technical introduction to the topic. %\citet{dror2018hitchhiker,raschka2018model} for a general introduction to the topic and \citet{azer2020not} for an overview over Bayesian significance tests. In \cref{app:experimental-analysis}, we also list a number of resources, such as Bayesian significance tests by \citet{azer2020not}, an implementation of the test by \citet{dror2019deep} by \citet{ulmer2022deep} and a test framework that is adapted for deep reinforcement learning by \citet{agarwal2021deep}. With the necessary tools at hand, we can now return to carefully answer the original research questions.
Current trends surrounding LLMs further make significance testing challenging, as training and evaluating multiple different model runs can be prohibitively expensive. 
We explore different strategies in this restrictive setting in the case study in \cref{sec:case-study}.

\paragraph{Critiques \& Alternatives.} 
Although statistical hypothesis testing\index{Hypothesis testing} is an established tool in many disciplines, its (mis-)use has received criticism for decades \citep{berger1987testing, demvsar2008appropriateness, ziliak2008cult}. 
For instance, \citet{wasserstein2019moving} criticize the $p$-value as reinforcing publication bias through the dichotomy of ``significant'' and ``not significant'', i.e.\@ by favoring positive results \citep{locascio2017results}. 
Instead, \citet{wasserstein2019moving} propose to report it as a continuous value and with the appropriate scepticism.\footnote{
    Or, as \citet{wasserstein2019moving} note: ``\emph{statistically significant}---don't say it and don't use it''.
} 
In addition to statistical significance, another approach advocates for reporting \emph{effect size} \citep{berger1987testing, lin2013research}, so for instance the mean difference, or the absolute or relative gain in performance for a model compared to a baseline. 
The effect size can be modeled using Bayesian analysis \citep{kruschke2013bayesian, benavoli2017time}, which better fit the uncertainty surrounding experimental results, but requires the specification of a plausible statistical model producing the observations\footnote{
    Here, we are \emph{not} referring to a neural network\index{Neural network}, but instead to a process generating experimental observations, specifying a prior\index{Prior distribution} and likelihood\index{Likelihood} for model scores. 
    Conclusions are drawn from the posterior distribution over parameters of interest (e.g.\@ the mean performance), as demonstrated by \citet{benavoli2017time}.
} 
and potentially the usage of markov chain Monte Carlo\index{Markov chain Monte Carlo} sampling \citep{brooks2011handbook, gelman2021bayesian}. 
\citet{benavoli2017time} give a tutorial for applications to ML and supply an implementation of their proposed methods in a software package and guidelines for reporting details are given by \citet{kruschke2021bayesian}, including for instance the choice of model and priors. 

\begin{tcolorbox}[
    title={\centering\textcolor{black}{Best Practices: \textbf{Experiments \& Analysis}}}, 
    colback=white,
    colframe=red!25!green!20!blue!25, 
    left=2pt, right=8pt, enlarge top by=5pt,
    boxsep=2pt,
    coltext=black
]

\footnotesize
   \begin{itemize}[leftmargin=12pt]
      \small
      \itemsep0em
       \item[$\diamond$] Report mean \& standard dev.\ over multiple runs;
       \item[$\diamond$] Perform significance testing or Bayesian analysis and motivate your choice of method;
       \item[$\diamond$] Carefully reflect on the amount of evidence regarding your initial hypotheses.
   \end{itemize}
\end{tcolorbox}

\subsection{Discussion}\label{sec:discussion}

%Since previous sections have emphasized the need to overhaul some experimental standards, we dedicate this last section to discuss some structural issues that might pose obstacles to this. 
%Here, we discuss some of the tensions in the field regarding these endeavors.\\
Previous sections have emphasized the need to overhaul some experimental standards and have describes their interactions with reducing and modeling uncertainty.\index{Experimental design}
But specifically with regard to statistical significance in \cref{sec:experiments-analysis}\index{Hypothesis testing}, there is a stark conflict between the hardware requirements of modern methods \citep{sevilla2022compute} and the computational budget of the average researcher. 
Only the best-funded research labs can afford the increasing computational costs to account for the statistical uncertainty of results and to reproduce prior works \citep{hooker2021hardware}. 
Under these circumstances, it becomes difficult to judge whether the results obtained via larger models and datasets \emph{actually} constitute substantial progress or just statistical flukes. 
While we present some alternatives in \cref{sec:case-study}, this environment also make the use of traditional Bayesian DL techniques like in \cref{sec:bayesian-neural-networks} more challenging.
For this reason, researchers should embrace data variability as a new avenues for modeling and reducing uncertainty in large contemporary models (as discussed in \cref{sec:data}).\\
%At the same time, such experiments can create environmental concerns \citep{strubell2019energy, schwartz2020green}.\footnote{
%    E.g.\@ GPT-3's training was estimated to have cost ca.\ 12M USD \citep{turner2020gpt3} or 188,702 kWh \citep{anthony2020carbontracker}.
%    The recently released Llama 3 $70B$ model required $700$ Watts of energy and emitted 1900 $\text{tCO}_2\text{eq}$ \citep{llama3}.
%}
%The community must decide collectively whether these factors, including impeded reproducibility and weakened empirical evidence, constitute a worthy price for the knowledge obtained from training large neural networks. 
%Further, it is an open question on how to develop new methods evaluation that could address these concerns.\\

Echoing our fundamental deliberations about the scientific process in \cref{sec:reproducibility-replicability}, being able to (re-)produce empirical findings is critical for scientific progress, particularly in fast-growing fields like NLP \citep{manning2017last}. 
To reduce the risks of a reproducibility crisis and unreliable research findings \citep{Ioannidis2005}, experimental rigor is imperative. 
Being aware of possible harmful implications and to avoid them is therefore important, since every step can carry possible biases \citep{hovyshrimai21, waseem2021disembodied}.
This chapter aims at providing a toolbox of actionable recommendations, and a reflection and summary of the ongoing broader discussion. 
To improve the experimental standard in the field overall, we can distill the following suggestions:
\textbf{As researchers}, we can start implementing the recommendations in this work in order to drive bottom-up change and reach a critical mass \citep{centola2018experimental}. 
\textbf{As reviewers}, we can shift focus from results to more rigorous methodologies \citep{rogers2021how}, and allow more critiques and reproductions of past works and meta-reviews to be published \citep{birhane2021values,lampinen2021publishing}. 
\textbf{As a community}, we can change the incentives around research and experiment with new initiatives.
With concrete best practices to raise awareness and a call for uptake, we hope to aid researchers in their empirical endeavors.
The rest of this chapter is dedicated to the practice of statistical hypothesis testing and its challenges in the era of LLMs.

\section[Statistical Hypothesis Testing]{Statistical Hypothesis Testing}\label{sec:hypothesis-testing}

\begin{footnotesize}
    \vspace{-2.5ex}
    \emph{The following work is based on \citet{ulmer2022deep}}.\\
    \vspace{2.5ex}
\end{footnotesize}

\index{Hypothesis testing}
In this part of the chapter, we are discussing statistical hypothesis testing with an application to comparing two models or \emph{algorithms}.
While terms like model or algorithms will be used almost synonymously in the rest of this thesis, it will aid the rest of this chapter to define these notions better.

\begin{definition}[Model]
    We define a model $f_{\btheta}$ to be the element of some hypothesis class $f_{\btheta} \in \mathbb{H}$.
    Here, the hypothesis class is loosely defined as all neural predictors trained using the same architecture and training data.  
    %to be the set of predictors $\{f_{\btheta}\}_{\btheta \in \bTheta}$ with the same parameterization $\btheta \in \bTheta$ and architecture.
\end{definition}

Importantly, the above definition does not imply that all predictors $\mathbb{H}$ comprise the same parameter values---they can be influenced by factors such as random seeds or the order of training samples, and in the case of LLMs, the use of different generation hyperparameters or prompt templates.

\begin{definition}[Metric \& Observation]
    Let us define $\phi: \mathbb{H} \times \mathcal{P}(\mathbb{D}) \rightarrow \mathbb{R}$ to be a function measuring the performance of a predictor $f_\theta$ on some dataset $\mathbb{D} \in \mathcal{P}(\mathbb{D})$ in form of a real number $s \in \mathbb{R}$, called \emph{observation} or \emph{score}, with $\phi$ called the \emph{metric}.
\end{definition}

We will assume in the following that a higher number for $s$ indicates a more desirable behavior. 
Now, we let $\mathbb{S}_\mathbb{A}$ denote a set of observations obtained from different instances of a specific hypothesis class $\mathbb{A}$. 
Ideally for deep neural networks, obtaining a set of observations $\mathbb{S}_\mathbb{A}$ would involve training multiple \emph{instances} of a network with the same architecture using different sets of hyperparameters and random initializations. 
Since the former part often becomes computationally infeasible in practice, we follow the advice of \citet{bouthillier2021accounting} and assume that it is obtained by fixing one set of hyperparameters after a prior search and varying as many other random elements as possible. 
Here, we only give a very brief introduction into statistical hypothesis testing\index{Hypothesis testing} using $p$-values, and refer the reader to resources such as \citet{japkowicz2011evaluating, dror2018hitchhiker, raschka2018model, azer2020not, dror2020statistical, riezler2021validity} for a more comprehensive overview. 
Using the introduced notation, we can define a one-sided test statistic $\delta(\mathbb{S}_\mathbb{A}, \mathbb{S}_\mathbb{B})$ based on the gathered observations. 
An example of such test statistics is for instance the difference in observation means $\delta(\mathbb{S}_\mathbb{A}, \mathbb{S}_\mathbb{B}) = \hat{\mu}_\mathbb{A} - \hat{\mu}_\mathbb{B}$ with $\mu_{(\cdot)} = \frac{1}{|\mathbb{S}_{(\cdot)}|}\sum_{s_i \in \mathbb{S}_{(\cdot)}}s_i$.
We then formulate the following null hypothesis\index{Null hypothesis}:

\begin{equation}
    \text{H}_0:\ \delta(\mathbb{S}_\mathbb{A}, \mathbb{S}_\mathbb{B}) \le 0.
\end{equation}

The null hypothesis $\text{H}_0$ \index{Null hypothesis} assumes the opposite of our desired case, namely that $\mathbb{A}$ is not better than $\mathbb{B}$, but equally as good or worse, as indicated by the value of the test statistic. 
Usually, the goal becomes to reject this null hypothesis. 
$p$-value testing is a frequentist method in the realm of statistical hypothesis tests\index{Hypothesis testing}. 
It introduces the notion of data that \emph{could have been observed} if we were to repeat our experiment again using the same conditions, which we will write with superscript $^\text{rep}$ in order to distinguish them from our actually observed scores \citep{gelman2021bayesian}. 
We then define the $p$-value as the probability that, under the null hypothesis\index{Null hypothesis} $\text{H}_0$, the test statistic using replicated observations is larger than or equal to the \emph{observed} test statistic:

\begin{equation}
    p(\delta(\mathbb{S}_\mathbb{A}^\text{rep}, \mathbb{S}_\mathbb{B}^\text{rep}) \ge \delta(\mathbb{S}_\mathbb{A}, \mathbb{S}_\mathbb{B}) \mid \text{H}_0).
\end{equation}

We can interpret this expression as follows: 
Assuming that $\mathbb{A}$ is not better than $\mathbb{B}$, the test assumes a corresponding distribution of statistics that $\delta$ is drawn from. 
So how does the observed test statistic $\delta(\mathbb{S}_\mathbb{A}, \mathbb{S}_\mathbb{B})$ fit in here? 
This is what the $p$-value expresses: 
When the probability is high, $\delta(\mathbb{S}_\mathbb{A}, \mathbb{S}_\mathbb{B})$ is in line with what we expected under the null hypothesis\index{Null hypothesis}, so we \emph{cannot} reject the null hypothesis, or in other words, we \emph{cannot} conclude $\mathbb{A}$ to be better than $\mathbb{B}$.
If the probability is low, that means that the observed $\delta(\mathbb{S}_\mathbb{A}, \mathbb{S}_\mathbb{B})$ is quite unlikely under the null hypothesis and that the reverse case is more likely---i.e.\@ that it is likely larger---and we conclude that $\mathbb{A}$ is indeed better than $\mathbb{B}$. 
In summary, the question that a $p$-value asks can be stated as follows:
Assuming the null hypothesis to be true, how likely is a test statistic to be at least as extreme as observed?
Note that \textbf{the $p$-value does not express whether the null hypothesis is true}. 
To make our decision about whether or not to reject the null hypothesis, we typically determine a threshold---the significance level $\alpha$, often set to $0.05$---that the $p$-value has to fall below. However, it has been argued that a better practice involves reporting the $p$-value alongside the results without a pigeonholing of results into significant and non-significant \citep{wasserstein2019moving}. 

\subsection{Almost Stochastic Order}\label{sec:aso}

Deep neural networks\index{Neural network} are known to be highly non-linear models \citep{li2018visualizing}, having their performance depend to a large extent on the choice of hyperparameters, random seeds and other (stochastic) factors \citep{bouthillier2021accounting}. 
This makes comparisons between algorithms more difficult, as illustrated by the motivating example below by \citet{dror2019deep}:

\definecolor{bettergray}{HTML}{F5F5F5}
\begin{tcolorbox}[
    title={}, 
    colback=bettergray,
    colframe=gray, 
    left=4pt, right=8pt, enlarge top by=5pt,
    boxsep=5pt,
    coltext=black
]
   \footnotesize
   \begin{minipage}{0.55\textwidth}
    \begin{example}[Part-of-Speech tagging]
        Consider the results for Part-of-Seech-tagging given in the table on the right, taken over 3898 and 1822 observations using different hyperparameter configurations and random seeds, respectively. 
        Optimizing with Adam \citep{kingma2015adam} gives a higher average word-level accuracy than using RMSprop \citep{tieleman2012lecture}, however the median favors the latter. 
        Furthermore, the minimum across a few runs favor Adam, but the maximum is higher for RMSprop. 
        So, which algorithm do we consider to be \emph{better}?   
    \end{example}
\end{minipage}
\begin{wraptable}[10]{r}{2cm}
    \vspace{-6.5cm}
    \hspace{-2.8cm}
    \setlength{\tabcolsep}{5pt}
    \renewcommand{\arraystretch}{1.2}
    \footnotesize
    \begin{tabular}{@{}lrr@{}}
        \toprule
        & \footnotesize Adam & RMSprop  \\
        \midrule
        \footnotesize Mean & .9224 & .9190 \\
        \footnotesize Std.\@ dev. & .0604 & .0920 \\
        \footnotesize Median & .9319 & .9349 \\
        \footnotesize  Min. & .1746 & .1420 \\
        \footnotesize  Max. & .9556 & .9573 \\
        \bottomrule
    \end{tabular}
\end{wraptable}
\end{tcolorbox}

Therefore, \citet{dror2019deep} propose \emph{almost stochastic order}\index{Stochastic order!Almost} (ASO) for Deep Learning models based on the work by \citet{del2018optimal}.\footnote{
    Implementation details and pseudo-code are given in \cref{app:implementation-details}.
}
It is based on a relaxation of the concept of \emph{stochastic order} by \citet{lehmann1955ordered}: 
A random variable $x_{\mathbb{A}}$ is defined to be \emph{stochastically larger} than $x_{\mathbb{B}}$ (denoted $x_{\mathbb{A}} \succeq x_{\mathbb{B}}$) if $\forall x: F(x) \le G(x)$, where $F$ and $G$ denote the cumulative distribution functions\index{Cumulative distribution function} (CDF) of the two random variables. 
The CDF is defined as $F(t) = p(x \le t)$, while the \emph{empirical} CDF\index{Cumulative distribution function!Empirical} given a sample $\{x_1, \ldots, x_n\}$ is defined as 

\begin{equation*}
    F_n(t) = \frac{1}{n}\sum_{i=1}^n \indicator{x_i \le t},
\end{equation*}

\begin{figure}[bt]
    \centering
    \begin{subfigure}[t]{0.49\textwidth}
        \includegraphics[width=0.99\textwidth]{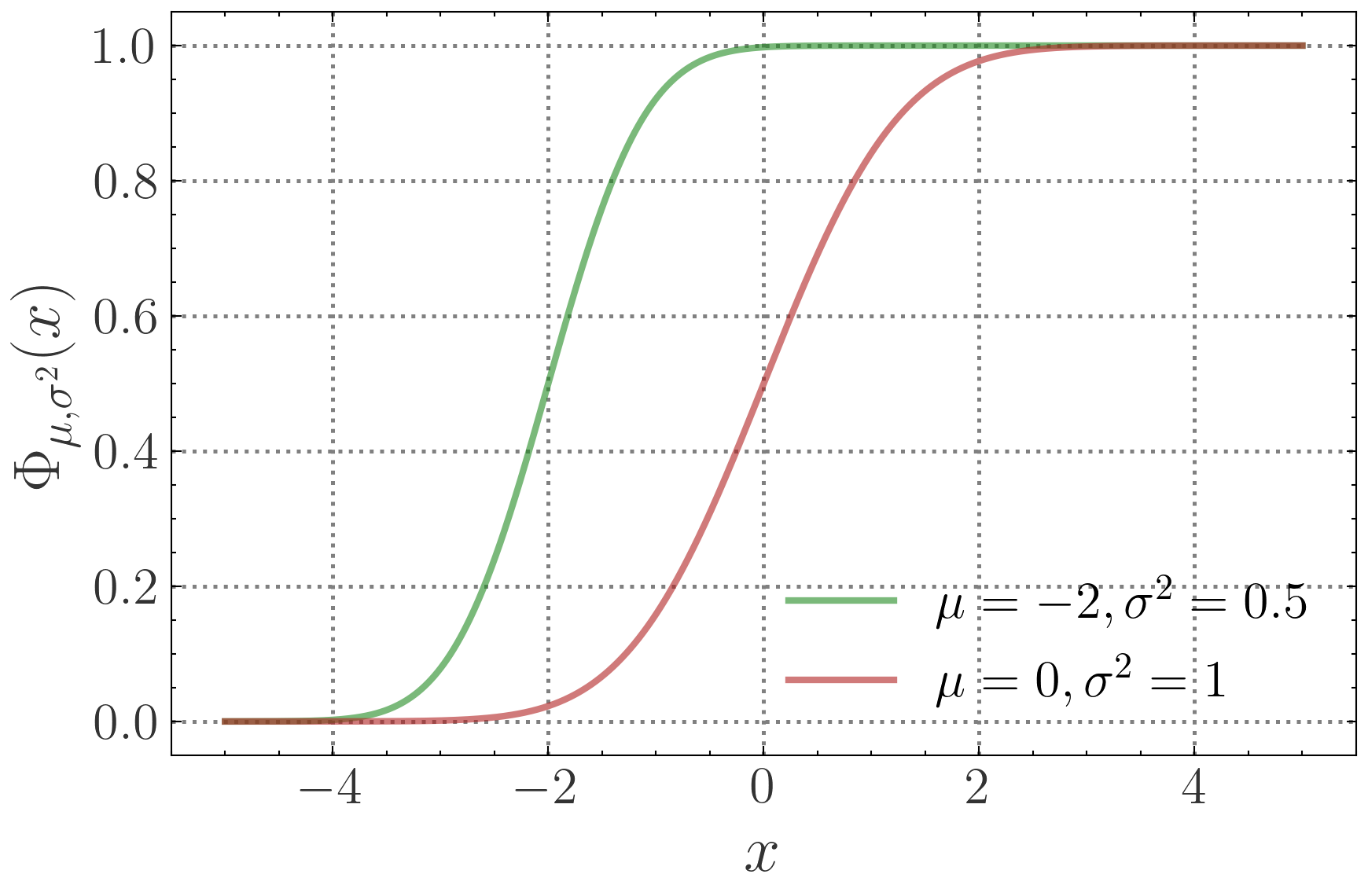}
        \caption{Stochastic order with red $\succeq$ green.}
        \label{subfig:so}
    \end{subfigure}
    %\hspace{0.2cm}
    \begin{subfigure}[t]{0.49\textwidth}
        \includegraphics[width=0.99\textwidth]{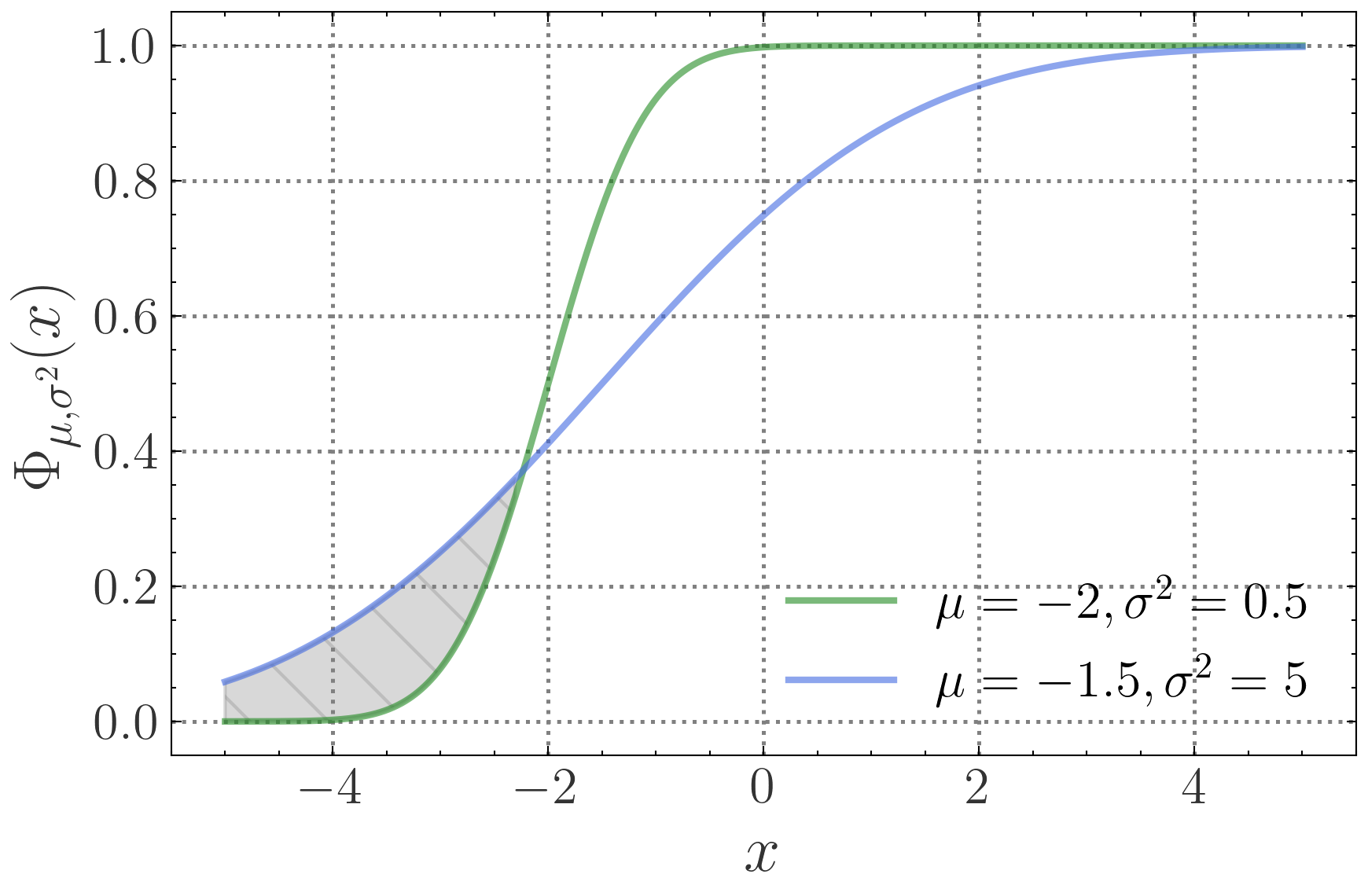}
        \caption{Almost stochastic order, with blue $\succsim$ green.}
        \label{subfig:aso}
    \end{subfigure}
    \caption[Examples of (almost) stochastic order between two CDFs.]{Examples for stochastic order (a) and almost stochastic order (b), illustrated using the CDFs of two normal random variables. Because stochastic order is too strict to be practical, almost stochastic order allows for some degree of violation of the order (gray area in (b)).}\label{fig:aso-so}
\end{figure}

\noindent with $\indicator{\cdot}$ being the indicator function. 
In practice, since we do not know the real score distributions $p(x_{\mathbb{A}})$ and $p(x_{\mathbb{B}})$, we cannot use the precise CDFs in subsequent calculations, and we rely on the empirical CDFs $F_N$ and $G_M$. 
A case of stochastic order\index{Stochastic order} is illustrated in \cref{subfig:so}, using the CDFs of two normal distributions\index{Normal distribution}. 
However, in cases such \cref{subfig:aso} we would still like to declare one of the algorithms superior, even though the stochastic order of the underlying CDFs is partially violated. 
Several ways to quantify the violation of stochastic dominance exist \citep{alvarez2017models,del2018some}, but here we elaborate on the optimal transport\index{Optimal transport} approach by \citet{del2018optimal}. 
They propose a the following expression quantifying the distance of each random variables from being stochastically larger than the other:

\begin{equation}\label{eq:epsilon-dist}
    \varepsilon_{W_2}(F, G) = \frac{\int_{\mathbb{V}_{x}} (F^{-1}(t) - G^{-1}(t))^2 \dd t}{(W_2(F, G))^2},
\end{equation}

\noindent with the \emph{violation ratio} $\varepsilon_{W_2}(F, G) \in [0, 1]$ and a \emph{violation set} $\mathbb{V}_{x} = \big\{t \in (0, 1): F^{-1}(t) < G^{-1}(t) \big\}$, i.e.\@ where the stochastic order is being violated. 
\cref{eq:epsilon-dist} contains the following components: Firstly, the quantile functions\index{Cumulative distribution function!Inverse} $F^{-1}(t)$ and $G^{-1}(t)$ associated with the corresponding CDFs:

\begin{equation*}
    F^{-1}(t) = \inf \big\{x: t \le F(x)\big\}, \quad t \in (0, 1).
\end{equation*}

The quantile functions allow us to define stochastic order via $X \succeq Y \iff \forall t \in (0, 1): F^{-1}(t) \ge G^{-1}(t)$. 
Secondly, it comprises the univariate $l_2$-Wasserstein distance:

\begin{equation}\label{eq:wasserstein}
    W_2(F, G) = \sqrt{\int_0^1 \big(F^{-1}(t) -  G^{-1}(t)\big)^2 \dd t},
\end{equation}

\noindent which for univariate functions can be expressed through their inverse CDFs\index{Cumulative distribution function!Inverse} \citep{de20211}.
Finally, \citet{del2018optimal, dror2019deep} define a hypothesis test based on this quantity by formulating the following hypotheses:

\begin{equation*}\begin{aligned}
    \text{H}_0:\ & \varepsilon_{W_2}(F, G) \ge \tau\\
    \text{H}_1:\ & \varepsilon_{W_2}(F, G) < \tau, \\
\end{aligned}\end{equation*}

\noindent for a pre-defined threshold $\tau > 0$, for instance $0.5$ or lower (see discussion in \cref{app:aso-error-rate} about the choice of threshold). 
Further, \citet{alvarez2017models, dror2019deep} produce a frequentist upper bound to this quantity, defining the minimal $\varepsilon_{W_2}$ for which we can reject the null hypothesis with a confidence of $1 - \alpha$ as 

\begin{equation}\label{eq:epsmin}
    \epsmin(F_N, G_M, \alpha) = \varepsilon_{W_2}(F_N, G_M) - \sqrt{\frac{N + M}{NM}}\hat{\sigma}_{N, M}\Phi^{-1}(\alpha).
\end{equation}

The variance term $\hat{\sigma}_{N, M}$ is estimated using a bootstrapping\index{Bootstrap} estimator (as introduced in \cref{sec:frequentist-perspective}) for the variance, with $F_N^*$ and $G_M^*$ denoting empirical CDFs\index{Cumulative distribution function!Empirical} based on sets of scores resampled from original sets of model scores, similar to re-sampling procedure in other tests like the bootstrap \citep{efron1994introduction} or permutation-randomization test\index{Permutation-randomization test} \citep{noreen1989computer}:

\begin{equation}\label{eq:aso-var}
    \hat{\sigma}_{N, M}^2 = \text{Var}\bigg[\sqrt{\frac{NM}{N + M}}\big(\varepsilon_{W_2}(F_N^*, G_M^*) - \varepsilon_{W_2}(F_N, G_M)\big) \bigg].
\end{equation}

Thus, if $\epsmin(F_N, G_M, \alpha) < \tau$, we can reject the null hypothesis\index{Null hypothesis} and claim that algorithm $\mathbb{A}$ is better than $\mathbb{B}$, with a growing discrepancy in performance the smaller the value becomes. 

\subsection{Experimental Comparison}\label{sec:experimental-comparison-aso}

\begin{figure}
    \centering 
    \includegraphics[width=0.6\textwidth]{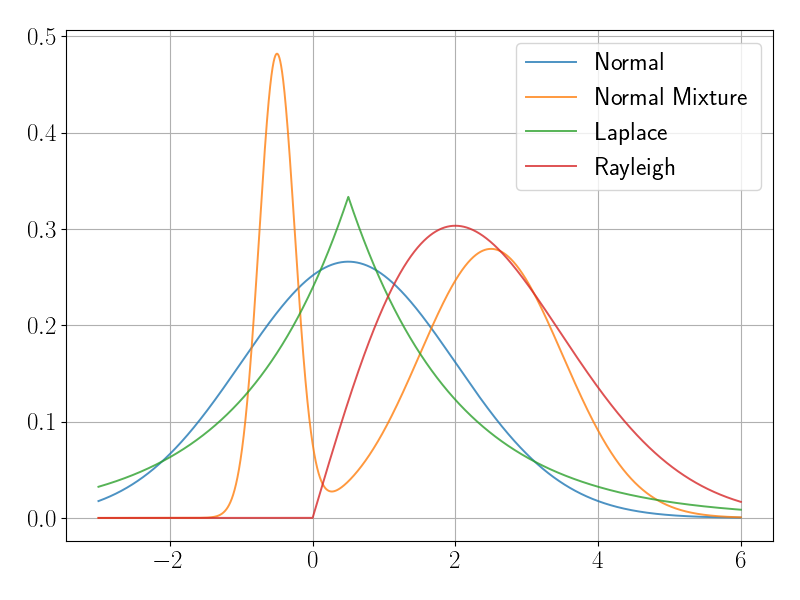}
    \caption[Plot of distributions used to benchmark significance tests.]{Plot of distributions used to empirically test the Type I and Type II error of significance tests in \cref{sec:experimental-comparison-aso}.}\label{fig:distributions}
\end{figure}

\begin{figure}[h]
    \centering
    \begin{subfigure}[t]{0.485\textwidth}
        \includegraphics[width=\textwidth]{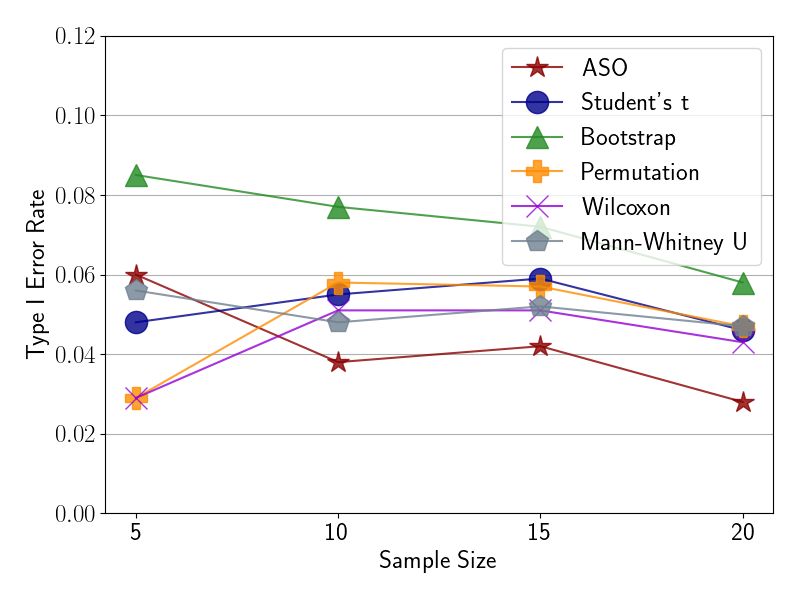}
        \caption{Rates for normal samples.}
        \label{subfig:type1_size_normal}
    \end{subfigure}%
    \hfill
    \begin{subfigure}[t]{0.485\textwidth}
        \includegraphics[width=\textwidth]{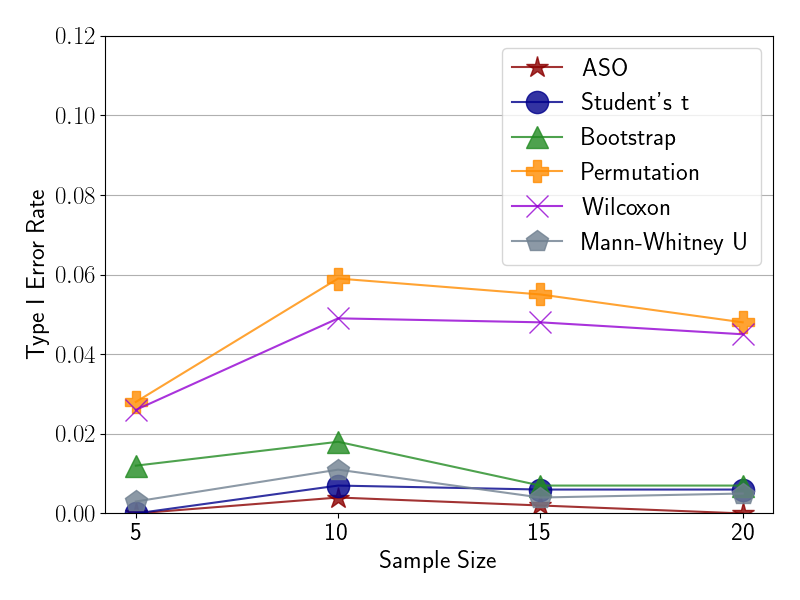}
        \caption{Rates for normal mixture samples.}
        \label{subfig:type1_size_normal_mixture}
    \end{subfigure}%

    \begin{subfigure}[t]{0.485\textwidth}
        \includegraphics[width=\textwidth]{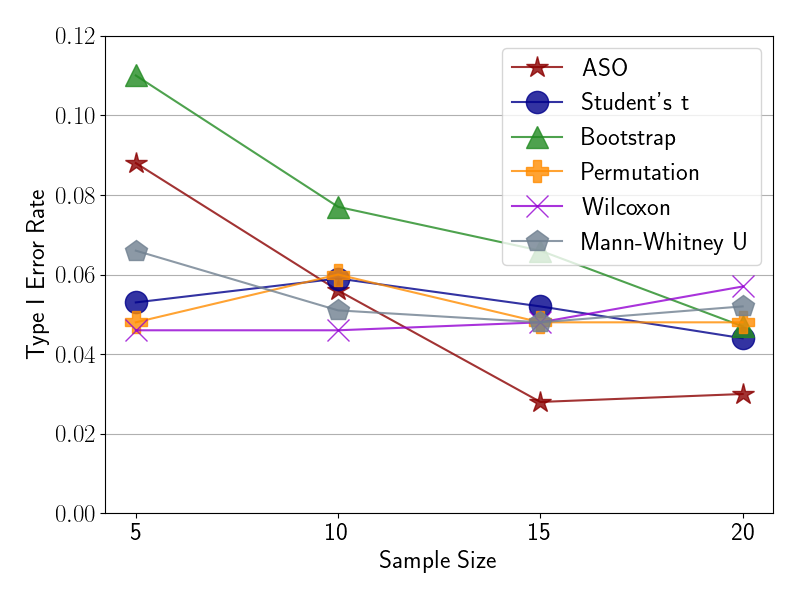}
        \caption{Rates for Laplace samples.}
        \label{subfig:type1_size_laplace}
    \end{subfigure}%
    \hfill
    \begin{subfigure}[t]{0.487\textwidth}
        \includegraphics[width=\textwidth]{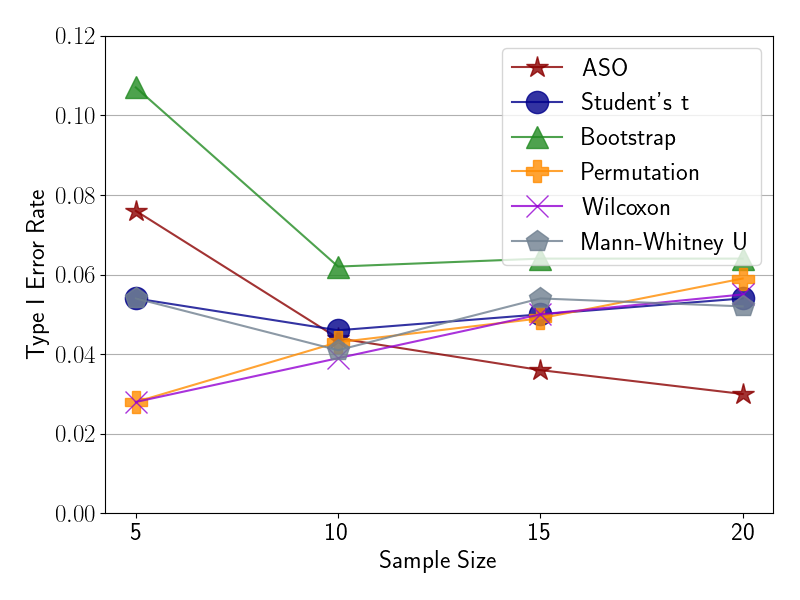}
        \caption{Rates for Rayleigh samples.}
        \label{subfig:type1_size_rayleigh}
    \end{subfigure}
    \caption[Comparisons of type I error rates.]{Comparing type I error rates for different tests and distributions as a function of sample size. Decisions are made using a confidence threshold of $\alpha = 0.05$ and $\tau = 0.2$ for $\varepsilon_\text{min}$.}\label{fig:aso-tests}
\end{figure}

We compare ASO\index{Stochastic order!Almost} to established significance tests such as the Student's-$t$\index{Student's-$t$ test}, the bootstrap \index{Bootstrap}\citep{efron1994introduction}, and the permutation-randomization test \index{Permutation-randomization test}\citep{noreen1989computer}, along with the Wilcoxon signed-rank\index{Wilcoxon signed-rank test} \citep{wilcoxon1992individual} and Mann-Whitney U test\index{Mann-Whitney U test} \citep{mann1947test} on different types of distributions, which are plotted in \cref{fig:distributions}. 
We plot the Type I error\index{Type I error} rate per $500$ simulations for ASO and $1000$ simulations for the other tests as a function of sample size in \cref{fig:aso-tests}, where we sample both sets of observation from the same distribution. 
For \cref{subfig:type1_size_normal}, we sample from $\mathcal{N}(0, 1.5^2)$ and try a bimodal normal mixture in \cref{subfig:type1_size_normal_mixture} (using the same parameter for the second component, and $\mathcal{N}(-0.5, 0.25^2)$ with mixture weights $\pi_1=0.75$ and $\pi_2=0.25$).
To test the behavior of tests on non-normal distributions, we also sample from a $\text{Laplace}(0, 1.5^2)$ distribution\index{Laplace distribution} in \cref{subfig:type1_size_laplace}, which possesses a different behavior around the main, as well as the Rayleigh distribution\index{Rayleigh distribution} with $\text{Rayleigh}(1)$ in \cref{subfig:type1_size_rayleigh}, which has a heavy tail.
We can see that ASO performs either en par or better than other tests in all scenarios, achieving \emph{lower} error rates the more samples are available, while other tests score around the expected type I error of $5 \%$. 
In \cref{app:aso-error-rate}, Type II error\index{Type II error} experiments reveal that the test produces comparatively higher error rates for ASO, though. 
This can be explained by the fact that we use the upper bound $\epsmin$ instead of $\varepsilon_{W_2}$ to evaluate the null hypothesis, which makes the test act more conservatively.
 We also find in \cref{app:aso-error-rate} that a decision threshold of $\tau = 0.2$ strikes an acceptable balance between Type I and II error\index{Type I error}\index{Type II error} rates across different scenarios. 
 Overall, we argue that a lower Type I error is more advantageous in the context of empirical research, and that a \emph{decreasing} error rate w.r.t.\@ higher sample sizes constitutes an appealing property when used on arbitrary distributions.
In these experiments, the score distributions were determined \emph{a priori} in order to create rigid experimental conditions. 
Naturally, a practitioner would not know these distribution in a typical setting, which is why we illustrate the usage of the test in the next section.

\subsection{Case study: Question-Answering with Large Language Models}\label{sec:case-study}

\begin{figure}[htb]
    \begin{subfigure}[t]{0.99\textwidth}
        \centering
        \includegraphics[width=0.99\linewidth]{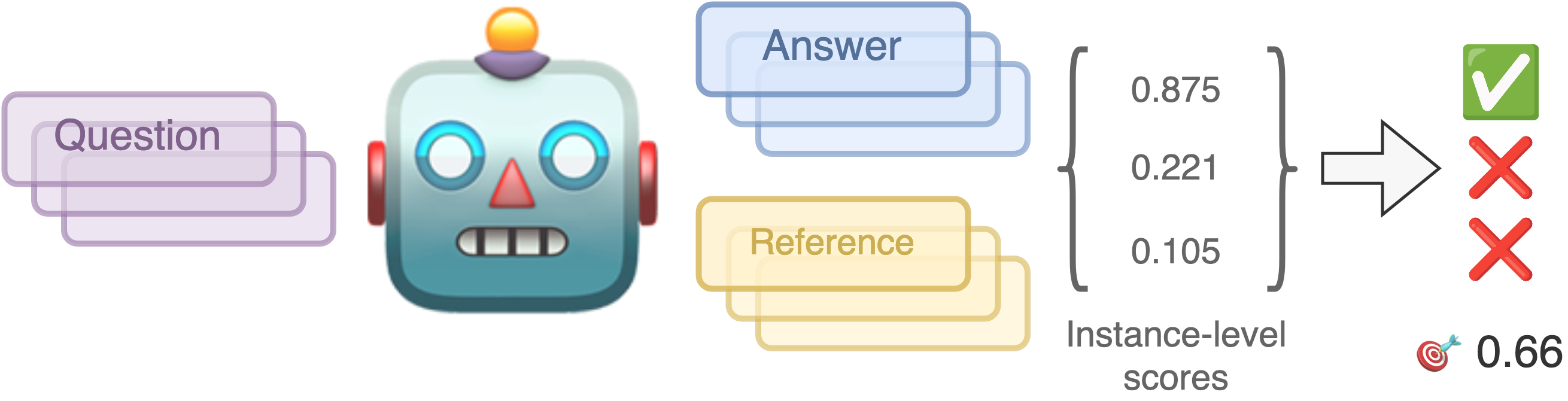}
        \caption{Setup for question-answering task.}\label{subfig:case-study-setup}
    \end{subfigure}%
    \vspace{0.5cm}
    \begin{subfigure}[t]{0.99\textwidth}
        \centering
        \includegraphics[width=0.80\linewidth]{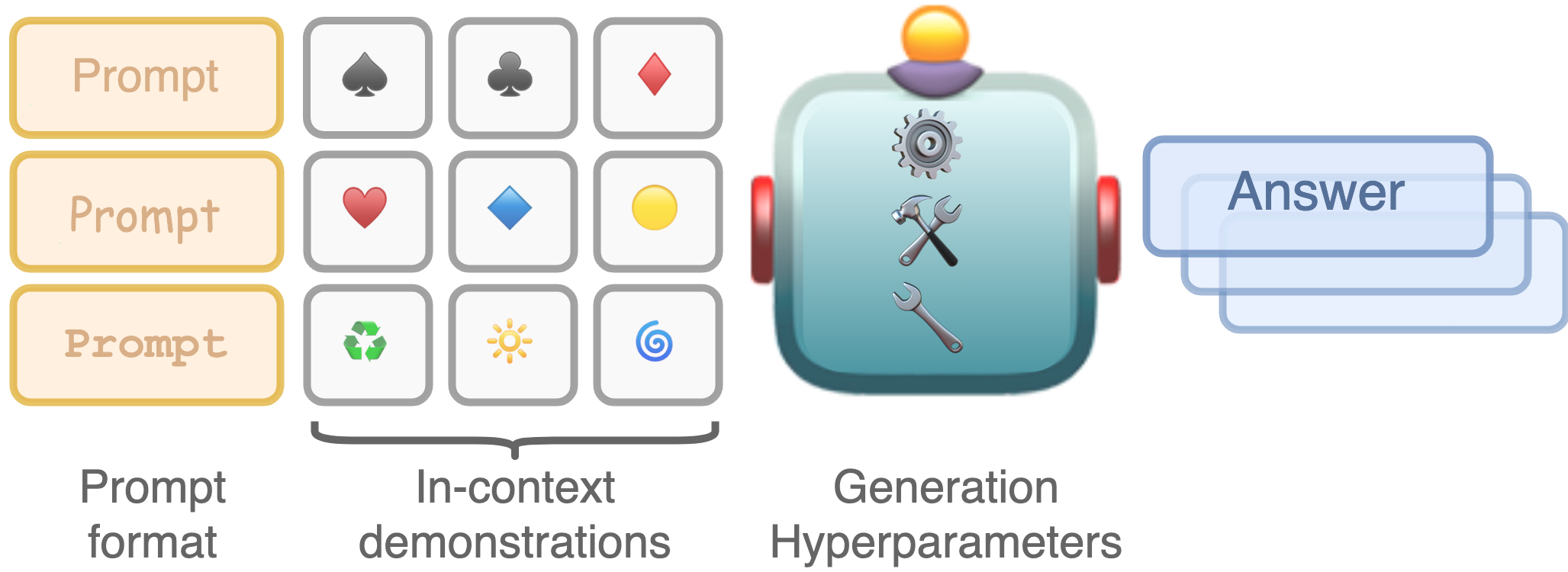}
        \caption{Strategies to produce varying answers.}\label{subfig:answer-strategies}
    \end{subfigure}
    \caption[Setup for the QA case study.]{
        Setup for the question-answering case study.
        In (a), we depict the general task setup: Questions are given to an LLM, which produces answers that are scored against reference answers using ROUGE-L.
        The scores for every question-answer pair are compared against a pre-defined threshold, which determines whether an answer is considered correct, and an accuracy score can be computed.
        (b) In order to produce different answers for the same question, we can vary different factors, including the prompt format, the in-context demonstrations, generation hyperparameters, or all of these factors together.
    }\label{fig:case-study-setup}
\end{figure}

We apply the ASO test\index{Stochastic order!Almost} to a very relevant problem in NLP\index{Natural language processing}:
Comparing the results from different LLMs\index{Large language model}, where models are already trained and multiple different seeds are not available.
Here, we explore ways in which we can still enable statistical hypothesis\index{Hypothesis testing} testing despite the more restrictive setup.
While we do not quantify any uncertainty in model \emph{predictions} here, introducing variability and employing hypothesis testing enables us to quantify uncertainty in model \emph{results}, therefore aiding model selection.

\paragraph{Setup.} 
We use three popular open-source models, namely MosaicAI MPT 7B \citep{MosaicML2023Introducing}, Mistral 7B \citep{jiang2024mistral}, and OLMo 7B \citep{groeneveld2024olmo},\footnote{
    More precisely, we use \texttt{\justify mosaicml/mpt-7b}, \texttt{\justify mistral-community/Mistral-7B-v0.2}, and \texttt{\justify allenai/OLMo-1.7-7B-hf}.
} and compare them on a closed-book question-answering\index{Question-answering} task on TriviaQA \citep{joshi2017triviaqa}.
The general task setup is shown in \cref{subfig:case-study-setup}:
Given a number of questions from the TriviaQA test set, we obtain the LLM's answers, which are scored against reference answers using ROUGE-L \citep{lin2004rouge}, which is a measure based on $n$-gram overlap.
If the obtained score surpasses a pre-defined threshold, we score an answer as correct.
From this, we obtain a single accuracy score for the whole test set.
In each case, we use their default generation methods set for the model on the HuggingfaceHub and $10$ other instances as in-context examples.\\

The goal is to show that even when we operate with monolithic models, we can still facilitate meaningful comparisons using statistical hypothesis testing\index{Hypothesis testing}.
The default option usually consist of just comparing the two accuracies (scalar comparison), however this does not take any uncertainty in the results into account.
Instead, we might compare the population of instance-level scores in \cref{subfig:case-study-setup} before thresholding (instance-level comparison), or use a bootstrap estimator on the instance-level scores to obtain multiple accuracy scores, similar to our estimation of the probability of heads using a sample of bootstrapped coin flips in \cref{sec:frequentist-perspective} (bootstrapping comparison).
Another approach is to vary the factors that produce an LLM's answer, which are depicted in \cref{subfig:answer-strategies}:
We can for instance change the prompt formatting (multi-prompt comparison), change the in-context demonstrations by re-sampling them from the training set for each inference (varying in-context samples), or modify the hyperparameters that influence the models generation (generation hyperparameters).
Lastly, we can also combine prompt formatting, varying in-context examples and generation hyperparameters by changing them jointly for every test instance (mix).
We briefly discuss each of these options in more detail.

\paragraph{Scalar Comparison.}
We first consider the potentially most common form of comparison, namely single scalars.
For this purpose, we compute the accuracy per model on the given test set of questions. 
To judge whether a question has been answered correctly, we use the same heuristic as employed by \citet{kuhn2023semantic}, where we compute the ROUGE-L\index{ROUGE} score \citep{lin2004rouge} as implemented by the \texttt{evaluate} package\footnote{See \url{https://huggingface.co/docs/evaluate/index}.} between a given model answer and gold answer.
When the resulting score surpasses a value of $0.3$, an answer is considered correct.

\paragraph{Instance-level Comparison.}
Instead of aggregating the measurements on all test instances into a single score, we can instead look at them as a set of observations.
This enables us to compare larger populations of observations, as opposed to having only one single observation per model.
For question-answering, we use the ROUGE-L scores\index{ROUGE}, but without applying a threshold.
A key difference to the other tested approaches is that this comparison answers a subtly different question about the models:
Instead of considering which hypothesis class of model is better by evaluating different model instances after training them with distinct random seeds, we instead ask which trained model \emph{instance} tends to give better-scored answers in general (as judged by the ROUGE-L heuristic).

\paragraph{Bootstrapping Comparison.}
In \cref{sec:frequentist-perspective}, we discussed bootstrapping as a way to quantify the uncertainty about a quantity of interest.
We can apply the same technique to the accuracy by bootstrapping\index{Bootstrap} samples of observations from the existing set of answered questions, and computing the accuracy on these pseudo-samples.
These scores can then be used to compute the standard error\index{Standard error} and to run them through the ASO test\index{Stochastic order!Almost}.

\paragraph{Multi-prompt Comparison.}
LLMs can be very sensitive to the chosen prompt format \citep{mizrahi2023state, sclar2023quantifying}. 
Therefore, instead of evaluating predictions from models trained with different random seeds, we can instead consider predictions from the same model but \emph{using different prompts}, and treat the resulting accuracies as observations.
Specifically, we test the following prompt templates:\index{Prompting}
\begin{enumerate}
    \item \texttt{Q: \{question\} A:}
    \item \texttt{Question: \{question\} Answer:}
    \item \texttt{Take the following question: '\{question\}'. Give the correct answer:}
\end{enumerate}

\paragraph{Varying In-context Examples.}
Various studies have pointed out the importance of in-context samples\index{In-context learning} for task-specific model capabilities \citep{xie2021explanation, min2022rethinking, hendel2023context}.
For this reason, we run four additional evaluations where we randomly sample a new set of in-context samples.

\paragraph{Generation Hyperparameters.}
We also consider generating answers using different generation parameters.
Specifically, we try the default approach for the three models, greedy decoding, as well as Nucleus sampling \citep{holtzman2020curious} with $p=0.9$, top-$k$ sampling \citep{fan2018hierarchical, holtzmann2018learning, radford2019language} with $k=60$ or beam search with three beams.
Lastly, we also combine this with multiple different prompts\index{Prompting} and different in-context samples\index{In-context learning}, where we answer every question $5$ times, each time sampling a different prompt, generation configuration, and in-context demonstrations randomly.

\begin{figure}
    \centering
    \begin{subfigure}[t]{0.99\textwidth}
        \centering
        \resizebox{0.99\textwidth}{!}{%
        \renewcommand{\arraystretch}{1.4}
        \addtolength{\tabcolsep}{3pt} 
        \begin{tabular}{@{}rllllll@{}}
            \toprule
            & \multicolumn{6}{c}{Accuracy} \\
            \cline{2-7}
            Model & Scalar & Bootstrapping & Multi-prompt & Generation & In-context & Mix\\
            \midrule
            MosaicAI MPT 7B & $.49$ & $.49\pm.00$ & $.51\pm.01$ & $.02\pm .02$ & $.51\pm.01$ & $.30\pm.02$ \\
            Mistral 7B v0.2  & $.37$ & $.37\pm.00$ & $.40\pm.04$ & $.15\pm .09$ & $.43\pm.04$ & $.33\pm.03$ \\
            OLMo 7B v1.7 & $\mathbf{.51}$ & $\mathbf{\underline{.51\pm.01}}$ & $\mathbf{\underline{.57\pm.04}}$ & $\mathbf{.17\pm .02}$ & $\mathbf{\underline{.59\pm.03}}$ & $\mathbf{\underline{.40\pm.03}}$ \\
            \bottomrule
        \end{tabular}%
        }
        \caption{
            Evaluation results, given in accuracy. Best results are bolded, significant differences according to the ASO test are underlined. 
            Shown are results from a scalar comparison (Scalar), bootstrapping instance-level observations (Bootstrapping), trying different prompt templates (Multi-prompt), generation hyperparameters (Generation), in-context demonstrations (In-context), or randomly sampling a prompt, generation settings and in-context examples for each input (Mix).
            Instance-level results were only used to perform hypothesis testing for the scalar results, and are therefore not included as a column.
        }\label{tab:case-study-aso}
    \end{subfigure}%
    \vspace{0.5cm}
    \begin{subfigure}[t]{0.99\textwidth}
        \centering
        \includegraphics[width=0.99\linewidth]{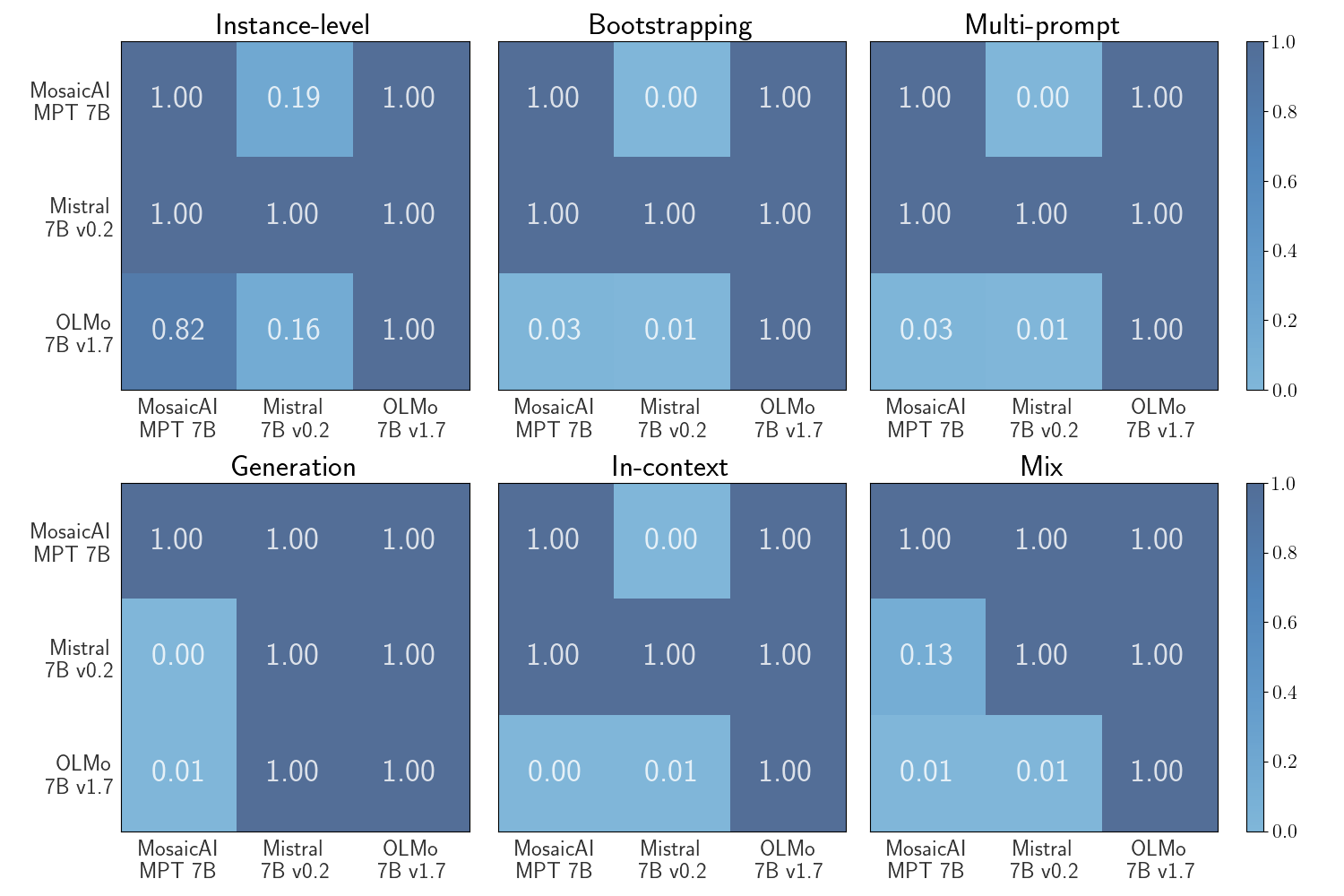}
        \caption{$\epsmin$ values comparing the LLMs using different sets of observations.}\label{subfig:case-study-eps-min}
    \end{subfigure}
    \caption[Results of LLM case study using the ASO test.]{
        Results for the case study. Given are (a) accuracy scores, either as single scalar or accuracy scores with confidence intervals as a result of bootstrapping or using multiple-prompts.
        Further shown are (b) the $\epsmin$ scores based on the instance-level observations, bootstrapping observations as well using multiple prompts, different generation parameters, different in-context demonstrations or a combination of the last tree (Mix).
    }\label{fig:case-study-results}
\end{figure}

\paragraph{Results.}
All accuracies for the different methods including standard deviations are shown in \cref{tab:case-study-aso}, with an overview of all the $\epsmin$ values calculated by the ASO test\index{Stochastic order!Almost} in \cref{subfig:case-study-eps-min}.
Recall that according to \cref{sec:aso}, we would declare one model superior to another when $\epsmin < \tau$, which empirically $\tau = 0.2$ to provide a good trade-off between Type I and Type II error.\index{Type I error}\index{Type II error}
We can see that the ordering of models largely agrees across settings, but can provide subtle differences.
All models usually generate through greedy sampling.
When using different generation hyperparameters, we can observe a noticeable degradation in results, although the OLMo model seems to be most robust to changes in generation parameters.
Interestingly, the $\epsmin$ values in \cref{subfig:case-study-eps-min} show that all evaluations mostly agree in their result;
however the comparison of instance-level scores seems to underestimate the degree of almost stochastic dominance (as shown through larger $\epsmin$ values for the best models).
A noticeable exception for this agreement is the experiment using different generation hyperparameters, where the severe loss in performance renders none of the results significant.
In the end, the mixture of a random prompt template, generation hyperparameters and in-context examples seems to portray the clearest picture of the model rankings through the $\epsmin$ values.

\paragraph{Formalization.}
In this case study, we discussed a number of ways we can use to perform statistical hypothesis testing\index{Hypothesis testing} using LLMs\index{Large language model}, assuming access to an already trained model.
All of these are subtly different in what the kinds of uncertainties they take into account to compare models.
To investigate the differences, we formalize the problem:
Let $\bx$ be shorthand for an input sequence, $\by$ for a generated sequence and $\phi$ a function mapping a generated sequence to an evaluation score (e.g.\@ an indicator function deciding whether an answer is correct).
Further, let $\brho$ be a prompt template, $\bgamma$ a set of generation parameters, $\mathcal{C}$ a set of in-context demonstrations and $\blambda$ a set of training hyperparameters (including architecture, optimizer, regularization, finetuning strategy etc.).
We are then interested in two quantities: 
Aggregate metrics such as accuracy, which we can formulate as the expected value of $\phi$ under the model on a given dataset, and the expected accuracy arising when varying all the other factors mentioned above, forming another expectation:

\begin{footnotesize}
\begin{align}\label{eq:bayesian-view-eval}
    \mathbb{E}_{p(\bgamma, \brho, \mathcal{C})}\Big[\mathbb{E}_{p(\by \mid \bx, \btheta, \brho, \bgamma)}\big[\phi(\by)\big]\Big] = \int\hspace{-0.15cm}\ldots\hspace{-0.15cm}\int & \underbrace{\phi(\by)}_{\text{Score}} \underbrace{p(\by \mid \btheta, \bx, \brho, \bgamma, \mathcal{C})}_{\text{LLM Predictive Dist.}}\underbrace{p(\brho)p(\bgamma)p(\mathcal{C})}_{\text{Generation Priors}}\nonumber \\[0.2cm]
    & p(\btheta \mid \mathbb{D}, \blambda)p(\blambda) \dd\!\by\dd\!\btheta\dd\!\bgamma\dd\!\brho\dd\!\hspace{0.1cm}\mathcal{C}\dd\!\blambda.
\end{align}
\end{footnotesize}

We can use this to analyze all the test setups above by applying Dirac delta functions (as previously used in \cref{eq:predictive-posterior-frequentist}) and Monte Carlo integration\index{Monte Carlo estimation} (see \cref{eq:neural-posterior-predictive-mc-estimate}) to evaluate \cref{eq:bayesian-view-eval}.
For instance, the scalar comparison assume an single prompt $\bm{\hat{\rho}}$, set of generation parameters $\bm{\hat{\gamma}}$, in-context samples $\hat{\mathcal{C}}$ \index{In-context learning} and weights $\bm{\hat{\theta}}$ and thus \cref{eq:bayesian-view-eval} becomes

\begin{align}
     &  \mathbb{E}_{p(\bgamma, \brho, \mathcal{C})}\Big[\mathbb{E}_{p(\by \mid \bx, \btheta, \brho, \bgamma)}\big[\phi(\by)\big]\Big] \nonumber \\
    \approx & \int \phi(\by)p(\by \mid \btheta, \bx, \brho, \bgamma, \mathcal{C})\delta(\brho - \bm{\hat{\rho}})\delta(\bm{\bgamma- \hat{\gamma}})\delta(\btheta - \bm{\hat{\theta}})\delta(\mathcal{C} - \hat{\mathcal{C}})\dd\!\by \nonumber \\
         = & \int \phi(\by)p(\by \mid \bm{\hat{\theta}}, \bx, \bm{\hat{\rho}}, \bm{\hat{\gamma}}, \hat{\mathcal{C}}) \dd\!\by\hspace{-0.1cm}.
\end{align}

The same assumptions are also applied for the instance-level comparison, with the difference that we only evaluate the outer expectation in \cref{eq:bayesian-view-eval}.
Further, we can interpret the bootstrapping procedure as a different outer expectation in \cref{eq:bayesian-view-eval}, where we instead evaluate the expectation over all possible samples (with replacement) of our original set of generated sequences.
The conclusion we can draw from this is the following:
To evaluate the overall performance of a model, we would like to approximate \cref{eq:bayesian-view-eval} as closely as possible, ideally by performing a full ancestral sampling scheme.
For LLMs\index{Large language model}, this is not feasible, since we often have to work with a single, already trained model.
\citet{bouthillier2021accounting} have unveiled the perhaps counter-intuitive intuition that \emph{increasing} amount of randomness in our experiments actually helps to \emph{decrease} the variance of our estimate of \cref{eq:bayesian-view-eval}.
We follow this idea and vary as many aspects as possible, which in this case study produces a clear ranking of the robustness of a model.
For language models, this implies running the model over the dataset multiple times, but sampling different generation parameters and prompt templates like in our mix variant (as advocated for by \citealp{mizrahi2023state}).
In cases where running the model multiple times for each input might still be prohibitively expensive, we can always fall back onto a bootstrap estimator.

\subsection{Discussion}\label{sec:discussion}

The previous sections have demonstrated the advantages of the ASO test in an neural network setting. 
Nevertheless, using these techniques in practice comes with limitations as well, which the end user should be aware of.
The first line of limits comes with ASO itself. 
Multiple steps of the procedure require different kinds of approximations or properties that are only guaranteed to hold in the infinite-sample limit, e.g.\@ the bootstrap estimator of the variance in \cref{eq:aso-var}. 
Furthermore, significance tests in general are known to sometimes provide unreliable results with small \citep{reimers2018comparing} or very large sample sizes \citep{lin2013research}, are prone to misinterpretation \citep{gibson2021role, greenland2016statistical}, and encourages binary significant / non-significant thinking \citep{wasserstein2019moving, azer2020not}. 
Bayesian analysis \citep{kruschke2013bayesian, benavoli2017time, gelman2021bayesian} is therefore an attractive alternative to statistical hypothesis testing\index{Hypothesis testing}, where the user draws conclusions from posterior distributions over quantities of interest. 
A potential drawback of this methodology is that it often comes at the cost of having to use Markov chain Monte Carlo\index{Markov chain Monte Carlo} methods, which require experience from the user to validate convergence and defining appropriate models and model priors.\\

For the application to LLMs, \cref{sec:case-study} has demonstrated that even with a fully trained model, we can still perform meaningful statistical hypothesis testing by either using bootstrapping or by varying prompt templates, generation hyperparameters and in-context demonstrations.
Some of these methods for model comparison will now be used in the remaining chapters of this thesis.
% Chapter X

\chapter{Uncertainty in Text Classification}\label{ch:uncertainty-classification} % Chapter title

\label{ch:uncertainty-classification} % For referencing the chapter elsewhere, use \autoref{ch:name} 

\begin{tikzpicture}[remember picture,overlay]
    \node[anchor=north,inner sep=0pt] at (current page text area.north) {\includegraphics[width=\linewidth, clip=true, trim = 8cm 75cm 8cm 50cm]{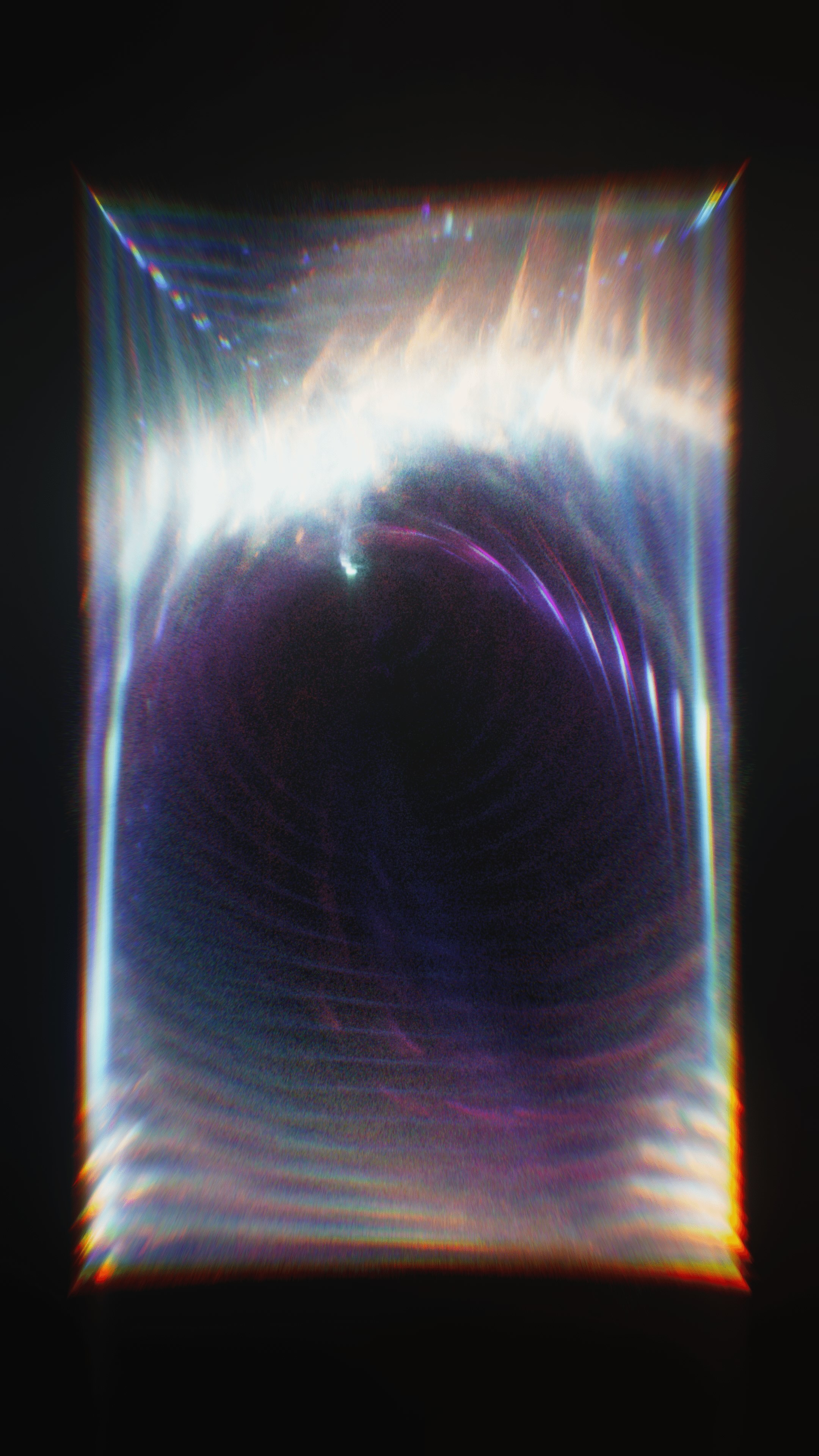}};
\end{tikzpicture}

%\epigraph{

 %   ``\emph{I go to the mountain side\\
 %   of the house to cut saplings,\\
 %   and clear a view to snow\\
 %   on the mountain. But when I look up,\\
 %   saw in hand, I see a nest clutched in\\
 %   the uppermost branches.\\
 %   I don’t cut that one.\\
 %   I don’t cut the others either.\\
 %   Suddenly, in every tree,\\   
 %   an unseen nest\\
 %   where a mountain\\   
 %   would be.}''
%}
%{---\emph{Choices} by Tess Gallagher.}

\epigraph{
    \begin{minipage}[t]{0.482\textwidth}
        \scriptsize
        ``\emph{Two roads diverged in a yellow wood,}\\
        \emph{And sorry I could not travel both}\\
        \emph{And be one traveler, long I stood}\\
        \emph{And looked down one as far as I could}\\
        \emph{To where it bent in the undergrowth;}\\[0.4cm]

        \emph{And both that morning equally lay}\\
        \emph{In leaves no step had trodden black.}\\
        \emph{Oh, I kept the first for another day!}\\
        \emph{Yet knowing how way leads on to way,}\\
        \emph{I doubted if I should ever come back.}\\

    \end{minipage} %
    \hfill
    \begin{minipage}[t]{0.482\textwidth}
        \scriptsize
        \emph{Then took the other, as just as fair,}\\
        \emph{And having perhaps the better claim,}\\
        \emph{Because it was grassy and wanted wear;}\\
        \emph{Though as for that the passing there}\\
        \emph{Had worn them really about the same,}\\[0.4cm]

        \emph{I shall be telling this with a sigh}\\
        \emph{Somewhere ages and ages hence:}\\
        \emph{Two roads diverged in a wood, and I—}\\
        \emph{I took the one less traveled by,}\\
        \emph{And that has made all the difference.}''
    \end{minipage}
}{---\emph{The Road Not Taken} by Robert Frost (1915).}

Assume we would like to automate the moderation of postings on a social media platform.
While it would be preferable to always use human moderators, this is often not feasible due to the deluge of posts, and also not desirable due to the psychological impact that the moderation of harmful content can have.
After having trained a classifier on some labeled training instances, we are ready to deploy.
And while we expect a large number of the flagged cases to be clear positives, there will inadvertently be instances for which the classifier struggles, for example sentences in which an toxic remark is quoted or lacks context.
The sentence below was taken from the Wikitalk dataset \citep{wulczyn2017ex, borkan2019nuanced}, which includes discussions among Wikipedia editors:

\begin{displayquote}
    \footnotesize
   ``\emph{I was responding to a post by AndyTheGrump at Talk: Communist terrorism, section `Marxism is not the only `communism'', where he called me `idiot' and then refused to retract his remark when I requested him to do so.}''   
\end{displayquote}

%In addition, this tasks becomes more challenging when considering phenomena as sarcasm or terms which might be considered inoffensive in a certain social group, but offensive otherwise.
The mention of ``idiot'' here might already set off the toxicity classifier, even though the sentence just quotes another user's remark.
In this case, we might want to defer to the decision to a human moderator when the classifier shows uncertainty.
To make the task of moderation easier, we could also employ another system to label the spans of text that contain harmful speech.
Here, we might only show the parts that the system is most uncertain about to limit the exposure to toxicity, and potentially ask the moderator to label them in order to improve the training data for future model updates.
To illustrate this, let us look at another (truncated) example from the dataset, labeling it in two different ways (assuming simplified tokenization):\\

\begin{centering}
    \footnotesize
    \resizebox{0.995\textwidth}{!}{
        \begin{tabular}{ccccccccccccc}
            I'd & like & to & offer & you & a & great &  big &  glass &  of &  shut-the-f$*@\#$-up &  juice  \\[0.2cm]
            \textcolor{BrickRed}{$\mathbf{-}$} & \textcolor{BrickRed}{$\mathbf{-}$} & \textcolor{BrickRed}{$\mathbf{-}$} & \textcolor{BrickRed}{$\mathbf{-}$} & \textcolor{BrickRed}{$\mathbf{-}$} & \textcolor{BrickRed}{$\mathbf{-}$} &\textcolor{BrickRed}{$\mathbf{-}$} & \textcolor{BrickRed}{$\mathbf{-}$} & \textcolor{BrickRed}{$\mathbf{-}$} & \textcolor{BrickRed}{$\mathbf{-}$} & \textcolor{ForestGreen}{$\mathbf{+}$} & \textcolor{BrickRed}{$\mathbf{-}$} \\[0.2cm]
            \textcolor{Gray}{O} & \textcolor{Gray}{O} & \textcolor{Gray}{O} & \textcolor{Gray}{O} & \textcolor{Gray}{O} & \textcolor{MidnightBlue}{B-TOX} & \textcolor{CadetBlue}{I-TOX} & \textcolor{CadetBlue}{I-TOX} & \textcolor{CadetBlue}{I-TOX} & \textcolor{CadetBlue}{I-TOX} & \textcolor{CadetBlue}{I-TOX} & \textcolor{CadetBlue}{I-TOX} \\
        \end{tabular}%
    }
\end{centering}
\vspace{0.3cm}

In the first annotation, we focus on whether single words could be considered toxic or not, while in the second annotation, we capture an entire toxic span using BIO-tags\index{BIO-tags} (which indicate the \textbf{b}eginning, \textbf{i}nside or  \textbf{o}utside of such a phrase).
What we outlined above are instances of classic NLP task formats, namely \emph{sequence classification}\index{Sequence classification} and \emph{sequence labeling}\index{Sequence labeling}, respectively.
In the former we simply assign a label to an entire sequence, whereas in the latter we label or classify parts of a sequence.
In this thesis, we will refer to both jointly as \emph{text classification}\index{Classification!Text}.
%This can include a tasks such as sentiment classification, where we want to classify a sentence or entire text (e.g.\@ a product review) using a category like ``very negative'' on a positive-negative spectrum.
%In the latter, \emph{sequence labeling}\marginpar{Sequence Labeling}, we assign every token in a sequence a tag or label from a predefined set.
Sequence labeling subsumes tasks such as part-of-speech tagging\index{Part-of-speech tagging},\footnote{PoS tagging also is common preprocessing step for parser that produce parse trees as the ones shown in \cref{sec:uncertainty-linguistics}.} where the labels can e.g.\@ be noun, verb or adverb, or \emph{named entity recognition}\index{Named entity recognition}, in which we identify named entities such as people, organization or locations.
%Another one is \emph{sequence classification}\marginpar{Sequence Classification}, where we assign one class for an entire sequence. 
In this work we will use the terms \emph{label} and \emph{class} interchangeably to refer to a category from a set of categories that is assigned to (part of) an input.\footnote{
    Even though this chapter focuses on \emph{multi-class classification}\index{Classification!Multi-class}, there is a subtle difference to \emph{multi-label classification\index{Classification!Multi-label}}:
    Multi-class means that we have multiple choices, but only one of them will be considered correct at a time, which makes sense when trying to choose from mutually exclusive options. %, for instance when classifying cats versus dogs.
    In the multi-label classes, we choose for each possible label whether it is applicable or not, allowing multiple labels to be correct at the same time. 
    For instance, when classifying legal judgments according to which human right articles are being violated, we can find that each judgement can violate multiple articles at once \citep{chalkidis-etal-2019-neural}.
}\\
%All the material in this chapter which is based on \citet{ulmer2020know,ulmer-etal-2022-exploring}, and we will always consider a multi-class setup for both sequence labeling and sequence classification.\footnote{Indeed, one can also always consider a multi-label classification tasks as a number of simultaneous binary classification tasks.}\\

However, quantifying the uncertainty in these decisions is challenging.
Uncertainty is not always present in predictions when we might expect them, and might be present if it is unwarranted.
This chapter aims to understand this behavior, both from a theoretical and empirical perspective.
Therefore, we demonstrate some shortcomings of UQ with ReLU networks\index{ReLU} in the next section, before returning to text classification in \cref{sec:benchmarking-nlp-uncertainty}\index{Classification!Text}.

%In order to introduce and illustrate some of the problems that we will encounter in practice with UQ in NLP, it helps to first discuss some theoretical arguments why some commonly used measures may already appear flawed from the start.

\section[Theoretical Pitfalls in Classification]{Theoretical Pitfalls in Classification}\label{sec:uq-classification-pitfalls}

\begin{footnotesize}
    \vspace{-2.5ex}
    \emph{The following work is based on \citet{ulmer2020know}}.\\
    \vspace{2.5ex}
\end{footnotesize}

\begin{figure}[htb]
    \centering
    \begin{subfigure}[t]{0.32\textwidth}
        \centering
        \includegraphics[width=0.985\textwidth]{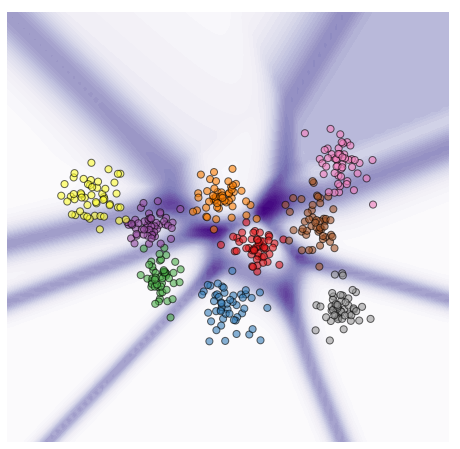}
        \caption{Predictive entropy.}\label{subfig:uncertainty}
    \end{subfigure}
    \hfill
    \begin{subfigure}[t]{0.32\textwidth}
        \centering
        \includegraphics[width=0.985\textwidth]{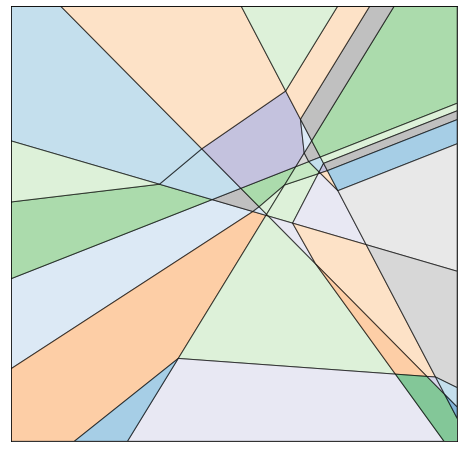}
        \caption{Polytopal regions.}\label{subfig:polytopes}
    \end{subfigure}
    \hfill
    \begin{subfigure}[t]{0.32\textwidth}
        \centering
        \includegraphics[width=0.985\textwidth]{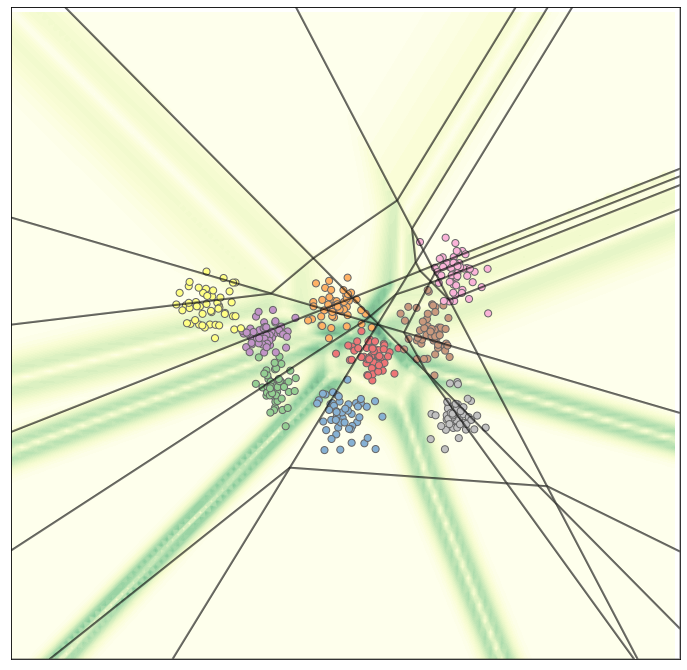}
        \caption{Magnitude of predictive entropy gradient.}\label{subfig:grad-polytopes}
    \end{subfigure}
    \caption[Uncertainty and linear regions of a ReLU classifier trained on example data.]{
        Uncertainty and linear regions of a ReLU classifier trained on example data. 
        (a) Uncertainty measured by predictive entropy on synthetic data, illustrated by increasing shades of purple, with white denoting absolute certainty. 
        (b) Polytopal, linear regions in the feature space induced by the same classifier (as introduced by \citealp{arora2018understanding}, plotted using the code by \citealp{jordan2019provable}). 
        (c) Gradient norm of the predictive entropy plotted in shades of green---small perturbations in the input have a decreasing influence on the uncertainty of the network as we stray away from the training data, creating large areas in which uncertainty levels are overgeneralized. 
        %Polytopes are drawn using the code of \citet{jordan2019provable}.
    }\label{fig:know-your-limits-figure}
\end{figure}

It is well-known that neural network classifiers\index{Neural network} tend to be overconfident in their predictions (\citealp{guo2017calibration}; see more related work in \cref{sec:frequentist-neural-networks}).
In addition, they can exhibit high levels of certainty when this is unwarranted, and often fail to correctly identify OOD samples \citep{ovadia2019can, nalisnick2019do}.\index{Confidence}\index{Out-of-distribution data} 
\citet{ulmer2020trust} showed that even techniques specifically developed to quantify the model's uncertainty struggle at detecting OOD samples for a relatively simple classification task. 
Crucially, it was shown that neural discriminators tend to project vast areas of high certainty far away from the training distribution---a behavior that seems completely at odds with reliable OOD detection\index{OOD detection}. 
These observations are replicated in \cref{fig:know-your-limits-figure}:
In \cref{subfig:uncertainty}, we can see that the predictive entropy\index{Entropy!Predictive} of a ReLU classifier displays low uncertainty in large regions behind the observed data clusters.
As \citet{arora2018understanding} showed, ReLU classifiers induce polytopal linear regions in the feature space shown in \cref{subfig:polytopes},
which was used by the previous work of \citep{hein2019relu} to show that the network's confidence is an unsuitable measure of uncertainty to detect OOD inputs.
%where one can observe open areas of constant certainty stretching beyond the training data.
However, the reasons for this behavior in a classification setting are less studied, and thus we study this behavior on additional uncertainty metrics such as predictive entropy in \cref{subfig:grad-polytopes}.\\

In this chapter, we present a theoretical argument to explain such phenomena, showing that certainty levels are generalized on sub-spaces defined by the network (see \cref{subfig:polytopes,subfig:grad-polytopes}). 
We do this by simulating covariate shift\index{Shift!Covariate} for single feature values of real variables and studying the asymptotic behavior of the model.
Our first result shows that, under mild assumptions about the network's behavior on certain subspaces, ReLU-based neural network classifiers coupled with widely used uncertainty metrics always converge to a fixed uncertainty level on OOD samples.
We extend this result by proving that variational inference-based\index{Variational inference} and ensembling\index{Ensembling} methods in combination with several uncertainty estimation techniques suffer from the same problem (\cref{theorem:know-your-limits-main-theorem}). 
This phenomenon is illustrated and discussed on synthetic data. 
These results entail that, when the conditions of the theorem are met, these models cannot be used to reliably detect OOD: 
since the level of certainty is generalized from seen to unseen data, the models are unable to differentiate between the two. 
The findings of this chapter have bearings on OOD detection\index{OOD detection} for several critical applications using neural classifiers with ReLU\index{ReLU} activation functions, and I will also discuss the impact on the following experiments for NLP.\index{Natural language processing}

\subsection{Preliminaries}\label{sec:know-your-limits-preliminaries}

We first introduce some relevant definitions for the rest of this chapter.

\paragraph{Out-of-distribution Data.}\index{Out-of-distribution data} A
Although there exist many different notions of dataset shift\index{Shift!Distributional} \citep{shimodaira2000improving, moreno2012unifying, hupkes2022state}, we particularly focus on \emph{covariate shift}\index{Shift!Covariate}, in which the distribution of feature values---the covariates---differs from the original training distribution $p(\bx)$. 
We focus on this kind of shift as it is especially common in non-stationary environments like healthcare \citep{curth2019transferring}, where distributional drifts over time are very common.
To simulate covariate shift, we obtain OOD samples by shifting points away from the training distribution by means of a scaling factor. This approach is in line with recent experiments on covariate shift and OOD detection \citep{ovadia2019can, ulmer2020trust}.
We would expect a reliable OOD detection model to display increasing uncertainty as points stray further and further away from the mass of $p(\bx)$, thus we study the behavior of OOD detection models\index{OOD detection} in the limit, when the scaling factor is allowed to grow indefinitely in at least one dimension.

\paragraph{Uncertainty Metrics.}\index{Uncertainty metric}
%In \cref{sec:uncertainty-deep-learning}, we examined the frameworks through which uncertainty is quantified, without however actually defining the specific method to do so.
We begin by first defining a neural discriminator in the form of a ReLU\index{ReLU} classifier, which we assume to follow common architectural conventions.
Thus, it consist of a series of affine transformations with ReLU \citep{glorot2011deep} activation functions, defined by $\text{ReLU}(x) = \max(0, x)$.
Together with a final softmax\index{Softmax function} function \citep{bridle1990probabilistic} as defined in \cref{eq:softmax}, it  parameterizes a categorical distribution\index{Categorical distribution} over classes.\footnote{
    The following proofs also hold for binary classifiers which are parameterized through a sigmoid function.\index{Sigmoid function} 
    For the connection between the softmax and sigmoid function, refer to \cref{app-softmax-sigmoid-connection}.
}

\begin{definition}[ReLU Classifier]\label{def:logit}
    Let $\bx \in \mathbb{R}^D$ be an input vector and $K$ the number of classes in a classification problem.
    The unnormalized output of the network after $L$ layers is a function $f_{\btheta}: \mathbb{R}^D \rightarrow \mathbb{R}^K$ with the final output following after an additional softmax function $\bar{\sigma}(\cdot)$ s.t.\@ $P_{\btheta} = \bar{\sigma}\circ f_{\btheta}$, so $P_{\btheta}(y=k \mid \bx) \equiv \bar{\sigma}(f_{\btheta}(\bx))_k$. 
    Thus, the discriminator is represented by a function $P_{\btheta}: \mathbb{R}^D \rightarrow [0, 1]^K$, which is parametrized by a vector $\btheta$.
\end{definition}

We will consider a set of popular uncertainty metrics\index{Uncertainty metric}, which we introduced in \cref{sec:frequentist-neural-networks,sec:bayesian-neural-networks} and restate them here.
Firstly, \citet{hendrycks2017baseline} introduce a simple baseline, which is the highest probability observed for any class, also referred to as confidence:

\begin{equation}\label{eq:max-softmax}
    \hat{p} = \max_{k \in [K]}P_{\btheta}(y=k \mid \bx).
\end{equation}

Ideally, the model's predictive distribution would become more uniform for challenging inputs (e.g.\@ in areas of class overlap) and thus produce a lower confidence score $\hat{p}$, which is why we measure \emph{un}certainty by $1-\hat{p}$. 
Another approach lies in measuring the Shannon entropy\index{Entropy!Shannon}\index{Entropy!Predictive} $\text{H}$ of the predictive distribution:

\begin{equation}\label{eq:predictive-entropy}
    \text{H}\big[P_{\btheta}(y \mid \bx) \big] = -\sum_{k=1}^K P_{\btheta}(y=k \mid \bx) \log P_{\btheta}(y=k \mid \bx).
\end{equation}

The entropy here is minimal when all probability mass is centered on a single class, and maximal when the predictive distribution is uniform.
The other uncertainty estimation techniques are based on the idea of Bayesian deep learning, where, the more predictions between different parameter sets disagree, the larger the uncertainty.
One straightforward way to measure this disagreement is the average variance of the predicted probability per class, as done in \citet{gal2018understanding}: \index{Class variance}

\begin{equation}\label{eq:class-variance}
    \bar{\sigma}^2 = \frac{1}{K}\sum_{K=1}^K \mathbb{E}_{p(\btheta \mid \mathbb{D})}\big[P_{\btheta}(y=k \mid \bx)^2\big] - \mathbb{E}_{p(\btheta \mid \mathbb{D})}\big[P_{\btheta}(y=k \mid \bx)\big]^2.
\end{equation}

Maximum softmax and predictive entropy\index{Entropy!Predictive} only capture the total uncertainty\index{Uncertainty!Total}, and while the class variance aims to quantify model uncertainty\index{Uncertainty!Epistemic}, it does so rather heuristically.
Thus, we also consider the mutual information\index{Mutual information} between model parameters and a data sample \citep{depeweg2018decomposition,gal2018understanding} as a more theoretically-motivated measure of epistemic uncertainty:

\begin{equation}\label{eq:mutual-information}
    \underbrace{\text{I}\big[y, {\btheta}\ \big|\ \mathbb{D}, \bx\big]}_{\text{Model uncertainty}} = \underbrace{\text{H}\Big[\mathbb{E}_{p(\btheta \mid \mathbb{D})}\big[P_{\btheta}(y \mid \bx)\big]\Big]}_{\text{Total uncertainty}} - \underbrace{\mathbb{E}_{p(\btheta \mid \mathbb{D})}\Big[\text{H}\big[P_{\btheta}(y \mid \bx)\big]\Big]}_{\text{Data uncertainty}}.
\end{equation}

The term itself can be interpreted as the gain in information about the ideal model parameters and correct label upon receiving an input. 
If we can only gain a little, that implies that parameters are already well-specified and that the epistemic uncertainty is low.
Especially when an input is OOD, we therefore expect this metric to display high uncertainty.

\subsection{Monotonicity \& Polytopes}\label{sec:general-definitions}

Before developing the main results, we introduce some concepts that will become central to the proofs in the next sections.
This includes the definition of unbounded polytopes\index{Polytope!Partially-unbounded} on which the model behaves linearly, and the monotonicity\index{Monotonicity} of multivariate functions, which lets us make statements about the output of the network when scaling its input.\\

In the univariate case, we call a function strictly increasing on an interval $\mathbb{I} = [a, b]$ with $a < b$ and $a, b \in \mathbb{R}$ if its derivative is strictly positive on the whole interval:

\begin{equation}\label{eq:monotonically-increasing}
    \forall x^\prime \in \mathbb{I}:\quad \frac{\partial}{\partial x}f(x)\big|_{x=x^\prime} > 0,
\end{equation}

\noindent where $\cdot|_{x=x^\prime}$ refers to evaluating the value of the derivative of $f$ at $x^\prime$. 
This definition can also be extended to multivariate functions by requiring strict monotonicity\index{Monotonicity!Strict} (strictly increasing or decreasing) in all dimensions:

\begin{definition}[Monotonicity in Multivariate Functions]\label{def:monotonicity-multivariate}
    We call a multivariate function $f: \mathbb{R}^D \rightarrow \mathbb{R}$ strictly monotonic on a subspace $\mathbb{P} \subseteq \mathbb{R}^D$ if it holds that the function is either strictly increasing or decreasing in every direction:

    \begin{align}\label{eq:monotonicity-multivariate}
        \forall d \in [D]:\quad & \forall \bx^\prime \in \mathbb{P}:\ \big(\nabla_{\bx} f(\bx)\big|_{\bx=\bx^\prime}\big)_d < 0 \nonumber \\ 
        \text{or}\quad & \forall \bx^\prime \in \mathbb{P}:\ \big(\nabla_{\bx} f(\bx)\big|_{\bx=\bx^\prime}\big)_d > 0,
    \end{align}

\noindent where $(\cdot)_d$ refers to $\frac{\partial f(x_d)}{\partial x_d}|_{x_d=x^\prime_d}$, i.e.\@ the $d$-th component of the gradient $\nabla_{\bx} f(\bx)$ evaluated at $\bx^\prime$. 
We call a multivariate function $f: \mathbb{R}^D \rightarrow \mathbb{R}^K$ \emph{component-wise strictly monotonic} if the above definition holds for the gradient of every output component $\nabla_{\bx} f(\bx)_k$. 
\end{definition}

We note here that the softmax function\index{Softmax function}, whose probabilistic output is used for the discussed uncertainty metrics, is an example for a component-wise strictly monotonic function\index{Monotonicity!Component-wise strict}. 
As later lemmas investigate the behavior of functions in the limit, it is furthermore useful to define regions of the feature space that are unbounded in at least one direction. 
We call a \emph{partially-unbounded polytope}\index{Polytope!Partially-unbounded} (henceforth abbreviated by PUP) a convex subspace of $\mathbb{R}^D$ that is unbounded in at least one dimension $d$, i.e.\@ if the polytope's projection onto $d$ is either left-bounded by $-\infty$ or right-bounded by $\infty$, or both.

\subsection{Convergence of Predictions on OOD Data}\label{sec:convergence-on-ood}

\begin{figure}[htb]
    \centering
    \includegraphics[width=0.985\textwidth]{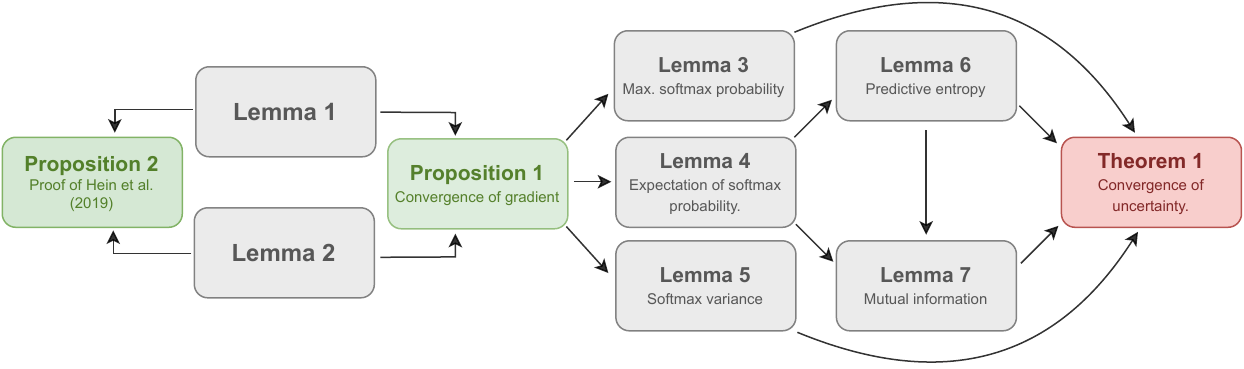}
    \caption[Dependencies between theoretical results in \cref{sec:uq-classification-pitfalls}.]{Dependencies between theoretical results. Information in parentheses denotes the section in the document.}
    \label{fig:flow-chart}
\end{figure}

In this section we will show that, moving the input to the extremes of the feature space, a ReLU classifier will converge to a fixed prediction. To demonstrate this, we must establish how the distance from the training data affects the network's logits. To this end, we utilize a known result stating that
neural networks employing piece-wise linear activation functions partition the input space into polytopes\index{Polytope} (such as in \cref{subfig:polytopes}; \citeauthor{arora2018understanding}, \citeyear{arora2018understanding}). 
Given the saturating nature of the softmax\index{Softmax function}, we conclude in \cref{proposition:overconfidence-softmax} that even for extreme feature values in the limit, the output distribution of the model will not change anymore. 
In order to help the reader untangle the interdependence of upcoming results, we provide a flow chart in \cref{fig:flow-chart}.
We first describe how to re-write a ReLU network---or any other network with piece-wise linear activation functions---as a piece-wise affine transformation, borrowing from \citet{croce2018randomized} and \citet{hein2019relu}. 
We start with the common form of $f_{\btheta}$ as a series of affine transformations, interleaved with ReLU activation functions, which we will denote by $\phi$:

\begin{equation}\label{eq:relu-net}
    f_{\btheta}(\bx) = \bW_L\phi\big(\bW_{L-1}\phi\big(\ldots \phi\big(\bW_1\bx + \bb_1\big) \ldots\big) + \bb_{L-1} \big)+ \bb_L.
\end{equation}

In the following, let $f_{\btheta}^l(\bx)$ denote the output of  layer $l$ before applying an activation function. We now define a layer-specific diagonal matrix $\Phi_l \in \mathbb{R}^{n_l \times n_l}$ in the following way, where $n_l$ denotes the hidden units in layer $l$:

\begin{equation}
    \bPhi_l(\bx) = \begin{bmatrix}
    \indicator{f_{\btheta}^l(\bx)_1 > 0} & \cdots & 0  \\
    \vdots &  \ddots & \vdots \\
    0 & \cdots & \indicator{f_{\btheta}^l(\bx)_{n_l} > 0} \\
    \end{bmatrix}.
\end{equation}

This allows us to rewrite \cref{eq:relu-net} by replacing the usage of $\phi$ with a matrix multiplication using $\bPhi_l$:

\begin{align}\label{eq:relu-net-replaced}
    f_{\btheta}(\bx)  = & \bW_L\bPhi_{L-1}(\bx)\big(\bW_{L-1}\bPhi_{L-2}(\bx) \nonumber \\ 
    & \big(\ldots \bPhi_1(\bx)\big(\bW_1\bx + \bb_1\big) \ldots\big) + \bb_{L-1} \big)+ \bb_L.
\end{align}

We can now distribute the matrix products inside-out, we which demonstrate below using a three-layer network:

\begin{align}
    f_{\btheta}(\bx) = & \bW_3 \bPhi_{2}(\bx)\big(\bW_2\bPhi_1(\bx)\big(\bW_1\bx + \bb_1) + \bb_2) + \bb_3 \\
    = & \bW_3 \bPhi_{2}(\bx)\big(\bW_2\bPhi_1(\bx)\bW_1\bx + \bW_2\bPhi_1(\bx)\bb_1) + \bb_2) + \bb_3 \\
    = & \underbrace{\bW_3 \bPhi_{2}(\bx)\bW_2\bPhi_1(\bx)\bW_1}_{=\ \bV(x)}\bx \nonumber \\
    & + \underbrace{\bW_3 \bPhi_{2}(\bx)\bW_2\bPhi_1(\bx)\bb_1 + \bW_3\bPhi_{2}(\bx)\bb_2 + \bb_3}_{=\ \ba(\bx)}.
\end{align}

This result lets rewrite the network as a single affine transformation $f_{\btheta}(\bx) = \bV(\bx)\bx + \ba(\bx)$ with

\begin{align}\label{eq:relu-linearization}
    \bV(\bx) & = \bW_L\bigg(\prod_{l=1}^{L-1}\bPhi_l(\bx)\bW_{L-l} \bigg) \\
    \ba(\bx) & = \bb_L + \sum_{l=1}^{L-1}\bigg(\prod_{l^\prime=1}^{L-l}\bW_{L+1-l^\prime}\bPhi_{L-l^\prime}(\bx)\bigg)\bb_l .
\end{align}

Note that the definition of $\bV(\bx)$ corresponds to the Jacobian\index{Jacobian} of $f_{\btheta}(\bx)$, meaning that $v_{kd} = \frac{\partial f_{\btheta}(\bx)_k}{\partial x_{d}}$. 
This is very useful, as it allows us to quickly check whether a network $f_{\btheta}$ is component-wise strictly monotonic by checking $\bV(\bx)$ for entries containing zeros. 
As \citet{hein2019relu} show, this formulation can also be used to characterize a set of polytopes $\mathcal{Q} = \{Q_1, \ldots, Q_M\}$ induced by $f_{\btheta}$ and that within each polytope, the function has a unique representation as an affine transformation. 
For this reason, we drop the dependence of $\bV$ and $\ba$ on $\bx$ when we refer to a specific polytope\index{Polytope}. 
Such polytopes are constructed by first retrieving the half-spaces induced by each of the network's neurons and then intersecting all said half-spaces to generate convex regions or polytopes.\footnote{
    We refer the reader to \cref{app:polytopes} or \citet{hein2019relu} for details on the construction, since it is not central to our reasoning.
} 
We are especially interested in polytopes that are unbounded in at least one direction. 
The results of \citet{croce2018randomized} and \citet{hein2019relu} show that there is a finite number of polytopes corresponding to the given network, and their Lemma 3.1 proves the existence of at least one unbounded polytope. 
Furthermore, under a mild condition on $\bV$, we can ascertain that $f_{\btheta}$ will be component-wise strictly monotonic on any polytope.

\begin{lemma}\label{lemma:strictly-monotonic}
    Suppose $f_{\btheta}$ is a ReLU network according to Definition \ref{def:logit}. 
    Then $f_{\btheta}$ is a component-wise strictly monotonic function on every of its polytopes $Q\in\mathcal{Q}$, as long as its corresponding matrix $\bV$ has no zero entries. 
\end{lemma}
\begin{proof}
Let $Q$ be one such polytope. As discussed, when restricted to $Q$, the network corresponds to an affine transformation $f_{\btheta}(\bx) = \bV\bx + \ba$ with $\bV \in \mathbb{R}^{K \times D}$ and $\ba\in \mathbb{R}^{K}$. $f_{\btheta}(\bx)_k$ thus corresponds to the dot product of the $k$-th row of $\bV$ and $\bx$ plus the $k$-th element of $\ba$. 
It follows that the partial derivative of $f_{\btheta}(\bx)_k$ with respect to a dimension $d$ equals the element $v_{kd}$ in $\bV$.
This entails that, if $v_{kd} \neq 0$, at any point $\bx \in Q$ the gradient will be always positive or always negative. 
\end{proof}

We note here that the component-wise strict monotonicity\index{Monotonicity!Component-wise} of $f_{\btheta}$ and softmax do not entail the same property for $P_{\btheta}$.\footnote{
    To see a counterexample, the reader can check that even assuming component-wise strict monotonicity for $f_{\btheta}$, if the matrix $\bV$ associated to $f_{\btheta}$ on a specific polytope has a column $d$ filled with the same value $a$, then the resulting $p_{\btheta}$ will have a gradient of 0 at dimension $d$, regardless of what class we are considering. 
    This is because the partial derivatives of the softmax, when all multiplied by the same constant $a$, add up to zero.
} 
Nonetheless, the monotonic behavior of $f_{\btheta}$ is sufficient to drive the logits to plus or minus infinity in the limit, a phenomenon that constrains the output of $p_{\btheta}$ as we scale a data sample away from training data.
We begin our investigation of behavior in the limit by showing that if we scale a vector only in a single dimension, we eventually always remain within a unique PUP.\index{Polytope!Partially-unbounded}

\begin{lemma}\label{lemma:unique-pup}
    Let $\bx^\prime \in \mathbb{R}^D$ and $\mathcal{Q} = \{Q_1, \ldots, Q_M\}$ be the finite set of polytopes generated by a network $f_{\btheta}$.
    Let $\balpha \in \mathbb{R}^D$ be a vector s.t. $\forall d^\prime \neq d, \alpha_{d^\prime} = 1$. 
    There exist a value $\beta >0$ and $m\in {1, \dots, M}$ such that for all $\alpha_d >\beta$, the product $\bx^\prime\circ\balpha$ lies within $Q_m$.
\end{lemma}

\begin{proof}
The proof mirrors the proof of Lemma 3.1 in \citet{hein2019relu}, so we only provide the intuition. 
By contradiction, suppose that there is no unique polytope and thus the point $\bx^\prime\circ\balpha$ must traverse different polytopes as we scale up $\alpha_d$. 
Since there are finitely many polytopes, eventually the same polytope $Q_m$ will have to be traversed twice. 
Since the polytopes are convex, all the points on the line connecting the locations of where the boundary of $Q_m$ was crossed the first and second time must lie within $Q_m$, but this contradicts the fact that the scaled point traverses different polytopes.
\end{proof}

From here onward, we adapt the following shorthand to simplify notation: 
Given a scaling vector  $\balpha \in \mathbb{R}^D$ s.t. $\forall d^\prime \neq d, \alpha_{d^\prime} = 1$, we use $\mathbb{P}(\bx^\prime, d)$ to denote the PUP that $\bx^\prime$ lands in when scaling it with $\alpha_d$ in the limit. 
This definition implies that we can only scale parallel to the basis vectors and not arbitrary directions (for a discussion on how restrictive this is, see \cref{sec:synthetic-experiments}).   
Finally, in the next lemma we establish that the output distribution converges to a fixed point using the $l_2$-norm of the gradient $\nabla_{\bx} P_{\btheta}(y=k \mid \bx)$. 
Generally, in regions of the feature space where the classifier predicts the same probability distribution over classes, small perturbations in the input $\bx$ will not change the prediction. 
Therefore, the gradient in these regions w.r.t.\@ the input will be small and potentially even correspond to the zero vector, with a norm of (or close to) zero.

\begin{proposition}[Convergence of predictions in the limit]\label{proposition:overconfidence-softmax}
    Suppose that $f_{\btheta}$ is a ReLU-network. 
    Let $\bx^\prime\in\mathbb{R}^D$, suppose $\balpha$ is a scaling vector and that the associated PUP $\mathbb{P}(\bx^\prime, d)$ has a corresponding matrix $\mathbf{V}$ with no zero entries.
    Then it holds that
    
    \begin{equation}
        \forall k \in [K]:\quad \lim\limits_{\alpha_d \to \infty}\big|\big|\nabla_{\bx} P_{\btheta}(y=k \mid \bx)\big|_{\bx = \balpha\circ\bx^\prime}\big|\big|_2 = 0.
    \end{equation}
\end{proposition}

The whole proof can be found in \cref{app:proposition1}, so we present the main intuitions here. 
Because of \cref{lemma:unique-pup}, we know the scaled point $\balpha\circ\bx^\prime$ will end up in a unique PUP. 
The assumption on $f_{\btheta}$ then triggers \cref{lemma:strictly-monotonic}, from which we can infer that scaling the input in a single dimension leads all logits to $\pm\infty$. 
Because of the saturating property of the softmax\index{Softmax function}, this will in turn provoke the output of $p_{\btheta}$ to converge to a fixed point.
As an aside, we recast Theorem 3.1 by \citet{hein2019relu} in our framework, showing that the model becomes increasingly certain in a single class, placing all its probability mass on it in the limit. 
The proof of this additional proposition is in \cref{app:softmax-limit}.

\begin{proposition}\label{proposition:softmax-limit}
    Let $f_{\btheta}$ be ReLU network. 
    Let $\bx^\prime \in \mathbb{R}^D$,  suppose $\balpha$ is a scaling vector and that the associated PUP $\mathbb{P}(\bx^\prime, d)$ has a corresponding matrix $\mathbf{V}$ with no zero entries. 
    Assume the $d$-th column of $\bV$ has no duplicate entries. 
    Then there exists a class $k$ such that  
    
    \begin{equation*}
         \lim_{\alpha_d \rightarrow \infty} \bar{\sigma}(f_{\btheta}(\balpha\circ\bx^\prime))_k = 1.
    \end{equation*}
\end{proposition}

In conclusion, we have shown in this section that the output probabilities of ReLU networks are less and less sensitive to small perturbations of the input in the limit and, under the assumptions of \cref{proposition:softmax-limit}, will converge to favor a single class with very high confidence\index{Confidence}. 
In the next section we prove that all other uncertainty metrics also converge to fixed values in the limit. 

\subsection{Convergence of Uncertainty Metrics on OOD Data}\label{sec:overconfidence-metrics}

In \cref{proposition:overconfidence-softmax}, we have established how the prediction of a model converges to a fixed point when feature values become extreme. 
We now show a similar property about the uncertainty estimation techniques introduced in \cref{sec:know-your-limits-preliminaries}. 
The fact that this the same pathologies appear for more complex metrics is not immediately obvious, and one might assume that we can curb the deficiency of the simple confidence\index{Confidence} score by using more sophisticated metrics and Bayesian deep learning techniques.
To this end, we have to establish how the predictions coming from multiple model instances interact, a point we analyze in \cref{aggregation-theorem}. 
Then, we demonstrate how the uncertainty metrics also converge to a fixed value in the limit by proving the case for each of them in turn, before bundling our results in \cref{theorem:know-your-limits-main-theorem}.
 We start with the easiest metric, which also applies to a single ReLU network.

\begin{lemma}[Maximum softmax probability]\label{lemma:max-prob}
     Suppose that $f_{\btheta}$ is a ReLU network. 
     Let $\bx^\prime\in\mathbb{R}^D$, suppose $\balpha$ is a scaling vector and that the associated PUP $\mathbb{P}(\bx^\prime, d)$ has a corresponding matrix $\mathbf{V}$ with no zero entries. 
     Then
     
    \begin{equation*}
        \lim\limits_{\alpha_d \to \infty}\big|\big|\nabla_{\bx} \max_{k \in [K]} P_{\btheta}(y=k \mid \bx) \big|_{\bx = \balpha\circ\bx^\prime}\big|\big|_2 = 0.
    \end{equation*}
\end{lemma}

\begin{proof}
    The gradient of the $\max$ function will be a specific $\nabla_{\bx}P_{\btheta}(y=k \mid \bx)$, which reduces this to the case already proven in \cref{proposition:overconfidence-softmax}.
\end{proof}

Note that for this metric, the combination with \cref{proposition:softmax-limit} shows that the model is fully confident in a single class in the limit. 
For our following lemmas, we consider uncertainty scores that are based on multiple instances, e.g.\@ different ensemble members or forward passes using re-sampled dropout masks. 
What all of these approaches have in common is that for every $b$ in $1, \ldots, B$, the network parameters $\btheta^{(b)}$ will differ, and thus also the polytopal tesselation of the feature space\index{Polytope}. 
Hence, we have to adjust our assumptions accordingly. 
For every instance $b$, let us denote the affine function on a polytope $Q^{(b)}$ as $f_{\btheta}^{(b)}(\bx)=\bV^{(b)}\bx + \ba^{(b)}$.
In order for our previous strategy to hold, we now assume for all $b \in [B]$  that $\mathbb{P}^{(b)}(\bx^\prime, d)$ has a matrix $\bV^{(b)}$ which does not have any zero entries. 
Note that even though this assumption has to hold for all $b$, this does not mean that the matrices have to be identical.

\begin{lemma}[Convergence of aggregated predictions in the limit]\label{aggregation-theorem}
    Suppose that $f_{\btheta}^{(1)}, \ldots, f_{\btheta}^{(K)}$ are ReLU networks. 
    Let $\bx^\prime\in\mathbb{R}^D$, suppose $\balpha$ is a scaling vector  and that for all $k$, the associated PUP $\mathbb{P}^{(k)}(\bx^\prime, d)$ has a corresponding matrix $\mathbf{V}^{(k)}$ with no zero entries. 
    Then
    
    \begin{equation*}
        \lim\limits_{\alpha_d \to \infty}\big|\big|\nabla_{\bx}\ \mathbb{E}_{p(\btheta \mid \mathbb{D})}\big[P_{\btheta}(y=k \mid \bx)\big]\big|_{\bx = \balpha\circ\bx^\prime}\big|\big|_2 = 0.
    \end{equation*}

\end{lemma}

The full proof of this lemma can be found in \cref{app:aggregation-theorem}. 
The analogous lemmas for the remaining uncertainty metrics\index{Uncertainty metric}---predictive entropy\index{Entropy!Predictive}, class variance\index{Class variance} and mutual information\index{Mutual information}---are stated and proved in \cref{app:asymptotic-softmax-variance,app:asymptotic-predictive-entropy,app:asymptotic-mutual-information}. 
The proof strategy for all further metrics is to simplify and reduce the uncertainty metrics such that \cref{aggregation-theorem} or \cref{proposition:overconfidence-softmax} can be applied. 
All of these results combined now pave the way for our central theorem.

\begin{theorem}[Convergence of uncertainty level in the limit]\label{theorem:know-your-limits-main-theorem}
Suppose that $f_{\btheta}^{(1)}, \ldots, f_{\btheta}^{(B)}$ are ReLU networks. 
Let $\bx^\prime\in\mathbb{R}^D$, suppose $\balpha$ is a scaling vector and that for all $b$, the associated PUP $\mathbb{P}^{(b)}(\bx^\prime, d)$ has a corresponding matrix $\mathbf{V}^{(b)}$ with no zero entries. 
Then, whenever uncertainty is measured via either of the following metrics

\begin{enumerate}
    \item Maximum softmax probability in \cref{eq:max-softmax};
    \item Predictive entropy in \cref{eq:predictive-entropy};
    \item Class variance in \cref{eq:class-variance};
    \item Approximate mutual information in \cref{eq:mutual-information};
\end{enumerate}

\noindent the network(s) will converge to fixed uncertainty scores for $\bx^\prime \circ \balpha$ in the limit of $\alpha_d \rightarrow \infty$.
\end{theorem}

\begin{proof}

The four parts of the theorem are proven separately by \cref{lemma:max-prob,lemma:asymptotic-softmax-variance,app:asymptotic-softmax-variance,lemma:asymptotic-predictive-entropy,lemma:asymptotic-mutual-information} in \cref{app:theoretical-appendix}.
\end{proof}

What follows from this result is that methods based on multiple instances of ReLU classifiers will suffer from the aforementioned problem as long as uncertainty\index{Uncertainty} is estimated with one of the techniques listed above. 
Next we demonstrate how these assumptions and results apply on synthetic data.

\subsection{Synthetic Data Experiments}\label{sec:synthetic-experiments}

To illustrate our findings, we plot the uncertainty surfaces and the gradient magnitudes of different models and uncertainty metric pairings on the half moons dataset, which we generate using the corresponding function in the \texttt{scikit-learn} package \citep{pedregosa2011scikit}. 
Detailed information about the procedure can be found in \cref{app:synthetic-data-experiments} along with additional plots.\\ %\footnote{The code used for the experiments is publicly available under \url{https://github.com/Kaleidophon/know-your-limits}.}\\

\begin{figure}[htb]
    \captionsetup{margin=8pt}
    \begin{subfigure}[t]{0.48\textwidth}
        \centering
        \begin{tabular}{cc}
            \includegraphics[width=0.42\textwidth]{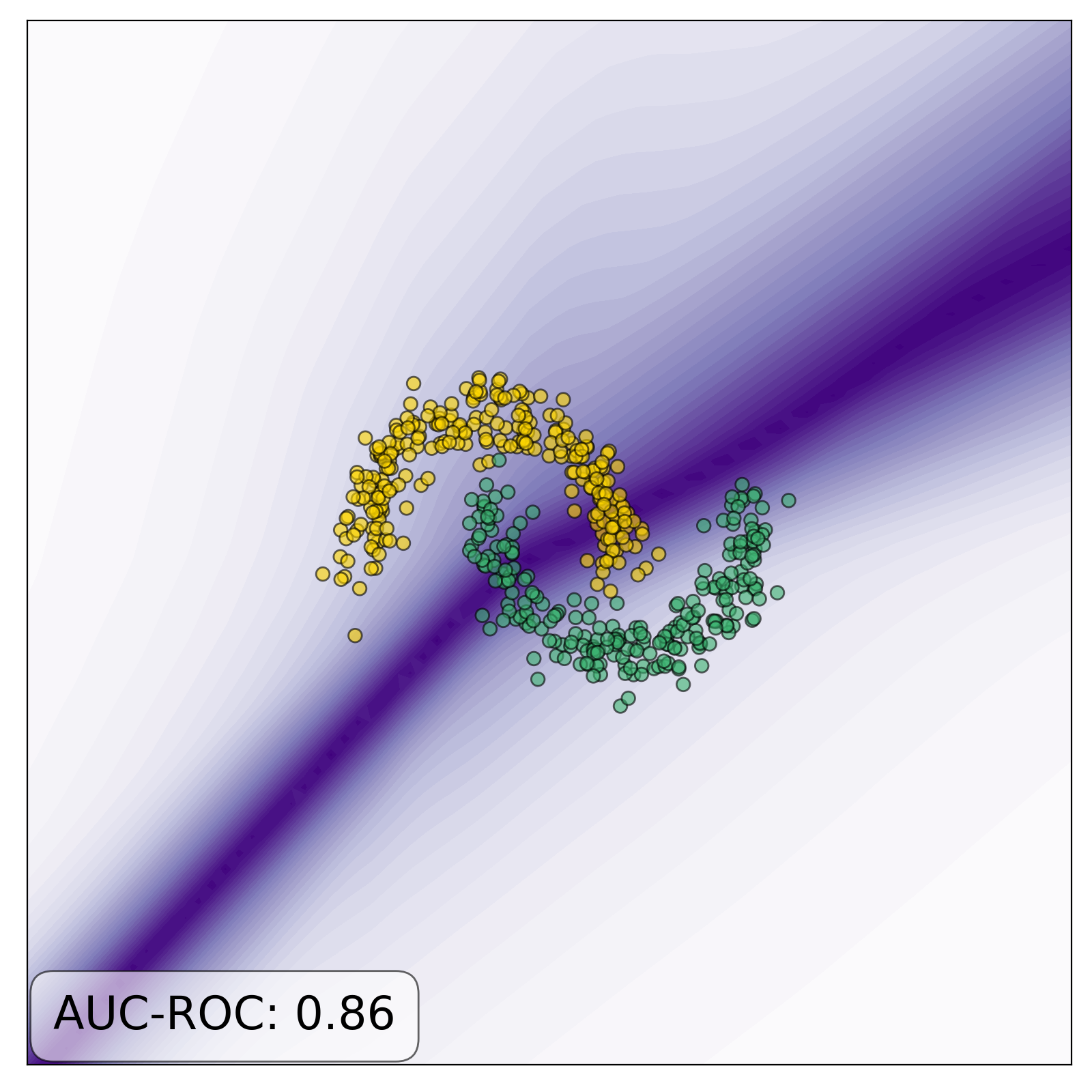} & \includegraphics[width=0.47\textwidth]{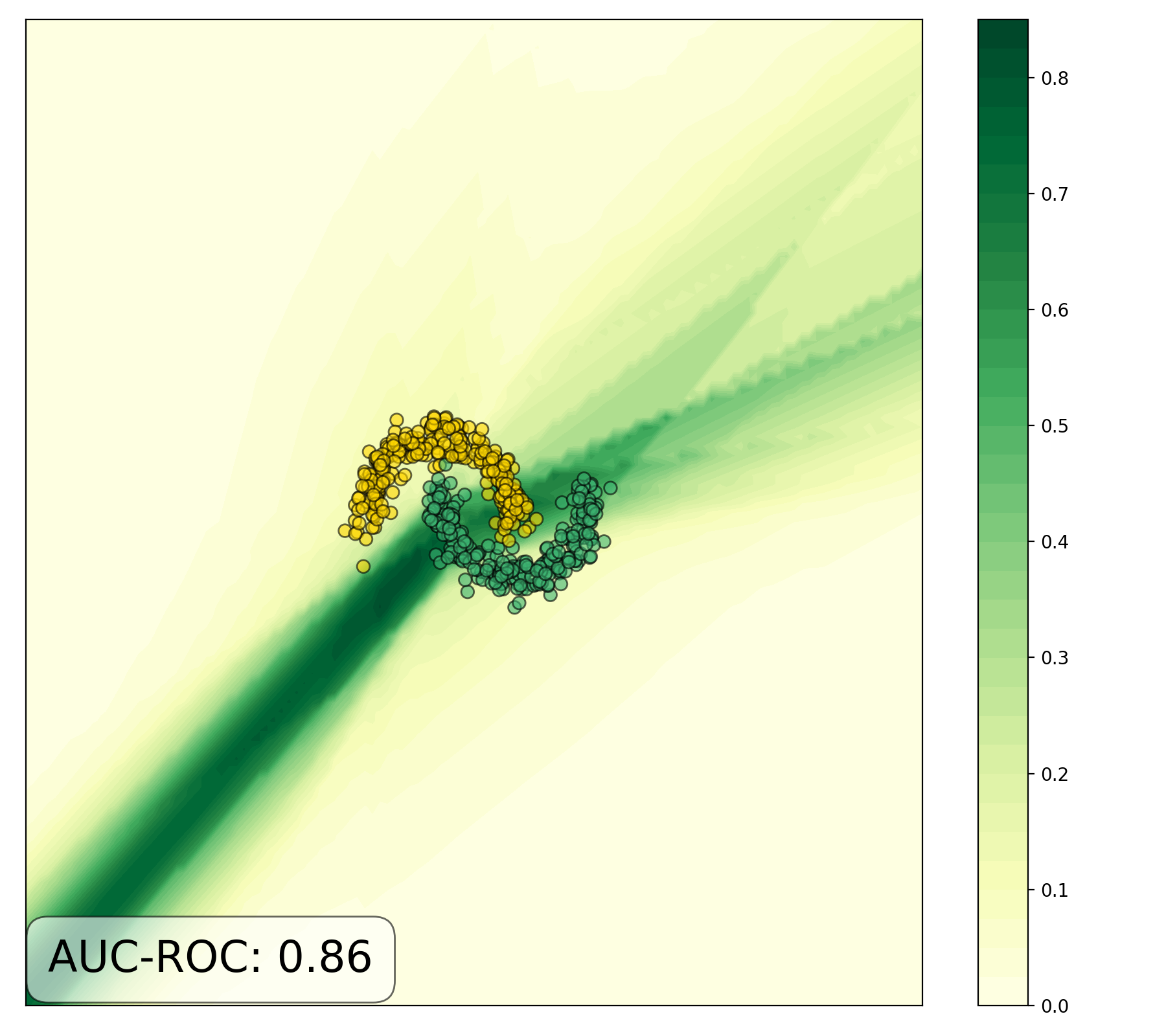} \\
        \end{tabular}
        \caption{Neural discriminator with maximum probability \citep{hendrycks2017baseline}.}\label{subfig:nn-maxprob}
    \end{subfigure}
    \hfill
    \begin{subfigure}[t]{0.48\textwidth}
        \centering
        \begin{tabular}{cc}
            \includegraphics[width=0.42\textwidth]{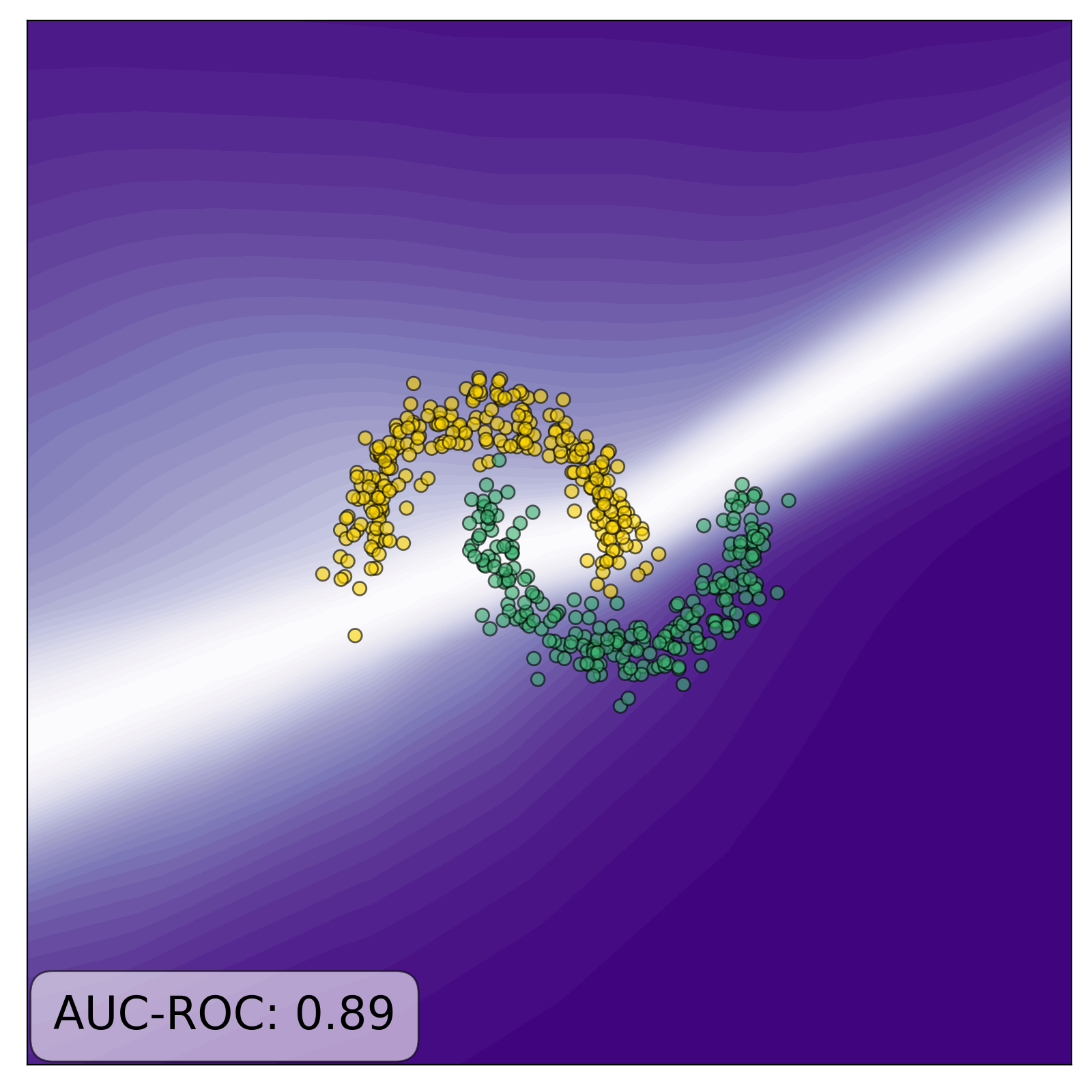} & \includegraphics[width=0.47\textwidth]{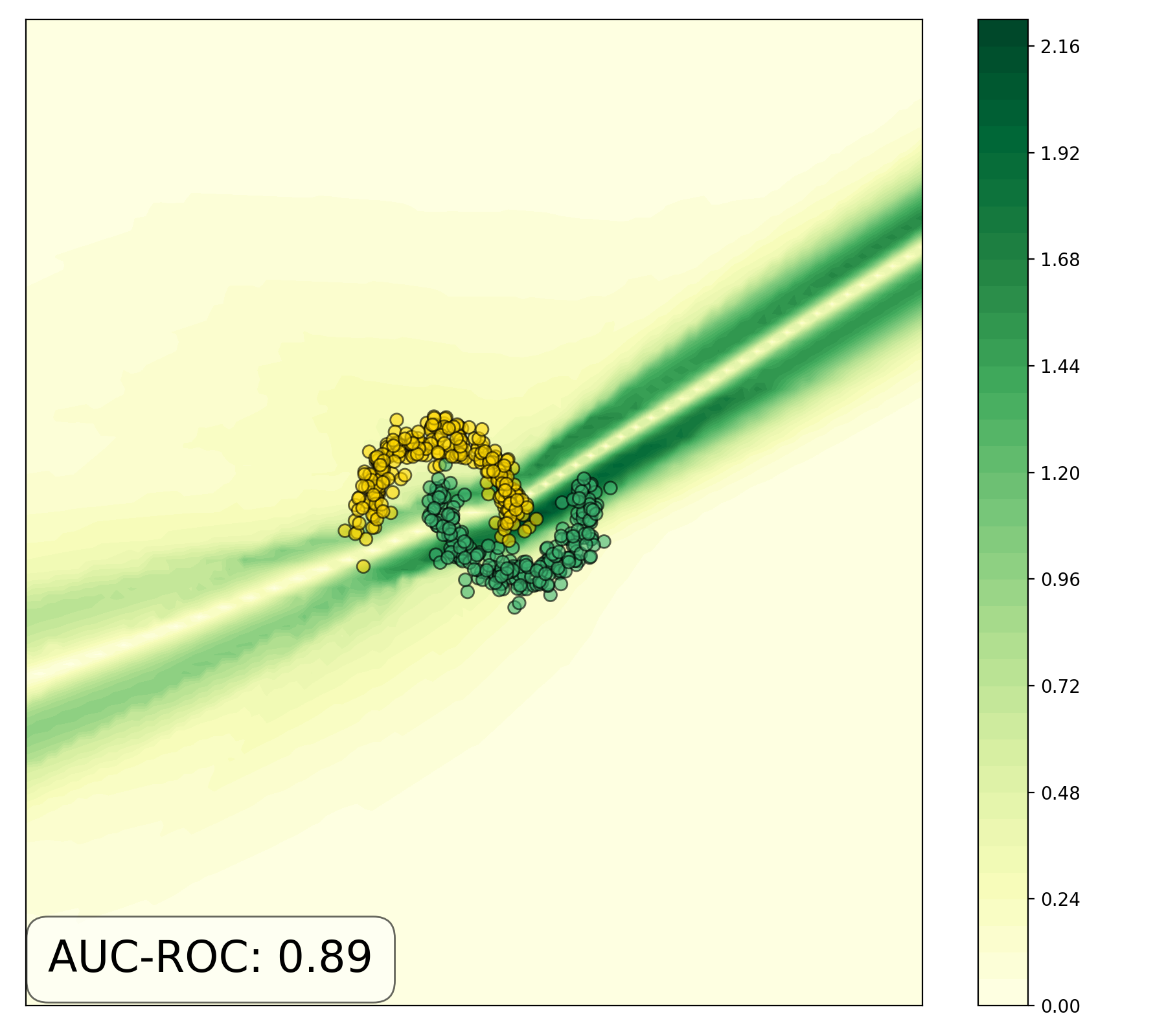} \\
        \end{tabular}
        \caption{MC Dropout \citep{gal2016dropout} with mutual information \citep{gal2018understanding}.}\label{subfig:nn-mcdropout}
    \end{subfigure}

    \begin{subfigure}[t]{0.48\textwidth}
        \centering
        \begin{tabular}{cc}
            \includegraphics[width=0.42\textwidth]{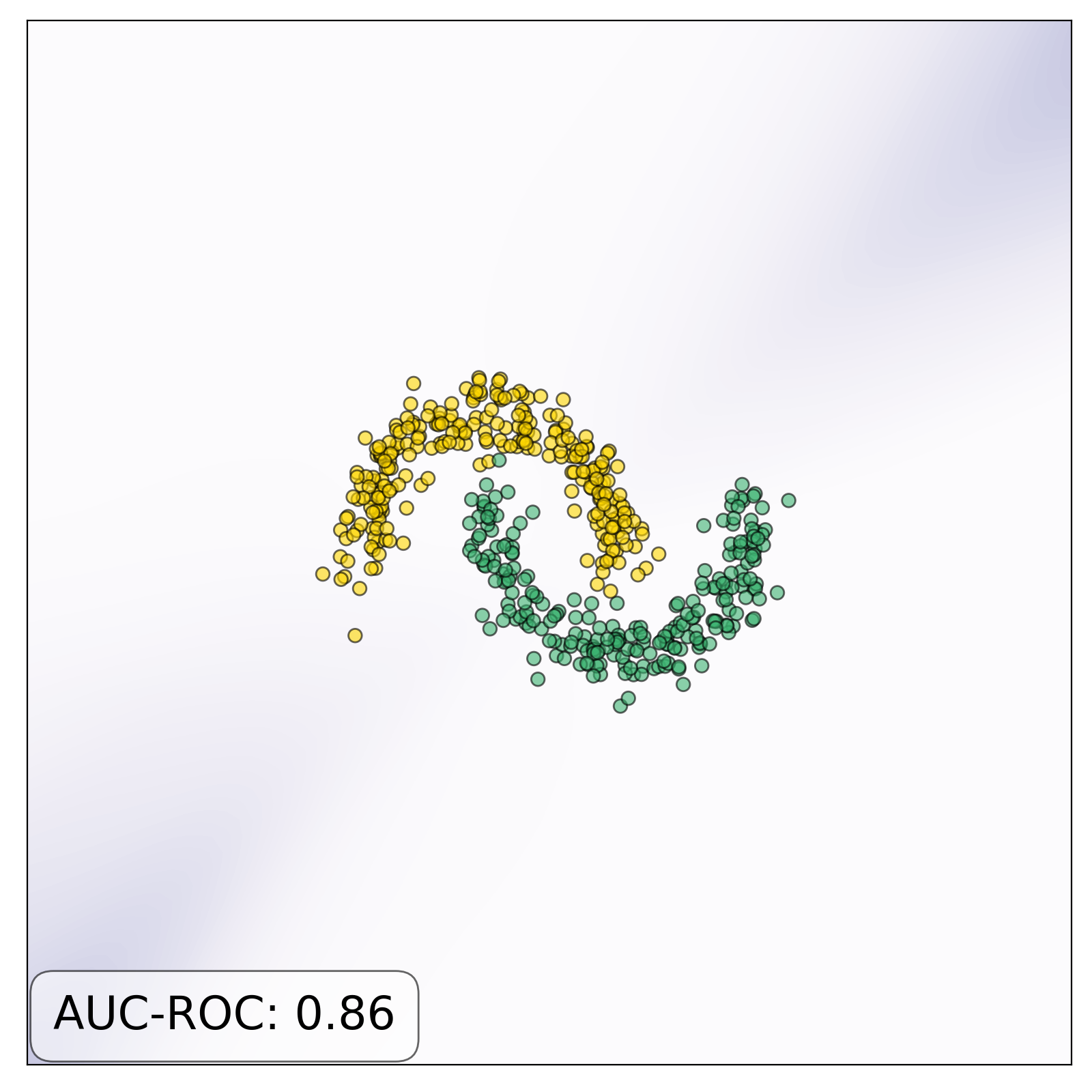} & \includegraphics[width=0.47\textwidth]{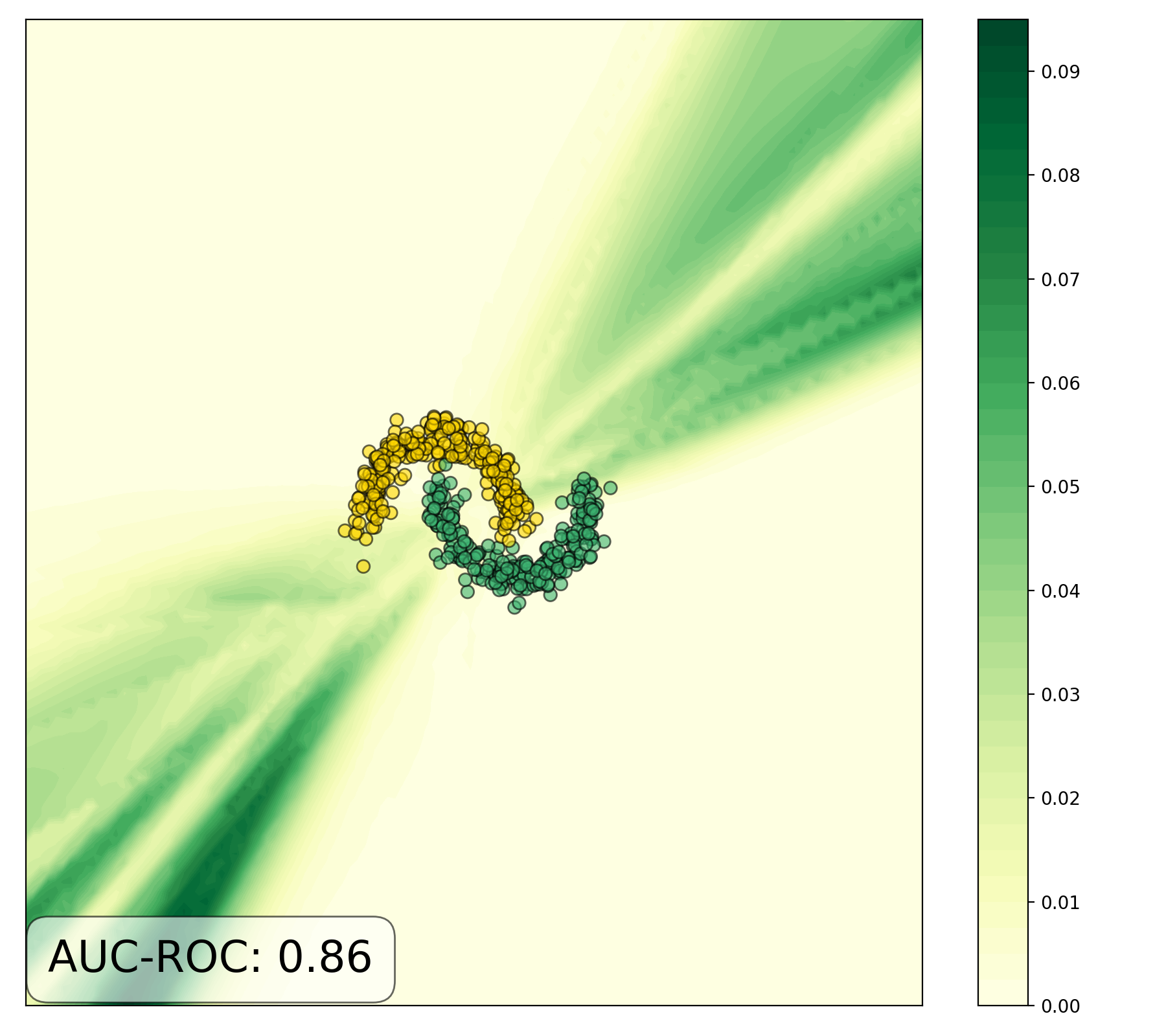} \\
        \end{tabular}
        \caption{Neural ensemble \citep{lakshminarayanan2017simple} with class variance.}\label{subfig:nn-ensemble}
    \end{subfigure}
    \hfill
    \begin{subfigure}[t]{0.48\textwidth}
        \centering
        \begin{tabular}{cc}
            \includegraphics[width=0.42\textwidth]{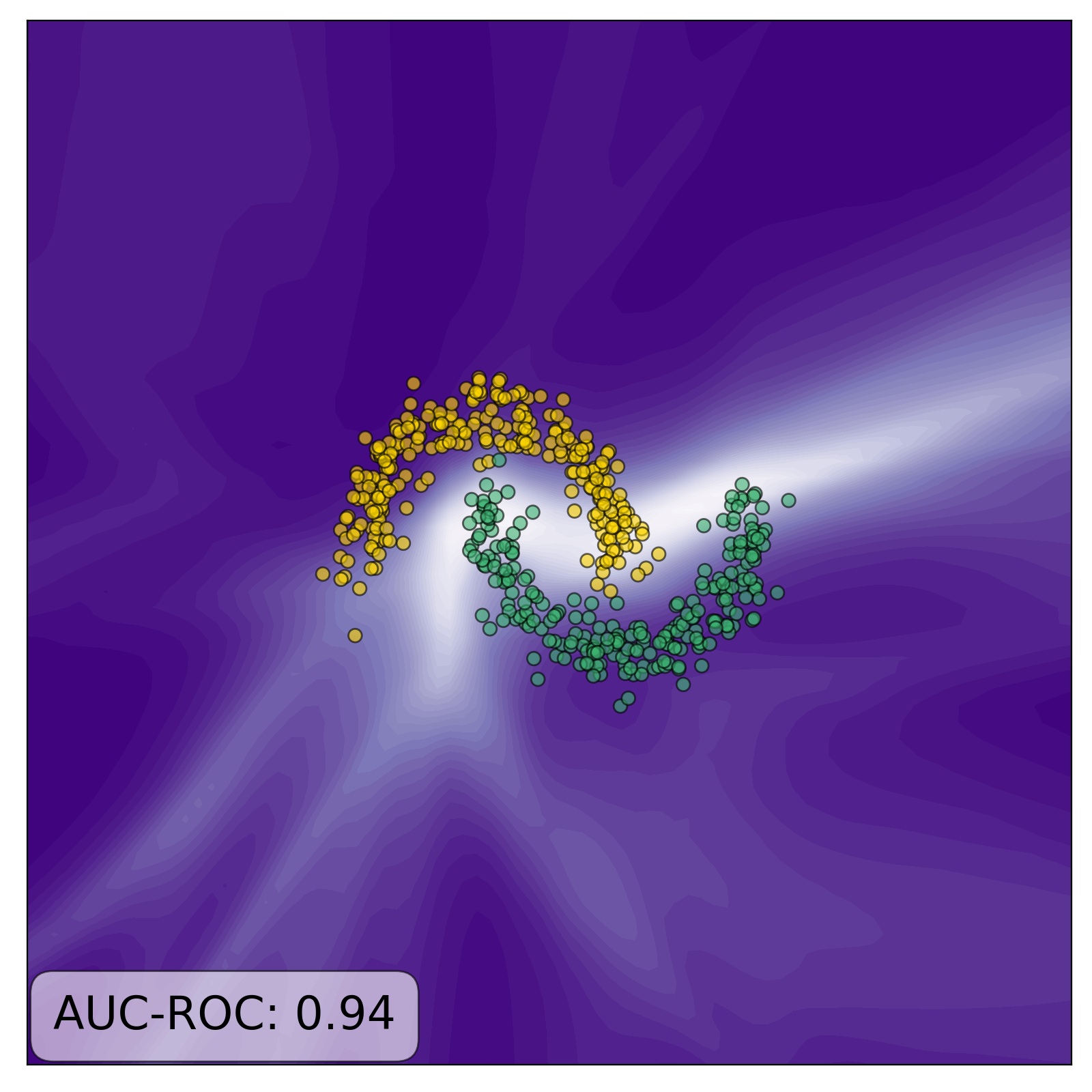} & \includegraphics[width=0.47\textwidth]{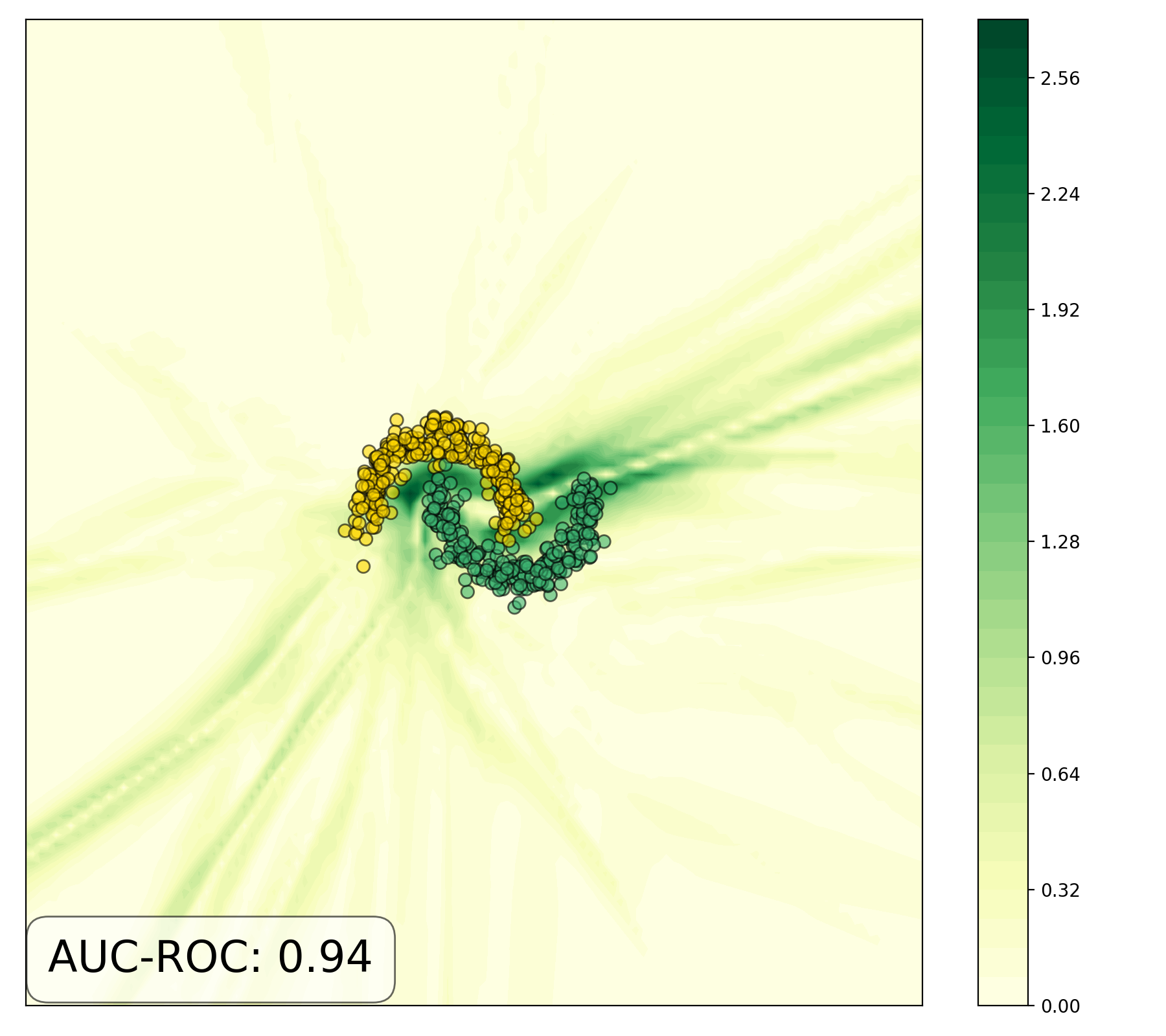} \\
        \end{tabular}
        \caption{Anchored ensemble \citep{pearce2020uncertainty} with mutual information \citep{gal2018understanding}.}\label{subfig:nn-anchored-ensemble}
    \end{subfigure}
    \caption[Uncertainty landscapes of ReLU classifiers trained on the half-moon dataset.]{Uncertainty on the half-moon dataset, including the binary classification AUROC. 
    (Left plots) The uncertainty surface is represented with increasingly darker shades of purple, with white being the lowest uncertainty. 
    Open-ended regions of static certainty appear across different models and metrics, and are extrapolated to unseen data (see \cref{subfig:nn-maxprob,subfig:nn-mcdropout,subfig:nn-ensemble,subfig:nn-anchored-ensemble}); 
    this phenomenon is less apparent in some instances (\cref{subfig:nn-anchored-ensemble}). 
    (Right plots) Increasing shades of green indicate the magnitude of the gradient of the uncertainty score w.r.t.\@ the input. 
    All metrics show open ended regions where the magnitude approaches zero.}
    \label{fig:entropy-plots}
\end{figure}

For a single network, we can observe in \cref{subfig:nn-maxprob} that there exist vast open-ended regions of stable confidence\index{Confidence}, confirming the findings of \cref{theorem:know-your-limits-main-theorem}. 
However, in the right part of \cref{subfig:nn-maxprob} we can observe green regions with high gradient magnitude which do not seem to comply with our findings. 
In this case, we can see that these regions follow the decision boundaries. 
Due to the exponential function in the softmax\index{Softmax function}, it is intuitive that small perturbation in these areas would have a large impact on the uncertainty score, resulting in a high gradient magnitude. 
But why does the magnitude not decrease in the limit as predicted by \cref{theorem:know-your-limits-main-theorem}? 
We formulated our scaling vector $\balpha$ in way that only allows scaling along one of the coordinate axes. 
Therefore, if the decision boundaries are not parallel to the axes, by scaling we eventually escape the green areas and arrive at an area with gradient of magnitude zero. 
If the green regions were parallel to the axes then this would result in a violation of our main assumption. 
Traversing the input space parallel to a decision boundary in direction $d$ will not influence the prediction within the polytope, meaning that there will be entries  $v_{cd}=0$.\footnote{
    A decision boundary in a polytope is not the only way in which this assumption can be broken, but it still appears to hold reasonably often. 
    For instance, just around $6.3 \%$ of plotted points in \cref{fig:know-your-limits-figure} possess a matrix $\bV$ with at least one zero entry---all located in the PUP in the top right corner.
}\\

Turning to predictions aggregated from multiple network instances in \cref{subfig:nn-mcdropout,subfig:nn-ensemble,subfig:nn-anchored-ensemble}, we again observe large regions of constant uncertainty. 
The high-confidence region in the plots using mutual information\index{Mutual information} (\cref{subfig:nn-anchored-ensemble}) displays a different behavior from the others. 
As this metric aims to isolate epistemic uncertainty\index{Uncertainty!Epistemic}, it makes sense that uncertainty would be lowest around the training data, i.e.\@ where the model is best specified. 
The character of the green regions in the bottom part of \cref{subfig:nn-ensemble} and \cref{subfig:nn-anchored-ensemble} can again be explained by decision boundaries: 
In these cases, we have multiple instances with parameters $\btheta^{(k)}$, all with their own polytopal structure. 
When they overlap, the regions of the feature space where the assumption of our theorem is violated can either grow (\cref{subfig:nn-ensemble}) or shrink (\cref{subfig:nn-anchored-ensemble}), depending on the diversity among instances. 
The fact that the anchored ensemble\index{Ensembling} in \cref{subfig:nn-anchored-ensemble} does not exhibit such uniform regions of uncertainty like the vanilla ensemble could be explained by the fact that its training procedure encourages diversification between members. 
The difference between MC Dropout\index{Dropout!Monte Carlo} and ensemble models can be elucidated using recent insights that variational methods tend to only explore a single mode of the weight posterior $p(\btheta \mid \mathbb{D})$, while ensemble members often spread across multiple modes \citep{wilson2020bayesian}.\\

Overall, we have seen that our theorem can explain why an overgeneralization of uncertainty scores beyond the training data results in failure in OOD detection. 
We also explored the cases in which our assumptions are violated, i.e.\@ by multiple, diverse model instances. 
In such scenarios, identification of OOD samples could in theory succeed, but often fails to do so reliably, see e.g.\@ \citet{ovadia2019can, ulmer2020trust}. 
These insights can also help explain many other empirical findings in this regard on a variety of real-world datasets, e.g.\@ \citet{gal2018understanding,kompa2021empirical}.

\section{Uncertainty \& Calibration in Low-Resource NLP}\label{sec:benchmarking-nlp-uncertainty}

\begin{footnotesize}
    \vspace{-2.5ex}
    \emph{The following work is based on \citet{ulmer-etal-2022-exploring}}.\\
    \vspace{2.5ex}
\end{footnotesize}

The previous section looked at a somewhat simplified setting using ReLU networks.
In practice, most contemporary NLP architectures are based on much more complex architectures like the transformer\index{Transformer} \citep{vaswani2017attention}.
Theoretical arguments like \cref{theorem:know-your-limits-main-theorem} then become harder, since making monotonicity\index{Monotonicity} arguments with model components such as multi-head attention is not trivial.
Additionally, the proof strategy of scaling a single feature value into the limit is not applicable, because the input changes from single feature vectors to a series of subword token embeddings.
For this reason, we turn to an empirical approach instead.\\

While there exist many works on images \citep{lakshminarayanan2017simple, ovadia2019can} and tabular data \citep{ruhe2019bayesian, ulmer2020trust, malinin2021shifts}, the quality of uncertainty estimates provided by neural networks remains underexplored in NLP.\index{Natural language processing} 
In addition, as model underspecification\index{Underspecification!Model} due to insufficient data presents a risk \citep{d2020underspecification}, the increasing interest in less-researched languages with limited resources raises the question of how reliably uncertain predictions can be identified. 
This motivates the following research questions:
 
\begin{enumerate}
    \item What are the best approaches in terms of uncertainty quality and calibration?
    \item How are models impacted by the amount of available training data?
    \item What are differences in how the different approaches estimate uncertainty?
\end{enumerate}%

 \paragraph{Contributions.} 
 We address these questions by conducting a comprehensive empirical study of eight different models for uncertainty estimation for classification and evaluate their effectiveness on three languages spanning distinct NLP tasks, involving sequence labeling and classification. 
 We show that while approaches based on pre-trained models and ensembles achieve the best results overall, the quality of uncertainty estimates on OOD data can become worse using \emph{more} data.
 In an analysis on an instance-level, we also discover that a model's total uncertainty seems to mostly consist of its data uncertainty. 
 %We make our experimental code and model implementations available open-source in separate repositories, aiding future research in this direction.\footnote{The model zoo is available under \url{https://github.com/Kaleidophon/nlp-uncertainty-zoo}, with the code for the experiments available under \url{https://github.com/Kaleidophon/nlp-low-resource-uncertainty}.}

\subsection{Methodology}\label{sec:exploring-methodology}

\paragraph{Models.} 
We choose a variety of models that cover a range of different approaches based on the two most prominently used architectures in NLP:\index{Natural language processing} 
Long-short term memory networks\index{Long-short term memory network} (LSTMs; \citealp{hochreiter1997long}) and transformers \citep{vaswani2017attention}. 
Inside the first family, we use the variational LSTM \citep{gal2016theoretically} based on MC dropout\index{Dropout!Monte Carlo} \citep{gal2016dropout}, the Bayesian LSTM\index{Long-short term memory network!Bayesian} \citep{fortunato2017bayesian} implementing Bayes-by-backprop\index{Bayes-by-backprop} \citep{blundell2015weight} and the ST-$\tau$ LSTM\index{Long-short term memory network!ST-$\tau$}  \citep{wang2021uncertainty}, modeling transitions in a finite-state automaton, as well as an ensemble\index{Ensembling} \citep{lakshminarayanan2017simple}. 
In the second family, we count the variational transformer\index{Transformer!Variational} \citep{xiao2020wat}, also using MC dropout, the SNGP transformer\index{Gaussian process}\index{Transformer!SNGP} \citep{liu2022simple}, using a Gaussian Process output layer, and the deep deterministic uncertainty transformer (DDU; \citealp{mukhoti2021deterministic}), fitting a Gaussian mixture model on extracted features. 
We elaborate on implementation details in \cref{app:exploring-predictive-uncertainty-implementation-details}.

\paragraph{Uncertainty Metrics.} \index{Uncertainty metric}
We test the same metrics as introduced in \cref{sec:know-your-limits-preliminaries}, but add a few additional ones.
One of them is the softmax\index{Softmax function} gap \citep{tagasovska2019single}, i.e.\@ the difference between the two largest probabilities of the classifier's output distribution. 
As another metric, we consider the Dempster-Shafer metric\index{Dempster-Shafer metric} \citep{sensoy2019evidential}, defined as $ K / (K + \sum_{k=1}^K \exp(z_k))$, where $z_k$ denotes the logit corresponding to class $k$. 
Since this metric considers logits, it might be able to avoid the saturation on OOD shown by \citet{hein2019relu} or in \cref{sec:overconfidence-metrics}.
While all metrics so far can be mixed and matched with all the tested models, there are also a few model-specific metrics.
For instance, the DDU transformer\index{Transformer!DDU} by \citet{mukhoti2021deterministic} uses the log-probability of the last layer network activation under a Gaussian mixture model fitted on the training set as an additional metric. 
Since all others models are trained or fine-tuned as classifiers, they cannot assign log-probabilities to sequences. 
Lastly, since some tasks require predictions for every time step of a sequence, we determine the uncertainty of a whole sequence in these cases by taking the mean over all step-wise uncertainties.\footnote{
    We also just considered the \emph{maximum} uncertainty over a sequence, with similar results.
} 
A more principled approach for sequences is for instance provided by \citet{malinin2021uncertainty} in the context of NLG, and we leave the extension and exploration of such methods for different uncertainty metrics, models and tasks to future work.

\subsection{Dataset Selection \& Creation}\label{sec:exploring-data-creation}

\begin{table}[ht!]
    \centering 
    \resizebox{0.995\textwidth}{!}{
        \renewcommand{\arraystretch}{1.5}
        \begin{tabular}{@{}llllrr@{}}
            \toprule
            \makecell[bl]{Lang.} & \makecell[bl]{Task} & \makecell[bl]{Dataset} & \makecell[bl]{OOD Test Set} & \makecell[br]{\# ID / OOD} & \makecell[br]{Training Sizes} \\
            \midrule
            EN & \makecell[tl]{Intent\\ Classification} & \makecell[tl]{Clinc Plus\\\citep{larson2019evaluation}} & \makecell[tl]{Out-of-scope\\ voice commands} & \makecell[tr]{15k/1k} & \makecell[tr]{15k/12.5k/10k}\\
            DA & \makecell[tl]{Named Entity\\ Recognition} & \makecell[tl]{Dan+ News\\\citep{plank2020dan+}} & Tweets &  \makecell[tr]{4382/109} &  \makecell[tr]{4k/2k/1k}\\
            FI & PoS Tagging & \makecell[tl]{Finnish UD Treebank \\(\citealp{haverinen2013tdt};\\\citealp{pyysalo2015udfinnish};\\\citealp{kanerva-2022-ood})} & \makecell[tl]{Hospital records,\\online forums,\\ tweets, poetry} & \makecell[tr]{12217/2122} &  \makecell[tr]{10k/7.5k/5k}\\
            \bottomrule
        \end{tabular}%
    }
    \caption[Datasets used for text classification experiments.]{Datasets used for our experiments. The original and sub-sampled number of sequences for experiments are given on the right.}\label{table:datasets}
\end{table}

% The experiments comprise a multitude of datasets in different language and from different sources, which are portrayed in this section. 

\paragraph{In-Distribution Training Sets.} 
In our experiments, we test three different languages combined with one NLP task, each.
For the languages, we choose English (Clinc Plus; \citeauthor{larson2019evaluation}, \citeyear{larson2019evaluation}), Danish in the form of the Dan+ dataset \citep{plank2020dan+} based on news texts from PAROLE-DK \citep{bilgram1998construction}, Finnish (UD Treebank; \citeauthor{haverinen2013tdt}, \citeyear{haverinen2013tdt}; \citeauthor{pyysalo2015udfinnish}, \citeyear{pyysalo2015udfinnish}; \citeauthor{kanerva-2022-ood}, \citeyear{kanerva-2022-ood}).
These datasets correspond to the NLP tasks of sequence classification\index{Classification!Sequence}, named entity recognition\index{Named entity recognition} and part-of-speech tagging\index{Part-of-speech tagging}, respectively.
An overview over the datasets is given in \cref{table:datasets}, with the preprocessing detailed in \cref{app:predictive-uncertainty-pre-processing}.
We use low-resource languages\index{Low-resource language} in the case of Finnish and Danish, and simulate a low-resource setting using English data.\footnote{
    The definition of low-resource actually differs greatly between works. 
    One definition by \citet{bird2022local} advocates the usage for (would-be) standardized languages with a large amount of speakers and a written tradition, but a lack of resources for language technologies. 
    Another way is a task-dependent definition: 
    For dependency parsing, \citet{eberstein2021genre} define low-resource as providing less than $5000$ annotated sentences in the Universal Dependencies Treebank. 
    \citet{hedderich2021survey,lignos2022toward} lay out a task-dependent spectrum, from a several hundred to thousands of instances.
} 
Starting with a sufficiently-sized training set and then sub-sampling allows us to create training sets of arbitrary sizes. 
%By using languages from different families, we hope to be able draw conclusions that generalize across a single language. 
We employ a specific sampling scheme that tries to maintain the sequence length and class distribution of the original corpus, which we explain and verify in \cref{app:exploring-predictive-uncertainty-training-set}.\\

\paragraph{Out-Of-Distribution Test Sets.} 
%While it is possible to create OOD text by for instance withholding classes from the training set or appending text from a different source \citep{arora2021types}, 
We create OOD\index{Out-of-distribution data} test sets from data sources that are qualitatively different from the in-distribution training data: 
Out-of-scope voice commands by users in \citet{larson2019evaluation},\footnote{
    Since all instances in this test set correspond to out-of-scope inputs and not to classes the model was trained on, we cannot evaluate certain metrics in \cref{table:predictive-uncertainty-results}.
} the Twitter split of the Dan+ dataset \citep{plank2020dan+}, and the Finnish OOD treebank \citep{kanerva-2022-ood}. 
In similar works for the image domain, OOD test sets are often chosen to be convincingly different from the training distribution, for instance MNIST versus Fashion-MNIST \citep{nalisnick2019do,van2021feature}. 
While there exist a variety of taxonomies for distributional shifts\index{Shift!Distributional}\citep{moreno2012unifying,wald2021calibration,arora2021types,federici2021information,hupkes2022state}, it is often hard to determine if and what kind of shift is taking place.
\citet{winkens2020contrastive} define \emph{near OOD}\index{Out-of-distribution data!Near} as a scenario in which the training and outlier distribution are meaningfully related, and \emph{far OOD}\index{Out-of-distribution data!Far} as a case in which they are unrelated. 
Unfortunately, this distinction is somewhat arbitrary and hard to apply in a language context, where OOD \emph{could} bde defined as anything ranging from a different language or dialect to a different demographic of an author or speaker or a new genre.
Therefore, we use a similar methodology to the validation of the sub-sampled training sets to make an argument that the selected OOD splits are sufficiently different in nature from the training splits. 
The exact procedure along some more detailed results is described in \cref{app:predictive-uncertainty-ood-test-set}.
Mainly, we examine the distribution of sequence lengths and labels, and score the OOD test set using the perplexity of a language model training on the training split.

\subsection{Model Training}\label{sec:model-training}

\begin{figure}[htb]
    \centering 
    \includegraphics[width=0.8\textwidth]{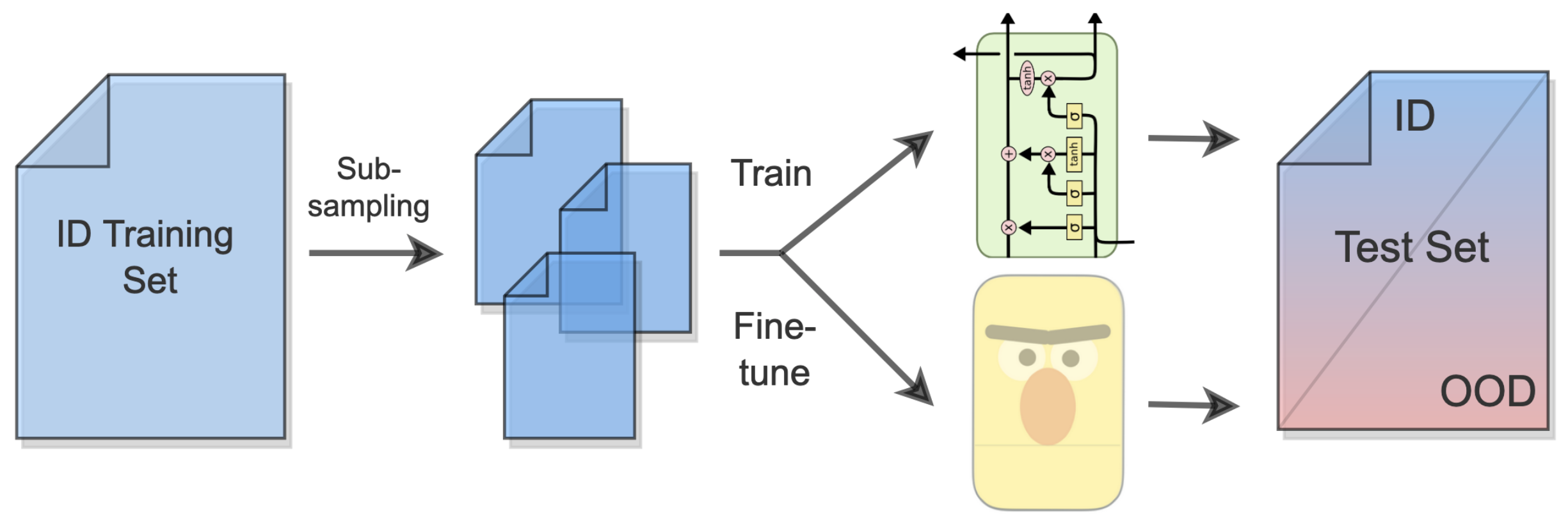}
    \caption[Schematic of text classification experiments.]{
        Schematic of our text classification experiments. 
        Training sets are sub-sampled and used to train LSTM-based models and fine-tune transformer-based ones, which are evaluated on in- and out-of-distribution test data.
        }\label{fig:training}
\end{figure}

Unfortunately, our datasets do not contain enough data to train transformer-based models from scratch\index{Transformer}. 
Therefore, we only fully train LSTM-based\index{Long-short term memory network} models, and use pre-trained transformers, namely Bert\index{Bert} (English; \citeauthor{devlin2019bert}, \citeyear{devlin2019bert}), Danish Bert \citep{hvingelby2020dane}, and FinBert (Finnish; \citealp{virtanen2019multilingual}), for the other approaches. 
The whole procedure is depicted in \cref{fig:training}. 
The way we optimize models as well as training hardware and hyperparameter information are listed in \cref{app:predictive-uncertainty-training-details}, with the environmental impact described in \cref{app:environmental-impact}.

\subsection{Evaluation}\label{sec:evaluation}

In addition to evaluating models on the task performance, we also evaluate the following calibration\index{Calibration} and uncertainty\index{Uncertainty}, painting a multi-faceted picture of the reliability of models. 
In all cases, we use the almost stochastic order test\index{Stochastic order!Almost} (ASO; \citealp{del2018optimal,dror2019deep}) as described in \cref{sec:aso} for significance testing.\index{Hypothesis testing}

\paragraph{Evaluation of Calibration.} 
First, we measure the calibration of models using the expected calibration error (ECE;\index{Expected calibration error} \citealp{naeini2015obtaining,guo2017calibration}), which we already discussed in \cref{sec:frequentist-neural-networks}.
In the same chapter, we introduced the frequentist measure of coverage\index{Coverage} \citep{larry2004all, kompa2021empirical}. 
Coverage here based on the (non-conformalized) prediction set of a classifier given an input, which includes the most likely classes adding up to or surpassing $1 - \alpha$ probability mass. 
A well-tuned classifier should contain the correct class in this prediction set\index{Prediction set}, while minimizing its width. 
The extent to which this property holds can be determined by the \emph{coverage percentage}, i.e.\@ the number of times the correct class in indeed contained in the prediction set, and its cardinality, denoted simply as \emph{width}. 

\paragraph{Evaluation of Uncertainty.} 
We compare uncertainty scores on the ID and OOD test set and measure the area under the receiver-operator curve (AUROC;\index{AUROC} evaluating the trade-off between sensitivity and specificity) and under the precision-recall curve (AUPR)\index{AUPR}, assuming that uncertainty will generally be higher on samples from the OOD test set.\footnote{
    We thus formulate a pseudo-binary classification\index{Classification!Binary} task as common in the literature, using the model's uncertainty score to try to distinguish the two test sets. 
    Note that we do not advocate for actually using uncertainty for OOD detection\index{OOD detection}, but only use it for evaluation purposes, since uncertainty on OOD examples should be high due to model uncertainty.
} 
An ideal model should create very different distributions of confidence scores on ID and OOD data, thus maximizing AUROC\index{AUROC} and AUPR\index{AUPR} (as opposed to the saturating confidence scores that we observed in \cref{sec:synthetic-experiments}).
However, we also want to find out to what extent uncertainty can give an indication of the correctness of the model, which is why we propose a new way to evaluate the \emph{discrimination} property proposed by \citet{alaa2020discriminative} based on \citet{leonard1992neural}: 
A good model should be less certain for inputs that incur a higher loss. 
To measure this both on a token and sequence level, we utilize Kendall's $\tau$\index{Kendall's $\tau$} \citep{kendall1938new}, which, given two lists of measurements, determines the degree to which they are \emph{concordant}---that is, to what extent the rankings of elements according to their measured values agree. 
This is expressed by a value between $-1$ and $1$, with the latter expressing complete concordance. 
In our case, these measurements correspond to the uncertainty estimate and the actual model loss, either for tokens (Token $\tau$) or sequences (Sequence $\tau$).

\subsection{Experiments}\label{sec:predictive-uncertainty-experiments}

%We will now address the initial research questions.

\begin{table}[htb!!]
    \centering 
    \resizebox{0.99\textwidth}{!}{
        \renewcommand{\arraystretch}{2}
        %\rowcolors{2}{gray!15}{white}
        \begin{tabular}{@{}ll@{\hspace{0.5cm}}cc@{\hspace{0.5cm}}ccc@{\hspace{0.5cm}}cccc@{}}
            \toprule
             & & \multicolumn{2}{c}{Task $\big( \text{ID} \big/ \text{OOD}\big)$} &  \multicolumn{3}{c}{Calibration $\big( \text{ID} \big/ \text{OOD}\big)$} & \multicolumn{4}{c}{Uncertainty $\big( \text{ID} \big/ \text{OOD}\big)$} \\
             \cmidrule(lr){3-4} \cmidrule(lr){5-7} \cmidrule(lr){8-11}
            & Model & Acc.$\uparrow$ & $F_1\uparrow$ & ECE$\downarrow$ & \%\@ Cov.$\uparrow$ &  $\varnothing$Width$\downarrow$ & \footnotesize AUROC$\uparrow$ &  \footnotesize AUPR$\uparrow$ & Token $\tau\uparrow$ & Seq. $\tau\uparrow$ \\
            \midrule
            %  $\rs{.98}{.46}{.58}{.38}$
             & LSTM & $\wsrs{.79}{.00}$ & $\wsrs{.62}{.01}$ & $\wsrs{.78}{.00}$ & $\wsrs{\cmbold{1.00}}{.00}$ & $\wsrs{144.00}{.00}$ & $\srs{.88^{\text{\tiny \Plus}}}{.01}$ & $\srs{.60^{\text{\tiny \Plus}}}{.01}$ & \backslashbox[15mm]{}{} & $\wsrs{.75^{\bigcirc}}{.01}$  \\
             % & Variational LSTM & & & & & & & & & & \backslashbox[15mm]{}{} & \\ 
             % & ST-$\tau$ LSTM & & & & & & & & & & \backslashbox[15mm]{}{} & \\ 
             & \makecell[cl]{Bayesian\\LSTM} & $\wsrs{.59}{.06}$ & $\wsrs{.46}{.05}$ & $\wsrs{.78}{.00}$ &  $\wsrs{.88}{.00}$ & $\srs{41.99}{1.94}$ & $\srs{.86^{\bigtriangleup}}{.01}$ & $\srs{.59^{\text{\tiny \XSolidBold}}}{.01}$ & \backslashbox[15mm]{}{} & $\wsrs{.66^{\bigcirc}}{.02}$  \\
             & \makecell[cl]{LSTM\\Ensemble} & $\wsrs{\cmbold{.81}}{.00}$ & $\wsrs{\cmbold{.64}}{.00}$ & $\wsrs{0.77}{.00}$ &  $\wsrs{.87}{.00}$ & $\wsrs{4.27}{.05}$ & $\srs{\cmbold{.92^{\text{\tiny \Plus}}}}{.00}$ & $\srs{\cmbold{.71^{\text{\tiny \Plus}}}}{.01}$ & \backslashbox[15mm]{}{} & $\wsrs{.73^{\Box}}{.01}$  \\
             & \makecell[cl]{Var.\\Bert} & $\wsrs{.45}{.16}$ & $\wsrs{.34}{.13}$ & $\wsrs{.78}{.00}$ &  $\wsrs{1.00}{.00}$ & $\wsrs{115.11}{11.38}$ & $\srs{.80^{\text{\tiny \XSolidBold}}}{.01}$ & $\srs{.53^{\text{\tiny \XSolidBold}}}{.01}$ & \backslashbox[15mm]{}{} & $\wsrs{.57^{\bigcirc}}{.09}$ \\ 
             % & SNGP Bert & & & & & & & & & & \backslashbox[15mm]{}{} &  \\ 
            \multirow{-5}{*}{\rotatebox{90}{\textbf{English}}}  & \makecell[cl]{DDU\\Bert} & $\wsrs{.79}{.00}$ & $\wsrs{.64}{.01}$ & $\wsrs{\cmbold{.77}}{.00}$ & $\wsrs{.82}{.00}$ & $\wsrs{\cmbold{1.46}}{.04}$ & $\srs{.88^{\bigcirc}}{.00}$ & $\srs{.62^{\bigcirc}}{.01}$ & \backslashbox[15mm]{}{} & $\wsrs{\cmbold{.87^{\bigcirc}}}{.00}$  \\[0.2cm]
            \cdashline{1-11} 
             & LSTM & $\rs{.93}{.00}{.92}{.00}$ & $\rs{.26}{.01}{.19}{.01}$ & $\rs{.17}{.00}{.17}{.00}$ & $\rs{\cmbold{1.00}}{.00}{\cmbold{1.00}}{.00}$ & $\rs{19.00}{.00}{19.00}{.00}$ & $\srs{.50^{\bigcirc}}{.02}$ & $\srs{.14^{\bigcirc}}{.01}$ & $\rs{.50^{\bigcirc}}{.01}{.47^{\bigcirc}}{.00}$ & $\rs{-.26^{\text{\tiny \Plus}}}{.02}{-.28^{\bigcirc}}{.05}$  \\
             & \makecell[cl]{Var.\\LSTM} & $\rs{.90}{.02}{.90}{.02}$ & $\rs{.08}{.02}{.09}{.02}$ & $\rs{.17}{.00}{.17}{.00}$ & $\rs{.99}{.01}{.98}{.01}$ & $\rs{6.62}{.37}{6.68}{.33}$ & $\srs{.60^{\text{\tiny \Plus}}}{.04}$ & $\srs{.21^{\text{\tiny \Plus}}}{.02}$ & $\rs{.23^{\bigcirc}}{.06}{.23^{\bigcirc}}{.05}$ & $\rs{-.04^{\text{\tiny \XSolidBold}}}{.02}{-.02^{\Box}}{.05}$  \\
             & \makecell[cl]{ST-$\tau$\\LSTM} & $\rs{.92}{.00}{.92}{.00}$ & $\rs{.12}{.00}{.09}{.00}$ & $\rs{.17}{.00}{.17}{.00}$ & $\rs{1.00}{.00}{.99}{.00}$ & $\rs{7.10}{.07}{7.03}{.08}$ & $\srs{.54^{\text{\tiny \Plus}}}{.01}$ & $\srs{.15^{\text{\tiny \Plus}}}{.01}$ & $\rs{.50^{\bigcirc}}{.00}{.48^{\bigcirc}}{.00}$ & $\rs{-.05^{\Box}}{.03}{-.01^{\Box}}{.05}$  \\
             & \makecell[cl]{Bayesian\\LSTM} & $\rs{.93}{.00}{.93}{.00}$ & $\rs{.07}{.00}{.07}{.00}$ & $\rs{.17}{.00}{.17}{.00}$ & $\rs{1.00}{.00}{1.00}{.00}$ & $\rs{1.68}{.04}{1.70}{.05}$ & $\srs{.65^{\pentagon}}{.17}$ & $\srs{.31^{\pentagon}}{.30}$ & $\rs{.53^{\bigcirc}}{.01}{\cmbold{.55^{\bigcirc}}}{.01}$ & $\rs{-.01^{\Box}}{.07}{-.02^{\text{\tiny \Plus}}}{.04}$  \\
             & \makecell[cl]{LSTM\\Ensemble} & $\rs{\cmbold{.95}}{.00}{\cmbold{.94}}{.00}$ & $\rs{\cmbold{.33}}{.01}{\cmbold{.25}}{.01}$ & $\rs{0.16}{.00}{\cmbold{0.16}}{.00}$ &  $\rs{.98}{.00}{.97}{.00}$ & $\rs{\cmbold{1.62}}{.00}{\cmbold{1.58}}{.01}$ & $\srs{.60^{\Box}}{.02}$ & $\srs{.18^{\Box}}{.01}$ & $\rs{.44^{\Box}}{.00}{.45^{\Box}}{.00}$ & $\rs{-.19^{\text{\tiny \Plus}}}{.01}{-.28^{\Box}}{.01}$  \\
             & \makecell[cl]{SNGP\\Bert} & $\rs{.22}{.35}{.19}{.34}$ & $\rs{.03}{.03}{.02}{.02}$ & $\rs{.17}{.00}{0.17}{.00}$ &  $\rs{1.00}{.00}{1.00}{.00}$ & $\rs{18.84}{.32}{18.83}{.34}$ & $\srs{.86^{\bigtriangleup}}{.06}$ & $\srs{.49^{\bigtriangleup}}{.12}$ & $\rs{.17^{\Box}}{.09}{.26^{\Box}}{.14}$ & $\rs{\cmbold{.29^{\text{\tiny \XSolidBold}}}}{.03}{\cmbold{.44^{\Box}}}{.11}$  \\
             & \makecell[cl]{Var.\\Bert}  & $\rs{.94}{.00}{.89}{.00}$ & $\rs{.29}{.01}{.17}{.00}$  & $\rs{\cmbold{0.16}}{.00}{0.16}{.00}$ &  $\rs{.99}{.00}{.98}{.00}$ & $\rs{2.25}{.01}{3.86}{.08}$ & $\srs{.86^{\text{\tiny \Plus}}}{.01}$ & $\srs{.46^{\text{\tiny \Plus}}}{.02}$ & $\rs{.42^{\bigcirc}}{.00}{.17^{\pentagon}}{.00}$ & $\rs{-.35^{\Box}}{.01}{-.41^{\Box}}{.01}$  \\
            \multirow{-8}{*}{\rotatebox{90}{\textbf{Danish}}}  & \makecell[cl]{DDU\\Bert} & $\rs{.92}{.00}{.89}{.00}$ & $\rs{.25}{.00}{.17}{.00}$ & $\rs{0.16}{.00}{0.16}{.00}$ & $\rs{.99}{.00}{.99}{.00}$ & $\rs{3.48}{.01}{4.04}{.03}$ & $\srs{.86^{\bigcirc}}{.01}$ & $\srs{.39^{\bigcirc}}{.02}$ & $\rs{\cmbold{.56^{\bigcirc}}}{.00}{.25^{\bigcirc}}{.01}$ & $\rs{-.24^{\bigcirc}}{.01}{-.38^{\bigcirc}}{.03}$  \\[0.2cm]
            \cdashline{1-11}
             & LSTM & $\rs{.75}{.00}{.69}{.00}$ & $\rs{.57}{.00}{.53}{.00}$ & $\rs{.07}{.00}{.07}{.00}$ & $\rs{1.00}{.00}{1.00}{.00}$ & $\rs{16.00}{.00}{16.00}{.00}$ & $\srs{.63^{\bigtriangleup}}{.01}$ & $\srs{.69^{\text{\tiny \Plus}}}{.01}$ & $\rs{.29^{\bigcirc}}{.00}{.19^{\bigcirc}}{.01}$ & $\rs{-.28^{\text{\tiny \Plus}}}{.02}{-.27^{\text{\tiny \Plus}}}{.02}$  \\
             & \makecell[cl]{Var.\\LSTM} & $\rs{.27}{.00}{.26}{.00}$ & $\rs{.03}{.00}{.03}{.00}$ & $\rs{.07}{.00}{.07}{.00}$ & $\rs{.97}{.00}{.96}{.00}$ & $\rs{1.35}{.23}{1.37}{.21}$ & $\srs{.51^{\text{\tiny \Plus}}}{.01}$ & $\srs{.59^{\text{\tiny \Plus}}}{.01}$ & $\rs{.00^{\bigtriangleup}}{.01}{.00^{\pentagon}}{.00}$ & $\rs{.01^{\bigtriangleup}}{.03}{.01^{\Box}}{.01}$  \\
             & \makecell[cl]{ST-$\tau$\\LSTM} & $\rs{.76}{.00}{.71}{.00}$ & $\rs{.58}{.00}{.55}{.00}$ & $\rs{.06}{.00}{.06}{.00}$ &  $\rs{.97}{.00}{.96}{.00}$ & $\rs{3.32}{.01}{3.57}{.01}$ & $\srs{.62^{\bigtriangleup}}{.01}$ & $\srs{.69^{\text{\tiny \Plus}}}{.01}$ & $\rs{.31^{\bigcirc}}{.00}{.21^{\bigcirc}}{.01}$ & $\rs{-.14^{\text{\tiny \Plus}}}{.02}{-.12^{\Box}}{.04}$  \\
             & \makecell[cl]{Bayesian\\LSTM} & $\rs{.27}{.00}{.26}{.00}$ & $\rs{.03}{.00}{.03}{.00}$ & $\rs{.07}{.00}{.07}{.00}$ &  $\rs{1.00}{.00}{1.00}{.00}$ & $\rs{16.00}{.00}{16.00}{.00}$ & $\srs{.51^{\pentagon}}{.01}$ & $\srs{.60^{\text{\tiny \XSolidBold}}}{.00}$ & $\rs{.00^{\pentagon}}{.00}{.00^{\pentagon}}{.00}$ & $\rs{.01^{\bigcirc}}{.01}{.04^{\text{\tiny \Plus}}}{.00}$  \\
             & \makecell[cl]{LSTM\\Ensemble} & $\rs{.81}{.00}{.75}{.00}$ & $\rs{.62}{.00}{.57}{.00}$ & $\rs{.06}{.00}{.06}{.00}$ &  $\rs{.99}{.00}{.98}{.00}$ & $\rs{3.46}{.01}{3.80}{.01}$ & $\srs{\cmbold{.67^{\text{\tiny \Plus}}}}{.01}$ & $\srs{\cmbold{.74^{\text{\tiny \Plus}}}}{.01}$ & $\rs{.29^{\bigcirc}}{.00}{.19^{\bigcirc}}{.01}$ & $\rs{-.28^{\text{\tiny \Plus}}}{.01}{-.31^{\text{\tiny \Plus}}}{.01}$  \\
             & \makecell[cl]{Var.\\Bert} & $\rs{.87}{.00}{.81}{.00}$ & $\rs{.74}{.00}{.70}{.00}$ & $\rs{.06}{.00}{.06}{.00}$ & $\rs{.99}{.00}{.99}{.00}$ & $\rs{4.68}{.03}{5.19}{.02}$ & $\srs{.64^{\bigtriangleup}}{.01}$ & $\srs{.70^{\bigcirc}}{.01}$ & $\rs{.14^{\bigcirc}}{.00}{.08^{\text{\tiny \Plus}}}{.00}$ & $\rs{-.19^{\text{\tiny \XSolidBold}}}{.00}{-.16^{\text{\tiny \XSolidBold}}}{.01}$  \\
             & \makecell[cl]{SNGP\\Bert}  & $\rs{.18}{.10}{.17}{.10}$ & $\rs{.07}{.02}{.08}{.02}$ & $\rs{.07}{.00}{.07}{.00}$ &  $\rs{1.00}{.00}{.99}{.01}$ & $\rs{15.00}{.00}{15.00}{.00}$ & $\srs{.54^{\bigtriangleup}}{.05}$ & $\srs{.63^{\bigtriangleup}}{.04}$ & $\rs{.15^{\Box}}{.04}{.15^{\Box}}{.03}$ & $\rs{\cmbold{.12^{\Box}}}{.05}{\cmbold{.14^{\Box}}}{.02}$ \\
            \multirow{-8}{*}{\rotatebox{90}{\textbf{Finnish}}} & \makecell[cl]{DDU\\Bert} & $\rs{.87}{.00}{.81}{.00}$ & $\rs{.72}{.03}{.68}{.03}$ & $\rs{\cmbold{.06}}{.00}{\cmbold{.06}}{.00}$ & $\rs{.94}{.00}{.91}{.00}$ & $\rs{\cmbold{2.16}}{.06}{\cmbold{2.31}}{.06}$ & $\srs{.61^{\bigcirc}}{.02}$ & $\srs{.69^{\bigcirc}}{.02}$ & $\rs{\cmbold{.39^{\bigcirc}}}{.04}{\cmbold{.26^{\bigcirc}}}{.03}$ & $\rs{-.07^{\bigcirc}}{.05}{-.16^{\bigcirc}}{.04}$  \\
            \bottomrule
        \end{tabular}%
        }
     \caption[Results for text classification experiments.]{
    Results on the tested datasets.
     Task performance is measured by macro $F_1$ and accuracy, calibration by different calibration errors, the coverage percentage the average prediction set width. 
     For every result, and value on the ID and OOD test set is shown. 
     For English, OOD scores are not available since the OOD set does not contain gold labels, and Token $\tau$ is missing due to CLINC being a sequence classification task. 
     Uncertainty quality is evaluated using its ability to discriminate between ID and OOD data, quantified by AUROC and AUPR. 
     Furthermore, Kendall's $\tau$ is measured between the uncertainty and losses on a sequence- and token-level. 
     Displayed are mean and standard deviation over five random seeds, with bolding and underlining indicating almost stochastic dominance with $\varepsilon_\text{min} \le 0.3$ over all other models. 
     For last section, the best value over uncertainty metrics is given, with symbols indicating the type of metric achieving it: 
     ${\bigcirc}$ Max.\@ probability, ${\bigtriangleup}$ Predictive entropy. ${\Box}$ Class variance. ${\pentagon}$ Softmax gap. ${\text{\tiny \Plus}}$ Dempster-Shafer. ${\text{\tiny \XSolidBold}}$ Mutual information.}\label{table:predictive-uncertainty-results}
     %${\text{\footnotesize \BigDiamondshape}}$: Log probability.
\end{table}

We present the results from our experiments using the largest training set sizes per dataset in \cref{table:predictive-uncertainty-results}.\footnote{
    For English, some models were omitted due to convergence issues, which are discussed in \cref{app:convergence-clinc-plus}.
} 

\paragraph{Task Performance.} 
Across datasets and models, we can identify several trends: 
some of the Bert-based models unsurprisingly perform better than LSTM-based\index{Long-short term memory network} models, which can be explained by the fact that they are pretrained on large datasets.
We observe worse performance for some LSTM and Bert-variants\index{Bert}, in particular the variational\index{Transformer!Variational}, Bayesian and ST-$\tau$ LSTM\index{Long-short term memory network!ST-$\tau$}, as well the SNGP Bert\index{Transformer!SNGP}. 
In accordance with the ML\index{Machine learning} literature (see e.g.\@ \citealp{lakshminarayanan2017simple, ovadia2019can}) and the discussions in \cref{sec:bayesian-neural-networks}, LSTM ensembles\index{Ensembling} actually perform very strongly and on par or sometimes better than fine-tuned Berts.

\paragraph{Calibration.} 
We also see that Bert\index{Bert} models generally achieve lower calibration\index{Calibration} errors across all metrics measured, which is in line with previous works \citep{desai2020calibration,dan2021effects}. 
It is interesting that the correct prediction is almost always contained in the $0.95$ confidence set\index{Prediction set} across all models, however these number have to be interpreted in the context of the set's width: 
It becomes apparent that for instance LSTMs achieve this coverage by spreading probability mass over many classes, while only Bert-based models, LSTM ensembles as well as the Bayesian LSTM\index{Long-short term memory network!Bayesian} (on Danish) and the Variational LSTM\index{Long-short term memory network!Variational} (on Finnish) are \emph{confidently} correct. 

\paragraph{Uncertainty Quality.} 
LSTM-based model seem to struggle to distinguish in- from out-of-distribution data based on predictive uncertainty. 
For Danish, only Berts\index{Bert} perform visibly above chance-level. 
For Finnish, the AUPR\index{AUPR} results suggest that although some OOD instances are quickly identified as uncertain, many other OOD inputs remain undetected among in-distribution samples. 
For English, OOD samples\index{OOD detection} are detected more effectively, which can be explained by them consisting of unknown voice commands, representing a potential instance of \emph{semantic} shift, which has been shown to be easier to detect by classifiers \citep{arora2021types}. 
Furthermore, it is striking that uncertainty and loss on a token-level (Token $\tau$) is only positive correlated for some models, using metrics such as the maximum probability score, softmax gap or the Dempster-Shafer metric, which are all entirely based on the categorical output distributions. 
On a sequence-level (Sequence $\tau$)\index{Kendall's $\tau$}, the correlation is often \emph{negative}, meaning that higher uncertainty goes hand in hand with a \emph{higher} loss.
This is the antithesis of the desired outcome and the opposite of the trend on the token-level, and suggests that few tokens-level scores distort the sequence-level aggregation of uncertainties.
Lastly, it should be noted that different uncertainty metrics yield diverse outcomes: 
There does not seem to be one superior metric across\index{Uncertainty metric} all experimental settings, as seen by the variety of markers shown in \cref{table:predictive-uncertainty-results}, which signify the best-performing uncertainty metrics per model and result.

\subsection{Dependence on Training Data}\label{sec:dependence-training-data}

\begin{figure}[htb]
    \centering
    \includegraphics[width=\textwidth]{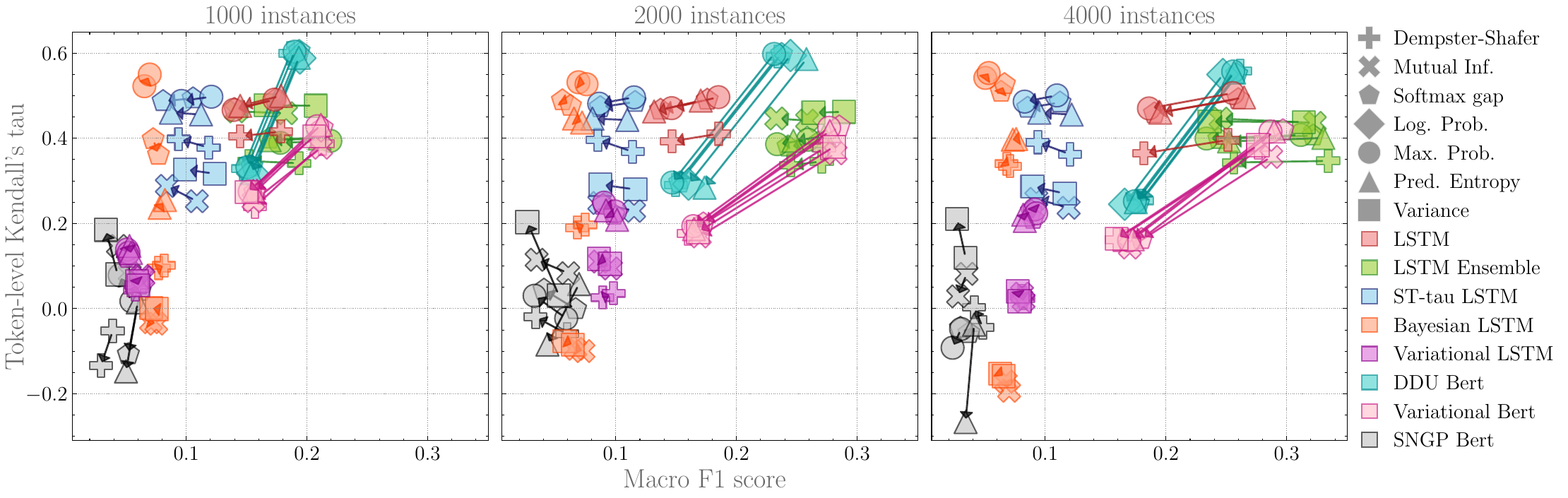}
    \caption[Scatter plots showing the connection between model performance and uncertainty quality on Dan+.]{
    Scatter plot showing the difference between model performance (measured by macro $F_1$  and the quality of uncertainty estimates on a token-level (measured by Kendall's $\tau$). 
    Shown are different models and uncertainty metrics and several training set sizes of the Dan+ dataset. 
    Arrows indicate changes between the in-distribution and out-of-distribution test set. 
    Best viewed electronically and in color.}\label{fig:scatter-plot-danplus-kendalls-tau-token}
\end{figure}

After presenting the best results for the biggest training set sizes in \cref{table:predictive-uncertainty-results}, we now continue to analyze the difference between models and metrics in a more fine-grained way. 
In \cref{fig:scatter-plot-danplus-kendalls-tau-token}, we show differences for the token-level correlation between a model's loss and its uncertainty measured by Kendall's $\tau$\index{Kendall's $\tau$}, with arrows indicating the shift from measurements on the in- to the out-of-distribution test set. 
Here, we see the same trend of more training data having a larger influence on Bert models\index{Bert}. 
Peculiarly, we also observe that the uncertainty of pre-trained models correlates less with their losses on the OOD data, while this property stays relative constant for LSTMs\index{Long-short term memory network}. 
We can recognize this trend also for the other datasets in \cref{fig:scatter-plot-danplus-kendalls-tau-token} and to a lesser degree on a sequence level \cref{subfig:clinc-plus-scatter-kendalls-tau-seq} in \cref{app:additional-scatters}, albeit with a \emph{negative} correlation in general in the latter case.
In \cref{fig:scatter-plot-auroc,fig:scatter-plot-aupr} in \cref{app:additional-scatters}, we show the AUROC and AUPR of different model-uncertainty metric combinations for all datasets and training set sizes. 
In both cases, we can notice that pre-trained models profit more from an increase in available training data than LSTM-based models that are trained from scratch. 
This improvement is observed both in task performance, as well as in the model's ability to discern ID from OOD data using its uncertainty, but more so for the Danish than English or Finnish. Like in the previous section, we often see that uncertainty metrics of the same model perform quite similarly. 
These results outline a seeming paradox: Pre-trained and then fine-tuned models (often) perform better on the task at hand, and provide better uncertainty estimates, but only on in-distribution data. 
Models trained from scratch that have seen less data overall, however provide more reliable uncertainty estimates on OOD data, but are also worse calibrated (\cref{sec:predictive-uncertainty-experiments}), with the exception of ensembles. 
This effect appears to largest on Danish, containing the least data. 

\subsection{Instance Analysis}\label{sec:qualitative-analysis}

\begin{figure}[htb]
    \centering
    \begin{subfigure}[t]{0.85\textwidth}
        \centering
        \includegraphics[width=\textwidth]{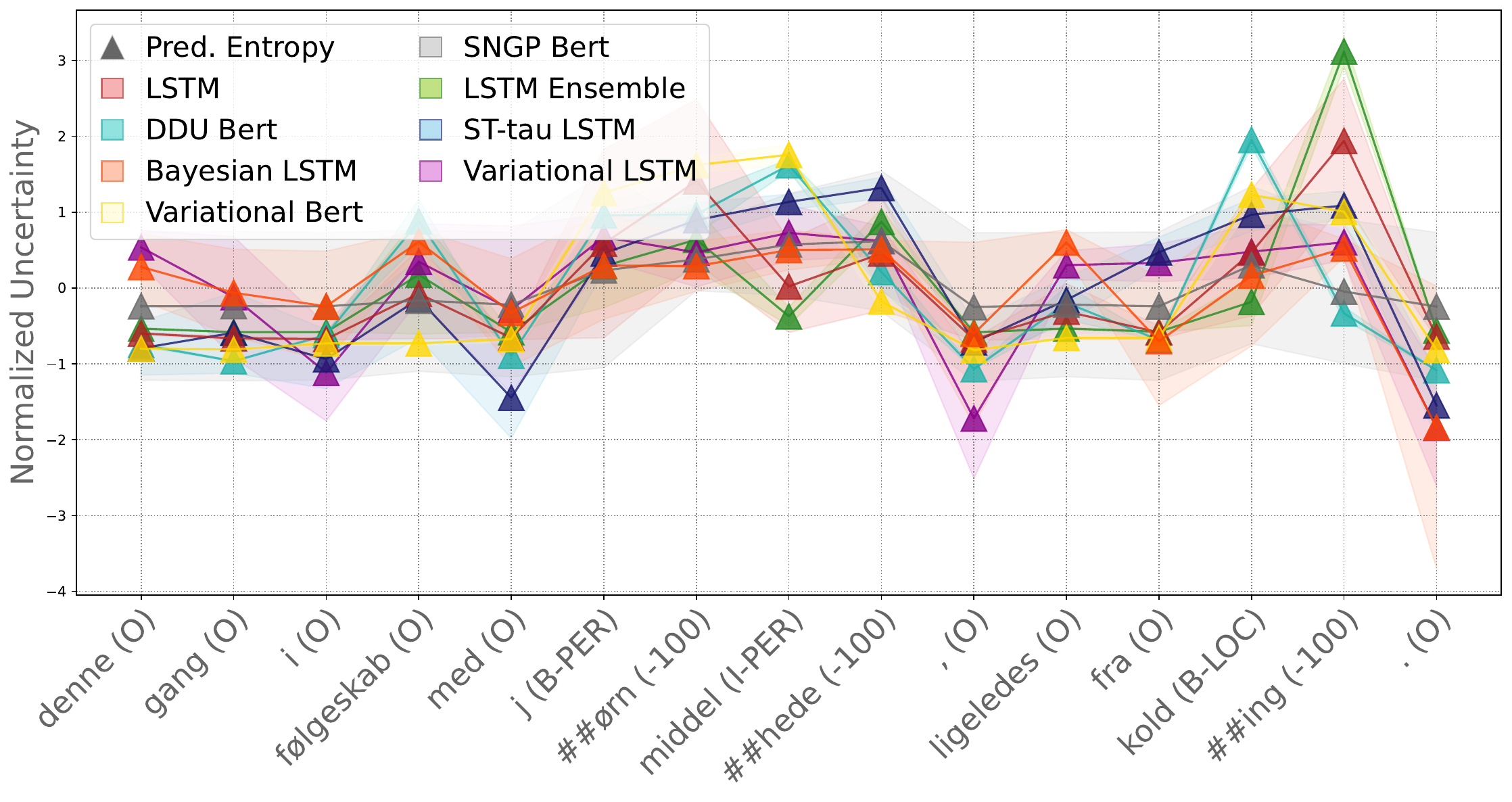}
        \subcaption{Predictive entropy over the sentence \emph{``This time in company with Jørn Middelhede, also from Kolding''}.}
    \end{subfigure}\label{subfig:qualitative-analysis-predictive-entropy-danplus}
    \vspace{0.5cm}
    \begin{subfigure}[t]{0.85\columnwidth}
        \centering
        \includegraphics[width=\textwidth]{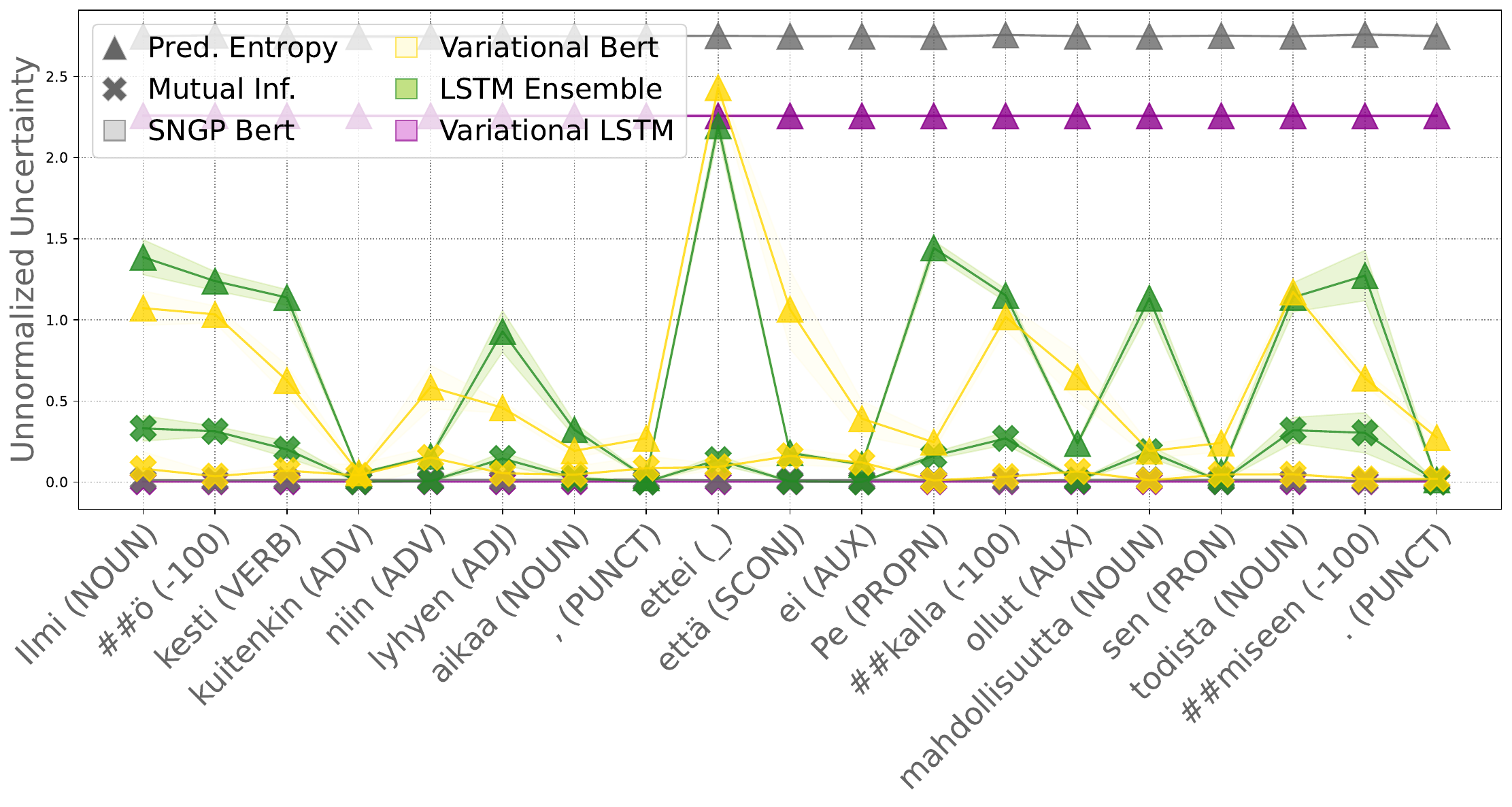}
        \subcaption{Predictive entropy and mutual information over the sentence \emph{``However, the phenomenon lasted for such a short time that Pekka did not have a chance to prove it''}.}\label{subfig:qualitative-analysis-mutual-information-finnish-ud}
    \end{subfigure}
    \caption[Predictive entropy estimates for a single sentence from Dan+ and mutual information for Finnish UD.]{Uncertainty estimates on single sequences, for (a) predictive entropy of different models on Danish and (b) predictive entropy and mutual information for multi-prediction models on Finnish (\cref{subfig:qualitative-analysis-mutual-information-finnish-ud}).}\label{fig:qualitative-analysis}
\end{figure}

We investigate the development of uncertainty estimates over the course of a single sequence for different datasets, models, and uncertainty metrics\index{Uncertainty metric}. 
Two examples are showcased in \cref{fig:qualitative-analysis}, with more examples in \cref{app:qualitative-analysis}. 
By looking at the predictive entropy of models in \cref{subfig:qualitative-analysis-predictive-entropy-danplus}, we can observe multiple things: 
First of all, we can observe some degree of agreement between models and their uncertainty: 
Uncertainty\index{Uncertainty} is higher for subword tokens, and the total uncertainty\index{Uncertainty!Total} always appears to reduce considerably on punctuation. 
Interestingly, the highest uncertainty seems to be produced by the DDU\index{Transformer!DDU} and variational Bert\index{Transformer!Variational} models as well as the ensembles\index{Ensembling}. 
In \cref{subfig:qualitative-analysis-mutual-information-finnish-ud}, we compare the estimates for predictive entropy\index{Entropy!Predictive} and mutual information\index{Mutual information}, the latter of which is supposed to only express model uncertainty\index{Uncertainty!Epistemic}. 
Here, uncertainty is generally low, indicating a large part of the total uncertainty might actually be of an aleatoric nature\index{Uncertainty!Aleatoric} (which is the gap between triangle and cross markers of the same color, due to \cref{eq:mutual-information}). 
These insights indicate that while aleatoric uncertainty might be a constant factor for all models, epistemic uncertainty expectedly differs noticeably between them. 
We use all of these insights to discuss the choice of model next.

\subsection{Discussion}\label{sec:discussion}

Our experiments in the previous sections have uncovered interesting nuances about uncertainty quantification\index{Uncertainty quantification} in text classification\index{Classification!Text}.
With respect to the first research question, we observed that fine-tuning Berts\index{Bert} and training LSTM\index{Long-short term memory network} ensembles\index{Ensembling} on different languages produces high task scores with low calibration errors and high-quality uncertainty estimates, but only on in-distribution data. 
On OOD data\index{Out-of-distribution data}, uncertainty estimates from fine-tuned models actually become less indicative of potential model loss compared to LSTM-based models. 
We also find that among the variety of uncertainty metrics proposed, there does not appear to be a superior metric, i.e.\@ most able to hint at mispredictions and OOD data.
Differences in Kendall's $\tau$\index{Kendall's $\tau$} on a token and sequence level suggest that loss and uncertainties fluctuate over the course of sequence.\\

Answering the second research question, more training data paradoxically decreases the quality of uncertainty estimates on OOD data for pre-trained models. 
We speculate that fine-tuning models increasingly lets them forget relevant features that would produce higher uncertainty. 
This might explain why for this effect is smaller for LSTM-type models, which are trained from scratch.\\

Lastly, we conclude about the third research question that all the total uncertainty of models behaves somewhat similarly, potentially due to the strong influence of aleatoric uncertainty\index{Uncertainty!Aleatoric}. 
From these insights, we summarize that the approaches using pre-trained models overall give the best trade-off between task performance, uncertainty quality and calibrations, however their failure on OOD samples opens up further directions of research.
Ensembles can provide an alternative here in data-scarce settings\index{Data scarcity}, when the task is sufficiently learnable without the need for pre-training.

\paragraph{Limitations.} 
Even though the experiments test a large array of models and metrics, the collection here shown is by no means exhaustive, and only a selection of popular models or approaches from very different families were considered. %Furthermore, in some cases we observed difficulties for some models to converge to good solutions over the training (\extodo{Name examples}), which might influence some of the discussed results.
Another glaring shortcoming is the focus on only three European languages: 
By comparing members of the Uralic, North Germanic and West Germanic families, we only scratch the surface when it comes to the morphological diversity of human language, as for instance illustrated in \cref{subfig:wikipedia-articles}.
Further, we only focused on languages with a Latin writing systems, as well as specific text domains and tasks. 
This is due to resource constraints and the availability of suitable OOD test sets. 
We hope that follow-up works will refine our insights on a more representative sample of natural languages.

\section{Summary}

This chapter explored some perspectives on uncertainty quantification in classification\index{Classification}\index{Uncertainty quantification}.
\cref{sec:uq-classification-pitfalls} demonstrated how the inductive bias of ReLU\index{ReLU} networks produces uncertainty estimates that are not indicative of the familiarity of data to the model;
instead, they converge to fix points in the limit. 
We were able to prove this formally and get an intuition of potential pitfalls in practice in \cref{sec:synthetic-experiments}, however text classification\index{Classification!Text} models in NLP possess different and more complex architecture, for which similar arguments are not easily applicable.
Therefore, we followed up with an empirical investigation into many popular models and uncertainty metrics\index{Uncertainty metric} on three different languages and tasks in \cref{sec:benchmarking-nlp-uncertainty}.
This came with some surprising insights, for instance that uncertainty can be unreliable on OOD data, and that more training or finetuning data can lead to decreased uncertainty quality.\\

As we argued in the introduction in \cref{sec:context}, data scarcity and the complexity of language are two core features that differentiate uncertainty quantification in NLP from other modalities.
In this chapter we discussed the arguably easier setting of text classification\index{Classification!Text}:
In text classification, we can treat predictions on a sequence-level as i.i.d., and the set of classes is usually much smaller than the number of tokens in a vocabulary.
This is why we now turn our attention to the more challenging problem of language generation\index{Natural language generation} in the next chapter.

\chapter{Uncertainty in Natural Language Generation}\ % Chapter title

\label{ch:uncertainty-generation} % For referencing the chapter elsewhere, use \autoref{ch:name} 

\begin{tikzpicture}[remember picture,overlay]
    \node[anchor=north,inner sep=0pt] at (current page text area.north) {\includegraphics[width=\linewidth, clip=true, trim = 8cm 85cm 8cm 40cm]{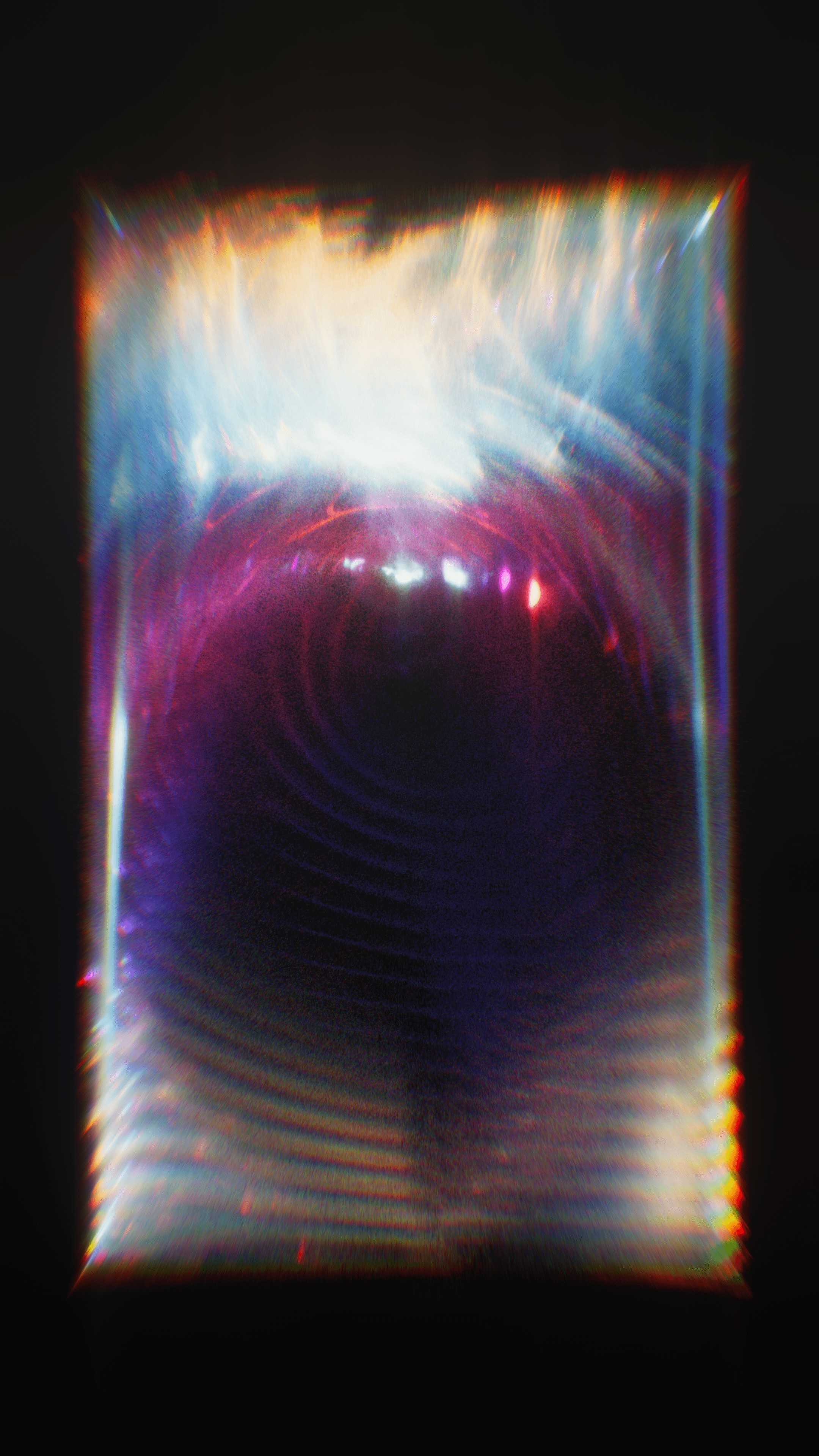}};
\end{tikzpicture}

\epigraph{``\emph{Obviously, a computer program that succeeded in generating sentences of a language would be, in itself, of no scientific interest unless it also shed some light on the kinds of structural features that distinguish languages from arbitrary, enumerable sets.}''}{---Noam Chomsky in \emph{Formal properties of grammars} (1963).}

Natural language generation (NLG)\index{Natural language generation} is a multi-faceted field spanning applications such as machine translation (MT)\index{Machine translation}, language modeling (LM)\index{Language modeling}, summarization\index{Text summarization}, question-answering\index{Question-answering} and dialogue generation. 
Owing to the recent success of large language models (LLMs)\index{Large language model} such as GPT-4 \citep{openai2023gpt4}, Bloom \citep{scao2022bloom} or Llama \citep{touvron2023llama}, natural language is increasingly used as an interface for end users to interact with models. 
In order to generate the tokens in a sentence, models typically predict a distribution over subword tokens at every step of the generation process.
Due to the paraphrastic nature of language discussed in \cref{sec:uncertainty-linguistics}, there is a large uncertainty about which token to select, since there might not be a single ``correct'' token.
Futhermore, just using the most likely token often results in text of low-quality \citep{holtzman2020curious, see2019massively, eikema2020map, zhang2020trading,  eikema2024effect}.
For this reason, this uncertain decision is often realized through specialized sampling procedures.
However, it has been shown that sampling from the tail of the token distribution also negatively impacts text quality, which is why token distributions are often truncated in practice \citep{holtzman2020curious, fan2018hierarchical, meister2023locally}.
While this kind of sampling allows for more fluent and varied text, there are no guarantees about the plausibility of the generated text. %few methods exist to evaluate the reliability of generated text and adequacy of the underlying sampling method. 
This is particularly relevant for generation scenarios where pre-trained models are applied to new data whose distribution is different from the training data, increasing the risk of generating erroneous, misleading, and potentially harmful text \citep{ji2023survey, guerreiro2023looking, pan2023risk, alkaissi2023artificial, azamfirei2023large}.
Therefore, this chapter introduces a way of creating calibrated prediction sets\index{Prediction set} to sample from for natural language generation\index{Natural language generation}, imbued with the guarantees of conformal prediction\index{Conformal prediction}.

\section{Conformalizing Natural Language Generation}

\begin{footnotesize}
    \vspace{-2.5ex}
    \emph{The following work is based on \citet{ulmer2024non}}.\\
    \vspace{2.5ex}
\end{footnotesize}

\begin{figure}[tb!]
    \centering 
    \includegraphics[width=0.85\textwidth]{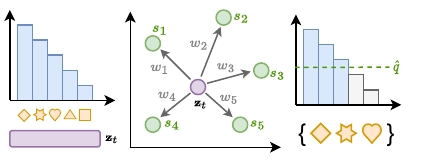}
    \caption[Schematic representation of non-exchangeable conformal language generation through nearest neighbors.]{
    Schematic representation of our approach. 
    A decoder hidden representation $\bz_t$ is used during inference to retrieve the nearest neighbors and their non-conformity scores $s_k$. 
    Their relevance is determined by using their distance to compute weights $w_k$, resulting in the quantile $\hat{q}$ that forms conformal prediction sets.
    }\label{fig:schematic}
\end{figure}

Conformal prediction\index{Conformal prediction} \citep{vovk2005algorithmic, papadopoulos2002inductive, angelopoulos2021gentle}, has recently gained popularity by providing calibrated prediction sets that are equipped with statistical guarantees about containing the correct solution (see for instance the introduction in \cref{sec:frequentist-neural-networks}). 
Nevertheless, applying conformal prediction to NLG\index{Natural language generation} is not trivial: 
The autoregressive generation process breaks the independence and identical distribution (i.i.d.)\@ assumption underlying conformal prediction techniques, since new predictions are conditioned on the sequence generated so far.
We tackle this problem by drawing inspiration from recent advances in nearest-neighbor language modeling\index{Language modeling!Nearest-neighbor} \citep{khandelwal2020generalization, he2023efficient, xu2023nearest} and machine translation \citep{khandelwal2020nearest, zheng2021adaptive, meng2022fast, henrique2022chunk}. 
This way, we can dynamically generate calibration sets during inference that maintain statistical guarantees. 
We schematically illustrate non-exchangeable conformal nucleus sampling\index{Nucleus sampling!Non-exchangeable conformal} in \cref{fig:schematic}:
In the first step, we obtain a (sorted) probability distribution over tokens and a latent representation $\bz_t$ for the current generation step from the model.
In a second step, we use the latent representation to query a datastore for similar, previously stored representations and their corresponding non-conformity scores\index{Non-conformity score}, $s_i$.
In the same way as in the standard conformal prediction algorithm, these non-conformity scores indicate how much a prediction conforms to the rest of the calibration set and its difficulty for the model.
These scores are then used to compute a threshold $\hat{q}$ based on the theory of non-exchangeable conformal prediction\index{Conformal prediction!Non-exchangeable} \citep{barber2023conformal}, which defines a smaller set of tokens that is sampled from.\footnote{
    For simplicity, the figure depicts the simplest form of prediction sets\index{Prediction set} used in conformal prediction. In practice, we use the adaptive prediction\index{Prediction set!Adaptive} sets explained in \cref{sec:nonex-nlg-method}.
}
The extension by \citeauthor{barber2023conformal} allows us to compensate a lack of i.i.d.\@ data by instead defining relevance weights between the test point and the calibration set.

\paragraph{Contributions.} 
We present a general-purpose extension of the conformal framework to NLG\index{Natural language generation} by tackling the problems above. 
Our contributions are as follows: 
First, to the best of our knowledge, we are the first to present a novel technique based on \textit{non-exchangeable} conformal prediction\index{Conformal prediction!Non-exchangeable} and to apply it to language generation to produce calibrated prediction sets using a theoretically sound motivation. 
Secondly, we validate the effectiveness of the method in a language modeling\index{Language modeling} and machine translation\index{Machine translation} context, evaluating the coverage\index{Coverage} of the calibrated prediction sets and showing that our method is on par with or even outperforms other sampling-based techniques in terms of generation quality, all while maintaining tighter prediction sets and better coverage. 
Lastly, we demonstrate that these properties are also maintained under distributional shift\index{Shift!Distributional} induced by corrupting the model's latent representations. 
% Lastly, we publish all the code for this project in an open-source repository.\footnote{\url{https://github.com/Kaleidophon/non-exchangeable-conformal-language-generation}.}

\section{Background}\label{sec:background}

We already discussed the basic formulation of conformal prediction\index{Conformal prediction} in \cref{sec:frequentist-neural-networks}:
We first define a non-conformity score\index{Non-conformity score} that provides an estimate of the distance of the test point to the rest of the data.
Then, we determine $\hat{q}$ as the $\big\lceil(N+1)(1 - \alpha)/N\big\rceil$-th quantile of the non-conformity scores on a held-out set.
Finally, we can create calibrated prediction sets of the form 

\begin{equation}\label{eq:conformal-prediction-sets}
    \mathcal{C}(\bx^\prime) = \big\{y\ \big|\ P_{\btheta}(y \mid \bx^\prime) \ge 1 - \hat{q}\big\}.
\end{equation}

Here, $\bx^\prime$ is a new test point for which we could like to construct a prediction set.
If a test point $\bx^\prime$ and the calibration set are i.i.d.\@, then this set fulfils the conformal guarantee 

\begin{equation}\label{eq:conformal-prediction-guarantees}
    p\big(y^\prime \in \mathcal{C}(\bx^\prime)\big) \ge 1 - \alpha.
\end{equation}

Nevertheless, this formulation is not directly applicable to NLG, as autoregressive generation violates the i.i.d.\@ assumption:
If we compare the token distributions at different time steps and different sequences, they will hardly be comparable.

\paragraph{Non-exchangeable Conformal Prediction.} 
\citet{barber2023conformal} address this shortcoming: 
When a test point and the calibration data are not i.i.d.,%
\footnote{
    In fact, the coverage guarantee in \cref{eq:conformal-prediction-guarantees} applies to the case where the data is \textit{exchangeable}, a weaker requirement than i.i.d. 
    Specifically, a series of random variables is exchangeable if their joint distribution is unaffected by a change of their order.
    The work by \citet{barber2023conformal} allows us to also forgo this requirement.
} %
the distributional drift\index{Shift!Distributional} causes any previously found $\hat{q}$ to be miscalibrated, so the intended coverage bound of $1 - \alpha$ can no longer be guaranteed. 
However, we can still perform conformal prediction by assigning a weight $w_i \in [0, 1]$ to every calibration data point, reflecting its relevance---i.e.\@ assigning lower weights to points far away from the test distribution. 
Then, by normalizing the weights with $\tilde{w}_i = w_i / (1 + \sum_{i=1}^N w_i)$, we define the quantile as

\begin{equation}\label{eq:non-exchangeable-quantile}
    \hat{q} = \inf \big\{q\ \big| \sum_{i=1}^N \tilde{w}_i \indicator{s_i\le q} \ge 1 - \alpha \big\}.
\end{equation}

The construction of the prediction sets then follows the same steps as before.
Most notably, the coverage guarantee\index{Coverage} in \cref{eq:conformal-prediction-guarantees} now changes to 

\begin{equation}\label{eq:non-exchangeable-conformal-prediction-guarantees}
    p\big(y^\prime \in \mathcal{C}(\bx^\prime)\big) \ge 1 - \alpha - \sum_{i=1}^N \tilde{w}_i\varepsilon_i,
\end{equation}

\noindent with an extra term including the \textit{total variation distance} ($d_\text{TV}$)\index{Total variation distance} between the distribution of a calibration and a test point, $\varepsilon_i = d_\text{TV}\big((\bx_i, y_i), (\bx^\prime, y^\prime)\big)$.%
\footnote{In this expression, $(\bx_i, y_i)$ and $(\bx^\prime, y^\prime)$ denote random variables and the total variation distance is between the two underlying distributions. See \citet{barber2023conformal} for details.} %
Unfortunately, this term is hard to estimate or bound, nevertheless, the selection of appropriate weights that captures the relevance of calibration points to the test set should moderate both the impact of the distant data points on the estimation of the prediction set and the impact of $d_{\mathrm{TV}}$ on the coverage bound. 
In other words, for large $d_{\mathrm{TV}}$ values we expect to have smaller weights, that allow us to achieve coverage close to the desired values.
We show in our experiments that the loss of coverage when using weights derived from the distance to nearest neighbor is limited, and revisit the practical implications in \cref{sec:discussion}.

\section{Method}\label{sec:nonex-nlg-method}

We now present a novel method to apply conformal prediction\index{Conformal prediction} in NLG\index{Natural language generation} by synthesizing the non-exchangeable approach of \citet{barber2023conformal} with $k$-NN search-augmented neural models \citep{khandelwal2020nearest, khandelwal2020generalization}.
In the latter case, the token distribution at the current generation step is interpolated with the predictive distributions of nearest neighbors in a datastore.\\

A related approach for conformal prediction for NLG by \citet{ravfogel2023conformal} calibrates prediction sets using the standard conformal procedure described in \cref{sec:background}.
In order to improve its effectiveness, the authors also determine multiple $\hat{q}$ values based on the entropy of the token distribution, grouping inputs into one of multiple bins.
However, this implies that we would use semantically unrelated (sub-)sequences to calibrate the model---in fact, we show experimentally that this approach generally obtains trivial coverage by producing extremely wide prediction sets. 
Instead, we propose to perform a \emph{dynamic} calibration step during model inference, only considering the most relevant data points from the calibration set. 
We do this in the following way: 
Given a dataset $\{(\bx^{(i)}, y^{(i)})\}$ of sequences $\bx^{(i)} = (\bx_1^{(i)}, \ldots, \bx_S^{(i)})$ and corresponding references consisting of gold tokens $y^{(i)} = (y_1^{(i)}, \ldots, y_T^{(i)})$, we extract the model's decoder activations $\bz_t^{(i)} \in \mathbb{R}^d$ and conformity scores $s_t^{(i)}$.\footnote{
    In this phase, we do not let the model generate freely, but feed it the gold prefix during the decoding process to make sure that conformity scores can be computed correctly.} 
We save those in an optimized datastore, allowing for fast and efficient nearest-neighbor search using the FAISS method by \citet{johnson2019billion} through techniques such as quantization and GPU acceleration.
In the inference phase, during every decoding step, we then use the decoder hidden state $\bz_t^\prime$ to query the data store for the $K$ nearest neighbors and their non-conformity scores\index{Non-conformity score} and record their distances. 
We use the squared $l_2$ distance to compute the weight $w_k$ as 

\begin{equation}\label{eq:weight-equation}
    w_k = \exp\big(-\big|\big|\bz_t - {\bz_k}\big|\big|^2_2\ /\ \tau\big),
\end{equation}

\noindent where $\tau$ corresponds to a temperature hyperparameter.\footnote{Using this formulation of the weights $w_k$ that depends on the data deviates from the assumptions of original proof, as discussed in \citet{barber2023conformal}, section 4.5.
Nevertheless, our results in \cref{sec:experiments} and those by  \citet{farinhas2024nonexchangeable} show that the obtained bound in \cref{eq:non-exchangeable-conformal-prediction-guarantees} still remains useful.}
This formulation is equivalent to a radial basis function kernel with scale parameter $\tau$.
Finally, we use the weights to compute the quantile $\hat{q}$ as in \cref{eq:non-exchangeable-quantile}.
% Given that we often will find no perfect match for the query, it has to be noted that there is likely some loss in coverage due to mismatch between the query distribution and the neighbor distributions in \cref{eq:non-exchangeable-conformal-prediction-guarantees}.
% However, we empirically validate in our experiments  in \cref{sec:experiments} that this loss in coverage is limited.
The entire algorithm is given in \cref{alg:cap}.

\begin{algorithm}[htb]
    \caption{Non-exchangeable Conformal Language Generation with Nearest Neighbors}\label{alg:cap}
    \begin{algorithmic}
    \Require Sequence $\bx$, model $f_{\btheta}$, datastore $\text{DS}(\cdot)$ with model activations collected from held-out set, temperature $\tau$
    \State
    \While{generating}
        \LineComment{1. Extract latent encoding for current input}
        \State $\bz_t \gets f_{\btheta}(\bx_t; y_{<t})$\\
        \LineComment{2. Retrieve $K$ neighbors \& non-conformity scores}
        \State $\{(\bz_1, s_1), \ldots (\bz_K,  s_K)\} \gets \text{DS}
        (\bz_t)$\\
        \LineComment{3. Compute weights $w_k$ and normalize}
        \State $w_k \gets\ \exp(-||\bz_t - {\bz_k}||^2_2\ /\ \tau)$
        \State $\tilde{w}_k\ \leftarrow\ w_k / (1 + \sum_{k=1}^K w_k)$\\
        
        \LineComment{4. Find quantile $\hat{q}$}
        \State $\hat{q} \gets\ \inf \{q \mid \sum_{i=1}^N \tilde{w}_i \indicator{s_i\le q } \ge 1 - \alpha \}$\\
    
        \LineComment{5. Create prediction set}
        \State $\hat{c} \gets\ \sup \{c^\prime \mid \sum_{j=1}^{c^\prime} P_{\btheta}(y=\pi(j) \mid \bx_t, y_{<t}) < \hat{q}\} + 1$
        \State $\mathcal{C}(\bx_t) \gets\ \{\pi(1), \ldots, \pi(\hat{c})\}$\\
    
        \LineComment{6. Generate next token}
        \State $y_t\ \leftarrow\ \text{generate}(\mathcal{C}(\bx_t))$
        
    \EndWhile
    
    \end{algorithmic}
    \end{algorithm}

\paragraph{Adaptive Prediction Sets.} \index{Prediction set!Adaptive}
The efficacy of conformal prediction hinges on the choice of non-conformity score, with the simple non-conformity score\index{Non-conformity score} $s_i = 1 - P_{\btheta}(y_t \mid \bx, y_{<t})$ known to undercover hard and overcover easy subpopulations of the data \citep{angelopoulos2021gentle}.
Due to the diverse nature of language, we therefore opt for \emph{adaptive prediction sets} \citep{angelopoulos2021uncertainty, romano2020classification}. 
Adaptive prediction sets redefine the non-conformity score as the cumulative probability over classes (after sorting in descending order) necessary to reach the correct class.
Intuitively, this means that we include all classes whose cumulative probability does not surpass $\hat{q}$.
Compared to the simple conformity score, this produces wider predictions sets for hard inputs, encompassing more potentially plausible continuations in a language context.
More formally, let $\pi$ be a permutation function mapping all possible output tokens $[C]$ to the indices of a permuted version of the set, for which tokens are sorted in descending oder by their probability under the model. 
We define the non-conformity score as 

\begin{equation}\label{eq:adaptive-non-conformity-score}
        s_i = \sum_{j=1}^{\pi(y_t)} P_{\btheta}(\pi^{-1}(j) \mid \bx, y_{<t}).
\end{equation}

Since we only include the cumulative mass up until the gold label, the summation stops at $\pi(y)$. 
The prediction sets are then defined as 

\begin{align}\label{eq:adaptive-prediction-set}
    \mathcal{C}(\bx, y_{<t}) & = \Big\{\pi^{-1}(1), \ldots, \pi^{-1}(\hat{c})\Big\},
\end{align}
with $\hat{c} = \sup \{c^\prime \mid \sum_{j=1}^{c^\prime} P_{\btheta}(\pi^{-1}(j) \mid \bx, y_{<t}) < \hat{q} \} + 1$, where we add one extra class to avoid empty sets.

\section{Experiments}\label{sec:experiments}

In the following sections, we conduct experiments in both language modeling\index{Language modeling} and machine translation\index{Machine translation}.
For machine translation we opt for the 400 million and 1.2 billion parameter versions of the M2M100 model \citep{fan2021beyond} on the WMT-2022 shared task datasets for German to English and Japanese to English \citep{kocmi2022findings}.
For language modeling, we use the 350 million and 1.3 billion parameter versions of the OPT model \citep{zhang2022opt} and replicate the setup by \citet{ravfogel2023conformal}:
We calibrate our model on $10000$ sentences from a 2022 English Wikipedia dump \citep{wikidump} and test coverage and generation on $1000$ sentences from OpenWebText \citep{Gokaslan2019OpenWeb}.\footnote{Data obtained through the Hugging Face \texttt{datasets} package \citep{lhoest2022datasets}: \url{https://huggingface.co/datasets/wikipedia} and \url{https://huggingface.co/datasets/stas/openwebtext-10k}.}
All models are used in a zero-shot setup \emph{without extra training or finetuning}.
For the datastore, we use the implementation of the FAISS library \citep{johnson2019billion}, computing $2048$ clusters in total and probing $32$ clusters per query. 
We also summarize the environmental impact of our experiments in \cref{app:environmental-impact}.

\subsection{Evaluating Coverage}\label{sec:retrieval-quality}

First of all, we demonstrate that the retrieved information from the data store enables us to successfully obtain calibrated prediction sets. 
\emph{Coverage}\index{Coverage} is an important notion in conformal prediction\index{Conformal prediction}, referring to the correct label being included in a prediction set or interval. 
Since we can always achieve coverage trivially by choosing the largest possible prediction set, an ideal method strikes a balance between high coverage and small prediction sets.
While it is not possible to measure coverage in a free generation setting (see next section), we can assess whether the correct class is contained in the prediction set if we feed the actual reference tokens into the decoder and check whether we include the true continuation.\footnote{We emphasize that access to gold tokens is not required by our method and only done here to measure the actual coverage.}
For our MT task, this is reminiscent of an interactive translation prediction\index{Machine translation!Interactive} setup \citep{knowles2016neural, peris2017interactive, knowles2019user}, where we propose possible continuations to a translator, suggesting the next word from a set of words that (a) contains plausible options and (b) is limited in size, in order to restrict the complexity for the end user.
Before we run our experiments, we need to determine $\tau$, which we tune on the calibration set using a stochastic hill-climbing procedure described in \cref{app:temperature-search}.  
We compare our \emph{non-exchangeable conformal nucleus sampling} (\emph{Non-Ex. CS})\index{Nucleus sampling!Non-exchangeable conformal} with the following sampling methods:
Nucleus sampling (\emph{Nucleus}; \citealp{holtzman2020curious})\index{Nucleus sampling}, which includes all tokens up to a pre-defined cumulative probability mass, and the conformal nucleus sampling\index{Nucleus sampling!Conformal} (\emph{Conf.};~\citealp{ravfogel2023conformal}) discussed earlier. 
The latter bins predictions on a calibration set by the entropy of the output distribution, and compute one $\hat{q}$ per such entropy bin using the standard conformal procedure given in the beginning of \cref{sec:background}. %using $10$ entropy bins and corresponding $\hat{q}$ values.
\looseness=-1

\paragraph{Evaluation.} We measure the total coverage using different distance metrics, namely, squared $l_2$ distance, normalized inner product, and cosine similarity (see \cref{tab:coverage-results-mt,tab:coverage-results-lm}),\footnote{For inner product and cosine similarity, we follow the same form as \cref{eq:weight-equation}, omitting the minus. We normalize the inner product by the square root of the latent dimension.} as well as binning predictions by set size and then measuring the per-bin coverage in \cref{fig:conditional-coverage} (more results given in \cref{app:coverage-experiments}).
We also summarize the plots in \cref{fig:conditional-coverage} via the \emph{expected coverage gap}\index{Expected coverage gap} (ECG)\footnote{This is inspired by the expected calibration error \citep{guo2017calibration}, comparing coverage to $1 - \alpha$, where overcoverage is not penalized due to \cref{eq:conformal-prediction-guarantees}'s lower bound.} that we define as 

\begin{equation}\label{eq:ecg}
    \text{ECG} = \sum_{b=1}^B \frac{|\mathcal{B}_b|}{N} \max \Big( 1 - \alpha - \text{Coverage}\big(\mathcal{B}_b\big), 0 \Big),
\end{equation}
\noindent where $\mathcal{B}_b$ denotes a single bin and $N$ the total number of considered predictions in the dataset.\footnote{Since conformal prediction produces a \emph{lower} bound on the coverage, we do not include overcoverage in \cref{eq:ecg}.} 
The ECG thus captures the average weighted amount of undercoverage across bins.
In our experiments, we use $75$ bins in total. 
The same bins are used to also evaluate the \emph{size-stratified coverage metric} (SSC)\index{Size-stratified coverage} proposed by \citet{angelopoulos2021uncertainty}, with a well-calibrated method resulting in a SCC close to the desired coverage $1-\alpha$:

\begin{equation}
    \text{SCC} = \min_{b \in \{1,\ldots, B\}} \text{Coverage}\big(\mathcal{B}_b\big).
\end{equation}

\noindent We can therefore understand the SCC as the worst-case coverage across all considered bins.
We present some additional experiments where we assess the impact of key hyperparameters in \cref{app:ablations}.

\begin{table*}[htb!]
    \centering
    \resizebox{0.98\textwidth}{!}{
    \renewcommand{\arraystretch}{1.4}
    \begin{tabular}{@{}lllrccccrcccc@{}}
        \toprule[0.05cm]
         & & & \multicolumn{5}{c}{de $ \rightarrow $ en} & \multicolumn{5}{c}{ja $ \rightarrow $ en} \\
         \cmidrule(lr){4-8}\cmidrule(lr){9-13} 
         & Method & Dist. & $\tau$ & \% \textsc{Cov.} & $\varnothing$ \textsc{Width} $\downarrow$ & \textsc{Scc} $\uparrow$ & \textsc{Ecg} $\downarrow$ & $\tau$ & \% \textsc{Cov.} & $\varnothing$ \textsc{Width} $\downarrow$ & \textsc{Scc} $\uparrow$ & \textsc{Ecg} $\downarrow$ \\
        \midrule 
        \multirow{5}{*}{\rotatebox[origin=c]{90}{M2M100$_\text{(400M)}$}} & Nucleus & -- & -- & $.9207$ & $.48$ & $.25$ & $.00$ & -- & $.9261$ & $.54$ & $.41$ & $.02$ \\
        & Conf. & -- & -- &  $.9951$ & $.94$ & $.33$ & $.03$ & -- & $.9950$ & $.96$ & $.14$ & $.00$\\
        & \multirow[t]{3}{*}{\makecell[tl]{Non-Ex.\\CS}} & IP &  $3.93$ & $.8251$ & $.16$ & $.63$ & $.26$  & $11.90$ & $.8815$ & $.24$ & $.67$ & $.03$ \\
         & & $l_2$ & $512.14$ & $.8334$ & $.17$ & $.60$ & $.06$ & $419.91$ & $.8468$ & $.18$ & $.61$ & $.05$ \\
        & & cos & $2.54$ & $.8371$ & $.17$ & $.63$ & $.06$ & $3.53$ & $.8540$ & $.17$ & $.62$ & $.04$ \\[0.1cm]
        \midrule
        \multirow{5}{*}{\rotatebox[origin=c]{90}{M2M100$_\text{(1.2B)}$}} & Nucleus & -- & -- & $ .8339$ & $.38$ & $.00$ &  $.08$ & -- & $.7962$ & $.42$ & $.03$ & $.10$ \\
        & Conf. & -- & -- & $.9993$ & $.99$ & $.34$ & $.00$ & -- & $.9998$ & $.99$ & $.60$ & $.00$\\
        & \multirow[t]{3}{*}{\makecell[tl]{Non-Ex.\\CS}} & IP & $15.79$ & $.8861$ & $.25$ & $.71$ & $.03$  & $10.45$ & $.9129$ & $.38$ & $.72$ & $.00$\\
        & & $l_2$  & $1123.45$ & $.8874$ & $.25$ & $.72$ & $.03$  & $605.97$ & $.8896$ & $.30$ & $.76$ & $.01$ \\
        & & cos & $3.21$ & $.8858$ & $.25$ & $.72$ &  $.03$ & $1.48$ & $.8897$ & $.30$ & $.75$ & $.01$ \\
        \bottomrule[0.05cm]
    \end{tabular}%
    }
    \caption[Coverage results for the de $\rightarrow$ en and ja $\rightarrow$ en MT tasks.]{
        Coverage results for the de $\rightarrow$ en and ja $\rightarrow$ en MT tasks. 
        We report the best found temperature $\tau$ while keeping the confidence level $\alpha$ and number of neighbors $k=100$ fixed. 
        We also show the coverage percentage along with the avg.\@ prediction set size as a proportion of the entire vocabulary ($\varnothing$ \textsc{Width}) as well as ECG and SSC. 
        Tested distance metrics are inner product (IP), (squared) $l_2$ distance, and cosine similarity (cos).
    }\label{tab:coverage-results-mt}
\end{table*}

\begin{table}[htb!]
    \centering
    \resizebox{0.675\textwidth}{!}{
    \renewcommand{\arraystretch}{1.4}
        \begin{tabular}{@{}lllrcccc@{}}
        \toprule[0.05cm]
         & & & \multicolumn{5}{c}{\textsc{OpenWebText}} \\
         \cmidrule(lr){4-8}
         & Method & Dist. & $\tau$ & \% \textsc{Cov.} & $\varnothing$ \textsc{Width} $\downarrow$ &  \textsc{Scc} $\uparrow$ & \textsc{Ecg} $\downarrow$ \\
        \midrule 
        \multirow{5}{*}{\rotatebox[origin=c]{90}{OPT$_\text{(350M)}$}} & Nucl.\@ Sampl. & - & - & $.8913$ & $.05$ & $.71$ & $.01$ \\
        & Conf.\@ Sampl. & -- & -- & $.9913$ & $.90$ & $.91$ &  $.00$ \\
        & Non-Ex.\@ CS & IP & $4.99$ & $.9352$ & $.19$ & $.80$ & $.00$ \\
        & & $l_2$ & $.31 \times 10^4$ & $.9425$ & $.17$ & $.80$ & $.00$  \\ %  $15538.91$
        & & cos & $4.98$ & $.9370$ & $.15$ & $.83$& $.00$ \\
        \midrule
        \multirow{5}{*}{\rotatebox[origin=c]{90}{OPT$_\text{(1.3B)}$}} & Nucl.\@ Sampl. & -- & -- & $.8952$ & $.05$ & $.00$ & $.01$ \\
        & Conf.\@ Sampl. & -- & -- & $.9905$ & $.88$ & $0.95$ &  $.00$ \\
        & Non-Ex.\@ CS  & IP & $.48$ & $.9689$ & $.59$ & $.84$ & $.00$  \\
        & & $l_2$ & $1.55 \times 10^4$ & $.9539$ & $.20$ & $.83$ & $.00$  \\
        & & cos & $.11$ & $.9512$ & $.20$ & $.875$ & $.00$ \\
        \bottomrule[0.05cm]
    \end{tabular}%
    }
    \caption[Coverage results for the LM task.]{
        Coverage results for the LM task. 
        We report the best found temperature $\tau$ while keeping the confidence level $\alpha$ and number of neighbors $k=100$ fixed. 
        We also show the coverage percentage along with the avg.\@ prediction set size as a proportion of the entire vocabulary ($\varnothing$ \textsc{Width}) as well as the ECG and SSC metrics.
        Tested distance metrics are inner product (IP), (squared) $l_2$ distance and cos. similarity (cos).}
    \label{tab:coverage-results-lm}
\end{table}

\begin{figure*}[htb!]
    \centering

    \begin{subfigure}[t]{0.485\textwidth}
        \centering
        \includegraphics[width=0.98\columnwidth]{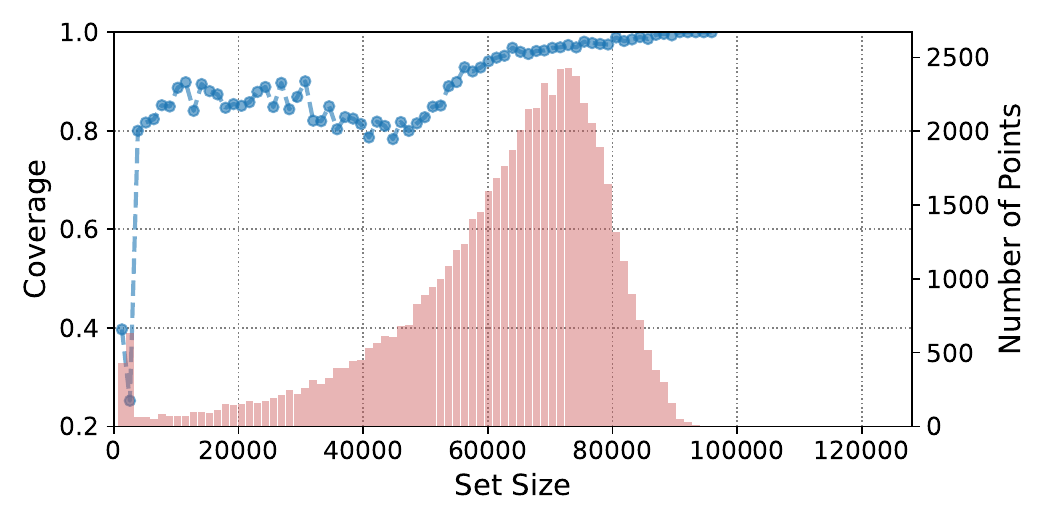}
        \caption{Nucleus Sampling on de $\rightarrow$ en.}
        \label{subfig:stratified-coverage-deen-nucleus}
    \end{subfigure}
    \hfill
    \begin{subfigure}[t]{0.485\textwidth}
        \centering
        \includegraphics[width=0.98\columnwidth]{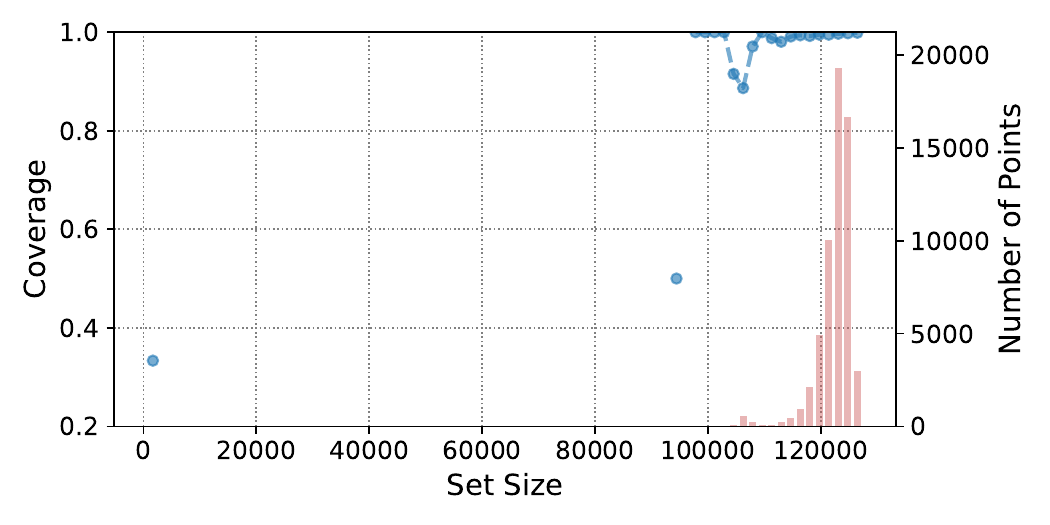}
        \caption{Conformal Nucleus Sampling on de $\rightarrow$ en.}
        \label{subfig:stratified-coverage-deen-conformal-nucleus}
    \end{subfigure}

    \begin{subfigure}[t]{0.485\textwidth}
        \centering
        \includegraphics[width=0.98\columnwidth]{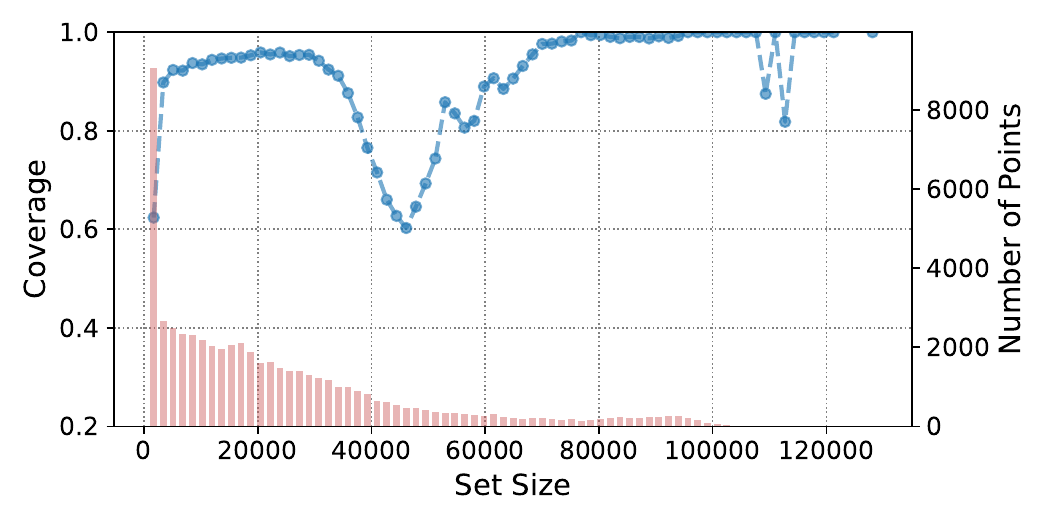}
        \caption{Non-Ex. Conformal Sampling on de $\rightarrow$ en.}
        \label{subfig:stratified-coverage-deen}
    \end{subfigure}
    \hfill
    \begin{subfigure}[t]{0.485\textwidth}
        \centering
        \includegraphics[width=0.98\columnwidth]{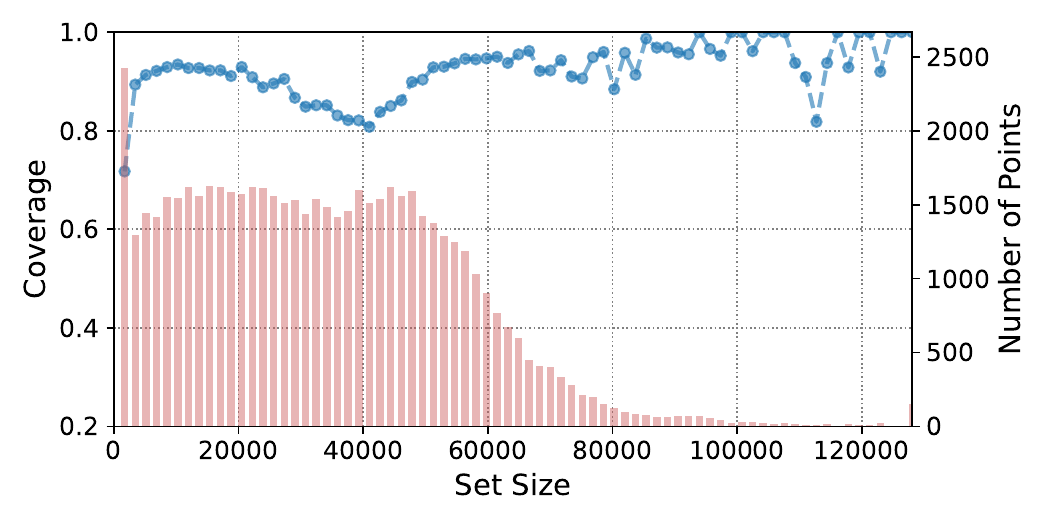}
        \caption{Non-Ex. CS on de $\rightarrow$ en with M2M100$_\text{(1.2B)}$.}
        \label{subfig:stratified-coverage-deen-large}
    \end{subfigure}
    \caption[Conditional converage results on the de $\rightarrow$ en MT task.]{Conditional coverage for the M2M100 on de $\rightarrow$ en with the small 418M model (\cref{subfig:stratified-coverage-deen-nucleus,subfig:stratified-coverage-deen-conformal-nucleus,subfig:stratified-coverage-deen}) and using the bigger 1.2B model (\cref{subfig:stratified-coverage-deen-large}). We aggregate predictions by set size using $75$ equally-spaced bins in total. The blue curve shows the conditional coverage per bin, whereas red bars show the number of binned predictions. }\label{fig:conditional-coverage}
\end{figure*}

\paragraph{Results.} 
The results are shown in \cref{tab:coverage-results-mt,tab:coverage-results-lm}.
We found that our method missed the desired coverage\index{Coverage} of $90 \%$ for MT\index{Machine translation} by only $8 \%$ or less. 
Beyond the best values shown in the tables, we were not able to further increase coverage by varying the temperature parameter without avoiding trivial coverage (i.e.\@, defaulting to very large set sizes).
This likely due to inherent coverage gap in \cref{eq:non-exchangeable-conformal-prediction-guarantees} that is due to distributional drift and is challenging to estimate directly.\\

Most notably, our method was able to achieve better SCC\index{Size-stratified coverage} scores while maintaining considerably smaller prediction sets than the baselines on average. 
The reason for this is illustrated in \cref{fig:conditional-coverage}:
while standard nucleus sampling produces some prediction sets that are small, the total coverage seems to mostly be achieved by creating very large prediction sets between 60k--80k tokens. 
The behavior of conformal nucleus sampling\index{Nucleus sampling!Conformal} by \citet{ravfogel2023conformal} is even more extreme in this regard, while our method produces smaller prediction sets, with the frequency of larger set sizes decreasing gracefully. 
In \cref{subfig:stratified-coverage-deen-large}, we can see that the larger M2M100 models also tend to produce larger prediction sets, but still noticeably smaller than the baselines. 
Importantly, for both M2M100 models, even very small prediction sets (size $\leq 1000$) achieve non-trivial coverage, unlike the baseline methods.\\

For LM\index{Language modeling}, we always found the model to slightly \emph{over}cover.
This does not contradict the desired lower bound on the coverage in \cref{eq:non-exchangeable-conformal-prediction-guarantees} and suggests a more negligible distributional drift\index{Shift!Distributional}.
While nucleus sampling produces the smallest average prediction sets, we can see that based on the SCC\index{Size-stratified coverage} values some strata remain undercovered. 
Instead, our method is able to strike a balance between stratified coverage and prediction set size.\index{Prediction set}
With respect to distance measures, we find that the difference between them is minimal, indicating that the quality largely depends on the retrieved local neighborhood of the decoder encoding and that finding the right temperature can help to tune the models to approximate the desired coverage.
We would now like to find out whether this neighborhood retrieval mechanism can prove to be robust under distributional shift as well.
Since we did not observe notable differences between the distance metrics, we continue with the $l_2$ distance.
%Furthermore, the ECG values suggest larger models are able to match the desired coverage better.
%Compared to MT, we show in \cref{app:coverage-experiments} that the OPT models also seem to mostly produce small prediction size, maintaining high coverage.

\subsection{Coverage Under Shift}\label{sec:shift-robustness}

\begin{figure*}
    \centering
    \begin{subfigure}[t]{0.99\textwidth}
        \centering
        \includegraphics[trim={0 1cm 0 0},clip,width=\textwidth]{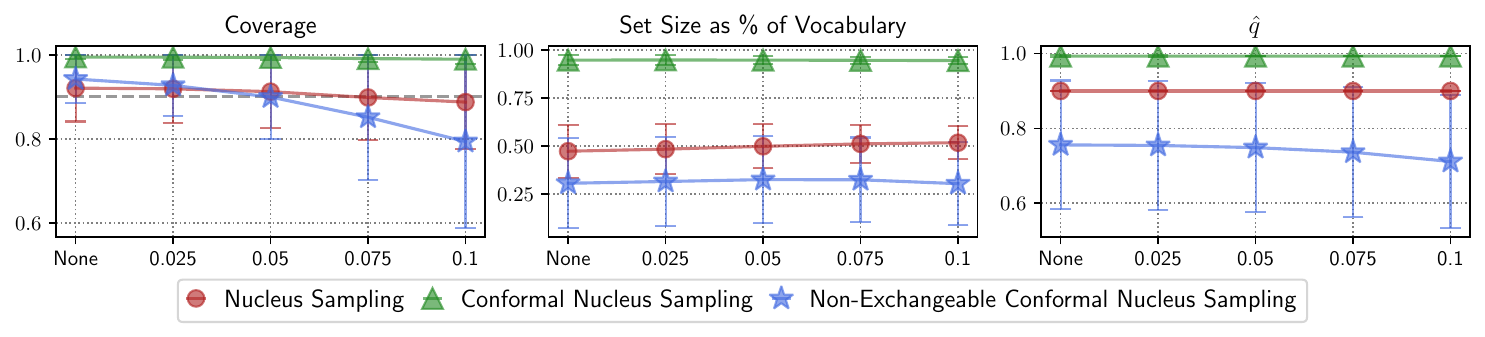}
    \end{subfigure}
    \begin{subfigure}[t]{\textwidth}
        \centering
        \includegraphics[width=0.99\textwidth]{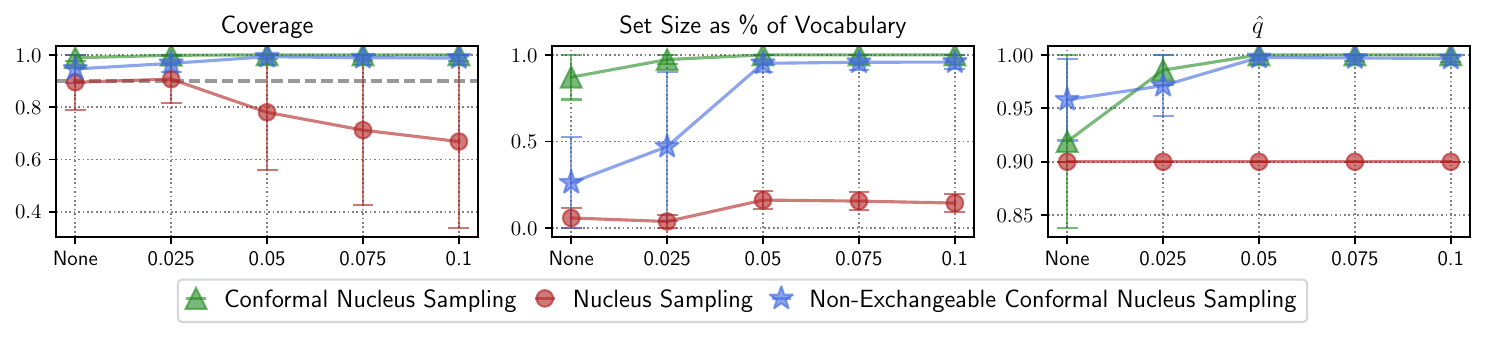}
    \end{subfigure}
    \caption[Coverage, average set size and $\hat{q}$ based on the noise level on the de $\rightarrow$ en MT task and open text generation task.]{
        Coverage, average set size and $\hat{q}$ based on the noise level on the de $\rightarrow$ en MT task (top) and open text generation task (bottom). 
        Error bars show one standard deviation.
    }\label{fig:noise-experiment-deen} 
\end{figure*}

To demonstrate how the retrieval of nearest neighbors can help to maintain coverage under distributional shift\index{Shift!Distributional}, we add Gaussian noise of increasing variance---and therefore intensity---to the last decoder hidden embeddings (for MT\index{Machine translation}) and the input embeddings (LM\index{Language modeling}).\footnote{
    A similar approach can be found for instance in the work of \citet{hahn2019self, zhang2023text} or by \citet{ovadia2019can,hendrycks2019benchmarking} in a computer vision context.}
This way, we are able to simulate distributional drift while still keeping the original sequence of input tokens intact, allowing us to measure the actual coverage.\index{Coverage}
We show the achieved coverage along with the average set size (as a percentage of the total vocabulary) and the average quantile $\hat{q}$ in \cref{fig:noise-experiment-deen}.
%In addition, we perform separate runs with the same noise levels for generation, and show the performance in the last column.
We can see that the conformal sampling method deteriorates into returning the full vocabulary as a prediction set. Thus it behaves similarly to simple sampling as indicated by the $\hat{q}$ values being close to $1$. 
Nucleus sampling\index{Nucleus sampling} provides smaller prediction sets compared to conformal sampling\index{Nucleus sampling!Conformal}, but they seem invariant to noise. 
As such, the method is not robust to noise injection in the open text generation task, and the obtained coverage deteriorates with noise variance $\geq 0.025$. 
Instead, the use of nearest neighbors allows for the estimation of prediction sets that are small but amenable to increase, such that the obtained coverage remains close to the desired one. 
We can specifically observe that the prediction set size increases considerably to mitigate the injected noise in the open-text generation case.  

\begin{table}
    \centering
    \resizebox{0.695\columnwidth}{!}{
    \renewcommand{\arraystretch}{1.4}
    \begin{tabular}{@{}lcccccc@{}}
        \toprule[0.05cm]
         & \multicolumn{5}{c}{\textsc{Noise level }} \\
         \cmidrule(lr){2-6}
          & \textsc{None} & $.025$ & $.05$ & $.075$ & $.1$ \\
        \midrule 
        $\varnothing$ Entropy & $8.46$ & $8.71$ & $9.20$ & $9.71$ & $10.08$ \\
        \midrule[0.005cm]
        Nucl.\@ Sampl. $(\rho)$  & $.87$ & $.86$ & $.84$ & $.82$ & $.81$ \\
        Conf.\@ Sampl. $(\rho)$ & $.60$ & $.60$ & $.60$ & $.57$ & $.55$\\
        Non-Ex.\@ CS $(\rho)$ & $-.14$ & $-.18$ & $-.27$ & $-.37$ & $-.45$ \\
        \bottomrule
    \end{tabular}
    }
    \caption[Relationship between prediction entropy and set sizes per noise level on the de $\rightarrow$ en task.]{
    Average entropy of 400M M2M100 model on de $\rightarrow$ en per noise level as well as the Spearman's $\rho$ correlation coefficients between the predictive entropy and the prediction set size of the different methods. 
    All results are significant with $p < 0.0001$.
    }\label{tab:analysis}
\end{table}

\paragraph{Neighbor Retrieval.} We further analyze how the retrieval enables this flexibility by relating it to the entropy of the output distribution of the 400M parameters M2M100 on German to English. 
Intuitively, the baseline methods, faced by high-entropy output distributions, need to produce wide prediction sets\index{Prediction set} in order to maintain coverage\index{Coverage}. 
In fact, we report such results by correlating entropy levels and prediction set sizes using Spearman's $\rho$ in \cref{tab:analysis}, showing strong positive correlations.
Our method in contrast consistently shows an \emph{anti}correlation between these two quantities, enabled by decoupling the creation of prediction sets from statistics of the output distribution to instead considering the non-conformity scores\index{Non-conformity score} of similar subsequences.
The fact that the prediction set size is not just dependent on the entropy of the predictions while maintaining coverage demonstrates the value of the nearest neighbors:
In this way, model uncertainty becomes more flexible and is corroborated by evidence gained from similar inputs.
%\chryssa{Perhaps we need another sentence here to explain why this is good? Is the assumption that in this way our method is able to estimate prediction set sizes that better represent the non-conformity of the test data and distance to the original distribution as opposed to simply reflecting the entropy and thus our approach is more trustworthy?} 
% Next, we explore whether these benefits in coverage and set sizes also translate into better generation results.
 
\subsection{Generation Quality}\label{sec:generation-quality}

\begin{table}[htb]
    \centering
    \resizebox{0.995\textwidth}{!}{
    \renewcommand{\arraystretch}{1.4}
    \begin{tabular}{@{}llcccccc@{}}
        \toprule[0.05cm]
            & &  \multicolumn{3}{c}{de $ \rightarrow $ en} & \multicolumn{3}{c}{ja $ \rightarrow $ en} \\
            \cmidrule(lr){3-5}\cmidrule(lr){6-8} 
            & Method & \textsc{Bleu} $\uparrow$ & \textsc{Comet} $\uparrow$ & \textsc{ChrF} $\uparrow$ & \textsc{Bleu} $\uparrow$ & \textsc{Comet} $\uparrow$ & \textsc{ChrF} $\uparrow$ \\
        \midrule 
        \multirow{7}{*}{\rotatebox[origin=c]{90}{M2M100$_\text{(400m)}$}} & Beam search & $28.53$ & $.88$ & $ 55.58$ & $11.37$ & $.63$ & $37.74$ \\
            & Greedy & $27.81$ & $.90$ & $54.9$ & $10.73$ & $.58$ & $36.5$ \\[0.2cm]
            & Nucleus Sampling & $27.63 {\scriptstyle\ \pm .03}$ & $.89 {\scriptstyle\ \pm .01}$ & $54.80 {\scriptstyle\ \pm .07}$ & $10.61 {\scriptstyle\ \pm .15}$ & $.59 {\scriptstyle\ \pm .01}$ & $36.52{\scriptstyle\ \pm .19}$ \\ % $17.42 \pm 0.19$ & $0.9$ & $43.34$ & $6.27$ & $0.51$ & $32.28$ \\
            & Top-$k$ Sampling & $27.63 {\scriptstyle\ \pm .03}$ & $.89 {\scriptstyle\ \pm .01}$ & $54.79 {\scriptstyle\ \pm .07}$ & $10.61 {\scriptstyle\ \pm .15}$ & $.59 {\scriptstyle\ \pm .01}$ & $36.52{\scriptstyle\ \pm .19}$ \\ % $19.57$ & $0.74$ & $49.47$ & $7.21$ & $0.54$ & $33.57$ \\
            & Conf.\@ Sampling &  $27.63 {\scriptstyle\ \pm .03}$ & $.89 {\scriptstyle\ \pm .01}$ & $54.80 {\scriptstyle\ \pm .07}$ & $10.61 {\scriptstyle\ \pm .15}$ & $.59 {\scriptstyle\ \pm .01}$ & $36.52 {\scriptstyle\ \pm .19}$\\ % $17.97$ & $0.9$ & $47.61$ & $6.37$ & $0.51$ & $32.46$ \\
            \cdashline{1-8}
            & Const.\@ Weight CS & $27.63 {\scriptstyle\ \pm .03}$ & $.89 {\scriptstyle\ \pm .01}$ & $54.80 {\scriptstyle\ \pm .07}$ & $10.61 {\scriptstyle\ \pm .15}$ & $.59 {\scriptstyle\ \pm .01}$ &  $36.52 {\scriptstyle\ \pm .19}$ \\% $17.76$ & $0.9$ & $47.34$ & $6.27$ & $0.51$ & $32.28$  \\
            & Non-Ex.\@ CS & $27.65 {\scriptstyle\ \pm .10}$ & $.90 {\scriptstyle\ \pm .01}$ & $54.82 {\scriptstyle\ \pm .14}$ & $\underline{10.74} {\scriptstyle\ \pm .11}$ & $.59 {\scriptstyle\ \pm .01}$ & $36.61 {\scriptstyle\ \pm .08}$ \\ % $24.81$ & $0.9$ & $53.14$ & $9.05$ & $0.56$ & $35.57$ \\[0.1cm]
            \midrule
            \multirow{7}{*}{\rotatebox[origin=c]{90}{M2M100$_\text{(1.2B)}$}} & Beam search & $30.89$ & $.90$ & $56.8$ & $13.76$ & $.63$ & $40.43$ \\
            & Greedy & $29.52$ & $.90$ & $55.67$ & $12.94$ & $.60$ & $39.91$  \\[0.2cm]
            & Nucleus Sampling & $29.37{\scriptstyle\  \pm .12}$ &  $.90{\scriptstyle\  \pm .00}$ & $55.55{\scriptstyle\  \pm .11}$ & $10.61{\scriptstyle\  \pm .15}$ & $.59{\scriptstyle\  \pm .01}$ & $36.52{\scriptstyle\  \pm .19}$ \\
            & Top-$k$ Sampling & $29.53{\scriptstyle\  \pm .00}$ & $.90{\scriptstyle\  \pm .00}$ & $55.67{\scriptstyle\  \pm .00}$ & $12.91{\scriptstyle\  \pm .08}$ & $.60{\scriptstyle\  \pm .01}$ & $39.95{\scriptstyle\  \pm .00}$ \\
            & Conf.\@ Sampling & $29.37{\scriptstyle\  \pm .12}$ & $.90{\scriptstyle\ \pm .00}$ & $55.55{\scriptstyle\  \pm .11}$ & $12.91{\scriptstyle\  \pm .08}$ & $.60{\scriptstyle\  \pm .00}$ & $39.95{\scriptstyle\  \pm .08}$\\
            \cdashline{1-8}
            & Const.\@ Weight CS & $29.37{\scriptstyle\  \pm 0.12}$ & $.90{\scriptstyle\ \pm .00}$ & $55.55{\scriptstyle\ \pm .11}$ & $12.91{\scriptstyle\ \pm .08}$ & $.60 {\scriptstyle\ \pm .01}$ & $39.95 {\scriptstyle\ \pm .08}$ \\
            & Non-Ex.\@ CS & $29.37 {\scriptstyle\ \pm 0.12}$ & $.90 {\scriptstyle\ \pm .00}$ & $55.55 {\scriptstyle\ \pm .11}$ & $12.91 {\scriptstyle\ \pm .08}$ & $.60 {\scriptstyle\ \pm .01}$ & $39.95 {\scriptstyle\ \pm .08}$ \\ %$19.42$ & $0.9$ & $48.36$ & $ 7.98$ & $0.67$ & $34.8$ \\
            %& Non-Ex. Conformal Sampling (IP) & $25.10$ & $0.9$ & $53.34$ & & & \\
            %& Non-Ex. Conformal Sampling (cosine) & $24.76$ & $0.9$ & $53.09$ & & & \\[0.2cm]
        %OPUS-MT$_\text{(base)}$ & Beam search & & & & & & \\
        % & Nucleus Sampling  & & & & & & \\
        % & Top-$k$ Sampling & & & & & & \\
        % & Conformal Sampling & & & & & & \\
        \bottomrule[0.05cm]
    \end{tabular}%
    }
    \caption[Generation results for the de $\rightarrow$ en and ja $\rightarrow$ en translation tasks.]{
        Generation results for the de $\rightarrow$ en and ja $\rightarrow$ en translation tasks.
        We report performance using $5$ beams for beam-search, top-$k$ sampling with $k=10$, and nucleus sampling with $p=0.9$. 
        Conformal methods all use $\alpha =  0.1$, with non-exchangeable variants retrieving $100$ neighbors, and sampling uses a softmax temperature of $0.1$. 
        Results using $5$ different seeds that are stat. significant according to the ASO test \citep{del2018optimal, dror2019deep, ulmer2022deep} with a confidence level of $0.95$ and threshold $\varepsilon_\text{min} \le 0.3$ are underlined.
    }\label{tab:generation-results-mt}
\end{table}%

\begin{table}[htb]
    \centering
    \resizebox{0.6\textwidth}{!}{
    \renewcommand{\arraystretch}{1.4}
    \begin{tabular}{@{}llcc@{}}
        \toprule[0.05cm]
            & & \multicolumn{2}{c}{\textsc{OpenWebText}} \\
            \cmidrule(lr){3-4}
            & Method & \textsc{MAUVE} $\uparrow$  & \textsc{BERTscore} $F_1$ $\uparrow$  \\
        \midrule 
        \multirow{6}{*}{\rotatebox[origin=c]{90}{OPT$_\text{(350M)}$}} & Beam search & $.12$ & $.79$ \\
        & Greedy & $.02$ & $.79$\\[0.2cm] % & $0.02$ & $0.70$ & $-0.27$\\[0.2cm]
            & Nucleus Sampling & $.91 {\scriptstyle\ \pm .02}$ & $.80{\scriptstyle\ \pm .00}$ \\ %$0.23 \pm 0.03$ & $0.74 \pm 0.00$ & $-0.61 \pm 0.01$ \\
            & Top-$k$ Sampling & $.90 {\scriptstyle\ \pm .03}$ & $\underline{.80} {\scriptstyle\ \pm .00}$ \\ % $0.24 \pm 0.01$ & $0.82 \pm 0.00$ & $-0.49 \pm 0.01$\\
            & Conf.\@ Sampling & $.91 {\scriptstyle\ \pm .02}$ & $.80 {\scriptstyle\ \pm .00}$ \\ % $0.03 \pm 0.00$ & $0.80 \pm 0.00$ & $-0.27 \pm 0.00$  \\
            \cdashline{1-4}
            & Const.\@ Weight CS  & $.91 {\scriptstyle\ \pm .02}$ & $.80 {\scriptstyle\ \pm .00}$\\ % $0.23 \pm 0.03$ &  $0.82 \pm 0.00$ & $-0.61 \pm 0.01$ \\
            & Non-Ex.\@ CS & $.92 {\scriptstyle\ \pm .01}$ & $.80 {\scriptstyle\ \pm .00}$  \\ % $0.26 \pm 0.00$ & $0.81 \pm 0.00$ & $-0.61 \pm 0.01$ \\
        \midrule
        \multirow{6}{*}{\rotatebox[origin=c]{90}{OPT$_\text{(1.3B)}$}} & Beam search & $.17$ & $.80$\\
        & Greedy & $.05$ & $.79$ \\[0.2cm] % $0.04$ & $0.80$ & $-0.31$\\[0.2cm]
            & Nucleus Sampling & $.91 {\scriptstyle\ \pm .02}$ & $.80 {\scriptstyle\ \pm .00}$\\ % $0.22 \pm 0.01$ & $0.81 \pm 0.00$ & $-0.64 \pm 0.00$ \\
            & Top-$k$ Sampling & $.93 {\scriptstyle\ \pm .01}$ & $\underline{.81} {\scriptstyle\ \pm .00}$\\ % $0.22 \pm 0.01$ & $0.82 \pm 0.00$ & $-0.55 \pm 0.01$ \\
            & Conf.\@ Sampling & $.93 {\scriptstyle\ \pm .01}$ & $.80 {\scriptstyle\ \pm .00}$ \\ % $0.05 \pm 0.00$ & $0.80 \pm 0.00$ & $-0.32 \pm 0.00$\\
            \cdashline{1-4}
            & Const.\@ Weight CS & $.91 {\scriptstyle\ \pm .02}$ & $.80 {\scriptstyle\ \pm .00}$ \\ % $0.22 \pm 0.01$ & $0.81 \pm 0.00$ & $-0.64 \pm 0.00$\\
            & Non-Ex.\@ CS & $.92 {\scriptstyle\ \pm .01}$ & $.81 {\scriptstyle\ \pm .00}$\\ % $0.23 \pm 0.01$ &  $0.81 \pm 0.00$ & $-0.63 \pm 0.00'$  \\
            \bottomrule[0.05cm]
    \end{tabular}%
    }
    \caption[Generation results for the open text generation.]{
        Generation results for the open text generation.
        We report performance using $5$ beams for beam-search, top-$k$ sampling with $k=10$, and nucleus sampling with $p=0.9$. 
        Conformal methods all use $\alpha =  0.1$, with non-exchangeable variants retrieving $100$ neighbors. 
        Results using $5$ different seeds that are stat.\@ significant according to the ASO test \citep{del2018optimal, dror2019deep, ulmer2022deep} with a confidence level of $0.95$ and threshold $\varepsilon_\text{min} \le 0.3$ are underlined.
    }
    \label{tab:generation-results-lm}
\end{table}

Crucially, our method should not degrade and potentially even improve generation quality. 
Thus, we evaluate the generation quality for the same tasks without supplying the gold prefix, instead employing standard language generation procedures. 
For language modeling\index{Language modeling}, we follow \citet{ravfogel2023conformal} and use the first $35$ tokens from the original sentence as input.
We compare against a set of generation strategies including top-$k$ sampling\index{Top-$k$ sampling} \citep{fan2018hierarchical, holtzmann2018learning, radford2019language}, nucleus sampling\index{Nucleus sampling} and conformal nucleus sampling\index{Nucleus sampling!Conformal}. 
We also test a variant of our method using constant weights $w_k = 1$ for retrieved neighbors (\emph{Const. Weight CS}) to assess the impact of the weighted neighbor retrieval 
procedure.
We further compare with beam search \citep{medress1977speech, graves2012sequence} with a softmax temperature of $0.1$, and greedy decoding.
Evaluation is performed using BLEU\index{BLEU} \citep{papineni2002bleu}, COMET-22\index{COMET} \citep{rei2020comet, rei2022comet22} and chrF\index{chrF} \citep{popovic-2017-chrf} for MT\index{Machine translation}, where COMET-22 is a trained neural metric.
For text generation, we use MAUVE\index{MAUVE} \citep{pillutla-etal:mauve:neurips2021} and BERTscore \citep{zhang2020bertscore}.\footnote{
    All metrics except for COMET\index{COMET} were used through Hugging Face \texttt{evaluate}. MAUVE uses \texttt{gpt2} as a featurizer.
    }
MAUVE is a neural metric that measures the divergence from human-written text, while BERTscore\index{BERTscore} involves a fine-tuned Bert model that aims to predict human quality judgments.

\paragraph{Results.} We show the results for the different methods in \cref{tab:generation-results-mt,tab:generation-results-lm}.
We see that beam search outperforms all sampling methods for MT.\index{Machine translation} 
This corroborates previous work by \citet{shaham2022you} who argue that (nucleus) sampling methods\index{Nucleus sampling}, by pruning only the bottom percentile of the
token distribution, introduce some degree of randomness that is beneficial for open text generation but may be less optimal for conditional language generation, where the desired output is constrained and exact matching generations are preferred (which is the case for MT).
Among sampling methods, we find nucleus sampling and conformal sampling to perform similarly (being in agreement with the findings of \citealp{ravfogel2023conformal}) but are sometimes on par or even outperformed by our non-exchangeable conformal sampling\index{Nucleus sampling!Non-exchangeable conformal} for MT.
For text generation, our method performs best for the smaller OPT model, but is slightly beaten by conformal nucleus sampling\index{Nucleus sampling!Conformal} in terms of MAUVE.
When using constant weights, performance deteriorates to the conformal sampling setup, emphasizing the importance of not considering all conformity scores equally when computing $\hat{q}$, even though the effect seems to be less pronounced for larger models.
This illustrates the benefit of creating flexible prediction sets that are adapted on token-basis, suggesting that both the latent space neighborhoods as well as the conformity scores are informative.
%We discuss examples of generated text in \cref{app:qualitative-analysis}.

\section{Discussion}\label{sec:discussion}

Our experiments have shown that despite the absence of i.i.d.\@ data in NLG and the loss in coverage induced by using dynamic calibration sets, the resulting coverage\index{Coverage} is still close to the pre-specified desired level for both LM\index{Language modeling} and MT\index{Machine translation}.
Additionally, even though the coverage gap predicted by the method of \citet{barber2023conformal} is infeasible to compute for us, we did not observe any critical degradation in practice. 
Further, we demonstrated how sampling from these calibrated prediction sets performs similarly or better than other sampling methods.
Even though our method is still outperformed by beam search in the MT setting, previous work such as minimum Bayes risk decoding has shown how multiple samples can be re-ranked to produce better outputs \citep{kumar2002minimum, kumar2004minimum, eikema2020map, fernandes2022quality, freitag2023epsilon}.
Additionally, recent dialogue systems based on LLMs use sampling instead of beam search for generation (e.g.\@ \citealp{openai2023gpt4, llama3modelcard}).
Since our prediction sets are more flexible and generally tighter, our results serve as a starting point for future work. 
For instance, our technique could be used with non-conformity scores\index{Non-conformity score} that do not consider token probabilities alone (e.g.\@ \citealp{meister2023locally}) or using prediction set widths as a proxy for uncertainty \citep{angelopoulos2021uncertainty}.
Furthermore, the extension with conformal risk control\index{Conformal risk control} \citep{angelopoulos2022conformal, farinhas2024nonexchangeable} enables guarantees with respect to a wider family of function than just coverage.\index{Coverage}
This opens up other directions, for instance defining functions that assess the desirability of the current generation, analogous to on-the-fly alignment procedures \citep{yang2021fudge, qin2022cold, mudgal2023controlled, gao2024linear}.

\paragraph{Limitations.} 
We highlight two main limitations of our work here:
Potential issues arising from different kinds of dataset shift as well as efficiency concerns.\index{Computational efficiency}
Even though any loss of coverage due to the term quantifying distributional drift\index{Shift!Distributional} in \cref{eq:non-exchangeable-conformal-prediction-guarantees} was limited in our experiments (see \cref{sec:retrieval-quality,sec:shift-robustness}), this might not hold across all possible setups. 
As long as we cannot feasibly approximate the shift penalty, it is impossible to determine a priori whether the loss of coverage might prove to be detrimental, and would have to be checked in a similar way as in our experiments. 
Furthermore, we only consider shifts between the models' training distributions and test data distributions here, while many other, unconsidered kinds of shifts exist \citep{moreno2012unifying, hupkes2022state}.
Additionally, even using optimized tools such as FAISS \citep{johnson2019billion}, moving the conformal prediction calibration step to inference incurs additional computational cost during generation.
Nevertheless, works such as \citet{he2023efficient, henrique2022chunk} show that there are several ways to improve the efficiency of $k$-NN approaches, and we leave such explorations to future work.

\section{Summary}\label{sec:conformal-nlg-summary}

In this chapter, we successfully demonstrated the application of a non-exchangeable variant of conformal prediction to machine translation and language modeling with the help of $k$-NN retrieval.
By retrieving a calibration set on the fly, one can create prediction sets for language generation based on the non-exchangeable conformal prediction\index{Conformal prediction!Non-exchangeable} algorithm by \citep{barber2023conformal}.
We demonstrated that this method best maintains the desired coverage across different dataset strata while keeping prediction sets smaller than other sampling methods, all while providing theoretical coverage guarantees about coverage that other comparable methods lack.\\
%We validated our method to produce encouraging results for generation tasks.
%Lastly, we analyzed the behavior under distributional drift, showing how the $k$-NN retrieval maintains desirable properties for the estimated prediction sets.
%We see our method as a step to provide a more principled way to perform sampling with conformal guarantees under more realistic assumptions.

However, this method has multiple shortcomings:
Except through the width of prediction sets, it does not explicitly quantify the uncertainty of the model, adds computational overhead to the inference process and furthermore requires access to the internal states of the model.
This becomes problematic when trying to apply to larger models than for instance M2M100$_{\text{(1.2B)}}$:
Many of the contemporary open-source models (like those in the case study in \cref{sec:case-study}) comprise 7 billion, 40 billion or even more parameters.
In addition, commercial closed-source models that can only be accessed through an API are estimated to be even larger.\footnote{
    For example, GPT-4's parameter count is rumored to be 1.76 trillion \citep{gpt4report}.
}
The nature of the API-only access further exacerbated this problem, as no information internal to the model, sometimes not even the token distribution, can be accessed.
The next chapter therefore proposes a method that operates within this very challenging and restrictive setup.

% Chapter X

\chapter{Uncertainty in Large Language Models} % Chapter title

\label{ch:uncertainty-llms} % For referencing the chapter elsewhere, use \autoref{ch:name} 

\epigraph{``\emph{In desperation I asked Fermi whether he was not impressed [\ldots]. 
He replied ``How many arbitrary parameters did you use for your calculations?'' 
I [\ldots] said ``Four.''
He said: ``I remember my friend Johnny von Neumann used to say, with four parameters I can fit an elephant, and with five I can make him wiggle his trunk.'' 
With that, the conversation was over.}''
}{---Freeman Dyson in \emph{A Meeting with Enrico Fermi} (2006).}

\begin{tikzpicture}[remember picture,overlay]
    \node[anchor=north,inner sep=0pt] at (current page text area.north) {\includegraphics[width=\linewidth, clip=true, trim = 8cm 100cm 8cm 25cm]{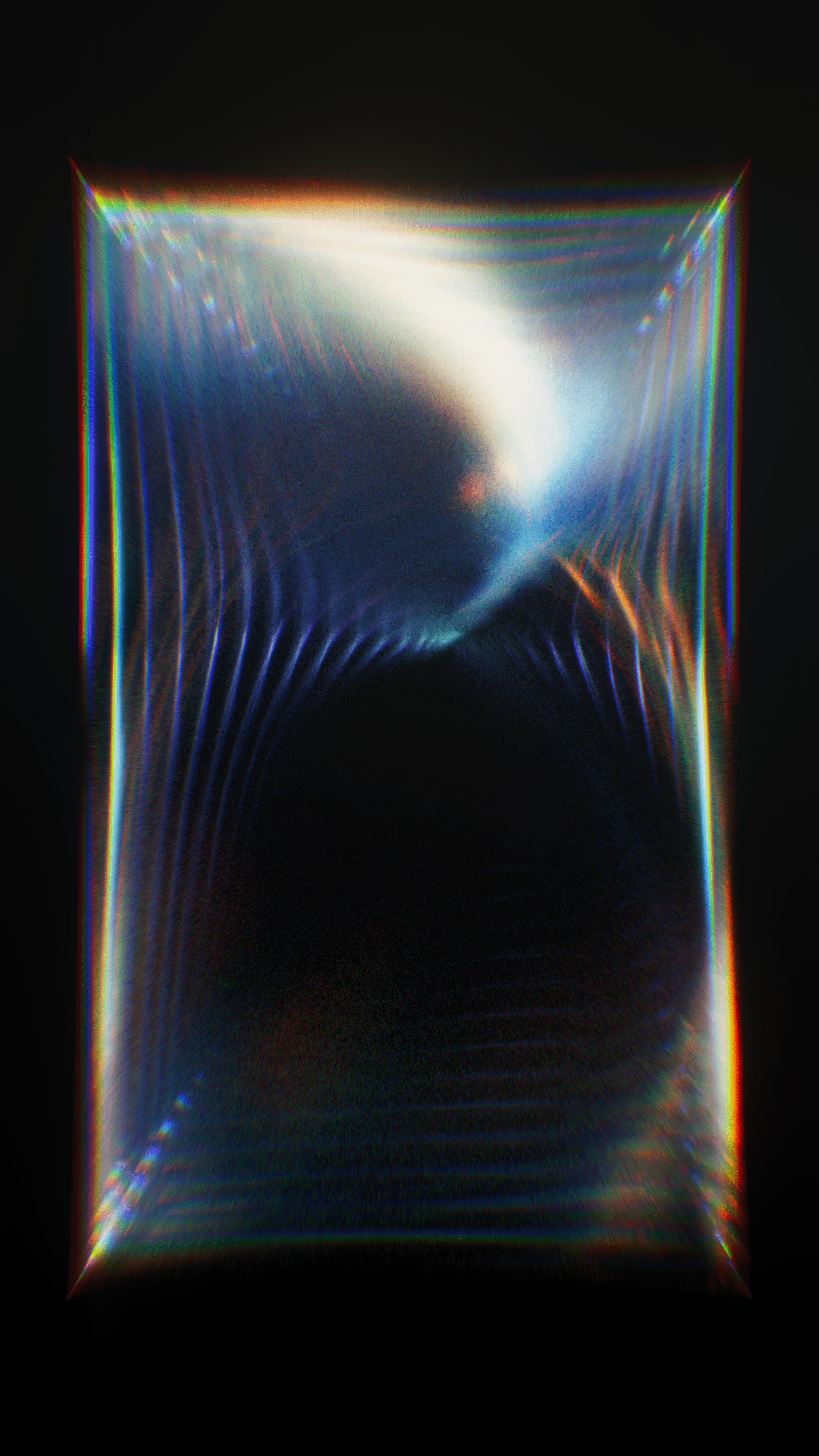}};
\end{tikzpicture}

\begin{footnotesize}
    \vspace{-2.5ex}
    \emph{The following work is based on \citet{ulmer2024calibrating}}.\\
    \vspace{2.5ex}
\end{footnotesize}

When given a case description of ``\emph{A man superglued his face to a piano and he says it's making it hard to get a full night of sleep}'', a medical LLM\index{Large language model} was found to list a plethora of potential causes in its diagnosis, including narcolepsy, sleep apnea and others.\footnote{
    \url{https://x.com/spiantado/status/1620459270180569090} (last accessed Nov. 7, 2023).
}
This, of course, ignores the seemingly obvious reason for the patient's complaints.
While humorous, this example illustrates the pitfalls of practical LLM applications:
Despite often looking convincing on the surface---especially to non-experts---model responses can be wrong or unreliable, leading to potentially harmful outcomes or a loss of trust\index{Trust}  in the system, foregoing its benefits.
Indeed, consistent behavior (imagine e.g.\@ reliably indicating a lack of confidence\index{Confidence} for unsure responses) has been argued as one way to build trust in automated systems \citep{jacovi2021formalizing}, while misleading predictions have been empirically shown to lead to a loss of trust that can be hard to recover from \citep{dhuliawala2023diachronic}.\\

The introductory example shows that as large language models (LLMs) are increasingly deployed in user-facing applications, building trust and maintaining safety by accurately quantifying a model's confidence\index{Confidence} in its prediction becomes even more important.
However, finding effective ways to calibrate LLMs\index{Large language model}---especially when the only interface to the models is their generated text---remains a challenge.
Most previously discussed methods to calibrate model predictions, such as the ones in \cref{sec:frequentist-neural-networks} or even the non-exchangeable conformal language generation\index{Non-exchangeable conformal language generation} from the previous \cref{ch:uncertainty-generation}, require some degree of retraining or at least access to model hidden states and / or logits.\\

In this chapter, we introduce APRICOT\index{APRICOT} \peach\@ (\textbf{a}uxiliary \textbf{pr}ed\textbf{i}ction of \textbf{co}nfidence \textbf{t}argets): 
A method to set targets to calibrate confidence scores\index{Calibration}\index{Confidence} to and train an additional model that predicts an LLM's confidence based on its textual input and output alone.
This approach has several advantages: 
It is conceptually simple, does not require access to the target model beyond its output, does not interfere with the language generation, and has a multitude of potential usages, for instance by verbalizing the predicted confidence or adjusting the given answer based on the confidence.
We show how our approach performs competitively in terms of calibration error for white-box and black-box LLMs on closed-book question-answering\index{Question-answering} to detect incorrect LLM answers\index{Error detection}.
Our contributions are as follows:
We propose to obtain calibration targets for LLM confidence scores without requiring any additional information about LLM internals or question metadata.
We show that using auxiliary models on the target LLM's input and output is sufficient to predict a useful notion of confidence for question-answering on TriviaQA \citep{joshi2017triviaqa} and CoQA \citep{reddy2019coqa}.
We also perform additional studies to identify which parts of the LLM's output are most useful to predict confidence.\index{Confidence}
%All the code is openly available.\footnote{\url{https://github.com/parameterlab/apricot}.}

\section{Calibrating LLMs with Auxiliary Models}\label{sec:apricot}

\begin{figure}[tb!]
    \centering
    \includegraphics[width=0.8\columnwidth]{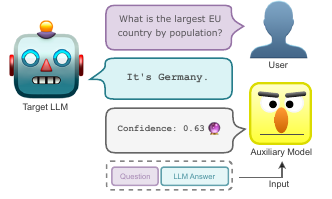}
    \caption[Illustration of APRICOT \peach.]{Illustration of APRICOT \peach: We train an auxiliary model to predict a target LLM's confidence based on its input and the generated answer.}
\end{figure}

\begin{table}[htb]
    \centering
    \renewcommand{\arraystretch}{1.35}
    \resizebox{.8\columnwidth}{!}{
    \begin{tabular}{@{}rccc@{}}
    \toprule
    Method & Black-box LLM? & Consistent? & Calibrated? \\
    \midrule
    Sequence likelihoods & \rcross & \gcheck & \rcross \\
    Verbalized uncertainty & \gcheck & \rcross & \rcross \\
    APRICOT \peach\@ (ours) & \gcheck & \gcheck & \gcheck \\
    \bottomrule
    \end{tabular}%
    }
    \caption[Comparison of appealing attributes that LLM confidence quantification techniques should fulfill.]{
        Comparison of appealing attributes that LLM confidence quantification techniques should fulfill. 
        They should ideally be applicable to black-box LLMs, be consistent (i.e.\@ always elicit a response that indicates confidence in contrast to an unrelated response), and produce calibrated estimates of confidence.
    }\label{tab:constrasting-methods}
\end{table}

\noindent Estimating the confidence of an LLM\index{Confidence}\index{Large language model} can be challenging, since their size rules out many traditional techniques that require finetuning or access to model parameters.
In this light, using the likelihood of the generated sequence\index{Likelihood!Sequence} might seem like an appealing alternative;
however, it may not actually reflect the reliability of the model's answer and often cannot be retrieved when using black-box models, where the only output is the generated text.
Verbalized uncertainty\index{Uncertainty!Verbalized}, i.e.\@ prompting the LLM to express its uncertainty in words, can be a solution when the model is powerful enough.
But as we later show in \cref{sec:experiments}, the generated confidence expressions are not very diverse, and results are not always \emph{consistent}, meaning that the model does not always generate a desired confidence self-assessment.
We will later see how for verbalized uncertainty for instance, models sometimes respond with unrelated answers, even when prompted to express their uncertainty.
As we illustrate in \cref{tab:constrasting-methods}, our method, APRICOT \peach\index{APRICOT}, fulfills all of these criteria:
Through a one-time finetuning procedure of an auxiliary model on the target LLMs outputs, we have full control over a calibrated model that gives consistent and precise confidence estimates.\\

\begin{figure}[htb]
    \centering
    \includegraphics[width=0.85\columnwidth]{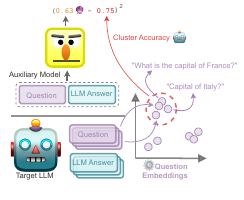}
    \caption[Full overview over APRICOT \peach.]{Full overview over APRICOT \peach. We collect a LLM's answer to a set of questions and embed the latter using an embedding model. After clustering similar questions and identifying the LLM's accuracy on them, we can use this value as reference when training to predict the confidence from a question-answer pair.}\label{fig:method}
\end{figure}

In \cref{fig:method} we give an overview of APRICOT \peach, which consists of three main steps:
Firstly, we prompt the target LLM to generate training data for our auxiliary model (\cref{sec:prompting-target-llm}).
Secondly, we set calibration targets in a way that does not require access to the target LLM beyond its generated outputs (\cref{sec:setting-calibration-targets}).
Lastly, we train the auxiliary calibrator to predict the target LLM's confidence for a given question (\cref{sec:training-auxiliary-model}).\footnote{
    In general, using secondary neural models to predict properties of the generated text also has connections to other tasks such as translation quality estimation \citep{blatz2004confidence, quirk2004training, wang2019improving, glushkova-etal-2021-uncertainty-aware, zerva2022better}, toxicity classification \citep{maslej2020comparison} or fine-grained reward modeling \citep{wu2023fine}.
}
Thereby, we add two parts that are agnostic to the LLM in question: A method that determines calibration targets, and their prediction through the auxiliary model.
Note that we use the terms auxiliary model or calibrator interchangeably in the following sections.

\subsection{Prompting the Target LLM}\label{sec:prompting-target-llm}

\begin{figure}[tb!]
    \centering
    %\vspace{0.5cm}
    \begin{subfigure}[b]{\columnwidth}
        \centering
        \includegraphics[width=0.65\textwidth]{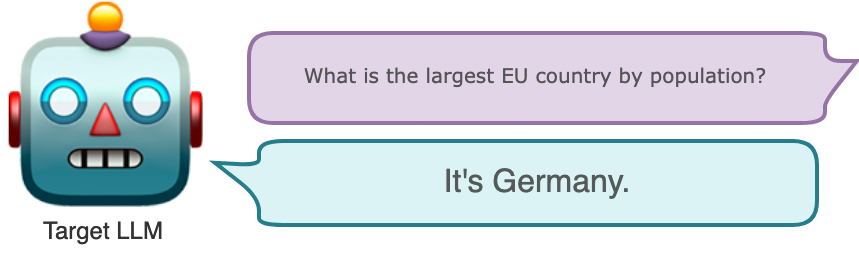}
        \caption{Default prompting.}
        \label{subfig:default-prompting}
    \end{subfigure}
    \par\bigskip
    \begin{subfigure}[b]{\columnwidth}
        \centering
        \includegraphics[width=0.65\textwidth]{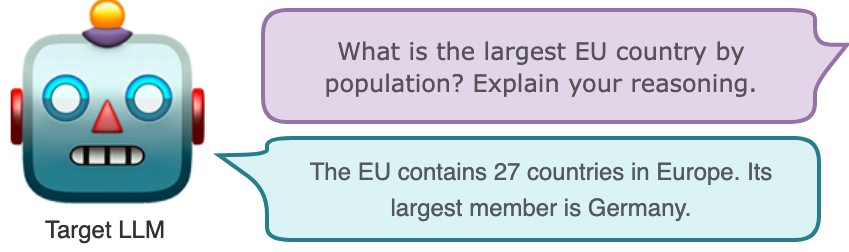}
        \caption{Chain-of-though prompting.}
        \label{subfig:cot-prompting}
    \end{subfigure}
    \par\bigskip
    \begin{subfigure}[b]{\columnwidth}
        \centering
        \includegraphics[width=0.65\textwidth]{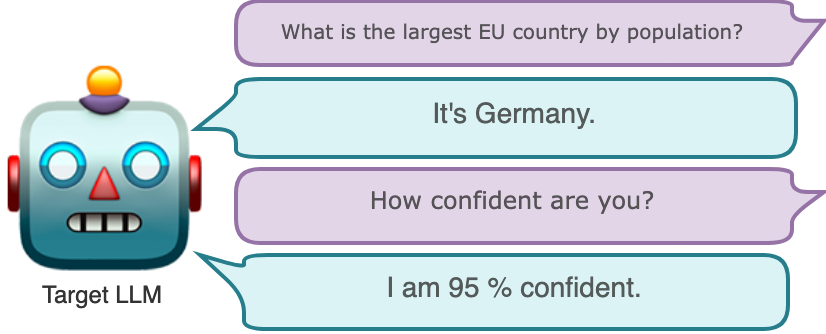}
        \caption{Prompting with verbalized confidence.}
        \label{subfig:verbalized-promptng}
    \end{subfigure}
    \caption[Illustration of the prompting strategies used to generate the input data for the auxiliary calibrator.]{Illustration of the prompting strategies used to generate the input data for the auxiliary calibrator. Note that (c) can also involve confidence expressed in words (``My confidence level is low'') and that (b) and (c) can be combined.}\label{fig:prompting-methods}
\end{figure}

In the first step, we generate finetuning data for the auxiliary model by prompting the target LLM on the given task. 
Here, we explore different variations to see which model response might provide the best training signal for the auxiliary calibrator.
More concretely, while the original prompt and model generation might already suffice to predict the model's confidence\index{Confidence}, we also ask the model to elaborate on its answer using chain-of-thought prompting \citep{wei2022chain}\index{Prompting!Chain-of-thought}.
We hypothesize that including additional reasoning steps exposes signals that are useful for the calibrator.\footnote{
    We do this while acknowledging evidence by \citet{turpin2023language} that shows that any chain-of-thought reasoning might not reflect the actual reasons for a specific model response.
    Nevertheless, even if chain-of-thought reasoning \emph{does not} unveil the actual process of the LLM, it can provide useful textual features to the auxiliary model,
    including unexpected intermediate result or linguistic markers of uncertainty.
}
We furthermore take a model's assessment of its confidence into account, too.
Recent works on \emph{verbalized uncertainty}\index{Uncertainty!Verbalized} \citep{lin2022teaching, tian2023just} investigated how to elicit such an assessment as a percentage value, e.g.\@ ``I am 95 \% confident in my answer'', or using linguistic expressions such as ``My confidence is somewhat low''.
While previous studies like \citet{zhou2024relying} have demonstrated the difficulty in obtaining reliable self-assessments, we can just treat verbalized uncertainties as additional input features, and let their importance be determined through the auxiliary model training.
We illustrate the different prompting strategies in \cref{fig:prompting-methods} and elaborate on the prompts in the following.

\paragraph{Prompt Design.} \index{Prompting}
We use a simple prompt for question-answering, where we fill in a template of the form ``Question: \textcolor{slotcolor}{\{Question\}} Answer:''. 
For in-context samples\index{In-context learning}, we prepend the demonstrations to the input, using the sample template as above.
In the case of chain-of-thought prompting\index{Prompting!Chain-of-thought}, we use the prompting below:

\begin{tcolorbox}[width=\columnwidth,colback={white},title={\small QA Chain-of-thought prompt},enhanced,attach boxed title to top right={yshift=-3.5mm, xshift=-5mm}, colbacktitle=white, coltitle=black, top=12pt]  
        \small
        Briefly answer the following question by thinking step by step.
        Question: \textcolor{slotcolor}{\{Question\}} Answer:\\
\end{tcolorbox}

In the case where the question is supposed to be answered given some context, we slightly change the prompt design:

\begin{tcolorbox}[width=\columnwidth,colback={white},title={\small Chain-of-thought prompt with context},enhanced,attach boxed title to top right={yshift=-3.5mm, xshift=-5mm}, colbacktitle=white, coltitle=black, top=12pt]  
        \small 
        Context: \textcolor{slotcolor}{\{Context\}}\\
        Instruction: Briefly answer the following question by thinking step by step.\\
        Question: \textcolor{slotcolor}{\{Question\}}\\
        Answer:\\
\end{tcolorbox}

Here, the passage that questions are based on is given first, and chain-of-thought prompting is signaled through the ``Instruction'' field. 
When no chain-of-thought prompting\index{Prompting!Chain-of-thought} is used, this field is omitted.
For the verbalized uncertainty\index{Uncertainty!Verbalized}, we use the following prompts, in which case we omit any in-context samples:

\begin{tcolorbox}[width=\columnwidth,colback={white},title={\small Verbalized uncertainty prompt (quantitative)},enhanced,attach boxed title to top right={yshift=-3.5mm, xshift=-5mm}, colbacktitle=white, coltitle=black, top=12pt]  
        \small
        \textcolor{slotcolor}{\{Question\}} \textcolor{slotcolor}{\{Model answer\}}
        Please provide your confidence in the answer only as one of 'Very Low' / 'Low' / 'Somewhat Low' / 'Medium' / 'Somewhat High' / 'High' / 'Very High':\\
\end{tcolorbox}

\begin{tcolorbox}[width=\columnwidth,colback={white},title={\small Verbalized uncertainty prompt (qualitative)},enhanced,attach boxed title to top right={yshift=-3.5mm, xshift=-5mm}, colbacktitle=white, coltitle=black, top=12pt]  
        \small
        \textcolor{slotcolor}{\{Question\}} \textcolor{slotcolor}{\{Model answer\}} Please provide your confidence in the answer only in percent (0--100 \%):\\
\end{tcolorbox}

We follow \citet{kuhn2023semantic} and use $10$ in-context samples for the original answer, which are randomly sampled from the training set (but in contrast to \citeauthor{kuhn2023semantic}, we sample different examples for each instance). 
When prompting for verbalized uncertainty, we remove these in-context samples.
Additionally, verbalized uncertainty expressions such as `very low' or `high' are mapped back onto the following numerical values for evaluation purposes (in the order of appearance in the template above): $0, 0.3, 0.45, 0.5, 0.65, 0.7, 1$.\footnote{
    The choice of these specific numbers is admittedly somewhat arbitrary, and other works such as \citet{lin2022teaching} have also employed similarly heuristic scales.
    \citet{tian2023just} motivate their mapping from expression to probabilities to a social media survey by \citet{fagen2015perception}, where respondents were asked to assign probabilities to different expressions.
    However, these values vary greatly between participants, and thus assigning a single numerical value is still challenging.
    %\citet{lin2022teaching} also found that random male first names (!) can provide replacements for these confidence expressions.
}

%\begin{table}[htb]
%    \centering
%    \renewcommand{\arraystretch}{1.25}
%    \resizebox{0.3\columnwidth}{!}{
%    \begin{tabular}{lr}
%        \toprule
%          Expression & Value  \\
%        \toprule
%          Very low & $0$ \\
%          Low & $0.3$ \\
%          Somewhat low & $0.45$ \\
%          Medium & $0.5$ \\ 
%          Somewhat high & $0.65$ \\
%          High & $0.7$ \\
%          Very high & $1$ \\
%          \bottomrule
%    \end{tabular}%
%    }
%    \caption{Mappings between confidence expressions for qualitative uncertainty and numerical confidence scores.}
%    \label{tab:confidence-mappings}
%\end{table}

\subsection{Setting Calibration Targets}\label{sec:setting-calibration-targets}

After explaining the inputs to the auxiliary model, the question naturally arises about what the calibrator should be trained to predict.
The work by \citet{mielke2022reducing} introduces an additional model that simply predicts the correctness of an individual answer (and does so by using the target model's internal hidden states, which is not possible for black-box models).
We test this type of output in \cref{sec:calibration-experiments}, but we also show that we can produce better calibration targets through clustering.\\

Recall the notion of calibration and calibration error\index{Calibration} from \cref{sec:frequentist-neural-networks}, where we saw that the expected calibration error\index{Expected calibration error} can be approximated by binning points into buckets $\mathcal{B}_m$ \citep{naeini2015obtaining} by confidence:

\begin{equation}\label{eq:ece}
    \sum_{m=1}^M \frac{|\mathcal{B}_m|}{N}\Big|\underbrace{\frac{1}{|\mathcal{B}_m|}\sum_{i \in \mathcal{B}_m}\indicator{\hat{y}_i = y_i}}_{\text{Bin accuracy (target)}} - \underbrace{\frac{1}{|\mathcal{B}_m|}\sum_{i \in \mathcal{B}_m}\hat{p}_i}_{\text{Avg. bin confidence}}\Big|,
\end{equation}

\noindent where $\hat{p}_i$ corresponds to some confidence score. 
Our key insight is here that we can optimize the expected calibration error through a similar approximation as in \cref{eq:ece}, without changing the LLM's\index{Large language model} original answers or access to token probabilities.
We can abstract the idea in \cref{eq:ece} as aggregating samples in homogeneous groups (in the above case, groups of similar confidence), and measuring the group-wise accuracy.
But now, instead of creating bins $\mathcal{B}_m$ by confidence, which is not possible in a black-box setting, we create clustered sets $\mathcal{C}_m$ of questions with similar sentence embeddings.
Calibration targets are then obtained by using the average accuracy of the LLMs answers per question set $\mathcal{C}_m$.
This is similar to the method of \citet{lin2022teaching}, who consider the accuracy per question category (e.g.\@ multiplication or addition math questions).
Yet in the absence of such additional categorization data, we expect good embedding and clustering algorithms to roughly group inputs by category.
\citet{holtgen2023richness} also echo a similar idea of more generalized grouping choices, describing how ECE's grouping by confidence can be abstracted to other kinds of similarities.
They also provide a proof that the calibration error of a predictor based on a $k$-nearest neighbor clustering tends to zero in the infinite data limit.\\

\paragraph{Implementation.}
Practically, we embed questions into a latent space using a light-weight model such as SentenceBert\index{Bert!Sentence} \citep{reimer2019sentence}, normalize the embeddings along the feature dimension \citep{timkey2021all}, and then use HDBSCAN \citep{campello2013density}, an unsupervised, bottom-up clustering algorithm, to cluster them into questions of similar topic. 
The use of HDBSCAN\index{HDBSCAN} has multiple advantages: 
Compared to e.g.\@ $k$-means, we do not have to determine the numbers of clusters in advance, and since the clustering is conducted bottom-up, clusters are not constrained to a spherical shape. 
Furthermore, compared to its predecessor DBSCAN \citep{ester1996density}, HDBSCAN does not require one to determine the minimum distance between points for clustering manually.
We evaluate this procedure in \cref{sec:clustering-eval,app:additional-clustering}.

%\paragraph{Setting calibration targets.} \ Many previous works on calibration considers confidence scores either directly based on generated sequence likelihoods \tdcite or from verbalized expression mapped back onto confidence scores \citep{tian2023just}. 
%Similarly to the algorithm behind the expected calibration error (ECE; \citealp{guo2017calibration}), inputs are then usually grouped by their uncalibrated confidence scores, and the accuracy per grouping is computed and set as a calibration target.
%Intuitively, we group inputs by a kind of similarity (namely, according to their raw confidence), calibrating the model to reflect the accuracy on that same subgroup as its confidence.
%However, one might also consider different grouping of inputs; for instance, \citet{lin2022teaching} aggregate mathematical question by question type, implying that we could also use metadata for this purpose. 
%Unfortunately, this information is often not available, and since we assume black-box access for our target model,
%we require a different way to achieve this objective.

%For this purpose, we propose the following procedure:

\subsection{Training the Auxiliary Model}\label{sec:training-auxiliary-model}

After determining the input and the training targets for the auxiliary model in the previous sections, we can now describe the actual training procedure that makes it predict the target LLM's confidence.
To start, we feed the questions alongside some in-context samples\index{In-context learning} into our target LLM.
We retain the generated answers and create a dataset that combines the question (without in-context samples) and the target model's answers.
These are used to train the auxiliary calibrator to predict the calibration targets obtained by the clustering procedure above.
In our experiments, we use DeBERTaV3 \citep{he2023debertav3}, an improvement on the original DeBERTa\index{DeBERTa} model \citep{he2021deberta} using variety of improvements with respect to its architecture and pre-trainign objective.
We then finetune it using the AdamW optimizer \citep{loshchilov2018fixing} in combination with a cosine learning rate schedule.
We minimize the following mean squared error, where $\hat{p}_i$ is the predicted confidence, $\mathcal{C}(i)$ the cluster that the input question with index $i$ belongs to, and $\hat{a}_j$ an answer given by the target LLM\index{Large language model}: 
\begin{equation}
\mathcal{L}\big(\hat{p}_i, \mathcal{C}(i)\big) = \Big(\hat{p}_i - \underbrace{\frac{1}{|\mathcal{C}(i)|}\sum_{j \in \mathcal{C}(i)}\indicator{\hat{a}_j\text{ is correct}}}_{\text{Cluster accuracy (target)}}\Big)^2.
\end{equation}
We also explore a variant that simply predicts whether the LLM's answer is expected to be correct or incorrect, so in this case, we simply optimize a binary cross-entropy loss:
\begin{equation}
    \mathcal{L}\big(\hat{p}_i, \hat{a}_i\big) = \indicator{\hat{a}_i\text{ is correct}} \log \hat{p}_i + \big(1 - \indicator{\hat{a}_i\text{ is correct}}\big)\log (1 - \hat{p}_i).
\end{equation}

Although omitted here for clarity, the actual loss also uses loss weights to balance the unequal distribution of correct and incorrect language model answers.\footnote{The loss weights are based on \texttt{scikit-learn}'s implementation using the ``balanced'' mode, see \url{https://scikit-learn.org/stable/modules/generated/sklearn.utils.class_weight.compute_class_weight.html}.}
Finally, we select the final model via the best loss on the validation set.
We determine the learning rate and weight decay term through Bayesian hyperparameter search \citep{snoek2012practical}, picking the best configuration by validation loss.
We detail search ranges and found values in \cref{app:hyperparameters}.
Training hardware and the environmental impact are discussed in \cref{app:environmental-impact}.

\section{Experiments}\label{sec:apricot-experiments}

We now demonstrate how APRICOT \peach\@\index{APRICOT} provides a simple yet effective solution to calibrate LLMs.
Before assessing the quality of the unsupervised clustering to determine calibration targets from \cref{sec:setting-calibration-targets}, we first introduce the dataset and models.

\paragraph{Datasets.} We employ TriviaQA \citep{joshi2017triviaqa}, a common (closed-book) question-answering dataset.
Open-ended question-answering is an ideal testbed for natural language generation tasks, since it is comparatively easy to check whether an answer is correct or not, so calibration has an intuitive interpretation.
To preprocess TriviaQA, we create a training set of 12k examples and choose another 1.5k samples as a validation and test split, respectively.\footnote{Since the original test split does not include answers, we generate the validation and test split from the original validation split.} 
Secondly, we run experiments on CoQA \citep{reddy2019coqa}, a conversational question-answering dataset in which the model is quizzed about the information in a passage of text.
We treat the dataset as an open-book dataset, where the model is shown the passage and then asked one of the corresponding questions at a time.
We extract a subset of the dataset to match the split sizes of TriviaQA.

\paragraph{Models.} 
We test two models settings:
A white-box setting, where the model can be run locally and we have full access to its internals, and a black-box setting, where the model is only available through an API, drastically reducing the options for uncertainty quantification methods.
For our white-box model experiments, we choose a 7 billion parameter variant of the Vicuna v1.5 model \citep{zheng2023judging},\footnote{\url{https://huggingface.co/lmsys/vicuna-7b-v1.5}.}
an instruction-finetuned model originating from Llama 2 \citep{touvron2023llama}.
For the black-box model, we opt for OpenAI's GPT-3.5 \citep{chatgpt}.\footnote{Specifically, using version \texttt{gpt-3.5-turbo-0125}.}
Despite recent API changes granting access to token probabilities,\footnote{\url{https://x.com/OpenAIDevs/status/1735730662362189872} (last accessed on 16.01.24).} creating methods for black-box confidence estimation is still relevant for multiple reasons: Token probabilities are not available for most black-box models, they might be removed again to defend against potential security issues; and they are not always a reliable proxy for confidence.

\subsection{Setting Calibration Targets by Clustering}\label{sec:clustering-eval}

Before beginning our main experiments, we would like to verify that our proposed methodology in \cref{sec:setting-calibration-targets} is sound. 
In particular, clustering the embeddings of questions and computing the calibration confidence targets rests on the assumption that similarly-themed questions are collected in the same cluster.
Ideally, we would like to check this using metadata, which however is usually not available.

\paragraph{Setup.} Instead, we evaluate this through different means:
We first use the \texttt{all-mpnet-base-v2} model from the sentence transformers\index{Transformer!Sentence} package \citep{reimer2019sentence} and HDBSCAN\index{HDBSCAN} with a minimum cluster size of $3$ to cluster questions.
We then analyze the textual and semantic similarity of questions in a cluster by computing the average pair-wise ROUGE-L score\index{ROUGE} (\emph{semantic}; \citealp{lin2004rouge})\footnote{As implemented by the \texttt{evaluate} package, see \url{https://huggingface.co/docs/evaluate/index}.} between questions, and cosine similarities between question \emph{embeddings} of the same cluster (\emph{semantic}).
Since we assume the sentence embedding model to capture the meaning of a sentence, the expect the semantic similarity to be high when questions are similar in topic, but might differ in their choice of words.
Since performing this evaluation on the entire dataset is computationally expensive, we approximate the score by using $5$ pairwise comparisons per cluster, with $200$ comparisons for ROUGE-L and $1000$ for cosine similarity in total, respectively. 
As a control for our method (\emph{clustering}), we also compute values between unrelated questions that are not in the same cluster (\emph{random}).

\begin{table}
    \renewcommand{\arraystretch}{1.5}
    \centering
    \resizebox{0.6\textwidth}{!}{
    \begin{tabular}{@{}lrrrr@{}}
        \toprule
        & \multicolumn{2}{c}{TriviaQA} & \multicolumn{2}{c}{CoQA}  \\
        \cmidrule(lr){2-3} \cmidrule(lr){4-5}
         & Textual & Semantic & Textual & Semantic \\
        \toprule
        Random & $.11{\scriptstyle\ \pm.08 }$ & $.00{\scriptstyle\ \pm.08}$ & $.08{\scriptstyle\ \pm.12}$ & $.00{\scriptstyle\ \pm.12}$ \\
         Clustering & $.39{\scriptstyle\ \pm.28}$ & $.60{\scriptstyle\ \pm.14}$ & $.47{\scriptstyle\ \pm.25}$ & $.70{\scriptstyle\ \pm.17}$  \\ %& $.44{\scriptstyle\ \pm.08}$ & $.39{\scriptstyle\ \pm.38}$  \\
        \bottomrule
    \end{tabular}%
    }
    \caption[Results of evaluation of found clusters on TriviaQA and CoQA.]{
        Results of evaluation of found clusters on TriviaQA and CoQA, including one standard deviation. 
        Textual refers to similarity scores computed using ROUGE-L, and semantic scores based on cosine similarities of question embeddings of the same cluster.
        Here, we use random comparisons between questions in the dataset as a baseline.
    }\label{fig:clustering-results}
\end{table}

\paragraph{Results.} We show the results of this analysis in \cref{fig:clustering-results}.
We observe noticeable differences between the random baseline and the similarity for the clustering scores, both on a textual and semantic level.
While there is smaller difference on a textual level due to the relatively similar wording of questions, the semantic similarity based on the encoded questions is very notable. 
We provide deeper analyses of this part in \cref{app:additional-clustering}, showing that this method creates diverse ranges of calibration confidence targets.
This suggests two things: On the one hand, our proposed methodology is able to identify fine-grained categories of questions. On the other hand, the diversity in calibration targets shown in \cref{app:additional-clustering} indicates that we detect sets of questions on which the LLM's accuracy varies---and that this variety should be reflected. 
We test the ability of different methods to do exactly this next.

\subsection{Calibrating White and Black-Box Models}\label{sec:calibration-experiments}

Next, we test whether auxiliary models can reliably predict the target LLM's confidence.\index{Confidence}
We describe our experimental conditions below.

\paragraph{Evaluation metrics.} Aside from reporting the accuracy on the question-answering task, we also report several calibration metrics, including the expected calibration error (ECE;\index{Expected calibration error}\citealp{naeini2015obtaining}) using $10$ bins.
%and the maximum calibration error (MCE; \citealp{naeini2015obtaining}), which considers the maximum difference between accuracy and confidence across bin and can be seen as an approximation for the worst-case calibration.
In order to address any distortion of results introduced by the binning procedure, we use smooth ECE \index{Expected calibration error!Smooth} (smECE; \citealp{blasiok2023smooth}), which avoids the binning altogether by smoothing observations using a radial basis function kernel. 
We also consider Brier score\index{Brier score} \citep{brier1950verification}, which can be interpreted as mean-squared error for probabilistic predictions.
We further show how indicative the predicted confidence is for answering a question incorrectly by measuring the AUROC.\index{AUROC}
The AUROC treats the problem as a binary error detection task\index{Error detection} based on the confidence scores, aggregating the results over all possible decision thresholds.
In each case, we report the result alongside a bootstrap estimate of the standard error \citep{efron1994introduction} estimated from $100$ samples and test for significance using the almost stochastic order \index{Stochastic order!Almost} test \citep{del2018optimal, dror2019deep, ulmer2022deep} with $\tau = 0.35$ and a confidence level of $\alpha = 0.1$.

\paragraph{Baselines.} To contextualize the auxiliary calibrator results, we consider the following baselines:
We consider the raw (length-normalized) sequence likelihoods\index{Likelihood!Sequence} (Seq.\@ likelihood) as well as variant using Platt scaling\index{Platt scaling} \citep{platt1999probabilistic}:
Using the raw likelihood $\hat{p} \in [0, 1]$ and the sigmoid function\index{Sigmoid function} $\sigma$, we fit two additional scalars $a, b \in \mathbb{R}$ to minimize the mean squared error on the validation set to produce a calibrated likelihood $\hat{q} = \sigma(a\hat{p} + b)$ while keeping all other calibrator parameters fixed.
We also compare it to the recent method of \emph{verbalized} uncertainty\index{Uncertainty!Verbalized} \citep{lin2022teaching, tian2023just}, where we ask the model to assess its confidence directly. 
We do this by asking for confidence in percent (Verbalized $\%$) and using a seven-point scale from ``very low'' to ``very high'', and which is mapped back to numeric confidence scores (Verbalized Qual.). 
Where applicable, we also distinguish between baselines with and without chain-of-thought prompting\index{Prompting!Chain-of-thought} (CoT; \citealp{wei2022chain}).
% We detail this scale and corresponding prompts in \cref{app:prompting}.
For our approach, we distinguish between confidence targets obtained through the procedure in \cref{sec:setting-calibration-targets} (clustering) and simply predicting whether the given answer is correct or incorrect (binary).

\begin{table*}[htb]
    \renewcommand{\arraystretch}{1.35}
    \centering
    \resizebox{0.995\textwidth}{!}{
    \begin{tabular}{@{}llcccccccc@{}}
        \toprule
         & & \multicolumn{4}{c}{TriviaQA} & \multicolumn{4}{c}{CoQA} \\
        \cmidrule(lr){3-6} \cmidrule(lr){7-10} 
        & Method & Brier$\downarrow$ & ECE$\downarrow$  & smECE$\downarrow$ & AUROC$\uparrow$ &  Brier$\downarrow$ & ECE$\downarrow$  & smECE$\downarrow$ & AUROC$\uparrow$  \\
        \toprule
        \multirow{10}{*}{\rotatebox{90}{Vicuna v1.5 (white-box)}} & Seq.\@ like.\@  & $.22{\scriptstyle\ \pm.01}$ & $.05{\scriptstyle\ \pm.00}$ & $.03{\scriptstyle\ \pm.00}$ & $.79{\scriptstyle\ \pm.01}$ & $.32{\scriptstyle\ \pm.01}$ & $.08{\scriptstyle\ \pm.00}$ & $.08{\scriptstyle\ \pm.00}$ & $.69{\scriptstyle\ \pm.01}$ \\
         & Seq.\@ like.\@ {\footnotesize (CoT)} & $.25{\scriptstyle\ \pm.01}$ & $.04{\scriptstyle\ \pm.00}$ & $.04{\scriptstyle\ \pm.00}$ & $.70{\scriptstyle\ \pm.01}$ & $.35{\scriptstyle\ \pm.01}$ & $.04{\scriptstyle\ \pm.00}$ & $.05{\scriptstyle\ \pm.00}$ & $.61{\scriptstyle\ \pm.01}$ \\
         & Platt & $.24{\scriptstyle\ \pm.00}$ & $.08{\scriptstyle\ \pm.00}$ & $.07{\scriptstyle\ \pm.00}$ & $.70{\scriptstyle\ \pm.01}$ & $.30{\scriptstyle\ \pm.00}$ & $.03{\scriptstyle\ \pm.00}$ & $.03{\scriptstyle\ \pm.00}$ & $.69{\scriptstyle\ \pm.01}$ \\
         & Platt {\footnotesize (CoT)} & $.24{\scriptstyle\ \pm.00}$ & $.12{\scriptstyle\ \pm.00}$ & $.11{\scriptstyle\ \pm.00}$ & $.79{\scriptstyle\ \pm.01}$ & $.30{\scriptstyle\ \pm.00}$ & $.02{\scriptstyle\ \pm.00}$ & $.02{\scriptstyle\ \pm.00}$ & $.61{\scriptstyle\ \pm.01}$ \\
         & Verb.\@  Qual.\@ & $.38{\scriptstyle\ \pm.03}$ & $.02{\scriptstyle\ \pm.00}$ & $.02{\scriptstyle\ \pm.00}$ & $.62{\scriptstyle\ \pm.03}$ & $.45{\scriptstyle\ \pm.01}$ & $\underline{\mathbf{.00}}{\scriptstyle\ \pm.00}$ & $\underline{\mathbf{.00}}{\scriptstyle\ \pm.00}$ & $.48{\scriptstyle\ \pm.01}$ \\
         & Verb.\@ Qual.\@ {\footnotesize (CoT)} &  $.39{\scriptstyle\ \pm.02}$& $\underline{\mathbf{.01}}{\scriptstyle\ \pm.00}$& $\underline{\mathbf{.01}}{\scriptstyle\ \pm.00}$& $.60{\scriptstyle\ \pm.02}$ & $.45{\scriptstyle\ \pm.01}$& $\underline{\mathbf{.00}}{\scriptstyle\ \pm.00}$& $\underline{\mathbf{.00}}{\scriptstyle\ \pm.00}$& $.48{\scriptstyle\ \pm.01}$\\
        & Verb.\@ $\%$ & $.39{\scriptstyle\ \pm.01}$ & $.38{\scriptstyle\ \pm.00}$ & $.27{\scriptstyle\ \pm.00}$ & $.52{\scriptstyle\ \pm.01}$ & $.49{\scriptstyle\ \pm.01}$ & $.48{\scriptstyle\ \pm.00}$ & $.32{\scriptstyle\ \pm.00}$ & $.53{\scriptstyle\ \pm.01}$ \\
        & Verb.\@ $\%$ {\footnotesize (CoT)} & $.39{\scriptstyle\ \pm.01}$ & $.38{\scriptstyle\ \pm.00}$ & $.26{\scriptstyle\ \pm.00}$ & $.49{\scriptstyle\ \pm.01}$ & $.48{\scriptstyle\ \pm.01}$ & $.06{\scriptstyle\ \pm.00}$ & $.06{\scriptstyle\ \pm.00}$ & $.55{\scriptstyle\ \pm.01}$ \\
        \cdashline{2-10}
        & Aux.\@ {\footnotesize (binary)} & $.20{\scriptstyle\ \pm.01}$ & $.16{\scriptstyle\ \pm.01}$ & $.15{\scriptstyle\ \pm.01}$ & $.81{\scriptstyle\ \pm.01}$ & $.20{\scriptstyle\ \pm.01}$& $.16{\scriptstyle\ \pm.01}$& $.15{\scriptstyle\ \pm.01}$& $\underline{\mathbf{.82}}{\scriptstyle\ \pm.01}$ \\
        & Aux.\@ {\footnotesize (clustering)} & $\underline{\mathbf{.18}}{\scriptstyle\ \pm.00}$& $.09{\scriptstyle\ \pm.01}$& $.09{\scriptstyle\ \pm.01}$& $\underline{\mathbf{.83}}{\scriptstyle\ \pm.01}$ &  $\underline{\mathbf{.18}}{\scriptstyle\ \pm.00}$& $.04{\scriptstyle\ \pm.01}$& $.04{\scriptstyle\ \pm.01}$& $\mathbf{.82}{\scriptstyle\ \pm.01}$ \\
         \midrule
        \multirow{10}{*}{\rotatebox{90}{GPT-3.5 (black-box)}} & Seq.\@ like.\@  &  $.15{\scriptstyle\ \pm.01}$& $.04{\scriptstyle\ \pm.00}$& $.04{\scriptstyle\ \pm.00}$& $.69{\scriptstyle\ \pm.02}$ & $.29{\scriptstyle\ \pm.01}$& $.11{\scriptstyle\ \pm.00}$& $.11{\scriptstyle\ \pm.00}$& $.70{\scriptstyle\ \pm.01}$ \\
        & Seq.\@ like.\@ {\footnotesize (CoT)} &  $.14{\scriptstyle\ \pm.00}$& $.05{\scriptstyle\ \pm.00}$& $.05{\scriptstyle\ \pm.00}$& $.60{\scriptstyle\ \pm.02}$ &  $.25{\scriptstyle\ \pm.00}$& $\underline{\mathbf{.01}}{\scriptstyle\ \pm.00}$& $\underline{\mathbf{.02}}{\scriptstyle\ \pm.00}$& $.52{\scriptstyle\ \pm.02}$\\
        & Platt & $.15{\scriptstyle\ \pm.00}$& $.04{\scriptstyle\ \pm.00}$& $.04{\scriptstyle\ \pm.00}$& $.69{\scriptstyle\ \pm.02}$ &  $.26{\scriptstyle\ \pm.01}$& $.03{\scriptstyle\ \pm.00}$& $.03{\scriptstyle\ \pm.00}$& $.70{\scriptstyle\ \pm.01}$ \\
        & Platt {\footnotesize (CoT)} & $.15{\scriptstyle\ \pm.00}$& $.12{\scriptstyle\ \pm.00}$& $.12{\scriptstyle\ \pm.00}$& $.60{\scriptstyle\ \pm.02}$ &  $.25{\scriptstyle\ \pm.00}$& $.06{\scriptstyle\ \pm.00}$& $.06{\scriptstyle\ \pm.00}$& $.52{\scriptstyle\ \pm.02}$  \\
        & Verb.\@ Qual.\@ &  $.14{\scriptstyle\ \pm.01}$& $.07{\scriptstyle\ \pm.00}$& $.04{\scriptstyle\ \pm.00}$& $.61{\scriptstyle\ \pm.02}$ &  $.27{\scriptstyle\ \pm.00}$& $.07{\scriptstyle\ \pm.00}$& $.05{\scriptstyle\ \pm.00}$& $.52{\scriptstyle\ \pm.01}$ \\ 
        & Verb.\@ Qual.\@ {\footnotesize (CoT)} &  $.15{\scriptstyle\ \pm.00}$& $.04{\scriptstyle\ \pm.00}$& $.03{\scriptstyle\ \pm.00}$& $.63{\scriptstyle\ \pm.02}$ &  $.30{\scriptstyle\ \pm.01}$& $.08{\scriptstyle\ \pm.01}$& $.04{\scriptstyle\ \pm.00}$& $.50{\scriptstyle\ \pm.01}$ \\ 
        & Verb.\@ $\%$ & $.13{\scriptstyle\ \pm.01}$& $.01{\scriptstyle\ \pm.00}$& $\underline{\mathbf{.01}}{\scriptstyle\ \pm.00}$& $.63{\scriptstyle\ \pm.02}$ & $.34{\scriptstyle\ \pm.01}$& $.25{\scriptstyle\ \pm.00}$& $.22{\scriptstyle\ \pm.00}$& $.54{\scriptstyle\ \pm.01}$ \\ % 0.99\\
        & Verb.\@ $\%$ {\footnotesize (CoT)} &  $.13{\scriptstyle\ \pm.01}$& $\underline{\mathbf{.00}}{\scriptstyle\ \pm.00}$& $\mathbf{.01}{\scriptstyle\ \pm.00}$& $.63{\scriptstyle\ \pm.02}$ & $.37{\scriptstyle\ \pm.01}$& $.09{\scriptstyle\ \pm.01}$& $.06{\scriptstyle\ \pm.00}$& $.49{\scriptstyle\ \pm.02}$ \\
        \cdashline{2-10}
        & Aux.\@ {\footnotesize (binary)} & $.14{\scriptstyle\ \pm.00}$& $.14{\scriptstyle\ \pm.01}$ & $.14{\scriptstyle\ \pm.01}$ & $.65{\scriptstyle\ \pm.02}$ & $.19{\scriptstyle\ \pm.01}$ & $.13{\scriptstyle\ \pm.01}$ & $.13{\scriptstyle\ \pm.01}$ & $\mathbf{.81}{\scriptstyle\ \pm.01}$ \\
        & Aux.\@ {\footnotesize (clustering)} & $\underline{\mathbf{.12}}{\scriptstyle\ \pm.01}$ & $.06{\scriptstyle\ \pm.01}$ & $.06{\scriptstyle\ \pm.01}$ & $\underline{\mathbf{.72}}{\scriptstyle\ \pm.02}$ & $\underline{\mathbf{.18}}{\scriptstyle\ \pm.00}$ & $.02{\scriptstyle\ \pm.01}$ & $\mathbf{.02}{\scriptstyle\ \pm.00}$ & $\mathbf{.81}{\scriptstyle\ \pm.01}$ \\
        \bottomrule
    \end{tabular}%
    }
    \caption[Calibration results for Vicuna v1.5 and GPT-3.5 on TriviaQA and CoQA.]{Calibration results for Vicuna v1.5 and GPT-3.5 on TriviaQA and CoQA. We bold the best results per dataset and model, and underline those that are statistically significant compared to all other results assessed via the ASO test. Results are reported along with a bootstrap estimate of the standard error.
    }\label{tab:calibration-results}
\end{table*}

\begin{figure*}[htb]
    \centering
    \begin{subfigure}[t]{0.32\textwidth}
        \centering
        \includegraphics[width=\textwidth]{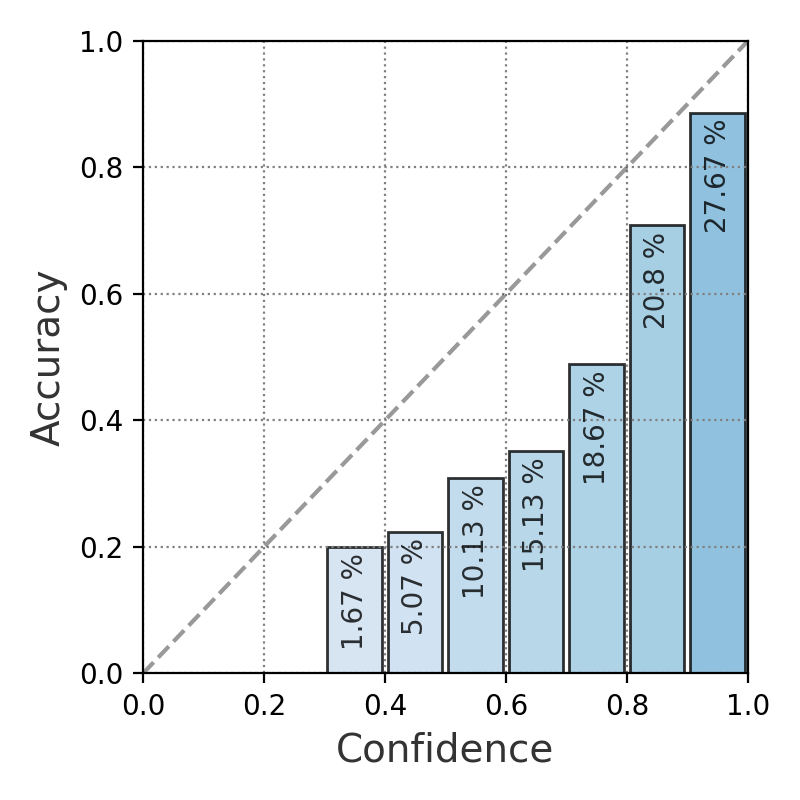}
        \caption{Seq.\@ likelihood.}
        \label{subfig:seq-likelihood}
    \end{subfigure}
    \hfill
    \begin{subfigure}[t]{0.32\textwidth}
        \centering
        \includegraphics[width=\textwidth]{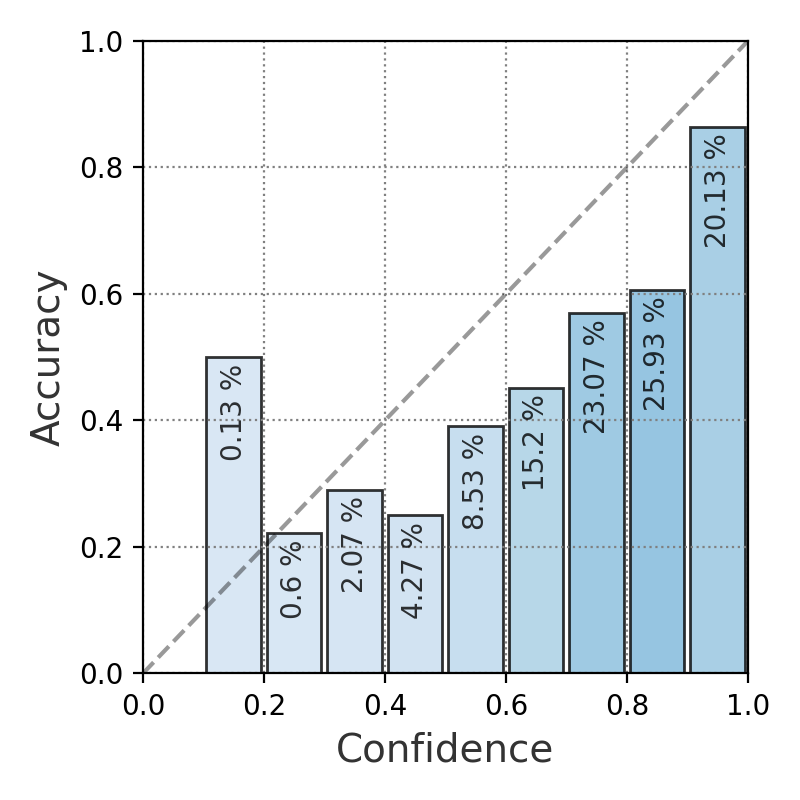}
        \caption{Seq.\@ likelihood (CoT).}
        \label{subfig:temperature-scaling}
    \end{subfigure}
    \hfill
    \begin{subfigure}[t]{0.32\textwidth}
        \centering
        \includegraphics[width=\textwidth]{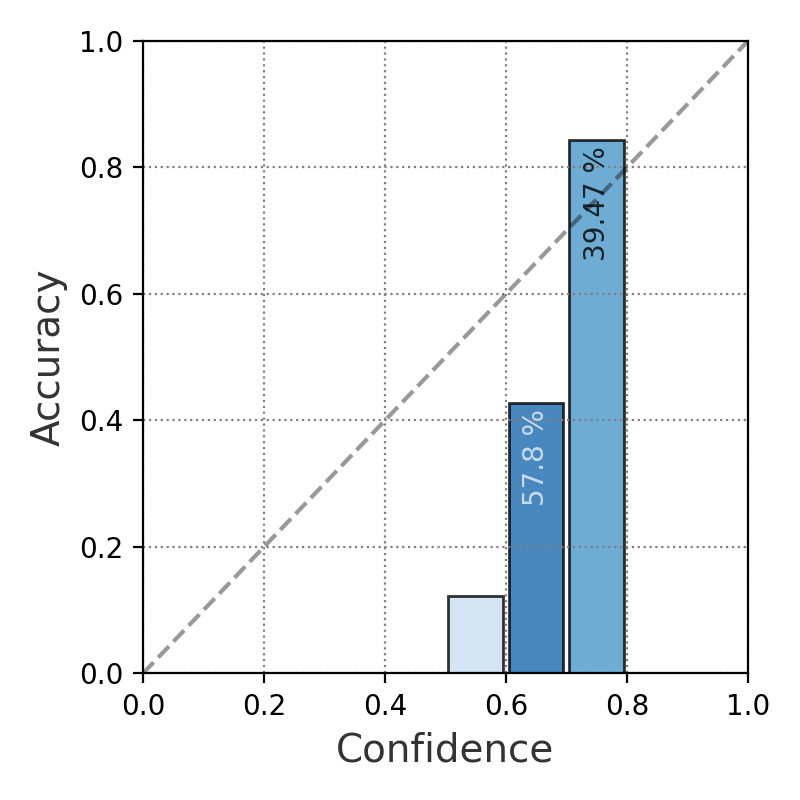}
        \caption{Platt scaling.}
        \label{subfig:verbalized-percentage}
    \end{subfigure}
    \hfill
    \begin{subfigure}[t]{0.32\textwidth}
        \centering
        \includegraphics[width=\textwidth]{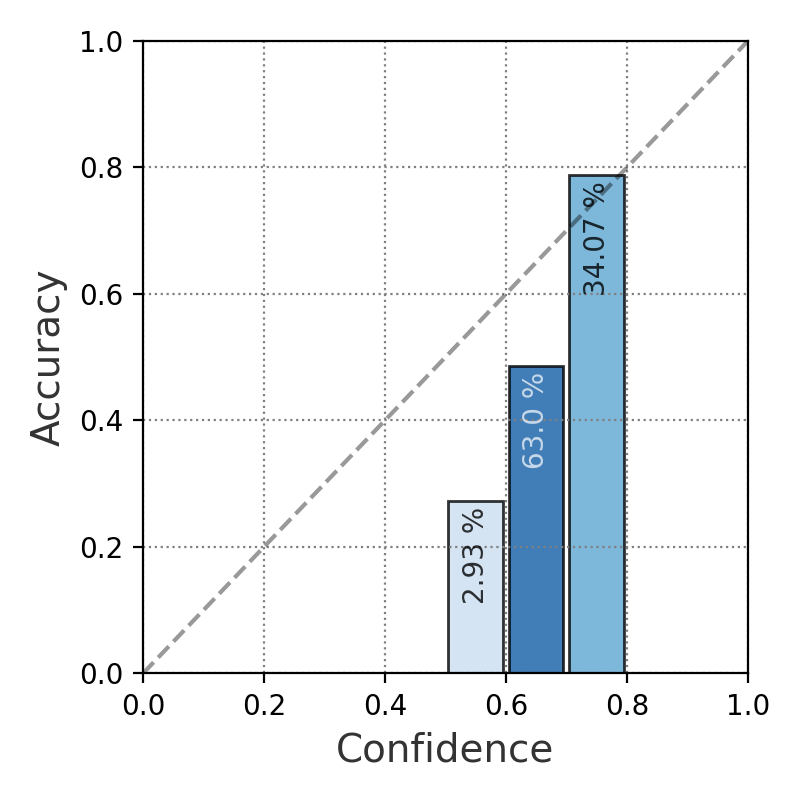}
        \caption{Platt scaling (CoT).}
        \label{subfig:verbalized-qualitative}
    \end{subfigure}
    \begin{subfigure}[t]{0.32\textwidth}
        \centering
        \includegraphics[width=\textwidth]{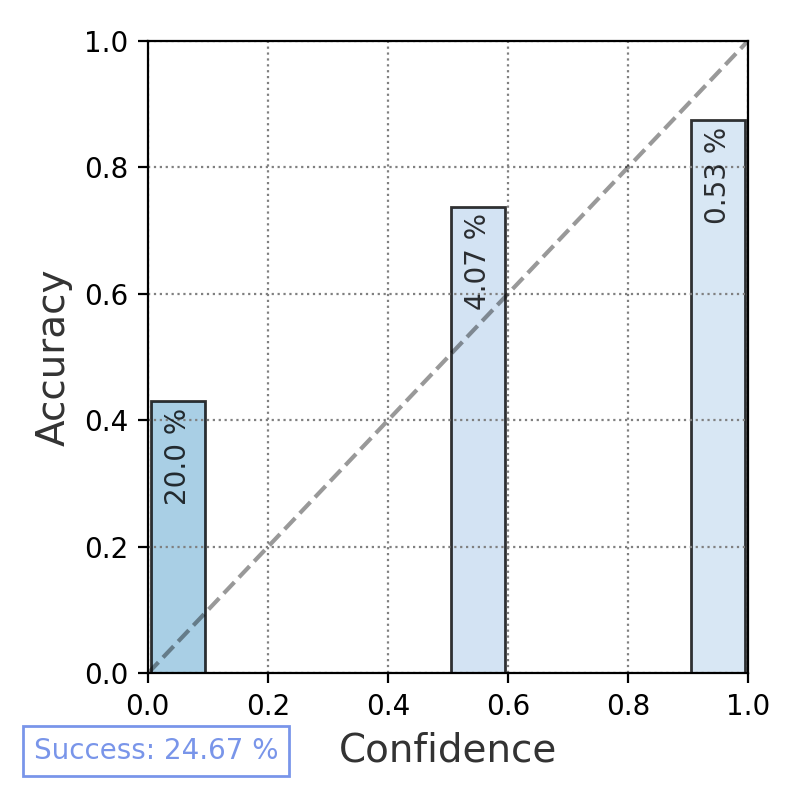}
        \caption{Verbalized Qual.\@}
        \label{subfig:seq-likelihood}
    \end{subfigure}
    \hfill
    \begin{subfigure}[t]{0.32\textwidth}
        \centering
        \includegraphics[width=\textwidth]{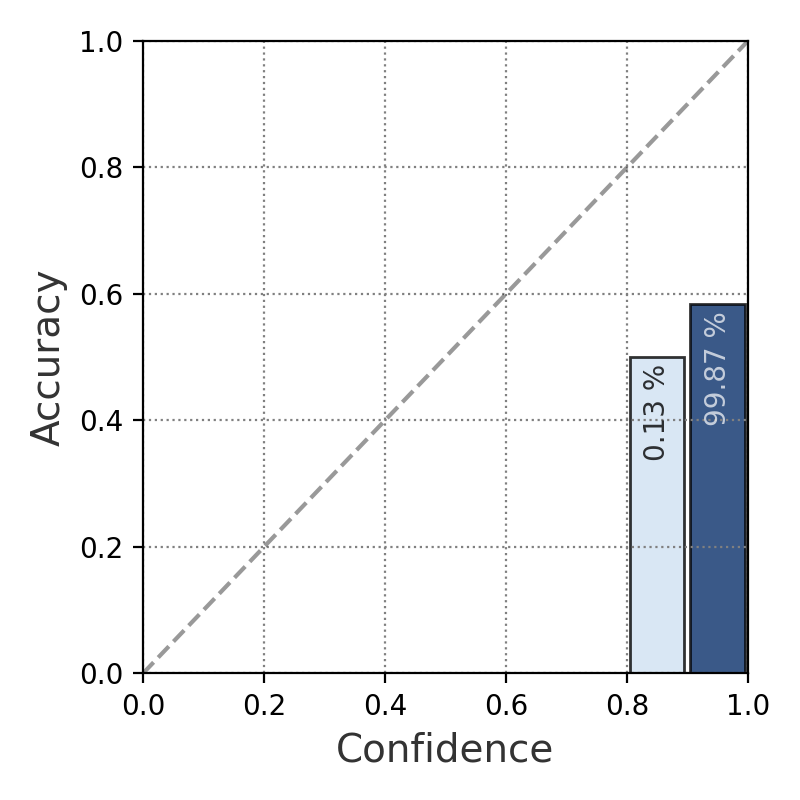}
        \caption{Verbalized $\%$.}
        \label{subfig:temperature-scaling}
    \end{subfigure}
    \hfill
    \begin{subfigure}[t]{0.32\textwidth}
        \centering
        \includegraphics[width=\textwidth]{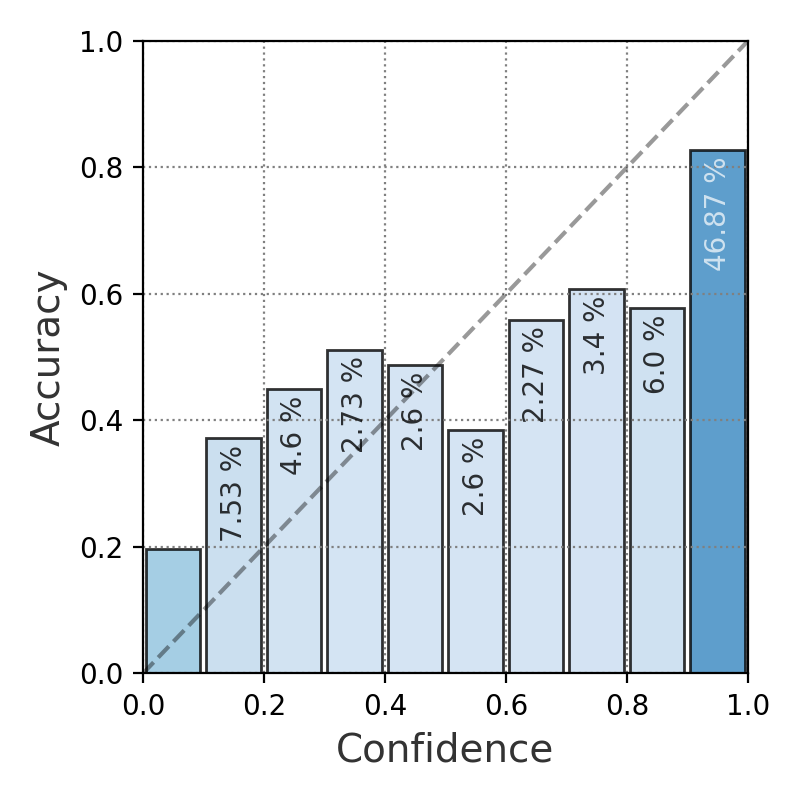}
        \caption{Auxiliary (binary).}
        \label{subfig:verbalized-percentage}
    \end{subfigure}
    \begin{subfigure}[t]{0.32\textwidth}
        \centering
        \includegraphics[width=\textwidth]{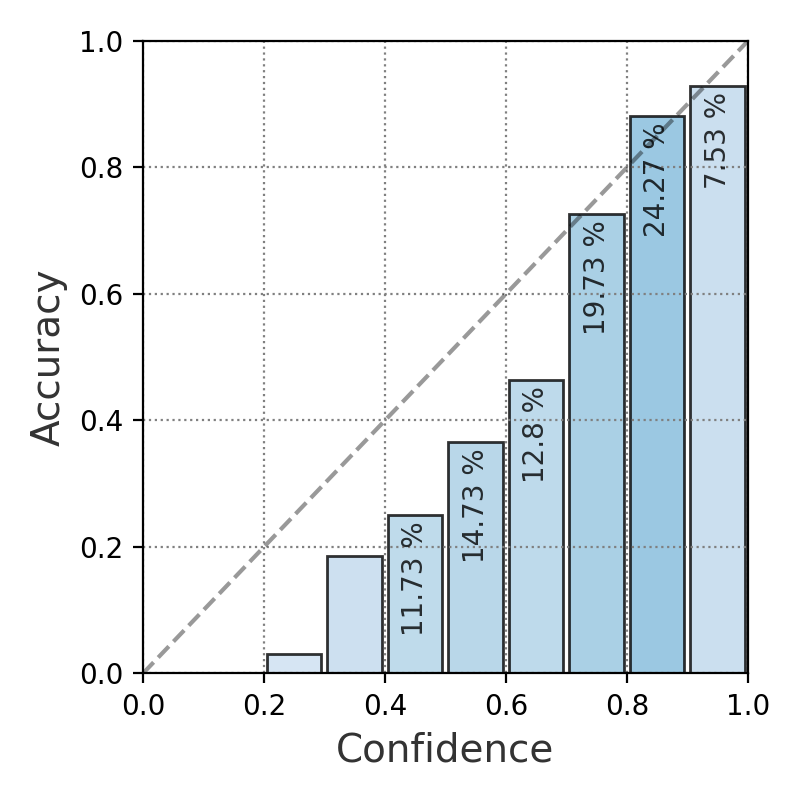}
        \caption{Auxiliary (clustering).}
        \label{subfig:verbalized-qualitative}
    \end{subfigure}
    \hfill
    \caption[Reliability diagrams for our different methods for Vicuna v1.5 on TriviaQA.]{Reliability diagrams for our different methods using $10$ bins each for Vicuna v1.5 on TriviaQA. The color as well as the percentage number within each bar indicate the proportion of total points contained in each bin.}\label{fig:reliabiliy-diagrams-vicuna-trivia-qa}
\end{figure*}

\paragraph{Results.} Vicuna v1.5 7B achieves $58 \%$ accuracy on TriviaQA and $44 \%$ on CoQA, while GPT-3.5 obtains $85 \%$ and $55 \%$ accuracy, respectively.\footnote{
    We use the same heuristic based on thresholded ROUGE-L\index{ROUGE} scores as in \cref{sec:case-study} or \citet{kuhn2023semantic} to determine whether an answer is correct. 
    Since GPT-3.5 is a closed-source model, it is hard to say whether the higher accuracy scores are due to better model quality, test data leakage, or overlap in questions in the case of TriviaQA \citep{merlo2021question}.} 
We present the calibration results in \cref{tab:calibration-results}. 
APRICOT \peach\@\index{APRICOT} achieves the highest AUROC\index{AUROC} in all settings and among the lowest Brier scores and calibration errors.
On the latter metric, verbalized confidence beats our method, but often at the cost of a higher worst-case calibration error and lower AUROC.\index{AUROC} 
The effect of CoT prompting on calibration, however, remains inconsistent across different baselines.
Lastly, APRICOT \peach\@\index{APRICOT} with clustering beats the use of binary targets for Vicuna v1.5 and GPT-3.5 on both TriviaQA and CoQA.
We also juxtapose reliability diagrams for the different methods for Vicuna v1.5 on TriviaQA in \cref{fig:reliabiliy-diagrams-vicuna-trivia-qa} (we show the other reliability diagrams, including for GPT-3.5, in \cref{app:additional-calibration}).
Here it becomes clear that verbalized uncertainties\index{Uncertainty!Verbalized} approaches usually do not emit a wide variety of confidence scores.\index{Confidence} 
This is in line with observations by \citet{zhou2023navigating}, who hypothesize the distribution of expressions generated by verbalized uncertainty heavily depend on the mention of e.g.\@ percentage values in the model's training data. 
While \cref{fig:reliabiliy-diagrams-gpt35-trivia-qa-full} shows that GPT-3.5 provides more variety in this regard, the overall phenomenon persists.\\

\begin{table*}[htb]
    \centering 
    \renewcommand{\arraystretch}{1.35}
    \resizebox{0.65\textwidth}{!}{
    \begin{tabular}{@{}lrrrr@{}}
        \toprule
         & \multicolumn{2}{c}{Vicuna v1.5} & \multicolumn{2}{c}{GPT-3.5} \\
         \cmidrule(lr){2-3} \cmidrule(lr){4-5} 
         Method & TriviaQA & CoQA & TriviaQA & CoQA \\
         \toprule
         Verb.\@ Qual.\@ & $.19$ & $.66$ & $1.00$ & $1.00$ \\
         Verb.\@ Qual.\@ (CoT) & $.25$ & $.73$ & $1.00$ & $1.00$ \\
         Verb.\@ \% & $1.00$ & $.99$ & $1.00$ & $1.00$ \\
         Verb.\@ \% (CoT) & $1.00$ & $.99$ & $.99$ & $.58$ \\
        \bottomrule
    \end{tabular}%
    }
    \caption[Consistency of verbalized uncertainty methods for Vicuna v1.5 and GPT-3.5 on TriviaQA and CoQA.]{Consistency of verbalized uncertainty methods for Vicuna v1.5 and GPT-3.5 on TriviaQA and CoQA.}\label{tab:consistency-verbalized-uncertainty}
\end{table*}

\paragraph{Consistency of Verbalized Uncertainty.}\index{Uncertainty!Verbalized}
While verbalized uncertainties often perform well according to calibration error, these results have to be taken with a grain of salt: 
Especially for the relatively small 7B Vicuna v1.5 model, the generations do not always contain the desired confidence expression, as visible by the low consistency in \cref{tab:consistency-verbalized-uncertainty}.
CoT prompting\index{Prompting!Chain-of-thought} seems to increase the success rate of verbalized uncertainty, and the additional results on GPT-3.5 suggests that this ability might also be dependent on model size.
But even when taking the generated confidence expression, their ability to distinguish potentially correct from incorrect LLM responses remains at or close to random level.
This suggests that due to the skewed distribution of confidence expressions, they can only be well-calibrated on datasets which are easy for the underlying model, which, naturally, is not known a priori.
%In addition, the qualitative verbalized uncertainty does not work reliably for the smaller Vicuna v1.5 model.
Next, we conduct some additional analyses based on the clustering-based variant of our method.\\

\subsection{Ablation Study}\label{sec:calibrator-inputs-ablation}

\begin{table*}
    \renewcommand{\arraystretch}{1.5}
    \centering
    \resizebox{.975\textwidth}{!}{
    \begin{tabular}{@{}lcccccccccccc@{}}
        \toprule
         & \multicolumn{4}{c}{Auxiliary Model Input}  & \multicolumn{4}{c}{TriviaQA} & \multicolumn{4}{c}{CoQA} \\
         \cmidrule(lr){2-5} \cmidrule(lr){6-9} \cmidrule(lr){10-13}
        & Quest. & Ans. & CoT & Verb. & Brier$\downarrow$ & ECE$\downarrow$  & smECE$\downarrow$ & AUROC$\uparrow$ & Brier$\downarrow$ & ECE$\downarrow$  & smECE$\downarrow$ & AUROC$\uparrow$  \\
        \toprule
        \multirow{7}{*}{\rotatebox{90}{Vicuna v1.5 (white-box)}} 
        & \gcheck & \rcross & \rcross & \rcross & $.21{\scriptstyle\ \pm.00}$& $\mathbf{.07}{\scriptstyle\ \pm.01}$& $\underline{\mathbf{.06}}{\scriptstyle\ \pm.01}$& $.74{\scriptstyle\ \pm.01}$ & $.22{\scriptstyle\ \pm.00}$& $\underline{\mathbf{.03}}{\scriptstyle\ \pm.01}$& $\mathbf{.03}{\scriptstyle\ \pm.00}$& $.70{\scriptstyle\ \pm.01}$ \\
        & \gcheck & \gcheck & \rcross & \rcross & $\mathbf{.18}{\scriptstyle\ \pm.00}$& $.09{\scriptstyle\ \pm.01}$& $.09{\scriptstyle\ \pm.01}$& $\underline{\mathbf{.83}}{\scriptstyle\ \pm.01}$ & $\underline{\mathbf{.18}}{\scriptstyle\ \pm.00}$& $.04{\scriptstyle\ \pm.01}$& $.04{\scriptstyle\ \pm.01}$& $\underline{\mathbf{.82}}{\scriptstyle\ \pm.01}$ \\
        & \gcheck & \gcheck & \rcross & {\color{green!55!black}\textbf{Qual.}} & $\mathbf{.18}{\scriptstyle\ \pm.00}$& $.08{\scriptstyle\ \pm.01}$& $.08{\scriptstyle\ \pm.01}$& $.82{\scriptstyle\ \pm.01}$ & $.19{\scriptstyle\ \pm.00}$& $.04{\scriptstyle\ \pm.01}$& $.04{\scriptstyle\ \pm.01}$& $.79{\scriptstyle\ \pm.01}$ \\
        & \gcheck & \gcheck & \rcross & {\color{green!55!black}$\mathbf{\%}$} & $\mathbf{.18}{\scriptstyle\ \pm.00}$& $\mathbf{.07}{\scriptstyle\ \pm.01}$& $.07{\scriptstyle\ \pm.01}$& $.82{\scriptstyle\ \pm.01}$ & $\mathbf{.18}{\scriptstyle\ \pm.00}$& $\mathbf{.03}{\scriptstyle\ \pm.01}$& $\mathbf{.03}{\scriptstyle\ \pm.01}$& $.80{\scriptstyle\ \pm.01}$ \\
        & \gcheck & \gcheck & \gcheck & \rcross & $.19{\scriptstyle\ \pm.01}$& $\mathbf{.07}{\scriptstyle\ \pm.01}$& $.07{\scriptstyle\ \pm.01}$& $.80{\scriptstyle\ \pm.01}$  & $.21{\scriptstyle\ \pm.00}$& $.04{\scriptstyle\ \pm.01}$& $\mathbf{.03}{\scriptstyle\ \pm.01}$& $.74{\scriptstyle\ \pm.01}$ \\
        & \gcheck & \gcheck & \gcheck & {\color{green!55!black}\textbf{Qual.}} & $.19{\scriptstyle\ \pm.00}$& $.08{\scriptstyle\ \pm.01}$& $.08{\scriptstyle\ \pm.01}$& $.80{\scriptstyle\ \pm.01}$ & $.22{\scriptstyle\ \pm.00}$& $\mathbf{.03}{\scriptstyle\ \pm.01}$& $\mathbf{.03}{\scriptstyle\ \pm.01}$& $.70{\scriptstyle\ \pm.01}$ \\
        & \gcheck & \gcheck & \gcheck & {\color{green!55!black}$\mathbf{\%}$} & $\mathbf{.18}{\scriptstyle\ \pm.00}$& $\mathbf{.07}{\scriptstyle\ \pm.01}$& $.07{\scriptstyle\ \pm.01}$& $.81{\scriptstyle\ \pm.01}$ & $.20{\scriptstyle\ \pm.00}$& $\mathbf{.03}{\scriptstyle\ \pm.01}$& $\mathbf{.03}{\scriptstyle\ \pm.00}$& $.75{\scriptstyle\ \pm.01}$ \\
        \midrule
        \multirow{7}{*}{\rotatebox{90}{GPT-3.5 (black-box)}} 
         & \gcheck & \rcross & \rcross & \rcross & $\mathbf{.12}{\scriptstyle\ \pm.01}$& $.05{\scriptstyle\ \pm.01}$& $.05{\scriptstyle\ \pm.01}$& $.71{\scriptstyle\ \pm.03}$ & $.21{\scriptstyle\ \pm.00}$& $.03{\scriptstyle\ \pm.01}$& $.03{\scriptstyle\ \pm.01}$& $.72{\scriptstyle\ \pm.01}$ \\
        & \gcheck & \gcheck & \rcross & \rcross & $\mathbf{.12}{\scriptstyle\ \pm.01}$& $.06{\scriptstyle\ \pm.01}$& $.06{\scriptstyle\ \pm.01}$& $.72{\scriptstyle\ \pm.02}$ &  $\underline{\mathbf{.18}}{\scriptstyle\ \pm.01}$& $.04{\scriptstyle\ \pm.02}$& $.04{\scriptstyle\ \pm.02}$& $\underline{\mathbf{.82}}{\scriptstyle\ \pm.02}$ \\
        & \gcheck & \gcheck & \rcross & {\color{green!55!black}\textbf{Qual.}} &  $\mathbf{.12}{\scriptstyle\ \pm.01}$& $\mathbf{.03}{\scriptstyle\ \pm.01}$& $\mathbf{.03}{\scriptstyle\ \pm.01}$& $.72{\scriptstyle\ \pm.03}$ & $\mathbf{.18}{\scriptstyle\ \pm.01}$& $\mathbf{.02}{\scriptstyle\ \pm.01}$& $\mathbf{.02}{\scriptstyle\ \pm.00}$& $.80{\scriptstyle\ \pm.01}$ \\
        & \gcheck & \gcheck & \rcross & {\color{green!55!black}$\mathbf{\%}$} & $\mathbf{.12}{\scriptstyle\ \pm.01}$& $\underline{\mathbf{.03}}{\scriptstyle\ \pm.01}$& $\underline{\mathbf{.03}}{\scriptstyle\ \pm.01}$& $.72{\scriptstyle\ \pm.02}$ & $\mathbf{.18}{\scriptstyle\ \pm.00}$& $.04{\scriptstyle\ \pm.01}$& $.03{\scriptstyle\ \pm.00}$& $.80{\scriptstyle\ \pm.01}$ \\ 
        & \gcheck & \gcheck & \gcheck & \rcross & $\mathbf{.12}{\scriptstyle\ \pm.01}$& $.06{\scriptstyle\ \pm.01}$& $.06{\scriptstyle\ \pm.01}$& $.72{\scriptstyle\ \pm.02}$ & $.21{\scriptstyle\ \pm.00}$& $.03{\scriptstyle\ \pm.01}$& $.03{\scriptstyle\ \pm.01}$& $.72{\scriptstyle\ \pm.01}$ \\
        & \gcheck & \gcheck & \gcheck & {\color{green!55!black}\textbf{Qual.}} & $\mathbf{.12}{\scriptstyle\ \pm.01}$& $.04{\scriptstyle\ \pm.01}$& $.04{\scriptstyle\ \pm.01}$& $\underline{\mathbf{.73}}{\scriptstyle\ \pm.02}$ & $.21{\scriptstyle\ \pm.00}$& $.04{\scriptstyle\ \pm.01}$& $.04{\scriptstyle\ \pm.01}$& $.72{\scriptstyle\ \pm.01}$ \\
        & \gcheck & \gcheck & \gcheck & {\color{green!55!black}$\mathbf{\%}$} & $\mathbf{.12}{\scriptstyle\ \pm.01}$& $.04{\scriptstyle\ \pm.01}$& $.04{\scriptstyle\ \pm.01}$& $.64{\scriptstyle\ \pm.02}$ & $.21{\scriptstyle\ \pm.00}$& $\mathbf{.02}{\scriptstyle\ \pm.01}$& $\mathbf{.02}{\scriptstyle\ \pm.00}$& $.72{\scriptstyle\ \pm.01}$  \\
        \bottomrule
    \end{tabular}%
    }
    \caption[Calibration results for Vicuna v1.5 and GPT-3.5 on TriviaQA and CoQA using the auxiliary (clustering) method.]{Calibration results for Vicuna v1.5 and GPT-3.5 on TriviaQA and CoQA using the auxiliary (clustering) method. We bold the best results per dataset, method and model.
    %and underline those that are statistically significant compared to all other results assessed via the ASO test \citep{del2018optimal, dror2019deep, ulmer2022deep} with $\tau = 0.5$ and a confidence level of $\alpha = 0.1$.
    }\label{tab:calibration-results-features}
\end{table*}

%\paragraph{Ablation results.} We show the results of the different ablations in \cref{tab:calibration-results-features}.
%For both LLMs, we can see that the auxiliary model is able to achieve already decent scores by using the question as input alone. This suggest that the calibrator is learning about the general difficulty of questions for the LLM.
%However, both cases also show that including more information---the LLM's answer, CoT reasoning or verbalized uncertainties---can help to improve the auxiliaries model calibration and misprediction AUROC even further, even though the effect remains somewhat inconsistent across models.

The previous results pose the question of which parts of input the auxiliary model actually learns from.
So, analogous to the different prompting strategies in \cref{fig:prompting-methods}, we explore different input variants:
First, we test a question-only setting, where the target LLM's answer is omitted completely.
%We suspect that the auxiliary model might already be able to learn a rough sense of difficulty for the LLM from the question alone.
We also test the performance of the calibrator when given more information, for instance the model answer with and without chain-of-thought prompting\index{Prompting!Chain-of-thought}, which could potentially expose flaws in the LLM's response.\footnote{Based on the recent study by \citet{turpin2023language}, we assume that CoT does \emph{not} expose the LLM's actual reasoning. Nevertheless, it provides more context about the given answer.}
Finally, we also expose the verbalized uncertainty of the LLM to the calibrator.\index{Uncertainty!Verbalized}

\paragraph{Results.} We show these results in \cref{tab:calibration-results-features} in \cref{app:additional-calibration}. 
Interestingly, we can observe that even based on the question to the LLM alone, APRICOT \peach\@ can already achieve respectable performance across all metrics.
This suggests that the calibrator at least partially learns to infer the difficulty of the LLM answering a question from the type of question alone.
Nevertheless, we also find that adding the LLM's actual answer further improves results, with additional gain when using CoT prompting.\index{Prompting!Chain-of-thought}
In some cases, the calibration error can be improved when using the LLM's verbalized uncertainties; in this sense, we can interpret the role of the calibrator as mapping the model's own assessment to a calibrated confidence score.

\section{Discussion}

Despite the difficulty of predicting the LLM's confidence\index{Large language model}\index{Confidence} from its generated text alone, our experiments have shown that APRICOT \peach\@\index{APRICOT} can be used to produce reasonable scores even under these strict constraints.
We showed in the past sections that the auxiliary model can be finetuned to learn from multiple signals.
On the one hand, the auxiliary calibrator learns a mapping from a latent category of question to the expected difficulty for a target LLM.
On the other hand, including the answer given through CoT prompting\index{Prompting!Chain-of-thought} and including the LLM's own assessment of its uncertainty helped to further improve results.
While sometimes beaten in terms of calibration error, our method consistently outperforms our baselines in error detection AUROC\index{Error detection}\index{AUROC}, meaning that it can provide the best signal to detect wrong LLM answers.
Compared to other approaches, this yields some desirable properties: 
APRICOT \peach\@\index{APRICOT} is available when sequence likelihood\index{Likelihood!Sequence} is not; it is more reliable than verbalized uncertainty\index{Uncertainty!Verbalized}; and it only needs a light finetuning once, adding negligible inference overhead.
Compared to other methods such as \citet{kuhn2023semantic, lin2023generating} in \cref{sec:uncertainty-nlp}, it also does not require more generations for the same input, reducing the more expensive LLM inference costs.

\paragraph{Limitations.} While yielding generally positive results in our case, the clustering methodology from \cref{sec:setting-calibration-targets} requires access to a sufficiently expressive sentence embedding model and a large enough number of data points.
When this is not given, we show that the binary approach---tuning the auxiliary model to predict errors---is a viable alternative.
As any neural model, the auxiliary calibrator is vulnerable to distributional shift\index{Shift!Distributional} and out-of-distribution data\index{Out-of-distribution data}.
Further research could help to understand how this issue can be reduced and which parts of the input the model identifies to predict confidence scores in order to unveil potential shortcut learning \citep{du2023shortcut}. 
Our experiments focused on open-ended question-answering\index{Question-answering} tasks, which provide a fairly easy way to check answer correctness.
In other types of language generation\index{Natural language generation} such as summarization\index{Text summarization}, translation\index{Machine translation} or open text generation, this notion of correctness is not given.
%However, we point out the relation of our approach to other NLP tasks in \cref{sec:related-work}, which might provide an opening for future research.

\section{Summary}

In this chapter, we presented APRICOT \peach\index{APRICOT}, a general method to obtain confidence scores from any language model on the input and text output alone. 
We showed that it is possible to compute calibration targets through the clustering of question embeddings.\index{Calibration}
Through the subsequent finetuning of a smaller language model, we then outperform other methods to distinguish incorrect from correct answers with competitive calibration scores, on different models and datasets.
While we only presented a first, more fundamental version this approach in this work, it lends itself naturally to a whole body of research that aims to improve the calibration of pretrained language models \citep{desai2020calibration, jian2021how, chen2023close}. 
Lastly, future studies might also investigate the uncertainty\index{Uncertainty} of the auxiliary model itself and use techniques such as conformal prediction\index{Conformal prediction} in \cref{sec:frequentist-neural-networks} to produce estimates of LLM\index{Large language model} confidence \emph{intervals}.

% Chapter X

\chapter{Discussion} % Chapter title

\label{ch:discussion} % For referencing the chapter elsewhere, use \autoref{ch:name} 

\begin{tikzpicture}[remember picture,overlay]
    \node[anchor=north,inner sep=0pt] at (current page text area.north) {\includegraphics[width=\linewidth, clip=true, trim = 8cm 75cm 8cm 50cm]{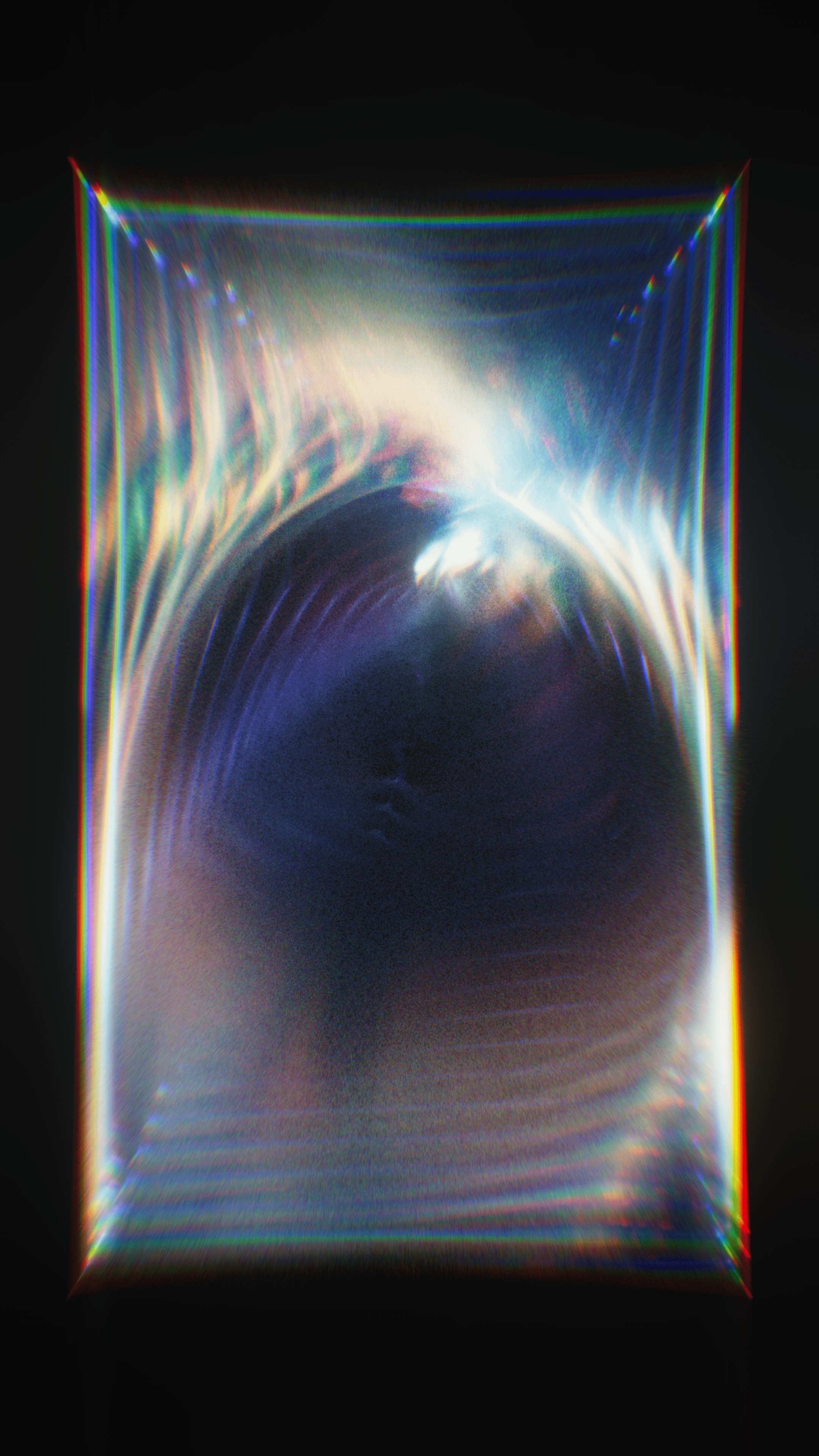}};
\end{tikzpicture}

\epigraph{
    ``\emph{When Sha Monk opened up a scroll of scripture that the other two disciples were clutching, his eyes perceived only snow-white paper without a trace of so much as half a letter on it. 
    Hurriedly he presented it to Tripitaka, saying, `Master, this scroll is wordless!' 
    Pilgrim also opened a scroll and it, too, was wordless. 
    Then Eight Rules opened still another scroll, and it was also wordless.
    `Open all of them!' cried Tripitaka. 
    Every scroll had only blank paper.}''
}
{
    ---\emph{The Journey to the West} (\begin{CJK*}{UTF8}{gbsn}西游记\end{CJK*}), Ch.\@ 94, as translated and edited by Anthony C.\@ Yu (1977). 
}

The last chapters have explored the various different definitions of and perspectives on uncertainty\index{Uncertainty} and how they materialize in the fields of machine learning\index{Machine learning} and natural language processing\index{Natural language processing}.
Despite the usefulness of uncertainty quantification\index{Uncertainty quantification} for a whole spectrum of applications (\cref{sec:applications-uncertainty}) and its importance to avoid negative outcomes and to build trust\index{Trust} in automation (\cref{sec:uncertainty-trust}),
a somewhat fractured research landscape emerges:
Uncertainty still remains a very under-defined and under-researched topic, especially in natural language processing.
Uncertainty within the experimental pipeline often stays unaddressed or outright ignored;
Uncertainty modeling poses a challenge under the current large language model\index{Large language model} paradigm and the successes and failures of uncertainty quantification\index{Uncertainty quantification} are equally poorly understood.
The efforts described in \cref{ch:uncertainty-experimental-design,ch:uncertainty-classification,ch:uncertainty-generation,ch:uncertainty-llms} can only work as a step to mitigate this fact,
and thus dedicate this chapter to revisit the initial research goals defined in this thesis, and discuss a number of fundamental open questions and research directions.

\section{Discussion of Research Questions}\label{sec:research-goals-discussion}

This thesis gave an overview over different notions of uncertainty from the perspectives of statistics\index{Statistics}, linguistics\index{Linguistics}, deep learning\index{Deep learning} and NLP\index{Natural language processing} in \cref{ch:background}, also discussing how uncertainty can be communicated and how it interacts with human-AI trust\index{Trust!Human-AI}.
The influence of uncertainty on the experimental pipeline was analyzed in \cref{ch:uncertainty-experimental-design}, where we could see how more careful experimental design\index{Experimental design} allows to quantify uncertainty in results, reduce it, and even open up new avenues for modeling it.
Some of the limits of uncertainty quantification for text classification\index{Text classification} were demonstrated in \cref{ch:uncertainty-classification} using the theoretical case of ReLU\index{ReLU} networks and a large variety of different models applied to text classification tasks in English, Danish and Finnish.
Lastly, non-exchangeable conformal prediction\index{Conformal prediction} enables us to develop a method to obtain calibrated token sets for generation in \cref{ch:uncertainty-generation} and APRICOT \peach\@, a method to obtain calibrated confidence scores from black-box LLMs in \cref{ch:uncertainty-llms}.
Based on this research, we now return to the research questions posed in \cref{sec:objective} and discuss them in turn.

\begin{researchquestion}\label[researchquestion]{rq1}
    \emojimangnifyingglass\textbf{RQ1}: How can uncertainty in NLP be characterized?
\end{researchquestion}

In \cref{ch:background} we discussed the multi-faceted views on uncertainty from a variety of perspectives, all of which coalesce in modern NLP applications.
This includes the linguistic uncertainties present in the input data, interacting with the statistical uncertainties lingering in the modeling aspect.\\

Linguistically, uncertainty materializes as an inherent property of language in the form of underspecification\index{Underspecification}, ambiguity\index{Ambiguity} and vagueness\index{Vagueness} (\cref{sec:uncertainty-linguistics}), but also as a tool for humans to express their state of knowledge about the world (\cref{sec:expressing-uncertainty}; 
this can also be used by language models to communicate uncertainty, see \cref{sec:communicating-uncertainty}).
Statistically, uncertainty is treated differently in the frequentist and Bayesian school of thought:
Frequentists\index{Statistics!Frequentist} see probabilities as the relative frequency of an event under continued repetitions of an experiment.
Bayesians\index{Statistics!Bayesian} interpret them as a degree of belief, with the parameter of interest turning from an unobserved constant into a random variable.
Both perspectives are echoed in the corresponding neural approaches:
Calibration\index{Calibration} techniques and conformal prediction\index{Conformal prediction} on the one hand allow us to create confidence scores that reflect the correctness of the model, or prediction sets\index{Prediction set} contain the ground truth in expectation.
Approximating the neural weights posterior or parameterizing higher-order distributions on the other hand permit a decoupling of different notions of uncertainty.\\

The latter notions mostly refer to \emph{predictive} uncertainty and are for example quantified in terms such as the total, data\index{Uncertainty!Aleatoric}, model\index{Uncertainty!Epistemic} and distributional uncertainties\index{Uncertainty!Distributional}.
As \citet{baan2023uncertainty} point out, these can be seen as a spectrum, in contrast to a fixed set of discrete categories.
This means that steps like data collection can be a source of model uncertainty when data is scarce, and can be reduced when more data is collected.
However it can also produce data uncertainty which, in some instances, can be reduced through e.g.\@ better annotation guidelines.
In this light, the choice of method can be informed by the kind of uncertainty most useful to the problem at hand, and if necessary and possible, the experimental pipeline can be adapted to reduce uncertainty further or to enable better modeling of it (see next \cref{rq2}).
For active learning for instance, we might care most about epistemic or distributional uncertainty and therefore refer to Bayesian or evidential methods, while for error detection we might be satisfied with easy-to-implement estimators of total uncertainty\index{Uncertainty!Total}.\\

It should be noted though that almost all methods discussed so far quantify uncertainty statistically rather than linguistically.
While verbalized uncertainty\index{Uncertainty!Verbalized} (\cref{sec:communicating-uncertainty}) is a step towards expressing uncertainty in words, it (thus far) ignores the rich shades of meaning that are at a human speaker's disposal (\cref{fig:uncertainty-taxonomy}).
Communicating uncertainty to humans can be challenging (\cref{sec:uncertainty-trust}), so more natural verbalized uncertainty could prove to be a fruitful avenue of research.

\begin{researchquestion}\label[researchquestion]{rq2}
    \emojimangnifyingglass\textbf{RQ2}: How can choices in experimental design help to reduce and quantify uncertainty?
\end{researchquestion}

In \cref{ch:uncertainty-experimental-design}, we discussed the role of uncertainty in experimental design in NLP.\index{Experimental design}
There, we argued that careful data collection can help to reduce uncertainty caused by noise, and enable new modeling options through multiple annotations.
Furthermore, hypothesis testing\index{Hypothesis testing} can help to quantify the uncertainty in results and aid model selection.\\

Uncertainty manifests in different stages of the experimental process and is often overlooked outside of the modeling stage;
however, steps that are undertaken to increase reproducibility can help to rein in uncertainty and open modeling options.
In NLP, this is exemplified by publishing all instance annotations (instead of an aggregate) and embracing human disagreements which arise from the ambiguities\index{Ambiguity} in language (\cref{sec:uncertainty-nlp}; \citealp{plank2022problem, baan2023uncertainty}).
As we discuss in \cref{sec:open-questions}, this could for instance be combined with recent advances in evidential deep learning\index{Deep learning!Evidential} to learn higher-order distributions (\cref{sec:evidential-neural-networks}).\\

Additionally, comparing different models, prompts or other settings can be difficult due to the non-linear nature of neural network and their increasing model sizes.
In \cref{sec:aso}, we showed how to quantify this uncertainty in modeling results using the ASO test.\index{Stochastic order!Almost}
As the test is non-parametric, we do not require any knowledge of the underlying distribution of scores.
In the case study in \cref{sec:case-study}, we furthermore demonstrated that even though modern LLMs\index{Large language model} tend to be pretrained, monolithic models, we can perform statistical hypothesis testing\index{Hypothesis testing} by obtaining observations from different prompts and thereby assessing their robustness \citep{mizrahi2023state, sclar2023quantifying}.
We also formalized the different distributions that are compared---in the LLM setting for instance, we keep the model architecture, pretraining data and hyperparameters constant while varying other factors such as prompt design and generation hyperparameters.\index{Prompting}
Generally speaking, all of these settings vary a certain number of variables on which the output is conditioned on, while keeping others fixed.
Although many variations of this setup are plausible, we believe it is important to make underlying assumptions more explicit and vary as many variables as feasible in order to arrive at a well-rounded estimate of model performance.

\begin{researchquestion}\label[researchquestion]{rq3}
    \emojimangnifyingglass\textbf{RQ3}: How do inductive model biases influence\\ uncertainty quantification?
\end{researchquestion}

Inductive biases\index{Inductive bias} describe the modeling assumptions present in a model's architecture and training procedure.
As we saw in \cref{ch:uncertainty-classification}, this can have unintuitive effects on the efficacy of uncertainty estimates, where models may act confidently when faced with OOD\index{Out-of-distribution data} inputs.\\

Many methods for uncertainty quantification\index{Uncertainty quantification} equip a model with some sort of metric that operates on the model's output and translates it into a usually scalar measure of its uncertainty.
While these have some expected or desired behaviors---such as the predictive entropy\index{Entropy!Predictive} being high on OOD data---this is often not true in practice.
This was illustrated for instance using ReLU\index{ReLU} networks in \cref{sec:uq-classification-pitfalls}:
Due to the inductive bias of the architecture, the network induces linear decision regions in the feature space, leading uncertainty metrics to provably converge to fix points in the limit (instead of being sensitive to the degree of familiarity with an input).\\

One might criticize the argument about ReLU networks for being too simplistic, since modern deep learning architecture are much more complex;
and while it is true that this fact prevents similar proofs, we empirically identified similar problems on a large variety of text classification models in \cref{sec:benchmarking-nlp-uncertainty}.
We explicitly tested a low-resource setting (simulated for English), where training data is scarce and behavior on OOD might be unreliable.
By testing on OOD test sets, we could show that similar failures occur in practice and that uncertainty measures are unable to effective distinguish in-distribution from foreign data.\\

How can we explain this behavior?
One possible hypothesis is to look at this problem through the lens of the information bottleneck principle\index{Information bottleneck principle} \citep{tishby2000information, tishby2015deep, saxe2019information}:
Neural predictors often map input representations into lower-dimensional latent spaces.
This way, they are incentivized during training firstly to recover the correct prediction, and secondly  to compress the input in a way that supports the first goal.
Intuitively, we can assume that this learned compression will favor features that are most useful to the predictive task, not necessarily ones that are useful to indicate uncertainty.
Indeed, some works in anomaly detection\index{Anomaly detection} have noted that neural models might fail to encode novel, unseen features that might indicate that a test point is out-of-distribution\index{Out-of-distribution data} \citep{dietterich2022familiarity, sivaprasad2023going}.
In addition, other works have noted how in- and out-of-distribution features overlap in latent space \citep{van2021feature}.
But these features are exactly what should indicate model uncertainty\index{Uncertainty!Epistemic}, since the model is likely to be misspecified on points different from the training distribution!
This means that this dynamic might make uncertainty quantification unreliable in cases where we cannot obtain good estimates of epistemic uncertainty, or where epistemic uncertainty accounts for a large portion of the total uncertainty.
In the theoretical analysis in \cref{sec:uq-classification-pitfalls}, uncertainty estimates can still be useful in regions of class overlap (hence, aleatoric uncertainty\index{Uncertainty!Aleatoric}), but fail to be informative in regions without model training data due to their convergence to fix points.
In the empirical study in \cref{sec:benchmarking-nlp-uncertainty}, we observe that the quality of uncertainty estimates can decrease as we add more training data, potentially due to the selective compression phenomenon.
From this we can deduce that the inductive biases\index{Inductive bias} of standard architectures are insufficient for reliable uncertainty quantification\index{Uncertainty quantification}, and better inductive biases are needed.\\

One possible solution of this lies in directly modeling the data density.
Language models do this already by assigning probabilities to entire sequences; 
however, \cref{sec:calibration-experiments} and \citet{kumar2019calibration} showed that sequence likelihoods\index{Likelihood!Sequence} are insufficient for error prediction, and other studies such as \citet{ren2022out} have demonstrated their failure on OOD detection.
This can be explained by the fact that language models are trained on only a single sequence in a combinatorically large space of possible continuations.
This automatically implies a sort of data scarcity\index{Data scarcity}, where the model fails to adequately capture the paraphrasticity\index{Paraphrasticity} of language (see \cref{sec:uncertainty-linguistics}).
\citet{lebrun2022evaluating} discovered how language models tend to overestimate the probability of frequent sequences and underestimate the ones coming from the tail end of the sequence distribution, with similar findings by \citet{ilia2024predict, liu2024infini}.\footnote{
    This phenomenon might also be the culprit behind the inadequacy of sampling from the mode in NLG, see for instance \citet{eikema2020map, holtzman2020curious, eikema2024effect}.
}\\

As another approach to better inductive biases for UQ\index{Inductive bias}, one might choose to model the distribution of latent representations instead. 
This is for instance done through normalizing flows\index{Normalizing flow} in the case of posterior networks\index{Posterior network} (\cref{sec:evidential-neural-networks}) or some methods regarding direct uncertainty prediction (\cref{sec:other-approaches}).
But since these components are trained on the latent encodings of an underlying model, they can only learn the distribution of latent features that are learned by the main model, and might thus fall into the same trap of not modeling features indicative of model uncertainty that were ``compressed away''\index{Uncertainty!Epistemic}.
This can explain why the DDU Bert in \cref{sec:predictive-uncertainty-experiments} does not attain its best results on OOD detection through the log probability of its latent density estimator, and why posterior networks\index{Posterior network} have been shown to not always detect OOD reliably \citep{kopetzki2021evaluating}.

\begin{researchquestion}\label[researchquestion]{rq4}
    \emojimangnifyingglass\@
    \textbf{RQ4}: How can we address some of the challenges of uncertainty quantification in NLP?\@
\end{researchquestion}

In this thesis, we addressed multiple of the challenges that we laid out in \cref{sec:context}, including data scarcity and sequentiality.
For clarity, we will discuss them here in turn and the corresponding insights gained from this work.

\paragraph{Challenges of Natural Language.} 
In this thesis, we mainly worked towards solving two of the challenges that come with natural language data, namely its diversity and sequentiality.
On the one hand, \cref{sec:benchmarking-nlp-uncertainty} tested different uncertainty methods for text classification\index{Text classification} on three different languages and OOD test sets that introduce novel domains.
While general trends are visible across all settings, we can also see that the best uncertainty quality in terms of model and corresponding metric differs across datasets.
This suggests that there might be complex underlying interactions between the model and the types of uncertainty that OOD data evokes in it, the uncertainty quantification method, and language-specific characteristics.\footnote{
    The ability to model linguistic idiosyncrasy's can to some degree also be influenced by the quality of tokenization and therefore the models' uncertainty.
    For investigation into the first point, refer e.g.\@ to \citet{graen2018cutter, virtanen2019multilingual, singh2019bert, rust2020good, pfeiffer2020unks, mielke2021between, maronikolakis2021wine}.
}
For the non-exchangeable conformal language generation\index{Non-exchangeable conformal language generation} in \cref{ch:uncertainty-generation}, we also tested on German and Japanese as different source language for the machine translation\index{Machine translation} task.
We measured coverage\index{Coverage}, namely whether conformal prediction sets\index{Prediction set} contain the ground truth continuation, and translation quality, but found only minor differences between languages, with similar trends across tested methods.
Importantly, this method addresses the sequentiality issue in natural language:
Even though it is possible to conformalize language generation\index{Natural language generation} on a sequence-level where the i.i.d.\@ assumption is maintained (see \citealp{quach2023conformal}), we were able to provide a method on a token-level that provides a well-motivated framework. 
This is different compared to cases like \citet{ravfogel2023conformal}, who operate on a token-level but have to make strong assumptions about the underlying data that might not be realistic in practice.

\paragraph{Data Scarcity.} \index{Data scarcity}
In \cref{sec:benchmarking-nlp-uncertainty}, we explicitly tested low-resource settings by using under-resourced languages\index{Low-resource language} such as Finnish and Danish, and by testing the relationship between training set size and uncertainty quality.
Unsurprisingly, we showed that task performance increases with the amount of data.
More surprisingly, we showed that increased amount of training data can have adverse effects on uncertainty quality on OOD\index{Out-of-distribution data} inputs, for possible reasons we discussed in the answer for \cref{rq3}.

\paragraph{Trust \& Safety.} 
Firstly, this thesis introduced non-exchangeable conformal language generation\index{Non-exchangeable conformal language generation} in \cref{ch:uncertainty-generation}, which provides a way to produce sets of token for generation with conformal guarantees.
Similarly to standard prediction sets\index{Prediction set} in \cref{sec:frequentist-neural-networks}, other ways of truncating the predictive distribution\index{Predictive distribution} over tokens do not provide any guarantees of containing the correct continuation.
Nevertheless, these prediction sets can be conformalized through our calibration method that utilizes information from nearest neighbors from a datastore. 
Not only does the generation process now (approximately) fulfill conformal guarantees, this also opens up new possibilities through the extension of (non-exchangeable) conformal risk control\index{Conformal risk control} \citep{angelopoulos2022conformal, farinhas2024nonexchangeable}:
Future approaches could provide bounds on a wider family of functions, more instance measuring toxicity, veracity or alignment with human values, similar to the works of \citet{mohri2024language, gui2024conformal}.
The latter has already been explored as an on-the-fly procedure (albeit, not conformal) instead of an additional finetuning stage \citep{yang2021fudge, qin2022cold, mudgal2023controlled, gao2024linear}.
The fact that conformal methods can provide statistical guarantees for otherwise unwieldy language models has also spurred additional work on the subject, for instance conformalizing generation on a sequence-level \citep{quach2023conformal}, for prompt selection \citep{zollo2023prompt}, conditional computation \citep{schuster2022confident, ren2023robots}, planning for LLM agents \citep{liang2024introspective}, and for black-box models \citep{su2024api}.
Secondly, for the most restrictive setup in which we are dealing with a black-box LLM\index{Large language model} and only have access to its input and generated text, we proposed APRICOT \peach\@ in \cref{ch:uncertainty-llms}.
We demonstrated that even in this context, using a secondary auxiliary model enables us to predict the target LLMs confidence\index{Confidence} reliably.
We also showed that by clustering the latent presentation of inputs, we can use these clusters to obtain more fine-grained information about the expected performance of the LLM on a certain category of inputs.
While we leave further exploration of this question to future work, it is intuitive to assume that this very extreme setup has limits on the reliability of confidence estimates.
In this way, we can order different methods on a spectrum from full access to the model, including latent representations, to access to logits and the predictive distribution to text-only access.
Some works have found that OOD inputs are detectable based on the model's hidden representations \citep{yoo2022detection, ren2022out}, with similar insights for hallucination detection \citep{ferrando2022towards, guerreiro2023hallucinations, chwang2023androids, duan2024llms} and general uncertainty quantification \citep{vazhentsev2023hybrid, liu2024uncertainty}, potentially suggesting a link back to the discussion about encoded and undecoded latent features from the previous \cref{rq3}.

\section{Open Questions \& Future Research Directions}\label{sec:open-questions}

% What makes a model uncertain?
% What should make a model uncertain?
% How to represent uncertainty quantification
% The limits of uncertainty quantification
    % Model uncertainty and higher order uncertainty
    % For LLMs: How much can we quantify through text alone

% Experimental design and data collection (human label variation) to reduce uncertainties
% Learning second-order predictors
% Uncertainty with guarantees
% Evaluating uncertainty (again, back to human label variation)
% Uncertainty for different model access setups, use cases
% Uncertainty for low-resource
% Influence of tokenization
% Quantifying human uncertainty

The answers to \cref{rq1,rq2,rq3,rq4} can only provide partial steps towards solving any of these complex questions.
As this thesis has argued, the topic of uncertainty quantification\index{Uncertainty quantification} in NLP lies in the intersection of multiple different fields such as statistics\index{Statistics}, linguistics\index{Linguistics} and deep learning\index{Deep learning}.
It has only recently started to garner more attention, as for instance demonstrated by the first UncertaiNLP workshop \citep{vazquez2024proceedings}, related surveys \citep{baan2023uncertainty, hu2023uncertainty, geng2023survey, campos2024conformal} or other dissertations \citep{he2024uncertainty}.
This creates ample space for future research, which we outline next.

\subsection{Modeling Uncertainty}\label{sec:modeling-uncertainties}

One focus of research about uncertainty in deep learning is---and has been---its modeling.
Despite the manifold of works in this direction however, a number of many open directions of research remain.
This includes everything from the modeling uncertainty on different input scales, obtaining guarantees, and how to properly represent and explain it.

\paragraph{Influence of Experimental Design.}
\cref{ch:uncertainty-experimental-design}\index{Experimental design} has argued how careful experimental design can reduce or help to quantify uncertainty, for instance by providing clearer annotation guidelines or model selection through statistical hypothesis testing\index{Hypothesis testing}.
An often overlooked aspect is how retaining multiple human labels per training instance also opens up new avenues for better modeling of uncertainty and paraphrasticity \index{Paraphrasticity}\citep{plank2022problem, baan2022stop, baan2023uncertainty}.

\paragraph{Uncertainty with Guarantees.}
Pivotally, uncertainty quantification can only increase trust\index{Trust} in ML systems when the estimate of uncertainty is itself reliable.
As for instance \citet{dhuliawala2023diachronic} showed, unreliable estimates can lead to a loss of trust in the model that can be hard to recover from.
Thus, conformal prediction\index{Conformal prediction} currently is a very promising research direction, since it supplies statistical guarantees about predictions that are furthermore agnostic to the underlying predictor.
This flexibility has enables the flurry of conformal works in NLP\index{Natural language processing} (e.g.\@ \citealp{schuster2022confident, ravfogel2023conformal, quach2023conformal, zollo2023prompt, su2024api, ulmer2024non, campos2024conformal}).
Conformal prediction however comes with two caveats:
Coverage is only guaranteed in expectation, and is \emph{marginal}\index{Coverage!Marginal} rather than \emph{conditional}\index{Coverage!Conditional}, i.e.\@ the guarantee is $p( y^\prime \in \mathcal{C}(\bx^\prime)) \ge 1 - \alpha$ rather than $p( y^\prime \in \mathcal{C}(\bx^\prime) \mid \bx^\prime) \ge 1 - \alpha$.
Unfortunately, conditional coverage is generally deemed unachievable under finite samples, with the guarantee approximately being fulfilled in some situations \citep{vovk2012conditional, foygel2021limits, gibbs2023conformal}.
Other ways to circumvent this issue lie in partitioning the dataset (similar to the binning in the ECE\index{Expected calibration error}, see \citealp{feldman2021improving, gibbs2023conformal, jin2024confidence}) or conditioning on the label $y^*$ instead of the input (see mondrian conformal predictors; \citealp{vovk2005algorithmic}).
Therefore, future research could investigate conformalizing other uncertainty methods or extending existing guarantees.

\paragraph{Hierarchical Uncertainty.}
Compared to other input modalities such as images, uncertainty in NLP exists on different scales.
Starting from (subword-)token uncertainty, uncertainty can also exist on a sequence, utterance, or paragraph or even dialogue-level.
So far, most uncertainty quantification techniques operate on a token-level or sequence-level, with pioneering work on higher scales such as the dialogue-level \citep{sicilia2024deal}.
While there are some theoretical frameworks like \citet{malinin2021uncertainty} to model how uncertainty from tokens affects the uncertainty in sequences, this is only given for certain metrics.
Therefore, an open question remains how to estimate uncertainty on these different levels and how uncertainty can be decomposed into smaller units.

\paragraph{Representing Uncertainty.}
In this thesis, we have mostly focused on representing uncertainty in the form of single scalars or prediction sets\index{Prediction set}.
However, uncertainty can also be represented in many other ways, for instance in the form of a posterior distribution\index{Posterior distribution} or the highest density interval\index{Highest density interval} in \cref{sec:bayesian-perspective}, uncertainty in the latent space \citep{kingma2013auto,rezende2014stochastic,daxberger2019bayesian,kong2020sde,miani2022laplacian},
or even linguistically (see discussion in \cref{sec:future-communicating-uncertainty}).
The representation of uncertainty should therefore not be overly restrictive, embrace the richness in options and explore new representations.

\paragraph{Quantifying Human Uncertainty.}
Most of this thesis was focused on modeling and quantifying the uncertainty in models operating on language data, but one might also want to model the human uncertainty underlying the data directly.
First advances in this direction have been made by estimating the uncertainty in human labels \citep{northcutt2021confident, jiang2023understanding, gruber2024more}, analyzing annotator disagreement \citep{baan2022stop, baan2024interpreting} or comparing the variability of humans to that of NLG systems \citep{giulianelli2023comes, lee2023can, ilia2024predict}.
Furthermore, a number works try to model the uncertainty in humans using neural language models \citep{hu2023expectations} or try to detect linguistic uncertainty\index{Uncertainty!Linguistic} in text \citep{szarvas2012cross, vincze2014uncertainty, kolagar2024aligning}.

\paragraph{Explaining Uncertainty.}
The answer to \cref{rq3} suggest a hypothesis with which the general behavior of uncertainty is influenced by neural inductive biases.
Nevertheless, there also lies tremendous value in understanding how uncertainty actually arises for a specific input.
This can for instance highlight erroneous or noisy parts of an input or help to understand model failure cases (see e.g.\@ \citealp{xu2020understanding} for an application to text summarization).
To this extent, some works have began to apply interpretability techniques to understand predictive uncertainty, including Shapley values\index{Shapley value} \citep{chen2022explaining, watson2024explaining} or feature attribution methods \citep{bley2024explaining}.

\subsection{Limits of Uncertainty Quantification}\label{sec:uq-limits}

Another often overlooked aspect of uncertainty is defining or exploring the boundaries in which the model's uncertainty is expected to operate;
this includes in particular cases in which uncertainty estimates themselves might be uncertain, ill-defined, limited, or reductive, and which are open for further exploration.

\paragraph{Limits of the Aleatoric--Epistemic Dichotomy.}
Uncertainty, in a statistical sense, is traditionally delineated along data (aleatoric)\index{Uncertainty!Aleatoric} and model (epistemic) uncertainty\index{Uncertainty!Epistemic} \citep{hora1996aleatory, der2009aleatory, huellermeier2021aleatoric}.
However, recent works such as \citet{baan2023uncertainty, gruber2023sources} have advocated to reject this dichotomy in favor of placing uncertainties and their sources on a spectrum.
This dichotomy becomes blurred further when considering that more far-reaching decompositions are possible (for instance adding distributional uncertainty like in \cref{sec:evidential-neural-networks}), and that estimates of epistemic uncertainty might be in themselves uncertain \citep{wimmer2023quantifying}.

\paragraph{Higher Order Uncertainties.}
Evidential deep learning\index{Deep learning!Evidential} (\cref{sec:evidential-neural-networks}) and credal learning (\cref{sec:other-approaches}) offer methods to model higher-order probability distributions or sets and quantify their uncertainty.
Having said that, evidential deep learning in particular has been criticized for not providing loss functions that can provably achieve well-behaved epistemic uncertainties in the model \citep{bengs2023second}, however alternative methods have been proposed for credal predictors\index{Credal sets} \citep{hullermeier2022quantification, sale2023second, sale2023second2, hofman2024quantifying}.

\paragraph{Features for Uncertainty Quantification.}
The previously mentioned methods quantify uncertainty based on properties of the underlying probability distribution parameterized by a neural network.
However, the considerations in \cref{rq3} might prompt one to consider whether this should be the only source from which we should deduce uncertainty.
In the previous section we discussed for instance modeling uncertainty in the latent space, and \cref{sec:calibrator-inputs-ablation} illustrated how, to some extent, we can infer uncertainty solely from the input to a model and train a secondary predictor to output uncertainty in a supervised learning task.
Thus there remain many avenues to explore to find the best features that can be used to obtain uncertainty estimates, which are already being explored by works such as \citet{fathullah2023needs, liu2024uncertainty}.

\subsection{Evaluating Uncertainty}\label{sec:evaluating-uncertainty}

One common conundrum in the research surrounding uncertainty quantification is the lack of ground truth about a predictors uncertainty.
Therefore---and in this regard \cref{ch:uncertainty-classification,ch:uncertainty-llms} are no different---one has to instead defer to approximations and proxy tasks.
For frequentists methods like confidence scores\index{Confidence} we can measure calibration errors, but have to make do with binning, kernel estimators or other approximations.
Otherwise we fall back other problems like error or OOD detection or measure correlations between predictive error and uncertainty.
These analyses need to be multi-dimensional to be cogent and can be gamed;
for example the SNGP Bert in \cref{sec:predictive-uncertainty-experiments} achieves high correlation between sequence uncertainties and loss by not converging properly, and verbalized uncertainty\index{Uncertainty!Verbalized} by GPT-3.5 in \cref{sec:calibration-experiments} is well-calibrated on TriviaQA since the dataset is too easy, despite only articulating the same (high) confidence scores.\\

Yet when multiple annotations are available, we can actually use this to our advantage to create a ground truth for uncertainty, as done for instance by \citet{baan2022stop, ilia2024predict}.
Here, the paraphrasticity\index{Paraphrasticity} of language can help to create ground truth distributions whose uncertainty can be measure and compared against.

\subsection{Communicating Uncertainty}\label{sec:future-communicating-uncertainty}

Communicating uncertainty\index{Uncertainty!Communication of} is difficult---\cref{sec:communicating-uncertainty} described how communicating uncertainty to different social groups while being both understandable and precise is challenging, and how the process can affect human-machine cooperations in sometimes unintuitive ways.
In this light, verbalized uncertainty\index{Uncertainty!Verbalized} (\cref{sec:uncertainty-nlp}) seems like an attractive tool for humanly intuitive ways of expressing uncertainty.
But the experiments in \cref{sec:calibration-experiments} and studies such as \citep{tian2023just} exemplified that such expressions tend to display lopsided distributions of confidence that are not desirable.
\citet{zhou2023navigating} show how this behavior might be rooted in the unequal distribution of these confidence expression (in their case, percentage values) in the training data.
This is not to say that this approach is moribund:
Works like \citet{mielke2022reducing, stengel2024lacie} train language models to produce more complex verbalized expressions of uncertainty, and \cref{sec:expressing-uncertainty} outlines the richness of human uncertainty expressions that can serve as a guide for future research.
% Chapter X

\chapter{Conclusion} % Chapter title

\label{ch:conclusion} % For referencing the chapter elsewhere, use \autoref{ch:name} 

\begin{tikzpicture}[remember picture,overlay]
    \node[anchor=north,inner sep=0pt] at (current page text area.north) {\includegraphics[width=\linewidth, clip=true, trim = 8cm 75cm 8cm 50cm]{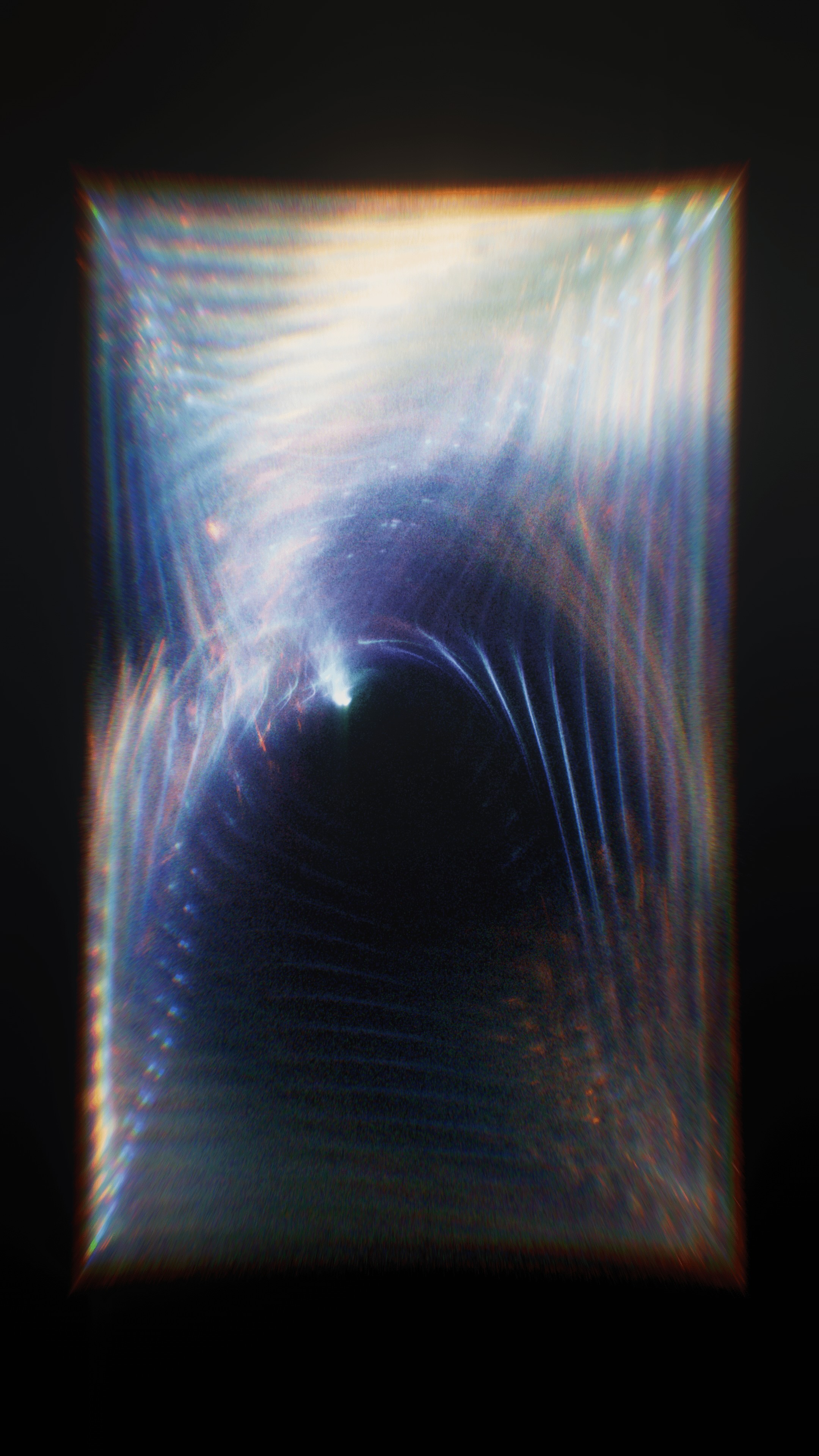}};
\end{tikzpicture}

\epigraph{
    ``\emph{These intelligent agents are the only way to sift through the oceans of data we are producing at an exponential rate [\ldots]. 
    It is important if you find this terrifying or wonderful because public sentiment drives education, investment and regulation. 
    If people find the rapid advance of intelligent machines terrifying instead of wonderful it won't stop it, but it could make the outcome worse for us all.}''
    }{---Garry Kasparov in \emph{Deep Thinking} \citep{kasparov2017deep}.}

On May 6th 2023, a document submitted to the United States District Court of the Southern District of New York \citep{courtcase} reads:\\

\begin{raggedright}
    \footnotesize
    ``\emph{The Court is presented with an unprecedented circumstance.}
    \emph{A submission filed by plaintiff’s counsel in opposition to a motion to dismiss is replete with citations to non-existent cases.} 
    \emph{When the circumstance was called to the Court's attention by opposing counsel, the Court issued Orders requiring plaintiffs counsel to provide an affidavit [\ldots].}
    \emph{Six of the submitted cases appear to be bogus judicial decisions with bogus quotes and bogus internal citations.}''\\[0.5cm]
\end{raggedright}

The document was submitted by the judge in the case of Roberto Mata versus the Columbian airline Avianca.
As it was revealed later, the plaintiff's lawyers used OpenAI's ChatGPT to find other relevant cases for their argument, which turned out to be non-existent.\footnote{
    See for example the corresponding articles by the Verge (\url{https://www.theverge.com/2023/5/27/23739913/chatgpt-ai-lawsuit-avianca-airlines-chatbot-research}) or the New York Times (\url{https://www.nytimes.com/2023/05/27/nyregion/avianca-airline-lawsuit-chatgpt.html}).
    Both were accessed last on 17-05-2024.
}
This curious case represents three different aspects about AI\index{Artificial intelligence} in modern society at once:
Firstly, AI in general and LLMs\index{Large language model} specifically are increasingly permeating society and culture.
This can be shown through their growing adoption \citep{humlum2024adoption}, their impact on art \citep{zulic2019how, du2021research, sivertsen2024machine} and by becoming an progressively political issue \citep{hovy2016social, mohamed2020decolonial, zuboff2023age, devenot2023tescreal}.
Secondly, current language models are prone to producing hallucinations\index{Hallucination}, i.e.\@ seemingly plausible but fabricated generations.
While detection and mitigation of hallucinations are very active areas of research \citep{ji2023survey}, some have argued that it is an unavoidable feature of current models \citep{kalai2023calibrated, xu2024hallucination}.
Thirdly, the way language models work remains too technical and opaque to most people and LLM-based\index{Large language model} chatbots are conceptualized as search engines rather than extremely powerful word predictors.
This becomes even more blatant when examining the details of the above case through one of the lawyers' affidavit:
In order to verify the veracity of the (later to be found fictitious) cited case studies, they asked ChatGPT questions such as 
``\emph{Is varghese a real case}'', to which the language model answered affirmatively.\\

The \emph{bitter lesson}\index{Bitter lesson} \citep{sutton2019bitter} states that ``general methods that leverage computation are ultimately the most effective, and by a large margin''.
In the past, it has proven time and time again that sophistication in AI\index{Artificial intelligence} research is outperformed by sheer scale.
Which, given the content of thesis, prompts the question of whether research on UQ \index{Uncertainty quantification} is necessary or yet another piece of unnecessary ornamentation on the road to more intelligent systems.

\paragraph{Do We actually Need UQ?}
Let us assume the role of a devil's advocate for a moment.
In this position, we can pose several counter-arguments to the necessity of UQ,\index{Uncertainty quantification} starting with

\begin{center}
    ``\emph{Current cutting-edge models work so well that UQ is not necessary.}''
\end{center}

While it is true that the bitter lesson\index{Bitter lesson} keeps materializing in current models, even an ever-increasing coverage of topics and tasks through larger amounts of training data does not shield them from an infinitely-large space of possible inputs, on which their behavior is hard to predict.
This phenomenon is referred to as \emph{model underspecification}.\index{Underspecification!Model}
Furthermore, increasing generalization by obtaining more and more training data is expensive;
estimations by works such as \citet{villalobos2022will} suggest that we are already starting to deplete the stock of high-quality language data to train on.
Counter-strategies to this problem have been to simply allocate resources to human data creation,\footnote{
    See for instance reporting about OpenAI's strategy to employ workers in Kenya to create new training data and improve existing data quality, e.g.\@ \url{https://time.com/6247678/openai-chatgpt-kenya-workers/} (last accessed 19.05.2024).
} to repeatedly use the same training data \citep{xue2024repeat} or to use synthetic training data, where the latter has shown mixed results \citep{guo2023curious, alemohammad2023self, briesch2023large, bohacek2023nepotistically, gulcehre2023reinforced, shumailov2023model, dohmatob2024tale, ulmer2024bootstrapping}.
However, this also ignores the inequality of available data in different languages \citep{singh2024aya}.
 Being able to guarantee robust model behavior on different topics, tasks and language this way thus appears unlikely.

\begin{center}
    ``\emph{Model capabilities have consistently improved with model size and the amount of available training data, and in the same way a model's uncertainty estimates will become more reliable by itself.}''
\end{center}
 
While there is some evidence that e.g.\@ a model's calibration\index{Calibration} increases with the available training data \citep{dan2021effects, chen2023close, tian2023just, zhu2023calibration, ulmer2024calibrating}, one can hypothesize that the increased coverage of training cases simply enables the model to better learn the actual distributions over targets (be it class labels or token distributions) for the most frequent types of input.
For LLMs,\index{Large language model} there is some evidence that verbalized uncertainty\index{Uncertainty!Verbalized} in its current form improves with model and training data size, but the distribution of uncertainty expressions still remains skewed \citep{tian2023just, ulmer2024calibrating}.

\begin{center}
    ``\emph{Smarter models will become better at admitting when they do not know an answer.}''
\end{center}

Compared to the previous question, here we wouldn't rely on additional uncertainty estimates to refuse a potential unreliable prediction, but assume that a smarter model would learn to refuse directly.
We can reason through this argument by realizing that in order to achieve these model refusals, they would either have to be explicit contents of their training data, or be the result of of some subsequent finetuning / alignment process.
The first case is unrealistic or at least conceptually misguided:
We would like models to respond to certain instructions by admitting their ignorance because the answer would otherwise likely be incorrect, not because they learned a mapping from certain instructions to these admissions---in the end, we still want models to learn to solve a given task!
This entails that such a behavior would be acquired during additional finetuning steps (instead of the pre-training phase, such as instruction finetuning, alignment using human feedback\index{Reinforcement learning from human feedback}, etc.), but in order to do so, one requires knowledge about when these statements are necessary.
This could come from signals from the model itself---however we have seen that models \emph{do not always know when they do not know}---or from human or automatic evaluation, which seems infeasible to perform on a comprehensive scale.
Thus, we can likely only adopt these behaviors for more common instructions, even though they would matter most on unseen or rare ones.

\begin{center}
    ``\emph{Current UQ quantification approaches are useless since they are not reliable themselves.}''
\end{center}

This is not an entirely unfair criticism, and we dedicated parts of \cref{ch:discussion} to the limits and failure cases of current UQ methods\index{Uncertainty quantification}.
One could explain the recent soaring in interest in conformal prediction\index{Conformal prediction} methods that they, in contrast to their alternatives, can provide formal guarantees.
Even though these might still be insufficient for many practical applications, we can expect future research to improve them further.
Furthermore, there is a case to made where the overall utility of UQ with even somewhat deficient guarantees exceeds the loss in utility without any UQ whatsoever.
Given this thought, one might wonder why we haven't seen wide-spread adoption of UQ techniques in user-facing applications.

\paragraph{What Hinders UQ in User-Facing Products?}
This point can only be answered speculatively, but what is true is that none of the large commercially available LLMs at the time of this writing offer any degree of uncertainty quantification.\footnote{
    This includes Anthropic, Cohere, OpenAI, Google and Mistral.
    OpenAI's API does allow access to token probabilities (\url{https://x.com/OpenAIDevs/status/1735730662362189872}, last accessed on 16.01.24), however they are not framed as confidence scores directly, confidence estimation is just mentioned as one possible application.
}
One potential reason could be that there is simply no or not enough demand;
this could be because models usually work sufficiently well for users on their specific use cases or that customers are not aware of the problem (or of UQ as a possible solution).
Another reason could be that UQ in its current form does not work reliably enough and would expose a company to too many risks;
an unreliable prediction that is accompanied with a high confidence\index{Confidence} value could potentially create PR and legal liability issues when found to have caused real-world harm.

% Importance of data, multi-modality
\paragraph{How does UQ Relate to Current Developments in the Field?}
At the time of writing of the author's master thesis in 2019, the field of NLP\index{Natural language processing} was experiencing an acceleration.
After the invention of the transformer\index{Transformer} two years prior \citep{vaswani2017attention}, models like Bert\index{Bert} \citep{devlin2019bert} and GPT-2 \citep{radford2019language} were heralding a paradigm shift in the field, as increasingly large models were demonstrating hitherto unseen abilities.
In this context, part of the conclusion of \citet{ulmer2019recoding} reads

\begin{center}
    \footnotesize
    ``\emph{On the flip side, these [language] models require huge amounts of data and computational resources.}
    \emph{[\ldots]\@ This has several, worrying implications: First, with these resource requirements, scientific papers become hard to reproduce. These costs only allow training of these models in the context of well-funded institutions, namely top-tier universities and affluent tech giants.} 
    \emph{Secondly, the reliance on large-scale hardware produces a high electricity consumption along with a worrisome carbon footprint, which bears a certain irony:}
    \emph{These models try to (loosely) imitate the human brain, a biological computer that is actually very energy efficient (Schwartz et al., 2019).}
    \emph{Lastly, scaling up data sets and the number of parameters does not necessarily increase the semblance to human cognition.}''
\end{center}

% Model size and data increase
% Still problems with reproducibility - problems on potential failure cases
% Semblance to human cognition still debated (sparkes of AGI, stochastic parrots, meaning from text alone)
% Convergence of representations?
% Decreasing importance of model underspecification
% Most uncertain topics covered by training data?
It is interesting to re-examine these thoughts in the light of current trends.
First of all, the size of language models and their training set sizes has risen tremendously.
\citeauthor{devlin2019bert}'s largest Bert\index{Bert} model comprised 340 million parameters, and was trained on around 3.3 billion words.
For comparison, the largest Llama 3 model comprises 90 billion parameters and was trained on 15 trillion tokens \citep{llama3modelcard}, with GPT-4 rumored to be 1.76 trillion parameters large \citep{gpt4report}.
The fact that GPT-4's parameter count is not public and that details about the training data for both GPT-4 and Llama 3 are unknown accentuate the most recent trend in language model development and echo some of the thoughts above:
With a few exceptions such like OLMo \citep{groeneveld2024olmo}, it has become infeasible for non-industry actors to train language models from scratch.
At the same time, companies have started to hide training details that they deem strategically important, hindering replication and research even when the final models become openly available.
This also makes it hard to assess for which kind of inputs we can expect models to behave reliably.
This is exacerbated by the fact that any semblance of human intelligence is still controversial---while recent models have displayed impressive abilities \citep{bubeck2023sparks}, some argue that outputs are ``haphazardly stitch[ed] together sequences of linguistic forms [the language model] has observed in its vast training data, according to probabilistic information about how they combine, but without any reference to meaning'' \citep{bender2021dangers}.
The consequence of this is that language models might fail in ways that are unpredictable and unintuitive to humans.
And as the introductory examples in this chapter and \cref{ch:uncertainty-llms} show, the more convincing generations appear, the harder any failures become to spot.\\

% EU AI act, regulation of high-risk applications
\paragraph{Policy and Societal Implications.}
The increased adoption of AI\index{Artificial intelligence} models has prompted a response from different regulatory bodies.
One instance of this is the EU AI act \citep{madiega2021artificial}.
The act sorts different applications into a four tier system, ranging from minimal risk to unacceptable risk.
While unacceptable risk applications are outright prohibited (e.g.\@ social scoring systems, facial recognition etc.), there also exists a tier of high-risks systems with applications in law enforcement, education or medicine that are allowed under strict regulations.
One prerequisite for high-risk systems is \emph{human oversight}, meaning that the system can be ``effectively overseen by natural persons during the period in which the AI system is in use'' and to ``prevent or minimize the risks to health, safety or
fundamental rights that may emerge'' (Article 14).
It should be clear that techniques like anomaly detection\index{Anomaly detection} and UQ\index{Uncertainty quantification} can help to fulfill these criteria by deferring decisions to human overseers and flagging inputs on which the system could behave abnormally.
Thus, in order to create commercial high-risk AI applications in the EU, the development of UQ methods with stronger guarantees might be one potential avenue.
Similar policies are still pending in the United States, where the Biden administration enacted an executive order on the development and use of AI \citep{biden2023executive}.
In its opening paragraph, it states

\begin{center}
    \footnotesize
    ``\emph{Harnessing AI for good and realizing its myriad benefits requires mitigating its substantial risks.} 
    \emph{This endeavor demands a society-wide effort that includes government, the private sector, academia, and civil society.}''
\end{center}

AI\index{Artificial intelligence} is a powerful technology, unfolding in an unequal world and already reshaping societies.
As researchers, we can help advance directions like UQ\index{Uncertainty quantification} alongside others such as generalization, bias mitigation, fairness\index{Fairness}, interpretability\index{Interpretability} and many more in order to help mitigate the risk of modern AI systems, so that any transformation may be a positive one.
%Or, adding to this thought with the words of the mathematician Richard Hamming \citep{hamming1986you}: 
%``One of the characteristics of successful scientists is having courage. Once you get your courage up and believe that you can do important problems, then you can.''

%----------------------------------------------------------------------------------------
%	BIBLIOGRAPHY
%----------------------------------------------------------------------------------------

\addcontentsline{toc}{chapter}{Bibliography}
\bibliographystyle{natbib} 
\bibliography{Bibliography_clean}

%----------------------------------------------------------------------------------------
%	APPENDICES
%----------------------------------------------------------------------------------------
\newpage
\appendix
% Appendix A

\chapter{Theoretical Appendix}\label{app:theoretical-appendix}

\epigraph{

    ``\emph{Physics is searching for a \textbf{theory of everything}. Deep learning is searching for a \textbf{theory of anything}.}''
}
{---Zachary Lipton on \href{https://twitter.com/zacharylipton/status/1625561952574443535}{Twitter}.}

\begin{table}[htb]
    \centering
    \resizebox{0.6\textwidth}{!}{%
    \begin{tabular}{rl}
        \toprule
        Thesis & Appendix \\
        \toprule
        \cref{sec:bayesian-perspective} & \cref{app:relationship-beta-gamma} \\
        \cref{sec:evidential-neural-networks} & \cref{app:expectation-dirichlet,app:entropy-dirichlet,app:kl-dirichlets,app:mutual-information} \\
        \cref{sec:know-your-limits-preliminaries} & \cref{app-softmax-sigmoid-connection} \\
        \cref{sec:convergence-on-ood} & \cref{app:proposition1,app:softmax-limit} \\
        \cref{sec:overconfidence-metrics} & \cref{app:aggregation-theorem,app:asymptotic-softmax-variance,app:asymptotic-predictive-entropy,app:mutual-information} \\
        \bottomrule
    \end{tabular}%
    }
    \caption[Correspondences between sections of the theoretical appendix and thesis chapters.]{Correspondences between sections of the theoretical appendix and thesis chapters.}\label{tab:theoretical-appendix-correspondence}
\end{table}

This appendix contains additional derivations and proofs for some of the main chapters in this thesis.
\cref{tab:theoretical-appendix-correspondence} gives an overview over the correspondences between thesis chapters and sections in this appendix.

\section{Relationship between Beta and Gamma function}\label{app:relationship-beta-gamma}

Here, we further elaborate on the connection between the Beta\index{Beta function} and the Gamma function\index{Gamma function}, used to derive the predictive prior and posterior distribution\index{Prior distribution}\index{Posterior distribution} of a Beta distribution with Bernoulli likelihood\index{Bernoulli distribution}\index{Likelihood} in \cref{eq:predictive-prior-coin-flip,eq:predictive-posterior-coin-flip}.
The Beta function is commonly defined in terms of Gamma functions, namely

\begin{equation}
    B(\alpha, \beta) = \frac{\Gamma(\alpha)\Gamma(\beta)}{\Gamma(\alpha + \beta)},
\end{equation}

\noindent and recall the definition of the Gamma function as 

\begin{equation}
    \Gamma(\alpha) = \int_0^{\infty} x^{\alpha - 1}\exp(-x)\dd\!\hspace{0.1cm}x.
\end{equation}

Alternatively, the Beta function can be stated as 

\begin{equation}
    \text{B}(\alpha, \beta) = \int_0^1 x^{\alpha - 1}(1 - x)^{\beta - 1}\dd\!\hspace{0.1cm}x.
\end{equation}

This connection arises by evaluating the following product:

\begin{align}
    \Gamma(\alpha)\Gamma(\beta) &= \Big(\int_0^{\infty} x^{\alpha - 1}\exp(-x)\dd\!\hspace{0.1cm}x\Big)\Big(\int_0^{\infty} y^{\beta - 1}\exp(-y)\dd\!\hspace{0.1cm}y\Big)\\
    & = \int_0^\infty\int_0^\infty x^{\alpha - 1} y^{\beta - 1}\exp\big(-(x+y)\big)\dd\!\hspace{0.1cm}x\dd\!\hspace{0.1cm}y.
\end{align}

In order to simplify the integration, we apply a change of variables by substituting $x = uv$ and $y = u(1 - v)$.
To account for the change of variables during the integration, we also need to evaluate the determinant of the Jacobian\index{Jacobian} as 

\begin{equation}
    |\mathbf{J}| = \begin{vmatrix} \frac{\partial x}{\partial u} & \frac{\partial x}{\partial v} \\\frac{\partial y}{\partial u} & \frac{\partial y}{\partial v} \end{vmatrix} = \begin{vmatrix} v & u \\ 1-v & -u \end{vmatrix} = -uv - u(1- v) = -u.
\end{equation}

By writing $u$ and $v$ in terms of $x$ and $y$, we obtain that $u = x + y$ and $v = x/(x + y)$, which implies that the limits for the integration remain $0$ to $\infty$ for $u$ and become $0$ to $1$ for $v$.
Using all of these insights, we can now show that 

\begin{align}
    \Gamma(\alpha)\Gamma(\beta)  = & \int_0^\infty\int_0^\infty x^{\alpha - 1} y^{\beta - 1}\exp\big(-(x+y)\big)\dd\!\hspace{0.1cm}x\dd\!\hspace{0.1cm}v \\
    = & \int_0^1\int_0^\infty (uv)^{\alpha - 1} \big(u(1-v)\big)^{\beta - 1} \nonumber \\
    & \exp\big(-(uv + u(1 -v))\big)|-u|\dd\!\hspace{0.1cm}u\dd\!\hspace{0.1cm}v \\
    = & \int_0^1\int_0^\infty u^{\alpha-1}v^{\alpha-1}u^{\beta-1}(1-v)^{\beta - 1}\exp(-u)u \dd\!\hspace{0.1cm}u\dd\!\hspace{0.1cm}v \\
    = &\int_0^1\int_0^\infty u^{\alpha + \beta -1}v^{\alpha-1}u^{\beta-1}(1-v)^{\beta - 1}\exp(-u) \dd\!\hspace{0.1cm}u\dd\!\hspace{0.1cm}v \\
    = & \Big(\int_0^1 v^{\alpha-1}(1-v)^{\beta - 1} \dd\!\hspace{0.1cm}v \Big)\Big(\int_0^\infty u^{\alpha + \beta -1}\exp(-u) \dd\!\hspace{0.1cm}u  \Big) \\
    = & B(\alpha, \beta)\Gamma(\alpha + \beta),
\end{align}

\noindent from which the connection between the two definition follows.

\section{Expectation of the Dirichlet Distribution}\label{app:expectation-dirichlet}
\index{Dirichlet distribution}
Here, we show results for the quantities $\mathbb{E}[\pi_k]$ and $\mathbb{E}[\log \pi_k]$ that appear in \cref{sec:evidential-neural-networks}.
For the first, we follow the derivation by \citet{miller2011dirichlet}. 
Another proof is given by \citet{lin2016dirichlet}.

\begin{align}
    \mathbb{E}[\pi_k] & = \int \cdots \int \pi_k \frac{\Gamma(\alpha_0)}{\prod_{k^\prime=1}^K \Gamma(\alpha_k^\prime)} \prod_{k^\prime=1}^K \pi_{k^\prime}^{\alpha_{k^\prime} - 1} \ddd\pi_1 \ldots \ddd\pi_K. \\
    \intertext{Moving $\pi_{k}^{\alpha_{k} - 1}$ out of the product:}
    & = \int \cdots \int \frac{\Gamma(\alpha_0)}{\prod_{k^\prime=1}^K \Gamma(\alpha_{k^\prime})} \pi_k^{\alpha_k - 1 + 1}\prod_{k^\prime \neq k} \pi_{k^\prime}^{\alpha_{k^\prime} - 1}\ddd\pi_1 \ldots \ddd\pi_{K}. \\
    \intertext{For the next step, we define a new set of Dirichlet parameters with $\beta_k = \alpha_k + 1$ and $\forall k^\prime \neq k: \beta_{k^\prime} = \alpha_{k^\prime}$. For those new parameters, $\beta_0 = \sum_k \beta_k = 1 + \alpha_0$. So by virtue of the Gamma function's property that $\Gamma(\beta_0) = \Gamma(\alpha_0 + 1) = \alpha_0\Gamma(\alpha_0)$, replacing all terms in the normalization factor yields}
    & = \int \cdots \int \frac{\alpha_k}{\alpha_0}\frac{\Gamma(\beta_0)}{\prod_{k^\prime=1}^K \Gamma(\beta_{k^\prime})} \prod_{k^\prime=1}^K \pi_{k^\prime}^{\beta_{k^\prime} - 1} \ddd\pi_1 \ldots \ddd\pi_K = \frac{\alpha_k}{\alpha_0},
\end{align}

\noindent where in the last step we obtain the final result, since the Dirichlet with new parameters $\beta_k$ must nevertheless integrate to $1$, and the integrals do not regard $\alpha_k$ or $\alpha_0$. 
For the expectation $\mathbb{E}[\log \pi_k]$, we first rephrase the Dirichlet distribution in terms of the exponential families \citep{kupperman1964probabilities}. 
The exponential families\index{Exponential families} encompass many commonly-used distributions, such as the normal\index{Normal distribution}, exponential\index{Exponential distribution}, Beta\index{Beta distribution} or Poisson\index{Poisson distribution}, which all follow the form 

\begin{equation}\label{eq:exp-family}
    p(\bx; \bm{\eta}) = h(\bx)\exp\big(\bm{\eta}\T u(\bx) - A(\bm{\eta})\big),
\end{equation}

\noindent with \emph{natural parameters} $\bm{\eta}$\index{Natural parameter}, \emph{sufficient statistic} $u(\bx)$\index{Sufficient statistic}, and \emph{log-partition function}\index{Log-partition function} $A(\bm{\eta})$. 
For the Dirichlet distribution\index{Dirichlet distribution}, \citet{winn2004variational} provides the sufficient statistic as $u(\bm{\pi}) = [\log \bm{\pi}_1, \ldots, \bm{\pi}_K]^T$ and the log-partition function 

\begin{equation}\label{eq:log-partition-dirichlet}
    A(\bm{\alpha}) = \sum_{k=1}^K \log \Gamma(\alpha_k) - \log \Gamma(\alpha_0).
\end{equation}

By \citet{introduction2019mao}, we also find that by the moment-generating function that for the sufficient statistic, its expectation can be derived by 

\begin{equation}\label{eq:expected-value-sufficient}
    \mathbb{E}[u(\bx)_k] = \frac{\partial A(\bm{\eta})}{\partial \eta_k}.
\end{equation}

Therefore, we can evaluate the expected value of $\log \pi_k$ (i.e.\@ the sufficient statistic\index{Sufficient statistic}) by inserting the definition of the log-partition function in \cref{eq:log-partition-dirichlet} into \cref{eq:expected-value-sufficient}:

\begin{equation}\begin{aligned}\label{eq:log-expectation}
    \mathbb{E}[\log \pi_k] = \frac{\partial}{\partial \alpha_k}\sum_{k=1}^K \log \Gamma(\alpha_k) - \log \Gamma(\alpha_0) = \psi(\alpha_k) - \psi(\alpha_0),
\end{aligned}\end{equation}

\noindent which corresponds precisely to the definition of the digamma function\index{Digamma function} as $\psi(x) = \frac{d}{d x}\log \Gamma(x)$.

\section{Entropy of the Dirichlet Distribution}\label{app:entropy-dirichlet}
\index{Entropy}\index{Dirichlet distribution}
The following derivation for the entropy of the Dirichlet which appears in \cref{sec:evidential-neural-networks} is adapted from \citet{lin2016dirichlet}, with the result stated in \citet{charpentier2020posterior} as well.

\begin{align}
    \text{H}[p(\bm{\pi} \mid \bm{\alpha})] & = - \mathbb{E}[\log p(\bm{\pi} \mid \bm{\alpha})] \\
    & = - \mathbb{E}\Big[\log\Big( \frac{1}{\text{B}(\bm{\alpha})}\prod_{k=1}^K\pi_k^{\alpha_k - 1}\Big)\Big] \\
    & = - \mathbb{E}\Big[-\log \text{B}(\bm{\alpha}) + \sum_{k=1}^K (\alpha_k - 1)\log \pi_k\Big] \\
    & = \log \text{B}(\bm{\alpha} ) - \sum_{k=1}^K (\alpha_k - 1)\mathbb{E}[\log \pi_k]. \\
    \intertext{Using \cref{eq:log-expectation}:}
    & = \log \text{B}(\bm{\alpha} ) - \sum_{k=1}^K (\alpha_k - 1)\big(\psi(\alpha_k) - \psi(\alpha_0)\big) \\
    & = \log \text{B}(\bm{\alpha} ) + \sum_{k=1}^K (\alpha_k - 1)\psi(\alpha_0) - \sum_{k=1}^K (\alpha_k - 1)\psi(\alpha_k) \\
    & = \log \text{B}(\bm{\alpha}) + (\alpha_0 - K)\psi(\alpha_0) - \sum_{k=1}^K (\alpha_k - 1)\psi(\alpha_k).
 \end{align}

\section{Expected Entropy of the Dirichlet Distribution}\label{app:expected-entropy}
\index{Entropy!Expected}\index{Dirichlet distribution}
The following derivation for the expected entropy of the Dirichlet which appears in \cref{sec:evidential-neural-networks} is adapted from \citet{malinin2018predictive} appendix section C.4. 
In the following, we assume that $\forall k \in \mathbb{K}: \pi_k > 0$:
    
\begin{align}
    & \mathbb{E}_{p(\bm{\pi} \mid \bx, \hat{\bm{\theta}})}\Big[\text{H}\big[P(y \mid \bm{\pi})\big]\Big] = \int p(\bm{\pi} \mid \bx, \hat{\bm{\theta}}) \Big(-\sum_{k=1}^K \pi_k\log \pi_k\Big) \dd\bm{\pi} \\
    & = - \sum_{k=1}^K \int p(\bm{\pi} \mid \bx, \hat{\bm{\theta}})\big(\pi_k \log \pi_k\big)\dd \bm{\pi}. \\
    \intertext{Inserting the definition of $p(\bm{\pi}|\bx, \hat{\bm{\theta}}) \approx p(\bm{\pi} \mid \bx, \mathbb{D})$:}
    & = - \sum_{k=1}^K \Bigg(\frac{\Gamma(\alpha_0)}{\prod_{k^\prime=1}^K \Gamma(\alpha_{k^\prime})}\int \pi_k \log \pi_k \prod_{k^\prime=1}^K\pi_{k^\prime}^{\alpha_{k^\prime} - 1} \dd\bm{\pi} \Bigg).\\
    \intertext{Singling out the factor $\pi_k$:}
    & = - \sum_{k=1}^K \Bigg(\frac{\Gamma(\alpha_0)}{\Gamma(\alpha_{k})\prod_{k^\prime \neq k} \Gamma(\alpha_{k^\prime})}\pi_k^{\alpha_k-1}\int \pi_k \log \pi_k \prod_{k^\prime \neq k}\pi_{k^\prime}^{\alpha_{k^\prime} - 1} \dd\bm{\pi} \Bigg).\\
    \intertext{Adjusting the normalizing constant (this is the same trick used in \cref{app:expectation-dirichlet}):}
     & = - \sum_{k=1}^K \Bigg(\frac{\alpha_k}{\alpha_0}\int\frac{\Gamma(\alpha_0+1)}{\Gamma(\alpha_{k}+1)\prod_{k^\prime \neq k} \Gamma(\alpha_{k^\prime})}\pi_k^{\alpha_k-1} \log \pi_k \prod_{k^\prime \neq k}\pi_{k^\prime}^{\alpha_{k^\prime} - 1}  \dd\bm{\pi} \Bigg).\\
    \intertext{Using the identity $\mathbb{E}[\log \pi_k] = \psi(\alpha_k) -  \psi(\alpha_0)$ (\cref{eq:log-expectation}). Since the expectation here is w.r.t.\@ to a Dirichlet with concentration parameters $\alpha_k + 1$, we obtain}
    & = - \sum_{k=1}^K\frac{\alpha_k}{\alpha_0}\bigg(\psi(\alpha_k+1) -  \psi(\alpha_0+1)\bigg).
\end{align}

\section{Kullback-Leibler Divergence between two Dirichlets}\label{app:kl-dirichlets}
\index{Kullback-Leibler divergence}\index{Dirichlet distribution}

The following result appearing in \cref{sec:evidential-neural-networks} is presented using an adapted derivation by \citet{lin2016dirichlet} and appears in \citet{chen2018variational} and \citet{joo2020being} as a starting point for their variational objective.\index{Variational inference}
In the following we use $\text{Dir}(\bm{\pi}; \bm{\alpha})$ to denote distribution to be optimized, and $\text{Dir}(\bm{\pi}; \bm{\gamma})$ for the reference or target distribution.

\begin{align}
    & \text{KL}\big[p(\bm{\pi} \mid \bm{\alpha})\ \big|\big|\ p(\bm{\pi} \mid \bm{\gamma})\big] \nonumber \\[0.2cm]
    & = \mathbb{E}\Big[\log\frac{p(\bm{\pi} \mid \bm{\alpha})}{p(\bm{\pi} \mid\bm{\gamma})}\Big] = \mathbb{E}\big[\log p(\bm{\pi} \mid \bm{\alpha})\big] - \mathbb{E}\big[\log p(\bm{\pi} \mid \bm{\gamma})\big] \\
    & = \mathbb{E}\Big[-\log \text{B}(\bm{\alpha}) + \sum_{k=1}^K (\alpha_k -1)\log \pi_k \Big] \nonumber \\
    & - \mathbb{E}\Big[-\log \text{B}(\bm{\gamma}) + \sum_{k=1}^K (\gamma_k -1)\log \pi_k\Big]. \\
    \intertext{Distributing and pulling out $\text{B}(\bm{\alpha})$ and $\text{B}(\bm{\gamma})$ out of the expectation (they don't depend on $\bm{\pi}$):}
    = & - \log\frac{\text{B}(\bm{\gamma})}{\text{B}(\bm{\alpha})} + \mathbb{E}\Big[\sum_{k=1}^K (\alpha_k -1)\log \pi_k - (\gamma_k -1)\log \pi_k\Big] \\
    = & - \log\frac{\text{B}(\bm{\gamma})}{\text{B}(\bm{\alpha})} + \mathbb{E}\Big[\sum_{k=1}^K (\alpha_k -\gamma_k)\log \pi_k\Big].
    \intertext{Moving the expectation inward and using the identity $\mathbb{E}[\pi_k] = \psi(\alpha_k) - \psi(\alpha_0)$ from \cref{app:expectation-dirichlet}:}
    = & - \log\frac{\text{B}(\bm{\gamma})}{\text{B}(\bm{\alpha})}+ \sum_{k=1}^K (\alpha_k - \gamma_k)\big(\psi(\alpha_k) - \psi(\alpha_0)\big).
\end{align}

The KL divergence is also used by some works as regularizer by penalizing the distance to a uniform Dirichlet with $\bm{\gamma} = \mathbf{1}$ \citep{sensoy2019evidential}.
In this case, the result above can be derived to be 

\begin{equation}
    \text{KL}\big[p(\bm{\pi} \mid \bm{\alpha})\ \big|\big|\ p(\bm{\pi}\mid\bm{1})\big] = \log \frac{\Gamma(K)}{\text{B}(\bm{\alpha})} + \sum_{k=1}^K (\alpha_k - 1)\big(\psi(\alpha_k) - \psi(\alpha_0)\big),
\end{equation}

\noindent where the $\log \Gamma(K)$ term can also be omitted for optimization purposes, since it does not depend on $\bm{\alpha}$.

\section{Mutual Information for Dirichlet Networks}\label{app:mutual-information}
\index{Mutual information}\index{Dirichlet distribution}
As stated in \cref{sec:evidential-neural-networks}, mutual information is a measure of distributional uncertainty\index{Uncertainty!Distributional} in Dirichlet networks.
To derive its closed-form expression, we start from \cref{eq:dirichlet-mi}:

\begin{align}
    \text{I}\Big[y, \bm{\pi}\ \Big|\ \bx, \mathbb{D}\Big] & = \text{H}\Big[\Expect_{p(\bm{\pi} \mid \bx, \mathbb{D})}\big[P(y \mid \bm{\pi})\big]\Big] - \mathbb{E}_{p(\bm{\pi} \mid \bx, \mathbb{D})}\Big[\text{H}\big[P(y \mid \bm{\pi})\big]\Big]. \\
    \intertext{Given that $\mathbb{E}[\pi_k] = \frac{\alpha_k}{\alpha_0}$ (\cref{app:expectation-dirichlet}) and assuming that point estimate $p(\bm{\pi} \mid \bx, \mathbb{D}) \approx p(\bm{\pi} \mid \bx, \hat{\bm{\theta}})$ is sufficient \citep{malinin2018predictive}, we can identify the first term as the Shannon entropy $-\sum_{k=1}^K \pi_k \log \pi_k = -\sum_{k=1}^K \frac{\alpha_k}{\alpha_0} \log \frac{\alpha_k}{\alpha_0} $.
    Furthermore, the second part we already derived in \cref{app:expected-entropy}, and thus we obtain:}
    & = -\sum_{k=1}^K \frac{\alpha_k}{\alpha_0}\log \frac{\alpha_k}{\alpha_0} + \sum_{k=1}^K\frac{\alpha_k}{\alpha_0}\Big(\psi(\alpha_k+1) -  \psi(\alpha_0+1)\Big) \\
    & = - \sum_{k=1}^K \frac{\alpha_k}{\alpha_0}\Big(\log \frac{\alpha_k}{\alpha_0} -\psi(\alpha_k+1) + \psi(\alpha_0+1)\Big).
\end{align}

\section{Connection between Softmax and Sigmoid}\label{app-softmax-sigmoid-connection}

In this section we briefly outline the connection between the softmax\index{Softmax function} and the sigmoid function\index{Sigmoid function}, in order to show the applicability of results in \cref{sec:uq-classification-pitfalls} to both binary and multi-class classification problems\index{Classification!Binary}\index{Classification!Multi-class}.
This connection was originally shown in \citet{bridle1990probabilistic}. 
Let the sigmoid function be defined as 
\begin{equation}
    \sigma(x) = \frac{\exp(x)}{1 + \exp(x)},
\end{equation}

\noindent and softmax according to the definition in \cref{eq:softmax}. 
The output of $f_{\btheta}$ in a multi-class classification problem with $K$ classes corresponds to a $K$-dimensional column vector that is based on an affine transformation of the network's last intermediate hidden representation $\bx_L$, such that $f_{\btheta}(\bx) = \bW_L\bx_L$.\footnote{
    The bias term $\bb_L$ was omitted here for clarity.
} 
Correspondingly, the output of $f_{\btheta}$ for a single class $c$ can be written as the dot product between $\bx_L$ and the corresponding row vector of $\bW_L$ denoted as $\bw_L^{(c)}$, such that $f_{\btheta}(\bx)_k \equiv {\bw_L^{(k)T}}\bx_L$. 
For a classification problem with $K=2$ classes, we can now rewrite the softmax probabilities in the following way:\footnote{
    The following argument holds without loss of generality for $P_{\btheta}(y=0 \mid \bx)$.
}

\begin{equation}
    P_{\btheta}(y=1 \mid \bx) = \frac{\exp({\bw_L^{(1)}\T}\bx_L)}{\exp({\bw_L^{(0)}\T}\bx_L) + \exp({\bw_L^{(1)}\T}\bx_L)}. \\
\end{equation} 

Subtracting a constant from the weight term inside the exponential function does not change the output of the softmax function.\index{Softmax function} 
Using this property, we can show the sigmoid function to be a special case of the softmax for binary classification:\index{Classification!Binary}

\begin{align}
     P_{\btheta}(y=1 \mid \bx) & = \frac{\exp((\bw_L^{(1)} - \bw_L^{(0)})\T\bx_L)}{\exp((\bw_L^{(0)} - \bw_L^{(0)})\T\bx_L) + \exp((\bw_L^{(1)} - \bw_L^{(0)})\T\bx_L)} \\
    & = \frac{\exp((\bw_L^{(1)} - \bw_L^{(0)})\T\bx_L)}{1 + \exp((\bw_L^{(1)} - \bw_L^{(0)})\T\bx_L}  = \frac{\exp({\bw_L^*\T}\bx_L)}{1 + \exp({\bw_L^*\T}\bx_L)},
\end{align} 

\noindent where $\bw_L^* = \bw_L^{(1)} - \bw_L^{(0)}$ corresponds to the new parameter vector which is used to parametrize a single output unit for a network in the binary classification setting.

\section{Construction of Polytopal Regions}\label{app:polytopes}
\index{Polytope}
In this section, we reiterate the reasoning by \citet{hein2019relu} behind the construction the polytopal regions mentioned in \cref{sec:convergence-on-ood}.
For this purpose, the authors define an additional diagonal matrix $\bDelta_l(\bx)$ per layer $l$: 

\begin{equation}
    \bDelta_l(\bx) = \begin{bmatrix}
    \text{sign}({f_{\btheta}^l(\bx)_1}) & \cdots & 0  \\
    \vdots &  \ddots & \vdots \\
    0 & \cdots & \text{sign}({f_{\btheta}^l(\bx)_{n_l}}) \\
    \end{bmatrix}.
\end{equation}

Together with the linearization of the network at $\bx$ explained in \cref{eq:relu-linearization}, this is used to define a set of half-spaces for every neuron in the network:

\begin{equation}
    \mathbb{H}_{l, i}(\bx) = \big\{\bz \in \mathbb{R}^d\ \big|\ \bDelta_l(\bx)\big(\bV_l(\bx)_i\bz + \ba_l(\bx)_i\big) \ge 0 \big\}.
\end{equation}

Here, $\bV_l(\bx)_i$ and $\bb_l(\bx)_i$ denote the parts of the affine transformation obtained for the $i$-th neuron of the $l$-th layer, so the $i$-th row vector in $\bV_l(\bx)$ and the $i$-th scalar in $\bb_l(\bx)$, respectively. 
Finally, the polytope\index{Polytope} $Q$ containing $\bx$ is obtained by taking the intersection of all half-spaces induced by every neuron in the network: 

\begin{equation}
    Q(\bx) = \bigcap_{l \in 1, \ldots, L}\bigcap_{i \in 1, \ldots, n_l} \mathbb{H}_{l, i}(\bx).
\end{equation}

\section{Proof of \cref{proposition:overconfidence-softmax}}\label{app:proposition1}

\begin{figure}
    \centering
    \includegraphics[width=0.9\columnwidth]{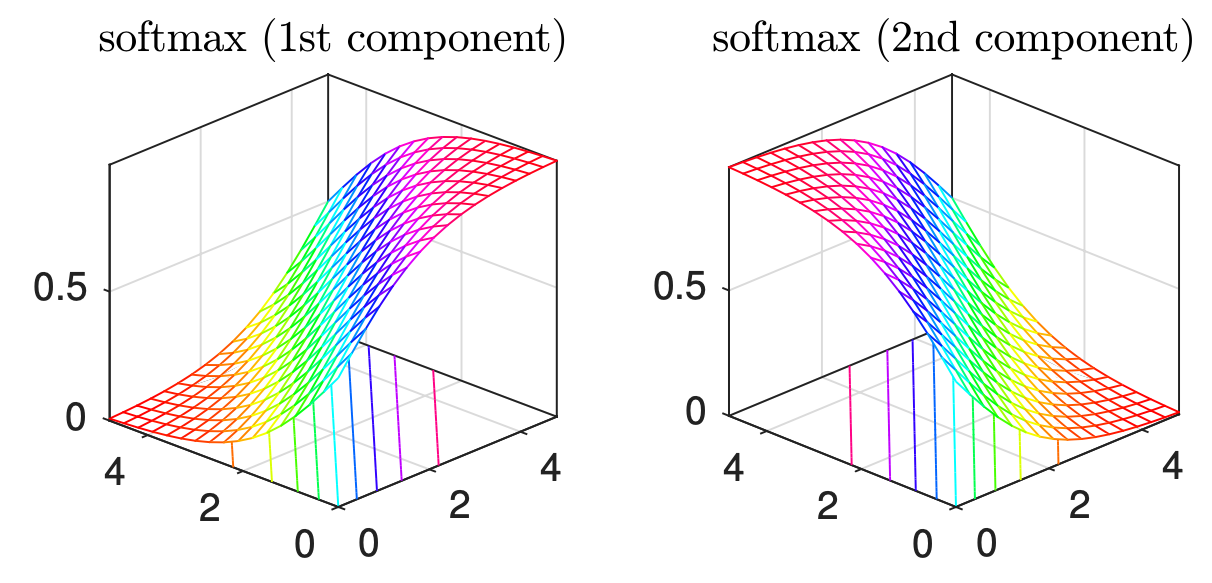}
    \caption[Illustrating the interplay of softmax probabilities between components for $K=2$ in $\mathbb{R}^2$.]{
        Illustration taken from the work of \citet{gao2017properties}, illustrating the interplay of softmax probabilities between components for $K=2$ in $\mathbb{R}^2$.
    }
    \label{fig:softmax-components}
\end{figure}

This section provides the proof of \cref{proposition:overconfidence-softmax} in \cref{sec:convergence-on-ood}.
 We proceed to analyze the behavior of gradients in the limit via two more lemmas; 
 First, we establish the saturating property of the softmax\index{Softmax function} in \cref{lemma:softmax-properties}, i.e.\@ the model doesn't change its decision anymore in the limit.

\begin{lemma}\label{lemma:softmax-properties}
    Let $k, k^\prime \in [K]$ be two arbitrary classes. It then holds for their corresponding output components (logits) that
    \begin{equation}\label{eq:softmax-asymptotic-behavior}
          \lim_{f_{\btheta}(\bx)_k \rightarrow \pm \infty} \frac{\partial}{\partial f_{\btheta}(\bx)_{k^\prime}}\bar{\sigma}(f_{\btheta}(\bx))_k = 0.
    \end{equation}
\end{lemma}

\begin{proof}
    Here, we first begin by evaluating the derivative of one component of the function w.r.t.\@ to an arbitrary component:
     \begin{align}
        & \frac{\partial}{\partial f_{\btheta}(\bx)_{k^\prime}}\bar{\sigma}(f_{\btheta}(\bx))_k = \frac{\partial}{\partial f_{\btheta}(\bx)_{k^\prime}} \frac{\exp(f_{\btheta}(\bx)_k)}{\sum_{k^{\pprime} \in [K]} \exp(f_{\btheta}(\bx)_{k^{\pprime}})} \\
        & = \frac{\indicator{k = k^\prime} \exp(f_{\btheta}(\bx)_k)}{\sum_{k^{\pprime} \in [K]} \exp(f_{\btheta}(\bx)_{k^{\pprime}})} - \frac{\exp(f_{\btheta}(\bx)_k)\exp(f_{\btheta}(\bx)_{k^\prime})}{\big(\sum_{k^{\pprime} \in [K]} \exp(f_{\btheta}(\bx)_{k^{\pprime}})\big)^2}.
        %& = \indicator{c = c^\prime}\bar{\sigma}(f_{\btheta}(\bx))_c - \bar{\sigma}(f_{\btheta}(\bx))_c\bar{\sigma}(f_{\btheta}(\bx))_{c^\prime} \\
        %& = \bar{\sigma}(f_{\btheta}(\bx))_c\big(\indicator{c = c^\prime} - \bar{\sigma}(f_{\btheta}(\bx))_{c^\prime} \big) 
    \end{align}
    
    This implies that
    \begin{align}\label{eq:softmax-derivative-cases}
        & \frac{\partial}{\partial f_{\btheta}(\bx)_{k^\prime}}\bar{\sigma}(f_{\btheta}(\bx))_k = \nonumber \\
        & \begin{cases}
        \displaystyle
        - \frac{\exp(2f_{\btheta}(\bx)_k)}{\big(\sum_{k^{\pprime} \in [K]} \exp(f_{\btheta}(\bx)_{k^{\pprime}})\big)^2} + \frac{\exp(f_{\btheta}(\bx)_{k})}{\sum_{k^{\pprime} \in [K]}\exp(f_{\btheta}(\bx)_{k^{\pprime}})}  & \quad\text{if }k=k^\prime\\[0.5cm]
        \displaystyle
        - \frac{\exp(f_{\btheta}(\bx)_k + f_{\btheta}(\bx)_{k^\prime})}{\big(\sum_{k^{\pprime}\in [K]} \exp(f_{\btheta}(\bx)_{k^{\pprime}})\big)^2} & \quad\text{if }k\neq k^\prime\\
        \end{cases}
    \end{align}
    or more compactly:
    \begin{equation*}
        \frac{\partial}{\partial f_{\btheta}(\bx)_{k^\prime}}\bar{\sigma}(f_{\btheta}(\bx))_k  = \bar{\sigma}(f_{\btheta}(\bx))_k\big(\indicator{k = k^\prime} - \bar{\sigma}(f_{\btheta}(\bx))_{k^\prime} \big).
    \end{equation*}
    
    Based on \cref{eq:softmax-derivative-cases}, we can now investigate the asymptotic behavior for $f_{\btheta}(\bx)_k \rightarrow \infty$ more easily, starting with the $k = k^\prime$ case:
    \begin{equation}\begin{aligned}\label{eq:softmax-derivative-equal}
       & \lim_{f_{\btheta}(\bx)_k \rightarrow \infty} \frac{\partial}{\partial f_{\btheta}(\bx)_{k^\prime}}\bar{\sigma}(f_{\btheta}(\bx))_k \\
       & = \underbrace{ \lim_{f_{\btheta}(\bx)_k \rightarrow \infty}-\frac{\exp(f_{\btheta}(\bx)_k)}{\sum_{k^{\pprime} \in [K]}\exp(f_{\btheta}(\bx)_{k^{\pprime}})}\frac{\exp(f_{\btheta}(\bx)_k)}{\sum_{k^{\pprime} \in [K]}\exp(f_{\btheta}(\bx)_{k^{\pprime}})}}_{\text{-1}} \\
       & + \underbrace{ \lim_{f_{\btheta}(\bx)_k \rightarrow \infty}\frac{\exp(f_{\btheta}(\bx)_k)}{\sum_{k^{\pprime} \in [K]}\exp(f_{\btheta}(\bx)_{k^{\pprime}})}}_{1} = 0. \\
    \end{aligned}\end{equation}
    
    With the numerator and denominator being dominated by the exponentiated $f_{\btheta}(\bx)_k$ in \cref{eq:softmax-derivative-equal}, the first term will tend to $-1$, while the second term will tend to $1$, resulting in a derivative of $0$. 
    The case $k \neq k^\prime$ can be analyzed the following way:
    \begin{equation}\begin{aligned}\label{eq:softmax-derivative-unequal}
         & \lim_{f_{\btheta}(\bx)_k \rightarrow \infty} \frac{\partial}{\partial f_{\btheta}(\bx)_{k^\prime}}\bar{\sigma}(f_{\btheta}(\bx))_k = \\
         & \lim_{f_{\btheta}(\bx)_k \rightarrow \infty}\underbrace{\bigg(-\frac{\exp(f_{\btheta}(\bx)_k)}{\sum_{k^{\pprime} \in [K]}\exp(f_{\btheta}(\bx)_{k^{\pprime}})} \bigg)}_{-1} \underbrace{\bigg(\frac{\exp(f_{\btheta}(\bx)_{k^\prime})}{\sum_{k^{\pprime} \in [K]}\exp(f_{\btheta}(\bx)_{k^{\pprime}})}\bigg)}_{0} = 0. 
    \end{aligned}\end{equation}
    
    Again, we factorize the fraction in \cref{eq:softmax-derivative-unequal} into the product of two softmax functions\index{Softmax function}, one for component $k$, one for $k^\prime$. 
    The first factor will again tend to $-1$ as in the other case, however the second will approach $0$, as only the sum in the denominator will approach infinity. 
    As the limit of a product is the products of its limits, this lets the whole expression approach $0$ in the limit.
    When $f_{\btheta}(\bx)_k \rightarrow -\infty$, both cases approach $0$ due to the exponential function, which proves the lemma.
\end{proof}

How the interplay between different softmax components produces zero gradients in the limit is illustrated in \cref{fig:softmax-components}. 
In \cref{lemma:growth-rate-softmax}, we compare the rate of growth of different components of $P_{\btheta}$. 
We show that for the decomposed function $P_{\btheta}$, the rate at which the softmax function converges to its output distribution in the limit outpaces the change in the underlying logits w.r.t.\@ the network input. 

\begin{lemma}\label{lemma:growth-rate-softmax}
    Suppose that $f_{\btheta}$ is a ReLU-network. Let $\bx^\prime\in\mathbb{R}^D$,  suppose $\balpha$ is a scaling vector and that the associated PUP $\mathbb{P}(\bx^\prime, d)$ has a corresponding matrix $\mathbf{V}$ with no zero entries. 
    Then it holds for all $k^\prime \in [K]$ that
    \begin{equation}\label{eq:growth-rate-softmax}
        \lim\limits_{\alpha_d \to \infty} \Big(\frac{\partial}{\partial f_{\btheta}(\bx)_{k^\prime}}\bar{\sigma}(f_{\btheta}(\bx))_k\Big)^{-1}\Big|_{\bx = \balpha\circ\bx^\prime} - \Big(\frac{\partial}{\partial x_d}f_{\btheta}(\bx)_{k^\prime}\Big)\Big|_{\bx = \balpha\circ\bx^\prime}= \infty.
    \end{equation}
\end{lemma}

\begin{proof}
    We evaluate the first term of \cref{eq:growth-rate-softmax} to show that it grows exponentially in the limit. 
    By \cref{lemma:unique-pup}, we know that in the limit $\alpha_d \to \infty$ the vector $\balpha\circ\bx^\prime$ will remain within $\mathbb{P}(\bx^\prime, d)$. 
    Since the matrix associated with this PUP\index{Polytope!Partially-unbounded} has no zero entries, we know by \cref{lemma:strictly-monotonic} that the gradient of $f_{\btheta}(\bx)_k$ on dimension $d$ is either always positive or negative, hence $f_{\btheta}(\bx)_k \rightarrow \pm \infty$.  
    Given \cref{lemma:softmax-properties} describing the asymptotic behavior in the limit, it follows that 
    \begin{equation}
        \lim_{f_{\btheta}(\bx)_k \rightarrow \pm \infty} \Big(\frac{\partial}{\partial f_{\btheta}(\bx)_{k^\prime}}\bar{\sigma}(f_{\btheta}(\bx))_k\Big)^{-1} = \infty,
    \end{equation}
    
    \noindent where we can see that the result is a symmetrical function displaying exponential growth in the limit of $f_{\btheta}(\bx)_k \rightarrow \pm \infty$. 
    We now show that because we assumed $f_{\btheta}$ to be a neural network consisting of $L$ affine transformations with ReLU\index{ReLU} activation functions, the output of the final layer is only going to be a linear combination of its inputs.\footnote{
        Here we make the argument for the whole function $f_{\btheta}: \mathbb{R}^D \rightarrow \mathbb{R}^K$, but the conclusions also applies to every output component of the function $f_{\btheta}(\bx)_k$.
    } 
    This can be proven by induction. 
    Let us first look at the base case $L=1$. 
    In the rest of this proof, we denote $\bx_l$ as the input to layer $l$, with $\bx_1 \equiv \bx$, and $\bW_l, \bb_l$ the corresponding layer parameters. 
    $\ba_l$ signifies the result of the affine transformation that is then fed into the activation function.
    \begin{equation}\begin{aligned}\label{eq:linear-trans-derivative}
        f_{\btheta}(\bx) & = \phi(\ba_1) = \phi(\bW_1\bx_1 + \bb_1) \\
        \frac{\partial f_{\btheta}(\bx)}{\partial \bx_1} & = \frac{\partial \phi(\ba_1) }{\partial \ba_1} \frac{\partial \ba_1}{\partial \bx_1}= \bind{\bx_1 > \mathbf{0}}\T\bW_1\ \\
        \frac{\partial f_{\btheta}(\bx)}{\partial x_{1d}} & = \indicator{x_d > 0}w_{1d},
    \end{aligned}\end{equation}
    
    \noindent where $\bind{\bx_1 > \mathbf{0}} = [\indicator{x_{11} > 0}, \ldots, \indicator{x_{1d} > 0}]\T$, $w_{1d}$ denoting the $d$-th column of $\bW_1$. 
    This is a linear function, which proves the base case. 
    Let now $\frac{\partial \bx_l}{\partial \bx_1}$ denote the partial derivative of the input to the $l$-th layer w.r.t.\@ to the input and suppose that it is linear by the inductive hypothesis. 
    Augmenting the corresponding network by another linear adds another term akin to the second expression in \cref{eq:linear-trans-derivative} to the chain of partial derivatives:
    \begin{equation}\label{eq:layer-induction-step}
        \frac{\partial \bx_{l+1}}{\partial \bx_1} = \frac{\partial \bx_{l+1}}{\partial \bx_l}\frac{\partial \bx_l}{\partial \bx_1},
    \end{equation}
    
    \noindent which is also a linear function, proving the induction step. 
    Because we know that both terms of the product in \cref{eq:layer-induction-step} are linear, the second term of the \cref{eq:growth-rate-softmax} is as well. 
    Together with the previous insight that the first term is exponential, this implies that it will outgrow the second in the limit, creating an infinitely-wide gap between them and thereby proving the lemma.
\end{proof}

Equipped with the results of \cref{lemma:softmax-properties,lemma:growth-rate-softmax}, we can finally prove \cref{proposition:overconfidence-softmax}:
\begin{proof}
    We show that one scalar factor contained in the factorization of the gradient $\nabla_{\bx}P_{\btheta}(y=k \mid \bx)$ tends to zero under the given assumptions, having the whole gradient become the zero vector in the limit. 
    We begin by again factorizing the gradient $\nabla_{\bx} P_{\btheta}(y=k \mid \bx)$ using the multivariate chain rule:
    \begin{equation}\label{eq:fac-gradient-softmax}
        \nabla_{\bx}P_{\btheta}(y=k \mid \bx) = \sum_{k^\prime=1}^K \frac{\partial}{\partial  f_{\btheta}(\bx)_{k^\prime}}\bar{\sigma}(f_{\btheta}(\bx))_k \nabla_{\bx}f_{\btheta}(\bx)_{k^\prime}.
    \end{equation}
    By \cref{lemma:strictly-monotonic,lemma:unique-pup} we know that $f_{\btheta}$ is a component-wise strictly monotonic function\index{Monotonicity!Component-wise strict} on $\mathbb{P}(\bx^\prime, d)$, which implies for the limit of $\alpha_d \rightarrow \infty$ that $\forall k \in [K]:\ f_{\btheta}(\bx)_k \rightarrow \pm \infty$. 
    Then, \cref{lemma:softmax-properties} implies that the first factor of every part in the sum of \cref{eq:fac-gradient-softmax} will tend to zero in the limit. 
    \cref{lemma:growth-rate-softmax} ensures that the first factor approximates zero quicker than every component of the gradient $\nabla_{\bx}f_{\btheta}(\bx)_{k^\prime}$ potentially approaching infinity, causing the product to result in the zero vector. 
    As this results in a sum over $K$ zero vectors in the limit, this proves the lemma.
    \end{proof}

\section{Proof of \cref{proposition:softmax-limit}}\label{app:softmax-limit}

This section contains the proof of \cref{proposition:softmax-limit} in \cref{sec:convergence-on-ood}.

\begin{proof}
    We start by rewriting the softmax probability for the $k$-th logit:
    \begin{equation}
        \bar{\sigma}(f_{\btheta}(\bx))_k = \frac{\exp(f_{\btheta}(\bx)_k)}{\sum_{k^\prime \in [K]}\exp(f_{\btheta}(\bx)_{k^\prime})} = 1 - \frac{\sum_{k^{\prime\prime} \in [K] \setminus \{k\}}\exp(f_{\btheta}(\bx)_{k^{\prime\prime}})}{\sum_{k^\prime \in [K]}\exp(f_{\btheta}(\bx)_{k^\prime})}.
    \end{equation}
    % Let $c \in \mathcal{C}$ be an arbitrary class s.t. $\forall c^\prime \neq c:\ v_{cd} > v_{c^\prime d}$.  
    %We proceed to prove the case of the limit $\alpha_d \rightarrow \infty$. 
    By \cref{lemma:strictly-monotonic,lemma:unique-pup}, we have shown that $f_{\btheta}$ is a component-wise strictly monotonic function\index{Monotonicity!Component-wise strict} on $\mathbb{P}(\bx^\prime, d)$, %by the previous Lemma we know that $f_{\btheta}$.
    so we know that for all $k^\prime \in [K]:\ f_{\btheta}(\bx)_{k^\prime} \rightarrow \pm \infty$ as $\alpha_d \rightarrow \infty$. 
    We now treat the two limits $\pm \infty$ in order.
    Because of the assumption that $d$-column of $\mathbf{V}$ has no duplicate entries, this implies that  there must be a $k \in [K]$ s.t.\@ $\forall k^\prime \neq k:\ v_{kd} > v_{k^\prime d}$. 
    Thus, in the limit of $f_{\btheta}(\bx)_k \rightarrow \infty$, the sum in the \emph{denominator} of the fraction including the logit of $k$ will tend to infinity faster than the the sum in the \emph{numerator} not including $k$'s logit, and thus the fraction itself will tend to $0$, proving this case. 
    In the case of $f_{\btheta}(\bx)_k \rightarrow -\infty$, the \emph{numerator} of the fraction will tend to $0$ faster than the \emph{denominator}, having the fraction approach $0$ in the limit as well, proving the second case and therefore the lemma. 
\end{proof}

\section{Proof of \cref{aggregation-theorem}}\label{app:aggregation-theorem}

This section contains the proof of \cref{aggregation-theorem} in \cref{sec:overconfidence-metrics}.

\begin{proof}
    \begin{align}
        & \lim\limits_{\alpha \to \infty}\big|\big|\nabla_{\bx}\mathbb{E}_{p(\btheta \mid \mathbb{D})}\big[{P_{\btheta}(y=k \mid \bx)}\big]\big|_{\bx = \balpha\circ\bx^\prime}\big|\big|_2 \\
        = & \lim\limits_{\alpha \to \infty}\big|\big|\mathbb{E}_{p(\btheta \mid \mathbb{D})}\big[{\nabla_{\bx} P_{\btheta}(y=k \mid\bx)}\big]\big|\big|_{\bx = \balpha\circ\bx^\prime}\big|\big|_2 \\
        \le & \lim\limits_{\alpha \to \infty} \mathbb{E}_{p(\btheta \mid \mathbb{D})}\big[\underbrace{\big|\big| \nabla_{\bx} P_{\btheta}(y=k \mid \bx)\big|_{\bx = \balpha\circ\bx^\prime}\big|\big|_2}_{=\ 0\text{ (\cref{proposition:overconfidence-softmax})}}\big] = 0.
    \end{align}
    Because the last expression is an upper bound to the original expression and the $l_2$ norm is lower-bounded by $0$, this proves the lemma.
\end{proof}

\section{Proof of \cref{lemma:asymptotic-softmax-variance}}\label{app:asymptotic-softmax-variance} %\ref{lemma:asymptotic-softmax-variance}}\label{app:asymptotic-softmax-variance}

This section contains the proof of \cref{lemma:asymptotic-softmax-variance} that is part of the proof of \cref{theorem:know-your-limits-main-theorem} in \cref{sec:overconfidence-metrics}.

\begin{lemma}{(Asymptotic behavior with softmax variance)}\label{lemma:asymptotic-softmax-variance}
      Suppose that $f_{\btheta}^{(1)}, \ldots, f_{\btheta}^{(K)}$ are ReLU networks. Let $\bx^\prime\in\mathbb{R}^D$,  suppose $\balpha$ is a scaling vector and that for all $k$, the associated PUP\index{Polytope!Partially-unbounded} $\mathbb{P}^{(k)}(\bx^\prime, d)$ has a corresponding matrix $\mathbf{V}^{(k)}$ with no zero entries. 
      It holds that 
    \begin{align}
        \lim\limits_{\alpha_d \to \infty} & \Big|\Big|\nabla_{\bx}\frac{1}{K}\sum_{k=1}^K \mathbb{E}_{p(\btheta \mid \mathbb{D})}\big[P_{\btheta}(y=k \mid \bx)^2\big] \nonumber \\ 
         & - \mathbb{E}_{p(\btheta \mid \mathbb{D})}\big[{P_{\btheta}(y=k \mid \bx)}\big]^2\big|_{\bx = \balpha\circ\bx^\prime}\Big|\Big|_2 = 0.
    \end{align}

\end{lemma}

\begin{proof}   
    \begin{align}
        \lim\limits_{\alpha_d \to \infty} & \Big|\Big|\nabla_{\bx}\frac{1}{K}\sum_{k=1}^K \mathbb{E}_{p(\btheta \mid \mathbb{D})}\big[P_{\btheta}(y=k \mid \bx)^2\big] \nonumber \\ 
        & - \mathbb{E}_{p(\btheta \mid \mathbb{D})}\big[{P_{\btheta}(y=k \mid \bx)}\big]^2\big|_{\bx = \balpha\circ\bx^\prime}\Big|\Big|_2 \\
          %\intertext{Linearity of gradient:}
          = & \lim\limits_{\alpha_d \to \infty}\Big|\Big|\frac{1}{K}\sum_{k=1}^K \nabla_{\bx} \mathbb{E}_{p(\btheta \mid \mathbb{D})}\big[P_{\btheta}(y=k \mid \bx)\big]^2 \nonumber \\ 
          - & \nabla_{\bx}\mathbb{E}_{p(\btheta \mid \mathbb{D})}\big[P_{\btheta}(y=k \mid \bx)\big]^2\big|_{\bx = \balpha\circ\bx^\prime}\Big|\Big|_2 \\
        \intertext{Apply triangle inequality $||x + y|| \le ||x|| + ||y||$ to sum over all $k$:}
         \le & \lim\limits_{\alpha_d \to \infty}\frac{1}{K}\sum_{k=1}^K \big|\big| \nabla_{\bx} \mathbb{E}_{p(\btheta \mid \mathbb{D})}\big[P_{\btheta}(y=k \mid \bx)^2\big] \nonumber\\
         - & \nabla_{\bx}\mathbb{E}_{p(\btheta \mid \mathbb{D})}\big[P_{\btheta}(y=k \mid \bx)\big]^2\big|_{\bx = \balpha\circ\bx^\prime}\big|\big|_2 \\
          \intertext{On the first term use linearity of gradients and apply chain rule, do it in the reverse order on the second term:}
          = & \lim\limits_{\alpha_d \to \infty}\frac{1}{K}\sum_{k=1}^K \big|\big| \mathbb{E}_{p(\btheta \mid \mathbb{D})}\big[2P_{\btheta}(y=k \mid \bx)\underbrace{\mystrut{0.325cm}{\nabla_{\bx}P_{\btheta}(y=k \mid \bx)}\big|_{\bx = \balpha\circ\bx^\prime}}_{=\ \bm{0} \text{ (\cref{proposition:overconfidence-softmax})}}\big] \nonumber \\
            - & 2\mathbb{E}_{p(\btheta \mid \mathbb{D})}\big[P_{\btheta}(y=k \mid \bx)\big] \mathbb{E}_{p(\btheta \mid \mathbb{D})}\big[\underbrace{\mystrut{0.4cm}{\nabla_{\bx}P_{\btheta}(y=k \mid \bx)}\big|_{\bx = \balpha\circ\bx^\prime}}_{=\ \bm{0} \text{ (\cref{proposition:overconfidence-softmax})}}\big]\big|\big|_2 = 0.
    \end{align} 
    We can see that due to an intermediate result of \cref{proposition:overconfidence-softmax}, i.e.\@ that $\nabla_{\bx}P_{\btheta}(y=k \mid \bx)$ approaches the zero vector in the limit, the innermost gradients tend to zero, bringing the whole expression to zero.
    Because the final is an upper bound to the original expression and because the $l_2$ norm has a lower bound of $0$, this proves the lemma.
\end{proof}

\section{Proof of \cref{lemma:asymptotic-predictive-entropy}}\label{app:asymptotic-predictive-entropy} %\ref{lemma:asymptotic-predictive-entropy}}\label{app:asymptotic-predictive-entropy}

This section contains the proof of \cref{lemma:asymptotic-predictive-entropy} that is part of the proof of \cref{theorem:know-your-limits-main-theorem} in \cref{sec:overconfidence-metrics}.

\begin{lemma}{(Asymptotic behavior for predictive entropy)}\label{lemma:asymptotic-predictive-entropy}
      Suppose that $f_{\btheta}^{(1)}, \ldots, f_{\btheta}^{(K)}$ are ReLU networks. \index{ReLU}
      Let $\bx^\prime\in\mathbb{R}^D$, suppose $\balpha$ is a scaling vector and that for all $k$, the associated PUP\index{Polytope!Partially-unbounded} $\mathbb{P}^{(k)}(\bx^\prime, d)$ has a corresponding matrix $\mathbf{V}^{(k)}$ with no zero entries. 
      It holds that
    \begin{equation}
        \lim\limits_{\alpha_d \to \infty}\Big|\Big|\nabla_{\bx}\mathrm{H}\Big[\mathbb{E}_{p(\btheta \mid \mathbb{D})}\big[P_{\btheta}(y \mid \bx)\big]\Big]\Big|_{\bx = \balpha\circ\bx^\prime}\Big|\Big|_2 = 0.
    \end{equation}
\end{lemma}

\begin{proof}
    \begin{align}
        & \lim\limits_{\alpha_d \to \infty}\Big|\Big|\nabla_{\bx}\mathrm{H}\Big[\mathbb{E}_{p(\btheta \mid \mathbb{D})}\big[P_{\btheta}(y \mid \bx)\big]\Big]\Big|_{\bx = \balpha\circ\bx^\prime}\Big|\Big|_2 \\
         = & \lim\limits_{\alpha_d \to \infty}\Big|\Big|\nabla_{\bx}\Big(\sum_{k=1}^K\mathbb{E}_{p(\btheta \mid \mathbb{D})}\big[P_{\btheta}(y=k \mid \bx)\big] \nonumber \\ 
         & \cdot \log\mathbb{E}_{p(\btheta \mid \mathbb{D})}\big[P_{\btheta}(y=k \mid \bx)\big]\Big)\Big|_{\bx = \balpha\circ\bx^\prime}\Big|\Big|_2 \\
         %\intertext{Linearity of gradient:}
         = & \lim\limits_{\alpha_d \to \infty}\Big|\Big|\sum_{k=1}^K\nabla_{\bx} \Big(\mathbb{E}_{p(\btheta \mid \mathbb{D})}\big[P_{\btheta}(y=k \mid \bx)\big] \nonumber \\ 
         & \cdot \log\mathbb{E}_{p(\btheta \mid \mathbb{D})}\big[P_{\btheta}(y=k \mid \bx)\big]\Big)\Big|_{\bx = \balpha\circ\bx^\prime}\Big|\Big|_2 \\
         %\intertext{Apply product rule:}
         = & \lim\limits_{\alpha_d \to \infty}\Big|\Big|\sum_{k=1}^K\nabla_{\bx}\mathbb{E}_{p(\btheta \mid \mathbb{D})}\big[p_{\btheta}(y=c|\bx)\big] \nonumber \\ 
        % & \cdot\Big(\mathbb{E}_{p(\btheta \mid \mathbb{D})}\big[P_{\btheta}(y=k\mid\bx)\big]\Big)^{-1} \nabla_{\bx}\Big(\mathbb{E}_{p(\btheta \mid \mathbb{D})}\big[p_{\btheta}(y=c|\bx)\big]\Big) \nonumber
        & + \nabla_{\bx}\Big(\mathbb{E}_{p(\btheta \mid \mathbb{D})}\big[P_{\btheta}(y=k \mid \bx)\big]\Big) \log\mathbb{E}_{p(\btheta \mid \mathbb{D})}\big[P_{\btheta}(y=k \mid \bx)\big]\Big|_{\bx = \balpha\circ\bx^\prime}\Big|\Big|_2 \\
        %\intertext{Factor out gradient:}
        = & \lim\limits_{\alpha_d \to \infty}\Big|\Big|\sum_{k=1}^K \nabla_{\bx}\mathbb{E}_{p(\btheta \mid \mathbb{D})}\big[p_{\btheta}(y=k \mid \bx)\big] \nonumber\\
        & \cdot \Big(1 + \log \mathbb{E}_{p(\btheta \mid \mathbb{D})}\big[P_{\btheta}(y=k\mid\bx)\big]\Big) \Big|_{\bx = \balpha\circ\bx^\prime}\Big|\Big|_2 \\
        \intertext{Apply triangle inequality to sum over all $k$:}
        \le & \lim\limits_{\alpha_d \to \infty}\sum_{k=1}^K\big|\big|\nabla_{\bx}\mathbb{E}_{p(\btheta \mid \mathbb{D})}\big[p_{\btheta}(y=k \mid \bx)\big] \nonumber\\
        & \cdot \big(1 + \log \mathbb{E}_{p(\btheta \mid \mathbb{D})}\big[P_{\btheta}(y=k\mid\bx)\big]\big) \big|_{\bx = \balpha\circ\bx^\prime}\big|\big|_2 \\
        %\intertext{As the log expectation just evaluates to a scalar, it can be pulled out of the norm and we can apply \cref{aggregation-theorem}:}
        = & \lim\limits_{\alpha_d \to \infty}\sum_{k=1}^K\Big(1 + \log\mathbb{E}_{p(\btheta \mid \mathbb{D})}\big[p_{\btheta}(y=k \mid \bx)\big]\Big) \nonumber \\
        & \cdot\underbrace{\big|\big|\nabla_{\bx}\mathbb{E}_{p(\btheta \mid \mathbb{D})}\big[p_{\btheta}(y=k \mid \bx)\big]\big|_{\bx = \balpha\circ\bx^\prime}\big|\big|_2}_{\mystrut{0.4cm}{=\ 0\text{ (\cref{aggregation-theorem})}}} = 0.
    \end{align}
    As the final result is an upper bound to the original expression and is lower-bounded by $0$ due to the $l_2$ norm, this proves the lemma.
\end{proof}

\section{Proof of \cref{lemma:asymptotic-mutual-information}}\label{app:asymptotic-mutual-information} %\ref{lemma:asymptotic-mutual-information}}\label{app:asymptotic-mutual-information}

This section contains the proof of \cref{lemma:asymptotic-mutual-information} that is part of the proof of \cref{theorem:know-your-limits-main-theorem} in \cref{sec:overconfidence-metrics}.

\begin{lemma}{(Asymptotic behavior for approximate mutual information)}\label{lemma:asymptotic-mutual-information}
      Suppose that $f_{\btheta}^{(1)}, \ldots, f_{\btheta}^{(K)}$ are ReLU networks\index{ReLU}. Let $\bx^\prime\in\mathbb{R}^D$, suppose $\balpha$ is a scaling vector and that for all $k$, the associated PUP\index{Polytope!Partially-unbounded} $\mathbb{P}^{(k)}(\bx^\prime, d)$ has a corresponding matrix $\mathbf{V}^{(k)}$ with no zero entries. 
      It holds that
    \begin{align}
        \lim\limits_{\alpha_d \to \infty} & \Big|\Big|\nabla_{\bx}\Big(\mathrm{H}\Big[\mathbb{E}_{p(\btheta \mid \mathbb{D})}\big[P_{\btheta}(y \mid \bx)\big]\Big] \nonumber \\
         & -  \mathbb{E}_{p(\btheta \mid \mathbb{D})}\Big[\mathrm{H}\big[P_{\btheta}(y \mid \bx)\big]\Big]\Big)\Big|_{\bx = \balpha\circ\bx^\prime}\Big|\Big|_2 = 0.
    \end{align}
    
\end{lemma}

\begin{proof}
    \begin{align}
        & \lim\limits_{\alpha_d \to \infty}\Big|\Big|\nabla_{\bx}\Big(\mathrm{H}\Big[\mathbb{E}_{p(\btheta \mid \mathbb{D})}\big[P_{\btheta}(y \mid \bx)\big]\Big] \nonumber \\
        & -  \mathbb{E}_{p(\btheta \mid \mathbb{D})}\Big[\mathrm{H}\big[P_{\btheta}(y \mid \bx)\big]\Big]\Big)\Big|_{\bx = \balpha\circ\bx^\prime}\Big|\Big|_2 \\
        %\intertext{Linearity of gradients:}
        %\le & \lim\limits_{\alpha_d \to \infty}\bbnorm\bigg(\nabla_{\bx}\mathbb{H}\bigg[\Expect[\Big]{{p(\btheta|\mathcal{D})}}{p_{\btheta}(y|\bx)}\bigg]\\
        %& - \nabla_{\bx}\Expect[\bigg]{{p(\btheta|\mathcal{D})}}{\mathbb{H}\Big[p_{\btheta}(y|\bx)\Big]}\bigg)\bigg|_{\bx = \balpha\odot\bx^\prime}\bbnorm_2 \\
        %\intertext{Linearity of gradients on second part of difference:}
        = & \lim\limits_{\alpha_d \to \infty}\Big|\Big|\Big(\nabla_{\bx}\mathrm{H}\Big[\mathbb{E}_{p(\btheta \mid \mathbb{D})}\big[P_{\btheta}(y \mid \bx)\big]\Big] \nonumber \\
        &  - \mathbb{E}_{p(\btheta \mid \mathbb{D})}\Big[\nabla_{\bx}\mathrm{H}\big[P_{\btheta}(y \mid \bx)\big]\Big]\Big)\Big|_{\bx = \balpha\circ\bx^\prime}\Big|\Big|_2 \\
        \intertext{Applying chain rule and intermediate result of \cref{proposition:overconfidence-softmax}:}
        = & \lim\limits_{\alpha_d \to \infty}\Big|\Big|\nabla_{\bx}\mathrm{H}\Big[\mathbb{E}_{p(\btheta \mid \mathbb{D})}\big[P_{\btheta}(y \mid \bx)\big]\Big] \nonumber \\
        & - \mathbb{E}_{p(\btheta \mid \mathbb{D})}\Big[\sum_{k=1}^K\big( 1 + \log P_{\btheta}(y=k \mid \bx)\big)\underbrace{\mystrut{0.275cm}{\nabla_{\bx}P_{\btheta}(y=k \mid \bx)}}_{=\ \bm{0} \text{ (\cref{proposition:overconfidence-softmax})}} \Big]\Big|_{\bx = \balpha\circ\bx^\prime}\Big|\Big|_2 \\
        % \intertext{Because this lets the entire second term become the zero vector in the limit, the remaining part reduces to the case proven in Lemma \ref{lemma:asymptotic-predictive-entropy}:}
        = & \underbrace{\lim\limits_{\alpha_d \to \infty}\Big|\Big|\nabla_{\bx}\mathrm{H}\Big[\mathbb{E}_{p(\btheta \mid \mathbb{D})}\big[P_{\btheta}(y \mid \bx)\big]\Big]\Big|_{\bx = \balpha\circ\bx^\prime} \Big|\Big|_2}_{\text{(\cref{lemma:asymptotic-predictive-entropy})}} = 0.
    \end{align}
   As the final result is an upper bound to the original expression and the $l_2$ norm provides a lower bound of $0$, this proves the lemma.
\end{proof}
% Appendix X

\chapter{Experimental Appendix}\label{app:empirical-appendix}
\epigraph{
    ``\emph{Machine learning has become alchemy.}''
}{---Ali Rahimi in his \href{https://www.youtube.com/watch?v=x7psGHgatGM}{NIPS 2017 Test of Time Award Talk}.}

%----------------------------------------------------------------------------------------

\begin{table}[htb]
    \centering
    \resizebox{0.6\textwidth}{!}{%
    \begin{tabular}{rl}
        \toprule
        Thesis & Appendix \\
        \toprule
        \cref{sec:experimental-comparison-aso} & \cref{app:aso-error-rate} \\
        \cref{sec:synthetic-experiments} & \cref{app:synthetic-data-experiments} \\
        \cref{sec:exploring-data-creation} & \cref{app:exploring-predictive-uncertainty-training-set,app:predictive-uncertainty-ood-test-set} \\
        \cref{sec:dependence-training-data} & \cref{app:additional-scatters} \\
        \cref{sec:qualitative-analysis} & \cref{app:qualitative-analysis} \\
        \cref{sec:retrieval-quality}  & \cref{app:coverage-experiments,app:ablations} \\
        \cref{sec:apricot-experiments} & \cref{app:additional-clustering} \\
        \cref{sec:calibration-experiments} & \cref{app:additional-calibration} \\
        \bottomrule
    \end{tabular}%
    }
    \caption[Correspondences between sections of the empirical appendix and thesis chapters.]{Correspondences between sections of the empirical appendix and thesis chapters.}\label{tab:empirical-appendix-correspondence}
\end{table}

This appendix involves a collection of additional empirical results stemming from the experiments in the different chapters.
An overview over the contents and their correspondence to thesis chapters is given in \cref{tab:empirical-appendix-correspondence}.
For more details regarding the reproducibility\index{Reproducibility} of experiments (hyperparameters, experimental settings etc.) refer to \cref{app:reproducibility-appendix}.

\section{Additional Error Rate Experiments}\label{app:aso-error-rate}

\begin{figure}[htb]
    \centering
    \begin{subfigure}[t]{0.485\textwidth}
        \includegraphics[width=\textwidth]{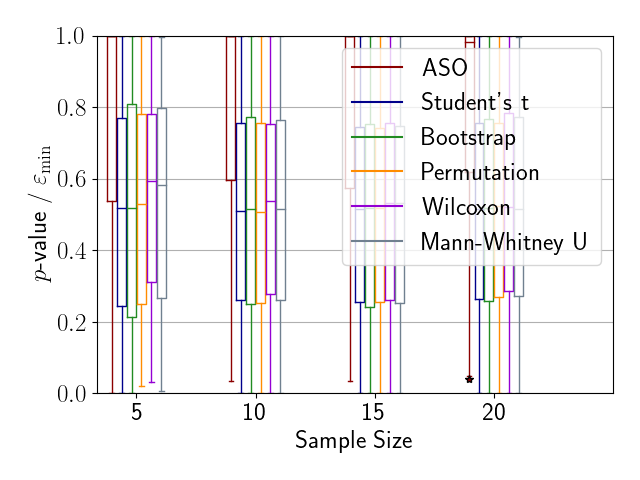}
        \caption{Dists.\@ for normal samples.}
        \label{subfig:type1_size_normal_dists}
    \end{subfigure}%
    \hfill
    \begin{subfigure}[t]{0.485\textwidth}
        \includegraphics[width=\textwidth]{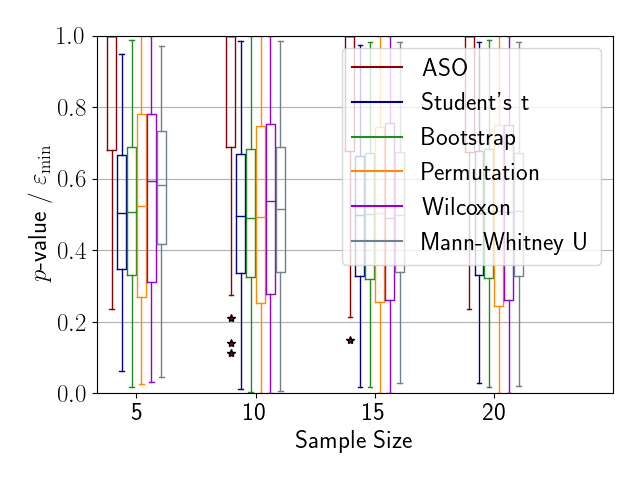}
        \caption{Dists.\@ for normal mixture samples.}
        \label{subfig:type1_size_normal_mixture_dists}
    \end{subfigure}%
    
    \begin{subfigure}[t]{0.485\textwidth}
        \includegraphics[width=\textwidth]{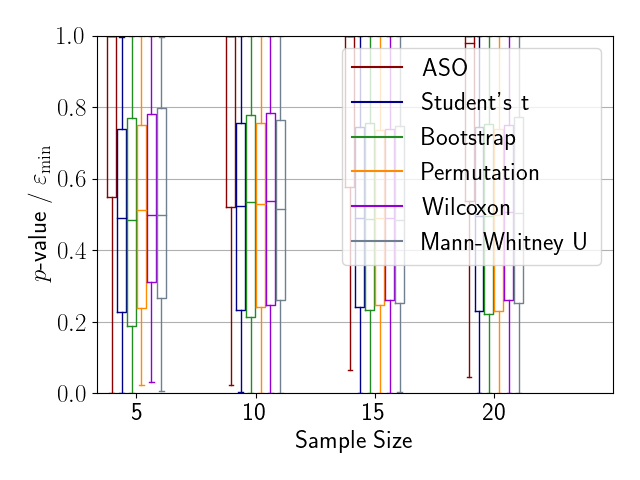}
        \caption{Dists.\@ for Laplace samples.}
        \label{subfig:type1_size_laplace_dists}
    \end{subfigure}%
    \hfill
    \begin{subfigure}[t]{0.485\textwidth}
        \includegraphics[width=\textwidth]{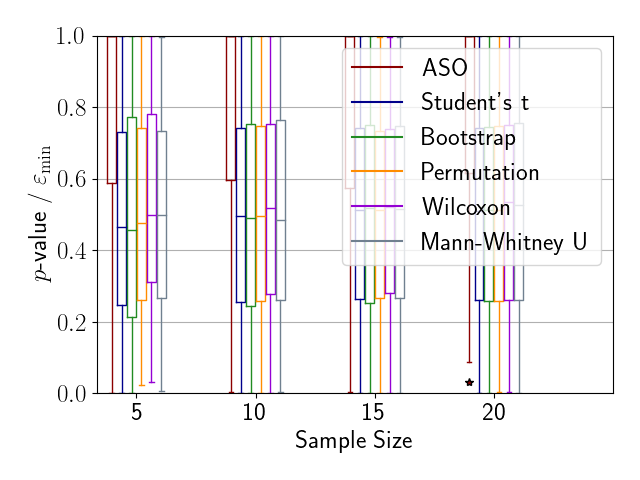}
        \caption{Dists.\@ for Rayleigh samples.}
        \label{subfig:type1_size_rayleigh_dists}
    \end{subfigure}
    \caption[Comparing test score distributions for different tests and distributions as a function of sample size.]{
        Comparing test score distributions for different tests and distributions as a function of sample size.
    }\label{fig:aso-tests-dists}
\end{figure}

We use this section to further shed light on the results in \cref{fig:aso-tests}. 

\paragraph{Test Score Distributions.} Instead of showing the Type I error\index{Type I error} rates based on thresholded test results, we instead plot the distributions over test scores in \cref{fig:aso-tests-dists}. 
We can observe that the lower ends of the interquartile range of $\epsmin$ distributions are either the same or higher than the ones for $p$-values (they do not need to be centered around $0.5$ since $\epsmin$ is an upper bound to $\varepsilon_{W_2}$), explaining the lower Type I error rate.

\begin{figure}[htb]
    \centering
    \begin{subfigure}[t]{0.475\textwidth}
        \includegraphics[width=0.95\textwidth]{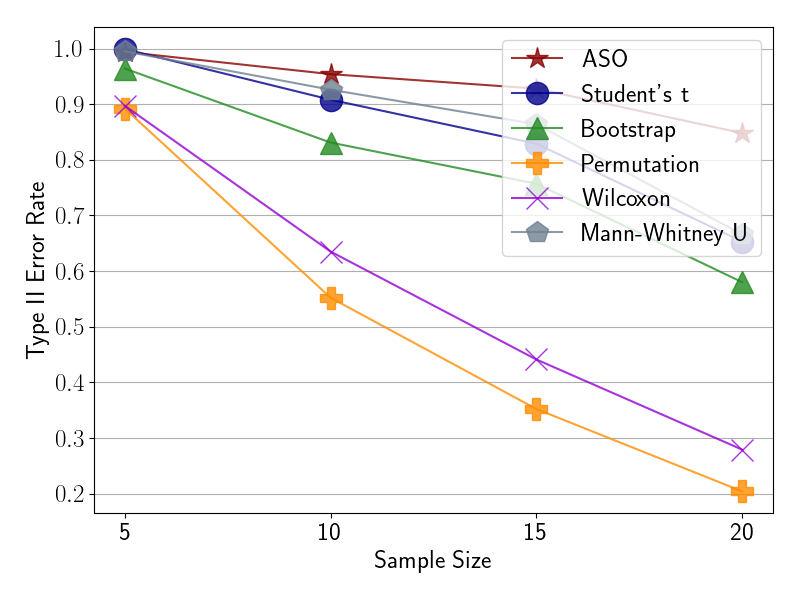}
        \caption{Type II error as a function of sample size.}
        \label{subfig:type2_normal_rates}
    \end{subfigure}%
    \hfill
    \begin{subfigure}[t]{0.475\textwidth}
        \includegraphics[width=0.95\textwidth]{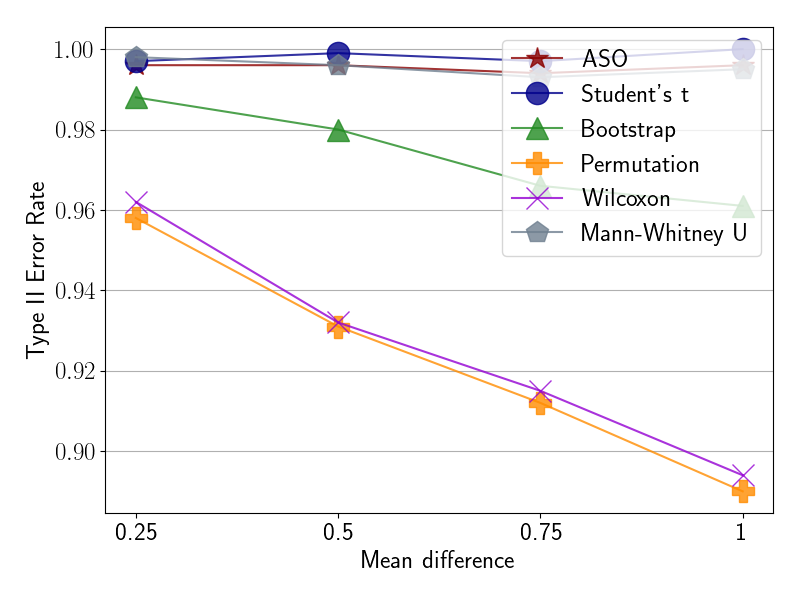}
        \caption{Type II error rate as a function of mean difference.}
        \label{subfig:type2_normal_mean_rate}
    \end{subfigure}%
    
    \begin{subfigure}[t]{0.475\textwidth}
        \includegraphics[width=0.95\textwidth]{img/deep-significance/normal_mixture/type2_rates.png}
        \caption{Type II error as a function of sample size.}
        \label{subfig:type2_normal_mixture_rates}
    \end{subfigure}%
    \hfill
    \begin{subfigure}[t]{0.475\textwidth}
        \includegraphics[width=0.95\textwidth]{img/deep-significance/normal_mixture/type2_mean_rates.png}
        \caption{Type II error rate as a function of mean difference.}
        \label{subfig:type2_normal_mixture_mean_rate}
    \end{subfigure}%
    \caption[Measuring the Type II error rate of the considered tests on normal and normal mixture distributions as a function of sample size.]{
        Measuring the Type II error rate of the considered tests on normal and normal mixture distributions as a function of sample size \cref{subfig:type2_normal_rates,subfig:type2_normal_mixture_rates} and mean differences \cref{subfig:type2_normal_mean_rate,subfig:type2_normal_mixture_mean_rate}.
    }\label{fig:type2-error-rates}
\end{figure}

\paragraph{Type II Error Rate Experiments.} 
We furthermore test the Type II error \index{Type II error} rates on samples from different distributions in \cref{fig:type2-error-rates}, sampling the score samples $500$ times for ASO and $1000$ times for the other tests from $\mathcal{N}(0.5, 1.5^2)$ and $\mathcal{N}(0, 1.5^2)$,\footnote{For the normal mixture, only the second mixture component is varied.} respectively, for a $p$-value threshold of $0.05$ and $\epsmin$ threshold of $0.2$. 
We see that the Type II error rate decreases with increasing sample size (\cref{subfig:type2_normal_rates,subfig:type2_normal_mixture_rates}), but is less sensitive for increasing mean difference than other tests (\cref{subfig:type2_normal_mean_rate,subfig:type2_normal_mixture_mean_rate}). 
Generally, we can observe the behavior to be very similar to Student's-$t$ \index{Student's-$t$ test} and Mann-Whitney U test\index{Mann-Whitney U test}.

\paragraph{Error Rates by Rejection Threshold.} Lastly, we report the Type I and II error rates \index{Type I error}\index{Type II error} on the tested distributions using different Type I / II error rates. 
In \cref{table:type1-normal-tresholds,table:type1-normal-mixture-tresholds,table:type1-laplace-tresholds,table:type1-rayleigh-tresholds}, we see that ASO achieves lower error rates than other tests in almost all scenarios when faced with the fame threshold. 
Naturally, these thresholds cannot be interpreted the same for ASO\index{Stochastic order!Almost} and the other significance tests. 
Nevertheless, we can see that a threshold of $\tau = 0.2$ seems to roughly correspond to a $p$-value threshold of $0.05$ in terms of Type I error rate\index{Type I error}. 
Type II error\index{Type II error} rates are given in \cref{table:type2-normal-tresholds,table:type2-normal-mean-tresholds,table:type2-normal-mixture-tresholds,table:type2-normal-mixture-mean-tresholds}. Here the difference between ASO and the other tests is not quite as pronounced, however, it always incurs higher error rates.

\begin{table}[htb]
    \resizebox{0.975\textwidth}{!}{%
    \begin{tabular}{llrrrrrr}
    \toprule
      Sample Size & $\tau$       &    ASO & Student's t & Bootstrap & Permutation & Wilcoxon & Mann-Whitney U \\
    \midrule
    5  & .05 &   \textbf{.020} &       .048 &     .085 &       .029 &    .029 &          .056 \\
       & .10 &  \textbf{.034} &       .093 &     .149 &       .079 &    .088 &          .085 \\
       & .20 &   \textbf{.006} &       .212 &     .241 &       .197 &     .160 &          .159 \\
       & .30 &  \textbf{.094} &       .299 &     .322 &       .286 &    .236 &          .284 \\
       & .40 &  \textbf{.146} &       .396 &     .403 &        .370 &    .315 &          .348 \\
       & .50 &  \textbf{.216} &       .483 &     .483 &       .468 &     .490 &          .498 \\
       \midrule
    10 & .05 &  \textbf{.004} &       .055 &     .077 &       .058 &    .051 &          .048 \\
       & .10 &  .014 &       .103 &      .130 &        \textbf{.110} &    .113 &            \textbf{.100} \\
       & .20 &  \textbf{.038} &       .196 &     .215 &       .201 &    .192 &          .194 \\
       & .30 &  \textbf{.084} &       .282 &       .300 &       .285 &    .261 &          .272 \\
       & .40 &  \textbf{.138} &       .394 &     .398 &       .395 &    .387 &          .378 \\
       & .50 &  \textbf{.204} &        .409 &     .486 &       .491 &    .499 &          .479 \\
       \midrule
    15 & .05 &  \textbf{.002} &       .059 &     .072 &       .057 &    .051 &          .052 \\
       & .10 &  \textbf{.014} &       .106 &     .123 &       .104 &    .095 &          .113 \\
       & .20 &  \textbf{.042} &       .198 &     .215 &       .199 &    .186 &          .196 \\
       & .30 &   \textbf{.080} &       .303 &     .309 &       .303 &    .295 &          .304 \\
       & .40 &  \textbf{.136} &       .395 &       .400 &       .392 &    .371 &          .368 \\
       & .50 &   \textbf{.190} &       .482 &     .485 &       .479 &     .470 &          .468 \\
       \midrule
    20 & .05 &  \textbf{.004} &       .046 &     .058 &       .047 &    .043 &          .047 \\
       & .10 &  \textbf{.006} &       .095 &     .105 &       .093 &    .085 &          .092 \\
       & .20 &  \textbf{.028} &       .181 &     .196 &       .177 &    .171 &          .183 \\
       & .30 &  \textbf{.074} &        .280 &      .290 &       .289 &    .284 &          .273 \\
       & .40 &   \textbf{.120} &       .384 &     .389 &       .381 &    .372 &          .394 \\
       & .50 &   \textbf{.170} &       .479 &     .478 &       .473 &    .477 &          .481 \\
    \bottomrule
    \end{tabular}%
    }
    \caption[Type I error rates for samples drawn from a normal distribution as a function of sample size and different rejection thresholds.]{Type I error rates for samples drawn from a normal distribution as a function of sample size and different rejection thresholds.}\label{table:type1-normal-tresholds}
\end{table}

\begin{table}[htb]
    \resizebox{0.975\textwidth}{!}{%
    \begin{tabular}{lllrrrrrr}
    \toprule
      Sample Size & $\tau$     &    ASO & Student's t & Bootstrap & Permutation & Wilcoxon & Mann-Whitney U \\
    \midrule
    5  & .05 &  .942 &       .883 &     \textbf{.796} &       .918 &    .925 &          .875 \\
   & .10 &  .916 &       .786 &     \textbf{.714} &       .802 &    .792 &          .819 \\
   & .20 &   .870 &       .623 &     \textbf{.585} &       .649 &    .691 &          .694 \\
   & .30 &  .792 &       .512 &      \textbf{.480} &       .521 &    .597 &          .539 \\
   & .40 &  .714 &       .399 &      \textbf{.309} &       .421 &    .498 &           .470 \\
   & .50 &   .650 &       \textbf{.302} &     .315 &       .318 &    .387 &          .391 \\
   \midrule
10 & .05 &  .978 &       .836 &     \textbf{.791} &       .853 &    .864 &           .840 \\
   & .10 &   .950 &        .703 &     \textbf{.695} &       .737 &    .743 &          .741 \\
   & .20 &  .868 &        .580 &     \textbf{.551} &        .58 &    .595 &          .576 \\
   & .30 &  .802 &       .428 &      \textbf{.41} &       .429 &    .462 &          .453 \\
   & .40 &  .708 &        .330 &     \textbf{.328} &       .327 &    .347 &          .329 \\
   & .50 &  .604 &       \textbf{.223} &     .223 &       .229 &    .272 &          .251 \\
   \midrule
15 & .05 &  .984 &       .769 &     .734 &       .781 &    .788 &          .787 \\
   & .10 &   .905 &       .643 &     \textbf{.615} &       .646 &    .672 &          .639 \\
   & .20 &   .840 &        \textbf{.470} &     .455 &        .480 &    .493 &          .481 \\
   & .30 &  .716 &       .348 &      \textbf{.340} &        .350 &    .355 &          .365 \\
   & .40 &   .610 &       \textbf{.244} &     .245 &       .246 &    .276 &          .261 \\
   & .50 &  .486 &       .177 &     .176 &       \textbf{.175} &    .185 &          .192 \\
   \midrule
20 & .05 &  .976 &       .732 &     \textbf{.709} &       .736 &     .750 &          .747 \\
   & .10 &  .946 &       .601 &     \textbf{.586} &       .601 &    .614 &           .610 \\
   & .20 &  .848 &       .406 &     \textbf{.396} &        .410 &    .421 &           .410 \\
   & .30 &  .704 &       .277 &     \textbf{.268} &       .272 &    .299 &          .289 \\
   & .40 &   .508 &         \textbf{.200} &     \textbf{.201} &       .198 &    .221 &          .206 \\
   & .50 &  .444 &       \textbf{.144} &     \textbf{.144} &       .147 &    .156 &          .152 \\
    \bottomrule
    \end{tabular}%
    }
    \caption{Type II error rates for normal samples as a function of sample size and different rejection thresholds.}\label{table:type2-normal-tresholds}
\end{table}

\begin{table}[htb]
    \resizebox{0.975\textwidth}{!}{%
    \begin{tabular}{llrrrrrr}
    \toprule
       Difference  & $\tau$       &    ASO & Student's t & Bootstrap & Permutation & Wilcoxon & Mann-Whitney U \\
    \midrule
    .25 & .05 &  .984 &       .925 &     \textbf{.857} &       .941 &    .945 &           .930 \\
     & .10 &  .954 &       .846 &     \textbf{.781} &       .859 &    .844 &          .881 \\
     & .20 &  .914 &       .705 &     \textbf{.659} &       .721 &    .761 &          .768 \\
     & .30 &  .872 &       \textbf{.585} &     .554 &       .606 &     .680 &          .622 \\
     & .40 &    .800 &       .482 &     \textbf{.462} &       .489 &    .594 &          .548 \\
     & .50 &  .714 &       \textbf{.381} &     .387 &       .394 &     .480 &          .465 \\
     \midrule
.50 & .05 &  .966 &       .888 &     \textbf{.805} &       .918 &     .920 &          .883 \\
     & .10 &  .932 &       .784 &       \textbf{.700} &       .811 &    .794 &           .830 \\
     & .20 &   .870 &       .616 &      \textbf{.570} &       .652 &    .696 &          .698 \\
     & .30 &  .812 &         .500 &     \textbf{.477} &       .523 &    .602 &          .535 \\
     & .40 &  .722 &       .406 &     \textbf{.397} &       .426 &    .505 &          .466 \\
     & .50 &  .606 &       \textbf{.313} &     .315 &       .326 &    .411 &          .401 \\
     \midrule
.75 & .05 &  .934 &       .822 &     \textbf{.707} &       .883 &    .885 &          .822 \\
     & .10 &  .896 &       .699 &      \textbf{.610} &       .725 &     .710 &          .764 \\
     & .20 &  .798 &       .514 &     \textbf{.469} &       .561 &    .599 &          .607 \\
     & .30 &  .702 &       .407 &      \textbf{.370} &       .421 &    .515 &          .455 \\
     & .40 &   .590 &       .308 &       \textbf{.300} &       .325 &    .406 &          .375 \\
     & .50 &  .482 &       \textbf{.223} &     \textbf{.222} &       .237 &    .303 &          .295 \\
     \midrule
1.00 & .05 &   .870 &       .739 &     \textbf{.609} &        .850 &     .850 &          .743 \\
     & .10 &  .796 &       .585 &     \textbf{.488} &       .678 &    .655 &          .659 \\
     & .20 &  .712 &       .386 &     \textbf{.327} &       .449 &    .497 &          .487 \\
     & .30 &   .580 &       .257 &     \textbf{.232} &       .289 &    .388 &          .307 \\
     & .40 &  .504 &       .178 &      \textbf{.170} &       .194 &    .278 &          .229 \\
     & .50 &  .384 &       \textbf{.115} &     \textbf{.115} &       .128 &    .189 &          .176 \\
    \bottomrule
    \end{tabular}%
    }
    \caption{Type II error rates for normal samples as a function of mean difference and different rejection thresholds.}\label{table:type2-normal-mean-tresholds}
\end{table}

\begin{table}[htb]
    \centering
    \resizebox{0.975\textwidth}{!}{%
    \begin{tabular}{llrrrrrr}
        \toprule
          Sample Size & $\tau$   & ASO & Student's t & Bootstrap & Permutation & Wilcoxon & Mann-Whitney U \\
        \midrule
        5  & .05 & \textbf{.000} & \textbf{.000} & .012 & .028 & .026 & .003 \\
           & .10 & \textbf{.000} & .013 & .035 & .079 & .085 & .004 \\
           & .20 & \textbf{.000} & .069 & .104 & .179 & .153 & .049 \\
           & .30 & \textbf{.008} & .169 & .213 & .281 & .208 & .160 \\
           & .40 & \textbf{.024} & .338 & .358 & .363 & .305 & .244 \\
           & .50 & \textbf{.058} & .494 & .493 & .483 & .484 & .478 \\
        \midrule
        10 & .05 & \textbf{.000} & .007 & .018 & .059 & .049 & .011 \\
           & .10 & \textbf{.000} & .031 & .050 & .110 & .109 & .030 \\
           & .20 & \textbf{.004} & .102 & .121 & .205 & .188 & .109 \\
           & .30 & \textbf{.008} & .221 & .229 & .302 & .273 & .211 \\
           & .40 & \textbf{.034} & .347 & .349 & .398 & .379 & .351 \\
           & .50 & \textbf{.070} & .511 & .515 & .506 & .491 & .495 \\
        \midrule
        15 & .05 & \textbf{.000} & .006 & .007 & .055 & .048 & .004 \\
           & .10 & \textbf{.000} & .022 & .033 & .106 & .097 & .017 \\
           & .20 & \textbf{.002} & .103 & .118 & .194 & .202 & .095 \\
           & .30 & \textbf{.006} & .215 & .220 & .301 & .308 & .208 \\
           & .40 & \textbf{.028} & .356 & .366 & .415 & .404 & .328 \\
           & .50 & \textbf{.082} & .501 & .499 & .496 & .502 & .501 \\
        \midrule
        20 & .05 & \textbf{.000} & .006 & .007 & .048 & .045 & .005 \\
           & .10 & \textbf{.000} & .019 & .027 & .088 & .085 & .021 \\
           & .20 & \textbf{.000} & .104 & .109 & .200 & .187 & .097 \\
           & .30 & \textbf{.006} & .214 & .218 & .307 & .289 & .221 \\
           & .40 & \textbf{.032} & .363 & .369 & .412 & .390 & .349 \\
           & .50 & \textbf{.082} & .494 & .495 & .492 & .496 & .485 \\
        \bottomrule
    \end{tabular}
    }
    \caption{Type I error rates for normal mixture samples as a function of sample size and different rejection thresholds.}\label{table:type1-normal-mixture-tresholds}
\end{table}

\begin{table}[htb]
    \resizebox{0.975\textwidth}{!}{%
    \begin{tabular}{llrrrrrr}
        \toprule
          Sample Size & $\tau$   & ASO & Student's t & Bootstrap & Permutation & Wilcoxon & Mann-Whitney U \\
        \midrule
        5  & .05 & 1.000 & .999 & .964 & \textbf{.892} & .897 & .995 \\
           & .10 & 1.000 & .962 & .874 & .728 & \textbf{.697} & .985 \\
           & .20 & .994 & .747 & .640 & \textbf{.474} & .525 & .870 \\
           & .30 & .976 & .476 & .422 & \textbf{.299} & .426 & .579 \\
           & .40 & .896 & .252 & .234 & \textbf{.206} & .326 & .414 \\
           & .50 & .748 & \textbf{.117} & \textbf{.118} & .122 & .222 & .280 \\
        \midrule
        10 & .05 & 1.000 & .908 & .831 & \textbf{.552} & .635 & .926 \\
           & .10 & .996 & .721 & .641 & \textbf{.354} & .419 & .730 \\
           & .20 & .954 & .390 & .354 & \textbf{.186} & .247 & .407 \\
           & .30 & .828 & .191 & .180 & \textbf{.108} & .156 & .219 \\
           & .40 & .642 & .089 & .087 & \textbf{.068} & .089 & .107 \\
           & .50 & .452 & .034 & \textbf{.031} & .037 & .056 & .052 \\
        \midrule
        15 & .05 & .996 & .829 & .757 & \textbf{.352} & .441 & .864 \\
           & .10 & .990 & .568 & .517 & \textbf{.213} & .272 & .628 \\
           & .20 & .928 & .251 & .234 & \textbf{.087} & .129 & .298 \\
           & .30 & .774 & .099 & .091 & \textbf{.033} & .058 & .116 \\
           & .40 & .498 & .027 & .026 & \textbf{.019} & .034 & .044 \\
           & .50 & .276 & \textbf{.009} & \textbf{.010} & \textbf{.010} & .013 & .014 \\
        \midrule
        20 & .05 & 1.000 & .653 & .580 & \textbf{.204} & .279 & .666 \\
           & .10 & .980 & .359 & .333 & \textbf{.105} & .162 & .392 \\
           & .20 & .848 & .107 & .101 & \textbf{.035} & .064 & .147 \\
           & .30 & .586 & .038 & .035 & \textbf{.013} & .022 & .047 \\
           & .40 & .344 & .010 & .010 & \textbf{.008} & .013 & .017 \\
           & .50 & .130 & \textbf{.003} & \textbf{.003} & \textbf{.004} & .006 & .006 \\
        \bottomrule
    \end{tabular}%
    }
    \caption{Type II error rates for normal mixture samples as a function of sample size and different rejection thresholds.}\label{table:type2-normal-mixture-tresholds}
\end{table}

\begin{table}[htb]
    \resizebox{0.975\textwidth}{!}{%
    \begin{tabular}{llrrrrrr}
        \toprule
         Diff.    & $\tau$     & ASO & Student's t & Bootstrap & Permutation & Wilcoxon & Mann-Whitney U \\
        %diff & threshold & & & & & & \\
        \midrule
        .25 & .05 & 1.000 & .997 & .988 & .958 & .962 & .998 \\
             & .10 & .998 & .988 & .960 & .894 & \textbf{.882} & .994 \\
             & .20 & .996 & .903 & .856 & \textbf{.754} & .792 & .945 \\
             & .30 & .978 & .762 & .727 & \textbf{.643} & .724 & .814 \\
             & .40 & .940 & .594 & .576 & \textbf{.530} & .621 & .704 \\
             & .50 & .886 & \textbf{.424} & \textbf{.424} & .444 & .532 & .563 \\
        \midrule
        .50 & .05 & .998 & .999 & .980 & \textbf{.931} & .932 & .996 \\
             & .10 & .998 & .978 & .931 & .820 & \textbf{.802} & .990 \\
             & .20 & .996 & .849 & .775 & \textbf{.647} & .695 & .905 \\
             & .30 & .976 & .659 & .603 & \textbf{.511} & .611 & .724 \\
             & .40 & .928 & .458 & .438 & \textbf{.407} & .504 & .577 \\
             & .50 & .840 & \textbf{.284} & .287 & .310 & .395 & .449 \\
        \midrule
        .75 & .05 & 1.000 & .997 & .966 & \textbf{.912} & .915 & .993 \\
             & .10 & .998 & .966 & .901 & .769 & \textbf{.746} & .985 \\
             & .20 & .994 & .802 & .707 & \textbf{.553} & .623 & .886 \\
             & .30 & .974 & .547 & .497 & \textbf{.397} & .516 & .651 \\
             & .40 & .922 & .355 & .337 & \textbf{.286} & .407 & .485 \\
             & .50 & .824 & \textbf{.191} & \textbf{.191} & .198 & .305 & .363 \\
        \midrule
        1.00 & .05 & 1.000 & 1.000 & .961 & \textbf{.890} & .894 & .995 \\
             & .10 & 1.000 & .958 & .868 & .714 & \textbf{.682} & .989 \\
             & .20 & .996 & .715 & .617 & \textbf{.445} & .505 & .872 \\
             & .30 & .962 & .432 & .380 & \textbf{.291} & .419 & .545 \\
             & .40 & .870 & .253 & .235 & \textbf{.204} & .308 & .408 \\
             & .50 & .702 & \textbf{.120} & \textbf{.120} & .132 & .208 & .263 \\
        \bottomrule
    \end{tabular}%
    }
    \caption{Type II error rates for normal mixture samples as a function of mean difference between two of the mixture components and different rejection thresholds.}\label{table:type2-normal-mixture-mean-tresholds}
\end{table}

\begin{table}[htb]
    \resizebox{0.975\textwidth}{!}{%
    \begin{tabular}{llrrrrrr}
        \toprule
          Sample Size & $\tau$       & ASO & Student's t & Bootstrap & Permutation & Wilcoxon & Mann-Whitney U \\
        \midrule
        5  & .05 & \textbf{.022} & .053 & .110 & .048 & .046 & .066 \\
           & .10 & \textbf{.038} & .117 & .164 & .106 & .116 & .097 \\
           & .20 & \textbf{.088} & .223 & .261 & .208 & .187 & .169 \\
           & .30 & \textbf{.124} & .319 & .343 & .295 & .234 & .286 \\
           & .40 & \textbf{.154} & .427 & .445 & .398 & .322 & .379 \\
           & .50 & \textbf{.218} & .509 & .510 & .491 & .506 & .508 \\
        \midrule
        10 & .05 & \textbf{.004} & .059 & .077 & .060 & .046 & .051 \\
           & .10 & \textbf{.012} & .114 & .142 & .111 & .106 & .098 \\
           & .20 & \textbf{.056} & .218 & .236 & .216 & .202 & .199 \\
           & .30 & \textbf{.104} & .314 & .330 & .318 & .290 & .291 \\
           & .40 & \textbf{.164} & .404 & .407 & .398 & .378 & .400 \\
           & .50 & \textbf{.238} & .475 & .475 & .473 & .481 & .486 \\
        \midrule
        15 & .05 & \textbf{.000} & .052 & .066 & .048 & .048 & .048 \\
           & .10 & \textbf{.012} & .100 & .117 & .103 & .100 & .101 \\
           & .20 & \textbf{.028} & .204 & .220 & .199 & .199 & .187 \\
           & .30 & \textbf{.070} & .311 & .319 & .303 & .296 & .294 \\
           & .40 & \textbf{.120} & .404 & .409 & .402 & .378 & .394 \\
           & .50 & \textbf{.194} & .510 & .514 & .511 & .504 & .519 \\
        \midrule
        20 & .05 & \textbf{.004} & .044 & .047 & .048 & .057 & .052 \\
           & .10 & \textbf{.010} & \textbf{.099} & .113 & .104 & .103 & .101 \\
           & .20 & \textbf{.030} & .214 & .232 & .215 & .199 & .202 \\
           & .30 & \textbf{.064} & .312 & .325 & .308 & .297 & .307 \\
           & .40 & \textbf{.138} & .414 & .413 & .415 & .381 & .405 \\
           & .50 & \textbf{.220} & .507 & .505 & .501 & .485 & .496 \\
        \bottomrule
    \end{tabular}
    }
    \caption{Type I error rates for samples drawn from a Laplace distribution as a function of sample size and different rejection thresholds.}\label{table:type1-laplace-tresholds}
\end{table}

\begin{table}[htb]
    \resizebox{0.975\textwidth}{!}{%
    \begin{tabular}{llrrrrrr}
        \toprule
          Sample Size & $\tau$   & ASO & Student's t & Bootstrap & Permutation & Wilcoxon & Mann-Whitney U \\
        \midrule
        5  & .05 & \textbf{.012} & .054 & .107 & .028 & .028 & .054 \\
           & .10 & \textbf{.034} & .108 & .147 & .089 & .096 & .088 \\
           & .20 & \textbf{.076} & .203 & .235 & .187 & .162 & .165 \\
           & .30 & \textbf{.110} & .319 & .342 & .302 & .229 & .291 \\
           & .40 & \textbf{.146} & .423 & .435 & .415 & .331 & .360 \\
           & .50 & \textbf{.198} & .532 & .539 & .523 & .530 & .524 \\
        \midrule
        10 & .05 & \textbf{.012} & .046 & .062 & .043 & .039 & .041 \\
           & .10 & \textbf{.018} & .087 & .107 & .093 & .094 & .084 \\
           & .20 & \textbf{.044} & .187 & .206 & .180 & .172 & .187 \\
           & .30 & \textbf{.064} & .295 & .314 & .297 & .265 & .284 \\
           & .40 & \textbf{.114} & .401 & .405 & .399 & .373 & .412 \\
           & .50 & \textbf{.180} & .507 & .514 & .505 & .500 & .508 \\
        \midrule
        15 & .05 & \textbf{.004} & .050 & .064 & .049 & .050 & .054 \\
           & .10 & \textbf{.010} & .100 & .115 & .100 & .103 & .104 \\
           & .20 & \textbf{.036} & .194 & .201 & .182 & .187 & .187 \\
           & .30 & \textbf{.070} & .295 & .302 & .287 & .294 & .291 \\
           & .40 & \textbf{.114} & .386 & .394 & .379 & .371 & .373 \\
           & .50 & \textbf{.198} & .481 & .484 & .487 & .472 & .497 \\
        \midrule
        20 & .05 & \textbf{.002} & .054 & .064 & .059 & .055 & .052 \\
           & .10 & \textbf{.004} & .115 & .121 & .113 & .103 & .113 \\
           & .20 & \textbf{.030} & .195 & .205 & .202 & .187 & .204 \\
           & .30 & \textbf{.058} & .281 & .287 & .277 & .283 & .291 \\
           & .40 & \textbf{.130} & .377 & .386 & .375 & .368 & .384 \\
           & .50 & \textbf{.190} & .489 & .493 & .493 & .468 & .469 \\
        \bottomrule
    \end{tabular}%
    }
    \caption{Type I error rates for samples drawn from a Rayleigh distribution as a function of sample size and different rejection thresholds.}\label{table:type1-rayleigh-tresholds}
\end{table}

\section{Synthetic Data Experiments}\label{app:synthetic-data-experiments}

This sections provides more details on the results in \cref{sec:synthetic-experiments}.
All of the plots produced can be found in \cref{fig:app-single-pred-nn,fig:app-multiple-pred-nn}, where uncertainty values where plotted for different ranges depending on the metric (variance: $0$-$0.25$; (negative) entropy: $0$-$1$; mutual information: $4-5$; ($1 -$) max.\@ prob: $0 - 0.5$), with deep purple signifying high uncertainty and white signifying low uncertainty / high certainty.\index{Uncertainty}
We can see in \cref{fig:app-single-pred-nn} that maximum softmax probability and predictive entropy\index{Entropy!Predictive} behave quite similarly, forming a tube-like region of high uncertainty\index{Uncertainty} along what appear to be the decision boundary. 
In both cases, the region appears to be sharper in the case of maximum softmax probability (right column) and also more defined after additional temperature scaling (bottom row). 
For all models and metrics, we see that the gradient magnitude decreases and approaches zero away from the training data (yellow / green plots), except for the cases discussed in \cref{sec:synthetic-experiments}.\\

\begin{figure}[htb]
    \centering
    \resizebox{0.8\textwidth}{!}{%
    \begin{tabular}{rll}
         & \multicolumn{1}{c}{\small Predictive Entropy} & \multicolumn{1}{c}{\small Maximum probability} \\
        \rotatebox{90}{\hspace{-1.5cm}Neural Network}             & \includegraphics[width=0.28\textwidth]{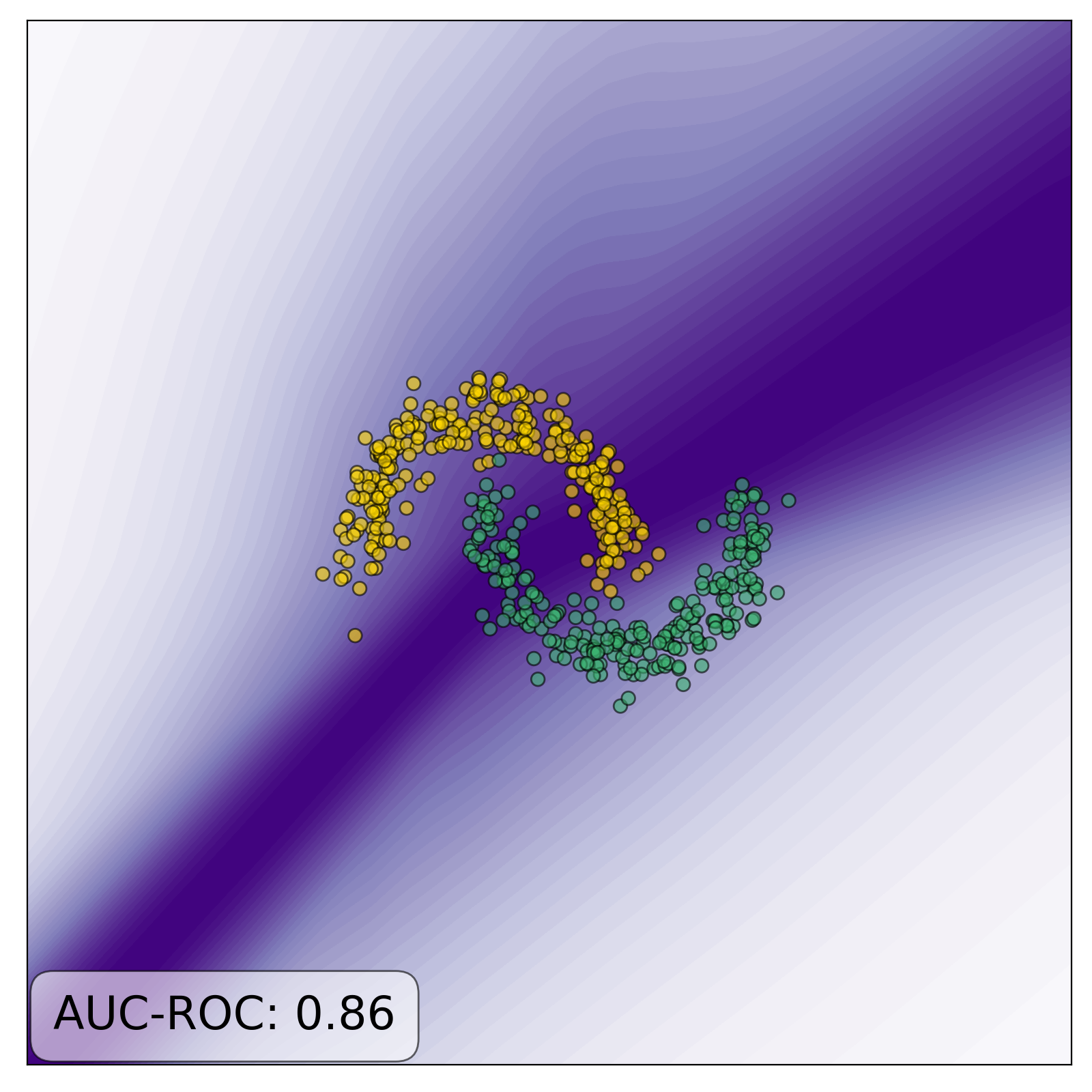} & \includegraphics[width=0.28\textwidth]{img/know_your_limits/nn_max_prob.png} \\
        & \includegraphics[width=0.3\textwidth]{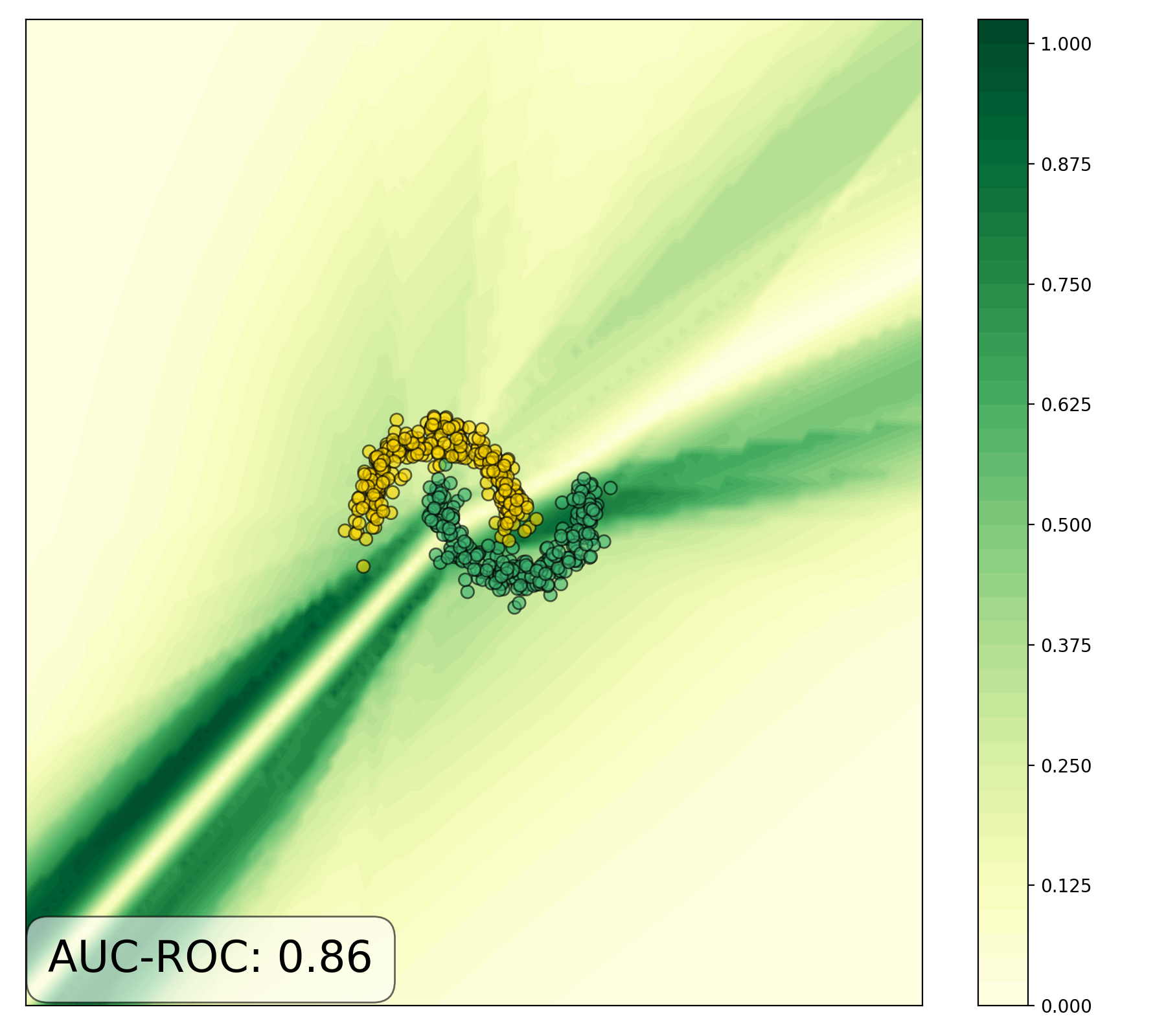} & \includegraphics[width=0.3\textwidth]{img/know_your_limits/nn_max_prob_grads.png} \\
        \midrule
        \rotatebox{90}{\hspace{-1.25cm}Platt scaling} & \includegraphics[width=0.28\textwidth]{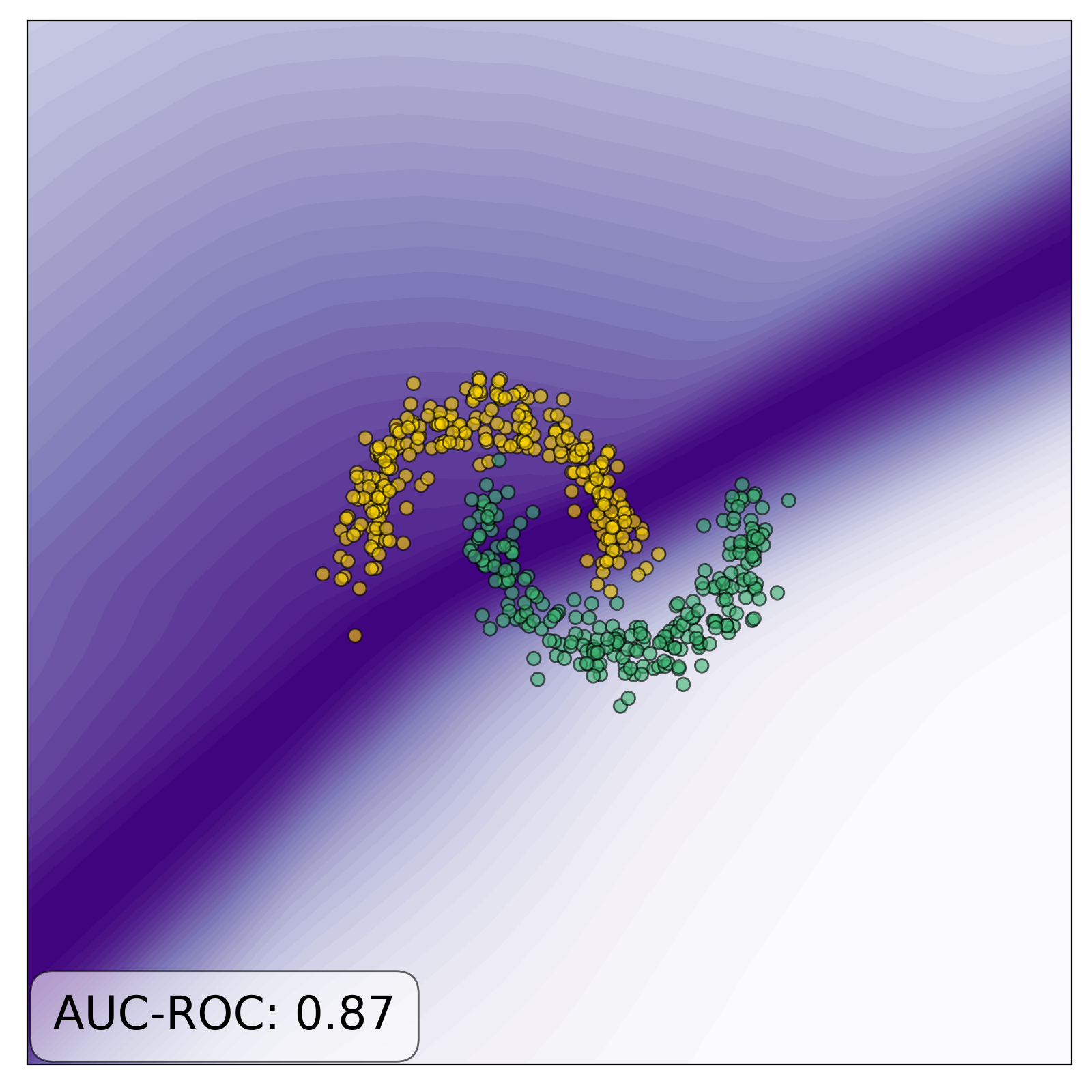} & \includegraphics[width=0.28\textwidth]{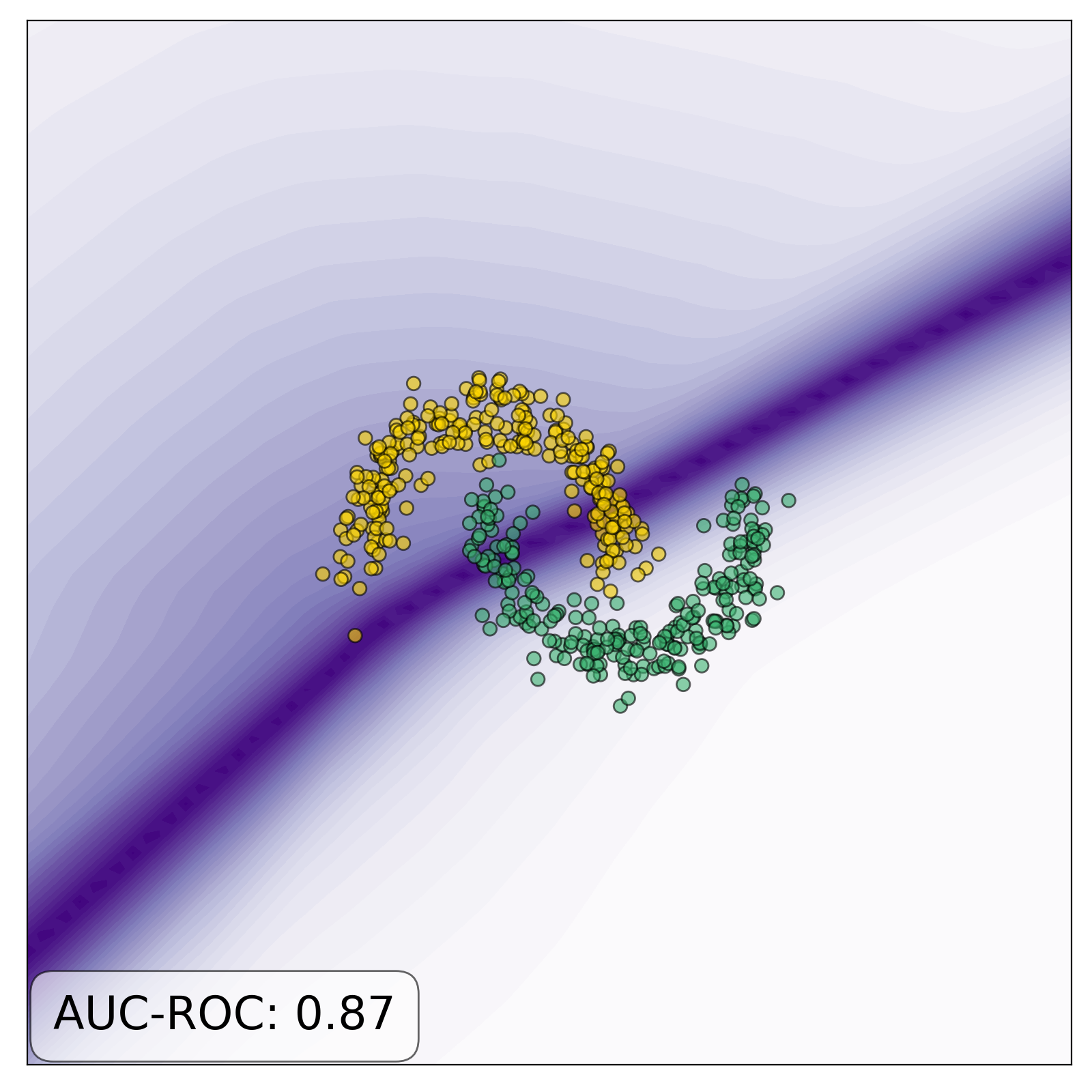}  \\
        & \includegraphics[width=0.3\textwidth]{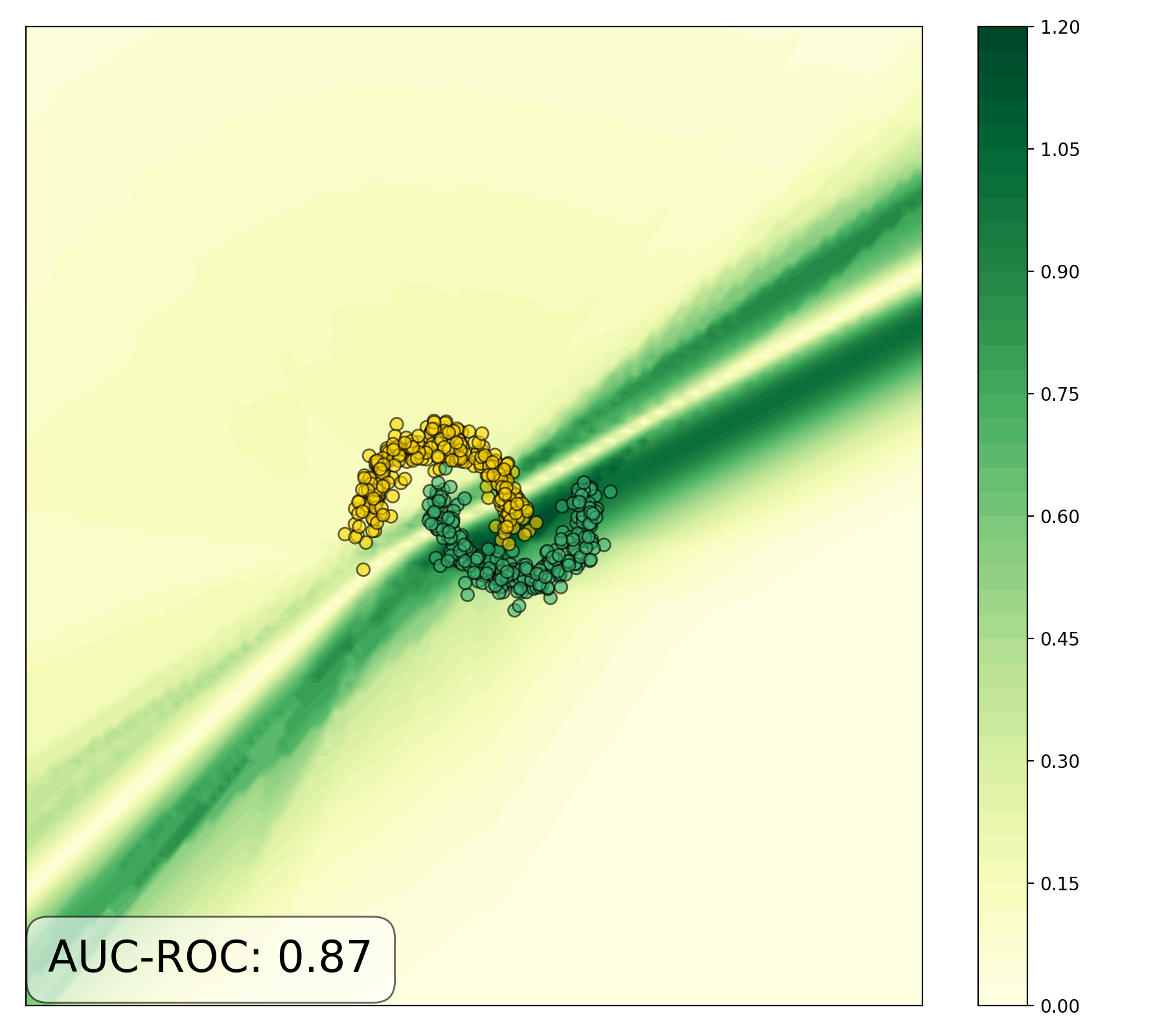} & \includegraphics[width=0.3\textwidth]{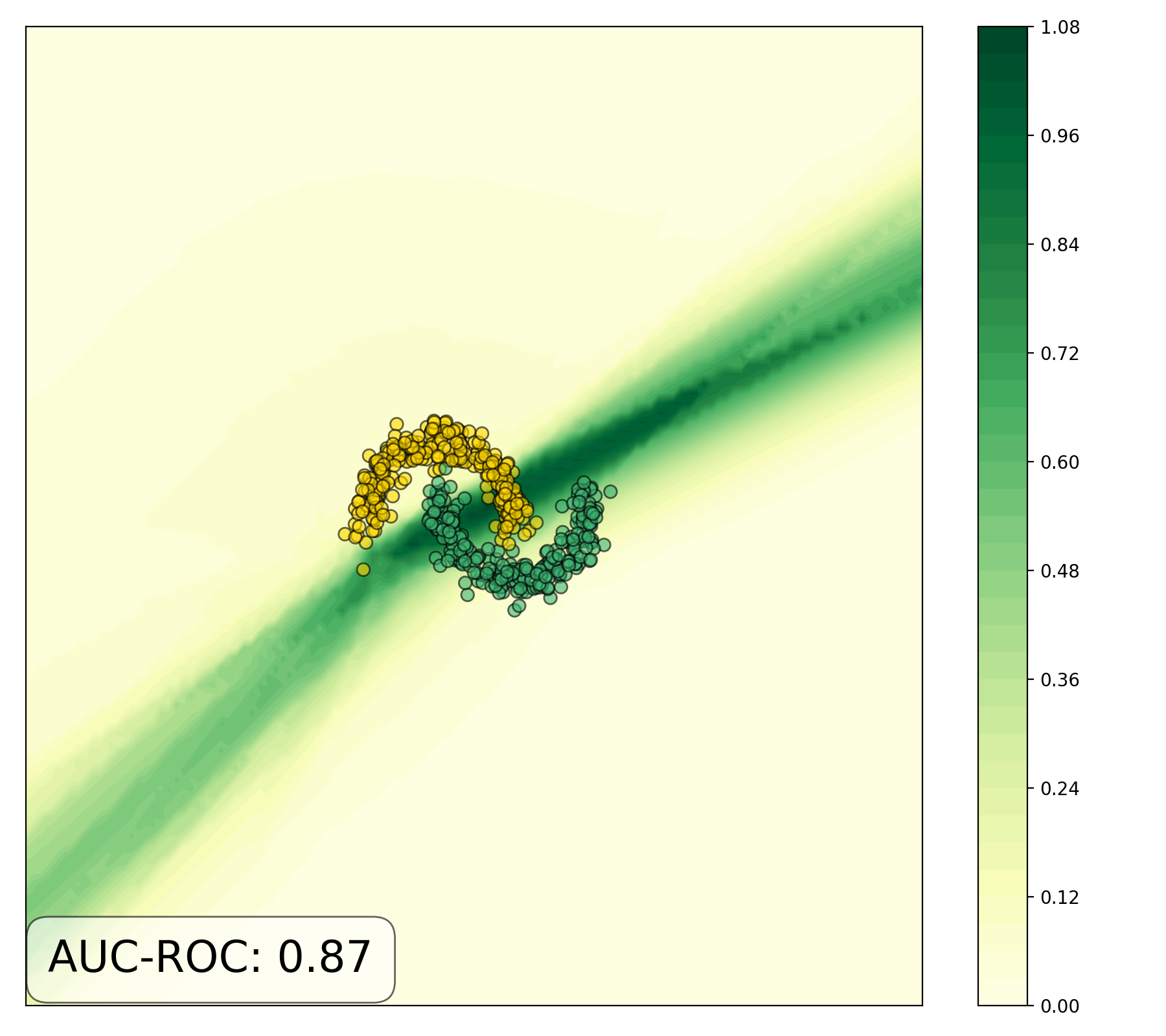}  \\
    \end{tabular}%
    }
    \caption[Uncertainty measured by different metrics for single-instance models and their gradient magnitude.]{
        Uncertainty measured by different metrics for single-instance models (purple plots) and their gradient magnitude (yellow / green plots).}
    \label{fig:app-single-pred-nn}
\end{figure}

In the next figure, \cref{fig:app-multiple-pred-nn}, we observe the uncertainty surfaces for models using multiple network instances.\index{Uncertainty} 
%as the AUC-ROC falls behind the other approaches and the resulting surfaces appear mostly flat, implying that the model didn't converge to a good local minimum. 
For the remaining models it is interesting to see that class variance\index{Class variance} (left column) didn't seem to produce significantly different values across the feature space except for the anchored ensemble. 
For predictive entropy\index{Entropy!Predictive} (central column), we can see a similar behavior compared to the single-instances models. 
Interestingly, the ``fuzziness'' of the high-uncertainty region increases with the ensemble and becomes increasing large with its anchored variant. 
Nevertheless, regions with static levels of certainty still exist in this case. 
For the mutual information\index{Mutual information} plots (right column), epistemic uncertainty\index{Uncertainty!Epistemic} is lowest around the training data, where the model is best specified, which creates another tube-like region of high confidence\index{Confidence} even where there is no training data, an effect that is reduced with the neural ensemble and almost completely solved by the anchored ensemble.\index{Ensembling} 
For all metrics, we see a magnitude close to zero for the uncertainty gradient away from the training data, except for the decision boundaries, as discussed in \cref{sec:synthetic-experiments}.

\begin{figure*}[htb]
    \centering
    \resizebox{0.98\textwidth}{!}{%
        \begin{tabular}{rlll}
             & \multicolumn{1}{c}{\small Class variance} & \multicolumn{1}{c}{\small Predictive Entropy} & \multicolumn{1}{c}{\small Mutual Information} \\
             \multirow{2}{*}{\rotatebox{90}{MC Dropout\hspace{-0.75cm}}}        & \includegraphics[width=0.24\textwidth]{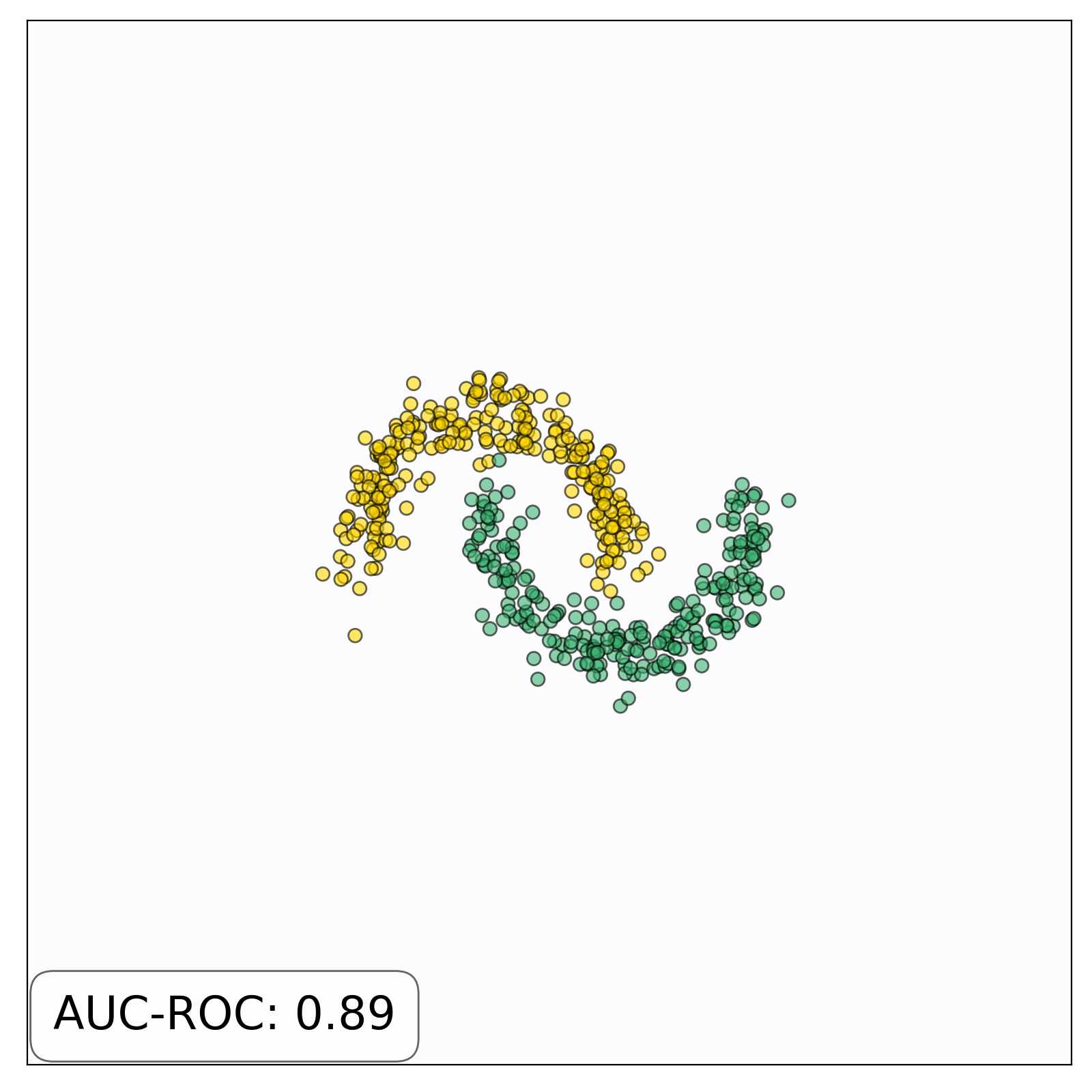} & \includegraphics[width=0.24\textwidth]{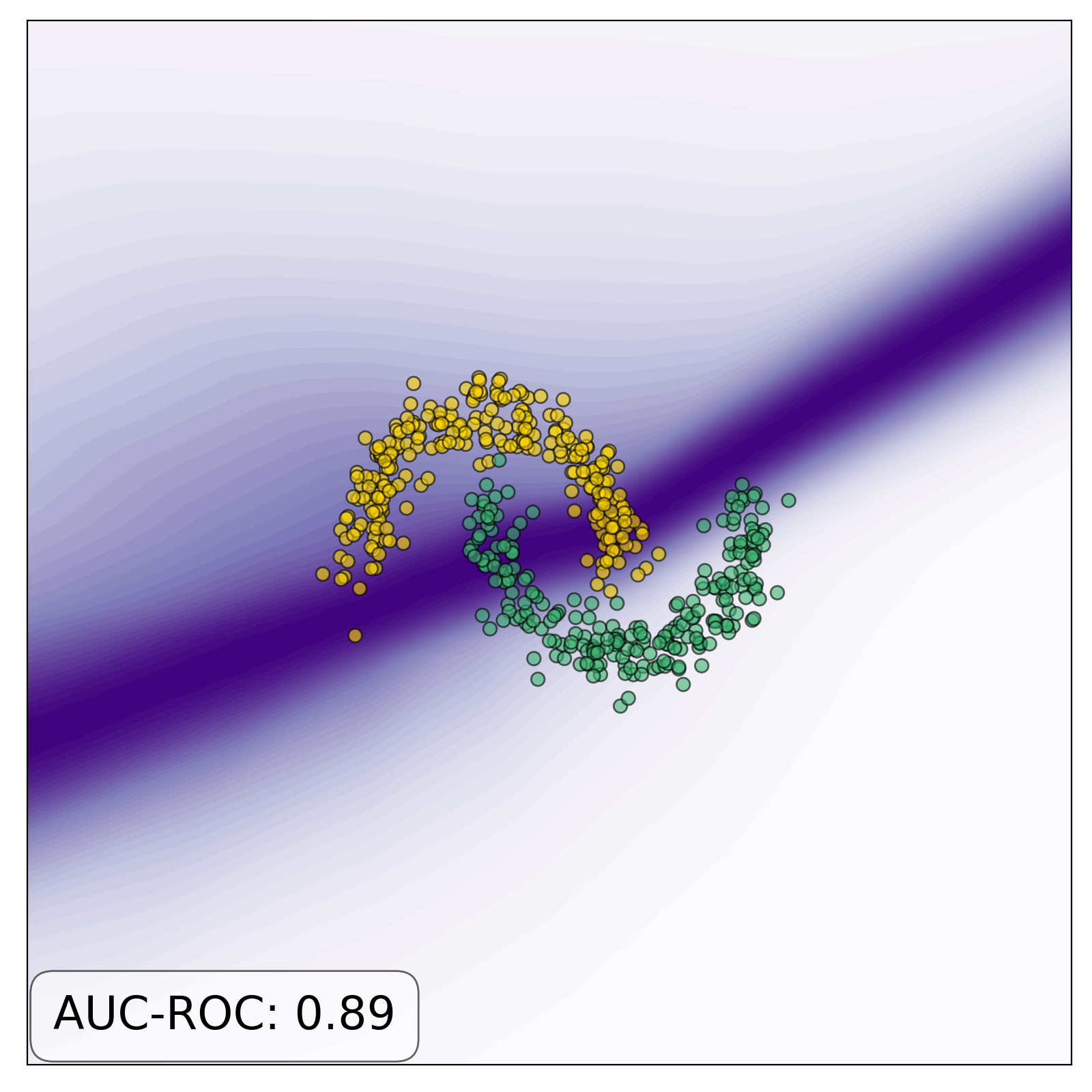}  & \includegraphics[width=0.24\textwidth]{img/know_your_limits/mcdropout_mutual_information.png} \\
             & \includegraphics[width=0.25\textwidth]{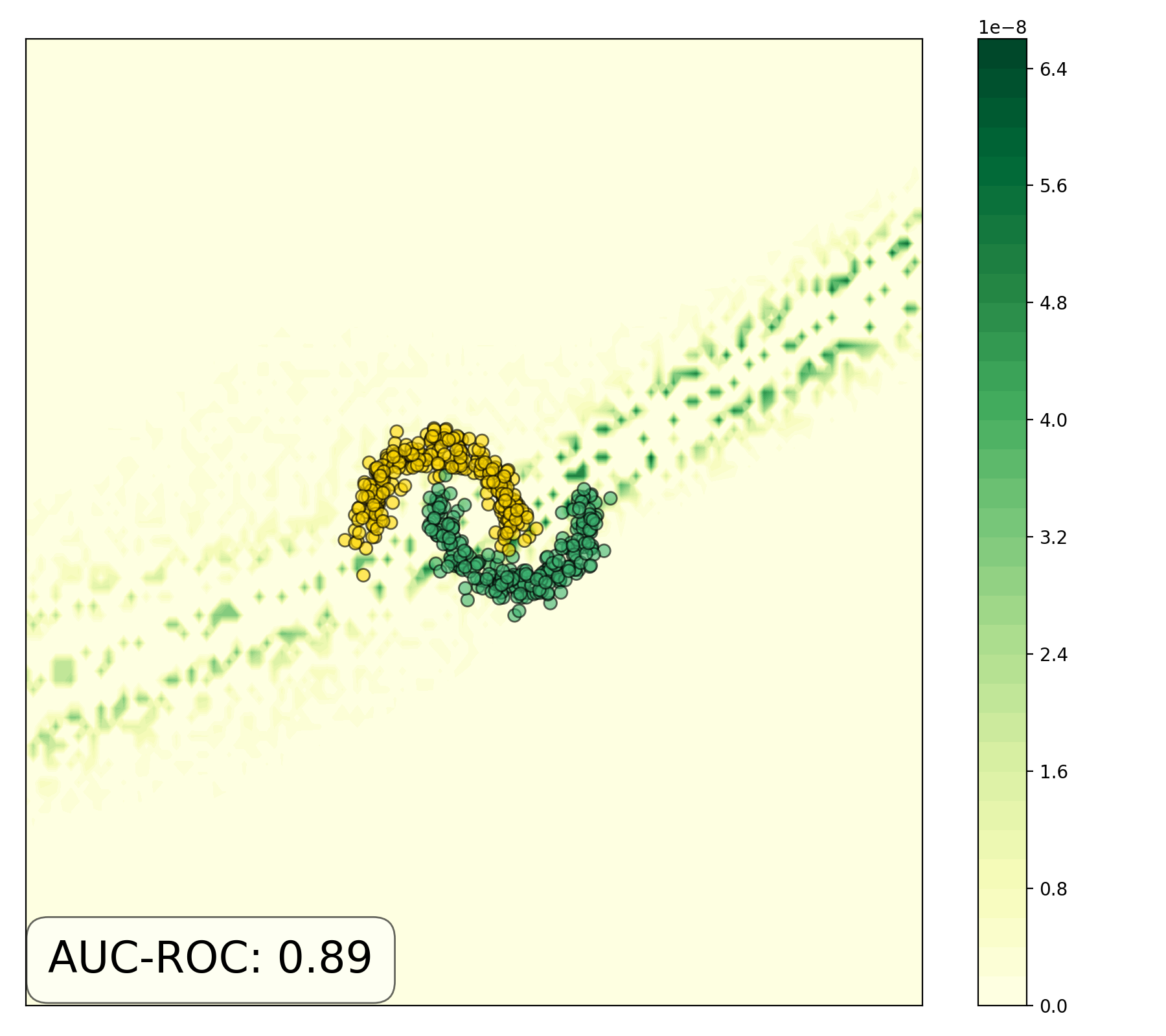} & \includegraphics[width=0.25\textwidth]{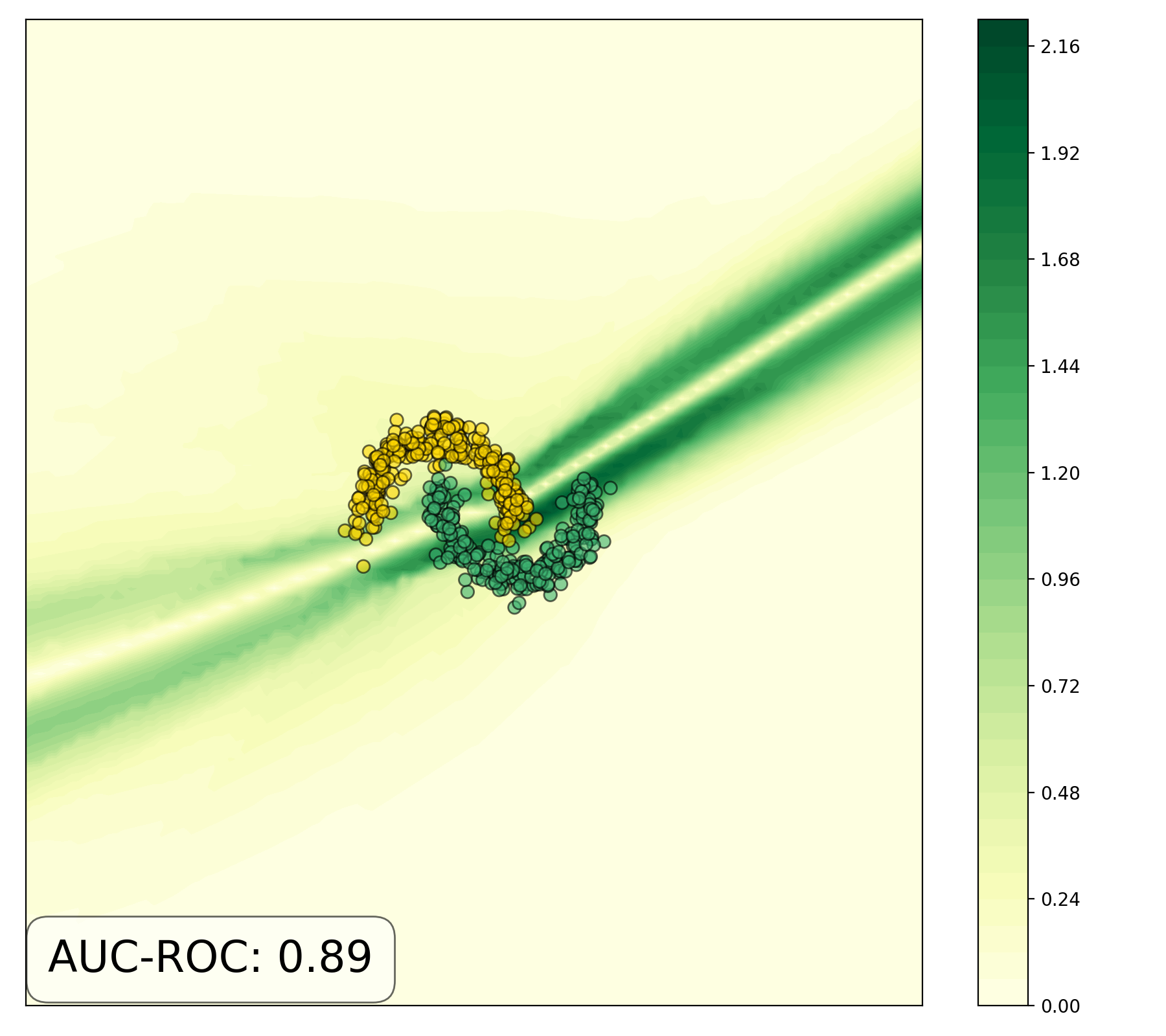}  & \includegraphics[width=0.25\textwidth]{img/know_your_limits/mcdropout_mutual_information_grads.png} \\
             \midrule
             %\rotatebox{90}{\hspace{2.5cm}\texttt{BBB}}               & \includegraphics[width=0.25\textwidth]{img/bbb_var.png} & \includegraphics[width=0.25\textwidth]{img/bbb_entropy.png}  & \includegraphics[width=0.25\textwidth]{img/bbb_mutual_information.png}  \\
             \multirow{2}{*}{\rotatebox{90}{Neural Ensemble\hspace{-1.25cm}}}         & \includegraphics[width=0.24\textwidth]{img/know_your_limits/nnensemble_var.png} & \includegraphics[width=0.24\textwidth]{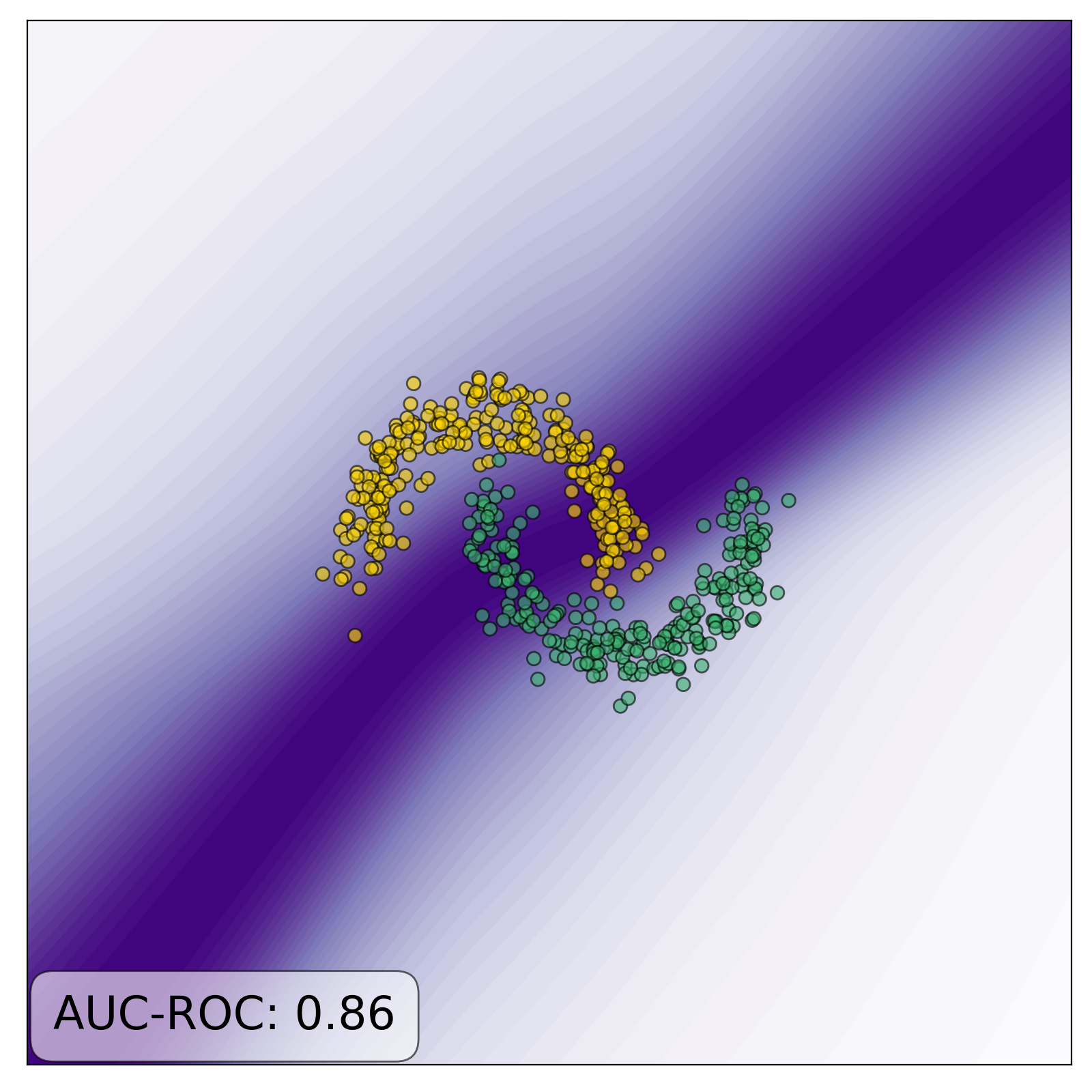}  & \includegraphics[width=0.24\textwidth]{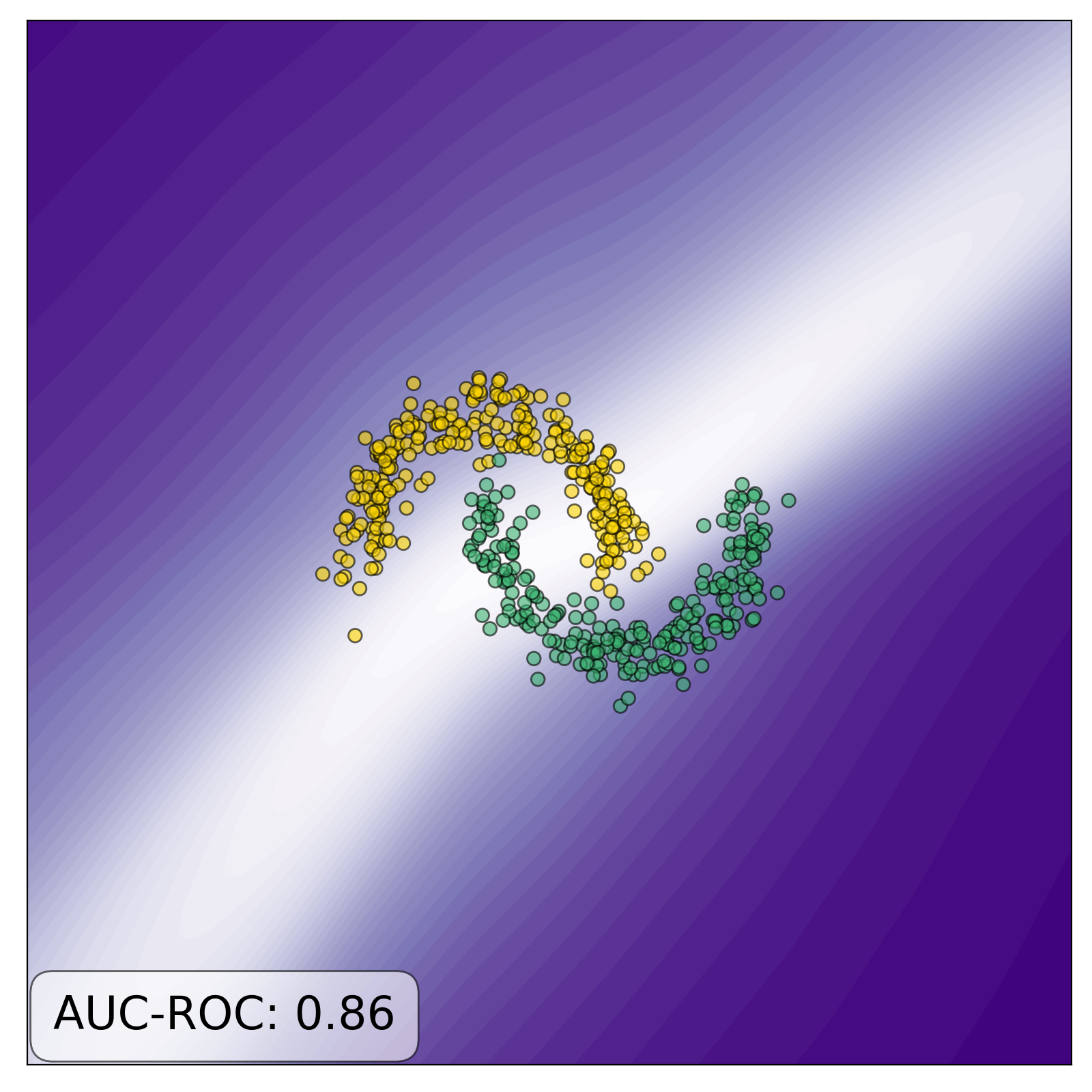} \\
              & \includegraphics[width=0.25\textwidth]{img/know_your_limits/nnensemble_var_grads.png} & \includegraphics[width=0.25\textwidth]{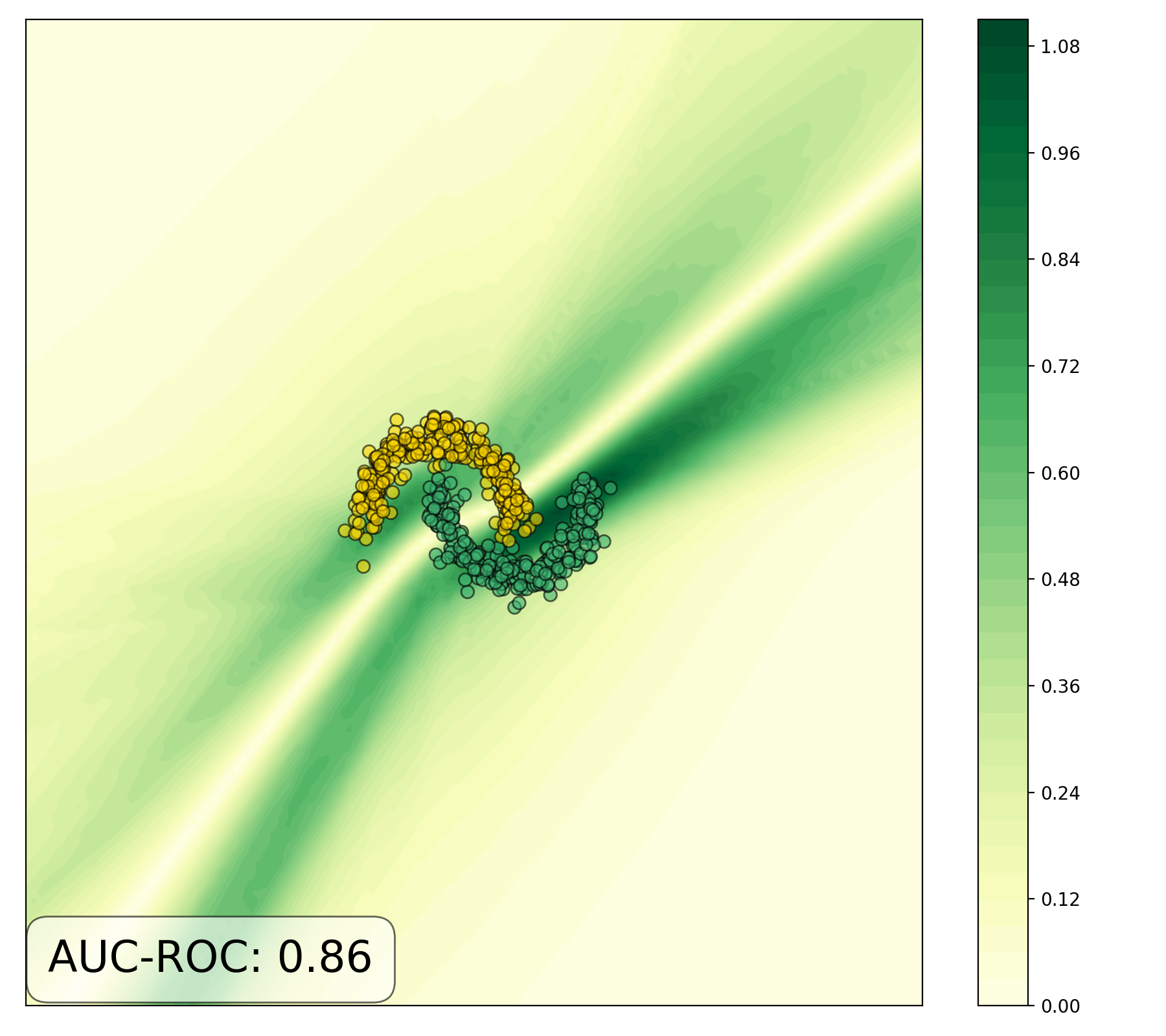}  & \includegraphics[width=0.25\textwidth]{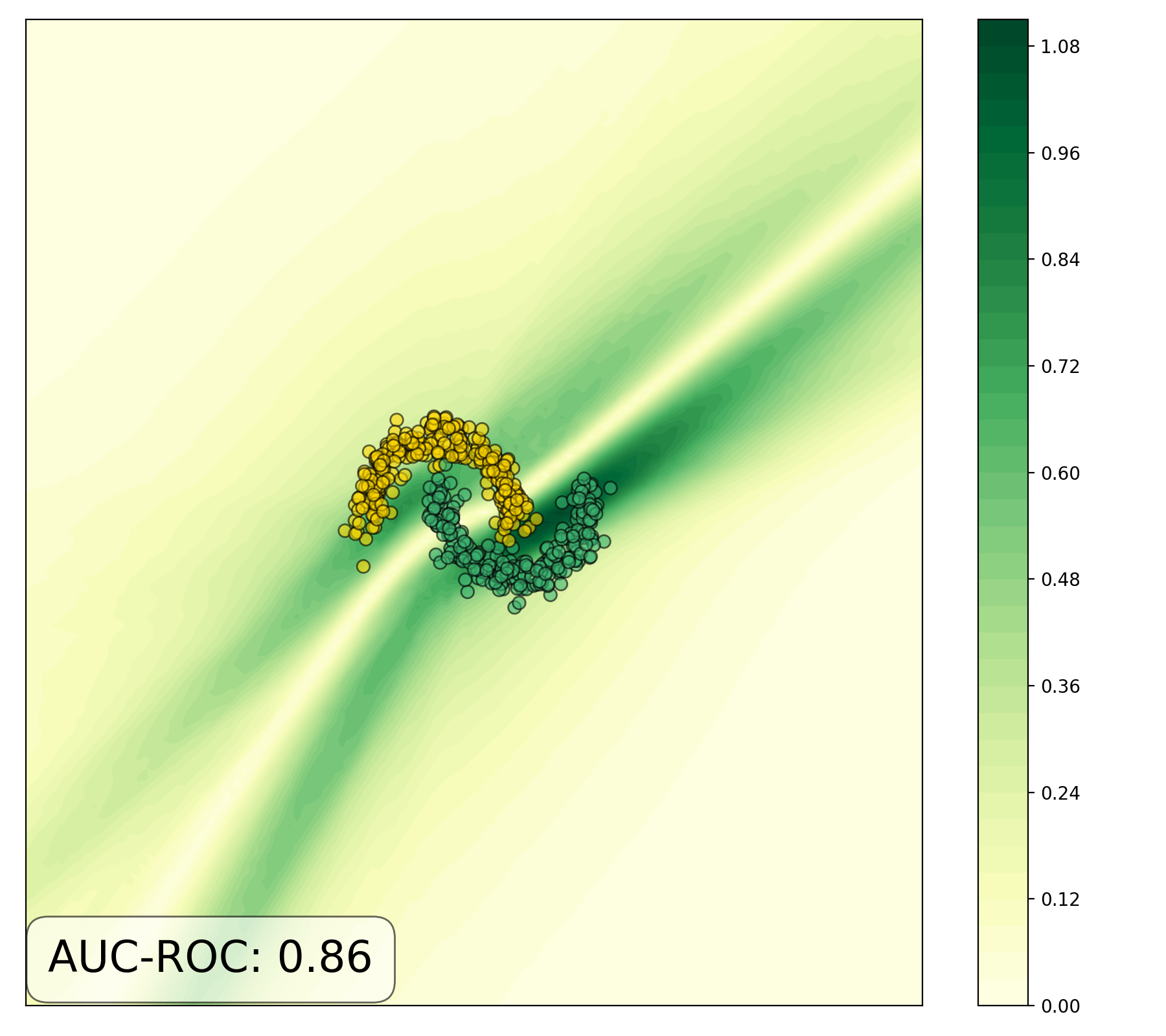} \\
              \midrule
            \rotatebox{90}{\hspace{-1.7cm}Anchored Ensemble}  & \includegraphics[width=0.24\textwidth]{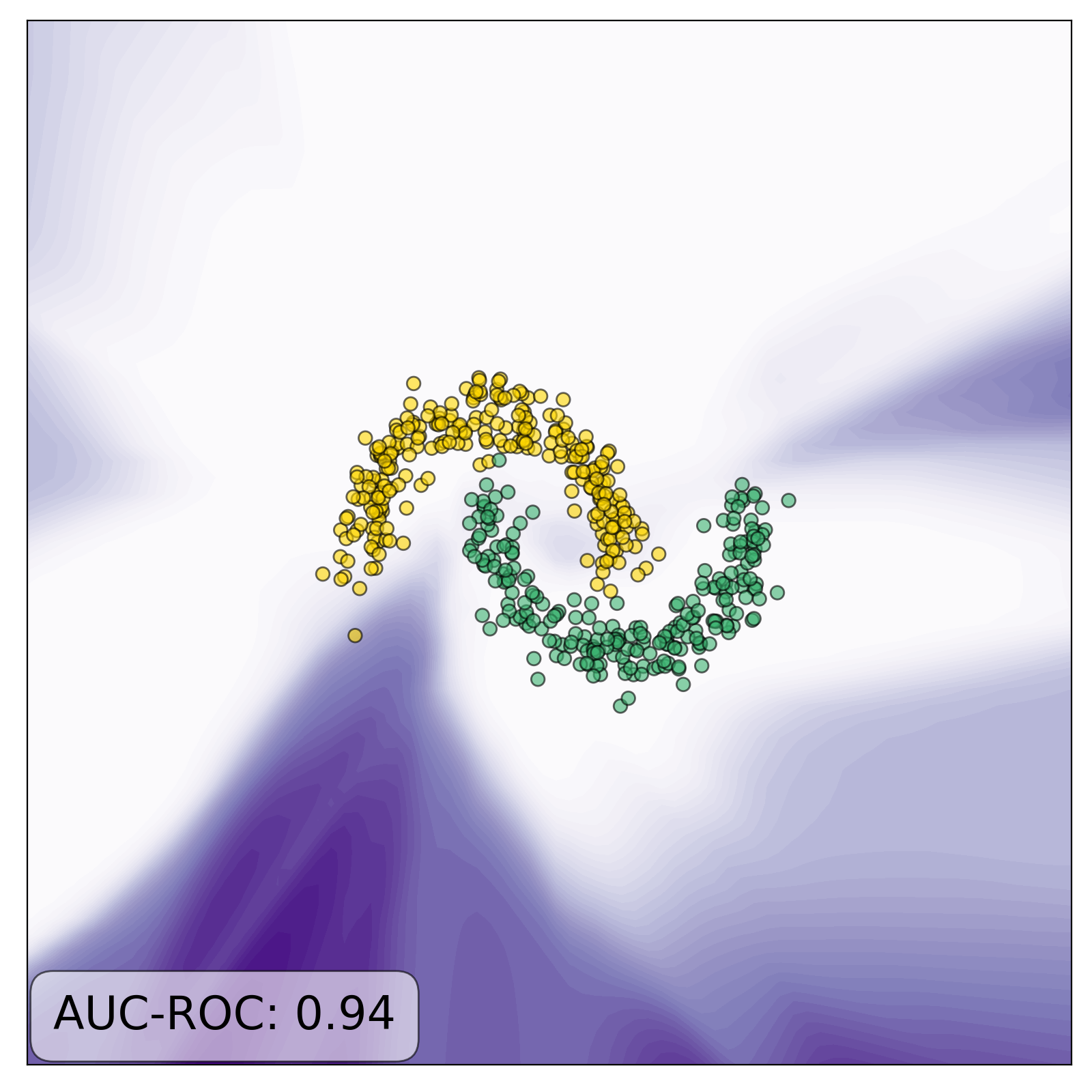} & \includegraphics[width=0.24\textwidth]{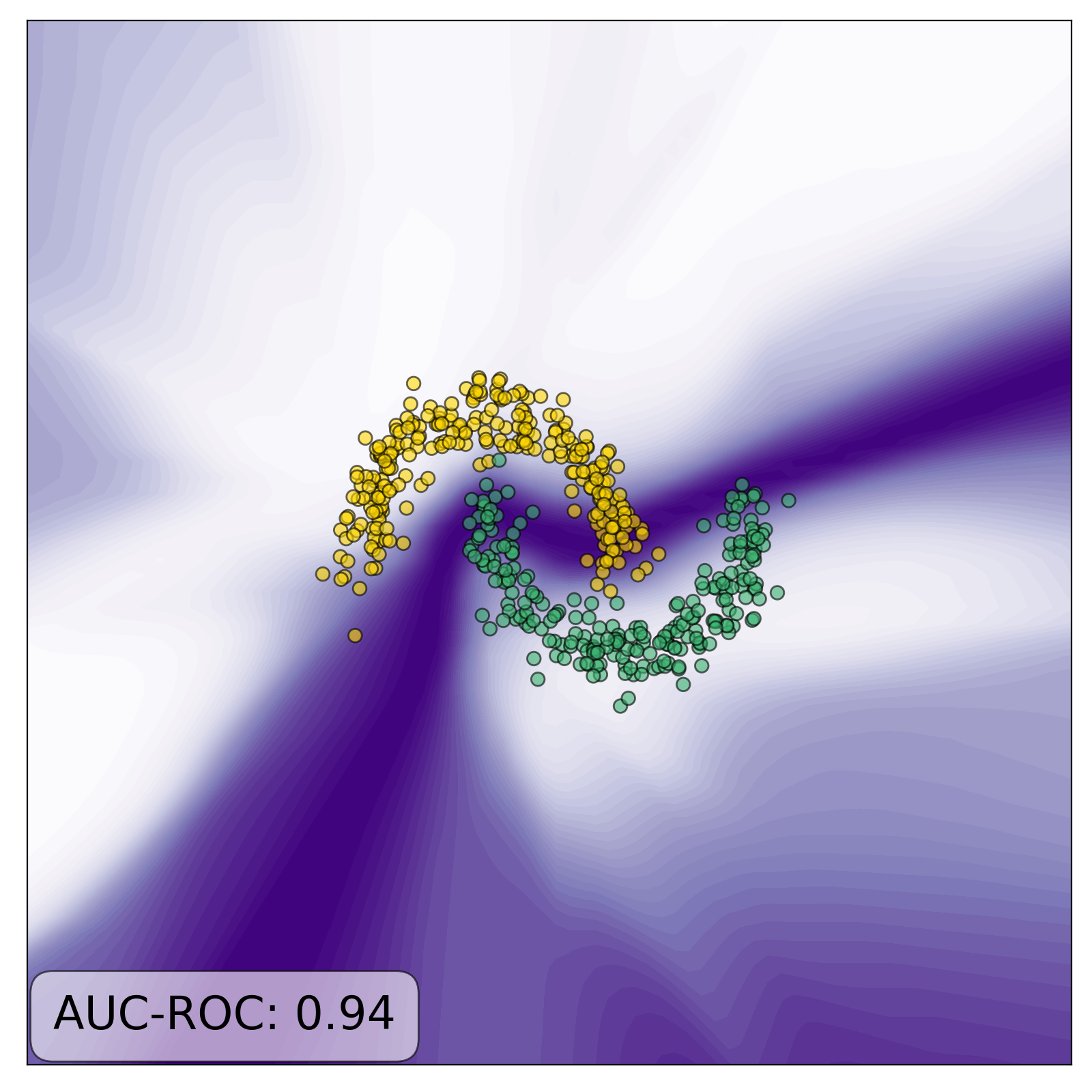}  & \includegraphics[width=0.24\textwidth]{img/know_your_limits/anchorednnensemble_mutual_information.png} \\
            & \includegraphics[width=0.25\textwidth]{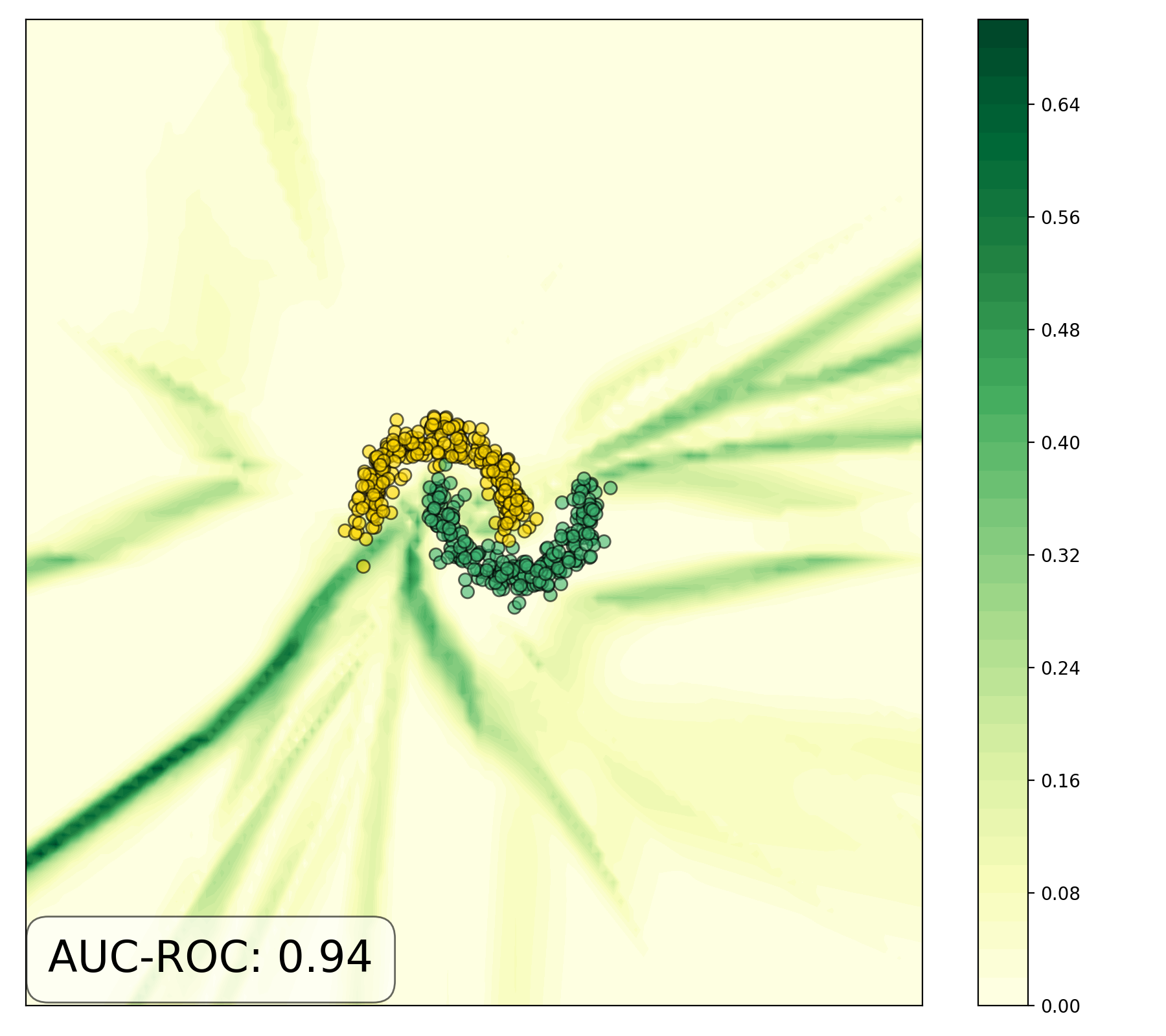} & \includegraphics[width=0.25\textwidth]{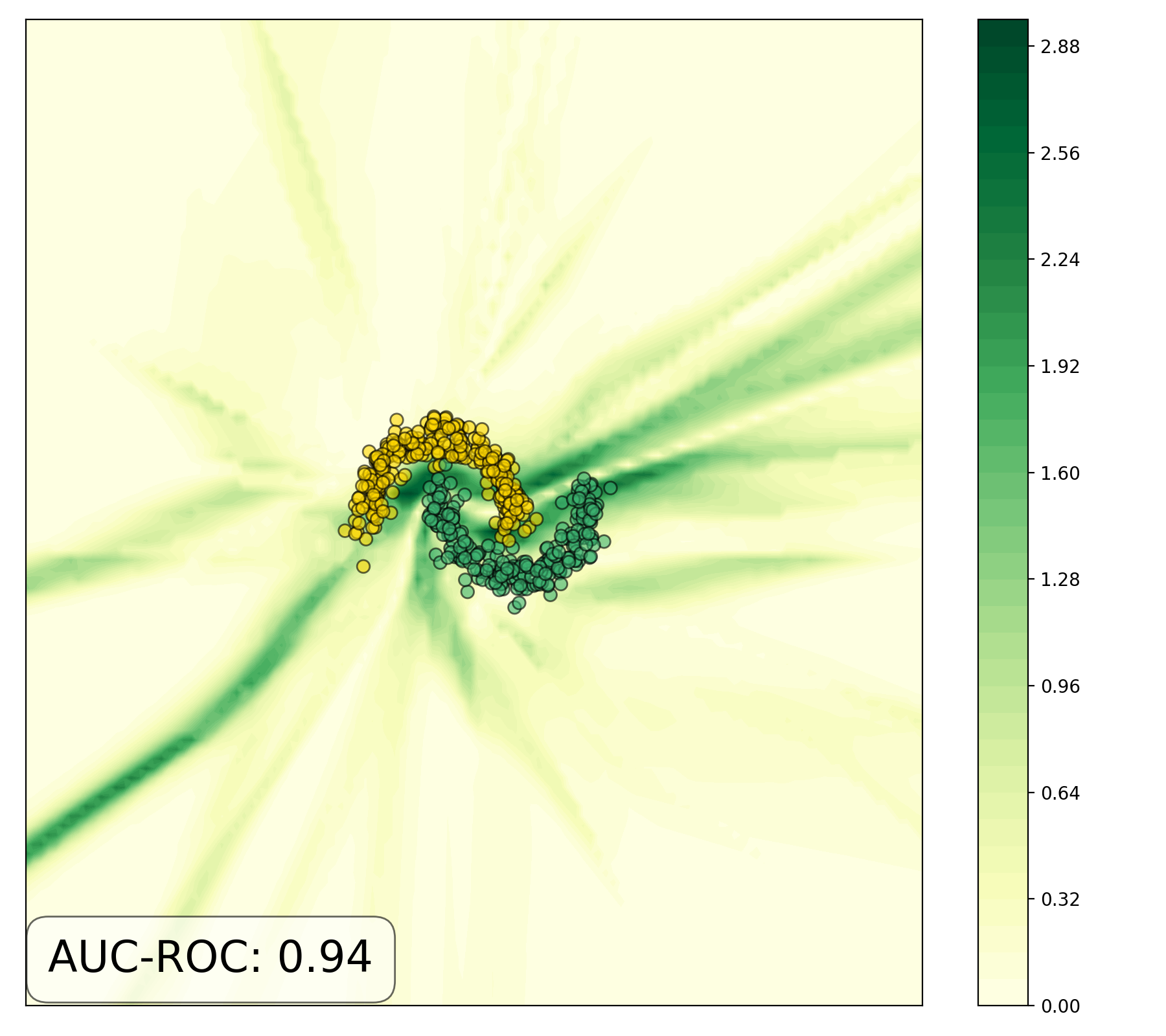}  & \includegraphics[width=0.25\textwidth]{img/know_your_limits/anchorednnensemble_mutual_information_grads.png} \\
        \end{tabular}
    }
    \caption[Uncertainty measured by different metrics for multi-instance models and the gradient of the uncertainty score w.r.t.\@ to the input.]{
        Uncertainty measured by different metrics for multi-instance models (purple plots) and the gradient of the uncertainty score w.r.t.\@ to the input (yellow / green plot).
    }
    \label{fig:app-multiple-pred-nn}
\end{figure*}

\section{Sub-Sampling of Training Sets}\label{app:exploring-predictive-uncertainty-training-set}

Since we sub-sample some of the data splits in \cref{table:datasets}, this bears the dangers of producing unnatural samples of text. 
For that reason, we use this appendix to describe the sampling strategies used for the methodology in \cref{sec:exploring-data-creation} in more detail.

\paragraph{Sub-Sampling Procedure.} 
% For Language Modelling tasks, we aim to maintain a sequence length distribution akin to the original corpus. For that purpose, sequences are aggregated of buckets of the same length. The sampling procedure then first samples a sequence length relative to its frequency in the corpus, and then a sequence from the corresponding bucket with uniform probability. Since it is important to maintain the context for Language Modelling, we also sample the $n$ next sequences occurring after the sequence in the corpus, where $n$ is drawn from a categorical distribution of paragraph lengths build from the original training set. All sampling is performed without replacement.\\
The procedure for sub-sampling text is that sequences are first placed into buckets of the same label, then into sub-buckets of the same length. 
Then, the sampling procedure consists of first drawing a label based on the observed label frequencies, after which the draw of sequence length, proportional to the frequency of this length inside the bucket, determines the final bucket from which a sequence is again drawn uniformly.
Lastly, the process for token classification involves the grouping into sequences by length at the highest level.
Inside a bucket, a sequence is not drawn uniformly, but with a probability according to the \emph{alignment} of the sequence's labels with the overall corpus label distribution. 
This alignment is calculated for each sequence by evaluating the expected log-probability of the sequence's label distribution w.r.t.\@ to the label distribution of the corpus (i.e.,\@ the cross-entropy\index{Entropy!Cross-}). 
The scores for all same-length sequences in a bucket are then normalized into a $[0, 1]$ interval in order to enable sampling, which is similar to the two-stage procedure used in the sequence classification\index{Classification!Sequence} case.

\begin{figure*}
    \centering
    \includegraphics[width=0.95\textwidth]{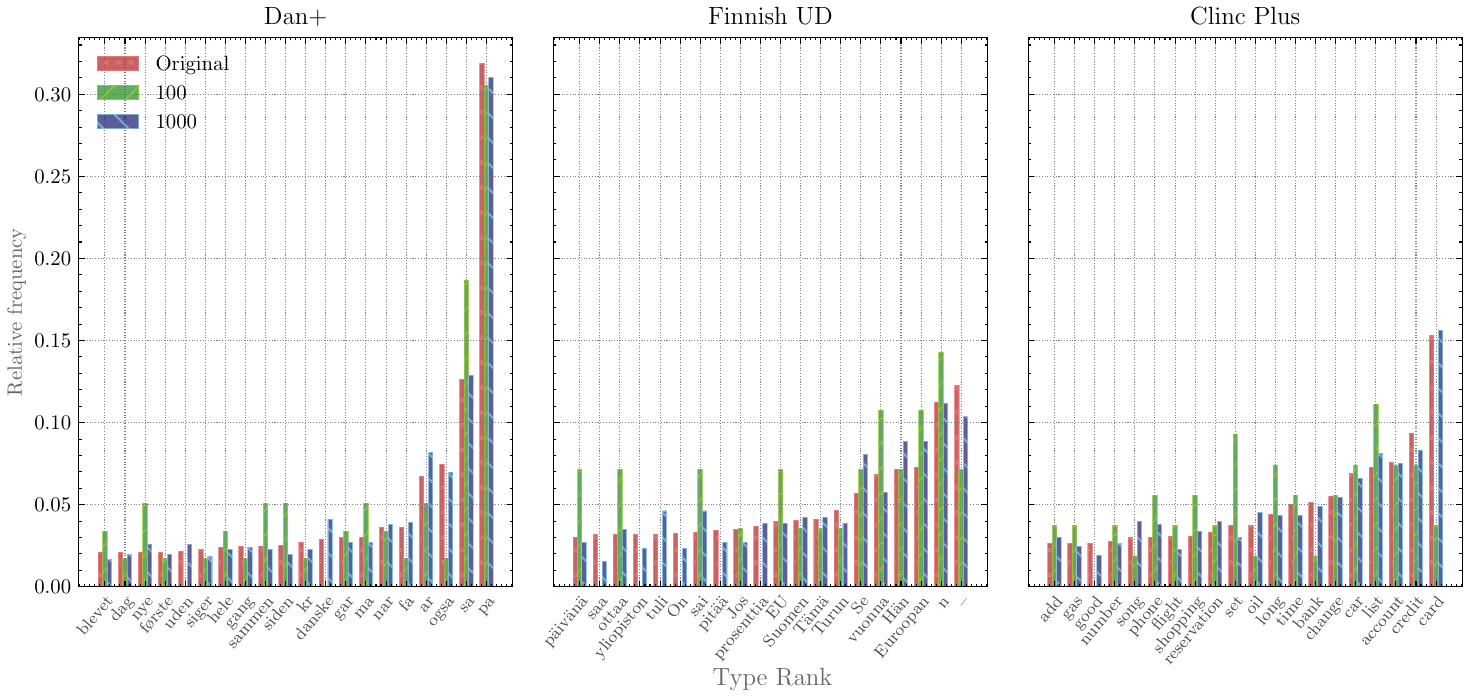}
    \caption[Comparing the relative type frequency in the original and sub-sampled training sets.]{
    Comparing the relative frequency of types in the original and sub-sampled training sets.
    Shown are the top $20$ types in the original training set, compared to sub-sampled training sets of $100$ and $1000$ sequences for Dan+, Finnish UD and Clinc Plus.
    It is shown that while the type frequencies differ noticeably for the small dataset, already $1000$ sequences suffice to approximate the original frequencies. 
    Numbers, stopwords and the most common punctuation were removed.}
    \label{fig:top50}
\end{figure*}

\begin{figure*}
    \centering
    \includegraphics[width=0.95\textwidth]{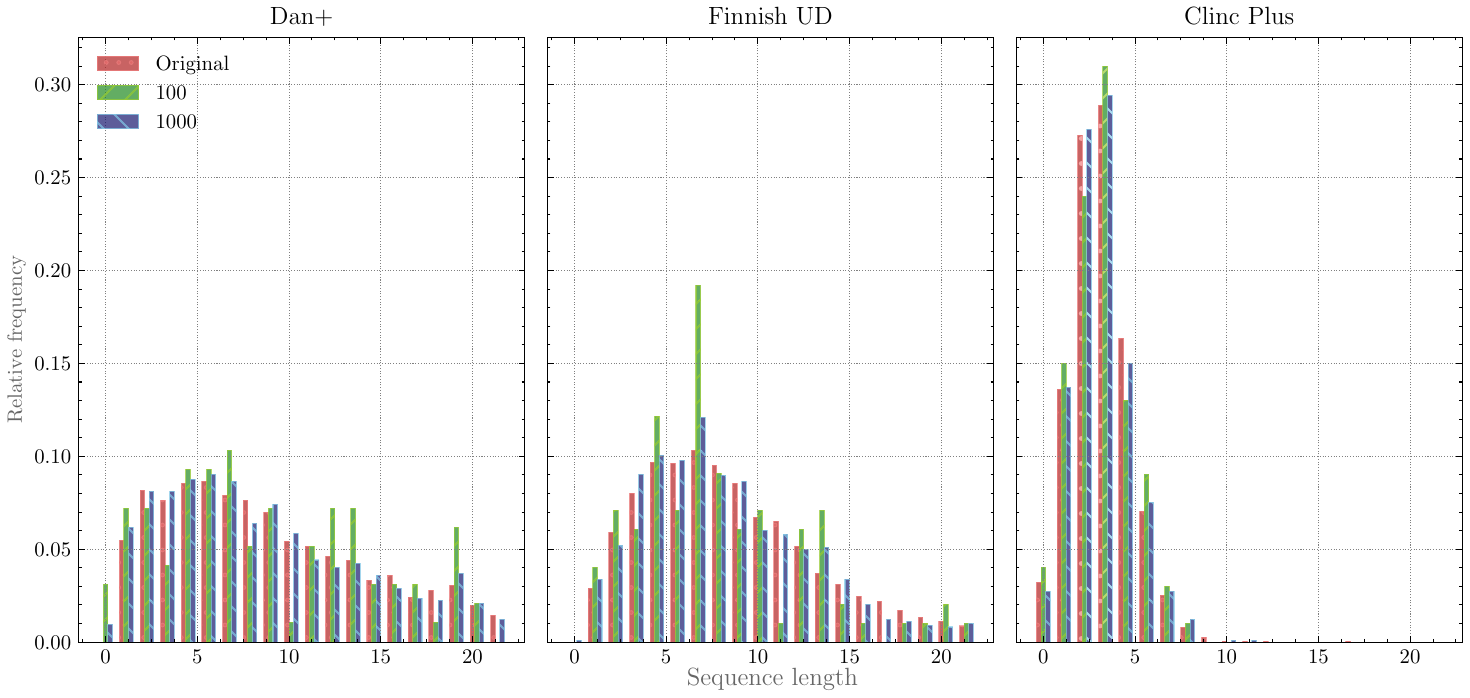}
    \caption[Comparing the relative sequence length frequency in the original and sub-sampled training sets.]{
    Comparing the relative frequency of sequence lengths in the original and sub-sampled training sets. 
    Shown are sequence lengths between $0$ and $25$ in the original test, compared to OOD test sets for Dan+, Finnish UD, Clinc Plus. 
    Not the whole distribution is shown in all cases, with many of the OOD sentences for Dan+ being very long. 
    For Dan+ and Finnish UD, the sentence length distributions are noticeably different. For Clinc Plus, they are very similar.}
    \label{fig:sentence-lengths}
\end{figure*}

\begin{figure*}
    \centering
    \includegraphics[width=0.985\textwidth]{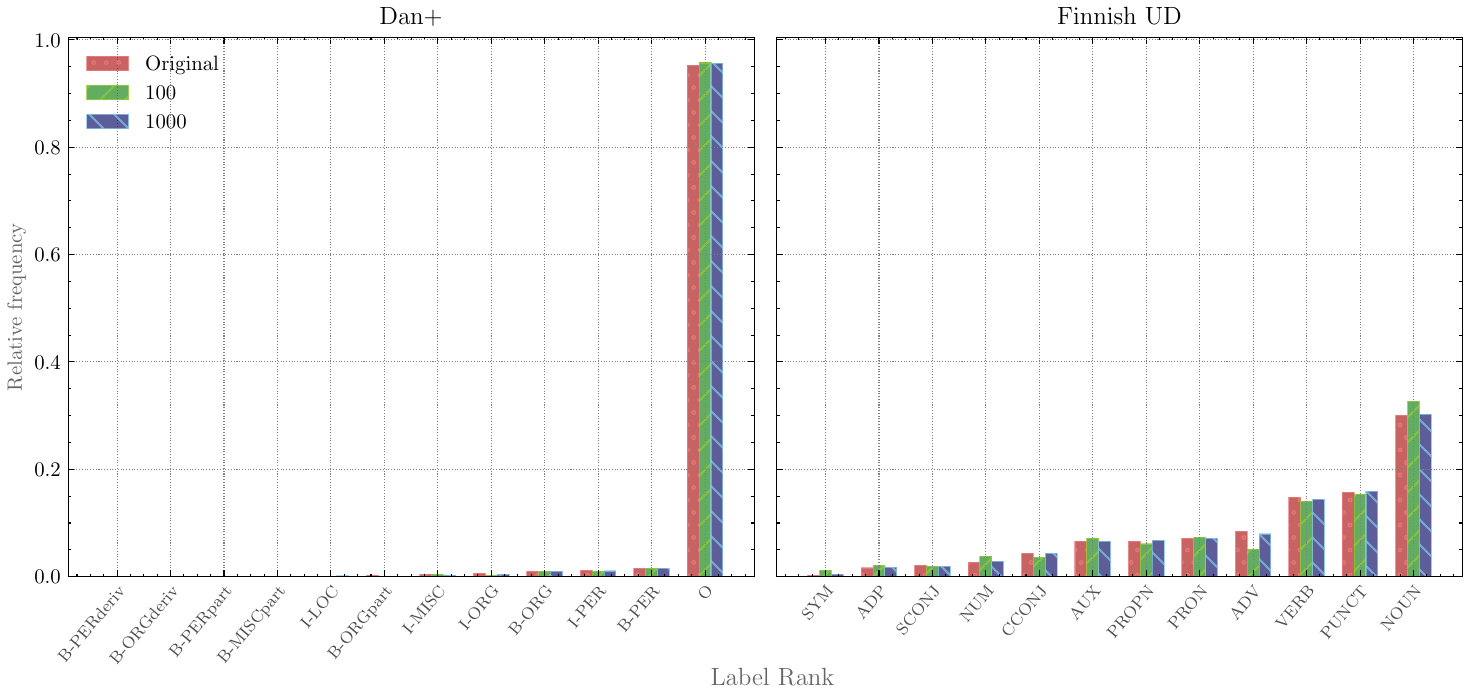} \\
    \caption[Comparing the relative label frequency in the original training set, compared to sub-sampled training sets.]{
    Comparing the relative frequency of labels in the original training set, compared to sub-sampled training sets.
    Shown are frequencies for $100$ and $1000$ sequences. 
    For Danish, the most frequent label by far is the neutral label indicating that no named entity is present.}
    \label{fig:class-labels}
\end{figure*}
\begin{figure*}
    \centering
    \includegraphics[width=0.985\textwidth]{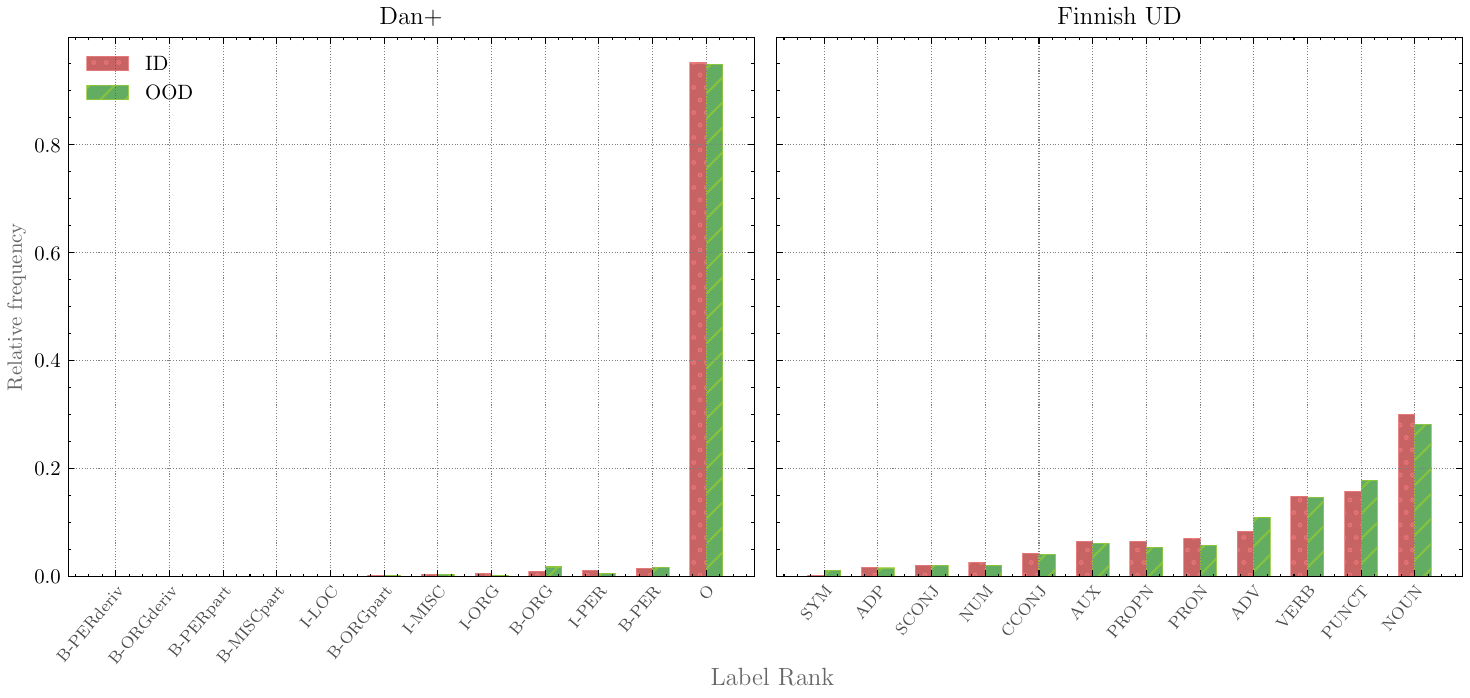}
    \caption[Comparison of the relative class frequencies between original training set compared to the OOD test set.]{
    Comparison of the relative class frequencies between original training set compared to the OOD test set.
    The proportions stay largely the same for Danish, while different more for Finnish.}
    \label{fig:ood-class-freqs}
\end{figure*}

\paragraph{Validation of Sub-Sampled Training Sets.} 
We take multiple steps to validate the representativeness of our sub-sampled data splits. 
First, we plot the distributions of the $50$ most frequent types in the original corpus in \cref{fig:top50}, where we see that distributions converge with increasing sample size. 
Secondly, we plot sentence length distributions in \cref{fig:sentence-lengths}, where we also see increasing alignment with sample size. 
We plot the class distributions in \cref{fig:class-labels}. 
Lastly, we train an interpolated trigram Kneser-Ney language model \citep{jelinek1980interpolated, ney1994structuring} with uniform interpolation weights trained on the original training set using the SRILM tool \citep{stolcke2002srilm} and sub-word tokens produced by the corresponding Bert\index{Bert} tokenizer, sub-sample multiple splits and compare their perplexity scores to those of the original corpus in \cref{table:data-validation}. 
While $n$-gram\index{$n$-gram} perplexities of sub-sampled training sets do lie over the ones of the original data, they are still upper-bounded by the in-distribution test set perplexities. 
Furthermore, this verification was not aimed to give the most precise results, as also the scoring using an $n$-gram model can be rather crude. 
Thus, with all these results, we conclude that our sub-sampling procedure produces sufficiently representative samples of the original data for the different tasks discussed.

\begin{table*}[htb]
    \centering 
    \resizebox{0.85\textwidth}{!}{
        \renewcommand{\arraystretch}{1.5}
        \begin{tabular}{@{}lrrrrrrr@{}}
            \toprule
            &  & \multicolumn{3}{c}{Sub-sampled Train ppl.$\downarrow$} & & \\
            \cmidrule(lr){3-5}
            Language & Train ppl.$\downarrow$  & $n=100$ & $n=500$ & $n=1000$ &  Test ppl.$\downarrow$ & OOD Test ppl.$\downarrow$ \\
            \midrule
            English & $31.54$ & $43.97\pm2.46$ & $44.50\pm0.68$ & $44.9\pm0.4$ & $53.11$ & $ 120.32$ \\
            Danish & $112.73$ & $252.52\pm 13.25$ & $247.09\pm 3.3$ & $249.27\pm 3.15$ & $418.71$ & $524.32$ \\
            Finnish & $116.49$ & $257.67\pm 10.96$ & $257.66 \pm 4.7$ & $260.36 \pm 5.36$ & $1374.76$ & $1284.82$ \\
            %Swahili & $128.60$ & $290.97\pm8.76$ & $280.39\pm26.52$ & $270.69\pm17.52$ & 343.13 & 1397.14 \\
            \bottomrule
        \end{tabular}%
    }
    \caption[Results of using an interpolated Kneser-Ney $n$-gram language model on selected datasets, including sub-sampled training splits and the OOD test set.]{
    Results of using an interpolated Kneser-Ney $n$-gram language model on selected datasets, including sub-sampled training splits and the OOD test set.
    Scores of sub-sampled training sets were obtained over five different attempts.} \label{table:data-validation}
\end{table*}

\section{Selection of OOD Test Sets}\label{app:predictive-uncertainty-ood-test-set}

\begin{figure*}[htb]
    \centering
    \includegraphics[width=0.985\textwidth]{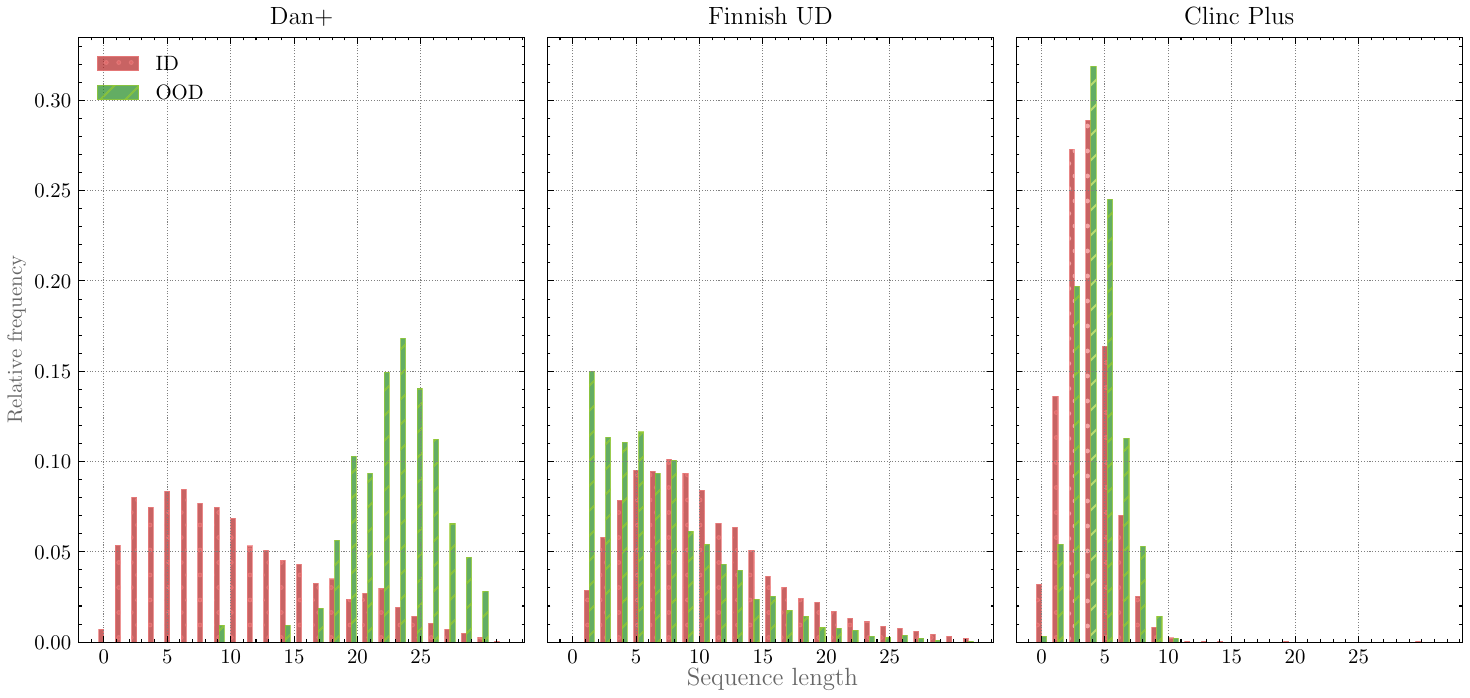}
    \caption[Comparison of sequence length distribution between the original training set and the OOD test set.]{Comparison of sequence length distribution between the original training set and the OOD test set. 
    For English, the distribution of lengths of voice assistant commands is quite similar, while the differences for Dan+ and Finnish UD are more pronounced.}
    \label{fig:ood-sequence-lengths}
\end{figure*}

\begin{figure*}[htb]
    \centering
    \includegraphics[width=0.985\textwidth]{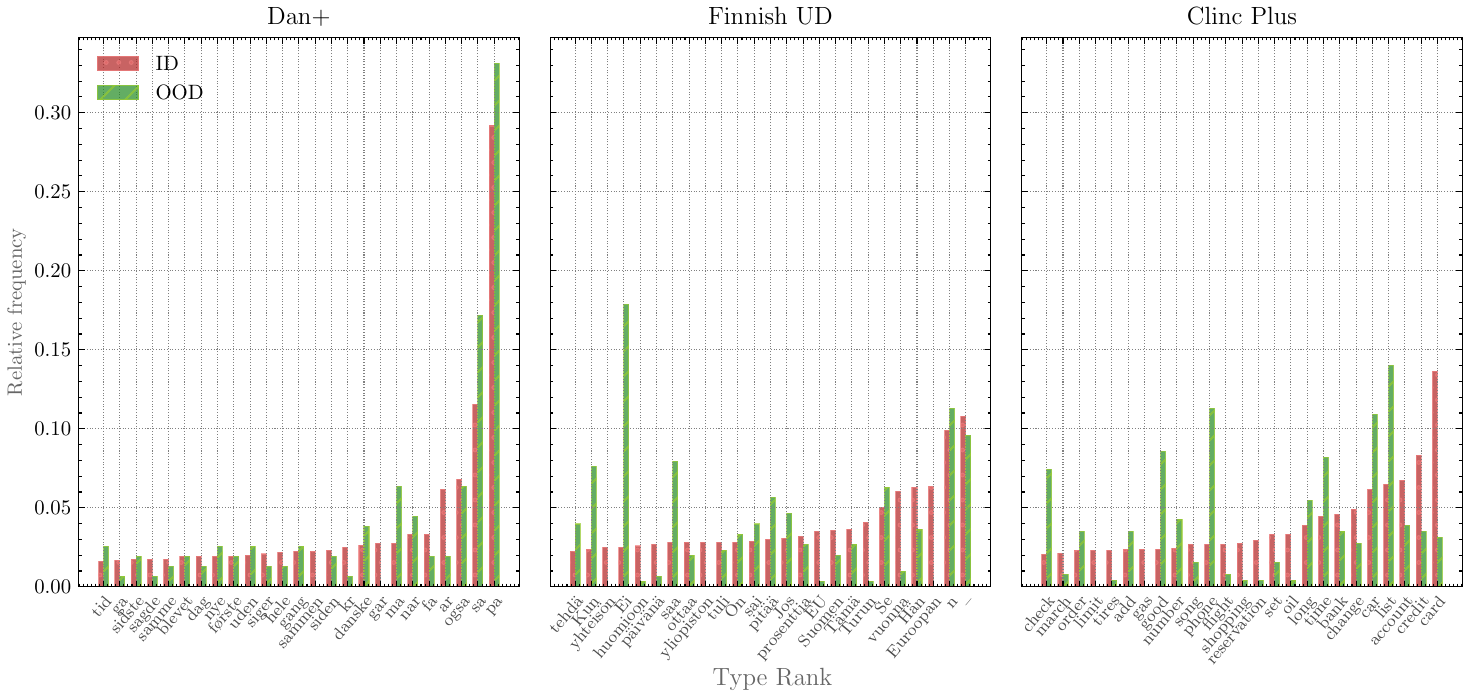}
    \caption[Comparison of the relative frequencies of the top 25 types in the original training set compared to the OOD test set.]{Comparison of the relative frequencies of the top 25 types in the original training set compared to the OOD test set. 
    Even among the most frequent and therefore usually common tokens, the plots show differences between the in-distribution train and out-of-distribution test set. 
    Numbers, stopwords and the most common punctuation were removed.}
    \label{fig:ood-type-freqs}
\end{figure*}

In this appendix section, we present additional evidence that the OOD\index{Out-of-distribution data} test splits shown in \cref{table:datasets} are sufficiently different from the training data---meaning, out-of-distribution---to enable our chosen methodology. 
To that end, we re-use similar ideas as described in \cref{app:exploring-predictive-uncertainty-training-set}, but with the opposite goal. 
In \cref{fig:ood-sequence-lengths}, we plot the distribution of sequence lengths of the training set compared with the OOD test set, with the same done for the most frequent $25$ types in \cref{fig:ood-type-freqs} and class labels in \cref{fig:ood-class-freqs}. 
Lastly, we again use a interpolated Kneser-Ney trigram language model to compute the perplexity of the training compared to the OOD test set in \cref{table:data-validation}. 
In all cases, OOD $n$-gram\index{$n$-gram} perplexities lie much over the training or sub-sampled data perplexities. 
Except for Finnish, they are also widely different from the test set perplexities. 
In that exceptional cases, an explanation could be given by the highly agglutinative nature of Finnish, increasing the sparsity of the language despite the subword tokenization.

\section{Additional Scatter Plots}\label{app:additional-scatters}

This section provides some additional scatter plots for the experiments in \cref{sec:dependence-training-data}. 
For all plots presented here as well as \cref{fig:scatter-plot-danplus-kendalls-tau-token}, some slight jitter sampled from $\mathcal{N}(0, 0.01)$ was added to x and y-coordinates to increase readability of overlapping points.

\paragraph{Clinc Plus.} 
In \cref{subfig:clinc-plus-scatter-auroc,subfig:clinc-plus-scatter-aupr}, we can see that the variational Bert model actually \emph{degrades} in performance as the more training data is added, both on a task and uncertainty dimensions, while other models stay relatively constant. The same trend can be detected using the sequence-level Kendall's $\tau$ for Clinc Plus. We suspect that the smallest training size of $10k$ examples does already provide enough data for models to converge to similar solutions even after adding more data, and that the variational Bert alone might be prone to overfitting in this case. 

\paragraph{Dan+.} 
Results for the Danish dataset are shown in \cref{subfig:danplus-scatter-auroc,subfig:danplus-scatter-aupr}. It is apparent that LSTM-based models stay mostly constant in their predictive performance, with the largest gains observed by the LSTM ensemble. 
We can also observe the DDU and variational Bert to increase both in task performance and uncertainty quality with increasing training data. 
Interestingly, we can see for the SNGP Bert that uncertainty estimates become more indicative of OOD with more training samples, but mostly only using predictive entropy and the maximum probability score. This might indicate that in these cases, the model actually achieves the desired distance-awareness posed by \citet{liu2022simple}. In \cref{subfig:dan+-scatter-kendalls-tau-seq}, we can see a similar behavior of the SNGP-Bert and its metrics w.r.t.\@ to the sequence-level correlation. Also, we see that the other Bert models and LSTM-Ensemble actually loose in uncertainty quality as more data is added.

\paragraph{Finnish UD.}
 In \cref{subfig:finnish-ud-scatter-auroc,subfig:finnish-ud-scatter-aupr}, we observe that the AUROC and AUPR scores of different models and metrics stay largely constant across dataset sizes, which could be explained with the larger amount of training data supplied compared to Dan+. 
  On the token-level correlation between uncertainty and loss in \cref{fig:finnish-ud-scatter-plot-kendalls-tau-token}, we see the DDU Bert profiting most from more data. 
  On a sequence-level, as depicted in \cref{subfig:finnish-ud-scatter-kendalls-tau-seq}, the correlation appears mostly static across training set sizes, with only small gaps between in-distribution and out-of-distribution data.\\

Overall, it seems that the range of dataset sizes for Dan+ show the most critical differences between models, while for the dataset sizes used for Finnish UD and Clinc Plus, enough data seems to be supplied for changes to be more miniscule. 
This result is particularly relevant for low-resource setting\index{Low-resource language}, although the dependency on the task can not be disentangled from these results. 

\begin{figure*}[htb]
    \centering
    \centering
    \begin{subfigure}{\textwidth}
        \centering
        \includegraphics[width=\columnwidth]{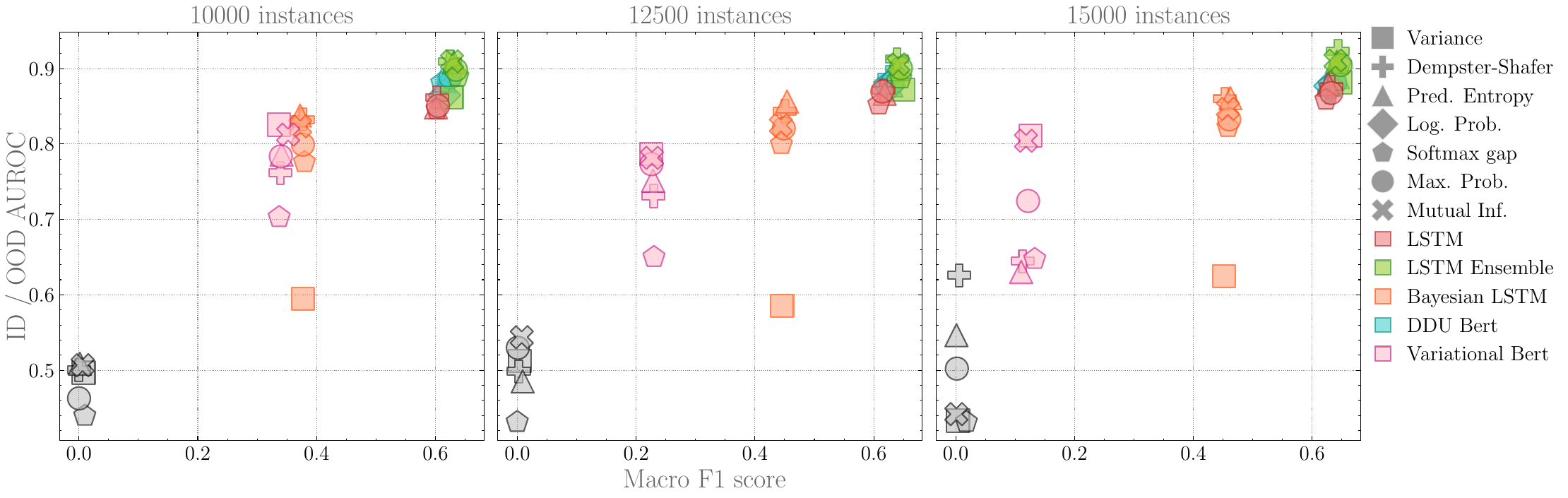}
        \subcaption{Scatter plot for the Clinc Plus dataset.}
        \label{subfig:clinc-plus-scatter-auroc}
    \end{subfigure}
    \begin{subfigure}{\textwidth}
        \centering
        \includegraphics[width=\columnwidth]{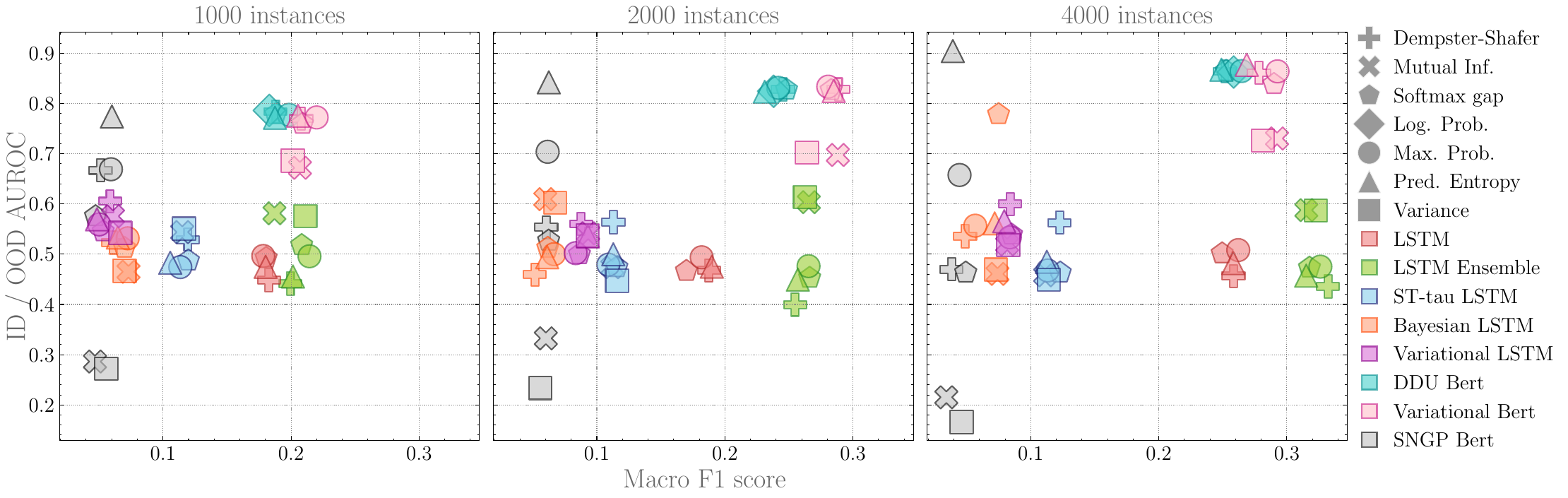}
        \subcaption{Scatter plot for the Dan+ dataset.}
        \label{subfig:danplus-scatter-auroc}
    \end{subfigure}
    \begin{subfigure}{\textwidth}
        \centering
        \includegraphics[width=\columnwidth]{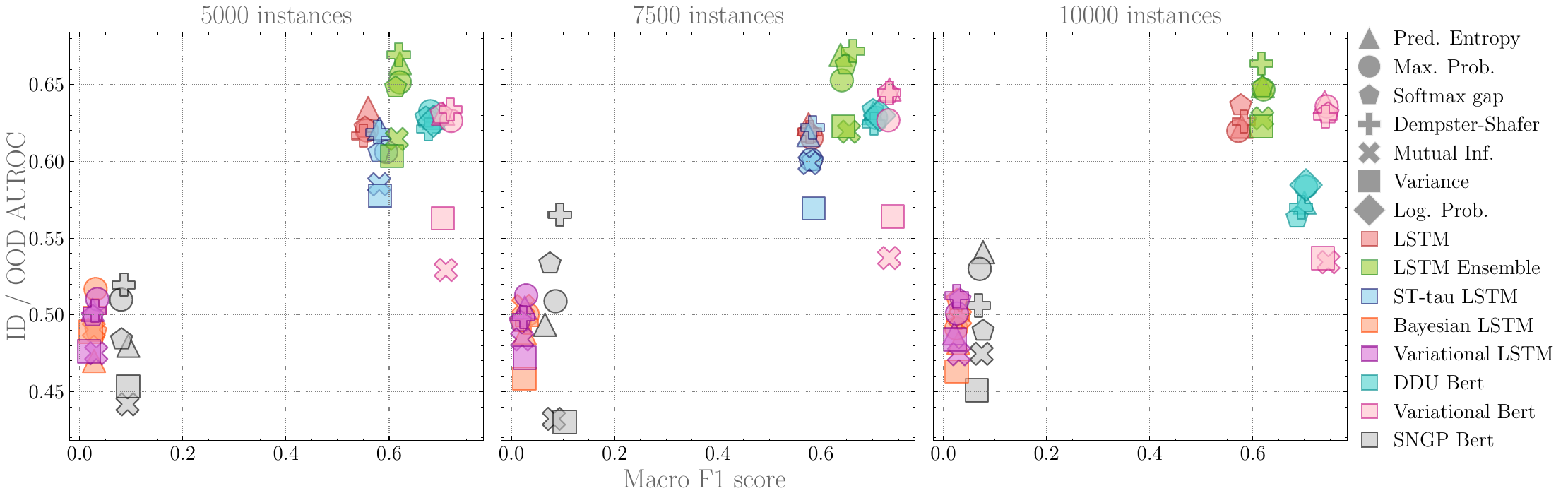}
        \subcaption{Scatter plot for the Finnish UD dataset.}
        \label{subfig:finnish-ud-scatter-auroc}
    \end{subfigure}
    \caption[Scatter plots showing the difference between model performance and the quality of uncertainty estimates.]{
        Scatter plots showing the difference between model performance (measured by macro $F_1$) and the quality of uncertainty estimates using AUROC. Shown are different models and uncertainty metrics and several training set sizes on the used datasets.}\label{fig:scatter-plot-auroc}
\end{figure*}

\begin{figure*}[htb]
    \centering
    \begin{subfigure}{\textwidth}
        \centering
        \includegraphics[width=\columnwidth]{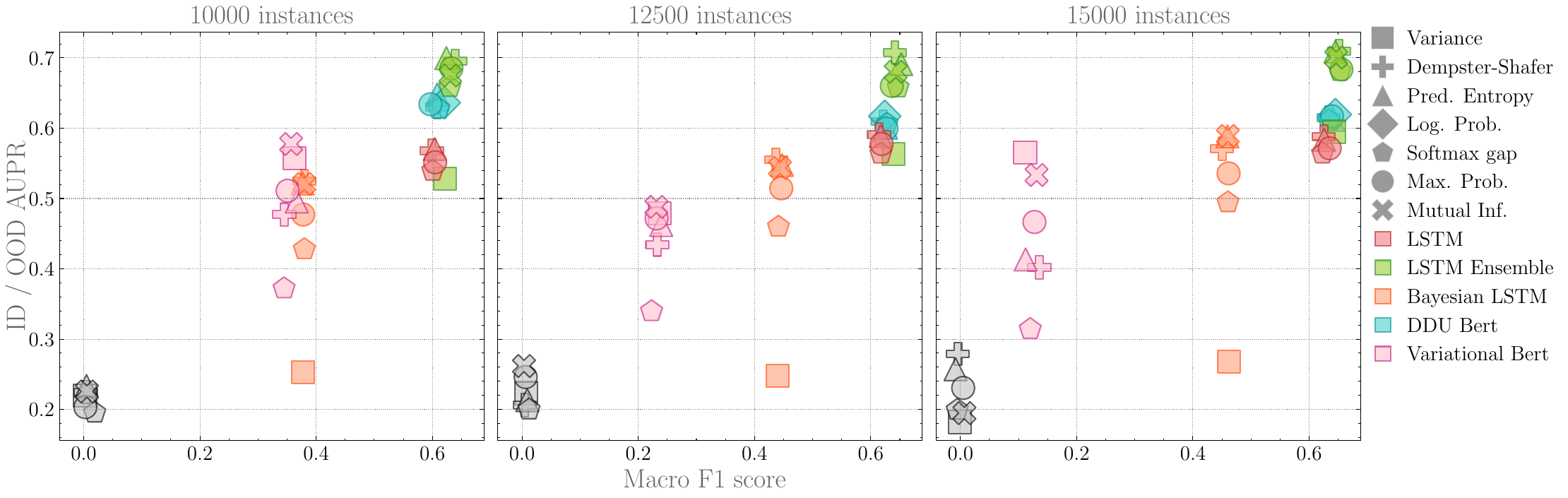}
        \subcaption{Scatter plot for the Clinc Plus dataset.}
        \label{subfig:clinc-plus-scatter-aupr}
    \end{subfigure}
    \begin{subfigure}{\textwidth}
        \centering
        \includegraphics[width=\columnwidth]{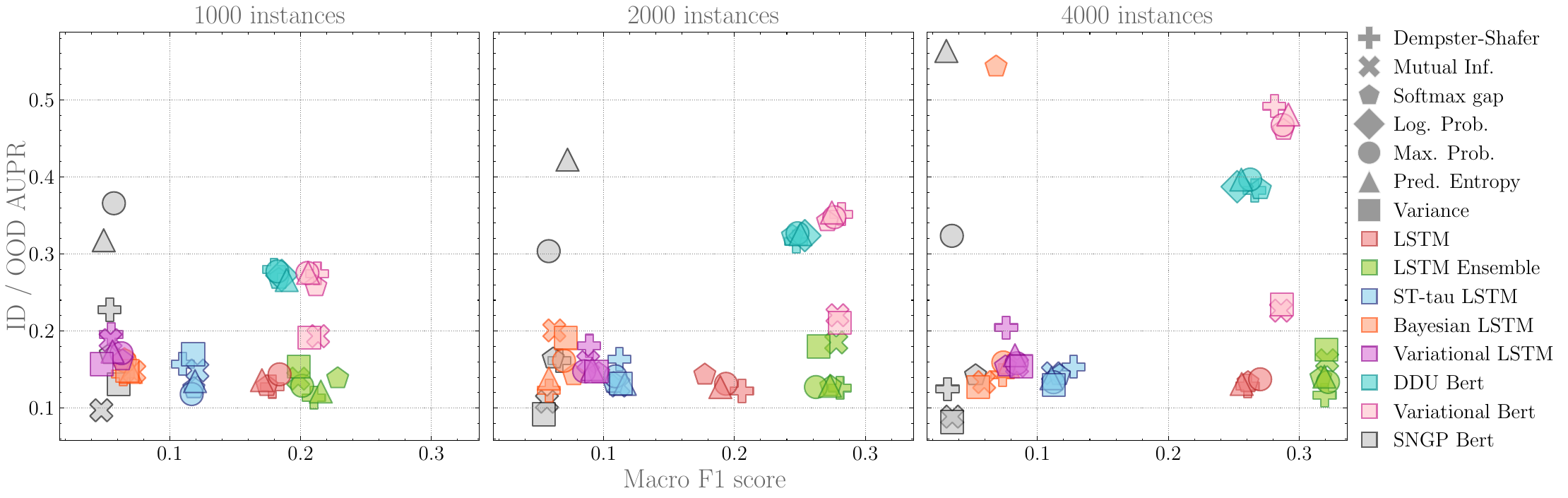}
        \subcaption{Scatter plot for the Dan+ dataset.}
        \label{subfig:danplus-scatter-aupr}
    \end{subfigure}
    \begin{subfigure}{\textwidth}
        \centering
        \includegraphics[width=\columnwidth]{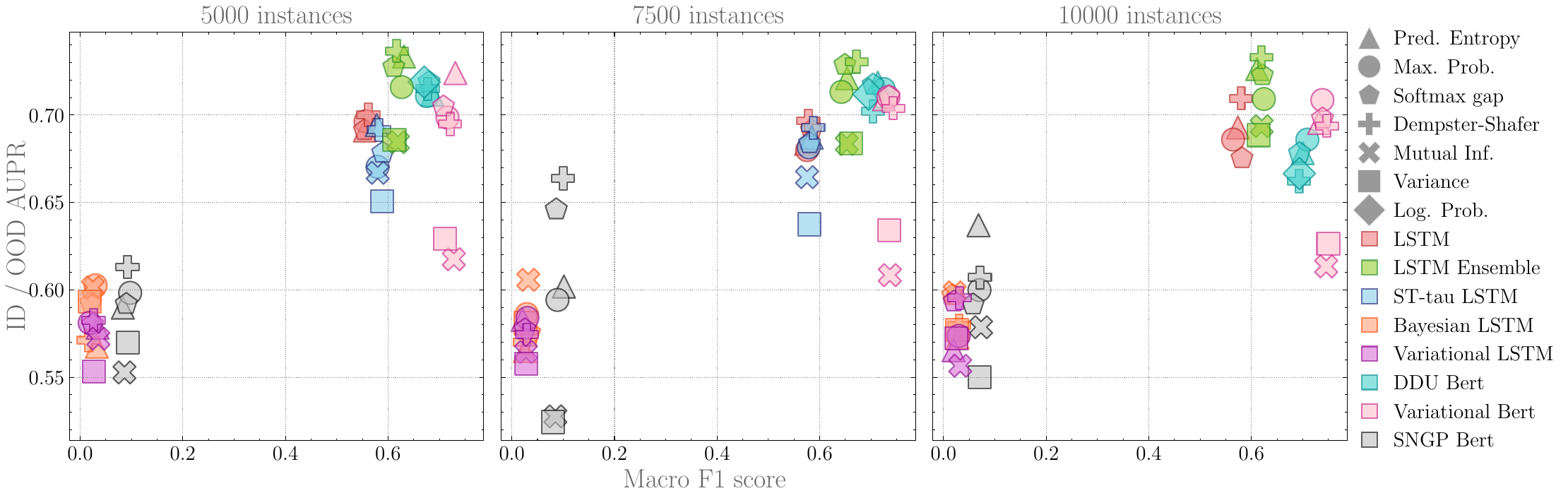}
        \subcaption{Scatter plot for the Finnish UD dataset.}
        \label{subfig:finnish-ud-scatter-aupr}
    \end{subfigure}
    \caption[Scatters plot showing the difference between model performance and the quality of uncertainty estimates.]{
        Scatters plot showing the difference between model performance (measured by macro $F_1$) and the quality of uncertainty estimates using AUPR. 
        Shown are different models and uncertainty metrics and several training set sizes on the used datasets.}\label{fig:scatter-plot-aupr}
\end{figure*}

\begin{figure*}[htb]
    \centering
    \centering
    \includegraphics[width=\columnwidth]{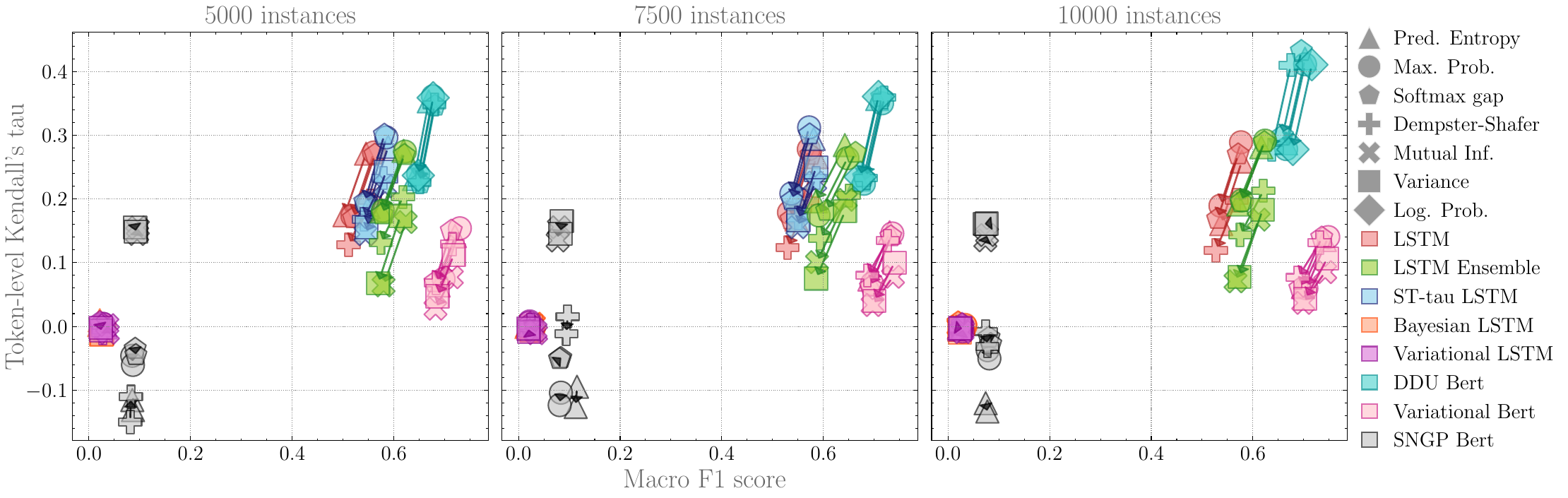}
    \caption[Scatter plot showing the difference between model performance and the quality of uncertainty estimates on a token-level. ]{
    Scatter plot showing the difference between model performance (measured by macro $F_1$) and the quality of uncertainty estimates on a token-level (measured by Kendall's $\tau$). 
    Results are shown for different models and uncertainty metrics and several training set sizes on the Finnish UD dataset. Arrows indicate changes between the in-distribution and out-of-distribution test set.}\label{fig:finnish-ud-scatter-plot-kendalls-tau-token}
\end{figure*}

\begin{figure*}[htb]
    \centering
    \begin{subfigure}{\textwidth}
        \centering
        \includegraphics[width=\columnwidth]{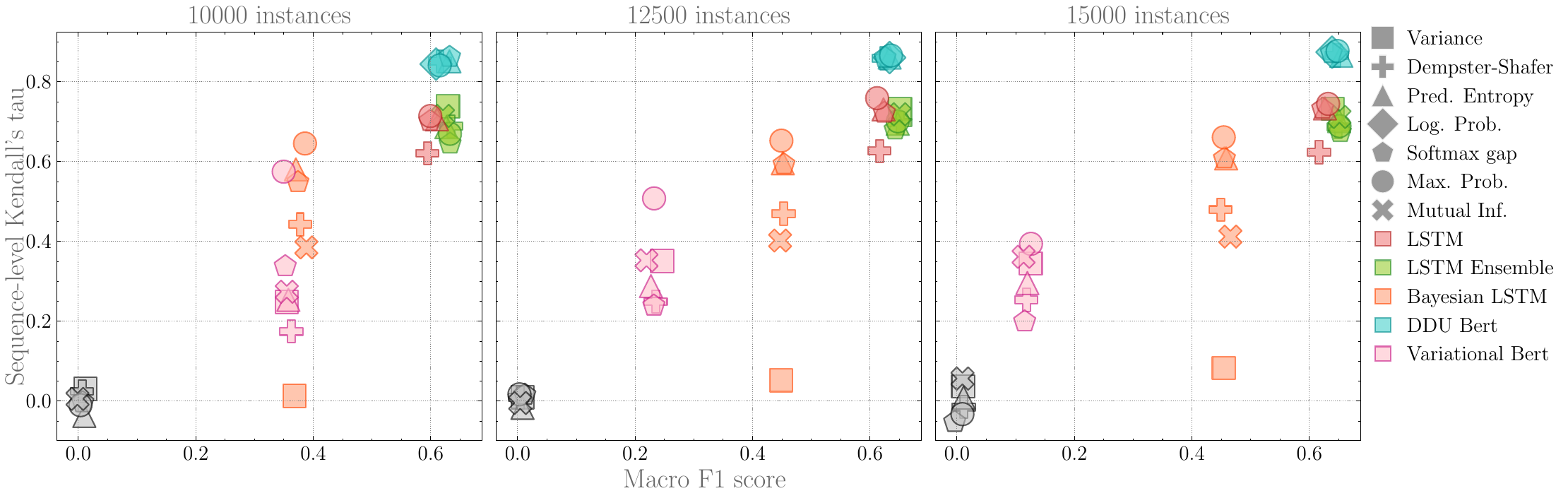}
        \subcaption{Scatter plot for the Clinc Plus dataset.}
        \label{subfig:clinc-plus-scatter-kendalls-tau-seq}
    \end{subfigure}\\
    \begin{subfigure}{\textwidth}
        \centering
        \includegraphics[width=\columnwidth]{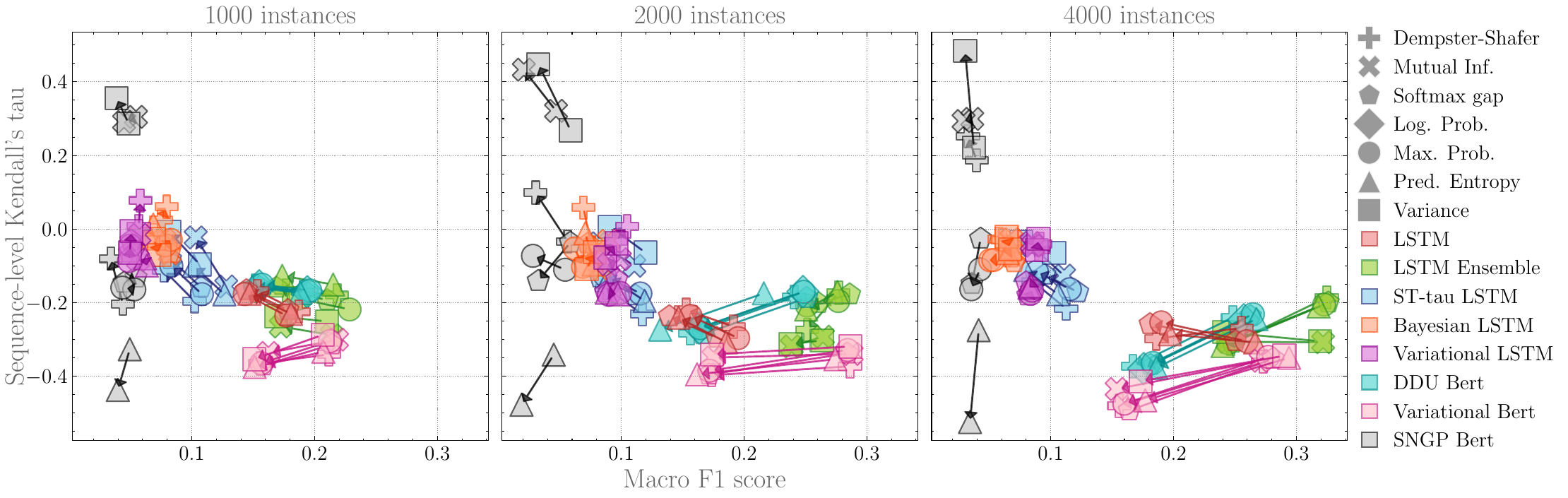}
        \subcaption{Scatter plot for the Dan+ dataset.}
        \label{subfig:dan+-scatter-kendalls-tau-seq}
    \end{subfigure}\\
    \begin{subfigure}{\textwidth}
        \centering
        \includegraphics[width=\columnwidth]{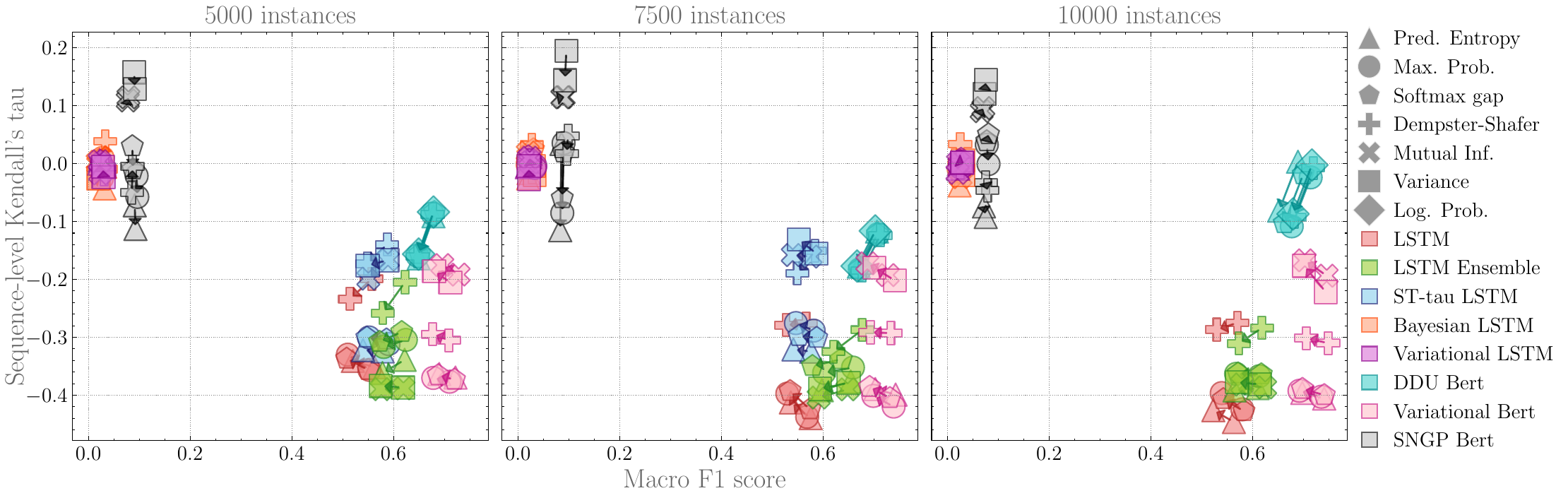}
        \subcaption{Scatter plot for the Finnish UD dataset.}
        \label{subfig:finnish-ud-scatter-kendalls-tau-seq}
    \end{subfigure}
    \caption[Scatter plot showing the difference between model performance and the quality of uncertainty estimates on a sequence-level.]{
    Scatter plot showing the difference between model performance (measured by macro $F_1$) and the quality of uncertainty estimates on a sequence-level (measured by Kendall's $\tau$). 
    Results are shown for different models and uncertainty metrics and several training set sizes on the Finnish UD and Clinc Plus dataset. Arrows indicate changes between the in-distribution and out-of-distribution test set.}\label{fig:scatter-plot-kendalls-tau-seq}
\end{figure*}

\section{Qualitative Analysis}\label{app:qualitative-analysis}

This section provides more examples for the qualitative analysis in \cref{sec:qualitative-analysis}.

\paragraph{Dan+.} 
We show more examples of the predictive entropies on samples from the Dan+ dataset in \cref{fig:qualitative-analysis-extra-danish}, where uncertainty values where jointly normalized by subtracting the mean and dividing by the standard deviation over all models and time steps. 
We can make the following observations: 
Firstly, uncertainty seems to decrease on punctuation marks such as commas and full-stops. 
Secondly, uncertainty appears higher on sub-word tokens and some named entities. 
Thirdly, DDU Bert\index{Transformer!DDU}\index{Bert} and the LSTM\index{Long-short term memory network} ensemble\index{Ensembling} produce the highest uncertainty values, which are also two of the best performing models on the task. 

\paragraph{Finnish UD.} 
Here, we give more examples of the analysis on the Finnish UD dataset in \cref{fig:qualitative-analysis-extra-finnish}.
 First of all, we see that the variational LSTM\index{Long-short term memory network!Variational} and SNGP Bert\index{Transformer!SNGP}\index{Bert} seem to produce almost constant uncertainty scores, which can be explained by their suboptimal performance in task, as shown by their results in \cref{table:predictive-uncertainty-results}. 
But even for the models that perform better, such as the variational Bert\index{Transformer!Variational}\index{Bert} and the LSTM\index{Long-short term memory network} ensemble\index{Ensembling}, the decomposition of predictive entropy\index{Entropy!Predictive} into aleatoric\index{Uncertainty!Aleatoric} and epistemic uncertainty\index{Uncertainty!Epistemic} reveals that model uncertainty generally remains low, and is overshadowed to a larger extent by the aleatoric uncertainty. 
We can observe that similar to Danish, uncertainty seems to be low on punctuation marks and high on subword tokens. 
Furthermore, aleatoric uncertainty seems to be higher on nouns and pronouns. 
This could be due to the sheer number of possible nouns and pronouns that could fill such a gap in a sentence.

\begin{figure}[htb]
    \centering
    \begin{subfigure}[t]{0.99\columnwidth}
        \centering
        \includegraphics[width=0.8\columnwidth]{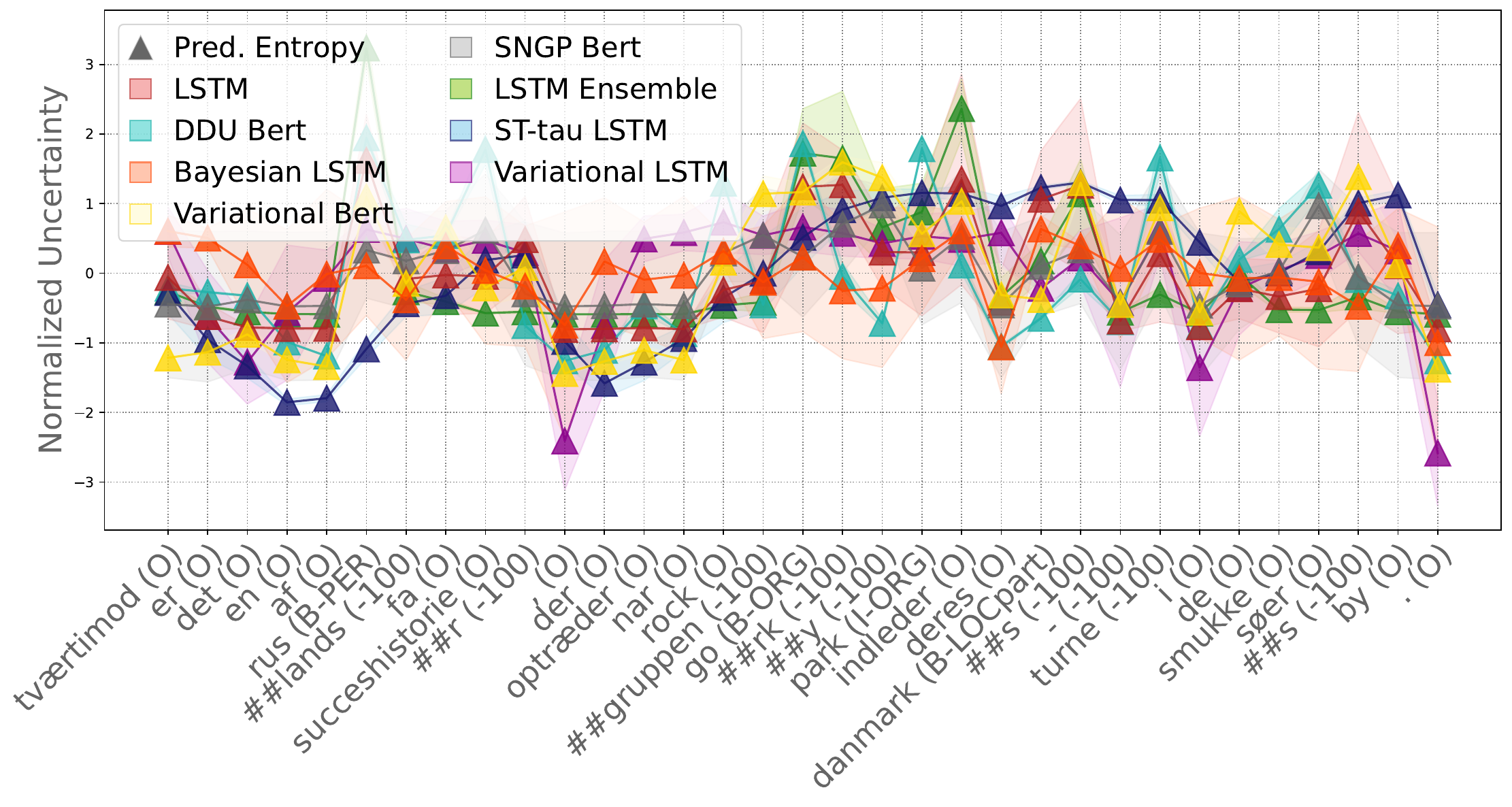}
        \subcaption{Predictive entropy over the sentence \emph{``On the contrary, it is one of Russia's few success stories that performs when the rock group Gorky Park begins their Danish tour in the city of the beautiful lakes''}.}
    \end{subfigure}
    \par\bigskip 
    \centering
    \begin{subfigure}[t]{0.99\columnwidth}
        \centering
        \includegraphics[width=0.8\columnwidth]{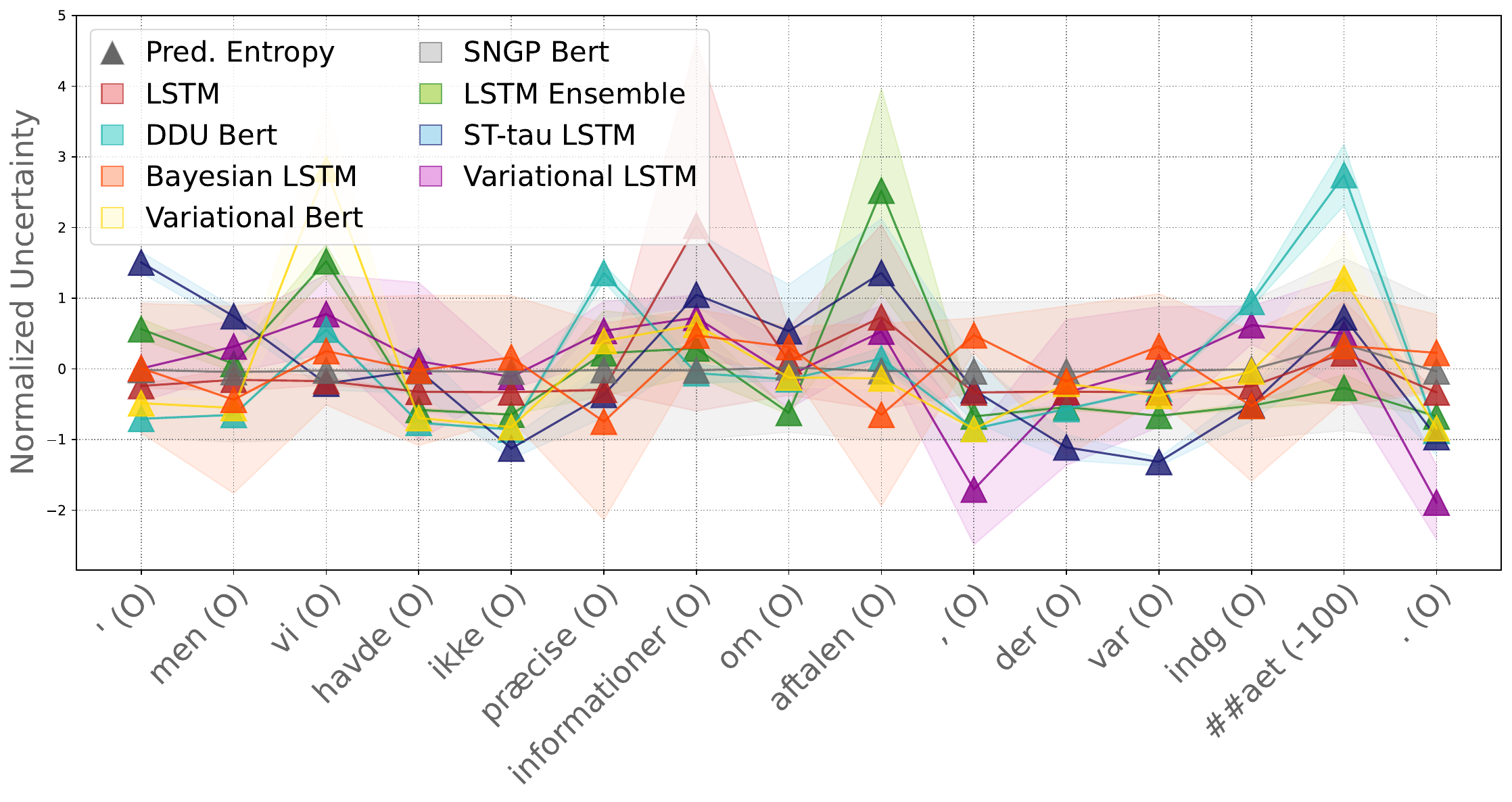}
        \subcaption{Predictive entropy over the sentence \emph{``However, we did not have precise information about what was agreed upon''}.}
    \end{subfigure}
    \par\bigskip 
    \begin{subfigure}[t]{0.99\columnwidth}
        \centering
        \includegraphics[width=0.8\columnwidth]{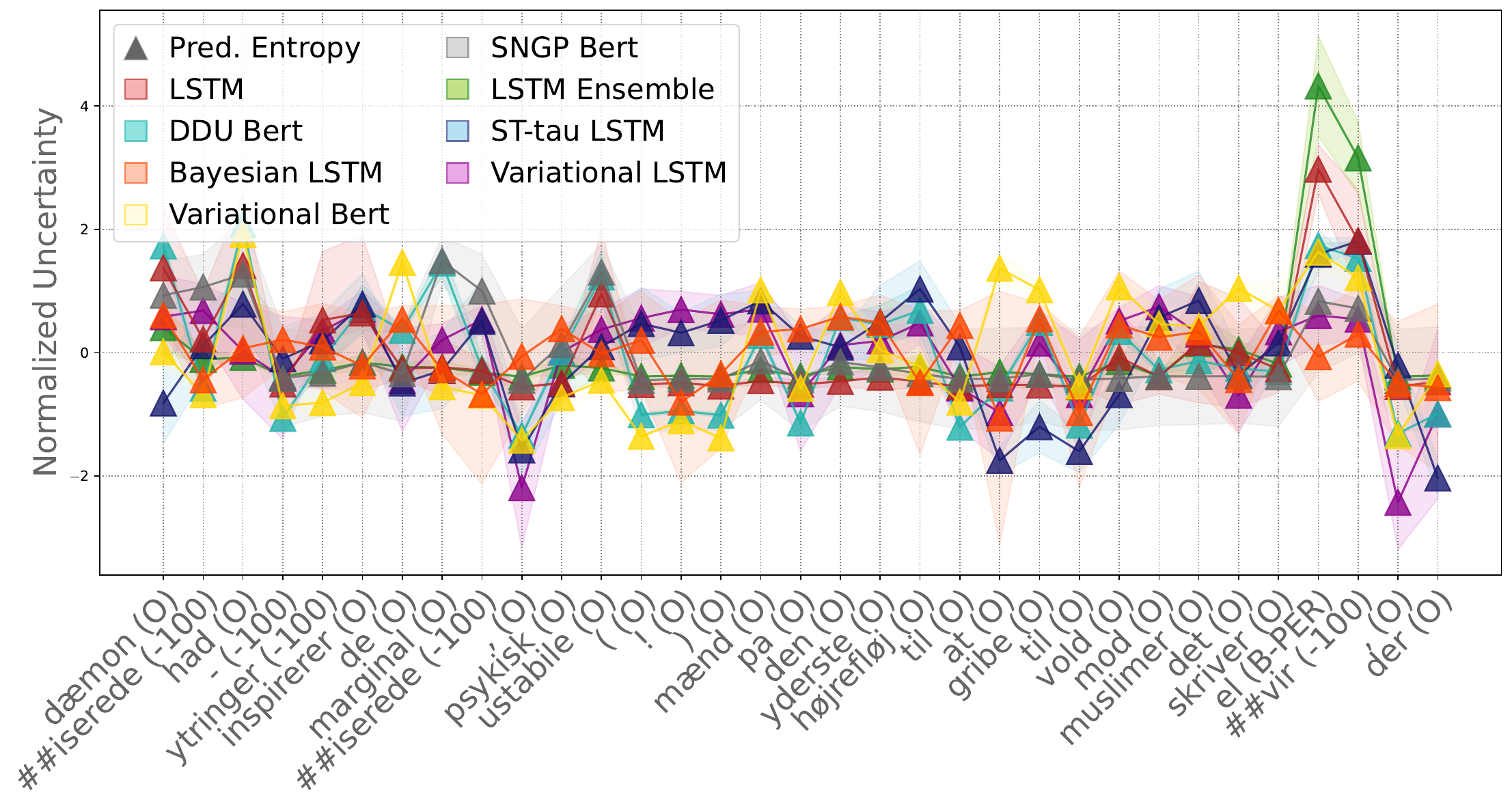}
        \subcaption{Predictive entropy over the sentence \emph{``Demonizing hate speech inspires the marginalized, PSYCHOLOGY UNSTABLE (!) Men on the far right to resort to violence against Muslims. This writes Elvir, who...''}.}
    \end{subfigure}
    \caption[Additional examples for uncertainty estimates on single sequences on the Dan+ dataset.]{
    Further examples for uncertainty estimates on single sequences. Taken from the Dan+ dataset.}\label{fig:qualitative-analysis-extra-danish}
\end{figure}

\begin{figure}[htb]
    \centering
    \begin{subfigure}[t]{0.99\columnwidth}
        \centering
        \includegraphics[width=0.8\columnwidth]{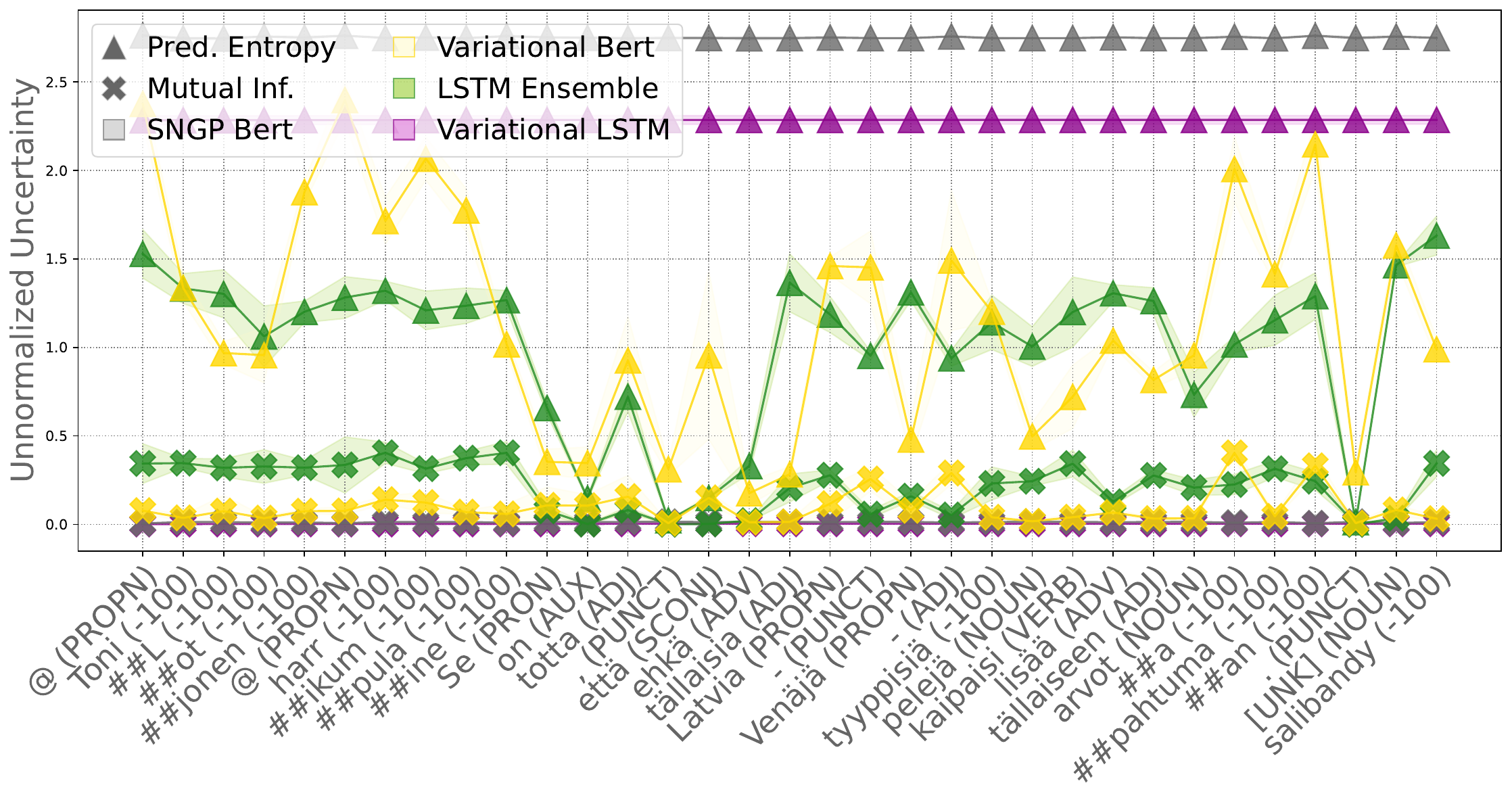}
        \subcaption{Predictive entropy over the sentence \emph{``@ToniLotjonen @harrikumpulaine It is true that I’d maybe like to see more of such Latvia–Russia type games in these kinds of major sports events. \#floorball''}.}
    \end{subfigure}
    \par\bigskip 
    \centering
    \begin{subfigure}[t]{0.99\columnwidth}
        \centering
        \includegraphics[width=0.8\columnwidth]{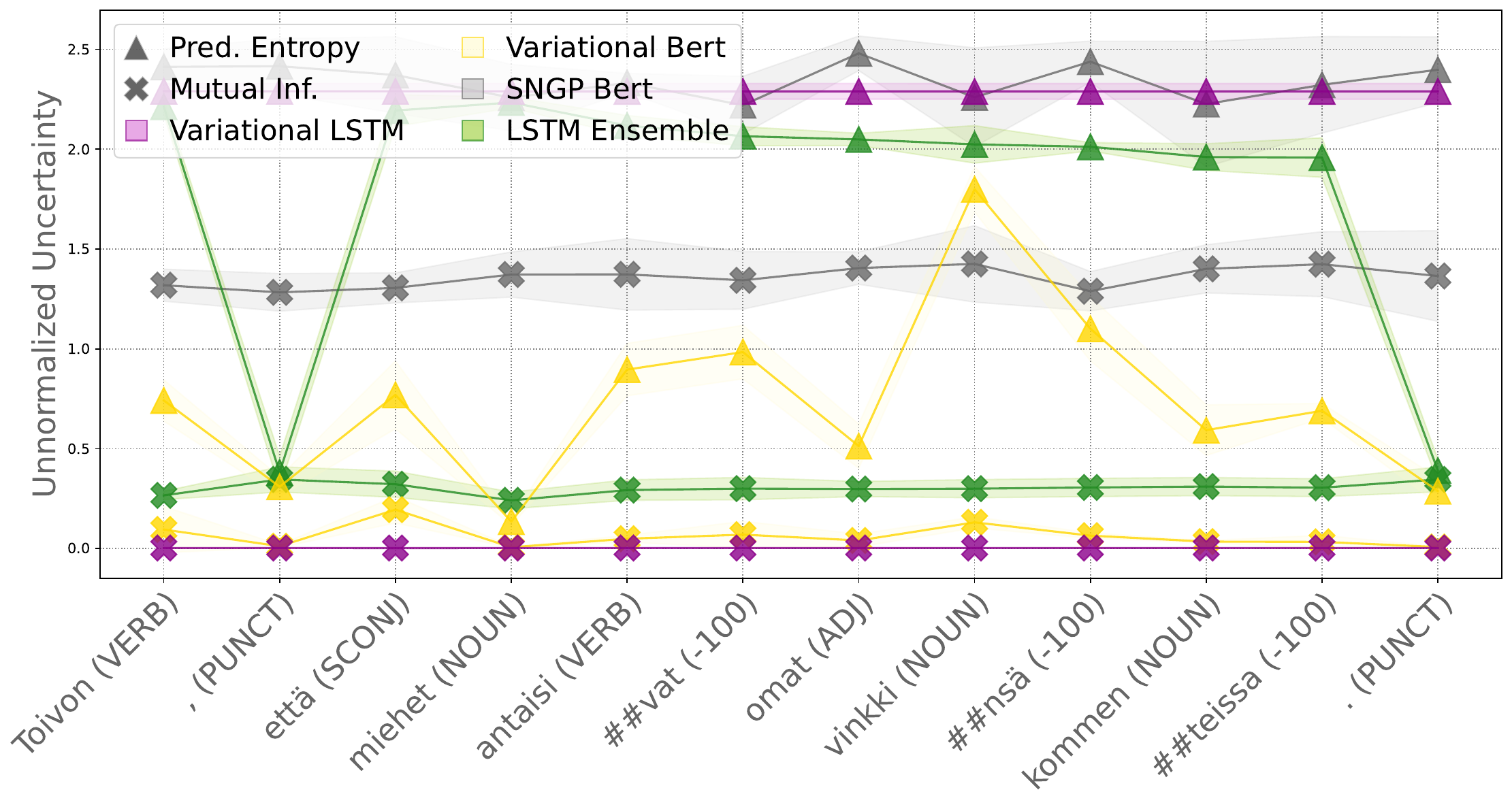}
        \subcaption{Predictive entropy over the sentence \emph{``I hope that the procedures done on the  person in question stop and he gives his body (and mind) time to recover from that poisoning!''}.}
    \end{subfigure}
    \par\bigskip 
    \begin{subfigure}[t]{0.99\columnwidth}
        \centering
        \includegraphics[width=0.8\columnwidth]{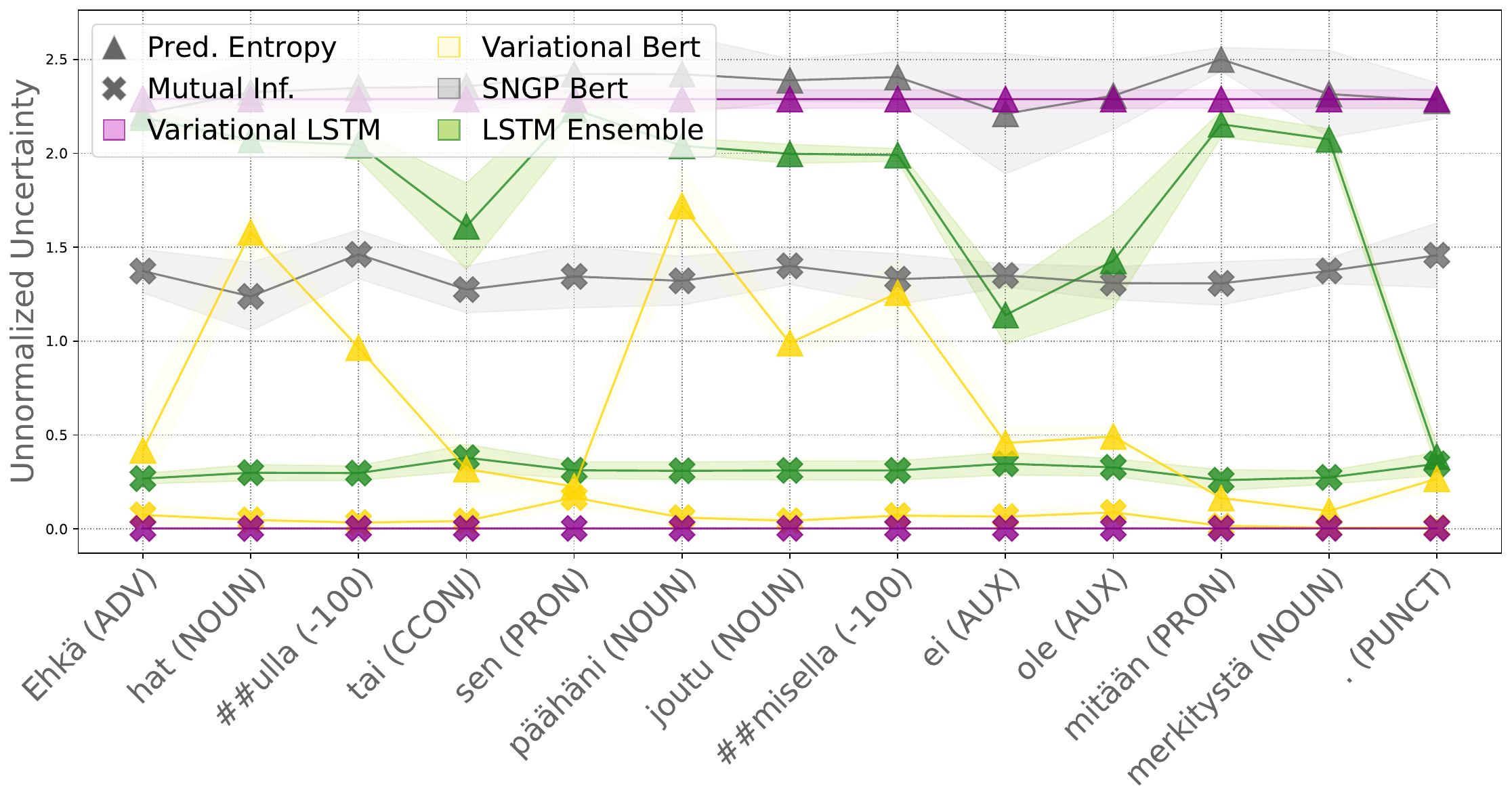}
        \subcaption{Predictive entropy over the sentence \emph{``Maybe the hat or how it got on my head doesn’t matter''}.}
    \end{subfigure}
    \caption[Additional examples for uncertainty estimates on single sequences on the Finnish UD dataset.]{
        Further examples for uncertainty estimates on single sequences. Taken from the Finnish UD dataset.}\label{fig:qualitative-analysis-extra-finnish}
\end{figure}

\section{Additional Coverage Results}\label{app:coverage-experiments}

\begin{figure}[htb!]
    \centering
    \begin{subfigure}[t]{0.48\textwidth}
        \centering
        \includegraphics[width=0.9975\textwidth]{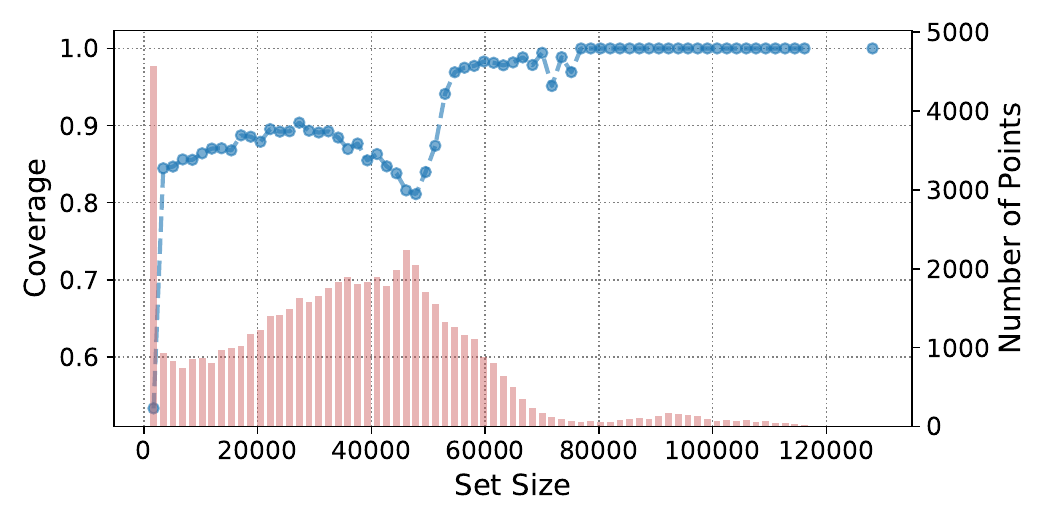}
        \caption{\raggedleft Conditional coverage of M2M100$_\text{(1.2B)}$ for de $\rightarrow$ en.}
        \label{subfig:stratified-coverage-deen-m2m00large}
    \end{subfigure}
    \hfill
    \begin{subfigure}[t]{0.48\textwidth}
        \centering
        \includegraphics[width=0.9975\textwidth]{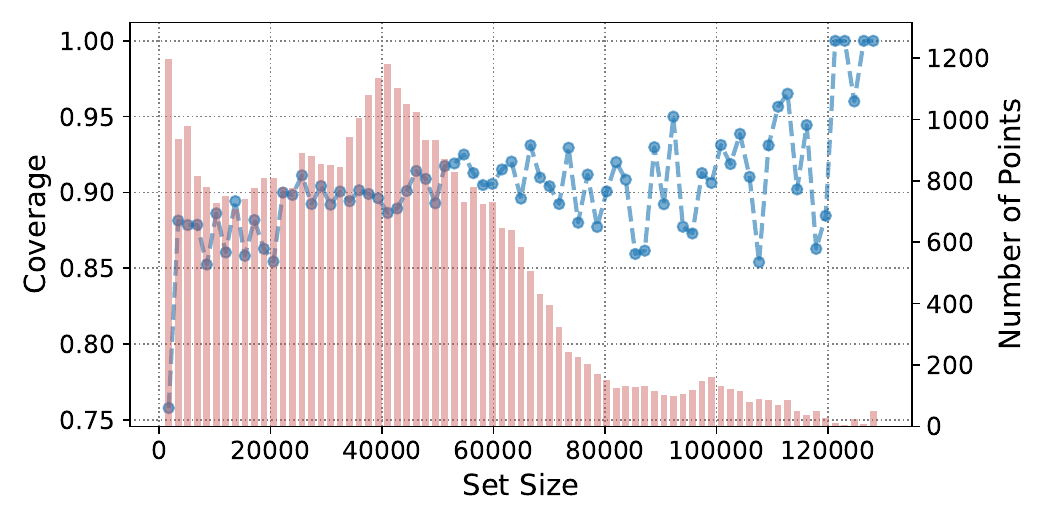}
        \caption{\raggedleft Conditional coverage of M2M100$_\text{(1.2B)}$ for ja $\rightarrow$ en.}
        \label{subfig:stratified-coverage-jaen-m2m00large}
    \end{subfigure}
    \begin{subfigure}[t]{0.48\textwidth}
        \centering
        \includegraphics[width=0.9975\textwidth]{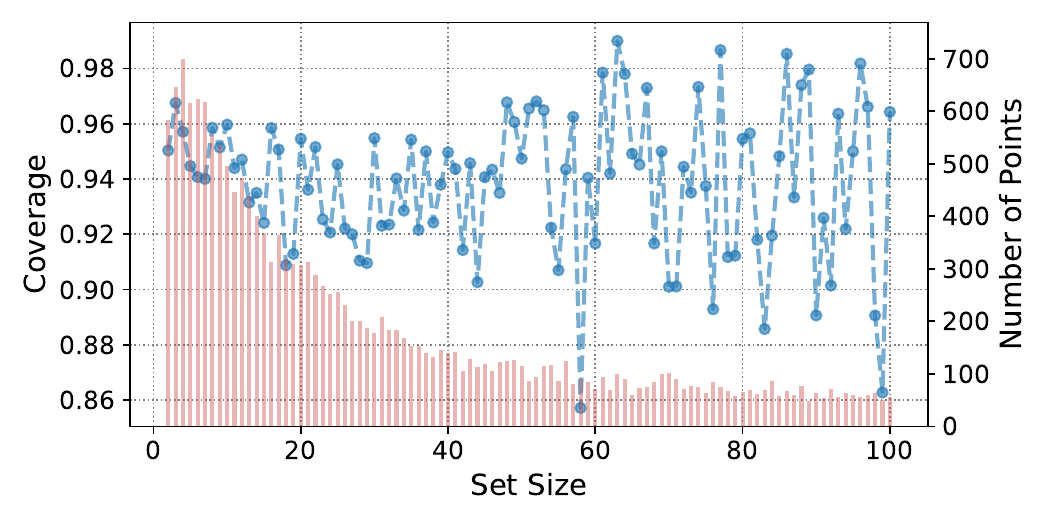}
        \caption{\raggedleft Conditional coverage for OPT$_\text{(350M)}$ on Language Modelling.}
        \label{subfig:tratified-coverage-opt}
    \end{subfigure}
    \hfill
    \begin{subfigure}[t]{0.48\textwidth}
        \centering
        \includegraphics[width=0.9975\textwidth]{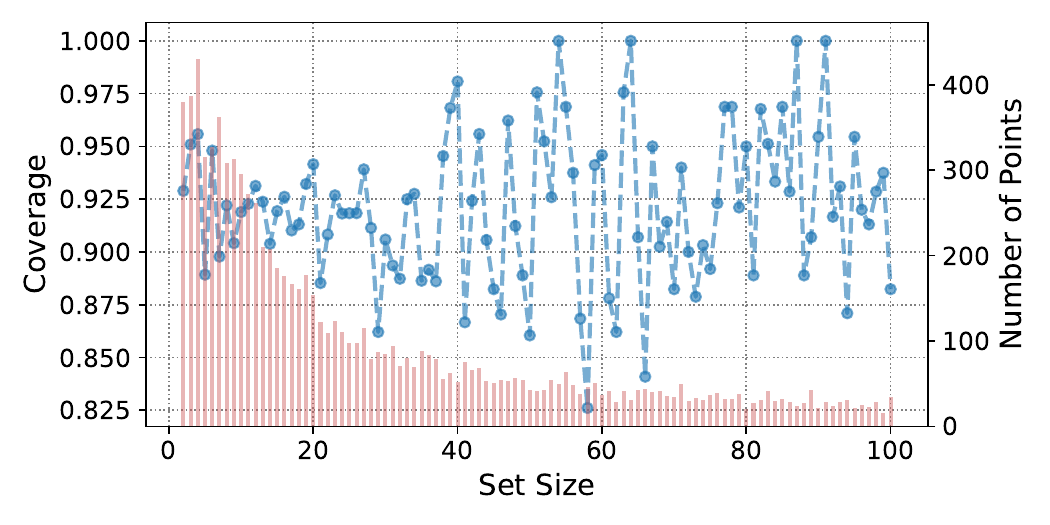}
        \caption{\raggedleft Conditional coverage for OPT$_\text{(1.3B)}$ on Language Modelling.}
        \label{subfig:stratified-coverage-opt-large}
    \end{subfigure}
    \caption[Additional conditional coverage plots for the MT and LM dataset using our non-exchangeable conformal natural generation.]{
        Additional conditional coverage plots for the MT and LM dataset using our non-exchangeable conformal prediction method, aggregating predictions by prediction set size. The blue curve shows the conditional coverage per bin, whereas red bars show the number of predictions per bin. For \cref{subfig:tratified-coverage-opt,subfig:stratified-coverage-opt-large}, we zoom in on the prediction set sizes from $1$ and $100$.}\label{fig:additional-conditional-coverage}
\end{figure}

We show additional plots for the experiments in \cref{sec:retrieval-quality}, illustrating the coverage\index{Coverage} per set size-bins in \cref{fig:additional-conditional-coverage}. 
We can see the counterparts for \cref{fig:conditional-coverage} using the larger M2M100$_\text{(1.2B)}$ model in \cref{subfig:stratified-coverage-deen-m2m00large,subfig:stratified-coverage-jaen-m2m00large}:
Instead of leveling off like for the smaller model, most prediction set sizes are either in a very small range or in a size of a few ten thousand.
In \cref{subfig:tratified-coverage-opt,subfig:stratified-coverage-opt-large}, we show similar plots for the two different OPT model sizes. 
Since in both cases, most prediction set\index{Prediction set} sizes are rather small, we zoom in on the the sizes from $1$ to $100$. Here, we can observe a similar behavior to the smaller M2M100$_\text{(400m)}$, gradually leveling off. 
We do not show similar plots for other distance metrics as they show similar trends.

\section{Ablating Neighborhood Size and Desired Coverage}\label{app:ablations}

In this section, we present experiments surrounding the two most pivotal parameters of our method in \cref{sec:nonex-nlg-method}: 
The desired confidence level $\alpha$, as well as the number of neighbors.

%\makebox[0pt][c]{\parbox{1.2\textwidth}{%
\begin{minipage}[b]{0.45\textwidth}
    \hspace{-0.5cm}
    \centering
    \begin{table}[H]
    \centering
    \resizebox{0.98\textwidth}{!}{
    \renewcommand{\arraystretch}{1.4}
        \begin{tabular}{@{}llrrrr@{}}
        \toprule[0.05cm]
         & $\alpha$ & \% \textsc{Cov.} & $\varnothing$ \textsc{Width} $\downarrow$ &  \textsc{Scc} $\uparrow$ & \textsc{Ecg} $\downarrow$ \\
         \midrule
        \multirow{9}{*}{\rotatebox[origin=c]{90}{M2M100$_\text{(400M)}$ / de $ \rightarrow $ en}} & $.1$ & $.9442$ & $.31$ & $.8702$ & $.0011$ \\
        & $.2$ & $.8767$ & $.18$ & $.7906$ & $8.63 \times 10^{-5}$ \\
        & $.3$ & $.7963$ & $.12$ & $.0000$ & $.0016$ \\
        & $.4$ & $.7058$ & $.09$ & $.1393$ & $.0082$ \\
        & $.5$ & $.6081$ & $.07$ & $ .2836$ & $.0055$ \\
        & $.6$ & $.5017$ & $.06$ & $.1393$ & $.0082$ \\
        & $.7$ & $.3896$ & $.05$ & $.0000$ & $.0091$ \\
        & $.8$ & $.2800$ & $.05$ & $.0000$ & $.0090$ \\
        & $.9$ & $.1762$ & $.04$ & $.0000$ & $.0071$ \\
        \midrule
        \multirow{9}{*}{\rotatebox[origin=c]{90}{M2M100$_\text{(400M)}$ / ja $ \rightarrow $ en}} & $.1$ & $.7453$ & $.15$ & $.3080$ & $.1511$ \\
        & $.2$ & $.5579$ & $.07$ & $.2728$ & $.2446$ \\
        & $.3$ & $.4277$ & $.04$ & $.2770$ & $.2779$ \\
        & $.4$ & $.3438$ & $.03$ & $.1212$ & $.2438$\\
        & $.5$ & $.2749$ & $.03$ & $.0455$ & $.1883$ \\
        & $.6$ & $.2175$ & $.02$ & $.0000$ & $.1207$\\
        & $.7$ & $.1685$ & $.02$ & $.0000$ & $.0560$ \\
        & $.8$ & $.1309$ & $.01$ & $.0000$ & $.0117$ \\
        & $.9$ & $.0989$ & $.02$ & $.000$0 & $.0099$\\
        \midrule
        \multirow{9}{*}{\rotatebox[origin=c]{90}{OPT$_\text{(350M)}$ / \textsc{OpenWebText}}} & $.1$ & $.9460$ & $.26$ & $.8$ & $1.85 \times 10^{-5}$ \\
        & $.2$ & $.8937$ & $.16$ & $.8000$ & $.000$ \\
        & $.3$ & $.8392$ & $.10$ & $.5000$ & $8.74 \times 10^{-6}$ \\
        & $.4$ & $.7782$ & $.08$ & $.6667$ & $.000$ \\
        & $.5$ & $.7171$ & $.06$ & $.0000$ & $1.19 \times 10^{-5}$ \\
        & $.6$ & $.6559$ & $.06$ & $.6033$ & $.000$ \\
        & $.7$ & $.5945$ & $.05$ & $.000$ & $8.21 \times 10^{-6}$ \\
        & $.8$ & $.5349$ & $.05$ & $.4462$ & $.000$ \\
        & $.9$ & $.4757$ & $.05$ & $.3580$ & $.000$ \\
        \bottomrule[0.05cm]
    \end{tabular}%
    }
    \caption{Results for varying values of $\alpha$ using different models and datasets.}
    \label{tab:alpha-ablation}
    \end{table}
\end{minipage}
\hspace{0.3cm}
\begin{minipage}[b]{0.45\textwidth}
    \centering
    \begin{table}[H]
    %\centering
    \resizebox{0.98\textwidth}{!}{
    \renewcommand{\arraystretch}{1.4}
        \begin{tabular}{@{}llrrrr@{}}
        \toprule[0.05cm]
         & $K$ & \% \textsc{Cov.} & $\varnothing$ \textsc{Width} $\downarrow$ &  \textsc{Scc} $\uparrow$ & \textsc{Ecg} $\downarrow$ \\
         \midrule
        \multirow{9}{*}{\rotatebox[origin=c]{90}{\hspace{0.75cm} M2M100$_\text{(400M)}$ / de $ \rightarrow $ en}} & $10$ & $.9923$ & $.39$ & $.9728$ & $.0000$ \\
        & $25$ & $.9563$ & $.37$ & $.8877$ & $.0011$ \\
        & $50$ & $.9504$ & $.32$ & $.8870$ &  $.0006$\\
        & $75$ & $.9444$ & $.32$ & $.8641$ & $.0014$ \\
        & $100$ & $.9442$ & $.31$ & $.8702$ & $.0011$ \\
        & $200$ & $.9422$ & $.31$ & $.8125$ & $.0016$ \\
        & $300$ & $.9404$ & $.31$ & $.8483$ & $.0019$ \\
        & $500$ & $.9389$ & $.31$ & $.8214$ & $.0023$ \\
        \midrule
        \multirow{9}{*}{\rotatebox[origin=c]{90}{\hspace{0.75cm} M2M100$_\text{(400M)}$ / ja $ \rightarrow $ en}} & $10$ & $.8013$ & $.17$ & $.2995$ & $.1606$ \\
        & $25$ & $.7353$ & $.17$ & $.2994$ & $.1438$ \\
        & $50$ & $.7540$ & $.17$ & $.3023$ & $.1603$ \\
        & $75$ & $.7368$ & $.16$ & $.3019$ & $.1603$ \\
        & $100$ & $.7453$ & $.15$ & $.3072$ & $.1529$ \\
        & $200$ & $.7295$ & $.14$ & $.2938$ & $.1787$ \\
        & $300$ & $.7192$ & $.13$ & $.2948$ & $.1788$ \\
        & $500$ & $.7110$ & $.13$ & $.2756$ & $.1867$ \\
        \midrule
        \multirow{9}{*}{\rotatebox[origin=c]{90}{\hspace{0.75cm} OPT$_\text{(350M)}$ / \textsc{OpenWebText}}} & $10$ & $.9438$ & $.35$ & $.8824$ & $.0019$ \\
        & $25$ & $.9522$ & $.33$ & $.8333$ & $2.06 \times 10^{-5}$ \\
        & $50$ & $.9442$ & $.27$ & $.0000$ & $1.86 \times 10^{-5}$\\
        & $75$ & $.9477$ & $.27$ & $.8000$ & $1.03 \times 10^{-5}$ \\
        & $100$ & $.9460$ & $.26$ & $.8000$ & $1.86 \times 10^{-5}$ \\
        & $200$ & $.9487$ & $.28$ & $.8571$ & $6.20 \times 10^{-5}$ \\
        & $300$ & $.9500$ & $.28$ & $.8181$ & $1.86 \times 10^{-5}$ \\
        & $500$ & $.9508$ & $.29$ & $.8181$ & $1.86 \times 10^{-5}$ \\
        \bottomrule[0.07cm]
    \end{tabular}%
    }
    \caption{Results for varying neighborhood sizes $K$ using different models and datasets.}
    \label{tab:neighbor-ablation}
    \end{table}
    \hspace{0.5cm}
\end{minipage}%
%}}

\paragraph{Coverage Threshold.} In \cref{tab:alpha-ablation}, we investigate the impact of different values on $\alpha$ on our evaluation metrics. 
We show that the increase in $\alpha$ does indeed produce the expected decrease in coverage\index{Coverage}, however with a certain degree of overcoverage for the de $ \rightarrow $ en MT\index{Machine translation} and the LM\index{Language modeling} task.
The loss in coverage always goes hand in hand with a decrease in the average prediction set\index{Prediction set} width as well, as the model can allow itself to produce tighter prediction sets at the cost of higher miscoverage. 
As this also produces bin in which all contained instances are uncovered, this produces zero values for the SCC\index{Size-stratified coverage}, while we cannot discern clear trends for the ECG\index{Expected coverage gap}.

\paragraph{Neighborhood Size.} In \cref{tab:neighbor-ablation}, we vary the effect of the chosen neighborhood size (with $100$ being the value we use in our main experiments). 
We make the following, interesting observations: Coverage\index{Coverage} on the MT\index{Machine translation} task seems to decrease with an increase in the neighborhood size as prediction set\index{Prediction set} widths get smaller on average, with a neighborhood size around $100$ striking a balance between coverage, width, computational cost and SCC\index{Size-stratified coverage} / ECG\index{Expected coverage gap}.
For LM\index{Language modeling}, coverage seems to be mostly constant, with prediction set\index{Prediction set} width hitting an inflection point for $100$ neighbors.
We speculate that initially there might be a benefit to considering more neighbors to calibrate $\hat{q}$, but that considering too large neighborhoods might introduce extra noise. 
While we found $100$ to be a solid choice for the purpose of our experiments, we leave more principled ways to determine the neighborhood size to future work.

\section{Additional Clustering Results}\label{app:additional-clustering}

\begin{figure}[htb]
    \centering
    \begin{subfigure}[t]{0.48\textwidth}
        \centering
        \includegraphics[width=0.98\textwidth]{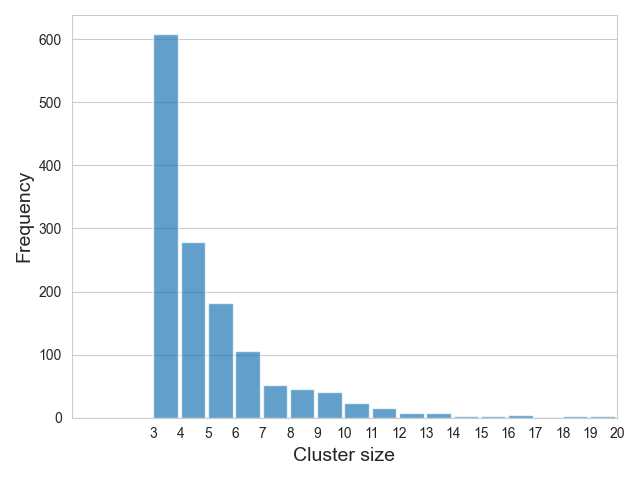}
    \caption{Cluster sizes on TriviaQA.}
    \end{subfigure}
    \hfill
    \begin{subfigure}[t]{0.48\textwidth}
        \centering
        \includegraphics[width=0.98\textwidth]{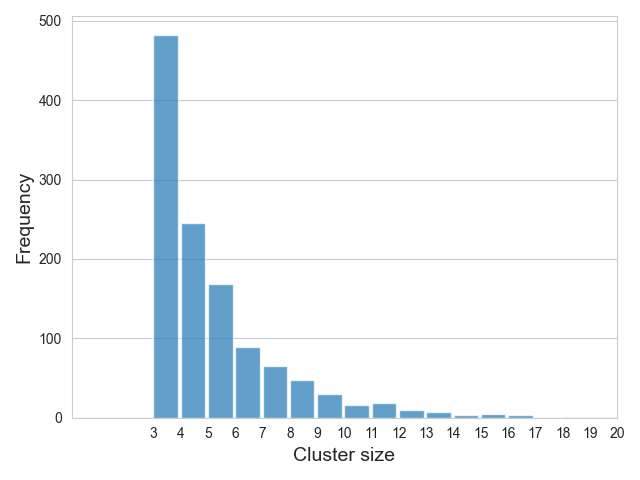}
    \caption{Cluster sizes on CoQA.}
    \end{subfigure}
    \caption[Bar plot of cluster sizes found.]{Bar plot of cluster sizes found. The plot is truncated at size 20.}\label{fig:cluster-sizes}
\end{figure}

\begin{minipage}[t]{0.47\textwidth}
    \centering
    \begin{table}[H]
    \renewcommand{\arraystretch}{0.68}
     \resizebox{0.95\textwidth}{!}{
    \begin{tabular}{@{}p{0.95\linewidth}@{}}
    \toprule
    \scriptsize Cluster 1120 \\
    \midrule
    \scriptsize How many fluid ounces are in one quarter of an imperial pint? \\
    \scriptsize How many fluid ounces in one Imperial pint? \\
    \scriptsize How many fluid ounces in half an Imperial pint? \\
    \toprule
    \scriptsize Cluster 920 \\
    \midrule
    \scriptsize Which famous US outlaw shot the cashier of a savings bank in Gallatin Missouri in 1869? \\
    \scriptsize What famous outlaw committed the Wild West's first train robbery on July 21, 1873 in Adair, Iowa? \\
    \scriptsize On July 21, 1873, Jesse James and the James-Younger gang pulled off the first successful what of the American West, in Adair Iowa? \\
    \toprule
    \scriptsize Cluster 984 \\
    \midrule
    \scriptsize In what country was the game of golf invented?\\
    \scriptsize Which ball game was invented by Dr James Naismith in Massachusetts USA  in 1891? \\
    \scriptsize It is generally accepted that the game of golf originated in what country? \\
    \scriptsize  What's a country where most people love playing rugby? \\
    \scriptsize What's a country where most people love playing golf? \\
    %\toprule
    %\scriptsize Cluster 64 \\
    %\midrule
    %\scriptsize Which alternative word for the Devil is a Hebrew word with translates as Lord Of The Flies? \\
    %\scriptsize Beelzebub is Hebrew for what phrase, which is also the title of a famous novel? \\
    %\scriptsize Diablo is another name for who? \\
    \toprule
    \scriptsize Cluster 811 \\
    \midrule
    \scriptsize How many colors are there in the spectrum when white light is separated? \\
    \scriptsize Which part of the eye contains about 137 million light-sensitive cells in one square inch? \\
    \scriptsize Which of the retina's cells can distinguish between different wavelengths of light? \\
    \scriptsize In four colour process printing, which is also known as CMYK, which are the only four colours that are used? \\
    \scriptsize How many colours are in the rainbow? \\
    \scriptsize In art, what are the three primary colours? \\
    \scriptsize What color consists of the longest wavelengths of lights visible by the human eye? \\
    \scriptsize What are the three primary colours of light? \\
    % \toprule
    % \scriptsize Cluster 74 \\
    %\midrule
    %\scriptsize What was the name of the computer in the movie 2001: A Space Odyssey? \\
    %\scriptsize What was the name of the computer in Blake's Seven? \\
    %\scriptsize Developed by IBM, Deep Blue was a computer that played what? \\
    %\scriptsize  What was the name of the PDA produced by Apple, most famous for its handwriting recognition software turning Random House into Condom Nose during a major presentation? \\
    \bottomrule
    \end{tabular}%
    }
    \caption{Contents of some randomly sampled clusters that result from the clustering procedure for TriviaQA.}\label{fig:cluster-contents}
    \end{table}
\end{minipage}
\hspace{0.2cm}
\begin{minipage}[t]{0.47\textwidth}
    \centering
    \begin{table}[H]
    \renewcommand{\arraystretch}{0.8}
    \resizebox{0.95\textwidth}{!}{
    \begin{tabular}{@{}p{0.95\linewidth}@{}}
    \toprule
    \scriptsize Cluster 823 \\
    \midrule
    \scriptsize Where in Europe is it located?\\
    \scriptsize Is it in the European Plain?\\
    \scriptsize Which region of Europe is it in?\\
    \toprule
    \scriptsize Cluster 1176 \\
    \midrule
    \scriptsize Did she have children?\\
    \scriptsize Does she have any children?\\
    \scriptsize Did she have any children?\\
    \scriptsize Did she have any other children?\\
    \toprule
    \scriptsize Cluster 2244 \\
    \midrule
    \scriptsize Who won the Kentucky Derby?\\
    \scriptsize as he won the Derby before?\\
    \scriptsize Has he raced in the Derby before?\\
    \scriptsize What were the winning horse's odds?\\
    \scriptsize How many Derbys have their been?\\
    \toprule
    \scriptsize Cluster 11 \\
    \midrule
    \scriptsize Are they densities of everything the same?\\
    \scriptsize What is the densest elements at regular conditions?\\
    \scriptsize What is density of a substance?\\
    \scriptsize What is another symbol for density?\\
    \scriptsize Who gives weight per unit volume as the definition?\\
    \scriptsize Where is density the same value as it's mass concentration?\\
    \scriptsize To make comparisons easier what stands in for density?\\
    \scriptsize What is the relative density of something that floats?\\
    \toprule
    \scriptsize  Cluster 1081 \\
    \midrule
    \scriptsize Who was murdered?\\
    \scriptsize who was killed?\\
    \scriptsize Who committed this murder?\\
    \scriptsize who was killed?\\
    \scriptsize Who was killed?\\
    \scriptsize who was killed?\\
    % \toprule
    %\scriptsize Cluster 579\\
    %\midrule
    %\scriptsize Did it succeed?\\
    %\scriptsize Did they continue to try?\\
    %\scriptsize did they succeed?\\
    %\scriptsize Was it successful?\\
    %\scriptsize Did they expect that success?\\
    \bottomrule
    \end{tabular}%
    }
    \caption{Contents of some randomly sampled clusters that result from the clustering procedure for CoQA.}\label{fig:cluster-contents-coqa}
    \end{table}
\end{minipage}%

In this section we take a closer look at the results of the clustering procedure described in \cref{sec:setting-calibration-targets}.
In our experiments, we run HDBSCAN\index{HDBSCAN} using a minimum cluster size of three, since preliminary experiments showed this number to produce the best trade-off between the coherence of cluster contents (as evaluated in \cref{fig:clustering-results}) and a diversity in cluster targets.
This setting yields a distribution of cluster sizes shown in \cref{fig:cluster-sizes}.
We can see that the majority of cluster sizes are rather small, including questions on specific topics, some of which we display in \cref{fig:cluster-contents,fig:cluster-contents-coqa}.
Not shown are cluster sizes over $20$ since the distribution quickly levels off, as well the set of all points that could not be sorted into any cluster.\\

\begin{figure*}[htb]
    \centering
    \begin{subfigure}[t]{0.48\textwidth}
        \centering
        \includegraphics[width=0.98\textwidth]{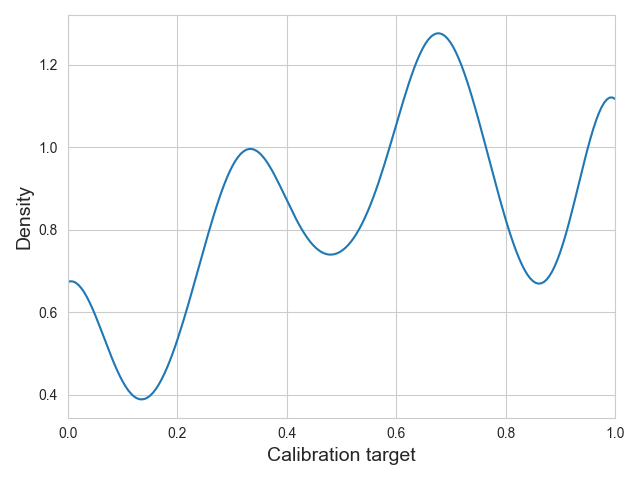}
        \caption{Vicuna v1.5 on TriviaQA.}\label{subfig:calibration-targets-vicuna}
    \end{subfigure}
    \hfill
    \begin{subfigure}[t]{0.48\textwidth}
        \centering
        \includegraphics[width=0.98\textwidth]{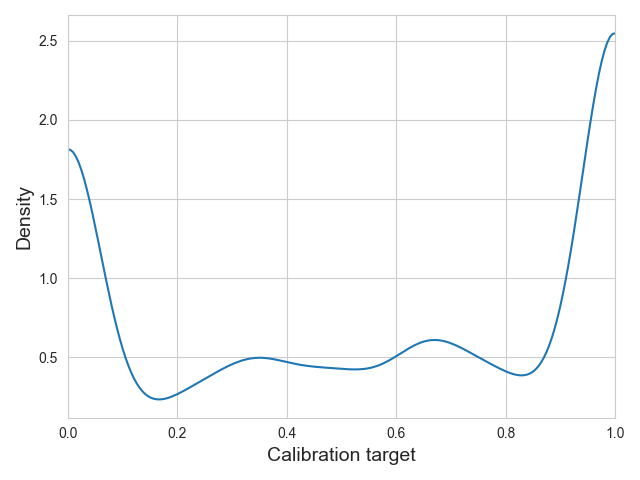}
        \caption{GPT-3.5 on TriviaQA.}\label{subfig:calibration-targets-gpt35}
    \end{subfigure}
    \vfill
    \begin{subfigure}[t]{0.48\textwidth}
        \centering
        \includegraphics[width=0.98\textwidth]{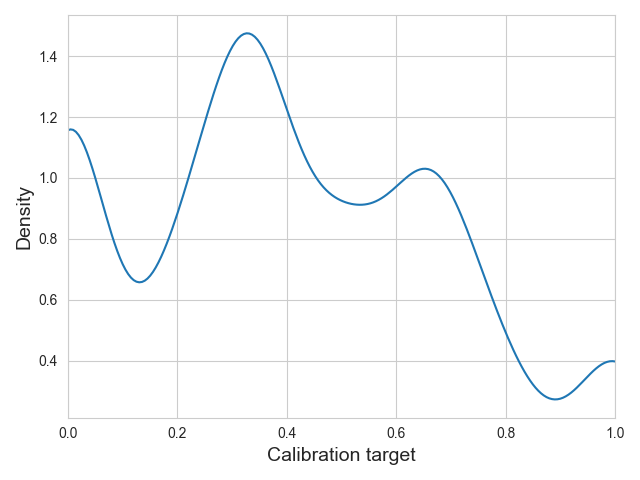}
        \caption{Vicuna v1.5 on CoQA.}\label{subfig:calibration-targets-coqa-vicuna}
    \end{subfigure}
    \hfill
    \begin{subfigure}[t]{0.48\textwidth}
        \centering
        \includegraphics[width=0.98\textwidth]{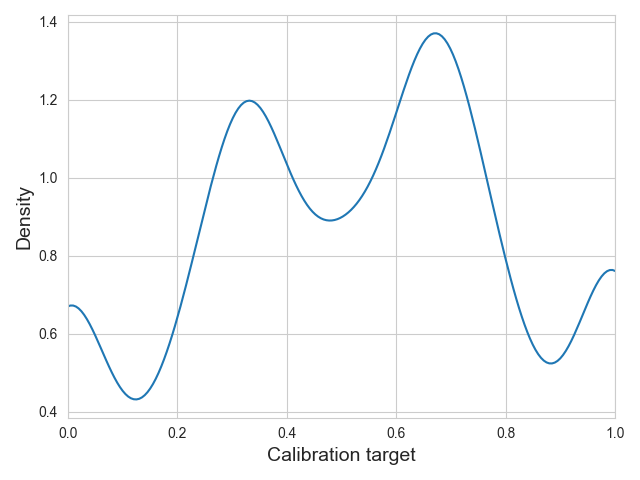}
        \caption{GPT-3.5 on CoQA.}\label{subfig:calibration-targets-coqa-gpt-3.5}
    \end{subfigure}
    \caption{Density plot of calibration targets generated through the clustering procedure for the two LLMs and TriviaQA / CoQA.}\label{fig:calibration-targets}
\end{figure*}

\begin{figure*}[htb]
    \centering
    \begin{subfigure}[t]{0.48\textwidth}
        \centering
        \includegraphics[width=0.98\textwidth]{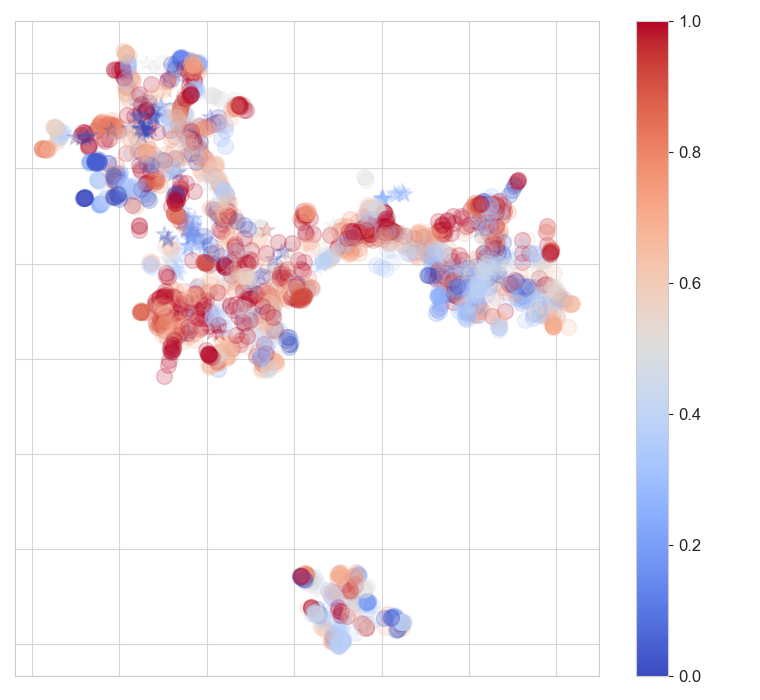}
        \caption{Vicuna v1.5 on TriviaQA.}\label{subfig:clustering-vicuna}
    \end{subfigure}
    \hfill
    \begin{subfigure}[t]{0.48\textwidth}
        \centering
        \includegraphics[width=0.98\textwidth]{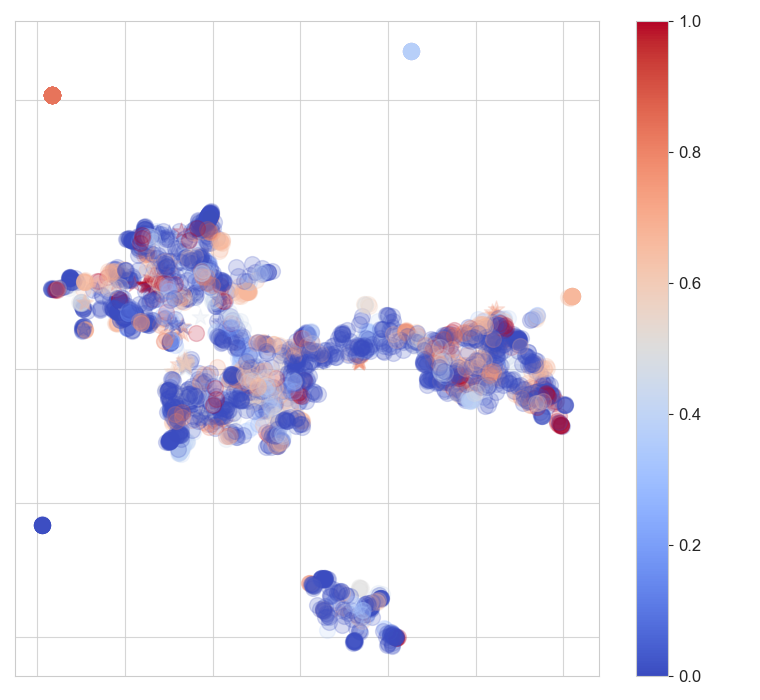}
        \caption{GPT-3.5 on TriviaQA.}\label{subfig:clustering-gpt35}
    \end{subfigure}
    \vfill
    \begin{subfigure}[t]{0.48\textwidth}
        \centering
        \includegraphics[width=0.98\textwidth]{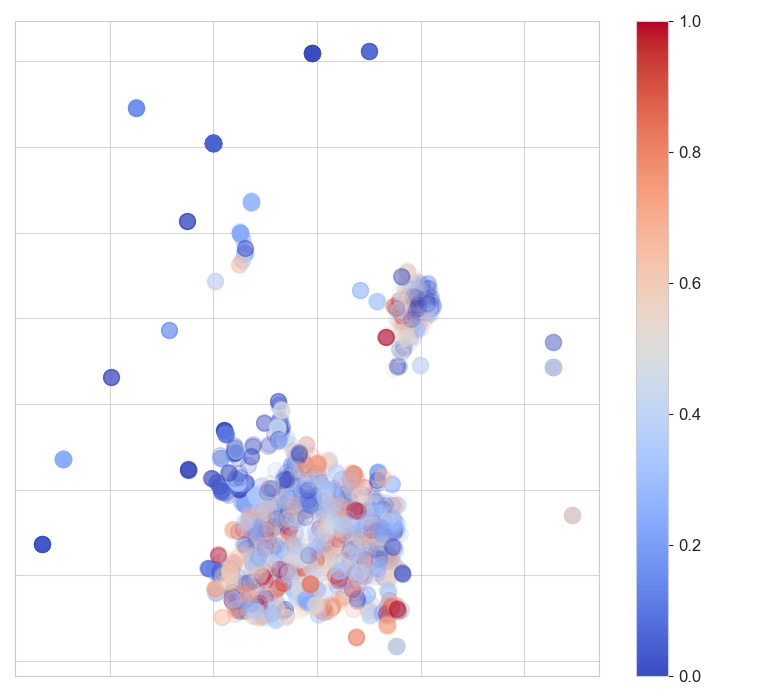}
        \caption{Vicuna v1.5 on CoQA.}\label{subfig:clustering-vicuna-coqa}
    \end{subfigure}
    \hfill
    \begin{subfigure}[t]{0.48\textwidth}
        \centering
        \includegraphics[width=0.98\textwidth]{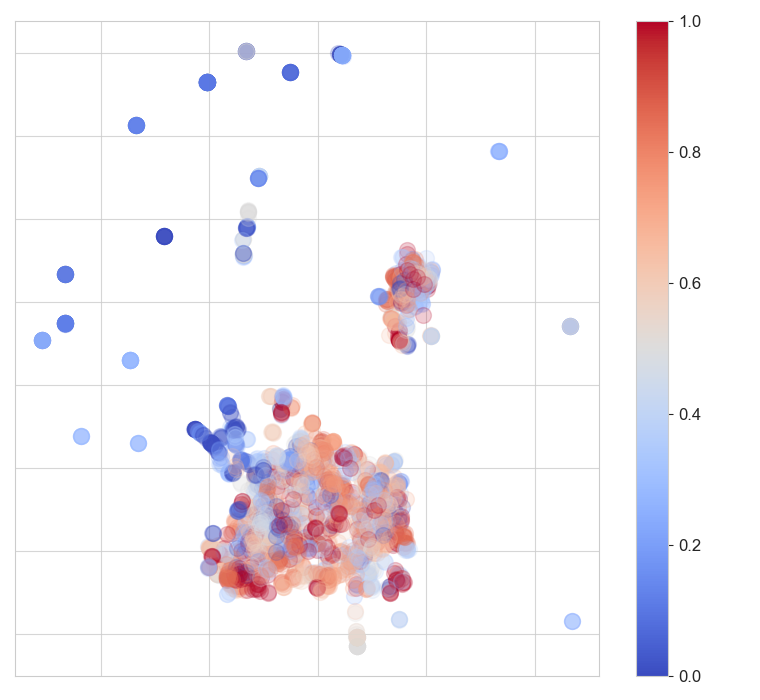}
        \caption{GPT-3.5 on CoQA.}\label{subfig:clustering-gpt35-coqa}
    \end{subfigure}
    % nd TruthfulQA ($\star$)
    \caption[Illustrating questions from TriviaQA along with their assigned confidence targets for the two LLMs.]{
        Illustrating questions from TriviaQA along with their assigned confidence targets for the two LLMs, signified through their color from dark blue (0) to dark red (1). 
        To avoid clutter, we subsampled $40 \%$ of the combined datasets to be shown here and used PacMAP \citep{wang2021understanding} to transform their sentence embeddings into 2D space.
        }\label{fig:clustering}
\end{figure*}

After clustering and computing the average accuracy per cluster, we obtain a distribution over calibration targets, which we show with density plots in \cref{fig:calibration-targets}.
Since most clusters are of size three, we can see clear modes around $0,\ 0.33,\ 0.66$ and $1$ for Vicuna v1.5 in \cref{subfig:calibration-targets-vicuna}.
For GPT-3.5 in \cref{subfig:calibration-targets-gpt35} these are however less pronounced: 
We see that targets are often concentrated on $0$ or $1$, respectively.
Similar spikes like in \cref{subfig:calibration-targets-vicuna} are observable for both models on CoQA in \cref{subfig:calibration-targets-coqa-vicuna,subfig:calibration-targets-coqa-gpt-3.5}.
This trend is also visible when plotting the assigned calibration targets per data point in \cref{fig:clustering}:
While we can spot more transitionary colors between the blue and red extremes in the manifold for \cref{subfig:clustering-vicuna}, the colors tend more to either of the options \cref{subfig:clustering-gpt35}.
These mode trends continue for CoQA in \cref{subfig:clustering-vicuna-coqa} and \cref{subfig:clustering-gpt35-coqa}.

\section{Additional Calibration Results}\label{app:additional-calibration}

\begin{figure*}[htb]
    \centering
    \begin{subfigure}[t]{0.32\textwidth}
        \centering
        \includegraphics[width=\textwidth]{img/apricot/reliability/vicuna-v1.5/test_seq_likelihood.png}
        \caption{Seq.\@ likelihood.}
        \label{subfig:seq-likelihood}
    \end{subfigure}
    \hfill
    \begin{subfigure}[t]{0.32\textwidth}
        \centering
        \includegraphics[width=\textwidth]{img/apricot/reliability/vicuna-v1.5/test_cot_seq_likelihood.png}
        \caption{Seq.\@ like. (CoT).}
        \label{subfig:temperature-scaling}
    \end{subfigure}
    \hfill
    \begin{subfigure}[t]{0.32\textwidth}
        \centering
        \includegraphics[width=\textwidth]{img/apricot/reliability/vicuna-v1.5/test_ts_seq_likelihood.png}
        \caption{Platt scaling.}
        \label{subfig:verbalized-percentage}
    \end{subfigure}
    \hfill
    \begin{subfigure}[t]{0.32\textwidth}
        \centering
        \includegraphics[width=\textwidth]{img/apricot/reliability/vicuna-v1.5/test_ts_cot_seq_likelihood.png}
        \caption{Platt scaling (CoT).}
        \label{subfig:verbalized-percentage}
    \end{subfigure}
    \hfill
    \begin{subfigure}[t]{0.32\textwidth}
        \centering
        \includegraphics[width=\textwidth]{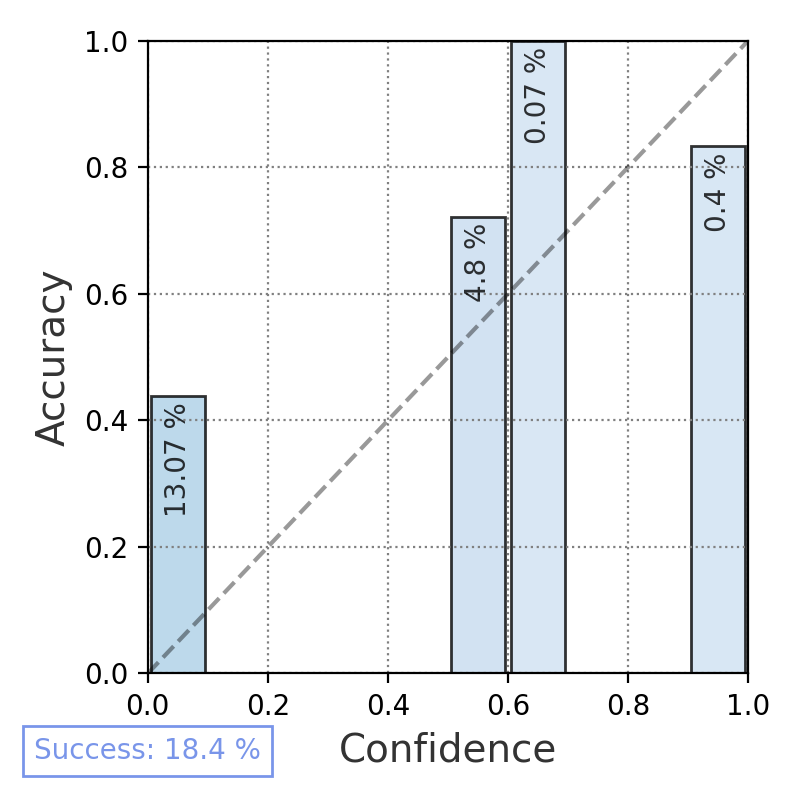}
        \caption{Verbalized Qual.\@}
        \label{subfig:verbalized-qualitative}
    \end{subfigure}
    \begin{subfigure}[t]{0.32\textwidth}
        \centering
        \includegraphics[width=\textwidth]{img/apricot/reliability/vicuna-v1.5/test_verbalized_cot_qual.png}
        \caption{Verb.\@ Qual.\@ (CoT).}
        \label{subfig:seq-likelihood}
    \end{subfigure}
    \hfill
    \begin{subfigure}[t]{0.32\textwidth}
        \centering
        \includegraphics[width=\textwidth]{img/apricot/reliability/vicuna-v1.5/test_verbalized_quant.png}
        \caption{Verbalized $\%$}
        \label{subfig:temperature-scaling}
    \end{subfigure}
    \hfill
    \begin{subfigure}[t]{0.32\textwidth}
        \centering
        \includegraphics[width=\textwidth]{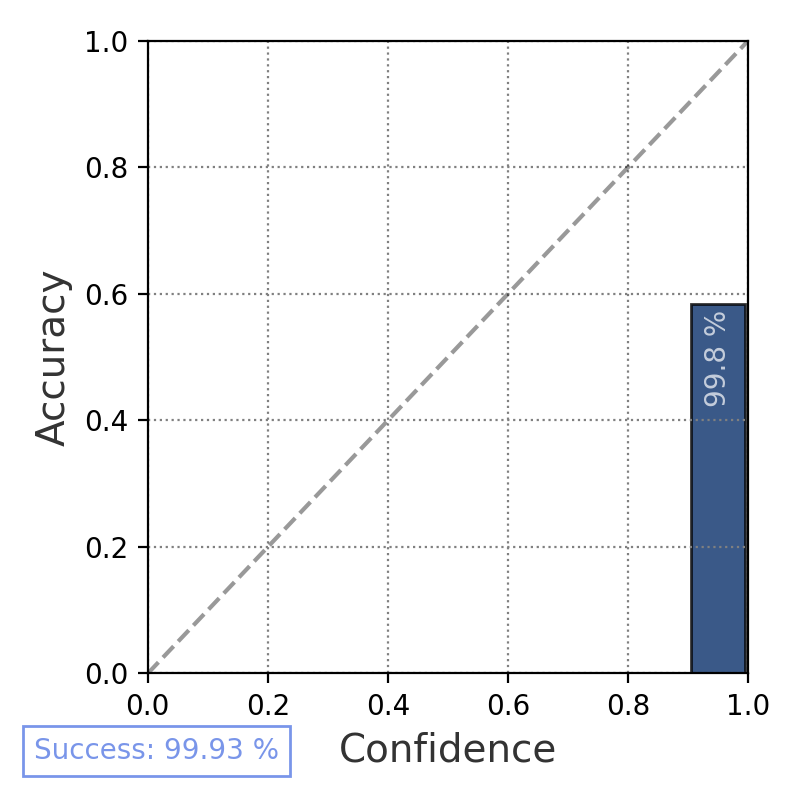}
        \caption{Verb.\@ $\%$\@ (CoT).}
        \label{subfig:verbalized-percentage}
    \end{subfigure}
    \hfill
    \begin{subfigure}[t]{0.32\textwidth}
        \centering
        \includegraphics[width=\textwidth]{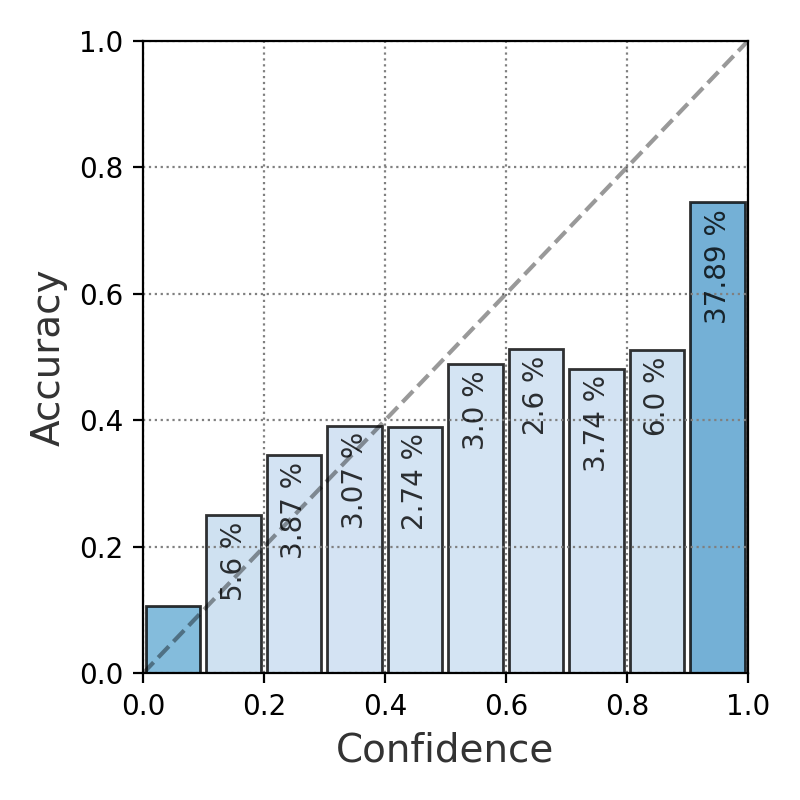}
        \caption{Auxiliary (binary).}
        \label{subfig:verbalized-qualitative}
    \end{subfigure}
    \begin{subfigure}[t]{0.32\textwidth}
        \centering
        \includegraphics[width=\textwidth]{img/apricot/reliability/vicuna-v1.5/16-02-2024_test_answer_question.png}
        \caption{Aux.\@ (clustering).}
        \label{subfig:verbalized-qualitative}
    \end{subfigure}
    \hfill
    \caption[Reliability diagrams for all methods for Vicuna v1.5 7B on TriviaQA.]{Reliability diagrams for all methods using $10$ bins each for Vicuna v1.5 7B on TriviaQA. The color as well as the percentage number within each bar indicate the proportion of total points contained in each bin.}\label{fig:reliabiliy-diagrams-vicuna-trivia-qa-full}
\end{figure*}

\begin{figure*}[htb]
    \centering
    \begin{subfigure}[t]{0.32\textwidth}
        \centering
        \includegraphics[width=\textwidth]{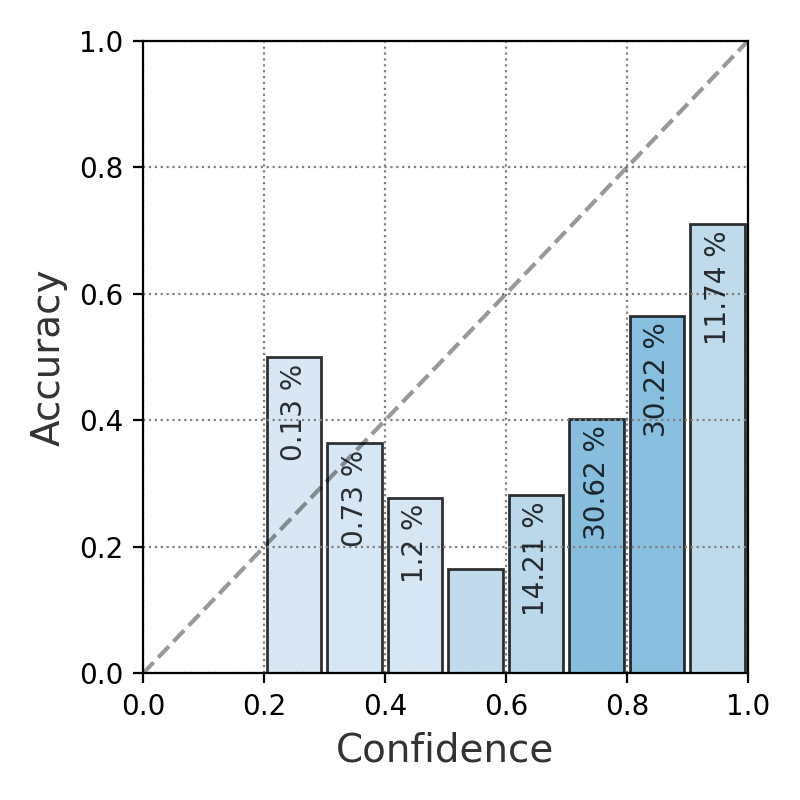}
        \caption{Seq.\@ likelihood.}
        \label{subfig:seq-likelihood}
    \end{subfigure}
    \hfill
    \begin{subfigure}[t]{0.32\textwidth}
        \centering
        \includegraphics[width=\textwidth]{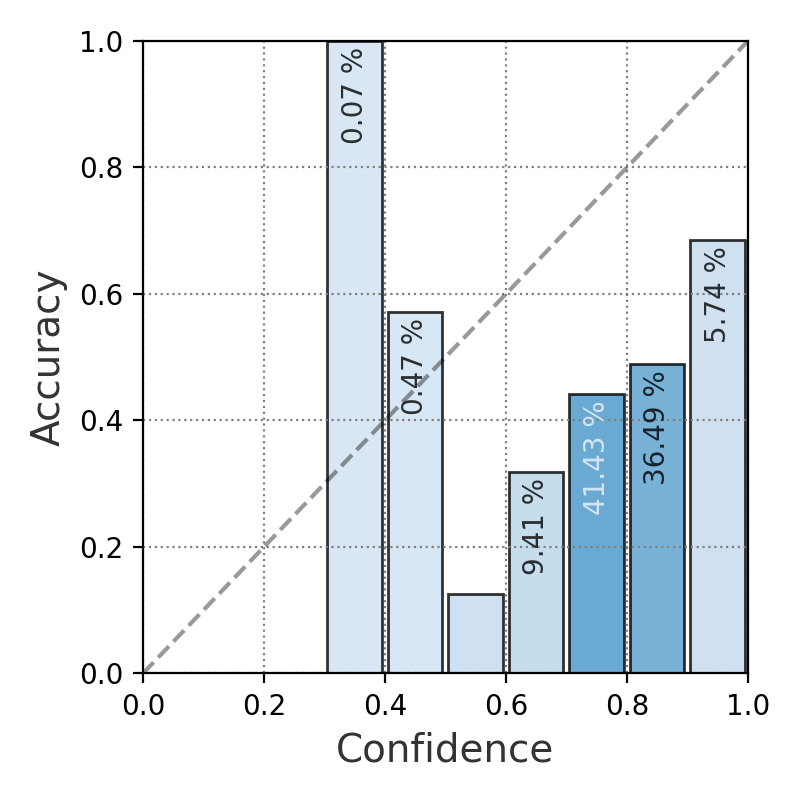}
        \caption{Seq.\@ like. (CoT).}
        \label{subfig:temperature-scaling}
    \end{subfigure}
    \hfill
    \begin{subfigure}[t]{0.32\textwidth}
        \centering
        \includegraphics[width=\textwidth]{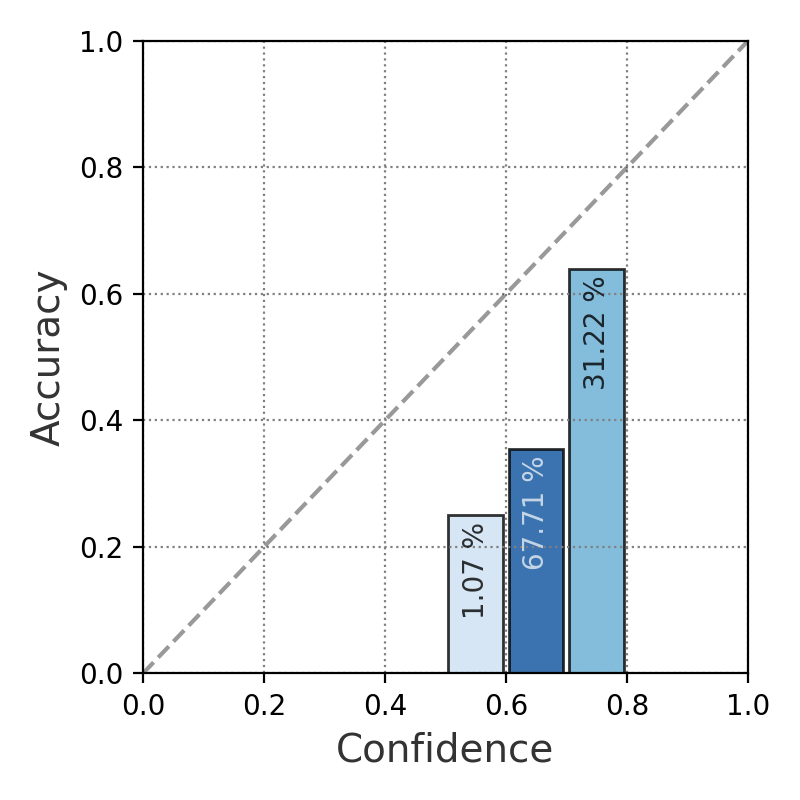}
        \caption{Platt scaling.}
        \label{subfig:verbalized-percentage}
    \end{subfigure}
    \hfill
    \begin{subfigure}[t]{0.32\textwidth}
        \centering
        \includegraphics[width=\textwidth]{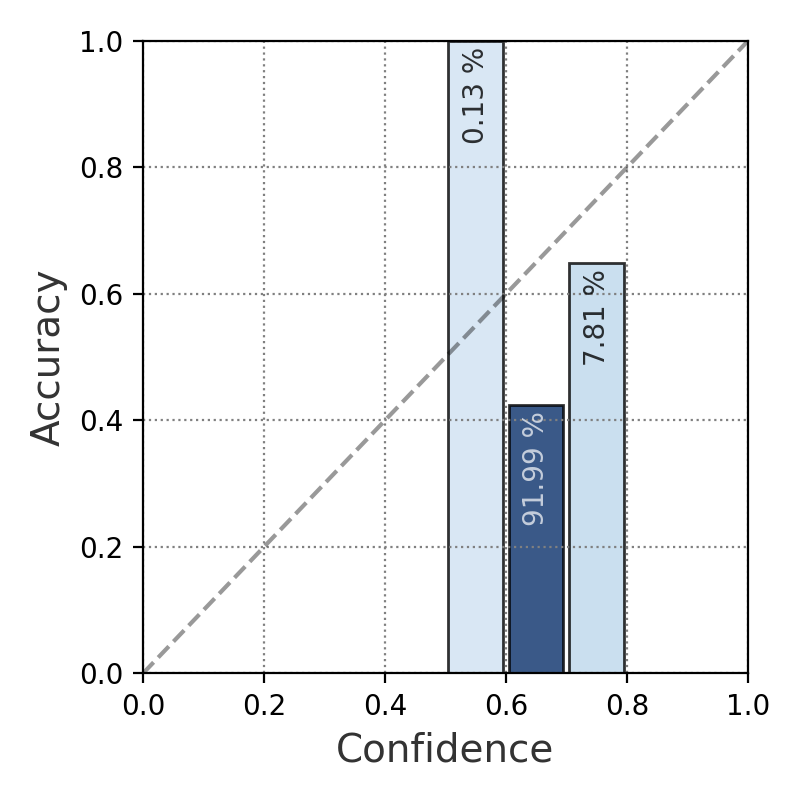}
        \caption{Platt scaling (CoT).}
        \label{subfig:verbalized-percentage}
    \end{subfigure}
    \hfill
    \begin{subfigure}[t]{0.32\textwidth}
        \centering
        \includegraphics[width=\textwidth]{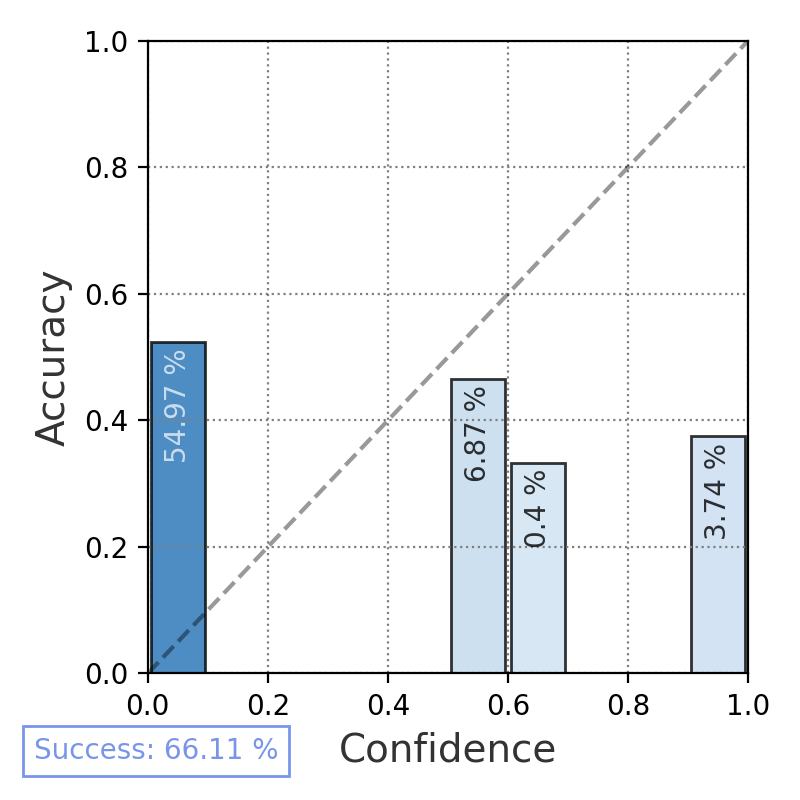}
        \caption{Verbalized Qual.}
        \label{subfig:verbalized-qualitative}
    \end{subfigure}
    \begin{subfigure}[t]{0.32\textwidth}
        \centering
        \includegraphics[width=\textwidth]{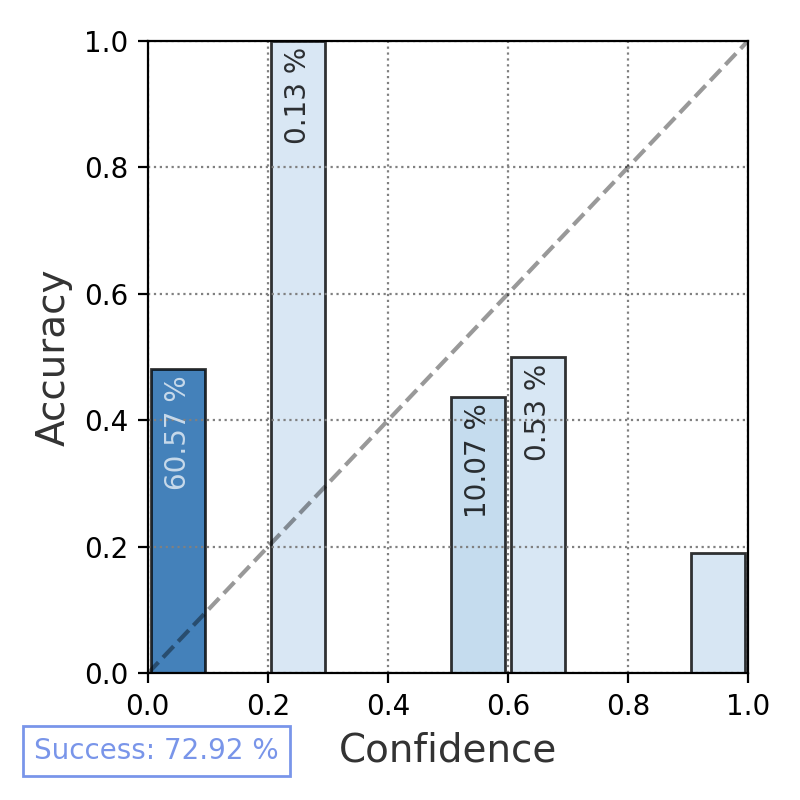}
        \caption{Verb.\@ Qual.\@ (CoT).}
        \label{subfig:seq-likelihood}
    \end{subfigure}
    \hfill
    \begin{subfigure}[t]{0.32\textwidth}
        \centering
        \includegraphics[width=\textwidth]{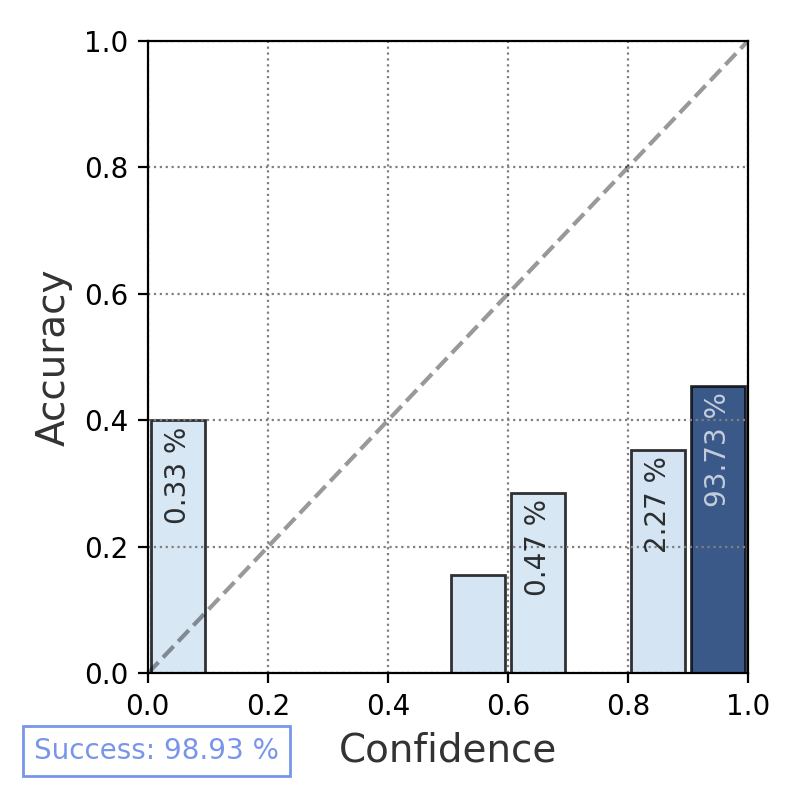}
        \caption{Verbalized $\%$.}
        \label{subfig:temperature-scaling}
    \end{subfigure}
    \hfill
    \begin{subfigure}[t]{0.32\textwidth}
        \centering
        \includegraphics[width=\textwidth]{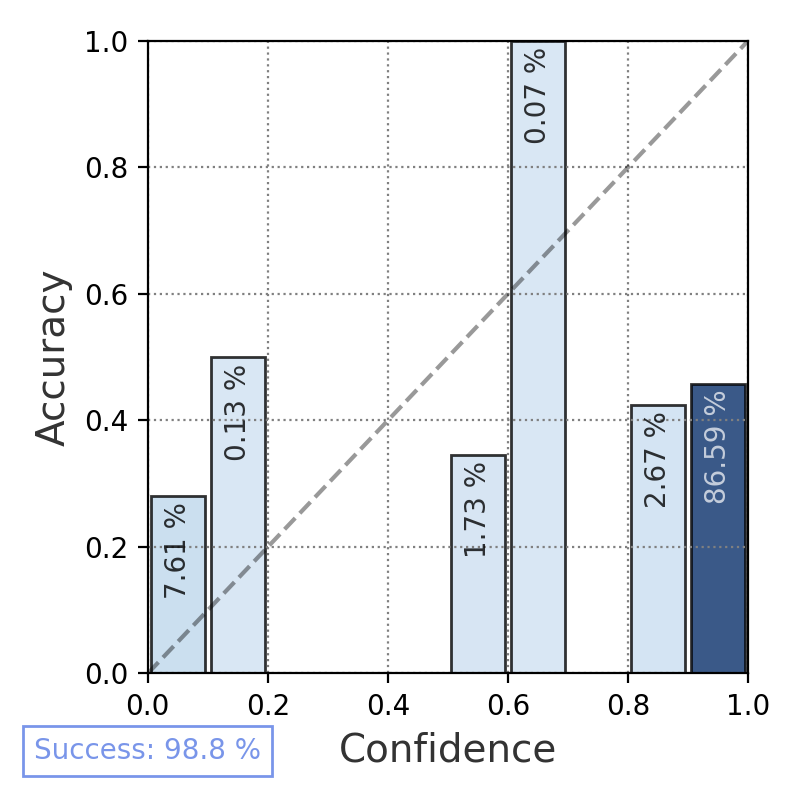}
        \caption{Verb.\@ $\%$ (CoT).}
        \label{subfig:verbalized-percentage}
    \end{subfigure}
    \hfill
    \begin{subfigure}[t]{0.32\textwidth}
        \centering
        \includegraphics[width=\textwidth]{img/apricot/reliability/vicuna-v1.5-coqa/24-02-2024_test_binary_answer_question.png}
        \caption{Auxiliary (binary).}
        \label{subfig:verbalized-qualitative}
    \end{subfigure}
    \begin{subfigure}[t]{0.32\textwidth}
        \centering
        \includegraphics[width=\textwidth]{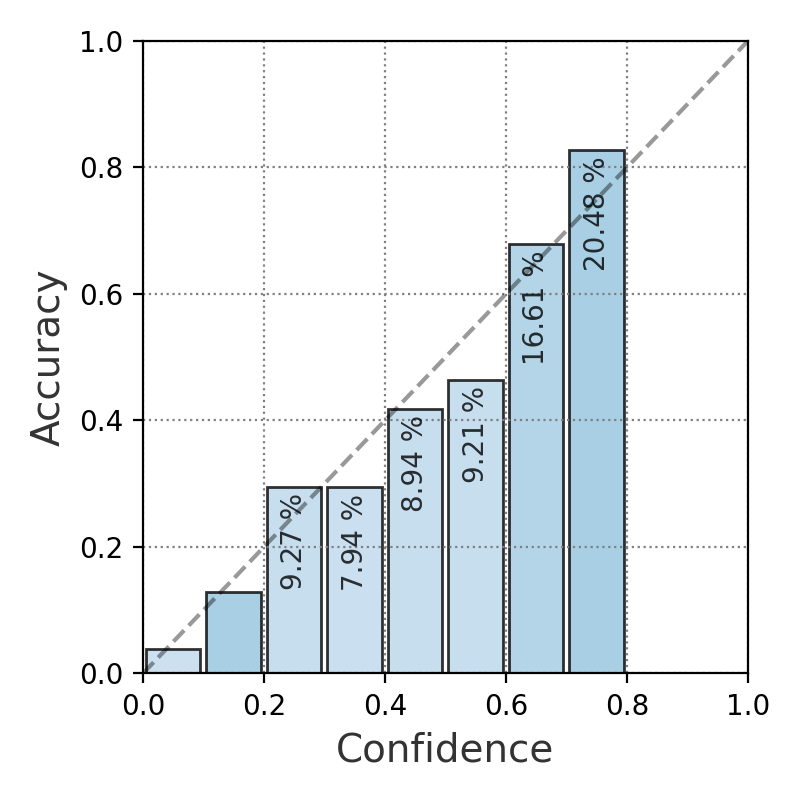}
        \caption{Aux.\@ (clustering).}
        \label{subfig:verbalized-qualitative}
    \end{subfigure}
    \hfill
    \caption[Reliability diagrams for all methods for Vicuna v1.5 7B on CoQA.]{Reliability diagrams for all methods using $10$ bins each for Vicuna v1.5 7B on CoQA. The color as well as the percentage number within each bar indicate the proportion of total points contained in each bin.}\label{fig:reliabiliy-diagrams-vicuna-coqa-full}
\end{figure*}

\begin{figure*}[htb]
    \centering
    \begin{subfigure}[t]{0.32\textwidth}
        \centering
        \includegraphics[width=\textwidth]{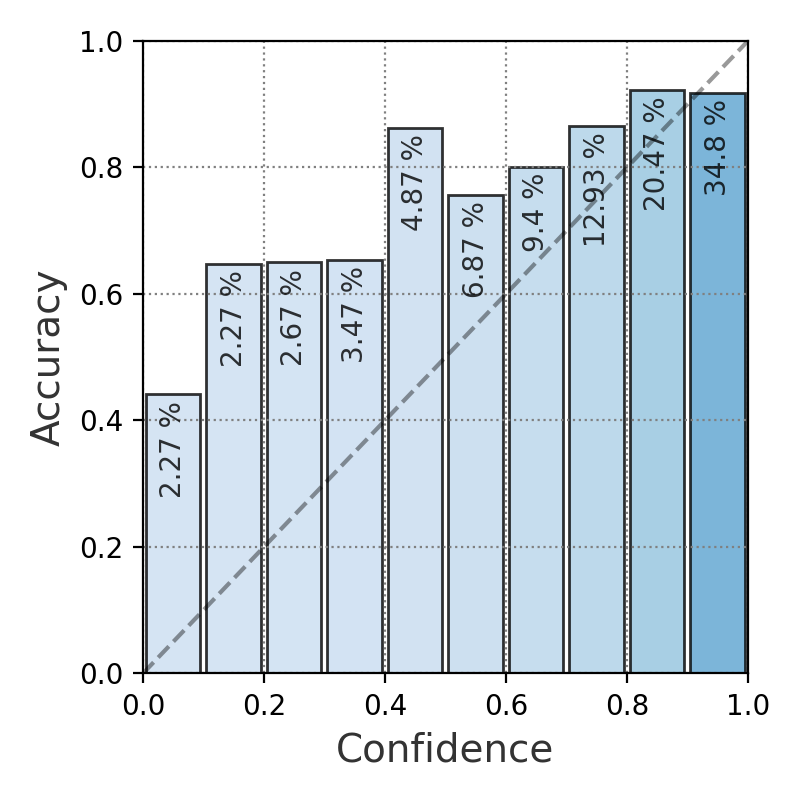}
        \caption{Seq.\@ likelihood.}
        \label{subfig:seq-likelihood}
    \end{subfigure}
    \hfill
    \begin{subfigure}[t]{0.32\textwidth}
        \centering
        \includegraphics[width=\textwidth]{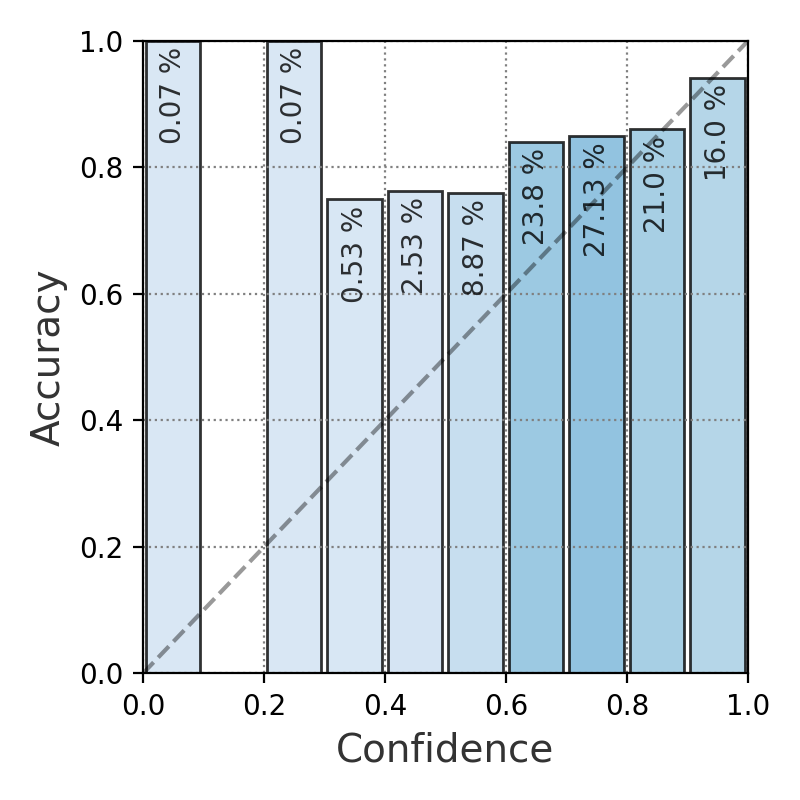}
        \caption{Seq.\@ like. (CoT).}
        \label{subfig:temperature-scaling}
    \end{subfigure}
    \hfill
    \begin{subfigure}[t]{0.32\textwidth}
        \centering
        \includegraphics[width=\textwidth]{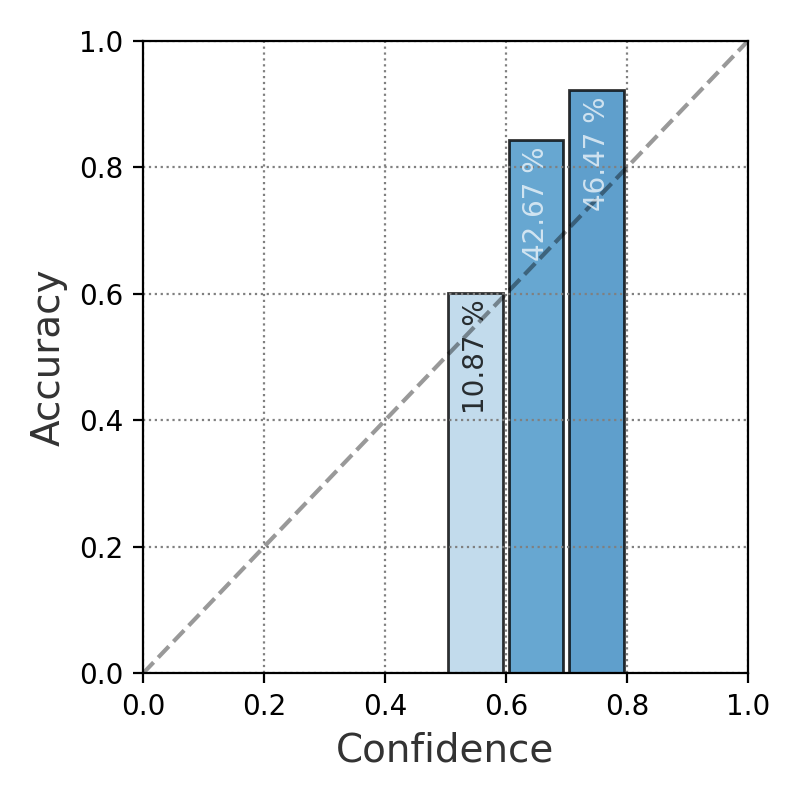}
        \caption{Platt scaling.}
        \label{subfig:verbalized-percentage}
    \end{subfigure}
    \hfill
    \begin{subfigure}[t]{0.32\textwidth}
        \centering
        \includegraphics[width=\textwidth]{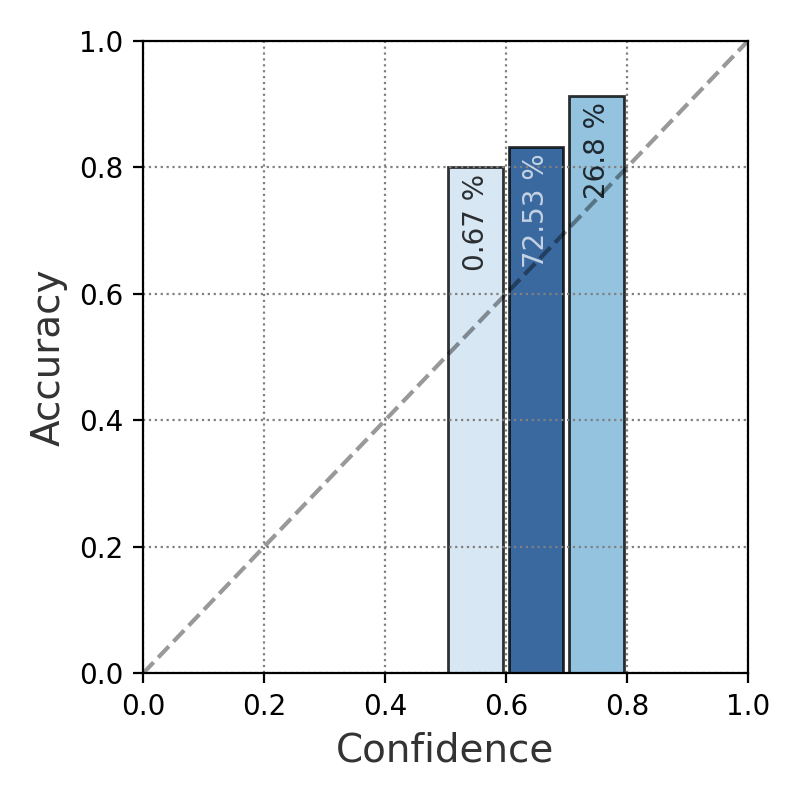}
        \caption{Platt scaling (CoT).}
        \label{subfig:verbalized-percentage}
    \end{subfigure}
    \hfill
    \begin{subfigure}[t]{0.32\textwidth}
        \centering
        \includegraphics[width=\textwidth]{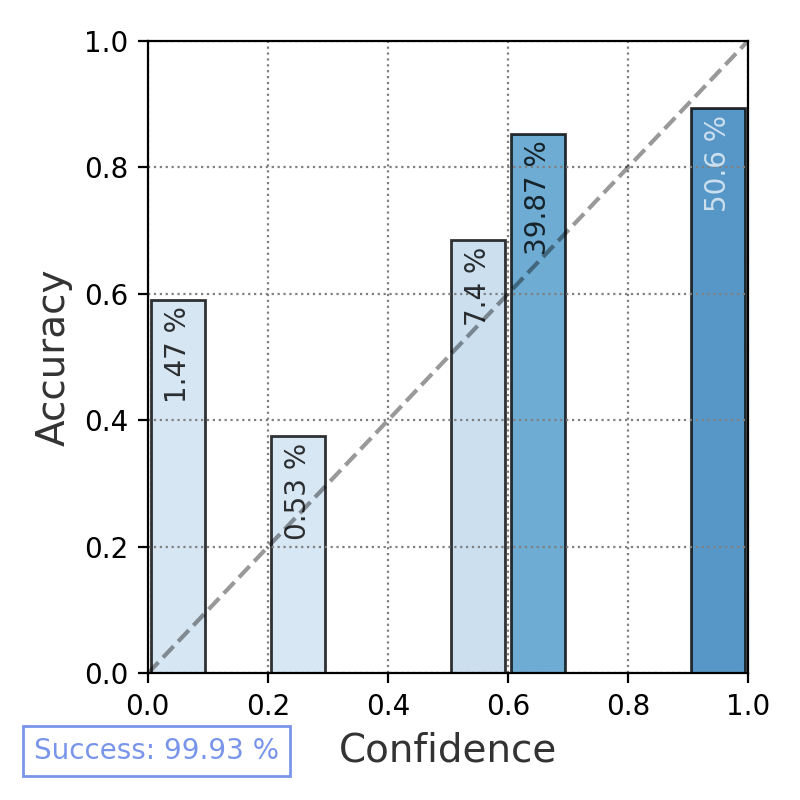}
        \caption{Verbalized Qual.}
        \label{subfig:verbalized-qualitative}
    \end{subfigure}
    \begin{subfigure}[t]{0.32\textwidth}
        \centering
        \includegraphics[width=\textwidth]{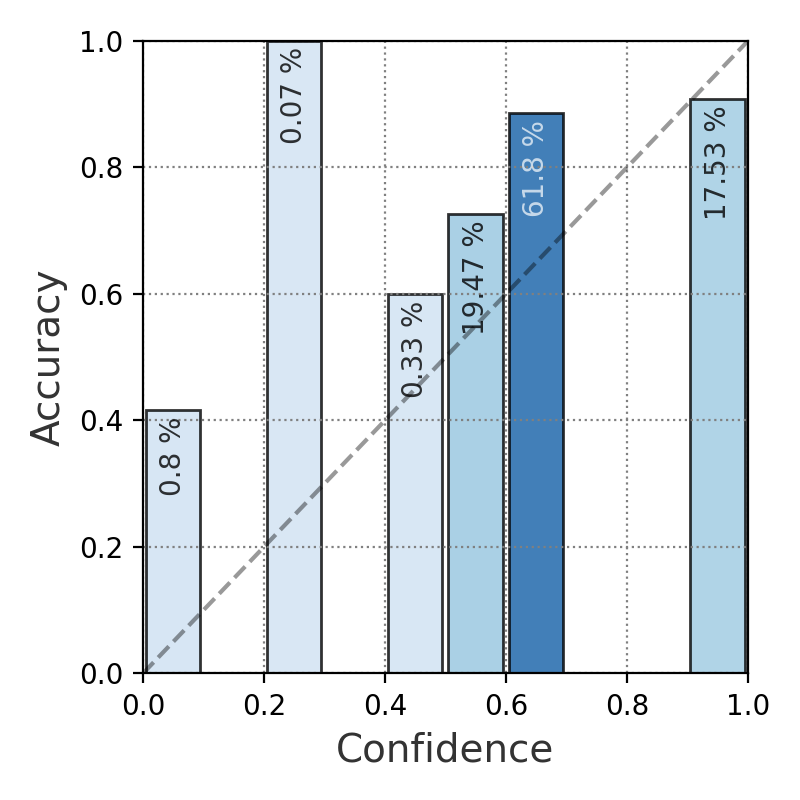}
        \caption{Verb.\@ Qual.\@ (CoT).}
        \label{subfig:seq-likelihood}
    \end{subfigure}
    \hfill
    \begin{subfigure}[t]{0.32\textwidth}
        \centering
        \includegraphics[width=\textwidth]{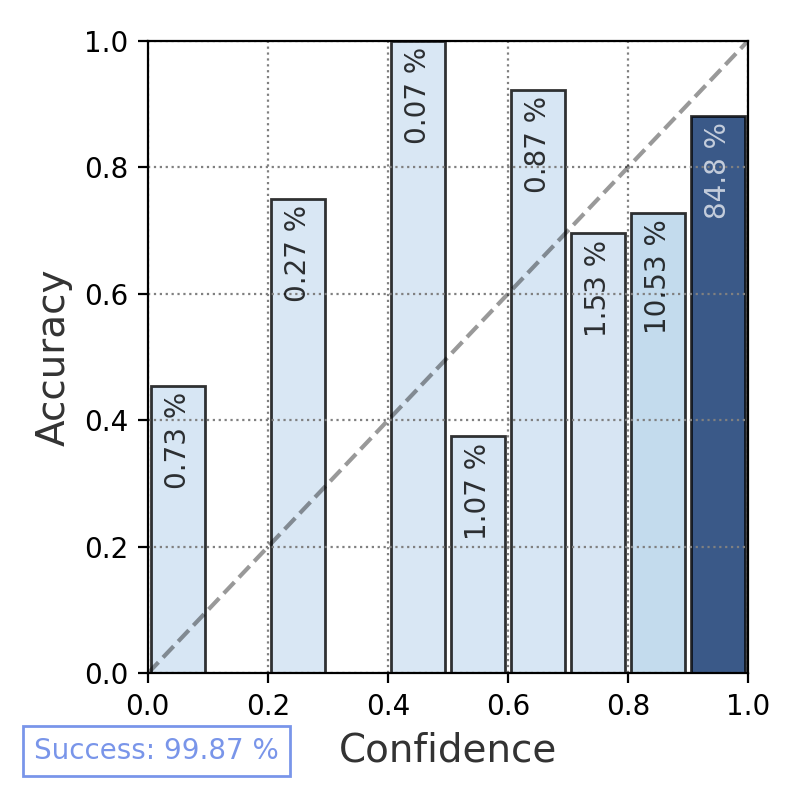}
        \caption{Verbalized $\%$.}
        \label{subfig:temperature-scaling}
    \end{subfigure}
    \hfill
    \begin{subfigure}[t]{0.32\textwidth}
        \centering
        \includegraphics[width=\textwidth]{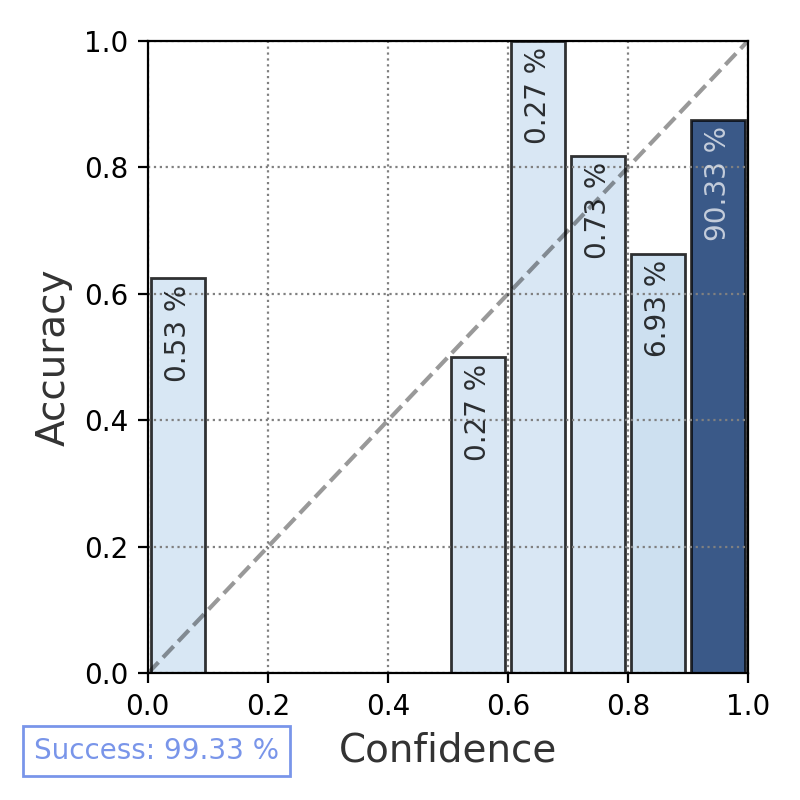}
        \caption{Verb.\@ $\%$ (CoT).}
        \label{subfig:verbalized-percentage}
    \end{subfigure}
    \hfill
    \begin{subfigure}[t]{0.32\textwidth}
        \centering
        \includegraphics[width=\textwidth]{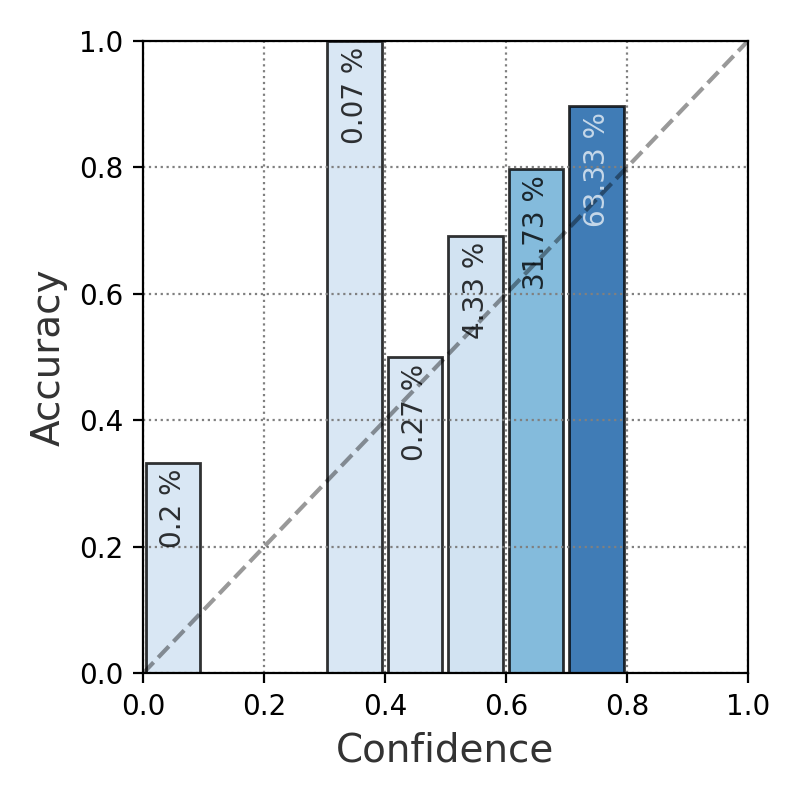}
        \caption{Auxiliary (binary).}
        \label{subfig:verbalized-qualitative}
    \end{subfigure}
    \begin{subfigure}[t]{0.32\textwidth}
        \centering
        \includegraphics[width=\textwidth]{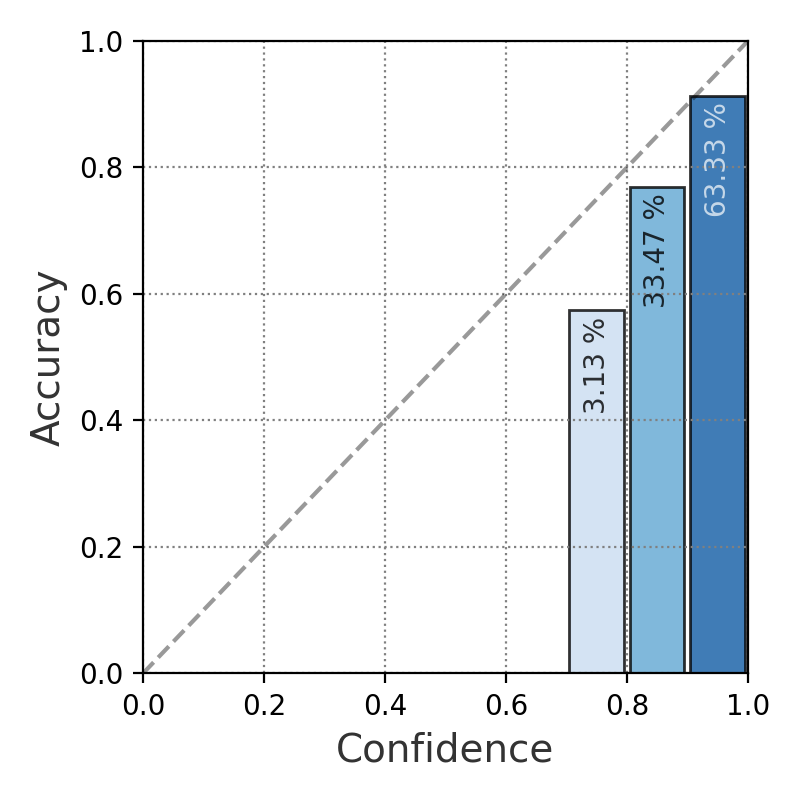}
        \caption{Aux.\@ (clustering).}
        \label{subfig:verbalized-qualitative}
    \end{subfigure}
    \hfill
    \caption[Reliability diagrams for all methods for GPT-3.5 on TriviaQA.]{Reliability diagrams for all methods using $10$ bins each for GPT-3.5 on TriviaQA. The color as well as the percentage number within each bar indicate the proportion of total points contained in each bin.}\label{fig:reliabiliy-diagrams-gpt35-trivia-qa-full}
\end{figure*}

\begin{figure*}[htb]
    \centering
    \begin{subfigure}[t]{0.32\textwidth}
        \centering
        \includegraphics[width=\textwidth]{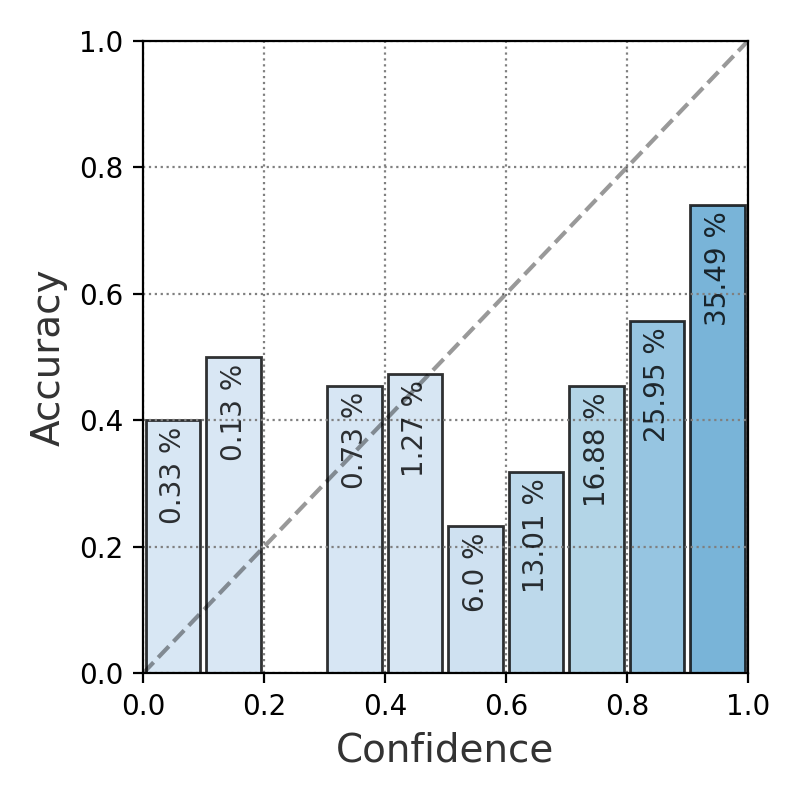}
        \caption{Seq.\@ likelihood.}
        \label{subfig:seq-likelihood}
    \end{subfigure}
    \hfill
    \begin{subfigure}[t]{0.32\textwidth}
        \centering
        \includegraphics[width=\textwidth]{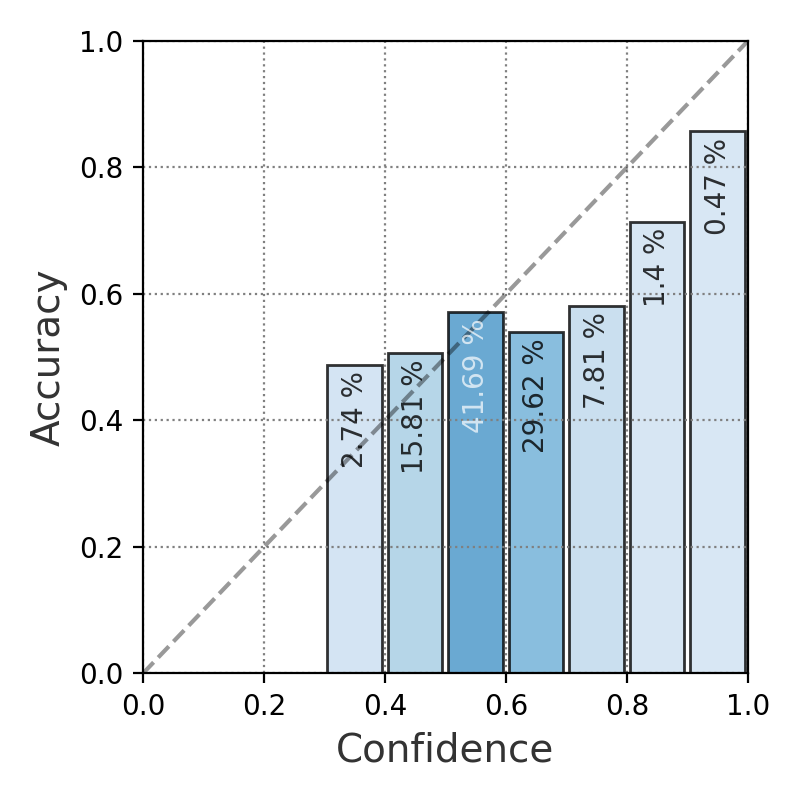}
        \caption{Seq.\@ like. (CoT).}
        \label{subfig:temperature-scaling}
    \end{subfigure}
    \hfill
    \begin{subfigure}[t]{0.32\textwidth}
        \centering
        \includegraphics[width=\textwidth]{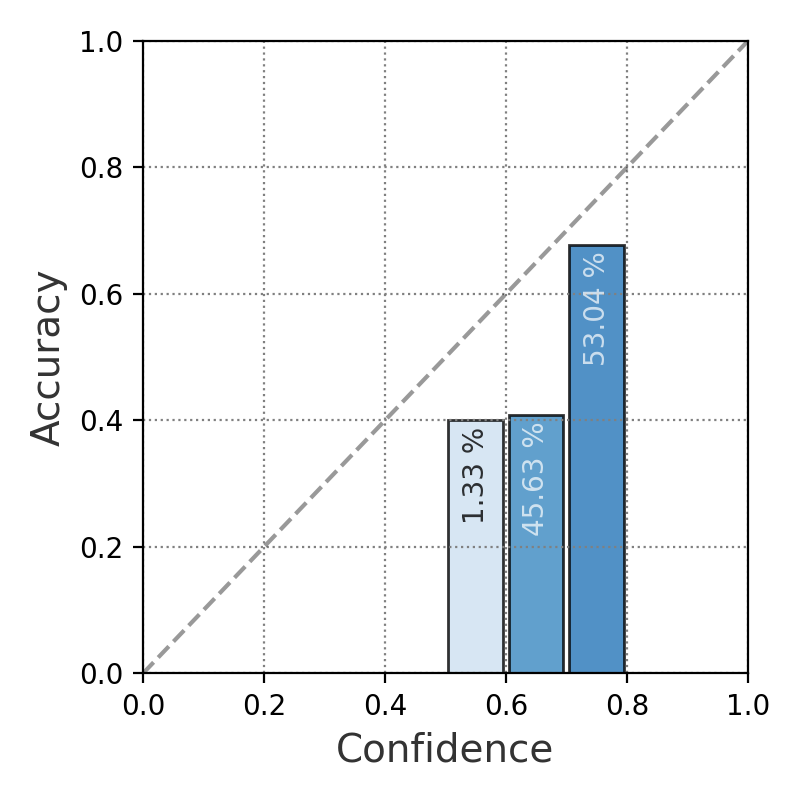}
        \caption{Platt scaling.}
        \label{subfig:verbalized-percentage}
    \end{subfigure}
    \hfill
    \begin{subfigure}[t]{0.32\textwidth}
        \centering
        \includegraphics[width=\textwidth]{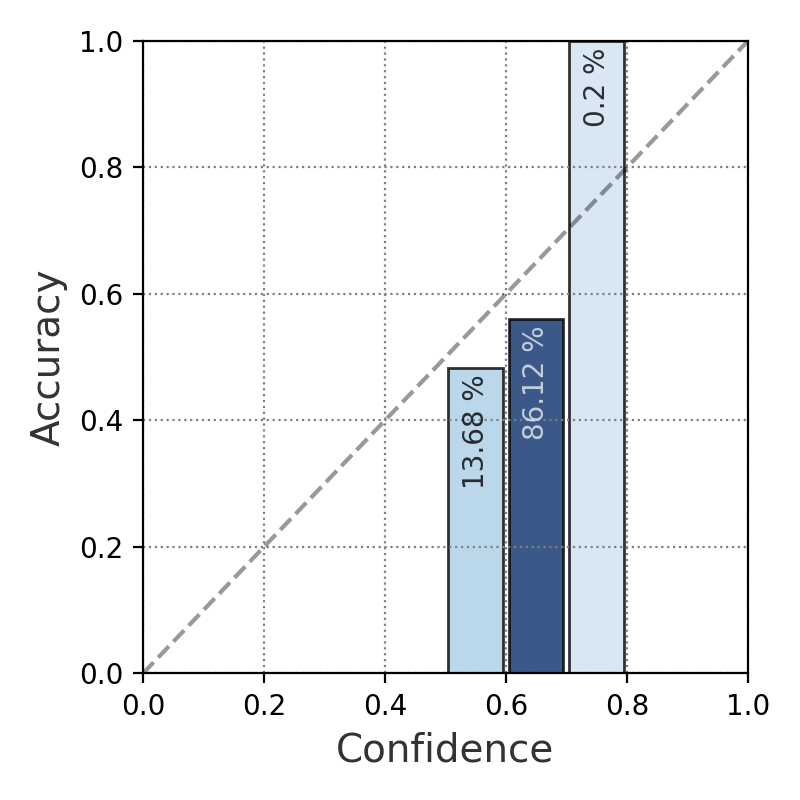}
        \caption{Platt scaling (CoT).}
        \label{subfig:verbalized-percentage}
    \end{subfigure}
    \hfill
    \begin{subfigure}[t]{0.32\textwidth}
        \centering
        \includegraphics[width=\textwidth]{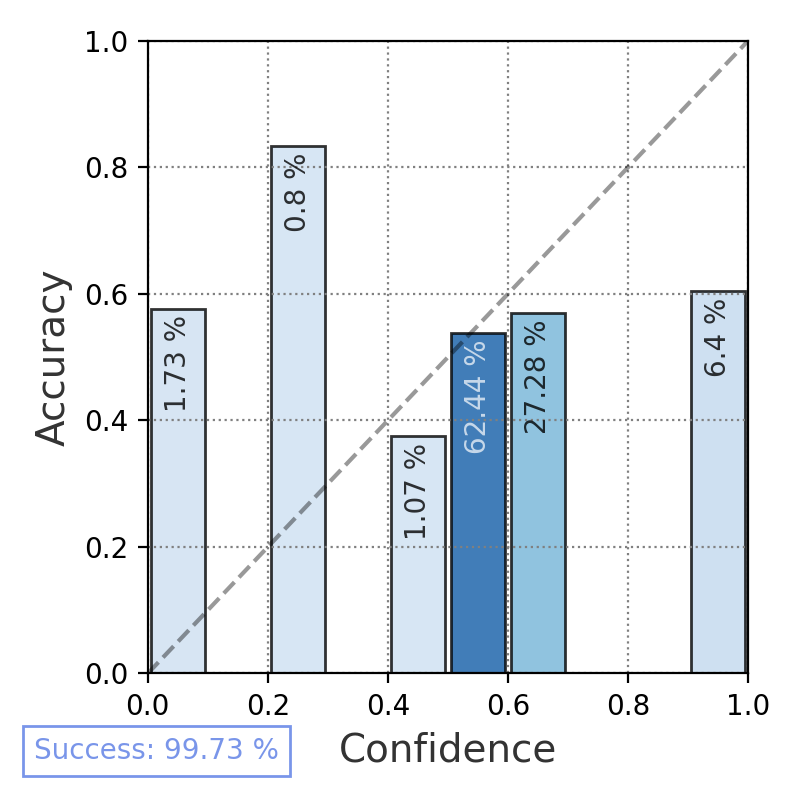}
        \caption{Verbalized Qual.}
        \label{subfig:verbalized-qualitative}
    \end{subfigure}
    \begin{subfigure}[t]{0.32\textwidth}
        \centering
        \includegraphics[width=\textwidth]{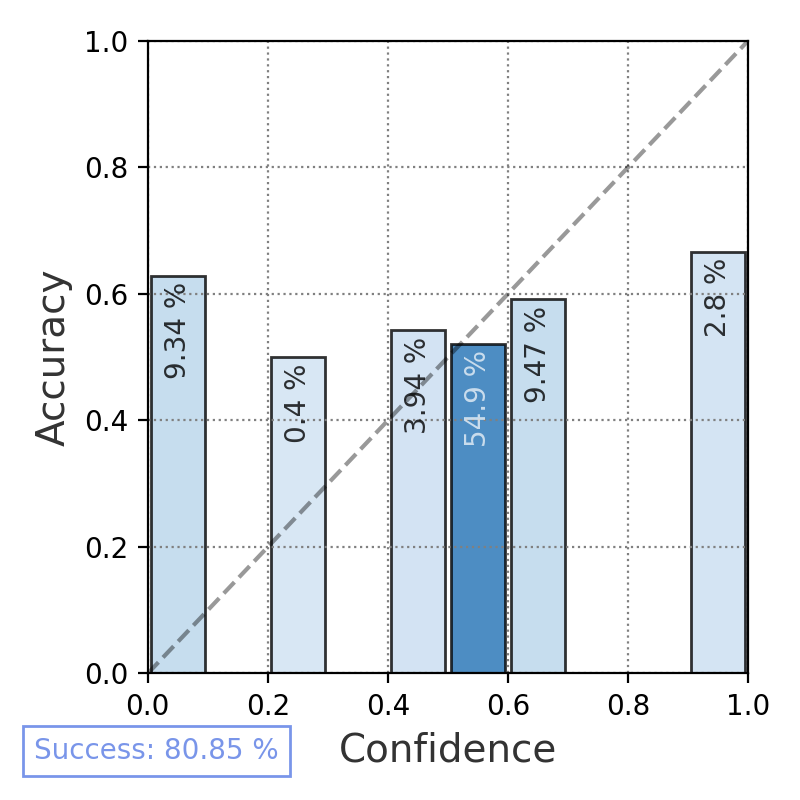}
        \caption{Verb. Qual. (CoT).}
        \label{subfig:seq-likelihood}
    \end{subfigure}
    \hfill
    \begin{subfigure}[t]{0.32\textwidth}
        \centering
        \includegraphics[width=\textwidth]{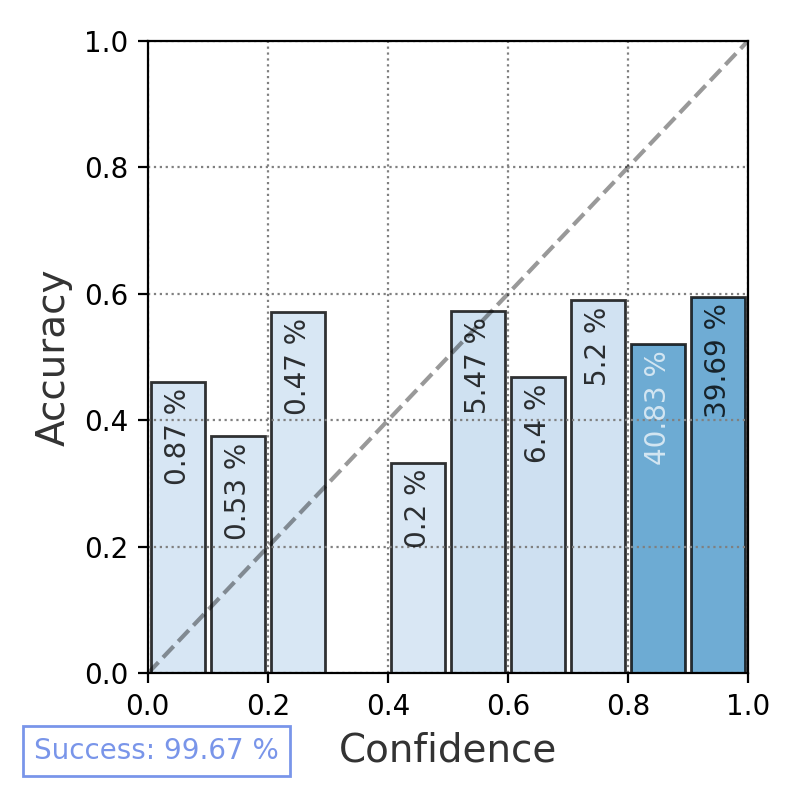}
        \caption{Verbalized $\%$.}
        \label{subfig:temperature-scaling}
    \end{subfigure}
    \hfill
    \begin{subfigure}[t]{0.32\textwidth}
        \centering
        \includegraphics[width=\textwidth]{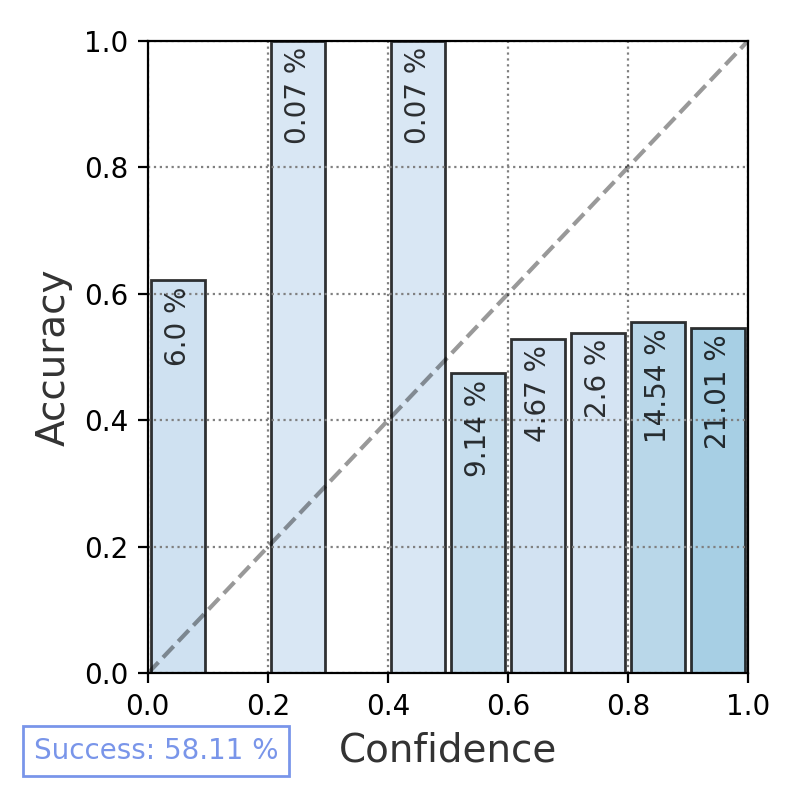}
        \caption{Verb.\@ $\%$ (CoT).}
        \label{subfig:verbalized-percentage}
    \end{subfigure}
    \hfill
    \begin{subfigure}[t]{0.32\textwidth}
        \centering
        \includegraphics[width=\textwidth]{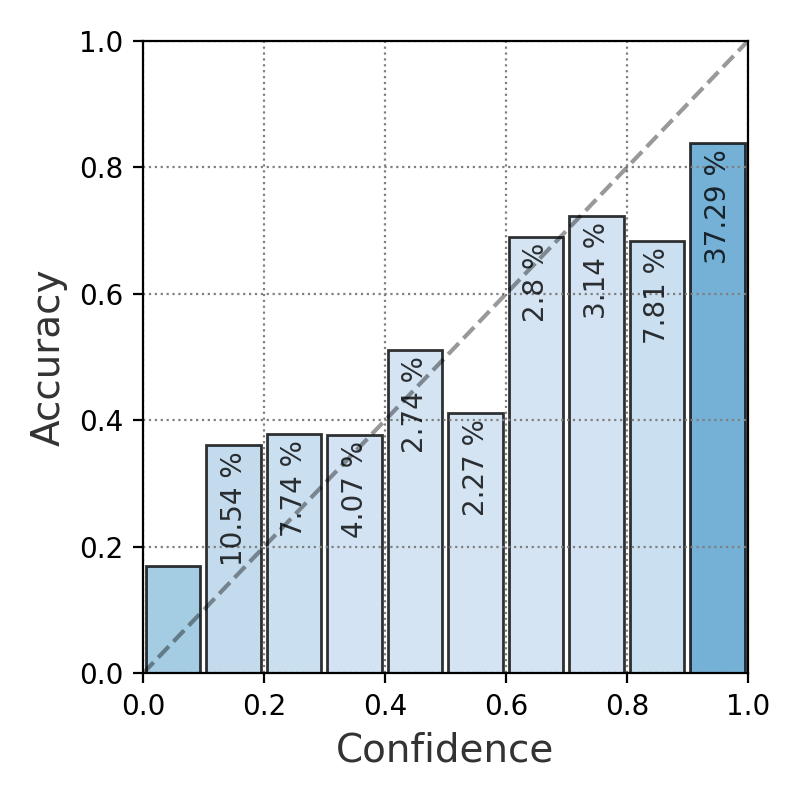}
        \caption{Auxiliary (binary).}
        \label{subfig:verbalized-qualitative}
    \end{subfigure}
    \begin{subfigure}[t]{0.32\textwidth}
        \centering
        \includegraphics[width=\textwidth]{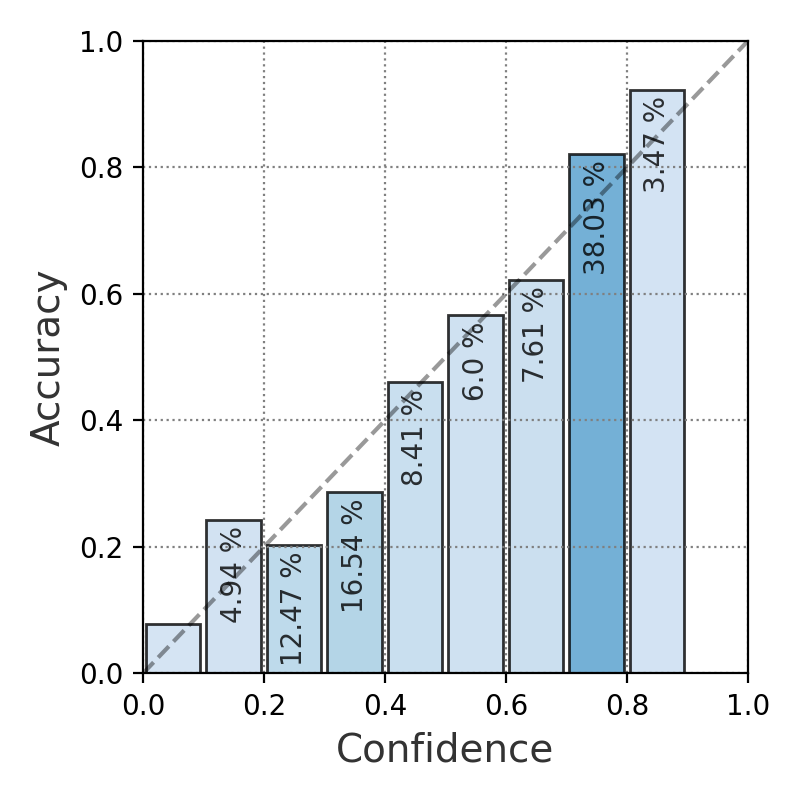}
        \caption{Aux.\@ (clustering).}
        \label{subfig:verbalized-qualitative}
    \end{subfigure}
    \hfill
    \caption[Reliability diagrams for all methods for GPT-3.5 on CoQA.]{Reliability diagrams for all methods using $10$ bins each for GPT-3.5 on CoQA. The color as well as the percentage number within each bar indicate the proportion of total points contained in each bin.}\label{fig:reliabiliy-diagrams-gpt35-coqa-full}
\end{figure*}

\paragraph{Additional reliability plots.}
We show the all available reliability diagrams for the experiments in \cref{sec:calibration-experiments} for Vicuna-v1.5 for TriviaQA in \cref{fig:reliabiliy-diagrams-vicuna-trivia-qa} and CoQA in \cref{fig:reliabiliy-diagrams-vicuna-coqa-full}, as well as the corresponding plots for GPT-3.5 in \cref{fig:reliabiliy-diagrams-gpt35-trivia-qa-full,fig:reliabiliy-diagrams-gpt35-coqa-full}.
%We can summarize some general trends:
Sequence likelihood\index{Likelihood!Sequence} \emph{can} be well-calibrated already, but this fact depends strongly on the dataset in question.
And while our version of Platt scaling\index{Platt scaling} can improve results, it also narrows the range of confidence values to a narrow window.
%Sequence likelihood already achieves relatively good calibration, but tends to suffer on the out-of-distribution test set.
%The Platt scaling baseline always seems to rescale sequence likelihoods to fall into a narrow band of confidence values.
Verbalized uncertainty\index{Uncertainty!Verbalized} in both of variants also is not able to produce a wide variety of responses, even though this effect is slightly less pronounced for GPT-3.5.
%This could be explained by previous observations by \citet{zhou2023navigating} that certain mentions of percentage values are over-represented in model training data.
%The binary version of the auxiliary binary also seems to concentrate on very specific confidence values in GPT-3.5, which could be due to the imbalance of correctly and incorrectly answered questions in its finetuning dataset. 
The auxiliary model is able to predict a wider array of confidence values in all settings, with the clustering variant achieving better calibration overall.

\chapter{Reproducibility Appendix}\label{app:reproducibility-appendix}

\epigraph{

    ``\emph{grad student descent: (machine learning, humorous) The process of choosing hyperparameters manually and in an ad-hoc manner, typical of work assigned to a graduate student.}''
}
{---\href{https://en.wiktionary.org/wiki/graduate_student_descent}{Wiktionary} definition.}

\begin{table}[htb]
    \centering
    \resizebox{0.4\textwidth}{!}{%
    \begin{tabular}{rl}
        \toprule
        Thesis & Appendix \\
        \toprule
        \cref{sec:aso} & \cref{app:implementation-details} \\
        \cref{sec:evidential-neural-networks} & \cref{app:iris-code-details} \\
        \cref{sec:exploring-data-creation} & \cref{app:predictive-uncertainty-pre-processing} \\
        \cref{sec:exploring-methodology} & \cref{app:exploring-predictive-uncertainty-implementation-details} \\
        \cref{sec:predictive-uncertainty-experiments} & \cref{app:convergence-clinc-plus} \\
        \cref{sec:retrieval-quality} & \cref{app:temperature-search} \\
        % \cref{sec:calibration-experiments} & \cref{app:prompting} \\
        \bottomrule
    \end{tabular}%
    }
    \caption[Correspondences between reproducibility appendix sections and thesis chapters.]{Correspondences between sections of the reproducibility appendix and thesis chapters.}\label{tab:reproducibility-appendix-correspondence}
\end{table}

This appendix contains additional information for reproducibility\index{Reproducibility} purposes, according to the guidelines by \citet{ulmer-etal-2022-experimental}.
In \cref{app:open-source}, a number of open-source software projects that were used in the making of this thesis are listed.
\cref{app:environmental-impact} discusses the compute hardware and environmental impact\index{Environmental impact} of the conducted experiments and other aspects.
In \cref{tab:reproducibility-appendix-correspondence}, we give an overview over the correspondence between thesis chapters and sections in this appendix.

\section{Open Source Software}\label{app:open-source}

\begin{table}[htb]
    \centering
    \resizebox{0.985\textwidth}{!}{%
    \begin{tabular}{lr}
        \toprule
        Repository & Chapters \\
        \toprule
        \href{https://github.com/Kaleidophon/phd-thesis}{\texttt{github.com/Kaleidophon/phd-thesis}} & \cref{sec:context,sec:bayesian-perspective,sec:other-approaches,sec:aso} \\[0.25cm]
        \href{https://github.com/Kaleidophon/evidential-deep-learning-survey}{\texttt{github.com/Kaleidophon/evidential-deep-learning-survey}} & \cref{sec:evidential-neural-networks} \\[0.25cm]
        \href{https://github.com/Kaleidophon/awesome-experimental-standards-deep-learning}{\texttt{github.com/Kaleidophon/awesome-experimental-standards-deep-learning}} & \cref{ch:uncertainty-experimental-design} \\[0.25cm]
        \href{https://github.com/Kaleidophon/deep-significance}{\texttt{github.com/Kaleidophon/evidential-deep-learning-survey}} & \cref{sec:aso} \\[0.25cm]
        \href{https://github.com/Kaleidophon/know-your-limits}{\texttt{github.com/Kaleidophon/know-your-limits}} & \cref{sec:uq-classification-pitfalls} \\[0.25cm]
        \multirow{2}{*}{\makecell[tl]{\href{https://github.com/Kaleidophon/nlp-low-resource-uncertainty}{\texttt{github.com/Kaleidophon/nlp-low-resource-uncertainty}} \\ \href{https://github.com/Kaleidophon/nlp-uncertainty-zoo}{\texttt{github.com/Kaleidophon/nlp-uncertainty-zoo}} }} & \cref{sec:benchmarking-nlp-uncertainty} \\
         & \\[0.25cm]
        \href{https://github.com/Kaleidophon/non-exchangeable-conformal-language-generation}{\texttt{github.com/Kaleidophon/non-exchangeable-conformal-language-generation}} & \cref{ch:uncertainty-generation} \\[0.25cm]
         \href{https://github.com/parameterlab/apricot}{\texttt{github.com/parameterlab/apricot}} & \cref{ch:uncertainty-llms} \\
        \bottomrule
    \end{tabular}%
    }
    \caption[List of open-source repositories for the contents of this thesis.]{List of open-source repositories for the contents of this thesis.}\label{tab:open-source-thesis}
\end{table}

This thesis would not have been possible without the usage of open-source tools and software.
Deep learning\index{Deep learning} models where implemented with \texttt{NumPy} \citep{harris2020array}, \texttt{SciPy} \citep{2020SciPy-NMeth}, \texttt{scikit-learn} \citep{pedregosa2011scikit}, \texttt{einops} \citep{rogozhnikov2021einops}, \texttt{PyTorch} \citep{paszke2019pytorch} and the \texttt{transformers} library \citep{wolf2020transformers}.
Experimental tracking and hyperparameter search was facilitated via Weights \& Biases \citep{wandb}, and tracking carbon emissions with \texttt{codecarbon} \citep{codecarbon,lacoste2019quantifying, lottick2019nergy}.
The code for the plots and experiments in this thesis is itself available open-source, and corresponding online repositories are listed in \cref{tab:open-source-thesis}.

\section{Environmental Impact}\label{app:environmental-impact}
\index{Environmental impact}
Here, we discuss the environmental impact of the experiments in the different chapters.
In all cases, carbon emissions have been estimated using \texttt{codecarbon} \citep{codecarbon, lacoste2019quantifying, lottick2019nergy}, although it should be noted that since the time of writing, more advanced tools like \texttt{LLMCarbon} \citep{faiz2023llmcarbon} have been developed to more accurately estimate the carbon footprint of language model training, specifically.\\

For \cref{ch:uncertainty-classification}, the carbon efficiency was estimated to be $0.61$ kgCO$_2$eq / kWh. 
$735$ hours of computation were performed on a Tesla V100 GPU. 
This includes hyperparameter search, failed runs, debugging, and discarded runs. 
As a rough upper bound, we estimate the compute time for a single replication of all experiments in \cref{ch:uncertainty-classification} to take around 73 hours.\footnote{
    Note that this number could be reduced further by using better hardware acceleration, larger batch sizes, and slightly reducing the training duration for some models. 
    Most importantly, this number also includes compute used for hyperparameter search.
} 
Total emissions were estimated to be $52.45$ kgCO$_2$eq.\\

For \cref{ch:uncertainty-generation}, the carbon efficiency was estimated to be $0.12$ kgCO$_2$eq / kWh. 
159.5 hours of computation were performed on a NVIDIA RTX A6000.
Total emissions are estimated to be $6.99$ kgCo2eq.
All of these values are upper bound including debugging as well as failed or redundant runs, and thus any replication of results will likely be shorter and incur fewer carbon emissions.\\

For \cref{ch:uncertainty-llms}, all experiments were run on a single V100 NVIDIA GPU. 
We estimate finetuning the auxiliary calibrator to amount to $0.05$ kgCO$_2$eq of emissions, with an estimated carbon efficiency of $0.46$ kgCO$_2$eq / kWH.
Therefore, we estimate total emissions of around $1$ kgCO$_2$eq to replicate all the experiments in this chapter.\\

\paragraph{Carbon Offsetting.} 
Carbon offsetting is a controversial topic \citep{watt2021fantasy, campbell2021offsetting, baras2023carbon}, and avoiding emission should always be the preferred option compared to post-hoc offsetting.
Nevertheless, the author believes in mitigating the impact of their emissions as best as possible.
The tracked carbon emissions from all the chapter in this thesis are $60.44$ kgCO$_2$eq.
An additional $20 \%$ is added to this number to account for variation in tracking, untracked debug runs or failed experiments, amounting to $72.53$ kgCO$_2$eq.
Furthermore, over the course of almost four years, the author attended a number of conferences during their PhD program, and conducted industrial internships as well as a research stay.
The travels related to these activities produced an estimated total of $12088$ kgCO$_2$eq in emissions.
Direct air capture by climeworks \citep{climeworks} was used to offset the emissions from the experiments, and carbon credits stemming from wind energy projects in Thailand and India were purchased through the Gold Standard Marketplace \citep{goldstandard} for travel-related emissions.

\section{ASO Test Implementation Details}\label{app:implementation-details}
\index{Stochastic order!Almost}

\begin{algorithm}
    \caption{Almost Stochastic Order (ASO) Significance Test}\label{alg:aso}
    \begin{algorithmic}
        \Require{Sets of observations $\mathbb{S}_\mathbb{A}$ and $\mathbb{S}_\mathbb{B}$, integration interval $\Delta_t$, number of bootstrap iterations $B$, desired confidence level $1 - \alpha$.}

        \State{$\varepsilon_{\mathcal{W}_2}(F_n, G_m)$ = \texttt{compute\_violation\_ratio}($\mathbb{S}_\mathbb{A}$, $\mathbb{S}_\mathbb{A}$, $\Delta_t$)}\\
        \LineComment{Bootstrapping}
        \For{$i \in\ 0, \ldots,B$}
            \State{$\mathbb{S}_\mathbb{A}^* =$ \texttt{bootstrap\_sample}($\mathbb{S}_\mathbb{A}$)}\\
            \State{$\mathbb{S}_\mathbb{B}^* =$ \texttt{bootstrap\_sample}($\mathbb{S}_\mathbb{B}$)}\\
            \LineComment{Store value below in list}
            \State{$\varepsilon^*_{\mathcal{W}_2}(F_n, G_m)$ = \texttt{compute\_violation\_ratio}($\mathbb{S}_\mathbb{A}^*$, $\mathbb{S}_\mathbb{A}^*$, $\Delta_t$)}\\
        \EndFor

        \LineComment{Compute value below based on variance of all the values in list}
        \State{$\hat{\sigma}_{n, m}^2 = \text{Var}\bigg[\sqrt{\frac{mn}{n + m}}\big(\varepsilon_{W_2}(F_n^*, G_m^*) - \varepsilon_{W_2}(F_n, G_m)\big) \bigg]$}\\
        \State{$\epsmin(F_n, G_m, \alpha) = \varepsilon_{W_2}(F_n, G_m) - \sqrt{\frac{n+m}{nm}}\hat{\sigma}_{n, m}\Phi^{-1}(\alpha)$}\\

        \Return $\epsmin(F_n, G_m, \alpha)$
    \end{algorithmic}

\end{algorithm}

This section details the Python implementation of the ASO test in \cref{sec:aso}.
The full algorithm to compute the $\epsmin$ score is given in \cref{alg:aso}, and will now be explained in full detail. 
We show how the violation ratio in \cref{eq:epsilon-dist} can be compute in Python:

\begin{minted}[breaklines, breakafter=d, fontfamily=tt, fontsize=\footnotesize, style=sas, bgcolor=bg, autogobble=true]{python}
def compute_violation_ratio(scores_a: np.array, scores_b: np.array, dt: float) -> float:
    quantile_func_a = get_quantile_function(scores_a)
    quantile_func_b = get_quantile_function(scores_b)

    t = np.arange(dt, 1, dt)  # Points we integrate over
    f = quantile_func_a(t)  # F-1(t)
    g = quantile_func_b(t)  # G-1(t)
    diff = g - f
    squared_wasserstein_dist = np.sum(diff ** 2 * dt)

    # Now only consider points where stochastic order is being violated and set the rest to 0
    diff[f >= g] = 0
    int_violation_set = np.sum(diff[1:] ** 2 * dt)  # Ignore t = 0 since t in (0, 1)

    violation_ratio = int_violation_set / squared_wasserstein_dist

    return violation_ratio
\end{minted}

We can see that the integration over the violation set $\mathbb{V}_X$ in \cref{eq:epsilon-dist} is being performed by masking out values for which the stochastic order\index{Stochastic order} is honored (i.e.\@ where $F_n^{-1}(t) \ge G_n^{-1}(t)$). 
Computing the violation ratio involves building the empirical inverse cumulative distribution function\index{Cumulative distribution function!Inverse}\index{Cumulative distribution function!Empirical} or empirical quantile function, the same method as in \citet{dror2019deep} is used, with the corresponding Python code given below:

\begin{minted}[breaklines, breakafter=d, fontfamily=tt, fontsize=\footnotesize, style=sas, bgcolor=bg, autogobble=true]{python}
def get_quantile_function(scores: np.array) -> Callable:
    def _quantile_function(p: float) -> float:
        cdf = np.sort(scores)
        num = len(scores)
        index = int(np.ceil(num * p))

        return cdf[np.clip(index - 1, 0, num - 1))]

    return np.vectorize(_quantile_function)
\end{minted}

This function is also used inside the bootstrap\index{Bootstrap} sampling procedure, the last missing part of the implementation. 
We again follow the implementation by \citet{dror2019deep} and employ the inverse transform sampling procedure, in which we draw $p \sim \mathcal{U}[0, 1]$ and run it through a quantile function to create a sample.

\section{Hyperparameters Search}

Here we list the hyperparameter search procedures, ranges and found values for the different experiments in this thesis, in the order of appearance.

\subsection{Iris Example}\label{app:iris-code-details}

%The code used to produce \cref{fig:iris-example} is available online.\footnote{
%    The code is available under \url{https://github.com/Kaleidophon/evidential-deep-learning-survey}.
%} 
This section describes the details for the Iris dataset example in \cref{fig:iris-example} in \cref{sec:evidential-neural-networks}.
All models use three layers with $100$ hidden units and ReLU activations each. 
We furthermore optimized all of the models with a learning rate of $0.001$ using the Adam optimizer \citep{kingma2015adam} with its default parameter settings. 
We also regularize the ensemble\index{Ensembling} and MC Dropout\index{Dropout!Monte Carlo} model with a dropout\index{Dropout} probability of $0.1$ each.

\paragraph{Prior Network Specifics.}\index{Prior network} 
 We choose the expected $l_2$ loss by \citet{sensoy2019evidential} and regularize the network using the KL divergence\index{Kullback-Leibler divergence} w.r.t.\@ to a uniform Dirichlet\index{Dirichlet distribution} as in \citet{sensoy2019evidential}. 
 In the regularization term, we do not use the original concentration parameters $\balpha$, but a version in which the concentration of the parameter $\alpha_k$ corresponding to the correct class is removed using a one-hot label encoding $\by$ by $\tilde{\balpha} = (1 - \balpha) \circ \balpha + \by \circ \balpha$.
 The regularization term is added to the loss using a weighting factor of $0.05$.

\subsection{Synthetic Data Experiments}\label{app:synthetic-data-experiments}

This sections gives more details on the synthetic data experiments in \cref{sec:synthetic-experiments}.
We perform our experiments on the half-moons dataset, using the corresponding function to generate the dataset in \texttt{scikit-learn} \citep{pedregosa2011scikit}, producing $500$ samples for training and $250$ samples for validation using a noise level of $.125$.
We do hyperparameter search using the ranges listed in \cref{tab:hyperparameters-search-space}, settling on the values given in \cref{tab:best_hyperparameters} after $200$ evaluation runs per model (for neural networks and MC dropout; the hyperparameters found for neural networks were then used for Platt scaling\index{Platt scaling}, anchored ensembles and neural ensembles\index{Ensembling} as well). 
We also performed a similar hyperparameter search for the Bayes-by-backprop\index{Bayes-by-backprop} \citep{blundell2015weight} model, which seemed to not have yielded a suitable configuration even after extensive search, which is why results were omitted here. 
All models were trained with a batch size of $64$ and for $20$ epochs at most using early stopping with a patience of $5$ epochs and the Adam optimizer. 

\begin{table}[htb]
    \centering
    \resizebox{0.6\textwidth}{!}{%
        \begin{tabular}{rrl}
            \toprule
            Model & Hyperparameter & Value  \\
            \midrule
            Neural Network & Hidden size & $[25, 25, 25]$  \\
            Neural Network & Dropout prob. & $.014552$ \\
            Neural Network & Learning rate & $.000538$  \\
            MC Dropout & Hidden sizes & $[25, 25, 25, 25]$  \\
            MC Dropout & Dropout prob. & $.205046$  \\
            MC Dropout & Learning rate & $.000526$\\
            %\texttt{BBB} & \texttt{hidden\_sizes} & $[25]$ \\
            %\texttt{BBB} & \texttt{dropout\_rate} & $0.267797$ \\
            %\texttt{BBB} & \texttt{lr} & $0.00689$  \\
            %\texttt{BBB} & \texttt{posterior\_mu\_init} & $0.284989$ \\
            %\texttt{BBB} & \texttt{posterior\_rho\_init} & $-7.649002$  \\
            %\texttt{BBB} & \texttt{prior\_pi} & $0.196422$  \\
            %\texttt{BBB} & \texttt{prior\_sigma\_1} & $0.606531$  \\
            %\texttt{BBB} & \texttt{prior\_sigma\_2} & $0.904837$  \\
            \bottomrule
        \end{tabular}% 
    }
    \caption{Best hyperparameters found on the half-moon dataset.}
    \label{tab:best_hyperparameters}
\end{table}

\begin{table}[htb]
    \centering
    \resizebox{0.625\textwidth}{!}{%
        \begin{tabular}{rrl}
            \toprule
            Hyperparameter & Chosen from \\
            \midrule
            Hidden layers & 1--5 layers of $15$, $20$, $25$ \\
            Learning rate & $\mathcal{U}(\log 10^{-4}, \log 0.1)$ \\
            Dropout rate & $\mathcal{U}(0.1, 0.5)$ \\
            \bottomrule
        \end{tabular}%
    }
    \caption{Distributions or options that hyperparameters were sampled from during the random hyperparameter search.} 
    %$^{(*)}$For more information about these hyperparameters the reader is referred to the work of \cite{blundell2015weight}.}
     \label{tab:hyperparameters-search-space}
\end{table}

\subsection{Text Classification Experiments}\label{app:predictive-uncertainty-training-details}

Here, we detail the hyperparameter search conditions for the experiments in \cref{sec:predictive-uncertainty-experiments}.
We perform hyperparameter search using random sampling \citep{bergstra2012random} using hyperband scheduling \citep{li2017hyperband}\footnote{
    Trials might be terminated using hyperband after $10k$ steps.
} on the entire training set, even if models are trained on sub-sampled training sets later.
This has the advantage of ensuring comparability between runs and eliminating suboptimal hyperparameter choices as a source of worse uncertainty estimation. 
We do $80$ trials for LSTM-based\index{Long-short term memory network} models, and $30$ for Bert-based\index{Bert} models. 
Furthermore, the hyperparameters for the LSTM are identical for the LSTM ensemble\index{Ensembling}  (10 instances are used per ensemble). 
Hyperparameters were picked by best final validation loss over search trials.

\paragraph{Chosen Hyperparameters.} 
We summarize some common hyperparameters here and show the rest in \cref{table:hyperparameter-values}. 
We commonly use a batch size of $32$, and sequence lengths of $35$ for LSTM-based\index{Long-short term memory network}  and $128$ for Bert-based\index{Bert}  models. 
All LSTM-based models are trained using $2$ layers, with the exception of the vanilla LSTM and the LSTM-ensemble on Clinc Plus with $3$ layers. 
Their hidden size and embedding sizes are set to $650$. 
For all models, gradient clipping is set to $10$. 
For models using multiple predictions to compute uncertainty estimates, $10$ predictions are used at a time. 

\begin{table}[htb]
    \centering 
    \resizebox{0.785\textwidth}{!}{
        \renewcommand{\arraystretch}{1.5}
        \begin{tabular}{@{}lrr@{}}
            \toprule
            Name & Tuned for & Search space \\
            \midrule
            Learning rate & \makecell[tr]{LSTM, LSTM Ensemble,\\ Bayesian LSTM, ST-$\tau$ LSTM\\ Variational LSTM} & $\mathcal{U}(0.1, 0.5)$ \\ 
            Learning rate & \makecell[tr]{DDU BERT, SNGP BERT,\\ Variational BERT} & $\log \mathcal{U}(10^{-5},10^{-3})$ \\ 
            Spectral norm upper bound & DDU BERT, SNGP BERT & $\mathcal{U}(0.95, 0.99)$ \\
            Kernel amplitude & SNGP BERT & $\log \mathcal{U}(0.01, 0.5)$ \\
            $\beta$ weight decay & SNGP BERT & $\log \mathcal{U}(10^{-3}, 0.5)$\\
            Weight decay & \makecell[tr]{LSTM, LSTM Ensemble,\\ ST-$\tau$ LSTM, Variational BERT}  & $\mathcal{U}(0.1, 0.5)$ \\ 
            Layers & LSTM, LSTM Ensemble & $\{2, 3\}$  \\ 
            Dropout & \makecell[tr]{LSTM, LSTM Ensemble,\\ ST-$\tau$ LSTM, Variational BERT} & $\mathcal{U}(0.1, 0.4)$   \\ 
            Layer Dropout & Variational LSTM & $\mathcal{U}(0.1, 0.4)$   \\ 
            Time Dropout & Variational LSTM & $\mathcal{U}(0.1, 0.4)$   \\ 
            Embedding Dropout & Variational LSTM & $\mathcal{U}(0.1, 0.4)$   \\ 
            Hidden size & LSTM, LSTM Ensemble & $\{350, 500, 650\}$ \\ 
            Prior $\sigma_1$ & Bayesian LSTM & $\log \mathcal{U}(-0.8, 0.1)$ \\ 
            Prior $\sigma_2$ & Bayesian LSTM & $\log \mathcal{U}(-0.8, 0.1)$ \\ 
            Prior $\pi$ & Bayesian LSTM & $\log \mathcal{U}(0.1, 0.9)$ \\ 
            Posterior $\mu$ init & Bayesian LSTM & $\mathcal{U}(-0.6, 0.6)$ \\ 
            Posterior $\rho$ init & Bayesian LSTM & $\mathcal{U}(-8, -2)$ \\
            Init weight & LSTM & $\mathcal{U}(0.1, 0.4)$ \\
            Number of centroids & ST-$\tau$ LSTM & $\{5, 10, 20, 30, 40\}$  \\
            \bottomrule
        \end{tabular}%
    }\caption[List of searched hyperparameters for the text classification experiments.]{List of searched hyperparameters. LSTM Ensemble hyperparameters are not searched, but simply copied from the found LSTM hyperparameters.} \label{table:hyperparameter-ranges}
\end{table}

\begin{table}[htb]
    \centering 
    \resizebox{0.675\textwidth}{!}{
        \renewcommand{\arraystretch}{1.5}
        \begin{tabular}{@{}ll|rrr@{}}
            \toprule
            Model & Hyperparameter & English & Danish & Finnish \\
            \midrule
            LSTM & Weight decay & $ .001337$ & $.001357$ & $.001204$ \\
            & Learning rate & $.4712$ & $.4931$ & $.2205$ \\
            & Init.\@ weight & $ .2830$ & $.5848$ & $.5848$ \\
            & Dropout & $.3379$ & $.2230$ & $.1392$ \\
            Variational LSTM & Weight decay & -- & $10^{-7}$ & $.01953$ \\
            & Learning rate & -- & $.3031$ & $.7817$ \\
            & Init.\@ weight & -- & $.1097$ & $.5848$ \\
            & Embedding Dropout & -- & $.1207$ & $.3566$ \\
            & Layer Dropout & -- & $.1594$ & $.3923$ \\
            & Time Dropout & -- & $.1281$ & $.1646$ \\
            Bayesian LSTM & Weight decay & $.001341$ & $.003016$ & $.03229$ \\
            & Learning rate & $.1704$ & $01114$ & $.1549$ \\
            & Dropout & $.3410$ & $.3868$ & $.331$ \\
            & Prior $\sigma_1$ & $.9851$ & $.7664$ & $.3246$ \\ 
            & Prior $\sigma_2$ & $.5302$ & $.851$ & $.5601$ \\ 
            & Prior $\pi$ & $1$ & $1$ & $.1189$ \\ 
            & Posterior $\mu$ init & $-.005537$ & $-.0425$ & $.4834$\\ 
            & Posterior $\rho$ init & $-7$ & $ -6$ & $.1124$ \\
            ST-$\tau$ LSTM & Weight decay & -- & $.001189$ & $.0007857$ \\
            & Learning rate & -- & $.01979$ & $.3601$ \\
            & Dropout & -- & $.1867$ & $.1737$ \\
            & Num.\@ centroids & -- & $5$ & $30$ \\
            DDU Bert & Learning Rate & $.003077$ & $.00006168$ & $.001825$ \\
             & Spectral norm upper bound & $.9753$ & $.9211$ & $.9410$ \\
             & Weight decay & $.0039=0$ & $.1868$ & $.09439$ \\
            Variational BERT & Learning Rate & $.0002981$ & $.00009742$ & $.00003483$ \\
             & Weight decay & $.01591$ & $.02731$ & $.09927$ \\
             & Dropout & $.2382$ & $.4362$ & $.4364$ \\
            SNGP Bert & Learning Rate & -- & $.0002332$ & $.0002919$ \\
             & Spectral norm upper bound & -- & $.99$ & $.96$ \\
             & Beta Weight decay & -- & $.001619$ & $.002438$ \\
             & Beta length scale & -- & $2.467$ & $2.254$ \\
             & Kernel amplitude & -- & $.3708$ & $.2466$ \\
            \bottomrule
        \end{tabular}%
    }
    \caption[List of model hyperparameters by dataset for the text classification experiments.]{List of used model hyperparameters by dataset.} \label{table:hyperparameter-values}
\end{table}

\index{Classification!Text} 
We further include some notes about the optimization of models for the experiments in \cref{sec:predictive-uncertainty-experiments}.
To make sure that all models are trained for the same number of steps regardless of the the size of (sub-sampled) training set, we set the training duration to the number of steps corresponding to a number of epochs using the original training set size, and name it \emph{epoch-equivalents} in the following. 
Due to the imbalance of classes in Finnish UD and Dan+, all models were trained using loss-weights that are inverse to the frequency of a label in the dataset.

\paragraph{Optimization of LSTMs.} \index{Long-short term memory network} 
We adopt different optimization schemes for transformer-\index{Transformer}  and LSTM-based models. 
For LSTMs, we choose stochastic gradient descent with a decaying learning rate schedule, decaying by $.8695$ after the equivalent of $14$ epochs for every following epoch-equivalent for $55$ epoch-equivalents in total. 
This corresponds to the setup in \citet{gal2016theoretically}, modified from the setup in \citet{zaremba2014recurrent}.

\paragraph{Optimization of Berts.} \index{Bert} 
We fine-tune Bert models using the shorter duration of $20$ epoch-equivalents, corresponding to the NLP\index{Natural language processing}  experiments in \citet{liu2022simple}. 
Adam \citep{kingma2015adam} is used for optimization with default parameters $\beta_1 = .9$ and $\beta_2 = .999$ alongside a triangular learning rate, using the first $10 \%$ of the training duration as warm-up.

\subsection{Auxiliary Calibrator Experiments}\label{app:hyperparameters}

This sections explains the hyperparameter tuning for the experiments in \cref{sec:calibration-experiments}.
We conduct suites of hyperparameter searches per target LLM,\index{Large language model} dataset and type of calibration\index{Calibration}  targets (binary and clustering) corresponding to the results in \cref{tab:calibration-results}, resulting in eight different suites. 
We then use these found hyperparameters for the results in \cref{tab:calibration-results-features}.

\begin{table*}[htb]
    \centering
    \renewcommand{\arraystretch}{1.5}
    \resizebox{0.675\textwidth}{!}{
    \begin{tabular}{rrllll}
        \toprule
        & & \multicolumn{2}{c}{TriviaQA} & \multicolumn{2}{c}{CoQA} \\
        \cmidrule(lr){3-4} \cmidrule(lr){5-6}
        & & Binary & Clustering & Binary & Clustering \\
        \midrule
        \multirow{2}{*}{Vicuna v1.5} & learning rate & $1.4 \times 10^{-5}$ & $3.37 \times 10^{-5}$ & $9.58 \times 10^{-5}$ & $8.84 \times 10^{-5}$ \\
        & weight decay & $.03184$ & $.008936$ & $.005793$ & $7.42 \times 10^{-4}$  \\
        \midrule
        \multirow{2}{*}{GPT-3.5}  & learning rate & $2.96 \times 10^{-5}$ & $1.62 \times 10^{-5}$ & $5.12 \times 10^{-5}$ & $5.59 \times 10^{-5}$ \\
        & weight decay & $.01932$ & $.01362$ & $.03327$ & $.03495$ \\
        \bottomrule
    \end{tabular}%
    }
    \caption{Chosen hyperparameters for our model on different datasets and for different calibration targets.}\label{tab:chosen-hyperparameters}
\end{table*}

\paragraph{Search Method and Ranges.} For the search, we opt for Bayesian hyperparameter search \citep{snoek2012practical} as implemented by Weights \& Biases \citep{wandb}.
We optimize only two hyperparameters: Learning rate and weight decay.
The learning rate is samples from a log-uniform distribution\index{Log-uniform distribution}  $\log\ \mathcal{U}[10^{-5}, 0.01]$ and weight decay from a uniform distribution\index{Uniform distribution}  $\mathcal{U}[10^{-4}, 0.05]$ for a total of $50$ runs and $250$ training steps each.
The final hyperparameters selected are given in \cref{tab:chosen-hyperparameters}.

\paragraph{Other Hyperparameters.} When obtaining the responses from Vicuna v1.5 7B, we use a batch size of $4$ and generate for a maximum of $50$ tokens and stop generation when the model tries to generate parts of the prompt, such as ``Question:'' / ``Q:'' or ``Answer:'' / ``A:''.
We also use $10$ in-context samples\index{In-context learning} for TriviaQA, but no in-context samples for CoQA. 
For the auxiliary calibrator, we use a context size of $512$ tokens, batch size of $32$, and gradient clipping with a maximum norm of $10$.

 \section{Pre-processing for Text Classification Benchmark}\label{app:predictive-uncertainty-pre-processing}

 This sections explains the data preprocessing for the datasets used in \cref{sec:predictive-uncertainty-experiments}.

 \paragraph{Tokenization.} 
 We use the corresponding Bert\index{Bert} tokenizer for each language, including for LSTM-based\index{Long-short term memory network} models to ensure compatibility. 
 For English, this corresponds to the original SentencePiece tokenizer used by \citet{devlin2019bert}, while we use the tokenizer of the Danish Bert \citep{hvingelby2020dane} and Finnish Bert \citep{virtanen2019multilingual} for those languages, respectively.
 
 \paragraph{Tags for Sub-Word Tokens.} 
 For named entity recognition and part-of-speech tagging\index{Part-of-speech tagging}, we follow \citet{jurafsky2022speech}, chapter 11.3.3 to deal with sub-word tokens: 
 For every token that is split into sub-word tokens, we assign the tag only to the first sub-word token, and $-100$ for the rest, which ignores them for evaluation purposes.

\section{Implementation Details of Text Classification Benchmark}\label{app:exploring-predictive-uncertainty-implementation-details}

This section gives additional implementation details for the models used in the text classification\index{Classification!Text}  benchmark in \cref{sec:predictive-uncertainty-experiments}.

\paragraph{Resources.} 
In addition to the resources in \cref{app:open-source}, the Bayesian LSTM\index{Long-short term memory network!Bayesian}  was developed using the \texttt{Blitz} package \citep{esposito2020blitzbdl} for PyTorch and the SNGP transformer\index{Transformer!SNGP}  using \texttt{gpytorch} \citep{gardner2018gpytorch}.

\paragraph{Models.} 
For the DUE transformer\index{Transformer!DUE}, we used principal component analysis\index{Principal component analysis}  on the latent representations for Clinc Plus to reduce the memory usage of the Gaussian discriminant analysis\index{Gaussian discriminant analysis} by reducing dimensionality to $64$. 
We initially also experimented with the usage of the DUE transformer\index{Transformer!DUE}  by \citet{van2021feature}, however found that it was not trivial to create the inducing points for the Gaussian process output layer in a sequential setting. 
For the variational transformer \index{Transformer!Variational} \citep{xiao2020wat}, the authors do not specify exactly how MC dropout\index{Dropout!Monte Carlo} is used. 
We use the existing dropout\index{Dropout}  layers in the corresponding model, and use a number of forward passes with different dropout masks to make predictions. 
Since the number of classes is prohibitive for the original formulation of the SNGP transformer\index{Transformer!SNGP}, we use the extension proposed by \citet{liu2022simple} in Appendix A.1 and only store one $\hat{\Sigma}^{-1}$ matrix for all classes. 
Furthermore, we update the matrix continuously during training and not just during the last epoch, in order to allow tracking of the predictive performance over the training time. 
Lastly, we also evaluate predictions using Monte Carlo approximations instead of the mean-field approach, since this allows us to compute a wider variety of uncertainty metrics\index{Uncertainty metric}.

\paragraph{Evaluation.} 
When computing uncertainty estimates and losses for evaluation purposes, the measurements for a number of tokens were discarded. 
These include the ignore token with ID $-100$, as well as the IDs corresponding to the \texttt{[EOS]}, \texttt{[SEP]}, \texttt{[CLS]} and \texttt{[PAD]} token, which might differ between tokenizers of different languages. 
For computing the ECE\index{Expected calibration error}, we use $10$ bins. %, and $10$ value ranges for ACE. 

\paragraph{Model Comparison.} 
We facilitate the comparison of models using the almost stochastic order\index{Stochastic order!Almost}  test from \cref{sec:aso}.
We use the test with a confidence level $\alpha = 0.05$ and a decision threshold of $\tau = 0.3$.

 \section{Convergence on Clinc Plus}\label{app:convergence-clinc-plus}

 Here, we briefly address the models missing from the English Clinc Plus experiments in \cref{sec:predictive-uncertainty-experiments}. 
 For the ST-$\tau$\index{Long-short term memory network!ST-$\tau$}  and variational LSTM\index{Long-short term memory network!Variational}, we could not identify clear reasons on why models did not converge. 
 Even after extensive hyperparameter searches and manual fine-tuning of hyperparameters (including different learning rate schedules and optimizers), we did not find a combination of options that resulted in convergence. 
 We also observed strange behavior for the Bayesian LSTM\index{Long-short term memory network!Bayesian}, which, after reaching a validation accuracy of $0.5$, would suddenly return to its initial training performance. 
 This could potentially be explained by the model accidentally escaping a low-loss basin due to a learning rate that is still too high, and thus we changed the model to only be trained for $18$ epoch-equivalents and initiate the learning rate decay after seven epoch-equivalents. 
 The puzzling fact is that SNGP Bert\index{Bert}\index{Transformer!SNGP} did not converge on Clinc Plus, since the authors successfully used the dataset in their own work \citep{liu2022simple}. 
 We put forth the following explanations: 
 First of all, we observed the model to generally possess a high variance, as demonstrated by the standard deviation on the Danish and Finnish data. 
 Secondly, we make at least two changes to their implementation: 
 Instead of using the mean-field approximation to the predictive distribution, we use the Monte Carlo approximation in order to compute metrics such as mutual information. 
 Also, we update the covariance matrix $\hat{\Sigma}$ over the whole training time in order to track the predictive performance for our experiments, and not just during the last epoch.
 
 \section{Temperature Search}\label{app:temperature-search}

This sections explains the temperature search procedure for the parameter $\tau$ in \cref{eq:weight-equation} in \cref{sec:nonex-nlg-method} further.
To determine the temperaturein \cref{eq:weight-equation} for the different distance metrics in \cref{tab:coverage-results-mt}, we adopt a variation of a simple hill-climbing procedure. 
Given user-defined bounds for the temperature search $\tau_\text{min}$ and $\tau_\text{max}$, we sample an initial candidate $\tau_0 \sim \mathcal{U}[\tau_\text{min}, \tau_\text{max}]$, and then evaluate the coverage of the method given the candidate on the first $100$ batches of the calibration dataset. 
The next candidate then is obtained via 

\begin{align}
    \tau_{t+1} & = \tau_{t} + \eta \cdot \varepsilon \cdot \text{sgn}\big(1 - \alpha - \text{Coverage}(\tau_t)\big); \nonumber \\
    \varepsilon & \sim \mathcal{N}(0, \tau_\text{max} - \tau_\text{min}), 
\end{align}

\noindent where $\eta$ is a predefined step size (in our case $0.1$) and $\text{Coverage}(\tau_t)$ the achieved coverage given a candidate $\tau_t$.
The final temperature is picked after a fixed number of steps ($t=20$ in our work) based on the smallest difference between achieved and desired coverage.\\

Overall, we found useful search ranges to differ greatly between experimental settings, as illustrated by the reported values in \cref{tab:coverage-results-mt} and \cref{tab:coverage-results-lm}.
In general, the stochastic hill-climbing could also be replaced by a grid search, even though we sometimes found the best temperature to be ``hidden'' in a very specific value range.
It also has to be noted that temperature for the $l_2$ distance is the highest by far since FAISS returns \emph{squared} $l_2$ distances by default.

%\begin{tikzpicture}[remember picture,overlay]
%    \node[anchor=north,inner sep=0pt] at (current page text area.north) {\includegraphics[width=\linewidth, clip=true, trim = 8cm 35cm 8cm 90cm]{img/covers/flux_equilibrium_4k_8.png}};
%\end{tikzpicture}

\chapter*{Abbreviations}
\addcontentsline{toc}{chapter}{Abbreviations}
% Overwrite these commands to have a custom chapter heading that also appears in the TOC
\renewcommand*{\glsclearpage}{}
\printglossary[title={}]

%\chapter*{Index}
%\addcontentsline{toc}{chapter}{Index}
\printindex

\end{document}